\documentclass[a4paper]{ociamthesis}
\usepackage[utf8]{inputenc}
\usepackage{amsthm, amsmath, amsfonts}
\usepackage{wrapfig}
\usepackage{bm}
\usepackage{thm-restate}

\usepackage{etoolbox}

\AtBeginEnvironment{subappendices}{%
\clearpage
  \addcontentsline{toc}{section}{Appendices}
  \counterwithin{figure}{section}
  \counterwithin{table}{section}
}

\newtheorem{definition}{Definition}[section]
\newtheorem{assumption}{Assumption}
\newtheorem{theorem}{Theorem}[section]
\newtheorem{proposition}[theorem]{Proposition}
\newtheorem{corollary}[theorem]{Corollary}
\newtheorem{lemma}[theorem]{Lemma}
\usepackage{natbib}

\newcommand{\ginv}{$\grp$-invariant}
\newcommand{\bbR}{\mathbb{R}} 
\newcommand{\bbE}{\mathbb{E}} 
\newcommand{\variance}{\text{\rm{Var}}}
\newcommand{\Var}{\variance}
\newcommand{\calX}{\mathcal{X}} 
\newcommand{\calY}{\mathcal{Y}} 
\newcommand{\symm}{S} 
\newcommand{\T}{\mathcal{T}}  
\newcommand{\grp}{\mathcal{G}} 
\newcommand{\calE}{\mathcal{E}}
\newcommand{\scD}{\mathscr{D}}
\newcommand{\calL}{\mathcal{L}}
\newcommand{\data}{\mathcal{D}}
\newcommand{\equdist}{\overset{D}{=}}
\newcommand{\bX}{\mathbf{X}} 
\newcommand{\calD}{\mathcal{D}}
\newcommand{\dgd}{P_{\calD}} 
\newcommand{\bw}{\mathbf{w}}

\newcommand{\haar}{\lambda}
\newcommand{\trdata}{\mathcal{D}^n}
\newcommand{\fclass}{F} 
\newcommand{\invf}[1]{#1^{\circ}} 

\newcommand{\borel}{\mathcal{B}} 
\newcommand{\indicator}{\mathds{\indicator}}
\newcommand{\Orbit}{\Phi}

\newcommand{\argdot}{{\,\vcenter{\hbox{\tiny$\bullet$}}\,}}

\def\tmu{\tilde{\mu}}
\def\tnu{\tilde{\nu}}

\newcommand{\loss}{\ell} 
\newcommand{\risk}{R_{\loss}} 
\newcommand{\eRisk}{\widehat{R}_{\loss}} 
\newcommand{\eRiskAug}{\invf{\widehat{R}}_{\loss}} 

\newcommand{\eRiskAugMC}{\widehat{R}_{\loss}^{\widehat{\circ}}} 
\newcommand{\KL}[2]{\text{\rm{KL}}(#1 \; || \; #2)} 
\newcommand{\invfMC}[1]{#1^{\widehat{\circ}}}

\newcommand{\statespace}{\mathcal{X}}
\newcommand{\actionspace}{\mathcal{A}}
\newcommand{\mdp}{\mathcal{M}}

\newcommand{\repdim}{K}
\newcommand{\repix}{k}
\newcommand{\qrdqnlosstau}{\mathcal{L}_{\kappa}(\hat{\bm{\tau}})}

\newcommand{\updatematrix}{\mathbf{I}_{\Delta}}
\newcommand{\xtest}{x_{\mathrm{test}}}

\newcommand{\Xtest}{X_{\mathrm{test}}}
\newcommand{\Xtrain}{X_{\mathrm{train}}}
\newcommand{\train}{{\mathrm{train}}}
\newcommand{\test}{{\mathrm{test}}}
\newcommand{\ttheta}{\tilde{\theta}}
\newcommand{\TD}{\mathrm{TD}}

\newcommand{\argmax}{\mathrm{arg max}}
\newcommand{\argmin}{\mathrm{arg min}}
\usepackage{mathrsfs}

\newcommand{\envs}{\mathcal{E}}
\newcommand{\states}{\mathcal{X}}

\newcommand{\cP}{\mathcal{P}}
\newcommand{\cX}{\mathcal{X}}

\newcommand{\cA}{\mathcal{A}}
\newcommand{\model}{\mathcal{M}}

\newcommand{\effdim}{\text{feature rank}\xspace} 
\newcommand{\Effdim}{\text{Feature rank}\xspace} 
 
\newcommand{\infer}{\pyoi}
\newcommand{\pyoi}{\text{InFeR}\xspace} 

\newcommand{\gradint}{\mathrm{I}_{\nabla}}
\newcommand{\deltaint}{\mathrm{I}_{\Delta}}

\newcommand{\bx}{\mathbf{x}}
\newcommand{\by}{\mathbf{y}}
\newcommand{\btheta}{\bm{\theta}}

\newcommand{\keyinsight}[1]{

\mdfsetup{%
backgroundcolor=Cerulean!3,
linecolor=NavyBlue,
linewidth=3pt}
\begin{mdframed}
\begin{minipage}[t]{\linewidth}
{\color{NavyBlue}\textbf{Key insight.}}\\
{#1}
\end{minipage}
\end{mdframed}
}
\newcounter{hyp}
\counterwithin{hyp}{section}

\newcommand*{\hypothesis}[2]{
\vspace{2mm}
\mdfsetup{%
backgroundcolor=Orange!5,
linecolor=BrickRed,
linewidth=3pt}
\begin{mdframed}
\begin{minipage}{\linewidth}
\refstepcounter{hyp}
{\color{BrickRed}\textbf{Hypothesis {\thehyp}:} }
{#2}
\end{minipage}
\end{mdframed}
}

\usepackage{cleveref}
\usepackage[framemethod=TikZ]{mdframed}
\usepackage[ruled,vlined,algosection]{algorithm2e}
\usepackage{float}
\usepackage{algorithmic}
\usepackage{bbm}
\usepackage{bibentry}
\usepackage{chngcntr}
\usepackage{etoolbox}
\usepackage{lipsum}

\AtBeginEnvironment{subappendices}{%
\chapter*{Appendix}
\addcontentsline{toc}{chapter}{Appendices}
\counterwithin{figure}{section}
\counterwithin{table}{section}
}

\title{Generalization Through the Lens of Learning Dynamics}
\author{Clare Lyle}
\college{University College}
\degree{Doctor of Philosophy}

\usepackage[toc]{glossaries}

\begin{document}
\setlength{\textbaselineskip}{22pt plus2pt}

\setlength{\frontmatterbaselineskip}{17pt plus1pt minus1pt}

\setlength{\baselineskip}{\textbaselineskip}

\maketitle
\thispagestyle{empty}
\begin{romanpages}
\begin{abstract}
    A machine learning (ML) system must learn not only to match the output of a target function on a training set, but also to generalize to novel situations in order to yield accurate predictions at deployment. In most practical applications, the user cannot exhaustively enumerate every possible input to the model; strong generalization performance is therefore crucial to the development of ML systems which are performant and reliable enough to be deployed in the real world. While generalization is well-understood theoretically in a number of hypothesis classes, the impressive generalization performance of deep neural networks has stymied theoreticians. In deep reinforcement learning (RL), our understanding of generalization is further complicated by the conflict between generalization and stability in widely-used RL algorithms. This thesis will provide insight into generalization by studying the learning dynamics of deep neural networks in both supervised and reinforcement learning tasks. 

We begin with a study of generalization in supervised learning. We propose new PAC-Bayes generalization bounds for invariant models and for models trained with data augmentation. We go on to consider more general forms of inductive bias, connecting a notion of training speed with Bayesian model selection. This connection yields a family of marginal likelihood estimators which require only sampled losses from an iterative gradient descent trajectory, and analogous performance estimators for neural networks.
We then turn our attention to reinforcement learning, laying out the learning dynamics framework for the RL setting which will be leveraged throughout the remainder of the thesis. We identify a new phenomenon which we term capacity loss, whereby neural networks lose their ability to adapt to new target functions over the course of training in deep RL problems, for which we propose a novel regularization approach. Follow-up analysis studying more subtle forms of capacity loss reveals that deep RL agents are prone to memorization due to the unstructured form of early prediction targets, and highlights a solution in the form of distillation. We conclude by calling back to a different notion of invariance to that which started this thesis, presenting a novel representation learning method which promotes invariance to spurious factors of variation in the environment. 

\end{abstract}

\begin{acknowledgements}
First and foremost, this thesis would not have been possible without the invaluable guidance and support of my supervisors, Yarin Gal and Marta Kwiatkowska. Their advice, mentorship, and insight over the past almost four years has been vital to my development as a scientist. I am deeply grateful to both for being so incredibly generous with their time, particularly in the early years of the DPhil, and for giving me the freedom to set my own research direction and explore a broad range of topics.  
Thanks are also due to Marc Bellemare, Pablo Samuel Castro, and Prakash Panangaden for their mentorship prior to the start of my DPhil. I learned many invaluable lessons about how to do good research from their example, which have served me well throughout my DPhil. 

I have had the good fortune to engage in a number of collaborations while at Oxford. Many thanks go to Amy Zhang, with whom I wrote the first published paper of my DPhil, for modelling how to organize and execute a research project. I am also indebted to Mark Rowland and Will Dabney, who have been a joy to work with on many of the papers that appear in this document. This thesis would not have been possible without valuable discussions and collaborations with Benjamin  Bloem-Reddy, Mark van der Wilk, Benjie Wang, Angelos Filos, Natasha Jaques, Greg Farquhuar, Lisa Schut, Robin Ru, Aidan Gomez, Lorenz Kuhn, Jannick Kossen, Neil Band, Georg Ostrovski, Shagun Sodhani, and Andreas Kirsch.

Beyond direct collaborations, I also benefited immensely from being part of the broader machine learning community at Oxford. My perspective as a researcher was enriched by having a group of brilliant people with whom I could run over nascent ideas, debug a tricky experiment, and gripe about Reviewer 2. Thanks go in particular to Joost, Milad, Seb, Panos, Tim, Jan, Freddie, Pascal, Matt, Pascale, Rhiannon, Luca, Andrea, Emi, Michael, Sahra, Charline, and Jakob. Thanks as well go to the many people I interacted with at DeepMind, including Daniel, Dave, Remi, Diana, Anna, Bilal, Bernardo, and Mo.

Finally, I would be remiss to omit my family, whose unconditional support has provided the solid foundation on which I have been able to take risks and and explore as a researcher. Thanks go as well to the friends I've made in Oxford: Alex, Caitlin, Nayani, Colin, Matt, Rahul, Imi, Becca, and my teammates from OUBbC and UCBC. Thanks in particular to the individual who did not want to be named in this document for challenging me to address meaningful problems and for proofreading many chapters of this thesis.

\end{acknowledgements}

\setcounter{tocdepth}{1}
\dominitoc 
\tableofcontents
\listoffigures 
\listoftables
\end{romanpages}

\dominitoc 
\adjustmtc 
\adjustmtc
\chapter{Introduction}
\label{chp:introduction}

\minitoc

\section{Learning to generalize}

The ability to generalize a lesson from the classroom to the real world is what separates \textit{learning} from \textit{memorization}. In a range of tasks ranging from mathematics to language, humans are remarkably skilled at identifying abstract patterns, and applying these patterns to novel contexts. The reader is unlikely to have previously encountered the sentence `the armadillo tipped its blue hat and returned to the game of marbles', and yet would likely have no trouble interpreting it, or answering questions concerning the colour of the armadillo's hat. However, the range of domains in which humans exhibit this ability to learn and generalize is limited. A human can easily parse natural language, but will struggle to identify structure in strings of base pairs arising from a genome. In these settings, we can benefit from computational tools. Machine learning approaches, in particular the training of deep neural networks on large datasets, present a promising direction towards the development of general algorithms which can identify patterns in data and solve a range of problems. 

In a typical machine learning pipeline, the practitioner collects data (a \textit{training set}) which is then fed into a machine learning algorithm, with the hope that a system which can accurately model this data will have captured the underlying structure of the task. The training set will not encompass the set of all possible inputs a model may receive at deployment; a learned predictor must \textit{extrapolate} from the data it was trained on in order to make useful predictions. Generalization is crucial both to obtain good performance and to ensure the safety and reliability of these systems when they encounter data that was not seen during training. Yet how do we ensure that these highly expressive systems are {learning} and not {memorizing}? This question is a central concern of a long line of literature, and of this thesis.

\subsection{Defining generalization}

The machine learning community broadly distinguishes between two classes of generalization: within-distribution generalization, and out of distribution (OOD) generalization. While both types of generalization concern the performance of a predictor on data not seen during training, they differ in their structural assumptions on how this data is generated. Within-distribution generalization assumes that the process by which the training data was collected will also generate the data on which we will evaluate the trained model. This assumption is used in a rich theoretical literature which provides provable guarantees on the generalization performance of certain classes of learning algorithms. 
However, it is not reflected in many of the settings found in practice, where the procedure by which the training data is collected differs from how data is generated at evaluation; this difference is widely referred to as \textit{distribution shift}. 

The real world is full of nonstationarities which can induce distribution shift: a user's taste in films will evolve as they age, slang terms enter and leave common usage, and interest in advertisements for winter boots may vary with the seasons. Data collected one month may quickly fall out of date and cease to be reflective of the phenomenon being modelled within a matter of weeks. Even the ways in which data is collected may introduce biases into the training set. Many image datasets, for example, include the main subject centred nicely in the middle of the frame, signalling to a learning algorithm that the corners of an image do not contain relevant information. These nonstationarities and biases can result in large {distribution shifts} when the model is deployed. The technical difficulty of overcoming such distribution shifts will depend on their structure and magnitude. Work on developing models which are robust to distribution shifts typically must make explicit assumptions on the type of shift being considered.

\subsection{The importance of generalization}

While we often marvel at great feats of memorization, such as memory champions who can recite the order of a shuffled deck of cards after a minute's concentration or taxi drivers who can recall the street maps of large cities by heart, memorization is discouraged when it occurs in place of understanding. The difference between the two is one of generalization: memorizing multiplication tables enables quick recall, but learning how to multiply numbers together algorithmically enables generalization to previously unseen number pairs. While we hesitate to anthropomorphize a linear regression model in saying that it `understands' the relationship between inputs and outputs, we nonetheless seek out a similar phenomenon in our machine learning systems. It is not sufficient to make perfect predictions on the training data: the model must be able to apply the relationship between input and output to new contexts, and make accurate predictions in these contexts. 

This generalization is crucial if we hope to see the power of machine learning deployed in real-world settings, where mistakes on novel inputs can have catastrophic consequences. Unforeseen distribution shifts in medical data, such as replacing an imaging device in a hospital, can result in a decline in performance that has serious ramifications for patients' health. Cross-validation, the canonical approach to evaluating generalization performance, will not identify the model's robustness to distribution shifts that the developer did not already foresee and include in the evaluation set. This presents a particular challenge for practitioners; even if a decline in performance is detected, the best we can hope to do is re-train or fine-tune the model to improve its performance on the new data. This can become much more expensive than if the model had learned relationships that generalized well off the bat. 

\subsection{Deep learning}
\begin{figure}
    \centering
    \includegraphics[width=13cm]{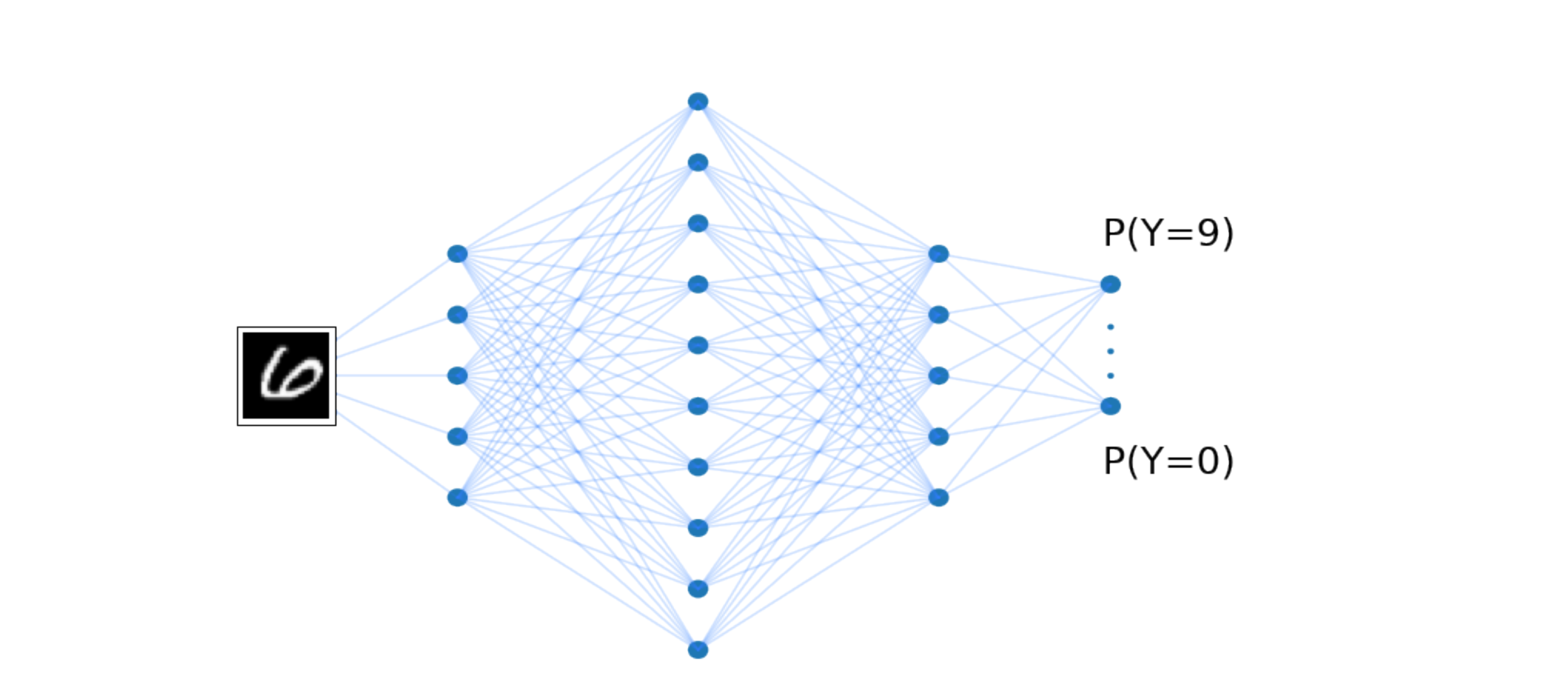}
    \caption[A visualization of a fully connected multi-layer perceptron.]{A visualization of a fully connected multi-layer perceptron. The neural network consists of a series of layers, where each layer applies a non-linear activation to a linear combination of activations from the previous layer. Each circle in this figure corresponds to an activation unit, while the lines correspond to weights connecting the left hand side layer to the right. The output of the network in this case corresponds to a probability distribution indicating the likelihood that the input image corresponds to a given digit.}
    \label{fig:dnn}
\end{figure}

The deep learning revolution has spurred the growth of several international conferences, the creation of multiple industrial AI labs, and the allocation of three Turing awards. This is due to the success of deep neural networks (DNNs, see Figure~\ref{fig:dnn}) at modelling a wide range of data modalities. The applications of DNNs range from the benign, such as helping people with visual impairments navigate a street and translating text from one language to another, to those with the potential for malicious use, such as identifying human faces in surveillance footage. Strikingly, DNNs often obtain impressive generalization performance and are considered robust enough to be used in many commercial applications.

The success of these models has brought attention to the discipline, but also stymied theoreticians. Compared to the expressivity of deep neural networks, traditional learning algorithms seek to model data by searching over a relatively small class of functions. Theoretical analysis of these algorithms crucially depends on the size of the function class in order to provide guarantees on the expected error of the function found by the algorithm on new data. In contrast, the set of functions expressible by a given neural network architecture is so large as to result in vacuous results when traditional analysis is applied to deep learning. This has opened up a number of exciting approaches to study the generalization of DNNs which often have a more experimental flavour than prior work on learning theory.

\section{Learning to act}

We will be particularly interested in studying machine learning systems which are capable of \textit{doing things} in the world, rather than passively outputting predictions. This is captured by the {reinforcement learning} (RL) framework (see Figure~\ref{fig:rl_formulation}). At its core, a reinforcement learning problem consists of an \textit{agent} which can interact with an \textit{environment} with the goal of maximizing the cumulative \textit{reward} signal that it receives. Just as a trainer can teach a dog to sit by providing suitable rewards to reinforce the desired behaviour, we can apply the power of machine learning algorithms to maximize a prespecified reward function in the RL framework. We use the terminology \textit{behaviour policy} to refer to the distribution over actions that the agent takes in each state of the environment. The \textit{optimal policy} is the action-selection rule which maximizes the expected cumulative reward from each state.

\subsection{The reinforcement learning problem}

Learning how to behave optimally is often aided by learning to predict the expected cumulative reward the agent will receive after it visits a state and then follows some behaviour policy.
Our usage of the word `learn' differs from that used in other areas of machine learning, where it refers to the identification of a relationship between inputs and outputs from a data set. An RL agent does not receive explicit information about the optimal policy from the environment; this policy must be deduced from the reward and transition structure via planning. Reinforcement learning is thus closely related the problem of \textit{optimal control}. Control problems assume an environment, modelled as a Markov Decision Process (MDP), with known transition and reward structure and seek to identify an optimal behaviour policy. Reinforcement learning also seeks to obtain an optimal policy, but does not assume that the structure of the MDP is known in advance. Instead, the agent must interact with the MDP in order to obtain information about the reward and transition structure, and use this information to identify an approximately optimal policy. The number of interactions with the environment needed for an agent to identify an approximately optimal policy, its \textit{sample complexity}, is a key criterion by which RL algorithms are evaluated, and which distinguishes reinforcement learning from optimal control, where the dynamics of the world are known a priori.
\begin{figure}
    \centering
    \includegraphics[width=9cm]{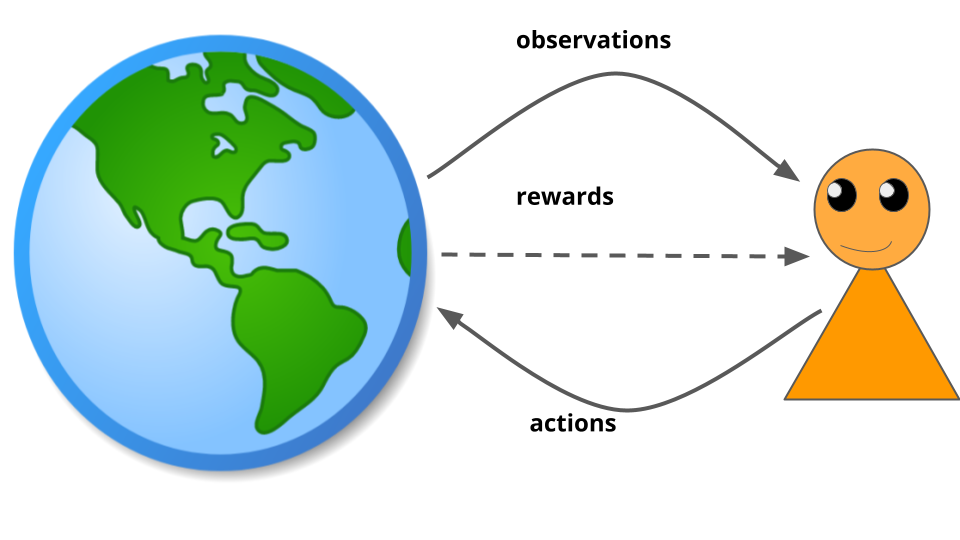}
    \caption{A figure visualizing the RL problem; featuring the environment (pictured on the left) and agent (right) dichotomy.}
    \label{fig:rl_formulation}
\end{figure}
\subsection{Generalization in reinforcement learning}

Many problems of interest in reinforcement learning do not require generalization, i.e. it is assumed that the agent will only encounter states at deployment that it encountered during training. Indeed, the field contains a rich literature on the analysis of \textit{tabular} problems, whereby the states of the MDP are simply an enumeration of integers and learning consists of updating a lookup table.
In many \textit{rich observation} settings, however, there may be an interesting functional relationship between the state observations emitted by the environment and the reward and transition structure at that state. While lookup tables are sufficient for small state spaces, most applications of RL to real-world data necessitate generalization either due to the magnitude of the state space or in order to be robust to distribution shifts. For example, the angles and torques of a robot's actuators can take on a continuum of values, and the set of possible image inputs far exceeds the size that can be fit into computer memory as a lookup table. Further, even if such a representation of the value function were possible  it is not clear whether that would be desirable. Visual similarity between states can provide information about the optimal policy and value function which may be useful to the agent. A function approximator with an appropriate inductive bias will be able to leverage this similarity to accelerate the learning process.

However, a dark side of generalization arises in reinforcement learning problems: instability. Instability is particularly problematic in some of the most popular algorithms in the deep RL literature, where careful hyper-parameter tuning and engineering tricks are needed to prevent the network parameters from diverging to infinite values. Excessive generalization can also slow down learning if the inductive bias encoded by the function approximator is not aligned with the structure of the environment. A bias towards smooth functions might thus result in a network that fails to accurately distinguish between states with large differences in value. The tension between generalization and stability in deep reinforcement learning presents an additional layer of difficulty to deep RL, as compared to supervised deep learning, and is a problem we will explore in later chapters.

\section{Understanding the learning process}
 
Throughout this thesis, we will seek to understand how a model will generalize by studying the optimization trajectory it took during training. In contrast, most theoretical results in the literature characterize generalization using only properties of the final outputs of a learning algorithm, i.e. the neural network's final trained parameters. Studying the trajectory of a learning algorithm (which we will refer to as its \textit{learning dynamics}) gives us the opportunity to gain insights into a model that cannot be obtained by only considering the final trained parameters. We will leverage these insights in later chapters to obtain novel estimators with significant predictive power over the ranking of a model's final generalization performance.

\subsection{Learning to generalize between data points}

Many neural network training procedures leverage large datasets which cannot fit onto a single GPU. To accelerate training, learning algorithms often partition the data into subsets called \textit{minibatches}. The learner then updates its predictions for each minibatch iteratively. This procedure provides extremely useful information about whether the agent is learning to generalize or to memorize, by revealing whether the learner's update based on one minibatch has \textit{generalized} to improve its predictions on the other minibatches. 

A key intuition throughout this thesis is that generalization between disjoint subsets of the training set can be indicative of generalization to the test set. If an update to the network intended to improve its predictions for one minibatch also improves the network's predictions on many other data points, this is likely to result in an improvement to the learner's predictions on novel inputs drawn from the same process that generated the training data. In contrast, if an update does not improve the learner's predictions on the other data points in the training set, it is unlikely to improve the agent's predictions on the data it will see at deployment. This relationship will be explored in greater depth in Chapter \ref{chp:supervised}.

\subsection{Stability vs extrapolation}

Our study of learning dynamics in RL agents will be particularly enlightening, as these dynamics are much more complex than their supervised counterparts. Supervised learning algorithms tend to induce well-behaved dynamics that, even if they correspond to a non-convex loss surface, nonetheless at least come with reasonable guarantees on the convergence of learning algorithms to local minima. Reinforcement learning, in contrast, incorporates nonstationarity as part of the learning process. This nonstationarity has many forms: the distribution of states that the agent visits will change as its policy improves, but so will the target values that the agent is trying to predict. This results in a dynamical system that superficially resembles the stable gradient descent regime of supervised learning algorithms, but lacks convergence guarantees many problem settings of interest. 

Even more pernicious, reinforcement learning agents must also face the challenge that the very properties of a function approximator which are associated with better generalization in the supervised learning setting are precisely those which can cause divergence in RL algorithms, as we will see in later sections. This means that deep RL agents must overcome two distinct hurdles in order to learn an optimal behaviour policy which also generalizes to new settings. First, they must learn to behave optimally in the training environment, while avoiding issues of divergence and other pathologies of function approximation in RL. Second, having achieved a high-performing policy, they must then ensure that the policy is robust to superficial changes to the observations they receive from the environment. 
 
\section{Thesis contributions and structure}

\subsection{Contributions}
Broadly speaking, this thesis presents a set of novel empirical and theoretical tools to predict, understand, and improve generalization in deep neural networks in a range of problem settings. Crucial to these results will be an analysis of the \textit{dynamics} of learning algorithms in various settings. The primary contributions of this thesis are enumerated as follows:
\begin{enumerate}
    \item A characterization of the relationship between invariance, training speed, and generalization, with theoretical results complemented by empirical validation of our main findings in practically relevant settings. 
    \begin{enumerate}
          \item A theoretical analysis of the effect of invariance on generalization via PAC-Bayes bounds, allowing an explicit characterization of the role of symmetries in generalization bounds via a quantity we term the \textit{symmetrization gap}. This theoretical analysis is complemented by an empirical study which highlights the limitations of the types of approximate invariance promoted by data augmentation to generalize to out of distribution inputs.
        \item A new estimator of the marginal likelihood which complements the analysis of upper bounds on generalization error to provide useful rankings of models for architecture search and hyperparameter selection. The analysis of this estimator reveals a deep connection between training speed and Bayesian model selection, and yields a novel performance estimator for architecture search in deep neural networks.
    \end{enumerate}
  
    \item A theoretical framework for the study of representation dynamics in deep reinforcement learning along with practical insights derived thereof. These insights concern both the ability of a learned feature representation to linearly approximate a prediction target, and its ability to generalize to novel prediction objectives over the course of training.
    \begin{enumerate}
        \item A theoretical model and analysis of the representation learning dynamics of value-based reinforcement learning algorithms, yielding an analytic characterization of the effect of auxiliary tasks on agents' learned representations. 
        \item The identification of the `capacity loss' phenomenon in deep RL which characterizes the tendency of neural networks to catastrophically overfit to early prediction targets in sparse-reward environments. 
        \item A new regularization method based on the insights from this previous analysis that improves generalization to new prediction objectives even after long training periods. 
    \end{enumerate}
    \item Theoretical analysis leveraging the above framework to provide insight into and algorithmic improvements to generalization to novel observations and environments.
    \begin{enumerate}
        \item A theoretically grounded explanation for prior empirical observations of overfitting in the broader deep RL literature, including dense-reward problems, and a set of recommendations for principled approaches to reduce memorization and improve generalization between observations. 
    \item A representation-learning objective which goes beyond the single-environment setting to enable zero-shot generalization to novel environments sharing underlying structure with the training environments.
    \end{enumerate}
     
\end{enumerate}

\subsection{Warmup: supervised learning}
The first two content chapters of this thesis will lay the groundwork for our understanding of generalization in deep neural networks. Chapter \ref{chp:invariance} presents a novel PAC-Bayes generalization bound for invariant models, characterizing the effect of invariance on generalization through a quantity which we term the \textit{symmetrization gap}. We empirically verify that the symmetrization gap appears in computations of PAC-Bayes bounds for invariant deep neural networks, and further show a strong correlation between the rankings of these upper bounds and the rankings given by final generalization performance. This chapter further contributes new empirical analysis of the optimization trajectories of DNNs trained to exhibit approximate invariances via data augmentation, illustrating the limitations of these trained invariances to {extrapolate} beyond their training distribution. The importance of invariance to generalization in reinforcement learning will be revisited in Chapter~\ref{chp:icp}.

A failing of PAC-Bayesian generalization bounds is that they do not offer predictions about which of a set of models will generalize best. In order to achieve correctness, these bounds pay the price of predictive power. Chapter~\ref{chp:supervised} shifts focus to consider performance estimators which can \textit{predict} the relative performance of different neural network architectures. It presents a novel marginal likelihood estimator which can be used to enable Bayesian model selection for a broader class of models from which we only require accurate posterior samples, overcoming the normative limitations of generalization bounds and computational challenges of the exact marginal likelihood. This estimator has a number of appealing properties, chief among which is that it can be applied to a subset of training losses from a gradient descent trajectory. This illustrates a deep relationship between generalization and learning dynamics, as a model's generalization performance can thus be said to depend on its \textit{training speed}. These chapters will principally be based on the following papers:
\nobibliography*
\begin{itemize}
    \item \bibentry{lyle2020bayesian}
    \item \bibentry{lyle2020benefits}
\end{itemize}
with supporting evidence drawn from
\begin{itemize}
    \item \bibentry{ru2020revisiting}
\end{itemize}
\subsection{Generalization in reinforcement learning}
The remaining chapters present an analysis of the learning dynamics of RL algorithms and explore a number of applications of this analysis to representation learning and generalization in deep RL. Chapter~\ref{chp:rl-dynamics} lays down a novel \textit{learning dynamics} framework which will be leveraged throughout the remainder of the thesis, and reveals that the dynamics followed by temporal difference methods can cause the learned representation to reflect aspects of the transition structure of the environment \citep{lyle2021effect}. In Chapter~\ref{chp:rep-learning}, we explore one application of this analysis which yields a novel regularization approach, $\pyoi$ \citep{lyle2021understanding}. This method enables deep RL agents to attain nontrivial return in the notoriously difficult Montezuma's Revenge game using only a naive $\epsilon$-greedy exploration algorithm. This section is based on the following papers:
\begin{itemize}
    \item \bibentry{lyle2021effect}
    \item \bibentry{lyle2021understanding}
\end{itemize}

Chapter~\ref{chp:gen-rl} explores implications of the previous chapters on generalization, providing novel analysis and insight into why value-based deep reinforcement learning often produces highly brittle agents \citep{lyle2022generalization}, as well as identifying principled approaches to remedy the tendency of deep RL methods to overfit. Chapter~\ref{chp:icp} revisits the discussion of invariance from Chapter~\ref{chp:invariance} with a new perspective grounded in causal inference, presenting a novel representation learning method to encourage generalization in some classes of MDPs by promoting invariance to spurious factors of variation in the environment \citep{zhang2020invariant, lyle2021causal}. These chapters are based on the following papers.
\begin{itemize}
    \item \bibentry{lyle2022generalization}
    \item \bibentry{zhang2020invariant}
\end{itemize}

A number of papers that I worked on during my PhD did not fit into this thesis, including the following.

\begin{itemize}
    \item \bibentry{wang2021provable}
    \item \bibentry{filos2021psiphi}
    \item \bibentry{kossen2021self}
    \item \bibentry{bellemare2019geometric}
    \item \bibentry{lyle2019comparative}
\end{itemize}

A discussion of contributions to joint work can be found at the end of the thesis.

The key idea driving this thesis is that properties of a network's training trajectory can tell us a great deal about how it will generalize to novel inputs. Chapters 3 and 4 ground this idea in the supervised learning regime by studying generalization between minibatches and validate its utility by developing both novel generalization bounds and practical model selection tools. Chapters 5 and 6 identify key properties of the training dynamics of reinforcement learning agents that differ from the supervised setting, and show how these properties can be both beneficial and detrimental to representation learning. Finally, Chapters 7 and 8 apply the notions of invariance and within-training-set generalization from Chapters 3 and 4 to the RL problem, leveraging the theoretical and empirical tools presented in Chapters 5 and 6 to quantify memorization and improve generalization in reinforcement learning. We will conclude with a discussion of how these ideas can be (and in some cases have already been) further leveraged in the pursuit of learning systems which can effectively learn and generalize across a range of tasks.

\chapter{Background \& literature review}
\label{chp:background}

\minitoc
This chapter will provide the high-level background and literature necessary to contextualize the contributions of this thesis, and set out a standard set of notation that will be used in the following chapters. Where necessary, individual chapters may also contain a background section; these sections will relay information that pertains only to the contents of the chapter that contains them.

\section{Learning frameworks}
\label{bkgd:learning-frameworks}
A learning algorithm can be applied to a wide variety of problems, and it is often useful to categorize learning problems based on the information available to the algorithm. We might task a learning system with identifying a mapping between input-label pairs (supervised learning), with constructing an embedding of inputs that captures relevant structure (unsupervised learning), with generating samples from some distribution (generative modelling), or with maximizing a reward signal via interaction with an environment (reinforcement learning). The types of algorithms we can deploy in each of these situations differ not only in terms of the types of outputs they produce, but also in the stability of their learning dynamics. This thesis will focus primarily on the distinction between supervised learning, where learning dynamics of gradient descent algorithms are relatively straightforward to analyze, and reinforcement learning, wherein analogues of even simple methods such as linear regression can suffer from instability and divergence.

\subsection{Supervised learning}
\label{bkgd:supervised-learning}
Supervised learning is concerned with characterizing a functional relationship between inputs and labels. This framework assumes that the data takes the form of input-label pairs $(\bx, y)$ generated by sampling from some distribution $\dgd$. The objective of the learning algorithm is to find a function $f$ such that $f(\bx) = y$. The task of finding such a function is nontrivial: one must both propose a suitable class of functions over which to search, and an effective means of identifying functions from this class which are likely to capture the target relationship.

\subsubsection{Empirical risk minimization}
\label{bkgd:erm}
We first consider the task of identifying a candidate function $f$ with the property that $f(\bx) = y$ on $(\bx, y)$ pairs sampled from $\dgd$, including those not seen during training. We use a loss function $\ell$ to quantify the quality of $f(\bx)$ as an approximator to $y$; we will also refer to the expectation of this quantity as the \textit{risk}. When $y$ belongs to a continuous space, we call this a \textit{regression} problem. When $y$ belongs to a finite set, we have a \textit{classification} problem.

For a class of functions $F= \{f : \calX \to \calY\}$, a set $\trdata = (\bx_i, y_i)_{i=1}^n \sim \dgd$, and a loss function $\loss : \calY \times \calY \to \bbR$, we define the expected risk $\risk$ as
\begin{align}
  \risk(f) &= \mathbb{E}_{(\bX,Y)\sim\dgd}[\ell (f(\bX), Y)] \; .\label{eq:risk} \\
  \intertext{ Similarly, we define the empirical risk $\eRisk$ as}
  \eRisk(f, \trdata) &= \textstyle\frac{1}{n}\textstyle\sum_{i=1}^n \ell(f(\bx_i), y_i) \; .\label{eq:eRisk}
\end{align}

In regression problems, $\ell$ is typically set to be the squared error $(f(\bx) - y)^2$. In classification, it is usually the cross-entropy loss between a categorical distribution $p(\cdot | \bx)$ and the Dirac delta distribution at the label $y$. In this case, the hypothesis class $\fclass$ will consist of mappings from inputs to \textit{distributions} over labels in $\calY$.

When the hypothesis class $\fclass$ contains the true functional relationship $f$, the learning problem is realizable. However, most settings of interest are not realizable, and so we seek instead a function $f^*$ which minimizes the expected loss over the data-generating distribution, which can be expressed formally as follows.
\begin{equation}
f^* = \argmin_{f \in \fclass} \risk(f) = \argmin_{f \in \fclass}  \mathbb{E}_{(\bX,Y)\sim\dgd}[\ell (f(\bX), Y)]
\end{equation}

In practice, computing this expectation is impossible and we instead use samples to estimate its true value. The empirical risk minimization framework assumes a finite sample $(\bx_i, y_i)_{i=1}^n = \trdata \sim \dgd$, and proposes to find a function $f \in \fclass$ that minimizes the empirical expectation of the loss over this sample, i.e. the empirical risk. Concretely, the function $\hat{f}$ is called the empirical risk minimizer if the following holds:

\begin{equation}
    \hat{f} = \argmin_{f \in \fclass} \eRisk(f) = \argmin_{f \in \fclass} \frac{1}{n} \sum_{i=1}^n \ell(f(\bx_i), y_i) \; .
\end{equation}

The empirical risk minimization principle has seen widespread application in the machine learning literature  \citep{vapnik1991principles, donini2018empirical}. \citet{vapnik1968uniform} characterizes a number of appealing asymptotic properties of the empirical risk minimizer under certain conditions on the function class $\fclass$; this analysis has formed the basis for the work on generalization bounds which will be discussed in Section~\ref{bkgd:generalization-bounds}. This principle is agnostic to the choice of function class, and does not give guidance on how to obtain a minimizer $f^*$ when such classes are too large for exhaustive search. The following discussion will focus on one such class and search procedure: neural networks trained with gradient-based optimization.

\subsubsection{Deep learning}
\label{bkgd:deep-learning}
Deep neural networks (DNNs) form a powerful and expressive class of function approximators \citep{raghu2017expressive}. A DNN is a parameterized function $f_{\theta}$ which consists of layer-wise computations going from input to output. Some neural architectures include recurrent connections, where the output of the network is fed back into itself as input \citep{hochreiter1997long}; this thesis will focus exclusively on feedforward architectures, where only a single pass through the network is executed in order to obtain the function outputs. Feedforward neural networks constitute a rich and widely-used class of models whose dynamics are more amenable to analysis, including fully-connected multi-layer perceptrons (MLPs), convolutional neural networks, transformers \citep{vaswani2017attention}, and ResNets \citep{he2016deep}. The function $f^k$ computed by each layer $k$ of a feedforward neural network typically consists of a linear transformation of the output of the previous layer, followed by a non-linear activation function. A variety of activations have been used in DNNs; one popular example is the Rectified Linear Unit (ReLU), of the form $\sigma(x) = \max (0, x)$. The output of  a neural network can thus be expressed as a composition of layer-wise operations,
\begin{equation}
    f_\theta(\bx) = f_\theta^L \circ \dots \circ f_\theta^1(\bx) \; .
\end{equation}
Deep neural networks are trained using gradient-based optimization algorithms. The most fundamental of these is \textit{stochastic gradient descent}. In this setting, the data $\trdata$ is uniformly at random divided into minibatches of size $k$, $((\bx_{b_1}, y_{b_1}), \dots, (\bx_{b_k}, y_{b_k}))_{b=1}^{\lfloor n/k \rfloor}$. At each minibatch, we compute the gradient of the loss to obtain an update direction $g$ as follows,
\begin{equation*}
   g(\theta, \data_b) = -\nabla_\theta \frac{1}{k} \sum_{i=1}^k \ell(f_\theta(\bx_{b_i}), y_{b_i})\; .
\end{equation*}
This yields an iterative algorithm where the parameters $\theta$ are updated according to the gradient for each successive minibatch. In most cases, we use a step-size parameter $\alpha \in (0, 1]$ to improve the stability and convergence properties of the algorithm. The iteration step typically takes the following general form:
\begin{equation}
    \theta_{t+1} \gets \theta_t + \alpha_t g(\theta_t, \data_t) \; .
\end{equation}
Many formulations of gradient descent allow the step size to depend on the iteration $t$; such dependence on the iteration is crucial to obtain convergence guarantees \citep{robbins1951stochastic}. Many adaptive optimization schemes \citep{duchi2011adaptive, kingma2014adam} further accumulate parameter-dependent learning rates, and allow the update direction $g(\theta_t, \data_t)$ to also depend on prior gradients. While these optimizers will feature in the empirical analysis that appears later, their precise form is not important to our discussion. 

\subsection{Reinforcement learning}
\label{bkgd:rl}
Whereas the supervised learning framework seeks to model a functional relationship between two variables, the reinforcement learning framework models an agent's interaction with an environment with the goal of identifying a behaviour policy which maximizes some reward signal. We model the environment as a Markov Decision Process (MDP) $\mathcal{M} = (\statespace, \actionspace, R, P, \gamma)$, where $\cX$ denotes the state space, $\cA$ the action space, $R:\statespace \rightarrow \mathbb{R}$ the reward function, $P:\statespace \times \actionspace \rightarrow \mathscr{P}(\statespace)$ the transition probability function, and $\gamma$ the discount factor. The agent obtains observation $x$ which indicates the environment's state. It may then take an action $a \in \cA$, after which the environment outputs a new observation $x'$ and a reward $r$. The agent's objective is to maximize the cumulative discounted reward it receives over time.  In \textit{finite-horizon} environments, the agent may only take a finite number of steps; in \textit{continuing} or \textit{infinite-horizon} environments, on which this thesis will predominantly focus, the agent may take an unlimited number of steps in the environment. The discount factor $\gamma$ determines the degree to which near-term rewards are preferred, resulting in the maximization target $\sum_{k=0}^{\infty} \gamma^k R_k$, called the \textit{return}. The return is a random variable that depends on the sequence of states and actions taken by the agent. The value of a state-action pair under some action-selection policy $\pi : \statespace \rightarrow \mathscr{P}(\actionspace)$ is equal to the expected value of the return starting at some state-action pair $(x,a)$ and following the policy $\pi$. This can be expressed by the action-value function $Q^\pi:\cX \times \cA \rightarrow \mathbb{R}$, defined as
\begin{equation}
    Q^\pi(x, a) = \mathbb{E}_{\pi , \cP}[\sum_{k=0}^\infty \gamma^k R(x_k, a_k)|x_0=x, a_0=a ] \;.
\end{equation} 

\textbf{Value-based methods} seek to learn the (resp. action-) value function $V^\pi:\cX\rightarrow \mathbb{R}$ (resp. $Q^\pi : \cX \times \cA \rightarrow \mathbb{R}$) associated with some policy $\pi$ \citep{sutton2018reinforcement}. In particular, we are interested in learning the value function associated with the optimal policy $\pi^*$ which maximizes the expected discounted sum of rewards from any state. Such a value function is then straightforward to translate into an optimal behaviour policy: the agent need only take the action with the highest predicted value at each state. 

At the core of value-based RL are the \emph{Bellman operators} \citep{puterman}. The Bellman policy evaluation operator $T^\pi : \mathbb{R}^{\statespace} \rightarrow \mathbb{R}^{\statespace}$ is defined with respect to a policy $\pi$ and provides a method to \textit{update} a predicted value function $V$ to more closely resemble the value $V^\pi$ of the policy $\pi$. It is defined as
\begin{align*}
    (T^\pi V)(x) = \mathbb{E}_{X_1 \sim P(\cdot|x,\pi(x))}[ R(x) + \gamma V(X_1)] \, .
\end{align*}
Introducing a matrix notation of the transition operator $P^\pi \in \mathbb{R}^{\statespace\times\statespace}$ defined by $P^\pi[x, x'] = \sum_{a \in \mathcal{A}}\pi(a|x)P(x'|x, a)$, and the expected reward vector $R^\pi \in \mathbb{R}^{\statespace}$ defined by $R^\pi(x) = \mathbb{E}_\pi[R_0|X_0=x]$, this can be expressed even more succinctly in matrix notation as
\begin{align*}
    T^\pi V = R^\pi + \gamma P^\pi V \, .
\end{align*}
$T^\pi$ is a contraction on the space of value functions \citep{puterman}, and so repeated application of $T^\pi$ to any initial value function converges to $V^\pi$ \citep{bertsekas1996neuro}. For control problems, where we seek to obtain the value of an unknown optimal policy, learning requires not only estimating the value of a policy but also improving that policy to increase its expected return. In this setting, we leverage an analogous operator termed the Bellman optimality operator $T^* : \mathbb{R}^{\statespace\times\actionspace} \rightarrow \mathbb{R}^{\statespace\times\actionspace}$. The action of $T^*$ on value functions is defined by
\begin{align*}
    (T^* Q)(x,a) \!=\! \mathbb{E}_{x' \sim P(x,a)}[R(x) \!+\! \gamma \max_{a' \in \mathcal{A}} Q(x', a')] \, .
\end{align*}

The Bellman optimality operator attains similar convergence guarantees as the policy evaluation operator, and can be shown in tabular state spaces to converge to the value of the optimal policy. In principle both $T^*$ and $T^\pi$ can be defined over action-value or state-value functions, but this thesis will predominantly consider the application of $T^\pi$ to value functions and $T^*$ to action-value functions We will refer to the value $T^\pi V$ or $T^* Q$ as the \textit{Bellman target} associated with the (resp. action-) value function $V$ (resp. Q), where the operator in use will be clear from context. 

In most settings of interest, we do not have access to the expected reward vector $R^\pi$ or the environment transition matrix $P$. As a result, value-based RL agents must use sampled transitions of the form ($x_t, a_t, r_t, x_{t+1}, a_{t+1}$) to approximate these updates. In the case of the policy evaluation operator, this sample-based approximation takes the form of the SARSA update, so called due to its use of (\textbf{S}tate-\textbf{A}ction-\textbf{R}eward-\textbf{S}tate-\textbf{A}ction) transitions. The SARSA algorithm assumes that the transition has been sampled from some fixed behaviour policy $\pi$, and estimates $Q^\pi$ by iteratively applying the update rule

\begin{equation}
    Q_{t+1}(x_t,a_t) = Q_{t}(x_t,a_t) + \alpha [r_t + \gamma Q_t(x_{t+1}, a_{t+1}) - Q_t(x_t, a_t)] \;.
\end{equation} 

Sample-based methods will use some step size $0<\alpha < 1$ in order to average out noise in the Bellman targets due to stochasticity in the environment. The seminal Q-learning algorithm \citep{watkins1992q}, which forms the basis of many deep RL agents \citep{mnih2015human}, can similarly be viewed as approximating the iterative application of $T^*$ and related operators \citep{tsitsiklis1994asynchronous,jaakola1994convergence,bertsekas1996neuro}. Q-learning is based on the following update
\begin{equation}
    Q_{t+1}(x_t,a_t) = Q_{t}(x_t,a_t) + \alpha [r_t + \gamma \max_{a' \in \actionspace} Q_t(x_{t+1}, a') - Q_t(x_t, a_t)] \; .
\end{equation}

\textbf{Policy gradient methods} \citep{sutton2000policy} operate directly on a parameterized policy $\pi_\theta$. We let $d^\pi$ denote the stationary distribution induced by a policy $\pi$ over states in the MDP. Rather than first going to the trouble of learning a value function, policy gradient methods directly optimize the parameters $\theta$ of the policy so as to maximize the expected return $J(\pi_\theta) = \mathbb{E}_{s_0 \sim P_{\mdp}(s_0)} V^{\pi_{\theta}}(s_0) $. The gradient of this loss can be estimated from sampled trajectories when expressed as follows:
\begin{equation}
    \nabla_\theta J(\pi_\theta) = \mathbb{E}_{x_t, a_t \sim P(\cdot | \pi_\theta)}[\nabla_\theta \log \pi_\theta(a_t|x_t) Q^{\pi_\theta}(x_t, a_t)] \; .
\end{equation}
Variations on this learning rule include \textit{actor-critic} methods \citep{konda2000actor}, which use a baseline given by a value-based learner to reduce update variance, and trust-region based methods, such as Trust Region Policy Optimization \citep{schulman2015trust} and Proximal Policy Optimization (PPO) \citep{schulman2017proximal}. 

\textbf{Function approximation} schemes enable RL agents to generalize their knowledge about the value or policy at one state to other states in the environment. Linear function approximation assumes state-action pairs are embedded as features $\phi(x,a) \in \mathbb{R}^d$, and some linear map $\bw \in \mathbb{R}^d$ is used to approximate the value function $Q^\pi(x,a) = \langle \phi(x,a), \bw \rangle$. This regime has been the study of a rich literature exploring the stability and sample complexity of RL in the presence of function approximation \citep{jin2020provably, wang2020optimism, precup2001off, tsitsiklis1996analysis}.

A second, more widely-used family of algorithms involve the use of neural networks as function approximators of the form $Q_\theta: \statespace \times \actionspace \rightarrow \mathbb{R}$. This is the \textit{deep RL} regime. This approach is well-suited to an array of complex tasks, ranging from systems control problems \citep{degrave2022magnetic} to video games \citep{mnih2015human}. At its core, value-based deep RL involves training a neural network to approximate the Bellman targets of a predicted value function using (semi-, see e.g. \citep{sutton2018reinforcement}) gradient descent. This approach computes a semi-gradient update direction $f(\theta)$ of the following form, where $a^* = \max_{a}(Q_\theta(x_{t+1}, a))$.
\begin{equation}
    f(\theta) = (\nabla_\theta Q_\theta) [r_t + \gamma Q_\theta(x_{t+1}, a^*) - Q_\theta(x_t, a_t)]
\end{equation}
Because the final layer of a neural network is usually linear, it is possible to express the output $Q_\theta(x,a)$ in the form $Q_{\theta}(x,a) = \langle \phi_{\theta'}(x,a) , \bw \rangle$, and $\theta = \theta' \oplus \bw$ where $\oplus$ denotes concatenation. Under this parameterization, the map $\phi_{\theta'}$ is referred to as the \textit{feature map}. This framework has been used to study the learned representations of deep reinforcement learning agents in a number of recent works \citep{kumar2021implicit, lan2022generalization, lyle2019comparative}.

\section{Generalization in supervised learning}

We now turn our attention to the principal object of interest in this thesis: generalization. The study of generalization in supervised learning problems spans decades, from the seminal work of \citet{vapnik1968uniform} to recent exciting developments in the kernel analysis of deep networks \citep{jacot2018neural} and the interpolation regime \citep{bartlett2020benign}. We will begin by presenting classical bounds on the generalization error of a learning algorithm's output. However, these classical approaches, based on quantifying the complexity of an algorithm's hypothesis class, fail to account for the generalization performance of deep neural networks, motivating more recent empirical approaches. We will conclude with an overview of the current state of the art of our understanding of generalization in deep learning. 

\subsection{Generalization bounds}
\label{bkgd:generalization-bounds}

The empirical risk minimization (ERM) framework puts forward the maxim: `always pick a hypothesis which minimizes the risk on the training set'. This hypothesis will not in general attain the lowest risk on the underlying data-generating distribution, however. A long line of work~\citep{vapnik1968uniform, vapnik1999nature, bousquet2002stability} has characterized upper bounds on the gap between the empirical risk of a hypothesis and its true risk, yielding \textit{generalization bounds} which provide high-probability guarantees on the expected risk of a hypothesis. Of particular interest to us will be the application of such bounds to hypothesis classes generated by neural networks \citep{baum1989size, dziugaite2017nonvacuous, bartlett2017spectrally}. While generalization bounds for any hypothesis class are almost always looser than the upper bound on the expected risk given by computing the model's validation loss on held-out data, the study of generalization bounds continues to thrive as a means of developing theoretical insight into learning algorithms, motivating their inclusion in our discussion.
 
\subsubsection{Formalism}
Most generalization bounds in the literature share a similar structure, consisting of the sum of a hypothesis' empirical risk and a complexity measure scaled by a function (typically the inverse square root) of the number of samples. We use the notation of \Cref{bkgd:supervised-learning}, and let $f$ be a function output by some learning algorithm with access to data $\trdata$ of size $n$, drawn from hypothesis class $\fclass$. We let $\risk$ and $\eRisk$ be defined as in Equations~\ref{eq:risk} and \ref{eq:eRisk} respectively. Letting $C(\fclass)$ denote some complexity measure (for example, the logarithm of the number of functions in the hypothesis class) and $g(\cdot)$ some non-negative function defined over $\mathbb{R}$, usually $g(x) = \log(\frac{1}{x})$ or something similar, we obtain the generic form
\begin{equation}
    \risk(f) \leq \eRisk(f) + \sqrt{\frac{C(\fclass) + g(\delta)}{n}} \text{ with probability $1-\delta$}.
\end{equation}

The magnitude of the complexity measure $C(\fclass)$ is crucial to the tightness of a bound. In some cases the complexity measure which bounds the expected risk may take a value so large that it dwarfs the maximal value the risk can obtain, resulting in bounds that are \textit{vacuous}. For example, a vacuous bound would guarantee that a neural network's probability of making a classification error on new samples from the the MNIST dataset will be less than 500\%. In neural networks, bounds based on the margin around the decision boundary \citep{wei2019improved}, the VC dimension \citep{harvey2019nearly}, and the spectral norm of the network weights \citep{bartlett2017spectrally} all become vacuous for network architectures of the scale typically applied to popular benchmarks such as CIFAR-10 or ImageNet \citep{bartlett1997valid, dziugaite2017nonvacuous}. However, recent work has found complexity measures which \textit{can} capture a reasonable notion of simplicity on neural networks -- at least to the point where bounds are non-vacuous \citep{dziugaite2017nonvacuous}. 

\subsubsection{Complexity and Occam's razor}
The term $C(\fclass)$ can be interpreted as a form of Occam's razor applied to the hypothesis class: given hypotheses drawn from two classes which attain the same empirical risk, we should prefer the hypothesis drawn from the simpler class. This simple idea drives most results in model selection and generalization in machine learning \citep{rasmussen2001occam}. 
However, while generalization bounds of the flavour shown above characterize the complexity of the entire class of hypotheses, one might also prefer to use Occam's razor as a tool to select hypotheses within a single class. This is the core idea behind PAC-Bayes generalization bounds \citep{mcallester1999, langford2003pac, leveretal2013tighterPACbayes, catoni}, which allow us to assign a complexity penalty that distinguishes between different hypotheses within a function class for randomized predictors. Leveraging a notion of hypothesis-level complexity is a powerful tool \citep{bartlett1997valid}, yielding some of the first non-vacuous generalization bounds for over-parameterized neural networks \citep{dziugaite2017nonvacuous}. 

Occam's razor does not give us an out-of-the-box definition of simplicity, however. Two such definitions are favoured by the machine learning community to explain generalization of deep networks. The first of these is flatness of the local minima to which gradient-based optimization tends to converge \citep{hochreiter1997flat}. Much of the folk wisdom surrounding generalization in deep learning hypothesizes that the reason neural networks generalize is because stochastic gradient descent drives the learned weights towards flat regions of the loss landscape \citep{keskar2016large}. This picture is slightly complicated by the observation of \citet{dinh2017sharp} showing that sharp minima can also generalize, but nonetheless demonstrates strong predictive power in empirical evaluations \citep{jiang2020fantastic}. Optimization algorithms which deliberately add noise to the training process, such as entropy-SGD and the direct optimization of a PAC-Bayes bound \citep{chaudhari2016entropy, dziugaite2017entropy, dziugaite2018dependent}, have been shown to increase flatness and improve the tightness of some generalization bounds on neural networks. PAC-Bayes risk bounds are also deeply connected to the Bayesian notion of marginal likelihood \citep{germain2016pac}. 

A second widely-used notion of simplicity is model compressibility, whereby models which can be compressed to a smaller length are said to be simpler. This is also referred to as the minimum description length (MDL) principle \citep{Akaike1998, hinton1993keeping}. Compressibility and flatness overlap significantly: parameters in a flat region of the loss landscape can be perturbed, for example by a quantization step of a compression algorithm, without significantly affecting the loss. However, the techniques used to measure the two quantities differ, with compression approaches typically being more compute-intensive \citep{zhou2018nonvacuous, ullrich2017soft}.
PAC-Bayes bounds are appealing as either the flatness of minima \citep{neyshabur2018the, neyshabur2017exploring}  or the compressibility \citep{zhou2018nonvacuous} notion of simplicity can be used to define the model complexity term in the bound. 

\subsubsection{The overparameterized regime}
At the core of the challenge of applying results from statistical learning theory to neural networks is the expressiveness of neural network function classes; the overparameterized models of recent years are capable of memorizing even uniform random labels of the data \citep{zhang2016understanding}. Uniform convergence bounds which depend on worst-case analysis are thus limited in the guarantees they can provide these networks.  Indeed, \citet{nagarajan2019uniform} argue that certain forms of uniform convergence results may be fundamentally incapable of explaining generalization in deep learning, presenting a simple example of a learning problem where it is impossible to use uniform convergence guarantees to characterize the generalization performance of a neural network function class. While \citet{negrea2020defense} and \citet{bartlett2020benign} show that modified analysis can still yield uniform convergence results for overparameterized predictors in similar contexts, these results require studying either a surrogate predictor, or a restricted problem setting. 

Even more intriguing has been the observation that, contrary to the received wisdom in learning theory that overparameterized models will overfit to their training data, increasing overparameterization can lead to \textit{improved} generalization \citep{neyshabur2014search}. These empirical observations have prompted the study of interpolating predictors \citep{bartlett2020benign}, which seeks to outline conditions by which a model which can attain zero empirical risk may still obtain near-optimal risk or robustness properties on the data-generating distribution \citep{bubeck2021universal, koehler2021uniform}. This work is intricately connected to the double descent phenomenon, whereby the risk of a model, when plotted against the number of parameters, exhibits two descent regions: one in the underparameterized regime, and one in the overparameterized regime \citep{advani2020high, belkin2018reconciling, nakkiran2019deep}. 

A gap remains, however, in leveraging these insights to attain tight bounds on the generalization error of modern neural network architectures trained on real-world datasets.
Even worse, \citet{jiang2020fantastic} show that many of the complexity measures that appear in generalization bounds are \textit{negatively} correlated with generalization. This presents a double blow to the argument that generalization bounds might provide insight into what makes neural networks generalize well: not only are such bounds too loose to give practically relevant information, but they do not even offer an accurate \textit{ranking} of models. 


\subsection{Explaining generalization without bounds}
Why does a given neural network generalize well or poorly? This is a question that generalization bounds cannot (currently) answer, yet is crucial to the principled development and application of neural networks to real-world datasets. A large value of a complexity measure does not entail that a learning algorithm will not generalize; it simply states that there is not sufficient information to guarantee a small generalization gap. To address questions of \textit{why} particular models generalize well, we must look to the particulars of the network architectures, training procedures, and datasets that arise from the natural world. The works discussed in this section will present such an empirical study. While sometimes similar notions of complexity to those leveraged in generalization bounds may be used, the philosophical underpinnings of this research differ fundamentally from the formal guarantees of learning theory.

\subsubsection{Scientific vs mathematical rigor}\label{sec:background-science}
The scientific study of any phenomenon depends on a mixture of empirical experiments, during which data is gathered, and theory-building, where an explanation of the phenomenon is proposed \citep{popper1968logic}. A scientific theory is one that makes falsifiable predictions, which can then be tested by experiments. The generalization bounds of \Cref{bkgd:generalization-bounds} can be viewed as theory-building in a loose sense; however, the goal of a generalization bound is to make a statement that will be \textit{guaranteed} to hold with high probability, accepting that this may result in a large gap between the predicted upper bound and the empirical realization of the generalization error in some settings. This runs counter to the qualities of a good scientific theory, whose goal is to explain a phenomenon with as simple a model as possible, and to expose itself to the risk of falsification in doing so. 

One notable step towards this framework of formulating and falsifying hypotheses arises from the work of \citet{jiang2020fantastic} and \citet{dziugaite2020search}, who conduct a large-scale empirical study of the correlation between \textit{complexity measures} and generalization in deep neural networks. This line of work seeks to translate the formal guarantees of generalization bounds into testable scientific theories which make predictions about which models will generalize best. The key to doing so is to treat the ranking over models given by a complexity measure as a \textit{prediction}, rather than an upper bound. If the generalization measure $C$ is higher for network A than for network B, we interpret this as a prediction that network A should generalize worse than network B stemming from the theory that a model which generalizes well must do so \textit{because} $C$ is low. 

\citet{dziugaite2020search} focus on identifying experimental settings where a measure \textit{fails} to predict generalization, seeking to identify instances where the theory can be falsified. Both works show that complexity measures based on flatness, such as PAC-Bayes generalization bounds \citep{dziugaite2017nonvacuous}, are highly predictive of generalization across a range of experimental settings. However, none of the generalization measures studied by \citet{dziugaite2020search} robustly predicts generalization in neural networks under all experimental conditions, motivating further investigation of quantites which can predict -- and ideally also {explain} -- generalization in neural networks.

\subsubsection{The loss landscape of neural networks}
Much theoretical work studying neural networks has considered questions on the \textit{existence} of parameters that allow the network to represent a function \citep{hornik1989multilayer, raghu2017expressive, dong2020expressivity}, but less attention has been paid to the search process by which such parameters might be found. To this end, recent work studying the loss landscape of neural network seeks to understand the properties of the optimization problem faced by neural networks in practice. These results have found that the initialization schemes used today \citep{he2015delving, glorot2010understanding} push the parameters towards a well-behaved region of the loss landscape that is relatively convex \citep{fort2019goldilocks} and induces stable learning dynamics in suitable architectures \citep{yang2017mean}. Indeed, many architecture design choices such as residual connections \citep{he2016deep} and batch normalization \citep{ioffe2015batch} that improve performance can be shown to increase the smoothness of the loss landscape \citep{li2018visualizing, santurkar2018does}, leading to minima which are flatter and should thus generalize better. 

Much recent work has focused in particular on the geometry of the loss landscape around the minima found by gradient descent \citep{maddox2020rethinking}. One intriguing property identified recently is that of \textit{linear mode connectivity}, whereby the convex combination of two locally optimal parameter vectors attains a similarly low loss as the two minima from which it was derived \citep{benton2021loss}. \citet{frankle2020linear} show that stochasticity in the optimization process \textit{early} in training leads the trajectory towards disconnected regions of parameter space, while stochasticity late in training results in a perturbation within a linearly connected loss basin. This geometric perspective is consistent with prior work studying the importance of the early learning period in determining the network's ability to represent certain functions \citep{achille2018critical}.
Flatness properties of local minima have also been studied in the context of Bayesian deep learning \citep{izmailov2018averaging}, highlighting the connection between flat minima and approximations of the Bayesian marginal likelihood \citep{daxberger2021laplace}. 

\subsubsection{Correlates of generalization in deep learning}

Much empirical work on generalization in deep learning seeks to identify measurable quantities which correlate with generalization in neural networks \citep{neyshabur2017exploring}. This work has found such quantities in the robustness of minima to perturbations \citep{keskar2016large} and to pruning \citep{bartoldson2020generalization}. Relatedly, \citet{morcos2018importance} identify a relationship between the dependence of a network on single dimensions of activation space and generalization. \citet{neyshabur_norm-based_2015} propose that implicit norm-based regularization may contribute to the generalization performance of deep networks; some evidence suggests, however, that the mechanism by which overparameterization regularizes norm may not apply to fully general problem constructions \citep{hanin2019deep}. 

Further attention has been paid to regularization methods which improve generalization such as dropout \citep{srivastava2014dropout}, data augmentation \citep{wu2020on}, batch normalization \citep{ioffe2015batch}, and weight decay \citep{arpit2017closer}. The study of early stopping \citep{duvenaud2016early} has further revealed intriguing connections between variational inference and flatness of the loss landscape. However, the precise relationship between training speed, early stopping, and generalization has yet to be fully understood: \citet{hardt2015train} provide a theoretical and empirical analysis of the relationship between training time and generalization error; in contrast, \citet{hoffer2017train} observe that longer training can improve regularization provided a suitable optimization procedure is used. This is emblematic of a broader lack of granular understanding of generalization in deep neural networks. While many correlates of generalization and techniques to improve generalization are widely used, the precise mechanisms by which they improve generalization have yet to be elucidated. 

\subsubsection{Studying the learning trajectory}
\label{sec:bkgd-trajectory}

The study of how networks evolve over the course of training has revealed a number of intriguing insights which promise to shed light on some of these mechanisms.
A number of works have endeavoured to explain the efficacy of early stopping by arguing that neural networks learn functions of increasing complexity over the course of training \citep{rahaman2019spectral, arpit2017closer, kalimeris2019sgd}, with similar biases towards other notions of simplicity also widely observed in the mapping between parameters and the resulting output functions of neural networks \citep{valle2018deep, de2019random}. 

Approximating a discrete gradient descent trajectory as a continuous-time process in the limit of infinite layer width has facilitated closed-form analysis of the trajectory of deep neural networks \citep{jacot2018neural, lee2019wide}, which has sparked a number of intriguing insights into generalization \citep{arora2019fine}, posterior sampling \citep{he2020bayesian}, and feature learning \citep{yang2021tensor}. The analysis of \citet{jacot2018neural} allows the optimization trajectory of a neural network to be modelled by a kernel gradient descent procedure with respect to a specific kernel called the neural tangent kernel (NTK). \citet{smith2020origin} and \citet{barrett2021implicit} analyze a similar continuous-time approximation of gradient descent to identify an explicit regularization term that biases gradient descent algorithms towards flatter regions of the parameter space, providing insight into the relationship between the flatness of the region of the loss landscape traversed by gradient descent and sources of stochasticity such as minibatch sizes and finite learning rates. 

\subsection{Summary}
As this section has shown, the question of why neural networks generalize has sparked a wide-ranging literature characterizing generalization and optimization dynamics in DNNs. In spite of this rich literature, the answer to this question remains open. Generalization bounds present a domain of great theoretical interest, but fail to \textit{explain} generalization in DNNs. Meanwhile, empirical work has identified many quantities corresponding to model complexity terms in generalization bounds that correlate with generalization \citep{jiang2020fantastic}. However, no single quantity studied thus far has been found to predict generalization performance under \textit{all} possible experimental conditions \citep{dziugaite2020search}. Complementary lines of work have shed significant insight into the optimization landscape and generalization properties of DNNs, but this picture is far from complete. A number of recent theoretical results, such as the analytic form of optimization dynamics in the infinite-width limit of DNNs \citep{jacot2018neural}, the double descent phenomenon \citep{belkin2018reconciling}, and benign overfitting \citep{bartlett2020benign}, hint at a theoretical explanation of the benefits of overparameterization for generalization, but these findings do not currently apply to modern deep learning training regimes. Similarly, empirical investigation into the loss landscapes and gradient structure in DNNs has yielded principled insights into the optimization dynamics of deep networks, but still admits a gap between theory and practice. One contribution of this thesis will be to propose a family of performance estimators based on a network's training speed which are predictive of generalization performance rankings in a number of experimental settings. These estimators will be motivated by a theoretical and empirical analysis linking generalization bounds, invariance, and training speed, providing a potential mechanism to explain their success.

\section{Generalization in reinforcement learning}
The supervised learning regime has proven to be fertile ground for the study of generalization. However, reinforcement learning agents also benefit from accurate extrapolation to unseen inputs, and moreover existing approaches are particularly prone to overfitting \citep{zhang2018study}. Though many of the subproblems of generalization in reinforcement learning mirror those found in supervised learning \citep{cobbe2019quantifying}, the nonstationarity of the training objectives used in deep RL adds additional challenges \citep{igl2021transient} which we will discuss in this section. 

\subsection{Single environment}
Generalization is not always necessary for reinforcement learning: in small state spaces, for example, one may assume that all states the agent will encounter are known in advance and that the value function can be represented as a lookup table; the scope for generalization in this setting is limited to the context of learning an optimal policy under new reward functions \citep{dayan1993improving}. However, in large spaces generalization is crucial for sample efficiency, enabling faster convergence to an optimal policy and, in cases where the state space is so large that not all states will be visited during training, ensuring robust performance in previously-unseen states at evaluation time. This thesis will explore two different perspectives on generalization in a single environment: \textit{state abstractions}, which can be leveraged to develop theoretical insights into the complexity of certain classes of environments, and \textit{interference}, which studies how deep neural networks disentangle and distinguish between states.

\begin{figure}
    \centering
    \includegraphics[width=0.485\linewidth]{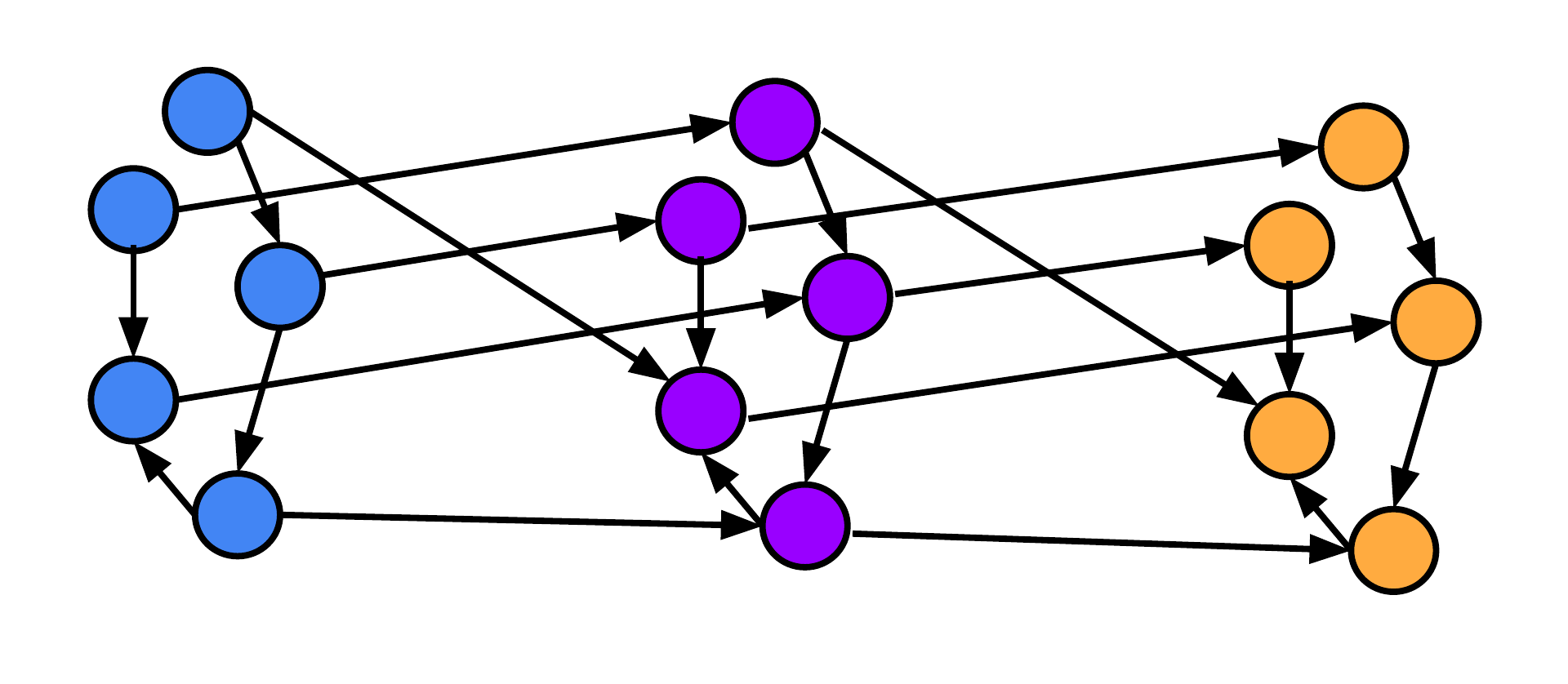}\vline 
    \includegraphics[width=0.485\linewidth]{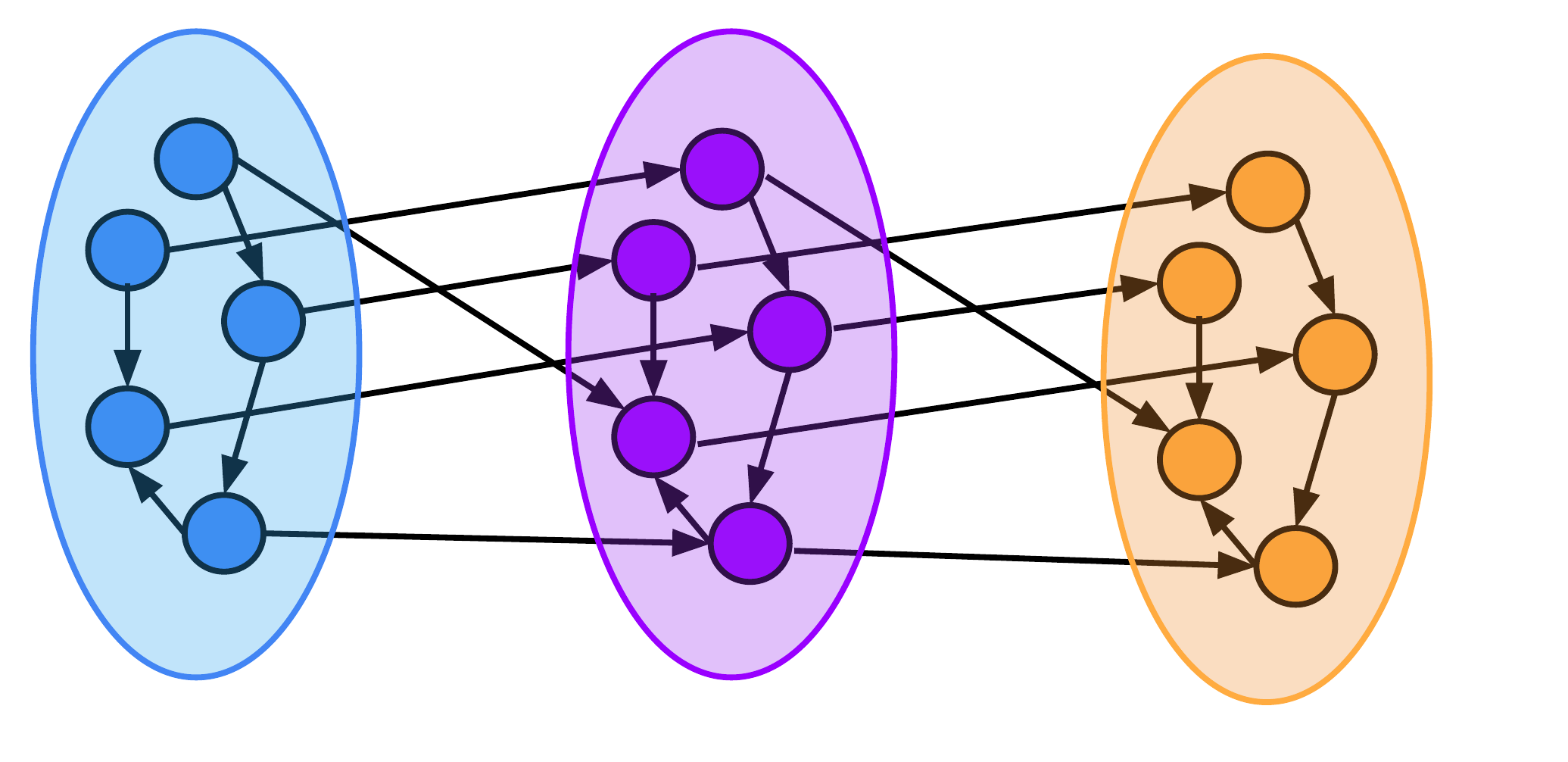}
    \caption[State abstractions allow an agent to treat similar states as though they were equivalent, potentially simplifying planning and generalization.]{State abstractions allow an agent to treat similar states as though they were equivalent, potentially simplifying planning and generalization. Here, we see an example of grouping together states with similar transition dynamics, reducing the size of the state space from 15 states to 3.}
    \label{fig:state_abstractions}
\end{figure}

\subsubsection{State abstractions}
\label{sec:background:stateabstractions}
State abstractions are used in reinforcement learning to simplify planning and to enable generalization in large or infinite state spaces; see Figure~\ref{fig:state_abstractions} for an illustration. Mathematically, a state abstraction is a mapping $\phi: \states \rightarrow \bar{\states}$, where $\bar{\states}$ is some simplified space (for example, $\mathbb{R}^d$). A state abstraction is useful for reinforcement learning if it captures the similarity between states in the environment. That is, $\phi(s) = \phi(s')$ if and only if $s$ and $s'$ have similar values or optimal policies. \citet{li2006towards} characterize a hierarchy of such abstractions, with the finest being a bisimulation of the original MDP and the coarsest being a partition of states according to the optimal action at each state. We will be particularly interested in the former, as our study of representation learning in Chapter~\ref{chp:icp} will focus on this class of state abstractions, known as \textit{model-irrelevance} state abstractions. The following definition is for finite state spaces, but can be generalized to continuous state spaces.

\begin{definition}\label{def:misa}
A \emph{model-irrelevance state abstraction} on a Markov Decision Process $\mdp = (\states, \actionspace, R, P, \gamma)$ is a mapping $\phi: \states \rightarrow \bar{\states}$ such that if $\phi(x_1) = \phi(x_2) = \bar{x}$, the following holds, where $x' \in \states$, and $\bar{x'} \in \bar{\states}$.
\begin{enumerate}
    \item $R(x_1,a) = R(x_2,a)$ $\forall a \in \actionspace$
    \item $\sum_{x' \in \phi^{-1}(\bar{x}')} P(x'|x_1, a) =\sum_{x' \in \phi^{-1}(\bar{x}')} P(x'|x_2, a) $, $\forall a \in \actionspace$
\end{enumerate}
\end{definition}

There are a range of approaches to discovering and selecting state abstractions: \citet{jong2005state} eliminate irrelevant variables from inputs where $\states \subset \mathbb{R}^n$; other approaches leverage statistical \citep{jiang2015abstraction} and information-theoretic \citep{abel2019state} tools to aggregate states, while \citet{taylor2008bounding} construct state abstractions based on state similarity in the MDP. State aggregation is further leveraged in algorithms for RL in rich-observation settings \citep{misra2020kinematic}.

In deep reinforcement learning, the state abstraction framework begins to merge with the feature-learning and representation-learning frameworks of deep learning. One can interpret the latent feature representation of the neural network as a state abstraction, and the remaining layers as a value function approximator. \citet{gelada2019deepmdp} and \citet{zhang2020learning} use ideas from the state abstraction literature to construct representation-learning objectives for DNNs which allow them to capture the structure of the environment. Such structure can be made explicit in the architecture of the neural network used as a function approximator to obtain guarantees on the equivariance of the learned value function to symmetries that arise in the MDP \citep{van2020mdp, van2020plannable}.

\subsubsection{Representations in deep RL}

A major challenge present in deep reinforcement learning is the sparsity of the reward signal relative to the rich structure of the environment. A broad range of auxiliary tasks have been developed over the years to overcome this sparsity in order to improve generalization and representation learning in deep RL \citep{jaderberg2016reinforcement, veeriah2019discovery, gelada2019deepmdp, machado2017eigenoption}. These contributions have in common the goal of enriching the scalar reward signal to encourage the network to pick up additional structure in the environment, either with the downstream objective of using this information later for policy improvement, or as a more stable learning signal for the network. Additional work has analyzed the geometry \citep{bellemare2019geometric} and stability \citep{ghosh2020representations} of the learned features in deep RL agents, along with their linear algebraic properties \citep{kumar2021implicit, gogianu2021spectral}. 

A second challenge in the application of deep neural networks to RL arises from the nonstationarity of the learning objective. The challenge of training a network on a sequence of tasks has been studied in great depth in both reinforcement learning \citep{schaul2019ray, teh2017distral, igl2021transient} and supervised learning settings \citep{sharkey1995analysis, ash2020warm, beck2021effective}. Of particular interest has been the problem of catastrophic forgetting, with prior work proposing novel training algorithms using regularization \citep{kirkpatrick2017overcoming, bengio2013empirical, lopez2017gradient} or distillation \citep{schwarz2018progress, silver2002task, li2017learning} approaches. \citet{benjamin2018measuring}, apply a function-space regularization approach, but require saving input-output pairs into a memory bank. However, in reinforcement learning the converse problem must also be considered: forward-interference, whereby networks may see deteriorating ability to solve later tasks encountered during training. Methods which involve re-initializing a new network have seen particular success at reducing interference between tasks in deep reinforcement learning \citep{igl2021transient, teh2017distral, rusu2016policy, fedus2020catastrophic}.

\subsubsection{Interference}\label{sec:bkgd-interference}
Even in the absence of an explicit abstraction-learning objective, generalization between states in deep RL will naturally arise from the network's inductive bias. In many problem settings, such as in large observation spaces or procedurally generated environments, some degree of generalization is necessary in order to obtain good performance at test time \citep{kirk2021survey}. However, when an update to the predicted value of one state exerts excessive influence on the network's output at other states, this form of generalization (which we will also refer to as \textit{interference}) can lead to instability \citep{jiang2021emphatic, van2018deep}.
The introduction of instability as a result of interference between states is not unique to deep RL; even in the case of \textit{linear} function approximation, off-policy temporal difference learning is not guaranteed to converge to the optimal value function \citep{baird1993advantage, tsitsiklis1996analysis}. Interference has been identified as a barrier to learning progress in both policy-based and value-based problems \citep{schaul2019ray, fedus2020catastrophic}. 

There are two principal approaches by which we will quantify this notion of interference. To a first-order approximation, interference is equivalent to \textit{gradient alignment}. To formalize this concept, we let $f$ be a function which takes as input an observation $\bx$ and a set of parameters $\theta$ (for example, $f$ might be the loss function of a neural network, or one dimension of its output).
We define gradient alignment $\mathrm{I}_\nabla(\bx, \by; \theta)$ (where $\nabla$ indicates that we are studying the \textit{gradient} of the loss) with respect to some loss function and network architecture parameterized by $\theta$ and taking inputs $\bx$ and $\by$ as follows.
\begin{equation}
  \gradint(\bx, \by; \theta) \overset{\text{def}}{=} \langle \nabla_\theta f(\bx; \theta), \nabla_\theta f(\by; \theta) \rangle  \label{eq:gradint}
\end{equation}
We may instead measure the effect of a gradient step computed on data point $(\bx_1, y_1)$ on the loss of another data point $\bx_2 \neq \bx_1$ with corresponding target $y_2$. Letting the targets $y$ be left implicit we obtain the following, where $f$ is some function as before and $\ell$ is the optimization objective (which may or may not be equal to $f$). We let $g_\ell$ be the update computed by some optimizer which may or may not be equal to the gradient $\nabla_\theta \ell(\bx_1, y_1; \theta)$, and which may also depend on the optimizer state $\eta$.
\begin{equation}
   \mathrm{I}_{\Delta}(\bx, \by; \theta)  \overset{\text{def}}{=} f(\by; \theta) - f(\by, \theta') \text{ where } \theta' = \theta + \alpha g_\ell (\bx; \theta, \eta) \label{eq:deltaint}
\end{equation}

Interference has been studied under a variety of names. \citet{fort2019stiffness} refer to the gradient alignment between data points (a measure related to $I_\nabla$ as `stiffness', while \citet{he2019local} call a notion that more closely resembles $\mathrm{I}_\Delta$ `elasticity'.

A number of approaches which specifically reduce the effect of a gradient update for state $s$ on the target $V(s')$ have been shown to improve the stability and robustness of off-policy algorithms \citep{ghassian2020improving, lo2019overcoming}. Prior works have also endeavoured to rigorously define and analyze interference in deep RL \citep{liu2020measuring, liu2020towards, bengio2020interference}, and to study its role in the stability of offline algorithms \citep{kumar2021dr3}.
Similarly, some recent methods \citep{shao2020self, pohlen2018observe} include an explicit penalty which discourages gradient updates from affecting the target values used in TD updates, though this is incidental to the principal contribution of these works. Tuning the degree of interference exhibited by the function approximator to maximize performance thus remains an open problem.

\subsection{Multiple environments}
\begin{figure}
    \centering
    \includegraphics[width=0.85\linewidth]{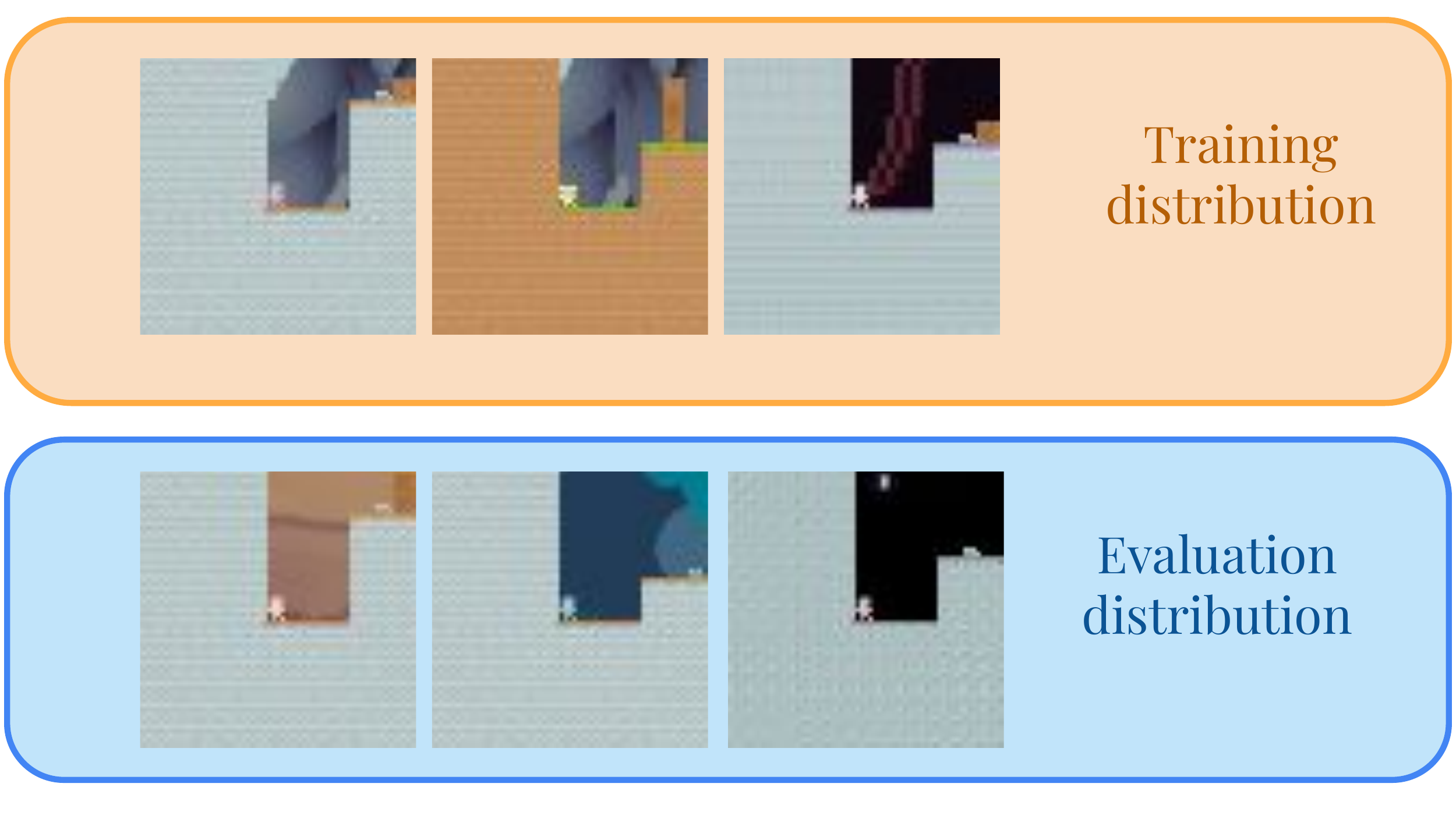}
    \caption{Visualization of multi-environment generalization framework with coinrun levels sampled from the ProcGen benchmark \citep{cobbe2020leveraging}.}
    \label{fig:multi-environment}
\end{figure}
Most standard formulations of generalization in RL do not concern generalization to unseen states within the training environment, but rather generalization to novel environments, MDPs with different observations, transition dynamics, and reward functions than those seen during training \citep{kirk2021survey}. This generalization problem in deep RL can be decomposed into two categories: generalization of the function approximator to novel observations from an MDP which follows the same underlying dynamics as the training environment, and generalization to novel tasks, where some combination of the reward function, observations, and transition dynamics are allowed to vary. 
Many techniques that benefit one of these categories also benefit the other \citep{packer2018genrl}. \citet{farebrother2018generalization} argue, for instance, that an agent which has learned a policy which generalizes well will be robust to changes in the difficulty of a game in addition to more superficial changes in the observation space.

\subsubsection{Formalism}
We will be concerned with the generalization gap incurred by a policy trained on a set of training environments $\mathcal{E}_{\train}$ when deployed to a novel \textit{test} environment.
\begin{equation}
   \mathbb{E}_{\mathcal{E}_{\train}, \pi} [ \sum_{t=0}^\infty \gamma^t R(x_t, a_t)] - \mathbb{E}_{\mathcal{E}_{\test}, \pi}[ \sum_{t=0}^\infty \gamma^t R(x_t, a_t)]
\end{equation}
Typically $\mathcal{E}_{\test}$ will be assumed to share some structure with $\mathcal{E}_{\train}$, but the nature of this shared structure may vary between problem settings. In large observation spaces, $\mathcal{E}_{\test}$ may be equal to $\mathcal{E}_{\train}$ with a different initial state distribution, while for multi-environment problems $\mathcal{E}_{\train}$ may be homomorphic to $\mathcal{E}_{\test}$ \citep{zhang2020invariant}. In the worst case, the evaluation environments may be chosen adversarially so as to make generalization difficult; we will not consider this setting, instead focusing predominantly on test environments that share some characteristics with the training environment.

\textit{Procedurally generated} environments \citep{cobbe2020leveraging, kuttler2020nethack, samvelyan2021minihack} enable us to sample an almost-unlimited number of evaluation environments which are visually diverse but which share important structural characteristics. A visualization of frames from different procedurally-generated levels of the \textit{CoinRun} game is given in Figure~\ref{fig:multi-environment}. Being able to generate an unlimited set of environments brings us closer to the setting of supervised learning, where training and evaluation data points are sampled from the same distribution, by allowing us to sample training and evaluation \textit{environments} from the same distribution. This regime opens interesting research questions concerning the optimal sampling of environments during training so as to maximize the agent's learning progress and generalization performance \citep{jiang2021prioritized, parker2022evolving}. However, the number of training environments in these settings is typically orders of magnitude smaller than the number of training data points seen in standard supervised learning benchmarks, making generalization more challenging.

\subsubsection{Regularization}
Regularization has been widely shown to improve generalization in deep learning \citep{srivastava2014dropout, krogh1991simple}, and its application to reinforcement learning tasks has a rich history extending back to even before the explosion of the deep RL paradigm \citep{farahmand2011regularization}. More recent works focus specifically on the regularization of neural networks trained on RL tasks \citep{cobbe2019quantifying, farebrother2018generalization}, and show that many regularization methods used to improve generalization in supervised deep learning, such as dropout and $\ell_2$ regularization, are also applicable to generalization in reinforcement learning. For example, \citet{schrittwieser2020mastering} use $\ell_2$ regularization in their model-based algorithm which attains state-of-the-art results in many game environments. Many recent works have further sought to leverage the benefits of data augmentation to improve generalization and robustness to input perturbations in reinforcement learning agents \citep{raileanu2021automatic, laskin2020curl,laskin2020reinforcement, hansen2021generalization}. \citet{li2021functional} use low-frequency learned Fourier features to regularize deep Q-networks towards smooth functions. Trust region-based methods further leverage a form of explicit regularization on a learned policy to stabilize learning \citep{schulman2015trust}.

The stochasticity induced by some regularization methods serves a dual purpose as a tool for exploration. \citet{igl2019generalization} use dropout masks to induce stochasticity into the exploratory policy of an agent, while \citet{fortunato2018noisy} apply a Gaussian perturbation to the network parameters. This approach is also leveraged by the Rainbow \citep{hessel2018rainbow} architecture. \citet{goyal2018transfer} use an information bottleneck architecture both to improve generalization and to generate exploration strategies.

\subsubsection{Transfer learning and meta-RL}

In some problem settings, we seek to train an agent that is capable of quickly adapting to novel environments. In the case where we wish to generalize to novel rewards under the same transition dynamics as we saw during training, this can be achieved via the \textit{successor representation} \citep{dayan1993improving, barreto2017successor}. In deep RL, we approximate the successor representation using \textit{successor features}, which have been shown to improve sample efficiency and generalization in both single-agent RL and multi-agent settings \citep{machado2017eigenoption, filos2021psiphi}. Recently, \citet{abdolshah2021new} present an approach to learning successor features which can generalize across different environments.

However, in many cases the novel environment may feature a novels observation space with different transition dynamics from those seen at training. With the advent of powerful function approximation architectures and the development of effective meta-learning algorithms \citep{finn2017model, zintgraf2019fast}, generalization across multiple environments exhibiting less shared structure may also be approached using transfer- and meta-learning techniques. \citet{teh2017distral} and \citet{schmitt2018kickstarting}, for example, use policy distillation to improve transfer learning to new tasks. Other approaches use context-dependent policies \citep{rakelly2019efficient, sodhani2021multi}, or encourage invariance across training environments \citep{zhang2020learning}.

\subsection{Summary}
The study of generalization in the reinforcement learning literature has historically been restricted to the framework of \textit{state abstractions}, which consider arbitrary aggregations of states rather than explicit inductive biases on some input space \citep{li2006towards}. Some attempts have been made to define generalization with respect to input observations in single-environment RL, but these definitions run into limitations due to the nonstationary nature of RL, wherein improving prediction accuracy does not necessarily entail improved performance \citep{bengio2020interference, liu2020measuring}. One contribution of this thesis will be to present explicit definitions of interference and learning capacity in deep RL that are less dependent on the chaotic relationship between performance and prediction error, which will be used to study the evolution of learned representations in deep RL agents. We will explore both the benefits, e.g. improved robustness to input distribution shifts, and the drawbacks, e.g. decreased ability to adapt to new reward signals, of generalization between inputs. This exploration will result in algorithmic tools to accelerate learning and improve generalization in deep RL agents. While recent analysis has studied interference between different regions of the state space as a source of instability and a cause of performance plateaus, this line of work largely ignores the beneficial effects of generalization. 

In contrast, the benefits of generalization are widely acknowledged in the context of multi-environment RL, where an agent is evaluated on previously-unseen environments at test time. However, the current state of the art in this regime requires prohibitive sample complexity in order to generalize well outside of their training set \citep{cobbe2019quantifying}, largely due to the weak assumptions placed on the generative process from which environments are sampled.
While some approaches from the supervised learning literature, such as weight decay, dropout, and data augmentation, have been shown to improve generalization in deep RL, these approaches do not have as great an effect as in supervised learning \citep{raileanu2021automatic}. Orthogonal approaches from the meta- and transfer-learning literature have seen more success as they allow for fine-tuning on the test set, though these methods encounter similar sample complexity issues when applied to zero-shot generalization. In Chapter~\ref{chp:icp} we will consider a characterization of families of environments in which generalization from a handful of training environments is theoretically possible, and present algorithms for identifying or learning state abstractions for which an optimal policy will generalize zero-shot to any novel environment from this family.

\chapter{The role of invariance in generalization}
\label{chp:invariance}
\minitoc 

\section{Introduction}
The first object of interest in our study of learning dynamics and generalization will be \textit{invariance}. 
Invariance arises naturally as a part of the data generating process in many real-world datasets, such as point clouds (invariant to permutation), speech recognition (invariant to the average pitch of the speaker's voice), natural language (different sentences may be semantically equivalent), and objects in natural images (often rotation- and translation-invariant). In many instances, an invariance can be precisely characterized in terms of the action of a specified group, meaning that the input distribution is partitioned into equivalence classes, for example all rotations of an image or all permutations of elements in a set. This structure yields straightforward desiderata on the generalization properties of a model trained to fit such a data-generating distribution: learning about one element of an equivalence class should be sufficient for the model to correctly identify any other element in the equivalence class. In this sense invariance presents an idealized form of \textit{interference}: updates to an invariant model's output on one element of an equivalence class will necessarily have an identical effect across all inputs in this class.

It is not surprising then that models which capture the invariant structure of a given distribution tend to perform better than those which do not on a variety of tasks \citep{Cohen:Welling:2016, fawzi2016adaptive, salamon2017deep}. Yet while the empirical benefits of invariant models have been widely corroborated, work which seeks to explicitly quantify and understand how invariance can benefit generalization is scarce. 
Part of this challenge stems from the breadth of approaches by which invariance can be incorporated into a model class. One can build the invariance into the network as a convolution or weight-tying scheme, average network predictions over transformations of the input (feature averaging), or simply train on a dataset augmented with these transformations (data augmentation). Comparison of these different approaches requires different tools, as each incorporates confounding factors into the learning process which are distinct from the effect of the training method on invariance.

This chapter will employ a diverse range of techniques to gain insight into how invariance benefits generalization. The analysis presented here will lay the foundation for our study of more nebulous forms of inductive bias in Chapter~\ref{chp:supervised}, and our study of generalization in reinforcement learning in Chapters~\ref{chp:gen-rl} and \ref{chp:icp}. We will provide both an empirical and theoretical analysis contrasting neural networks trained to exhibit approximate invariance via data augmentation and networks which exhibit exact invariance via architectural design or feature averaging. Our empirical analysis will study the degree to which data augmentation yields models whose invariances generalize outside of the training set, evaluating the effect of these methods on interference between inputs in the same equivalence class. We further characterize conditions under which data augmentation will provably induce convergence to an invariant model parameterization in Section~\ref{sec:da-optim}, and study the differential effects of data augmentation and feature averaging on {interference} in Section~\ref{sec:variance-reduction}.   

In our theoretical analysis, we will shed light onto the role of invariant structure in PAC-Bayes generalization bounds. Our main contribution will be a ranking of a set of PAC-Bayes bounds on models trained with data augmentation and with feature averaging, resulting in the practical conclusion that feature averaging at test time improves upon data augmentation, which in turn outperforms training on data sampled independently from $\dgd$.
We find that the differential effects of feature averaging and data augmentation on interference mirror those on the model complexity term in our PAC-Bayes bound. This analogy between model simplicity and interference in the setting of group-theoretic invariances grounds our further study of interference in the learning problems of Chapters~\ref{chp:supervised}, \ref{chp:gen-rl}, and \ref{chp:icp} where more nuanced inductive biases must be identified.

\section{Background} \label{sec:background-invar}

``Invariance'' has been used to describe a number of related but distinct phenomena in the machine learning and statistics literature concerning the stability of a function's output under some set of transformations of its input \citep{zou2012deep, raj2017orbit_embeddings, jaderberg2015spatial, van2018learning, chen2019invariance}. 
We focus on invariance under the action of a group $\grp$, a setting shared by several prior works \citep[e.g.,][]{Cohen:Welling:2016,Kondor:Trivedi:2018,invariantdistributions}. The {action} of $\grp$ on a set $\calX$ is a mapping $\alpha : \grp \times \calX \to \calX$ which is compatible with the group operation, i.e. ${\alpha(e, x) = x}$ and ${\alpha(g_1, \alpha(g_2, x)) = \alpha(g_1g_2, x)}$. We will abbreviate this operation as $\alpha(g,x) = gx$ when the form of $\alpha$ is not important. The {orbit} of any $x\in\calX$ is the subset $\grp_x$ of $\calX$ that can be obtained by applying an element of $\grp$ to $x$, $\grp_x = \{ gx : g \in \grp \}$. We note that the sets $\grp_x$ are equivalence classes which partition the underlying set $\calX$, and will refer to them as the orbits of $\grp$.

For mathematical simplicity, we assume $\grp$ to be compact, with
(unique) normalized Haar measure denoted by $\haar$.\footnote{$\haar$ is analogous to the uniform distribution on $\grp$ \citep{haar1933massbegriff}.} Critically, any finite group is compact under the discrete topology. For the purposes of this chapter, the reader may therefore substitute `compact' with `finite' and `Haar measure' with `uniform distribution over $\grp$' and preserve the correctness (though not the full generality) of the theoretical results.

We denote a random element of $\grp$ by $G$. 
A mapping $f : \calX \to \calY$ is {invariant} under $\grp$ (or $\grp$-invariant) if
\begin{align} \label{eq:invariant:function}
  f(gx) = f(x) \;, \quad \forall \, g\in \grp,\ x \in \calX \;.
\end{align}
The act of transforming an arbitrary function $f: \calX \to \bbR$ to be $\grp$-invariant is known as {symmetrization}, and can be performed by averaging the output of $f$ over the orbits of $\grp$. We define this symmetrization operator $\symm_{\grp}$ as follows:
\begin{align} \label{eq:symmetrization:def}
  \invf{f}(x) := \symm_{\grp}f(x) = \bbE_{G\sim\haar}[f(G x)] \;, \quad x \in \calX \;.
\end{align}
While other symmetrization methods exist, such as taking the supremum over the group's orbit, we will always use $\invf{f}$ to refer to the \textit{average} of $f$ over the orbit of its input.
We consider a typical machine learning scenario, with a training data set $\trdata$ of $n$ observations $(X_i, Y_i)_{i=1}^n\in (\calX,\calY)^n$ sampled i.i.d.\ from some (unknown) probability distribution $\dgd$. 
Additionally, $\dgd$ \emph{is assumed to be $\grp$-invariant}, i.e.
\begin{align} \label{eq:invariant:dgd}
  \dgd(gX, Y) = \dgd(X ,Y ) \;, \quad g \in \grp \;.
\end{align}
For example, $X$ may be an image of an animal, $Y$ a label of the animal, and $\grp$ the group of two-dimensional rotations. 

Given a group action on the space $\cX$, we can partition it into disjoint \textit{orbits} of the form $
\{g\Orbit : g \in \grp \}$, where $\Orbit \in \cX$ denotes an orbit representative. The $\grp$-invariance of $\dgd$ enables the decomposition of $\dgd$ into a distribution over orbits $P_{\Orbit}$ along with a conditional distribution $P_{X|\Orbit} = \haar(g)$ where $g\Orbit = X$ \citep[see, e.g.,][]{invariantdistributions}. Letting $G \sim \haar(\grp)$, this entails $(X,Y) \equdist (G\Orbit,Y)$ and $\dgd = P_{\Orbit} \times P_{X | \Orbit} \times P_{Y|X}$. In other words, the distribution $\dgd$ can be interpreted as first sampling an equivalence class $\Orbit$, then sampling an element $X$ uniformly at random from the set $\Orbit$, and then sampling $Y$ based on the conditional distribution $\dgd(Y|X)$.

This decomposition will play a key role in our analysis in later sections, as it will allow us to iterate expectations over $\dgd$ as
  \begin{align} \label{eq:iterated:expectations}
  	& \bbE_{(X,Y)\sim\dgd}[f(X,Y)] = \bbE_{\Orbit\sim P_{\Orbit}}[ \bbE_{G\sim\haar(\grp)}[\bbE_{Y\sim P_{Y|X}}[f(G\Orbit,Y) \mid \Orbit,Y] \mid Y ] ] \;. \nonumber
  \end{align}
  This will allow us to marginalize over and condition on orbits of $\grp$, making it possible to reason about the expected behaviour of a function within and between orbits.
  
\subsection{Modes of invariance}

When the data-generating distribution and target function are known to be $\grp$-invariant, incorporating this invariance into the model class is an intuitive approach to accelerate convergence and improve generalization. This can be done in a variety of ways.

\textbf{Trained invariance} is implemented as data augmentation (DA) \citep{fawzi2016adaptive,Cubuk2018}: elements $G_{ij}$ of $\grp$ are applied to each observation $X_i$ of the training data, with the label $Y_i$ left unchanged. The result is an augmented dataset $\trdata_{\grp} = ( (G_{ij}  X_i, Y_i)_{j \leq m} )_{i\leq n}$ used to minimize the augmented empirical risk
\begin{equation} \label{eq:aug:risk}
  \eRiskAug(f,\trdata) = \frac{1}{n}\sum_{i=1}^n  \bbE_{G\sim\haar}[\loss(f(G  X_i),Y_i) ] \end{equation}
which can be approximated using samples as follows 
\begin{equation}\label{eq:approx:aug:risk}  \eRiskAug(f,\trdata) \approx \frac{1}{nm} \sum_{i=1}^n \sum_{j=1}^m \loss(f(G_{ij} X_i),Y_i) = \eRiskAugMC(f, \trdata) \; .
\end{equation}
DA is now a standard method in the practitioner's toolkit \citep{iyyer-etal-2014-neural,zhou2015predicting,salamon2017deep, machinehealth}, and recent theoretical work has further established connections to variance reduction methods \citep{kerneltheory}. 
A major benefit of data augmentation is its versatility: while our study in this chapter focuses on augmentations that can be expressed as the action of a group, data augmentation can also be used to promote other properties of the target function such as smoothness by leveraging a broader class of transformations. 

\textbf{Architectural invariance} restricts the function class being learned to contain only invariant functions, typically through either feature averaging (FA) or symmetric function composition. FA symmetrizes the output of an arbitrary neural network by computing an average over $\grp$ at one or more layers, such that the overall network is invariant under the action of $\grp$. 
In practice, averaging is typically done at the penultimate or final layer, resulting in a $\grp$-invariant network ${f}^{\circ}$. A network $f$ with $D$ layers is written as the composition of $h_D \circ \dots \circ h_1$, with the shorthand $h_{d}^{d'}$ referring to the composition of layers $d$ through $d'$. The empirical risk of a network with FA at layer $d$ evaluated on $\trdata$ is
\begin{align*} 
  \eRisk(\invf{f},\trdata) 
    = \frac{1}{n}\sum_{i=1}^n \loss\big(h_{d}^D\circ  \bbE_{G\sim\haar}[ h_{1}^{d-1} (G X_i) ], Y_i \big) \;.
\end{align*}

The development of \textit{invariant} or \textit{equivariant} neural network architectures enables users to avoid the potentially costly computation of this expression for a number of different group invariances. A body of literature of varying degrees of generality has developed to propose new architectures which characterize invariant function classes under a variety of group actions, and to study their properties \citep{Wood:ShaweTaylor:1996,ravanbakhsh2017equivariance,Kondor:Trivedi:2018,invariantdistributions,cohenetal2019generaltheory, Cohen:Welling:2016}.

\subsection{PAC-Bayes generalizations bounds}
\label{sec:pac:bayes}
As discussed in Section~\ref{bkgd:generalization-bounds}, PAC-Bayes bounds characterize the risk of a randomized prediction rule; the randomization can be interpreted as a posterior distribution over functions $Q$ that may depend on $\trdata$, though this distribution need not correspond to the exact posterior of a Bayesian model. A PAC-Bayes bound on generalization error is typically expressed in terms of a sum of the empirical risk and the Kullback-Leibler (KL) divergence \citep{kullback1951} between $Q$ and a fixed prior distribution $P$, which as mentioned in Chapter~\ref{chp:background} can be interpreted as characterizing the \textit{complexity} of functions in the support of $P$.

We define the risk of a probability distribution $Q$ defined on a function class $\fclass$ in the natural way,
\begin{equation*}
    \risk(Q) = \mathbb{E}_{f \sim Q} \risk(f) \quad \text{ and analogously } \quad \eRisk(Q) = \mathbb{E}_{f \sim Q} \eRisk(f)  \; .
\end{equation*}

PAC-Bayes bounds are highly correlated with generalization performance \citep{jiang2020fantastic}, and in some cases provide non-vacuous results for even highly over-parameterized neural networks \citep{dziugaite2017nonvacuous,dziugaite2018dependent,zhou2018nonvacuous}.
The following is a standard bound due to \citet{catoni}, which holds for general data generating distributions $\dgd$ evaluated on the 0-1 loss.
\begin{theorem}[\citet{catoni}]\label{thm:catoni:bound}
  Let $\mathcal{D}^n$ be sampled i.i.d. from $\dgd$, and let $\loss$ be the binary 0-1 loss. 
  For any prior $P$ and any $\delta\in (0,1)$, with probability $1-\delta$ over samples $\trdata$,  for all posteriors $Q$ and for all $\beta > 0$,
	\begin{equation} \label{eq:catoni:bound}
	     \risk(Q) \leq \frac{ 1 - e^{-\beta \eRisk(Q, \trdata) - \frac{1}{n}(\KL{Q}{P} + \log \frac{1}{\delta})} }{1 - e^{-\beta}} \;.
	\end{equation}
\end{theorem}

For bounded loss functions, analogous bounds rely on a sum of the empirical risk and a function of the KL divergence \citep[see, e.g.,][]{dziugaite2017nonvacuous,dziugaite2018dependent}. We state all results only for variations of Catoni's bound \eqref{eq:catoni:bound}, but analogous findings hold for bounds of the form proposed by \citet{mcallester1999}. 

\section{Invariance and optimization dynamics}
\label{sec:invar-optim}
We begin with an analysis of the manner in which data augmentation and feature averaging influence the optimization dynamics and resulting function output by gradient descent. This analysis will yield two main insights, both concerning the setting of convex losses: first, that while data augmentation can cause gradient descent to converge to a $\grp$-invariant solution in certain realizable settings, in general it leads to networks that are only approximately invariant. In fact, these networks can overfit to the training dataset such that their variance over orbits of out-of-distribution data increases over the course of training. Second, we will show that feature averaging, in addition to ensuring $\grp$-invariance, reduces the variance of both the loss and its gradients. We will relate this finding to the notion of gradient interference discussed in Section~\ref{sec:bkgd-interference}, illustrating the effect of trained and architectural invariance on interference within and between orbits of a group $\grp$.

\subsection{Data augmentation and trained invariance}
\label{sec:da-optim}
To begin our study of optimization dynamics and invariance, we lay out some basic properties of the symmetrized risk $\eRiskAug$ and its interaction with invariant models. When the loss function is convex, Jensen's inequality can be applied to the augmented risk to compare DA and FA risk estimates. We note that the relevant notion of convexity to this line of argument is with respect to the loss function's first argument, not the parameters of the model. Many (though not all) widely used loss functions are convex in this way (e.g., squared error, cross-entropy). We provide counterexamples highlighting the necessity of this property in Appendix~\ref{appx:counterexamples}.

\begin{restatable}{proposition}{propEmpRisk} \label{prop:empirical:risk:order}
  Let $\loss : \bbR \times \bbR \to \bbR_+$ be a loss function that is convex in its first argument. 
  Then for any $f : \calX \to \calY$, 
  \begin{align*} 
    \eRisk(\invf{f},\trdata) = \eRiskAug(\invf{f},\trdata)  \leq \eRiskAug(f,\trdata) \;,
  \end{align*}
  and therefore analogous inequalities hold for $\eRisk(\invf{Q},\trdata)$, $\risk(f)$, and $\risk(Q)$. 
  Furthermore, if $\loss(f(\argdot),\argdot) \in L_2(\dgd)$ (i.e., has finite second moment), then
  \begin{align*}
  	\Var_{\trdata\sim\dgd^n}\big[ \eRisk(\invf{f},\trdata)  \big] \leq \Var_{\trdata\sim\dgd^n}\big[ \eRiskAug(f,\trdata)  \big] \;.
  \end{align*}
\end{restatable}  
The proof of this result can be found in Appendix~\ref{apx:invar-proofs}; it follows from Jensen's inequality and the invariance of the data-generating distribution to augmentations. The beneficial effect of an invariant function on the expected risk stems from its reduced variance over orbits. This observation can be flipped to characterize the detrimental effect of incorporating an invariance in the model which is not present in the data: by mapping all elements of an orbit to the same output, the model's expected risk can be lower bounded by a function of the average variance of the true labels over orbits. 

In principle, symmetrizing the risk function via data augmentation should regularize the training procedure towards invariant functions; however, in most settings it is not possible to guarantee that the learned function will be exactly $\grp$-invariant, particularly on inputs that differ from those seen during training. Neural networks trained with data augmentation will in general exhibit low but nonzero variance over equivalence classes, and their failure to generalize this low variance to out-of-distribution is illustrated in Figure~\ref{fig:ood}. One notable setting in which it \textit{is} possible to obtain such guarantees is linear function approximation with groups admitting a linear representation over the input space. To state the result, we define a data-generating distribution $\dgd$ supported on the set $\calX \times \calY$, where $\calY \subseteq \mathbb{R}$ and $ \calX \subseteq V \simeq \mathbb{R}^d$ is a subset of a $d$-dimensional vector space over $\bbR$, $V$, with dual vector space $V^*$. We assume that $\calX$ spans $V$. Furthermore, let $\grp$ admit a linear representation \citep{serre1977linear}, $\rho : \grp \to GL(V)$, with corresponding dual $\rho^*_g = \rho_{g^{-1}}^{\top}$. Finally, we assume that $\dgd$ is invariant to the action of $\grp$ and that $\mathbb{E}[y | \bx] = \langle \bw^*, \bx \rangle$ for some unique $\bw^*$.

\begin{restatable}{proposition}{propSymmGD}\label{theorem:symmgd}
   Let $\dgd$ and $\bw^*$ be defined as above, and let $\ell(\bw) = \mathbb{E}_{\bx, y \sim \dgd}[ (\langle \bw, \bx \rangle - y)^2]$. Then the (global) minimizer $\bw^*$ of $\ell$ satisfies $\rho_g^* \bw^*= \bw^*$ for $\haar$-almost all $g\in\grp$. In particular, the function $f(\bx) = \langle \bw^*, \bx \rangle$ is $\grp$-invariant.
\end{restatable}
In brief, this result shows that in the case of linear models, data augmentation with respect to certain types of symmetry can induce \textit{exact} invariance in the learned function. The proof of Propsition~\ref{theorem:symmgd} can be found in Appendix~\ref{apx:invar-proofs}. It follows straightforwardly from the condition that the group act linearly on the input space along with the linearity of the function $f(\bx) = \langle \bw, \bx \rangle$, as in this case the symmetrization of the parameters $\bw$ is equivalent to the symmetrization of the function, so by Jensen's inequality any minimizer of $\ell$ must be invariant to the group action. Under suitable step-size conditions (e.g., the Robbins--Monro conditions) this proposition implies that SGD will converge to an invariant set of weights. Thus, to learn a predictor that exhibits the desired invariance on the entire dataset, it is sufficient to train with SGD on augmented data with a convex loss.


\begin{figure}[t]
    \centering
    \includegraphics[width=0.58\textwidth]{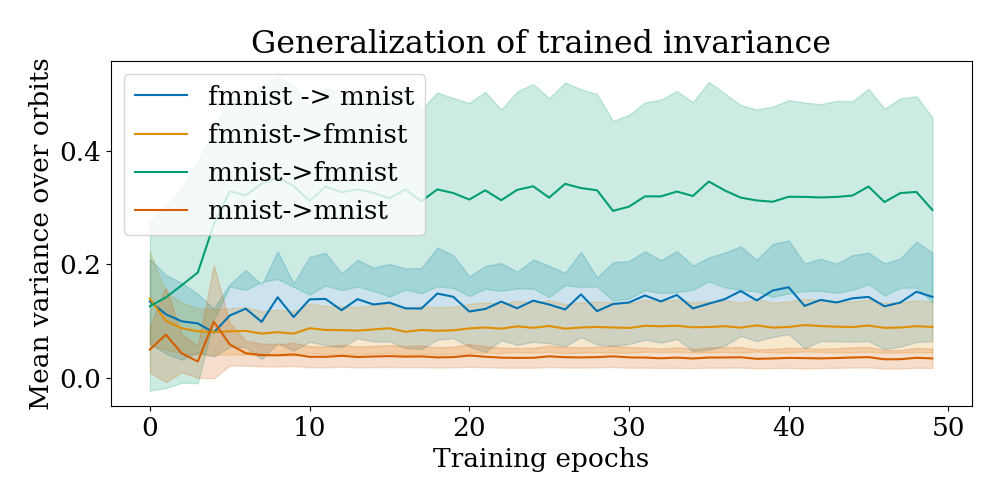}
    \caption[A visualization of a fully connected multi-layer perceptron.]{Variance of predictions with respect to 90 degree rotations of an input over the course of training. Evaluations performed on either the same dataset as used during training, or on out-of-distribution images. Labels are of the form: training set $\rightarrow$ evaluation set, where \texttt{fmnist} refers to FashionMNIST and \texttt{mnist} to MNIST.}
    \label{fig:ood}
\end{figure}

 Once non-linearity is introduced into the model architecture ,Propsition~\ref{theorem:symmgd} will no longer hold, and in the worst case data augmentation may simply result in the function approximator memorizing an invariance on its training set. The example depicted in Figure~\ref{fig:ood} demonstrates such a failure: the learned function appears to capture the target invariance on the training data, but, having not learned the appropriate symmetry in weight space, fails to generalize this invariance to out-of-distribution data. We measure the generalization of a learned invariance by computing the variance of the network outputs over orbits of data drawn from a dataset distinct from the training distribution. In this case, we train fully connected neural networks using DA on one of two related datasets: MNIST and FashionMNIST ($28 \times 28$ pixel black and white images of handwritten digits and clothing categories respectively), each augmented by rotations of multiples of 90 degrees. We then evaluate the variance of the outputs over orbits (rotations by 0, 90, 180, and 270 degrees) in the test set. Finally, we evaluate the two networks on orbits in the complementary dataset. In one sense the results are predictable: networks attain lower variance over orbits of within-distribution as compared to out-of-distribution data points. Somewhat surprising is the discrepancy in out-of-distribution generalization between models trained on MNIST and FashionMNIST. As the FashionMNIST inputs are more visually diverse than the MNIST inputs, one explanation of this phenomenon is that the learned invariances are capable of interpolating between training points but struggle to extrapolate, for some loose notion of interpolation and extrapolation. 

\keyinsight{Data augmentation can lead to invariant parameter configurations in restricted settings, but is prone to memorize invariances on its training distribution, leading to out-of-distribution generalization failures.}



\subsection{Symmetrization}
\label{sec:variance-reduction}
Unlike the approximate invariance induced by data augmentation, architectural invariance will trivially generalize to any data-generating distribution. Architectural invariance may also improve computational efficiency by reducing the number of parameters needed to fit the training data. 

\subsubsection{Feature averaging and Jensen's inequality}

\begin{figure*}[bt]
    \begin{center}
    \includegraphics[width=0.43\textwidth]{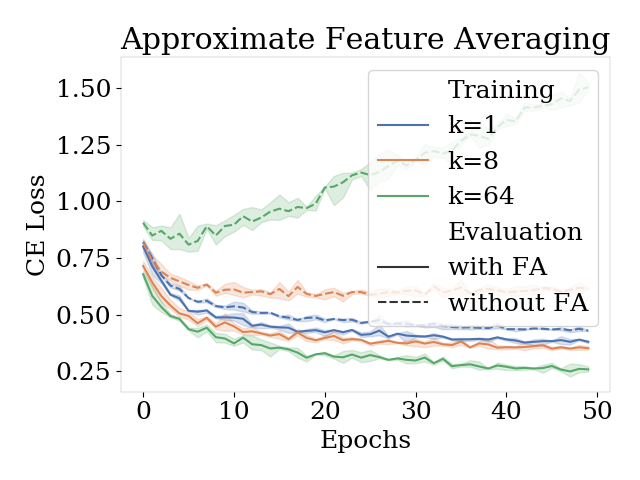}
    \includegraphics[width=0.42\textwidth]{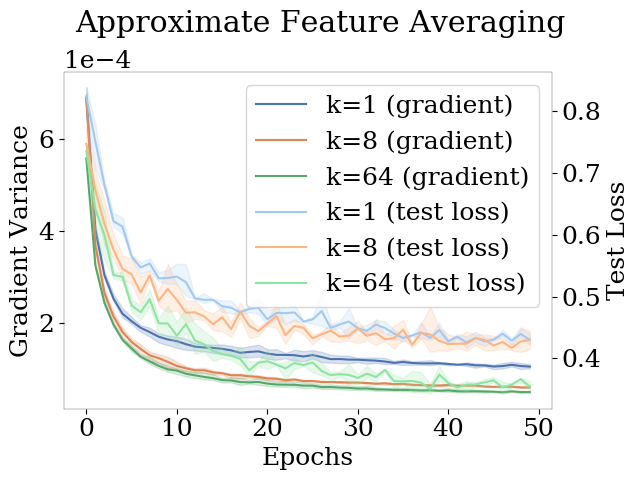}
    \caption[Measurements over the course of training a CNN using different DA and FA approaches.]{Measurements over the course of training a CNN using different DA and FA approaches. \emph{Left}: Models are trained with approximate FA using $k$ rotations sampled from the range $\{1, \dots, 360\}$, and then evaluated with and without approximate FA. Note that for $k=1$, the approximate versions of DA and FA are equivalent. \emph{Right}: Average per-epoch gradient variance and test loss in the same setting. }
    \label{fig:acc}
    \end{center}
\end{figure*}

Jensen's inequality suggests that incorporating FA at evaluation time will improve performance on convex losses. It does not, however, directly yield insights into whether it is worthwhile to also deploy FA during training. Feature averaging, even when performed approximately, changes the mapping between network parameters and functions, and it is not immediately obvious whether this will benefit the optimization dynamics of the learning process. One potential detractor of training with FA is that the network may not learn parameters which can attain a low loss in its absence.

\hypothesis{H1}{models trained with feature averaging will be more dependent on FA to attain a low loss compared to those trained with data augmentation. \label{hyp:variance}}

To test Hypothesis~\ref{hyp:variance}, we train a series of convolutional neural networks on an augmentation of the FashionMNIST dataset. The labels of this dataset are invariant to rotations: put simply, there is no way of rotating a shoe such that it can be mistaken for a t-shirt. We augment FashionMNIST via the group $\grp$ of rotations in the set $\{1^\circ, \dots, 360^\circ\}$, and perform approximate FA with $k$ samples, where $k << |\grp|$.
We see in Figure~\ref{fig:acc} that feature averaging profoundly changes the optimization dynamics of the network. The left plot demonstrates that the model trained with FA becomes increasingly dependent on feature averaging to obtain a low loss: the loss of the network's output in the absence of FA increases during training, particularly for larger $k$ (better approximations of FA); it is only when averaging over orbits that the network attains the lowest loss. The right hand side confirms that the models trained with feature averaging exhibit a variance reduction in their gradients and a reduced test loss. The models trained with approximate FA thus converge to regions of parameter space which are distinct from those learned by data augmentation, and which do not yield accurate predictions in the absence of FA. 

\subsubsection{Invariance and interference}

A hallmark of memorization occurs when the network's predictions on its training data match the target function, but its gradients are orthogonal, suggesting it is not able to use information from one input to improve its predictions on others. We therefore study the alignment of gradients, i.e. \textit{interference}, within and between orbits of $\grp$; this will amount to generalizing the quantity $\gradint$ described in Equation~\ref{eq:gradint} to apply to \textit{sets} of states. Colinearity of gradients is a much stronger notion of invariance than predictive variance over an orbit. While the latter can easily be set to zero by simple memorization, gradient alignment requires that the network correctly generalize changes in its predictions from one input to others in its orbit. The gradients of a network evaluated on inputs from the same equivalence class under FA will by necessity be colinear; we now investigate whether this property also holds in models trained with DA.

We study gradient colinearity by computing the rank of the matrix of gradients of inputs computed on a minibatch $\bX$ of independently sampled data points, along with the rank of the matrix of gradients of inputs in a batch which contains the full orbit of a subset of these inputs $\bX_{\Phi}$. We subsample $\bX$ when generating $\bX_{\Phi}$ to ensure that both batches contain the same number of inputs. Concretely, given a batch of data $\bX, \by$ of size $n$ and parameters $\theta \in \mathbb{R}^d$ we are interested in the matrix $\mathbf{I}_\nabla(\bX) \in \mathbb{R}^{n \times d}$ defined as follows
\begin{equation}
    [\mathbf{I}_\nabla(\bX)]_{i,j} \overset{\text{def}}{=} \nabla_{\theta_j} \ell (f_\theta(\bX_i), \by_i)
\end{equation}
where $\ell$ is a loss function and $f_\theta$ is the function computed by the neural network with parameters $\theta$.  
By considering $\bX$ sampled independently and $\bX_{\Phi}$ containing entire equivalence classes, we can distinguish between gradient interference over the whole dataset, and gradient interference specifically along an orbit. An invariant neural network architecture will produce the same gradient for every input from a given equivalence class, and thus the rank of $\mathbf{I}_\nabla(\bX_{\Phi})$ will be upper bounded by the number of distinct equivalence classes. A non-invariant model might in the worst case assign orthogonal gradients to elements of the same equivalence class; however, such a model could also assign colinear gradients to inputs from different equivalence classes, thus attaining a low rank on $\bX_{\Phi}$ despite exhibiting no invariance. In this case, the network would also attain low rank on $\bX$, giving us a means of identifying whether a low gradient rank is due to interference across or within orbits. 

\begin{figure}
    \centering
    \includegraphics[width=\textwidth]{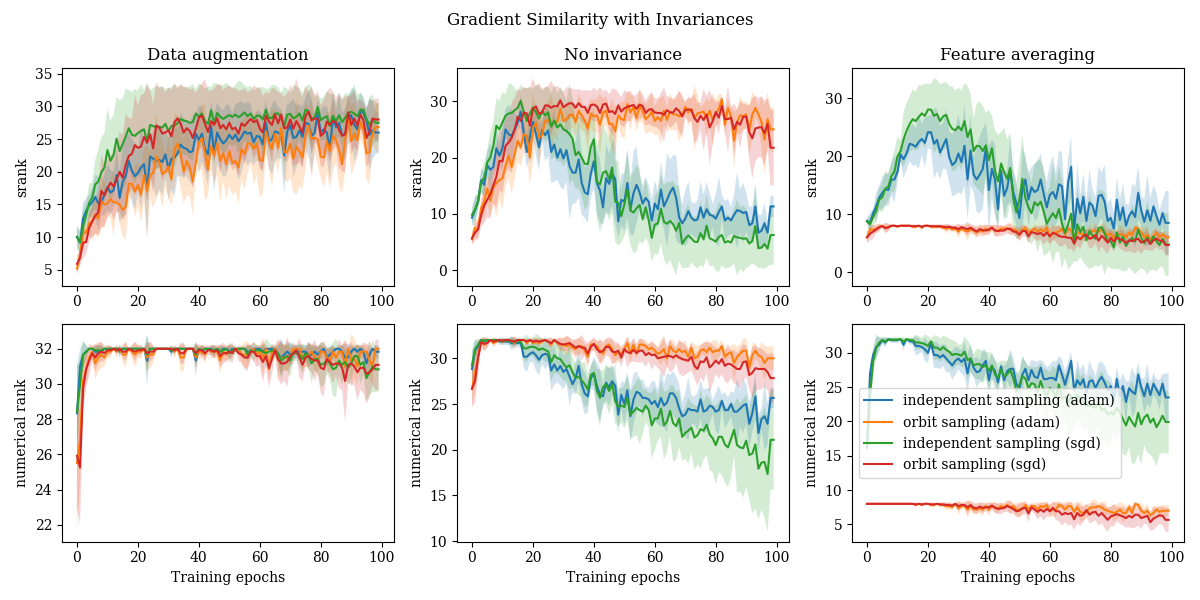}
    \caption[Comparing the ranks of gradients computed on minibatches consisting of 32 independent random input samples, or 8 random samples from the training set augmented with all 90 degree rotations (resulting in 32 non-independent samples).]{Comparing the ranks of gradients computed on minibatches consisting of 32 independent random input samples, or 8 random samples from the training set augmented with all 90 degree rotations (resulting in 32 non-independent samples). While the gradients for inputs from the same orbit are colinear for the feature averaging networks, as evidenced by the low rank of the gradients of augmented minibatches compared to those on independently sampled inputs, the networks without this architectural invariance produce linearly independent gradients. Indeed, the model trained only on non-rotated images exhibits greater independence on rotations than it does on independent random samples from the training set.}
    \label{fig:cifar-invariances}
\end{figure}

In order to be robust to small perturbations of the matrix $\mathbf{I}_\nabla$, we use the numerical rank~\citep{golub1976rank} and srank~\citep{kumar2021implicit}, which count the number of (resp. normalized by the maximum) singular values of a matrix that lie above some threshold $\epsilon$. In our evaluations we set $\epsilon=0.01$ in accordance with \citet{kumar2021implicit}. Note that the singular values of $\mathbf{I}_\nabla$ are equal to the eigenvalues of the matrix whose entries are given by the dot product $\gradint(\bx_i, \bx_j)$. We plot these values over the course of training in Figure~\ref{fig:cifar-invariances}. Our training data consists of the CIFAR-10 dataset augmented by rotations of 90 degrees and the evaluations are run on both stochastic gradient descent (SGD) and adaptive (Adam \citep{kingma2014adam}) optimizers. For computational tractability, we use a batch size of 32 and subsample parameters in the network to compute the gradient dot products. Despite the batch size being orders of magnitude lower than the number of parameters, we still find that the gradient matrices in most cases do not exhibit full row rank. 

We see in Figure~\ref{fig:cifar-invariances} that networks trained with FA naturally have colinear gradients over orbits, while networks trained without FA do not. In particular, the network trained with DA does not appear to distinguish between elements from the same equivalence class and elements from different classes, at least as far as gradient alignment is concerned. Intriguingly, the network trained without data augmentation exhibits \textit{higher} gradient rank on the augmented minibatches than on the independently sampled data, exhibiting the opposite trend of the network trained with FA. In combination with Proposition~\ref{theorem:symmgd}, this suggests that while data augmentation is capable of learning invariant functions over its training set, it does not capture the same degree of invariance in its gradient structure as FA.

\section{Invariance and generalization bounds}
\label{sec:data:augmentation}
Our previous analysis showed that data augmentation alone is not sufficient to induce invariances in a model that generalize to arbitrary inputs. Instead, models trained with data augmentation have the potential to learn an invariance only on the training distribution, in a sense memorizing the invariant structure. We now explore whether this phenomenon is mirrored in generalization bounds. We will focus in particular on PAC-Bayes bounds, which can be decomposed into a \textit{data fit} and a \textit{model complexity} term. The previous section showed that symmetrization improves the data fit (when this term is given by a convex loss function) of a function class on invariant data-generating distributions. In this section, we will show that invariance can also reduce a notion of model complexity, and combine these two findings to show that invariance leads to lower PAC-Bayes generalization bounds. These results will be complementary to our previous analysis of the training dynamics of neural networks trained with data augmentation and feature averaging. 



\subsection{Validity of the augmented risk}

The PAC-Bayes bound in Theorem~\ref{thm:catoni:bound} holds for a binary classification loss. However, the average of the 0-1 loss over an orbit will lie in the continuum, i.e. $\bbE_{G\sim\haar}[\loss(f(GX),Y)] \in [0,1]$. Monte Carlo approximations of $\bbE_{G\sim\haar}[\loss(f(GX),Y)]$ also violate the assumptions of Theorem~\ref{thm:catoni:bound} because the augmented data set is not identically distributed (note that sampling two elements of the same orbit in the augmented dataset will be guaranteed, whereas this event will have vanishing probability under $\dgd$).
Encouragingly, the following theorem shows that neither of these issues prevents us from leveraging the augmented risk in PAC-Bayes bounds.

\begin{restatable}{theorem}{thmpacbayesda}
  \label{thm:pac:bayes:da}
  Assume that $\dgd$ is \ginv. Then Theorem~\ref{thm:catoni:bound} holds with $\eRiskAug(Q,\trdata)$ as in \eqref{eq:aug:risk} substituted for $\eRisk(Q,\trdata)$.
\end{restatable}

The proof of Theorem~\ref{thm:pac:bayes:da} can be found in Appendix~\ref{apx:invar-proofs}. The key step in obtaining this result is to define a convex function which measures the difference between $\eRisk(f)$ and $\risk(f)$ (which can be translated to a bound on the generalization gap), and then apply Jensen's inequality to show that substituting $\eRiskAug$ for $\eRisk$ in this expression will yield a lower bound on this difference, which can then be used to obtain the bound of Theorem~\ref{thm:catoni:bound}.
\subsection{Invariance and model simplicity}
\label{sec:invar-gap}
In the case of architectural invariance, particularly via FA, we can go beyond the validity of the PAC-Bayes bound and show a provable reduction in the model complexity term, resulting in a {symmetrization gap} in the PAC-Bayes bound. Given a function class $\fclass$, we consider the set $\invf{\fclass}$, denoting the class of $\grp$-invariant functions obtained by symmetrizing the functions belonging to $\fclass$. Symmetrization is a surjective map from $\fclass$ to $\invf{\fclass}$ by its definition. However, symmetrization is not necessarily injective, meaning that $\invf{\fclass}$ is in a sense smaller than $\fclass$, and so a distribution over this set should be \textit{simpler}. 
It is then intuitive that symmetrization should also reduce the model complexity term appearing in a generalization bound. This section will characterize these intuitions rigorously. In order to do so, we must first construct an analogous symmetrization operator on probability measures as we saw for functions in \eqref{eq:symmetrization:def}. 
This is fortunately straightforward: 
for any probability measure $P$ on $\fclass$, the induced probability measure on $\invf{\fclass}$ is the image of $P$ under $\symm_{\grp}$, $\invf{P} = P \circ \symm_{\grp}^{-1}$. Intuitively, incorporating invariances into an architecture should reduce its complexity by forcing $P$ and $Q$ to agree on this property; this intuition is quantified in the following lemma.


\begin{restatable}{lemma}{lempushforwardKL} \label{lem:pushforward:KL}
  Suppose that $(E_i,\calE_i)$, $i=1,2$, are two measurable spaces, the second of which is standard, $\mu$ and $\nu$ are two probability measures on $(E_1,\calE_1)$, and $\psi : (E_1,\calE_1) \to (E_2,\calE_2) $ is a measurable map. Then
  \begin{align} \label{eq:kl:inequality:pushforward}
  	\KL{\mu\circ\psi^{-1}}{\nu\circ\psi^{-1}} \leq \KL{\mu}{\nu} \;.
  \end{align}
  Furthermore, if $\mu\ll \nu$ with density $m$, then $\mu\circ\psi^{-1} \ll \nu\circ\psi^{-1}$ with density $m_{\psi}$, and the {$\psi$-gap} is
  \begin{align}
  	\Delta_{\psi}(\mu\ ||\ \nu)  : & = \KL{\mu}{\nu} - \KL{\mu\circ\psi^{-1}}{\nu\circ\psi^{-1}} \nonumber \\
  	  & = \int_{E_1} \mu(dx) \log \frac{m(x)}{(m_{\psi}\circ\psi)(x)} \;. \label{eq:psi-gap}
  \end{align}
\end{restatable}

We provide a sketch here for intuition, and defer the full proof of Lemma~\ref{lem:pushforward:KL} to Appendix~\ref{apx:invar-proofs}. The lemma leverages the chain rule of the relative entropy to decompose the KL divergence between two distributions into a component which marginalizes over equivalence classes of $\psi$ in the space $E_1$, and one which conditions on them. We will use the pushforward measure $\mu \circ \psi^{-1}$ (and anlogously $\nu \circ \psi^{-1}$), defined over sets in $\calE_2$ as $(\mu \circ \psi^{-1})(A) = \mu(\psi^{-1}(A))$, to `marginalize' over equivalence classes. Concretely, for probability measures $P$ and $Q$ on the measurable space $\mathcal{E}_1 \times \mathcal{E}_2$, under mild assumptions to guarantee that the following terms are well-defined, we have 
\begin{equation} \small
    \KL{P(X_1,X_2)}{Q(X_1,X_2)} = \KL{P(X_1)}{Q(X_1)} + \KL{P(X_2|X_1)}{Q(X_2|X_1)} \; .\label{eq:simple-chain-rule}
\end{equation}

By constructing the joint space $(\calE_1,\calE_2)$ to correspond to $(\invf{\fclass}, \fclass)$, we can design a probability measure on this space with the property that the left hand side of the equality in \eqref{eq:simple-chain-rule} will be equal to $\KL{\mu}{\nu}$, while the right hand side will consist of the sum $\KL{\mu \circ \psi^{-1}}{\nu \circ \psi^{-1}}$, the KL divergence of the pushforward measures $\mu \circ \psi^{-1}$ and $\nu \circ \psi^{-1}$, and $\Delta_{\psi}(\mu \ll \nu)$, which measures the degree to which the density of $\mu$ varies over equivalence classes of $\psi$. If the density $m$ is constant wherever it is not zero, then this gap will be zero.

\textbf{The symmetrization gap.} Lemma~\ref{lem:pushforward:KL} simply states that when we compress the input space $E_1$ via a mapping $\psi$, we analogously compress probability distributions over this space and obtain a subsequent reduction in the KL divergence of the corresponding pushforward measures. The precise gap in the KL divergence has the closed form \eqref{eq:psi-gap} which measures the degree to which the measure $\mu$ varies over the equivalence classes induced by $\psi$. This indicates that symmetrization can reduce the KL divergence term in the PAC-Bayes bound \eqref{eq:catoni:bound}, as we show in the following theorem. 

\begin{theorem} \label{lemma:KL:gen}
  Let $\mathcal{X}$ be a compact metric space and $\mathcal{Y}$ a Polish space, $\grp$ a group acting measurably on $\mathcal{X}$, and $\fclass = C(\mathcal{X},\mathcal{Y})$ the class of continuous functions $\mathcal{X}\to \mathcal{Y}$.\footnote{The result can hold for other function classes $\fclass$; the key requirement is that conditioning is properly defined in $\fclass$ and $\invf{\fclass}$.}  
  Let $Q$ and $P$ be probability measures on $\fclass$ such that $Q \ll P$ with density $q$, and $\invf{Q}\ll\invf{P}$ (density $\invf{q}$) their images under $\symm_{\grp}$ on $\invf{\fclass}$. Then 
  \begin{equation*}
    \KL{\invf{Q}}{\invf{P}} \leq \KL{Q}{P} \;.
  \end{equation*}
  Furthermore, the {symmetrization gap} is 
  \begin{align} \label{eq:symm:gap}
  	\invf{\Delta}(Q \ ||\ P) = \bbE_{f\sim Q}\bigg[ \log \frac{q(f)}{\invf{q}(\symm_{\grp} f)}  \bigg] \;.
  \end{align}
\end{theorem}
\begin{proof}[Proof of Theorem~\ref{lemma:KL:gen}]
  The proof of this result follows straightforwardly from Lemma~\ref{lem:pushforward:KL}; we need only show that the conditions of the lemma hold. To see this, we note that for $\calX$ a compact metric space and $\calY$ a Polish space, the space $\fclass = C(\calX,\calY)$ of continuous functions $f : \calX \to \calY$ is a Polish space, and therefore it (along with its Borel $\sigma$-algebra $\borel(C(\calX,\calY))$) is a standard Borel space. 
  For a group $\grp$ acting measurably on $\calX$, the symmetrization operator $\symm_{\grp} : \fclass \to \invf{\fclass}$ is measurable, and the product space $(\fclass \times \invf{\fclass}, \borel(\fclass)\otimes\borel(\invf{\fclass}))$ is a standard Borel space. Thus, the conditions of Lemma~\ref{lem:pushforward:KL} are satisfied and the result follows.
\end{proof}



In practice, computing the expectation over orbits required by exact feature averaging may be intractable. Instead, one may sample a set of $k$ transformations with which to average the function output. While this will not output the exact expectation, it still takes advantage of a simplification of the function space via Lemma~\ref{lem:pushforward:KL}, by aggregating functions that have some probability of being mapped to the same approximately averaged function. To formalize the idea, let $g^k = \{g_1,g_2,\dotsc,g_k\}$ be a set of $k$ elements of $\grp$, and $G^k$ a random realization sampled i.i.d.\ from $\haar$. Let $\symm_{g^k} f(x) = k^{-1} \sum_{j\leq k}f(g_j x)$ denote the approximate symmetrization of $f$ by $g^k$. Finally, let $\invfMC{Q}_{g^k} = Q \circ \symm_{g^k}^{-1}$ denote the image of a distribution $Q$ on $\fclass$ under $\symm_{g^k}$. The following result is a consequence of the fact that Lemma~\ref{lem:pushforward:KL} is true for every $g^k$, and that for $g^{k+1}=g^k\cup\{g_{k+1}\}$, $\symm_{g^{k+1}}f(x) = f(g_{k+1}x) + \frac{k}{k+1}\symm_{g^k}f(x)$.

\begin{proposition} \label{prop:kl:chain}
  Assume the conditions of Lemma~\ref{lemma:KL:gen}. Let $G_{s}=G_1,G_2,\dotsc$ be a sequence of elements sampled i.i.d.\ from $\haar$. Then with probability one over $G_s$, 
  \begin{align*}
  	 \KL{Q}{P} & \geq \KL{\invfMC{Q}_{G^1}}{\invfMC{P}_{G^1}} \geq \dotsb \\ 
  	   & \geq \KL{\invfMC{Q}_{G^k}}{\invfMC{P}_{G^k}} \geq \dotsb \\
  	   & \geq \KL{\invf{Q}}{\invf{P}} \;.
  \end{align*}
\end{proposition}
As the results of Section~\ref{sec:da-optim} suggest, the interplay between SGD and symmetrization remains an open question. However, Proposition~\ref{prop:kl:chain} makes it clear that at least at evaluation time, feature averaging is preferred. Indeed, an intriguing feature of this result is that the model simplicity improvement in the PAC-Bayes bound holds even if FA was not applied during training, meaning that the final bound benefits from both the variance reduction on the empirical risk and the $\psi$-gap provided by symmetrization. 


\keyinsight{Feature averaging, even when only performed approximately, reduces the model complexity term in PAC-Bayes generalization bounds. The \textit{symmetrization gap} quantifies the degree to which feature averaging simplifies the resulting function class.}

Though its exact computation will be intractable for FA in neural networks, the symmetrization gap in the theoretical bounds corresponds to real improvements of generalization, as shown by our empirical analysis in Section~\ref{sec:empirical-pacbayes}. However, the question of causality remains open. Do invariant models generalize better \textit{because} of the reduced KL penalty? Or do they generalize better and obtain better generalization bounds for some other reason? Understanding the answer to this type of question has been the focus of the long line of literature outlined in Section~\ref{sec:background-science}. It is likely that the underlying mechanism relating the two also appears in our analysis from Section~\ref{sec:da-optim}: the symmetrization gap arises because an invariant function's output on one element of an equivalence class uniquely determines its output over the whole class. This makes the class of functions computed by a given network architecture smaller, which is reflected in the reduced KL divergence. A similar notion of simplicity when FA is used during training can be seen in the gradient structure illustrated by Figure~\ref{fig:cifar-invariances}, where an update to one element of an equivalence class automatically generalizes to all other elements in that class. 

\subsection{Ordering of PAC-Bayes bounds}
\label{sec:ordering}

Combining the previous results yields an ordering of the PAC-Bayes generalization upper bounds. Let $B_0$ be the upper bound on the right-hand side of \eqref{eq:catoni:bound}, with $B_{\text{\rm DA}}$ and $B_{\text{\rm FA}}$ corresponding to the upper bounds for DA (using the augmented empirical risk $\eRiskAug(Q,\trdata)$) and FA (using $\KL{\invf{Q}}{\invf{P}}$), respectively. Finally, let $B_{\text{\rm DA *}}$ denote the computationally intractable bound for DA given in Appendix~\ref{appx:tighter:pacbayes:da}. 

\begin{theorem} \label{thm:bound:order}
  Assume the conditions of Theorem~\ref{thm:catoni:bound}, and also that $\dgd$ is \ginv. Then $B_{\text{\rm FA}} \leq B_{\text{\rm DA *}} \leq B_{\text{\rm DA}} = B_{0}$.
\end{theorem}

The proof of Theorem~\ref{thm:bound:order} follows directly from our previous results. Importantly, this ordering on upper bounds does not imply a strict ordering on generalization error. It is nonetheless encouraging, as prior work has shown that many design choices which reduce PAC-Bayes risk bounds result in a reduced generalization gap. We investigate whether this finding holds for feature averaging in Section~\ref{sec:empirical-pacbayes}. 

\subsection{PAC-Bayes bounds for invariant architectures}
\label{sec:empirical-pacbayes}

We conclude with a demonstration of the effect of invariance on PAC-Bayes bounds for neural networks. We use the ModelNet10 dataset, which consists of LiDAR point cloud data for 10 classes of household objects. This dataset exhibits permutation invariance: the LiDAR reading is stored as a sequence of points defined by $\{x,y,z\}$ coordinates, and the order in which the points are listed is irrelevant to the class.  We consider three different architectures: a PointNet-like architecture \citep{qi2017pointnet}, which is invariant to permutations; a partitioned version of the PointNet architecture which is invariant to subgroups of the permutation group (details in Appendix~\ref{appx:invar-pb}); and a fully connected model where the invariant pooling operation in the PointNet is replaced by a fully connected layer. The invariance in the network is implemented via a max-pooling layer instead of an averaging layer and so is not a direct application of feature averaging; however, the results of \eqref{eq:kl:inequality:pushforward} still apply due to the injective mapping over functions induced by max-pooling.

\begin{table}
	\caption{Generalization performance for a permutation-invariant point cloud classification task (see text for details).}
    \label{table:decomp}
    \begin{center}
    \resizebox{0.75\textwidth}{!}{
    \begin{tabular}{ccccc}
    \toprule
         \textbf{Network} & \textbf{Train} & \textbf{Test} &  \textbf{KL} &  \textbf{PAC-Bayes} \\
          & \textbf{Error} & \textbf{Error} & \textbf{Divergence} & \textbf{Bound} \\
         \midrule
         Fully connected & 0.002 & 0.65 & 24957 & 1.75 \\
         Partial-PointNet & 0.172 & 0.248 & 1992 & 0.67\\
         PointNet &  0.24 & 0.245 & 944 & 0.533 \\
         \bottomrule
    \end{tabular}
    }
    \end{center}
\end{table}
We compute the PAC-Bayes bounds following the procedure in \citet{dziugaite2017nonvacuous}: we convert a deterministic network to a stochastic network by adding Gaussian noise to the weights, and then train this stochastic model using a differentiable surrogate loss that bounds the true PAC-Bayes bound. After this training procedure converges, we then compute the PAC-Bayes bound. We attain an ordering consistent with our theoretical results: the invariant architecture attains the lowest bound, followed by the partially invariant architecture, and finally followed by the fully connected network. We provide a decomposition of the distinct terms in the bound in Table~\ref{table:decomp}. While the fully-connected architecture obtains the lowest loss on the training set, it significantly overfits to this data and we see a large generalization gap which is mirrored in the PAC-Bayes bound. In contrast, the invariant architectures exhibit a reduced PAC-Bayes bound and improved accuracy on the test set, despite obtaining a greater loss on the training set.

\section{Conclusions} \label{sec:conclusion}

This chapter has posed a simple question: how do different approaches to incorporating invariance into a model influence its generalization performance? To answer this question, we have proposed a novel PAC-Bayes generalization bound which applies to models trained with or without data augmentation. Our theoretical analysis gives the answer that invariant model classes are simpler, where this change in simplicity can be quantified by the relative `size' of the invariance relative to the size of the data generating distribution. We have further characterized settings under which gradient descent optimization will converge to solutions which satisfy the invariances  exhibited by the data generating distribution. 

To gain deeper insight into the interplay between invariance and optimization, this chapter further presented an empirical analysis of the effects of data augmentation and feature averaging on the optimization dynamics of neural networks.
For non-convex problems, this analysis reveals that networks develop a notion of approximate invariance that holds on held-out points from the data-generating distribution, but can exhibit increasing variance on the orbits of out-of-distribution data points over the course of training, suggesting that neural networks are able to pick up approximate invariance properties that generalize well within distribution but do not correspond to architectural invariance. Our analysis in this setting has also revealed a variance reduction effect of both data augmentation and feature averaging.

The contributions of this chapter exhibit two principal limitations: we consider only model selection problems concerning the selection of exact invariances, and our generalization bounds are not sufficiently tight to be useful for model selection in most settings of interest. The former allowed us to obtain a precise characterization of the benefit of incorporating symmetries into a model class, but limited the scope of inductive bias which we could consider. The latter is a generic failing of current generalization bounds for neural networks, as outlined in Section~\ref{sec:background-science}. While realizations of these bounds were correlated with generalization in our experiments, the bounds are not intended to directly translate to a model ranking. The following chapter will address both of these issues by performing \textit{marginal likelihood estimation} rather than generalization bound computation, and will be applicable to a broad range of hyperparameter and inductive bias selection problems beyond symmetries.

\chapter{Training speed and model selection}
\label{chp:supervised}
\minitoc

\section{Introduction}
\label{sec:introduction-supervised}
\begin{figure}
    \centering
     \includegraphics[width=0.65\linewidth]{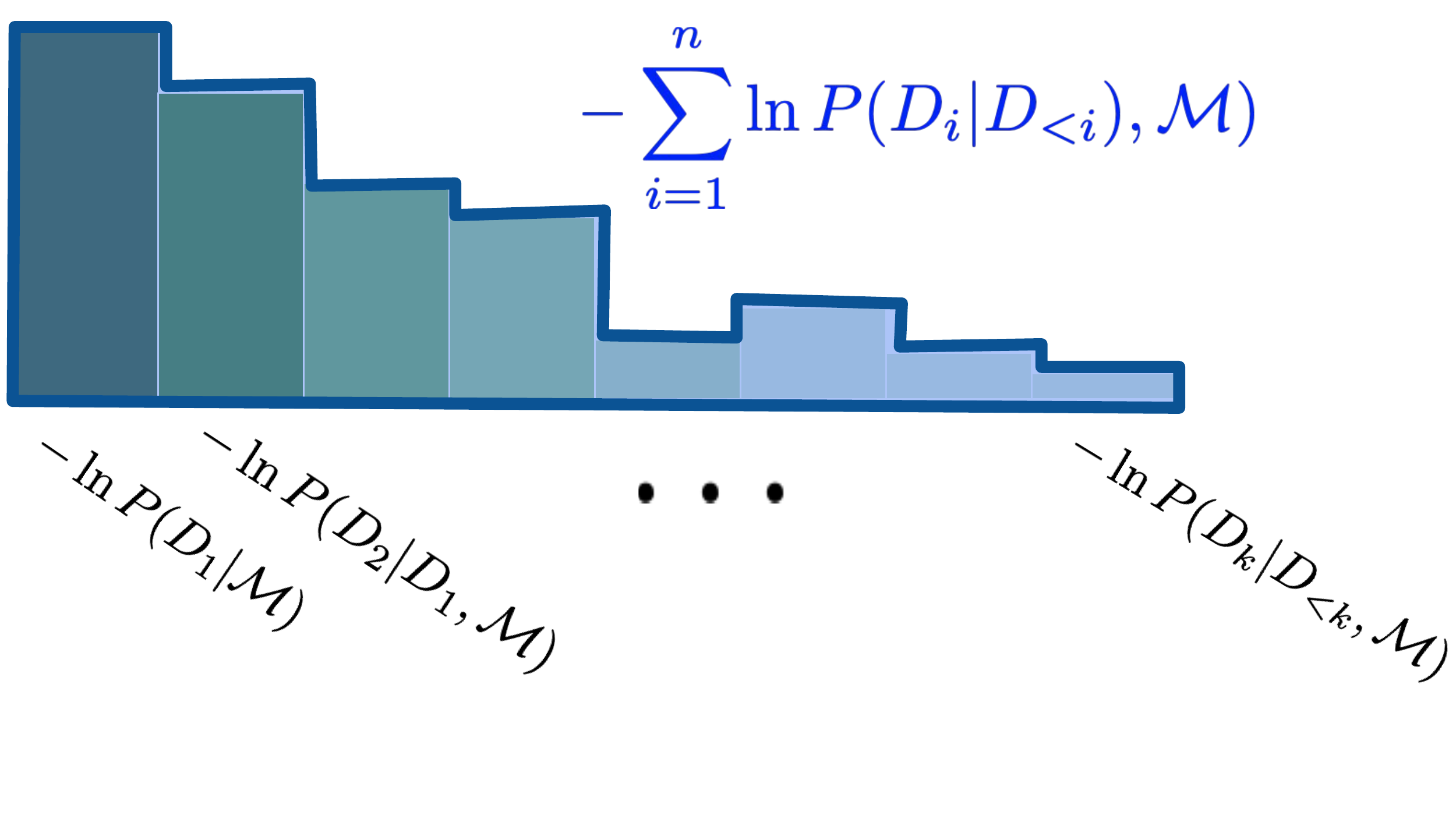}
    \caption[Visualization of the relationship between the Bayesian marginal likelihood and training speed.]{Visualization of the relationship between the Bayesian marginal likelihood and training speed. The log predictive likelihood of each data point $D_i$ in a dataset after conditioning the model $\model$ on earlier indices $D_{<i}$ is plotted on the left hand side in green. The shaded blue area on the right corresponds to the area under this curve, i.e. value of the negative log marginal likelihood. This yields the interpretation of the negative log ML as a measure of training speed.}
    \label{fig:sotl-auc}
\end{figure}
As we saw in the previous chapter, choosing the right inductive bias for a machine learning model, in particular incorporating {symmetries} exhibited by the data, is critical for good generalization. The problem of \emph{Bayesian model selection} concerns itself with identifying good inductive biases for a given dataset, including but not limited to symmetries. Model selection arises in both Bayesian inference, which seeks to identify a probabilistic model that best captures some distribution $\dgd$, and in deep learning in the form of neural architecture search. This chapter will show that computing a measure of \textit{training speed}, the area under a model's learning curve, provides a principled approach to Bayesian model selection. Further, we will show that an analogue of this measure adapted to deep learning, the area under a network's training curve, produces competitive performance estimators for neural architecture search.

Critical to obtaining these results is our attention to the Bayesian learning paradigm, which is outlined in Section~\ref{sec:background-bayes}. In Bayesian inference, the marginal likelihood (ML) provides a principled tool for model selection. In contrast to cross-validation, for which computing gradients is cumbersome, the ML can be conveniently maximised using gradients when its computation is tractable.
Unfortunately, computing the marginal likelihood for complex models such as neural networks is typically \textit{in}tractable. Workarounds such as variational inference suffer from expensive optimization of many parameters in the variational distribution and differ significantly from standard training methods for deep neural networks, which optimize a single parameter sample from initialization. A method for estimating the ML that closely follows standard optimization schemes would pave the way for new practical model selection procedures, yet remains an open problem.

The Bayesian model selection perspective is complementary to the study of PAC-Bayes generalization bounds as a tool with which to rank models \citep{dziugaite2020search}, though the two are deeply connected \citep{germain_pac-bayesian_2016}. In parallel, theoretical analysis has hinted at a connection between training speed and generalization error \citep{hardt2015train}, though the resulting generalization bounds from this analysis are only valid for extremely short training budgets. These results, while providing insight into the stability of stochastic gradient descent, do not present an \textit{explanation} (in the sense of that described in Section~\ref{sec:background-science}) of the empirical observation that models which train faster tend to generalize better. This is because the stability-based generalization bounds of \citet{hardt2015train} implicitly propose a notion of hypothesis class complexity that depends on the number of steps taken in the optimization trajectory, a measure which becomes vacuous after the first training epoch. However, \citet{jiang2020fantastic} show strong empirical performance of some measures of training speed at predicting generalization performance in DNNs, suggesting that this relationship may provide clues towards understanding generalization in deep learning.

The principal contribution of this chapter is to show that Bayesian model selection, training speed, and generalization are in fact deeply connected. To do so, we take a prequential coding perspective on the marginal likelihood, framing the log ML as the sum of predictive log likelihoods of data points, conditioned on preceding data in the dataset. This perspective reveals a family of estimators of the log ML which depend only on predictions sampled from the posterior of an iterative Bayesian updating procedure. 
We study the proposed estimators in the context of linear models, where we can conclusively analyze their theoretical properties.
Leveraging the fact that gradient descent can produce exact posterior samples for linear models \citep{matthews2017} and the infinite-width limit of deep neural networks \citep{matthews2018gaussian,lee2018deep}, we show that this estimator can be viewed as the sum of a subset of the model's training losses in an iterative optimization procedure. 
This immediately yields an interpretation of marginal likelihood estimation as measuring a notion of training speed in linear models, visualized in Figure~\ref{fig:sotl-auc}. 

We demonstrate the utility of this estimator through empirical evaluations on a range of model selection problems, confirming that it can effectively approximate the marginal likelihood of a Bayesian model in a manner useful for model selection. We go on to investigate whether our theoretical results for linear models may have explanatory power over generalization in deep learning. We construct and justify an analogue of our training speed estimator which can be applied to neural networks trained with stochastic gradient descent, and show that this estimator is predictive of final test accuracy in neural architecture search problems. These findings suggest that studying generalization between points on a training set can be informative of generalization to unseen data, an observation that will inspire our study of interference in reinforcement learning in Chapter~\ref{chp:gen-rl}.

\section{Background on Bayesian inference}
\label{sec:background-bayes}
Philosophically, Bayesian model selection addresses the question \textit{which out of this set of models is most likely to have generated my data?} This superficially resembles empirical risk minimization, which seeks to answer the question \textit{which set of {(hyper-)}parameters minimizes the loss function on my training data?} However, whereas empirical risk minimization is a principle for selecting a single predictor with the goal of minimizing its risk on the data-generating distribution, Bayesian model selection concerns itself with fitting probability distributions over parameters and corresponding functions, without explicit concern for the model's future predictions. 

\subsection{Bayesian modelling}

A Bayesian model $\model$ is defined by a prior distribution over parameters $\theta$, $P(\theta | \model)$, and a prediction map from parameters $\theta$ to a likelihood over the data $\data$, $P(\data|\theta, \model)$.
Parameter fitting in the Bayesian framework entails finding the posterior distribution $P(\theta|\data)$, which yields robust and principled uncertainty estimates. 
Though exact inference is possible for certain models such as Gaussian processes (GPs) \citep{rasmussen2003gaussian}, it is intractable for DNNs \citep{neal2012bayesian}. Here approximate methods such as variational inference \citep{blei2017variational, gal2016dropout, blundell2015weight,mackay1992bayesian, graves2011practical, duvenaud2016early} and Laplace approximations \citep{daxberger2021laplace, mackay1998choice} are used to improve robustness and obtain useful uncertainty estimates. 

Variational approximations require optimization over the parameters of the approximate posterior distribution. 
This optimization over distributions changes the loss landscape, and is significantly slower than the pointwise optimization used in standard DNNs.
Pointwise optimization methods inspired by Bayesian posterior sampling can produce similar uncertainty estimates as variational inference, while improving computational efficiency \citep{welling2011bayesian,mandt2017stochastic,maddox2019simple}.
An appealing example of this is ensembling \citep{lakshminarayanan2017simple}, which works by training a collection models in the usual pointwise manner, starting from $k$ independently initialized parameter values.

In the case of linear models, marginalization over an ensemble of models trained with gradient descent is exactly equivalent to Bayesian inference, as this sample-then-optimize approach yields exact posterior samples \citep{matthews2017, osband2018randomized}. \citet{he2020bayesian} extend this approach to obtain posterior samples from DNNs in the infinite-width limit.

\subsection{Model selection}
\label{sec:bayesian-model-selection}
In addition to finding model parameters, Bayesian inference can also perform \textit{model selection} over different inductive biases, which are specified through both model structure (e.g. convolutional vs fully connected Bayesian neural network architectures) and the prior distribution on parameters. The Bayesian approach relies on finding the posterior over models conditioned on the data $P(\model|\data)$, which uses the \emph{marginal likelihood} (ML) as its likelihood function:
\begin{equation}\label{eq:marg-lik}
P(\data | \model) = \int_\theta P(\data|\theta)P(\theta|\model_i)d\theta = \mathbb{E}_{P(\theta|\model)} [P(\data | \theta)] \,.
\end{equation}
Instead of computing the full posterior, it is common to select the model with the highest marginal likelihood. This is known as type-II maximum likelihood \citep{mackay1992bayesian,mackay2003information} and is less prone to overfitting than performing maximum likelihood over the parameters and model combined. This is because the marginal likelihood is able to trade off between model fit and model complexity \citep{rasmussen2001occam}, while addressing a problem with fewer degrees of freedom than that of parameter fitting.
Maximizing the ML is standard procedure when it is easy to compute. For example, in Gaussian processes it used to set simple model parameters like smoothness \citep{rasmussen2003gaussian}, while recent work has demonstrated that complex inductive biases in the form of invariances can also be learned \citep{van2018learning}.

For many deep models, computing Equation~\ref{eq:marg-lik} is intractable, and obtaining approximations that are accurate enough for model selection and that scale to complex models is an active area of research \citep{khan2019approximate}. In general, variational lower bounds that scale are too loose when applied to DNNs \citep{blundell2015weight} to provide useful estimates for model selection. Deep Gaussian processes provide a case where these bounds do work \citep{damianou13a,dutordoir20a}, but heavy computational load holds performance several years behind that of deep learning. While ensembling methods provide useful uncertainty estimates and improve the computational efficiency of the variational approach, they have not yet provided a solution for Bayesian model selection.

Further work on generalization in deep learning has examined training dynamics as a tool for understanding how deep neural networks generalize. Many of these works point obliquely at a connection with the marginal likelihood. For example, \citet{arora2019fine} obtain generalization bounds for shallow ReLU neural networks that bear striking resemblance to the marginal likelihood of a Gaussian process with kernel given by the neural tangent kernel \citep{jacot2018neural}. 
The marginal likelihood also appears in other analyses of stochastic gradient descent. \citet{duvenaud2016early} use a single run of stochastic gradient descent to approximate the marginal likelihood of a Bayesian model, and then use tools for model selection to define a criterion for early stopping. Similarly, \citet{smith2018} use the marginal likelihood to identify the optimal batch size for a learning problem, using the width of the minimum found by stochastic gradient descent as a Laplace approximation for the marginal likelihood of the model.

\section{Marginal likelihood estimation with training statistics}\label{sec:ml-estimation}
In this section, we investigate the equivalence between the marginal likelihood (ML) and a notion of training speed in models trained with an exact Bayesian updating procedure. For linear models and infinitely wide neural networks, exact Bayesian updating can be performed using gradient descent optimization. For these cases, we derive an estimator of the marginal likelihood which
\begin{enumerate}
\item is related to how quickly a model learns from data,
\item only depends on statistics that can be measured during pointwise gradient-based parameter estimation, and
\item becomes tighter for ensembles consisting of multiple parameter samples.
\end{enumerate} We also investigate how gradient-based optimization of a linear model combination can implicitly perform approximate Bayesian model selection in Appendix~\ref{sec:optimize-then-prune}.

\subsection{Training speed and the marginal likelihood} \label{sec:decomposing_ML}
\label{sec:speed-and-ml}
We begin by developing intuition on the marginal likelihood. Let $\data$ denote a dataset of the form $\data = (\data_i)_{i=1}^n = (x_i, y_i)_{i=1}^n$, and let $\data_{<i}=(\data_j)_{j=1}^{i-1}$ with $\data_{<1}=\emptyset$.
We will abbreviate $P(\trdata|\model) \equiv P(\trdata)$ when considering a single model $\model$. We observe that $P(\trdata) = \prod_{i=1}^n P(\trdata_i|\trdata_{<i})$ to get the following form of the \textit{log} marginal likelihood:

\begin{equation}
    \log P(\data) = \log \prod_{i=1}^n P(\data_i|\data_{<i}) = \sum_{i=1}^n \log P(\data_i | \data_{<i}) = \sum_{i=1}^n \log [\mathbb{E}_{P(\theta|\data_{<i})} P(\data_i|\theta) ].
\end{equation}

If we define training speed as the number of data points required by a model to form an accurate posterior, then models which train faster -- i.e. whose posteriors assign high likelihood to the data after conditioning on only a few data points -- will obtain a higher marginal likelihood. Interpreting the negative log posterior predictive probability $\log P(\data_i|\data_{<i})$ of each data point as a loss function, the log ML then takes the form of the sum over the losses incurred by each data point during training, i.e. the area under a training curve defined by a Bayesian updating procedure. Crucially, this notion of area under a curve differs from the number of steps required to attain a pre-specified average conditional likelihood or loss. Under the area-under-curve definition of training speed, if model $\model_1$ trains faster than $\model_2$, then for some accuracy $\epsilon$ the number of steps required for $\model_1$ to attain average conditional log likelihood $\epsilon$ will be lower than for $\model_2$. There is no fixed value of $\epsilon$ under which this statement will hold for all possible model pairs; instead, the area under a training curve can be thought of as averaging out many values of $\epsilon$ to arrive at a ranking.

\keyinsight{The marginal likelihood measures the degree to which a posterior update on each data point increases the likelihood of not-yet-seen data points. It can be computed by taking the area under the loss curve of a particular updating procedure, yielding a notion of \textit{training speed}.}
\subsection{Unbiased estimation of a lower bound} \label{sec:LB}
\label{sec:unbiased-estimation}

In practice, computing $\log P(\data_i|\data_{<i})$ may be intractable, necessitating approximate methods to estimate the model evidence. In our analysis, we are interested in estimators of $\log P(\data)$ computed by drawing $k$ samples of $\theta \sim P(\theta|\data_{<i})$ for each $i=1, \dots, n$. We can directly estimate a lower bound $\mathcal{L}(\data) = \sum_{i=1}^n\mathbb{E}[\log P(\data_i|\data_{<i})]$ using the log likelihoods of these samples, yielding the estimator
\begin{equation}
    \hat{\mathcal{L}}(\data) = \sum_{i=1}^n \frac{1}{k}\sum_{j=1}^k\log P(\data_{i}|\theta^i_j).
\end{equation}
This will produce a biased estimate of the log marginal likelihood due to Jensen's inequality. We can get a tighter lower bound by first estimating $\mathbb{E}[\log P(\data_i|\theta)]$ using our posterior samples before applying the logarithm, obtaining

\begin{equation}\label{eq:estimators}
    \hat{\mathcal{L}}_k(\data) = \sum_{i=1}^n \log \frac{1}{k}\sum_{j=1}^k P(\data_{i}|\theta^i_j).
\end{equation} 
The two estimators converge when $\mathcal{L}$ is estimated with a single sample and $\hat{\mathcal{L}}_k$ is evaluated at $k=1$, however we consider them separately in the following theorem due to their disagreement on larger sample sizes.
\begin{restatable}{proposition}{PropLk}\label{prop:lk}
Both $\hat{\mathcal{L}}$ and $\hat{\mathcal{L}}_k$ as defined in Equation~\ref{eq:estimators} are estimators of lower bounds on the log marginal likelihood; that is
\begin{equation}
   \mathbb{E}[\hat{\mathcal{L}}(\data)] = \mathcal{L}(\data) \leq  \log P(\data) \quad \text{ and } \quad 
    \mathbb{E}[\hat{\mathcal{L}}_k(\data)] = \mathcal{L}_k(\data) \leq \log P(\data) \; .
\end{equation}
Further, the bias term in $\mathcal{L}$ can be quantified as follows.
\begin{equation}\mathcal{L}(\data) = \log P(\data) - \sum_{i=1}^n \KL{(P( \theta | \data_{<i})}{ P(\theta|\data_{< {i+1}}))} \label{eq:lb-decomp}
\end{equation}
\end{restatable}
The proof of this result follows from a straightforward application of Jensen's inequality and can be found in Appendix~\ref{sec:proofs-supervised}.
We observe that both lower bound estimators exhibit decreased variance when using multiple posterior samples; however, $\hat{\mathcal{L}}_k$ also exhibits decreasing bias (with respect to the log ML) as $k$ increases; each $k$ defines a distinct lower bound $\mathcal{L}_k = \mathbb{E}[\hat{\mathcal{L}}_k ]$ on $\log P(\data)$. The gap induced by the lower bound $\mathcal{L}(\data)$ is characterized by the
information gain each data point provides to the model about the posterior, as given by the KL divergence between the posterior at iteration $i$ and the posterior at iteration $i+1$. Thus, while $\mathcal{L}$ has a Bayesian interpretation it is arguably more closely aligned with the minimum description length notion of model complexity \citep{hinton1993keeping}. Increasing $k$ in $\mathcal{L}_k$ thus allows us to interpolate between the minimum description length and marginal likelihood maximization principles.

When the posterior predictive distribution of our model is Gaussian, we consider a third approach which, unlike the previous two methods, also applies to noiseless models\footnote{Note that the log probability $\log P(\data_i | \theta^i_j)$ will not necessarily be well-defined for noiseless models and so the previous estimators cannot be naively applied.} given in Equation~\ref{eq:l-s}. We shift focus slightly to the regression setting, where each data point is in the form of an input-label pair $(X, y)$, such that $\trdata =(X_i, y_i)_{i=1}^n$, and $(\theta^i_j)_{j=1}^k$ be $k$ parameter samples from $P(\theta|\data_{<i})$. We assume some uniform distribution over inputs $X$ and a structured conditional distribution consisting of a mapping $f: \Theta \times X \rightarrow Y$ which given a set of parameters and an input induces a Gaussian likelihood $P(\cdot | \theta, X) = \mathcal{N}(f(\theta, X), \sigma_N^2)$ for some variance $\sigma_N^2$ $P(\cdot|\data_{<i}, X_i)$. The likelihood of a given data point $\data_i = (x_i,y_i)$ under the model thus takes the form 
\begin{equation*}
    P_{\model}(\data_i | \data_{<i}) = \int_{\theta} \mathcal{N}(y_i|f(\theta, X_i), \sigma_N^2) P(\theta|\data_{<i})d\theta
\end{equation*}
We can then obtain the following proposition, which characterizes an estimator of a lower bound on $\log \mathcal{P}(\data)$.

\begin{restatable}{proposition}{PropLS}\label{prop:ls}
Let $P(Y_i|\data_{<i}, X_i) = \mathcal{N}(\mu_i, \sigma^2_i)$ for some $\mu_i, \sigma_i^2$. Define the standard mean and variance estimators $\hat{\mu}_i = \frac{1}{N} \sum_{j=1}^N f(\theta^i_j, x_i)$ and $\hat{\sigma}^2_i = \frac{1}{N-1} \sum (f(\theta_{j}^i, x_i) - \hat{\mu})^2$. Then the  estimator
\begin{equation}\label{eq:l-s}
 \hat{\mathcal{L}}_S(\data) = \sum_{i=1}^n \log P(Y_i|\hat{\mu}_i, \hat{\sigma}^2_i) 
\end{equation}
is a lower bound on the log ML: i.e. $\mathbb{E}[\hat{\mathcal{L}}_S(\data)] \leq \log P(\data) $. 
\end{restatable}
The proof of this result leverages the independence of the sample mean and variance in order to iteratively apply Jensen's inequality to the estimator, and can be found in Appendix~\ref{sec:proofs-supervised}.
We provide an empirical evaluation of the rankings provided by the different estimators in Section~\ref{sec:BMS}. We find empirically that $\hat{\mathcal{L}}_S$ exhibits the least bias in the presence of limited samples from the posterior, though we emphasize its limitation to Gaussian posteriors; for more general posterior distributions, $\hat{\mathcal{L}}_k$ minimizes bias for large $k$ while still estimating a lower bound. 

\subsubsection{Lower bounds via gradient descent trajectories}
\label{sec:lower-bounds-gd}
The bounds on the marginal likelihood we introduced in Section~\ref{sec:unbiased-estimation} required samples from the sequence of posteriors $P(\theta|\data_{<i})$ as data points were incrementally added. However, such bounds were agnostic to the procedure by which the posterior was sampled. We now draw a connection between Bayesian model selection and gradient descent on linear model classes via a result of \citet{matthews2017}:
ensembles of linear models trained with gradient descent yield samples from the model posterior. In particular, we show that we can use these samples to estimate the log ML using the estimators introduced in Section~\ref{sec:unbiased-estimation}.

We will consider the Bayesian linear regression problem of modelling data $\data = (x_i, y_i)_{i=1}^n$ assumed to be generated by the process $Y = \theta^\top \Phi(X) + \epsilon \sim \mathcal{N}(0, \sigma_N^2 I)$ for some unknown $\theta$, known $\sigma_N^2$, and feature map $\Phi$. Typically, a Gaussian prior is placed on $\theta$; this prior is then updated as data points are seen to obtain a posterior over parameters. In the overparmeterized, noiseless linear regression setting, \citet{matthews2017} show that the distribution over parameters $\theta$ obtained by sampling from the prior on $\theta_0$ and running gradient descent to convergence on the data $\data_{<i}$ is equivalent to sampling from the posterior conditioned on $\data_{<i}$. \citet{osband2018randomized} extend this result to posteriors which include observation noise $\sigma^2_N \neq 0$ under the assumption that the targets $y_i$ are themselves noiseless observations. 

\begin{algorithm} 
\SetAlgoLined
\KwIn{A dataset $\data =(x_i, y_i)_{i=1}^n $, parameters $\mu_0, \sigma_0^2, \sigma_N^2$}
\KwResult{An estimate of $\mathcal{L}(\data)$}
$\theta_t \gets \theta_0 \sim \mathcal{N}(\mu_0, \sigma_0^2)$; \quad 
$\tilde{Y} \gets Y + \epsilon \sim \mathcal{N}(0, \sigma_N^2)$;  \quad sumLoss $\gets$ 0 \; $\ell(\data_{\le i}, w) \gets \|\tilde{Y}_{\le i} - \theta^\top X_{\le i} \|_2^2 + \frac{\sigma_N^2}{\theta_0^2}\|\theta - \theta_0\|_2^2 $\;

 \For{$\data_i \in \data$}{
  sumLoss $ = $ sumLoss $ + \; \frac{(\theta_t^\top x_i - y_i)^2}{2\sigma_N^2}$ \;
  $\theta_t \gets$ GradientDescent($ \ell, \theta_t, \data_{\le i}$) \;
 }
 \KwRet sumLoss
 \caption{Marginal likelihood estimation for linear models}
 \label{alg:estimate}
\end{algorithm}

We can use this procedure to obtain posterior samples for our estimators by iteratively running sample-then-optimize on the sets $\data_{<i}$. Algorithm \ref{alg:estimate} outlines our approach, which executes gradient descent optimization on iterative subsets of the data to obtain the necessary posterior samples for our estimator. We note that the GradientDescent subroutine in Algorithm~\ref{alg:estimate} will in idealized settings output the cluster point of the gradient descent algorithm under decreasing step sizes to give an exact posterior sample; in practice a finite step size may be used and this output will only be an approximate sample from the posterior. Theorem \ref{thm:sto} shows that this procedure yields an unbiased estimate of $\mathcal{L}(\data)$ when posterior samples are used to estimate $\mathbb{E} [\log P(\data_i|\theta)]$, and an unbiased estimate of $\mathcal{L}_k(\data)$ when an ensemble of $k$ models are trained in parallel to estimate $\mathbb{E}[P(\data_i | \theta)]$.

\begin{restatable}{theorem}{ThmSTO} \label{thm:sto}
Let $\data = (X_i, Y_i)_{i=1}^n$ and let $(\theta_j^i)_{i,j=1}^{n,J}$ be generated by the procedure outlined in Algorithm~\ref{alg:estimate}. Then the estimators $\hat{\mathcal{L}}, \hat{\mathcal{L}}_S,$ and $ \hat{\mathcal{L}}_k$, applied to the collection $(\theta_j^i)$, are lower bounds on $\log P(\data)$. Further, expressing $-\log P(\data_i|\theta)$ as the $\ell_2$ regression loss plus a constant, we then obtain 
\begin{equation}
    \log P(\data) \geq \sum_{i=1}^n \mathbb{E}_{\theta_i \sim P(\cdot | \data_{<i})}[\log P(\data_i|\theta_i)] = \mathbb{E}\sum_{i=1}^n -\ell_2 (\data_i, \theta_i) + c = \mathcal{L}(\data)
\end{equation}
\end{restatable}

We highlight that Theorem \ref{thm:sto} precisely characterizes the lower bound on the marginal likelihood as a sum of training losses based on the regression loss $\ell_2(\data_i, \theta_i)$ when the likelihood $P(\data_i | \theta)$ is a Gaussian. 

\subsubsection{Infinite-width neural networks}

\label{sec:ntk-ml}
Beyond linear models, our estimators can further perform model selection in the infinite-width limit of neural networks. Using the optimization procedure described by \citet{he2020bayesian}, we can obtain an exact posterior sample from a GP with kernel equal to the NTK. The iterative training procedure described in Algorithm~\ref{alg:estimate} will thus yield a lower bound on the marginal likelihood of this GP using sampled losses from the optimization trajectory of the neural network. We evaluate this bound in Section \ref{sec:BMS}, and formalize this argument in the following corollary. 
\begin{restatable}{corollary}{CorNTK}\label{cor:ntk}
Let $\trdata$ be a dataset indexed by our standard notation. Let $f_0$ be sampled from an infinitely wide neural network architecture $\mathcal{F}$ under some initialization distribution, and let $f_\infty^i$ be the limiting solution under the training dynamics defined by \citet{he2020bayesian} applied to the initialization $f_0$ and using data $\trdata_{< i}$. Let $K_\infty$ denote the neural tangent kernel for $\mathcal{F}$, and $\mathcal{M}=\mathrm{GP}(\mathbf{0}, K_\infty)$ the induced Gaussian Process. Then $f_\infty^i \sim P(f|\trdata_{< i}, \model)$, and in the limit of infinite training time, the iterative sample-then-optimize procedure yields an unbiased estimate of $\mathcal{L}(\trdata |\model)$. Letting $\ell_2$ denote the scaled squared $\ell_2$ regression loss and $c$ be a constant, we obtain as a direct corollary of Theorem~\ref{thm:sto}
\begin{equation}
   \log P(\data) \geq \mathbb{E}_{f_\infty^i \sim P(\cdot | \data_{<i})}[\log P(\data_i|f_{\infty}^i)] = \mathbb{E}\sum_{i=1}^n -\ell_2 (\data_i, f^i_{\infty}) + c = \mathcal{L}(\data) \; .
\end{equation}
\end{restatable}


It is natural to ask if such a Bayesian interpretation of the sum over training losses can be extended to non-linear models trained with stochastic gradient descent.  Although SGD lacks the exact posterior sampling interpretation of our algorithm, we conjecture a similar underlying mechanism connecting the sum over training losses and generalization. Just as the marginal likelihood measures how well model updates based on previous data points generalize to a new unseen data point, the sum of training losses in a stochastic gradient descent trajectory measures how well parameter updates based on one minibatch generalize to the rest of the training data. If the update generalizes well, we expect to see a sharper decrease in the training loss, i.e. for the model to train more quickly and exhibit a lower sum over training losses. This intuition can be related to the notion of `stiffness' proposed by \citet{fort2019stiffness}. We provide empirical evidence supporting our hypothesis in Sections \ref{sec:ts-interference} and \ref{sec:DNN_exp}; for now, we focus on the utility of this estimator in Bayesian model selection.

\subsection{Empirical evaluation} \label{sec:BMS}
Section \ref{sec:ml-estimation} focused on two key ideas: that online training statistics can be used in an estimator of a Bayesian model's marginal likelihood (or a lower bound thereof), and that gradient-based optimization can produce the samples needed for this estimation problem. We further conjectured that similar phenomena may also hold for deep neural networks. We now illustrate these ideas in a range of settings. Section \ref{sec:BMS} provides confirmation and quantification of our results for linear models, the model class for which we have theoretical guarantees, while Section \ref{sec:DNN_exp} provides preliminary empirical confirmation that the mechanisms at work in linear models also appear in DNNs. 

While we have shown that our estimators correspond to lower bounds on the marginal likelihood, in order to be useful for model selection also we need the rankings given by the estimators to agree with those assigned by the marginal likelihood. 
We first evaluate the relative rankings given by the true marginal likelihood with those given by our estimators on a simple feature selection task, consisting of 15 informative features $\phi_1, \dots, \phi_{15}$ and 15 features containing random noise $\phi_{16}, \dots, \phi_{30}$. Each model $\mathcal{M}_{i}$ uses features $\phi_1, \dots, \phi_i$ in its regression objective and ignores the rest. Full experiment details, along with a description of two additional model selection problems evaluated in Appendix~\ref{sec:optimize-then-prune}, can be found in Appendix~\ref{sec:ex_ms_blr_synthetic_data}. Naturally, the optimal model is $\model_{15}$, which uses all informative features and no extraneous ones. We compare $\mathcal{L}_S$, $\mathcal{L}$ and $\mathcal{L}_k$ to see whether each can identify the optimal model. We first observe that all methods agree on the optimal model: this is a consistent finding across all of the model selection tasks we consider. While all methods lower bound the log marginal likelihood, $\mathcal{L}_k(\data)$ and $\mathcal{L}_S(\data)$ exhibit a reduced gap compared to the naive lower bound.
In the rightmost plot of Figure~\ref{fig:lse-estimator}, we further quantify the reduction in the bias of the estimator $\mathcal{L}_k(\data)$ described in Section~\ref{sec:unbiased-estimation}. We use exact posterior samples (which we denote in the figure simply as posterior samples) and approximate posterior samples generated by the gradient descent procedure outlined in Algorithm~\ref{alg:estimate} using a fixed step size and thus inducing some approximation error.  We find that both sampling procedures exhibit decreasing bias as the number of samples $k$ is increased, with the exact sampling procedure exhibiting a slightly smaller gap than the approximate sampling procedure.

\begin{figure}
    \centering
    \includegraphics[width=0.47\linewidth]{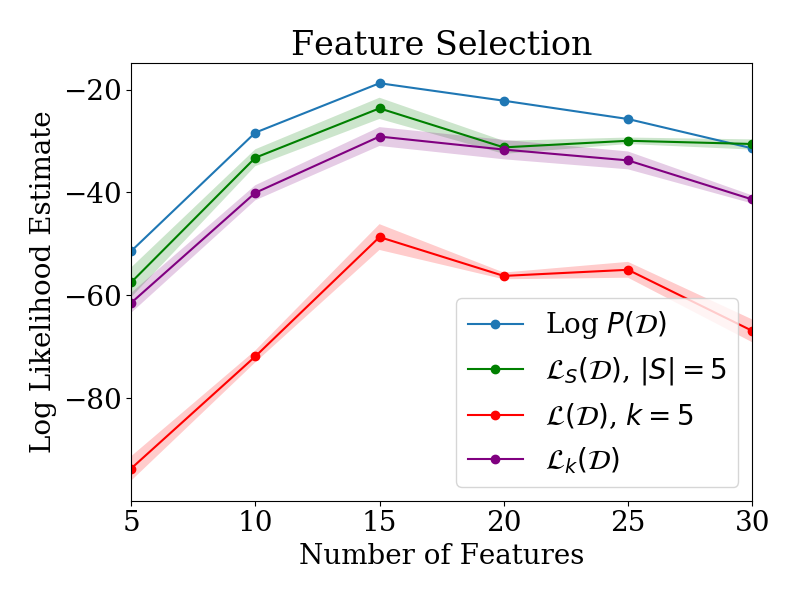}
    \includegraphics[width=0.47\linewidth]{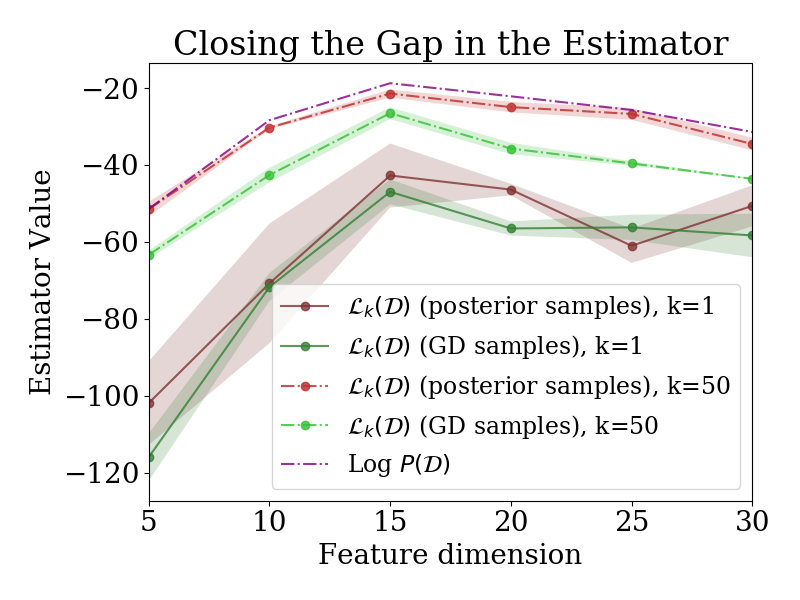}
    \caption[Evaluation of our estimators.]{Left: ranking according to $\log P(\data)$, $\mathcal{L}(\data)$ with exact posterior samples, and $\mathcal{L}(\data)$ computed on samples generated by gradient descent. Right: gap between true marginal likelihood and $\mathcal{L}_k(\data)$ estimator shrinks as a function of $k$ for both exact and gradient descent-generated samples. }
    \label{fig:lse-estimator}
    \vspace{-5pt}
\end{figure}

\begin{figure}
    \centering
    \includegraphics[width=0.41\linewidth]{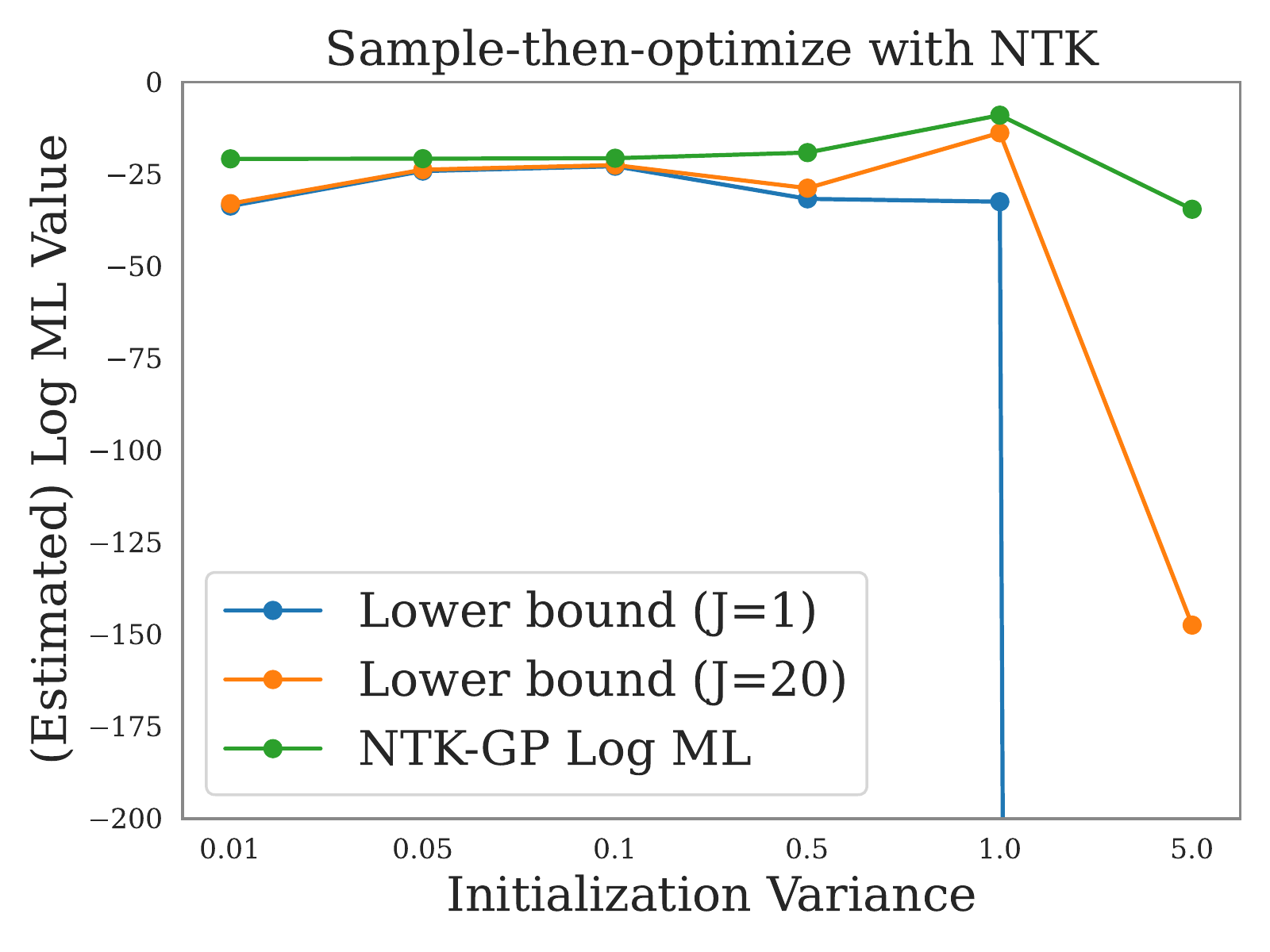}\hfill 
    \includegraphics[width=0.44\linewidth]{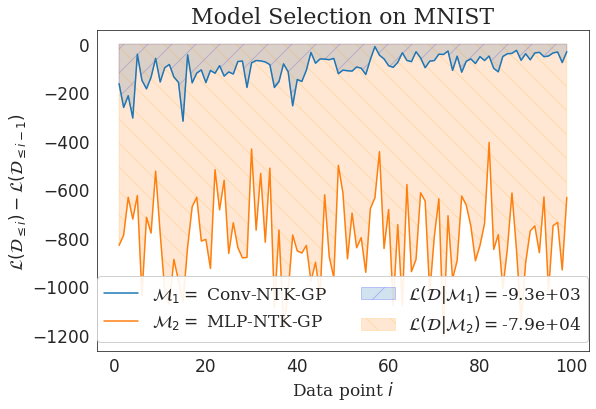}
    
    \caption[Illustration of marginal likelihood estimation via our method on NTK models.]{Left: approximate lower bound on log ML of NTK model on a prior variance selection task. Right: visualizing the interpretation of $\mathcal{L}(\data)$ as the `area under the curve' of training losses: we plot the relative change in the estimator  $\mathcal{L}(\data_{\le i}) - \mathcal{L}(\data_{<i}) $ for convolutional and fully-connected NTK-GP models, and shade their area. 
    }
    \label{fig:ntk_evaluation}
\end{figure}

We further illustrate how our estimator of $\mathcal{L}(\data)$ can select inductive biases in the infinite-width neural network regime in Figure~\ref{fig:ntk_evaluation}. The left hand side of the figure shows the marginal likelihoods of NTK-GP models for a fully-connected MLP architecture given different prior variances, illustrating that the sample-based estimator $\hat{\mathcal{L}}$ does indeed give a lower bound which becomes more accurate as the number of samples grows. In the right hand side of the figure, we evaluate the relative change in the log ML of a Gaussian Process induced by a fully-connected MLP (MLP-NTK-GP) and a convolutional neural network (Conv-NTK-GP) which regress on the MNIST dataset.
The fully-connected model sees a consistent decrease in its log ML with each additional data point added to the dataset, whereas the convolutional model sees the incremental change in its log ML become less negative as more data points are added, as well as a less negative incremental decrease in the log ML at the start of training.
This leads to the Conv-NTK-GP having a higher value of $\mathcal{L}(\data)$, and a higher log marginal likelihood, than the MLP-NTK-GP. This provides reassurance that the estimators we have proposed behave reasonably in simple architecture selection problems.

\section{Training speed and interference}
\label{sec:ts-interference}
Having revealed that statistics from a particular gradient descent training procedure can yield an accurate estimator of the log marginal likelihood, we now turn our attention towards generalizing these findings to the setting of multi-epoch stochastic gradient descent optimization in deterministic function approximators. We begin by relating the posterior predictive likelihoods used in our estimator to the agreement between the gradients of disjoint minibatches (i.e. interference exhibited by the function approximator) in a dataset. This motivates an analogous estimator for SGD trajectories, which is highly correlated with generalization performance. 

\subsection{Generalization and multi-epoch optimization} \label{sec:theory_multi_epoch}
We have thus far studied the relationship between training speed and the marginal likelihood in the context of Bayesian models updated iteratively on successive elements of the training data. This has yielded an effective model selection tool based on computing the area under a training curve for specific updating procedures. Translating this approach to the optimization methods used to train deep neural networks faces two major hurdles: first, although it is possible to train neural networks to approximate a Bayesian posterior over weights \citep{neal2012bayesian}, standard training schemes deal only with a deterministic initialization and so do not directly correspond to probabilistic models. It is possible to argue that many stochastic optimization algorithms sample from an approximate Bayesian posterior \citep{welling2011bayesian, mandt2017stochastic}, but this argument only highlights the second hurdle facing our estimator: most optimization schemes used in practice run over a fixed dataset for several epochs. In contrast, our estimator requires an iterative training procedure whereby each data point is sequentially added to the training set, and we never evaluate the loss on revisited data points.

However, independent of the relationship between gradient descent and Bayesian posterior sampling, the same intuition relating training speed to the marginal likelihood hints at a mechanism by which training speed correlates with generalization error in the gradient descent setting. Recall that the marginal likelihood can be written as the sum

 \begin{equation}
    \log P(\data) = \sum_{i=1}^n \log P(\data_i | \data_{<i}).
\end{equation}

Now consider the term $\log P(\data_i | \data_{<i})$. This quantity characterizes how accurately the model predicts $\data_i$ given that it already knows the values of $\data_{<i}$. The change in the log conditional probability of a data point $\data_{i+1}$ after conditioning on data point $\data_i$ thus presents a Bayesian analogue of interference as defined in \eqref{eq:deltaint}. We now investigate whether interference in minibatch optimization might provide a means of leveraging the insights described for Bayesian models in a broader class of learning algorithms.

We first establish notation. We denote data $\trdata = \{(x_i, y_i)\}_{i=1}^n$, let $\ell$ be a loss function, $f_\theta$ the function induced by parameters $\theta$, and $R(\theta) = \mathbb{E}_{\dgd}[\ell(f_\theta(x), y)]$. The sequence $(\theta_t)_{t=0}^T$ denotes the series of parameters obtained over a gradient descent trajectory. Recalling that $\log P(\trdata)$ has the interpretation of an area under a training curve, we therefore consider the area under the training curve of a stochastic gradient descent trajectory. For the moment, we will consider minibatches of size one to simplify our analysis; it is straightforward to extend this reasoning to larger minibatches. We first consider a single training epoch, denoting by $\theta_k$ the value of the parameters at step $k$ after performing gradient steps $\theta_{t+1} = \theta_t - \alpha \nabla_\theta \ell (x_t, y_t , \theta_t)$. With a slight abuse of notation, we repurpose $\hat{\mathcal{L}}$ to apply to SGD trajectories as follows.

\begin{align}
   \hat{\mathcal{L}}(\trdata) &= \sum_{k=0}^n \ell(x_k, y_k, \theta_k) \\
   &= \sum_{k=0}^n [\ell(x_k, y_k, \theta_0) + \sum_{j=1}^{k} (\ell(x_k, y_k, \theta_j) - \ell(x_k, y_k, \theta_{j-1}) ) ]  \\
   \intertext{Recall that this is precisely our definition of \textit{interference} from Equation~\ref{eq:deltaint}: $\deltaint(x_k, x_j)$.}
   &= \sum_{k=0}^n[ \ell(x_k, y_k, \theta_0) + \sum_{j=1}^k \deltaint(x_j, x_k)] \\
   &= \eRisk(\trdata; \theta_0) + \sum_{j=1}^n \sum_{k>j} \deltaint(x_j, x_k)
\end{align}
The term $\deltaint(x_j, x_k)$ measures the effect of the gradient step computed on data point $(x_j, y_j)$ on the loss at the data point ($x_k, y_k)$, thus $\sum_{j<k} \deltaint(x_j, x_k)$ measures the cumulative effect of gradient updates computed on earlier data on the loss at point $k$. This quantity philosophically resembles the $\log P(\data_i | \data_{ <i})$ term appearing in the expression of the log marginal likelihood. Indeed, in the first epoch of training it is an unbiased estimate of the effect of the gradient updates performed up to step $k$ on the expected risk $R$. 

\begin{align}
    \mathbb{E}_{x, y \sim \dgd}[\ell(x,y; \theta_k)] &= \mathbb{E}_{\trdata \sim \dgd}[\ell(x_k, y_k; \theta_k)] \\
    &= \mathbb{E}_{\trdata \sim \dgd}[\ell(x_k, y_k; \theta_0) + \sum_{j< k} \deltaint(x_j, x_k)]
\end{align}

Unbiased estimation of the change in the true risk is only attainable due to the independence of $\theta_k$ and $x_k, y_k$. After the first epoch of training, this property no longer holds as the pair $(x_k, y_k)$ was used to obtain the network's current parameters.
Applying similar techniques to those of \citet{hardt2015train} may yield bounds on the bias induced by this dependency, but these tend to be overly pessimistic, resulting in loose upper bounds on the expected risk. Concretely, two problems arise in the case of multiple epochs: first, the loss on the training set will decrease, meaning that the raw change in the loss for any given data point will eventually tend to zero and cease to provide informative updates; second, over the course of many epochs the network may overfit not just its predictions but also its gradient structure to the training set. Of course, neither of these problems is guaranteed to prevent effective performance estimation. In the first case, the losses from earlier in the training trajectory are still likely to be informative, and the low magnitude of the later losses when the network has overfit mean that this later period is unlikely to significantly influence the area under the training curve. In the second case, we note that in general changing the structure of gradients is a much more difficult task than changing predictions (recall the invariance-learning experiments of Figure~\ref{fig:cifar-invariances}), and so we expect the bias induced by these second-order effects to be small. Taking into account the large variance of using a minibatch to estimate the change in loss over the whole dataset, this yields the following practical conjecture.


\hypothesis{Bias-variance trade-off in interference estimation}{the interference between minibatches in the training set $I_{\Delta}(x_i, x_j)$ will exhibit low bias as an estimator of the change in the validation loss relative to its variance due to stochastic optimization. \label{hyp:minibatch}}

We evaluate Hypothesis~\ref{hyp:minibatch} in  Figure~\ref{fig:sotl-bias-variance}, where we plot $\deltaint(x_j, x_k)$, for $y_k$ equal to the same minibatch as $x_k$ (`same minibatch'), a larger subset of the training set (`training holdout set'), and a minibatch drawn from the validation set (`validation set'). For the training holdout set and the validation set, we use minibatches of size 5000 to isolate the variance in the estimator due to only the optimizer minibatch sampling. Intriguingly, these results suggest that even after several epochs, $\deltaint(x_j, x_k)$ will exhibit low bias as an estimator of $\risk(\theta_t) - \risk(\theta_{t-1})$ in a range of neural network architectures and optimization schemes. This suggests that smoothing out the variance in the loss estimate induced by stochastic parameter updates will plausibly provide a greater benefit than using an unbiased estimate of performance on the validation set. This is likely to be the case in particular when a large learning rate is used early in training, as is common in practice.

\begin{figure}
    \centering
    \includegraphics[width=\linewidth]{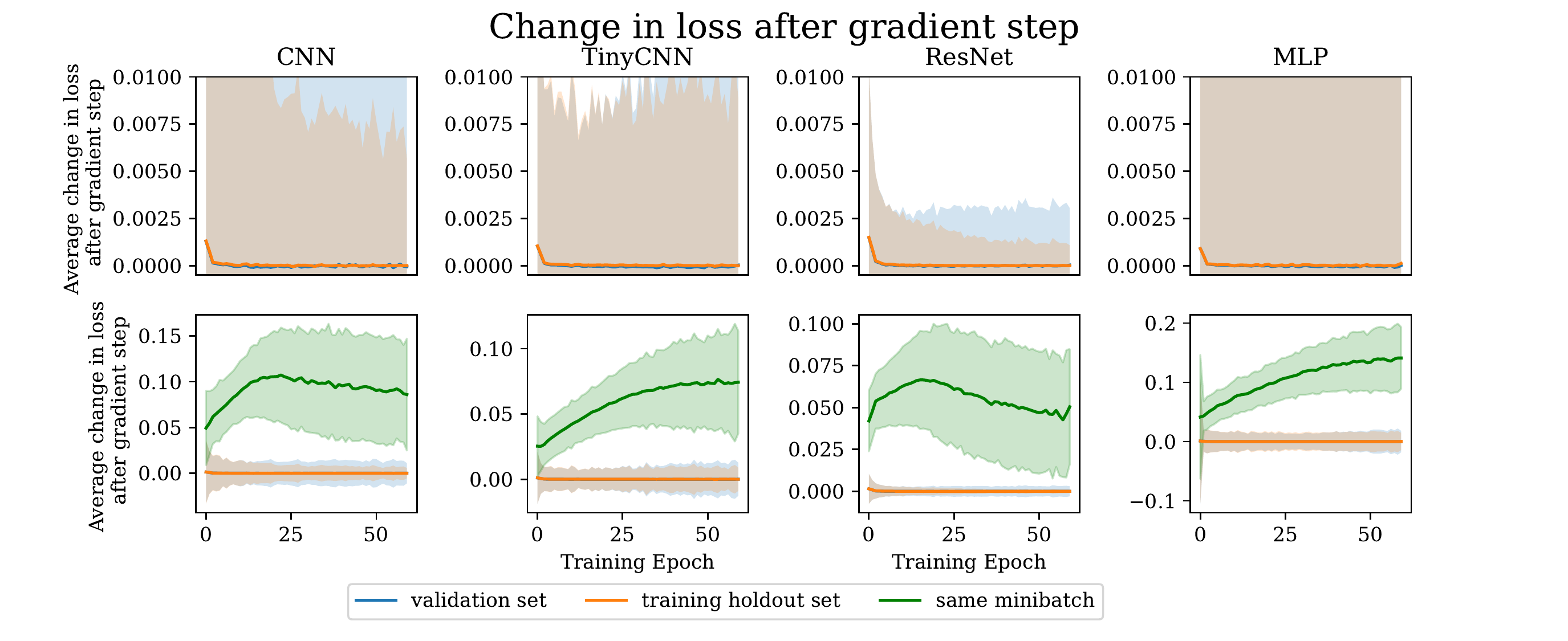}
    \caption[Effect of minibatch gradient steps on train, test, and minibatch loss over the course of training a variety of neural network architectures on CIFAR-10.]{Effect of minibatch gradient steps on train (orange), test (blue), and minibatch (green) loss over the course of training a variety of neural network architectures on CIFAR-10. Top and bottom rows plot the same quantities and differ only in the range of the y axis, which shows the average change in the loss before and after a minibatch gradient step in each epoch. Shaded region indicates one standard deviation. Values for the train and test loss are similar, with the difference in means indistinguishable relative to the variance due to minibatch sampling over the epoch, while the change in loss on the minibatch for which the gradient step was taken dwarfs both. }
    \label{fig:sotl-bias-variance}
\end{figure}

Motivated by these observations, we propose the following generalization measure which directly computes the area under a network's training curve. We refer to it as the SOTL, due to its interpretation as measuring the sum of training losses from the model's training trajectory. 
\begin{equation}
	\mathrm{SOTL }= \sum^T_{t=1} \left[ \frac{1}{B} \sum^B_{i=1} \ell \left( f_{\theta_{t, i}}(\mathbf{X}_i), \mathbf{y}_i \right) \right]
\end{equation}

A more detailed study of the SOTL is outside the scope of this thesis, though validation of its utility in neural architecture search can be found in the related work of \citet{ru2020revisiting}. We perform a small-scale proof of concept in this chapter. 

\subsubsection{Training speed in DNNs} \label{sec:sgd_dnn}

\begin{figure}
    \begin{minipage}{.27\textwidth}
    \includegraphics[ width=\linewidth]{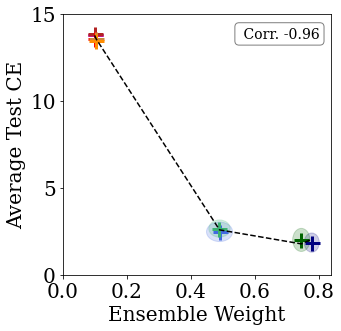}
    \end{minipage}
    \begin{minipage}{.27\textwidth}
    \includegraphics[ width=\linewidth]{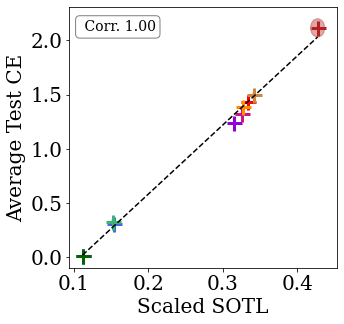}
    \end{minipage}
    \begin{minipage}{.27\textwidth}
    \includegraphics[ width=\linewidth]{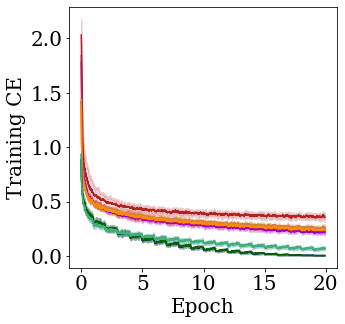}
    \end{minipage}
    \begin{minipage}{.17\textwidth}
    \includegraphics[ width=\linewidth]{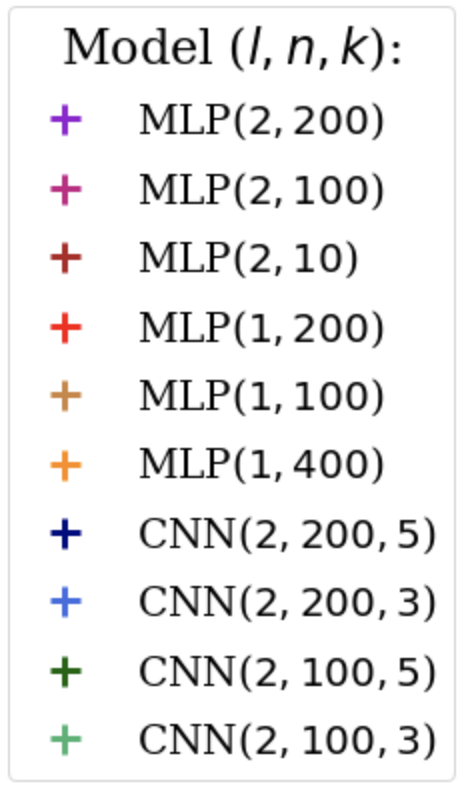}
    \end{minipage}
    \caption[Linear combinations of DNNs on FashionMNIST.]{Left: ensemble weights versus the test loss for concurrent training. Middle: sum over training losses (SOTL), standardized by the number of training samples, versus test loss for parallel training. Right: training curves for the different models trained in parallel. All results are averaged over $10$ runs, and standard deviations are shown by the shaded regions around each observation. The model parameters, given in the parentheses, are the number of layers ($l$), nodes per layer ($n$) and kernel size ($k$), respectively. }
    \label{fig:mod_select_dnn}
\end{figure}

Motivated by the previous discussion, we conjecture that just as the sum of the log posterior likelihoods is useful for Bayesian model selection, the sum of minibatch training losses will be useful to predict generalization error. In this section, we evaluate whether this conjecture holds for a simple convolutional neural network trained on the FashionMNIST dataset. Our results provide preliminary evidence in support of this claim, and suggest that further work investigating this relationship may reveal valuable insights into how and why neural networks generalize. We first evaluate whether the sum over training losses (SOTL) obtained over an SGD trajectory correlates with a model's generalization error, and whether SOTL predicts the weight assigned to a model by a linear ensemble.  A discussion of the link between linear ensemble weight and training speed is provided in Appendix~\ref{sec:optimize-then-prune}; to summarize, we note that (sub-)models with the best sum of training losses will also exhibit the greatest time-averaged correlation with the target, and so in some sense present the `best' feature for the linear model combination to use. To evaluate the connection between these three concepts, we train a linear combination of DNNs with stochastic gradient descent and evaluate a) the SOTL of each model's training trajectory, b) the final validation loss of each model, and c) the weight assigned to each model by a linear ensembling layer trained concurrently with the networks.

We observe a strong correlation between the sum of training losses (SOTL) and average test cross-entropy (see Figure \ref{fig:mod_select_dnn} middle column), validating that the SOTL provides useful model rankings. Further, we find that architectures with lower test error (when trained individually) are given higher weight by the linear ensembling layer -- as can be seen from the left plot in Figure \ref{fig:mod_select_dnn}.  This finding may have intriguing implications on empirical phenomena such as the lottery ticket hypothesis~\citep{frankle2018the}. Further details of the experiment can be found in Appendix \ref{sec:exp_details_sgd_dnn}. Our results are summarized in Figure \ref{fig:mod_select_dnn}. Though the architecture sets in these illustrative experiments are small, these results are replicated in much larger architecture search spaces by \citet{ru2020revisiting}.

\subsection{Bias and variance in performance estimation} \label{sec:DNN_exp}
We now investigate \textit{why} the sum of training losses seems to predict generalization despite our departure from the theoretically grounded regime of Bayesian models.
We have seen in Figure~\ref{fig:mod_select_dnn} that the SOTL is predictive of generalization performance in simple model selection problems -- indeed, further work \citep{ru2020revisiting} has robustly demonstrated that it is often a \textit{better} predictor of generalization than the early-stopping validation loss. This is quite remarkable as most of the received wisdom in deep learning recommends against using the training set for hyperparameter search in favour of a held-out validation set. We recall Figure~\ref{fig:sotl-bias-variance}, which revealed that the effect of one minibatch gradient step on the loss of a disjoint minibatch exhibited low bias as an estimator of the change in the expected risk, while the variance of this estimator due to stochasticity in the gradient descent trajectory was itself significant. This suggests that using training speed as a performance estimator may be performing a bias-variance trade-off: training speed introduces a small amount of bias in the performance estimator as the model overfits, but by marginalizing over all of the parameters visited during an epoch or a trajectory, significantly reduces the variance in the performance estimation. Concretely, we obtain the following hypothesis.
\hypothesis{Bias-variance trade-off in training speed}{training speed, as measured by the sum of training losses, provides a lower-variance estimate of performance than the validation loss in noisy optimization procedures such as stochastic gradient descent.}
\begin{figure}
    \centering
    \includegraphics[width=\linewidth]{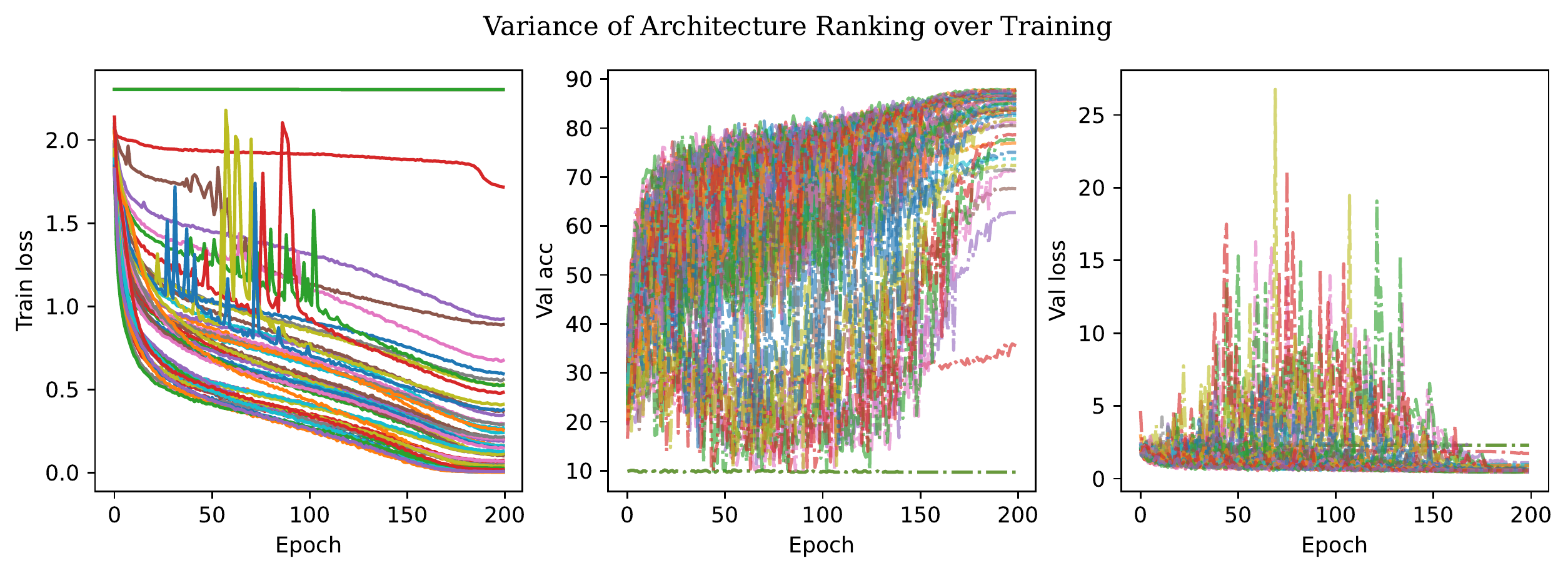}
    \caption[Rankings of a randomly-selected subset of models from the NASBench-201 dataset given by training speed estimators. ]{The ranking of a randomly-selected subset of models from the NASBench-201 dataset given by training speed estimators is consistent across subsequent epochs, whereas the ranking given by a point estimate of the validation accuracy exhibits high variance between epochs for the same architecture, illuminating the mechanism by which training speed can yield more informative performance estimates for neural architecture search.}
    \label{fig:tse-variance}
\end{figure}

We validate this hypothesis in Figure~\ref{fig:tse-variance}, where we evaluate a randomly sampled subset of models from the NASBench-201 dataset \citep{dong2020nasbench201}, the dataset in which variants of the SOTL evaluated by \citet{ru2020revisiting} obtained impressive rank correlation performance, outperforming learning curve extrapolation methods and the early stopping validation loss. We see that the curves corresponding to the validation accuracy and validation loss at each epoch are indistinguishable for most architectures, while the training loss, which is averaged over minibatches sampled during an entire epoch, exhibits much lower variance. This suggests that using the early-stopping validation loss would yield an extremely noisy estimator of final performance, whereas the SOTL, even if only averaged over the losses of a single epoch late in training, yields a relatively consistent ranking.

\keyinsight{The sum of training losses yields an analogous performance estimator to $\mathcal{L}(\data)$ which can be applied to neural networks. It measures, among other things, \textit{interference}: the degree to which a gradient update on one minibatch generalizes to the rest of the training set. } 


\section{Conclusions}

In this chapter, we have proposed a family of estimators of the marginal likelihood which illustrate the connection between training speed and Bayesian model selection. Because gradient descent can produce exact posterior samples in linear models, our result shows that Bayesian model selection can be performed by training a linear model with gradient descent and tracking how quickly it learns. This approach also applies to the infinite-width limit of deep neural networks. The intuition behind this estimator is appealing: a model which can better leverage information from one data point to improve its predictions on other data points will generalize better to new data than one which is not able to do so. We apply similar intuition to propose an analogous performance estimator for deep neural networks, which computes the area under a model's training curve. This estimator obtains high rank correlation with the final generalization error of a neural network, yielding a competitive approach to performance estimation for neural architecture search. 
We provide evidence that the connections shown in linear models have predictive power towards explaining generalization and training dynamics in DNNs. This empirical analysis highlights the importance of generalization between inputs in the training set, i.e. interference, as a factor influencing both training speed and generalization error.

One limitation of the connection illustrated in this chapter is that it can only be straightforwardly applied to problems of identifying a suitable inductive bias or network architecture. Regularization schemes that influence training dynamics, such as dropout or weight regularization, introduce confounding factors into the training process under which the relationship between training speed and generalization breaks down. Training speed alone may also be insufficient for problems of optimizer or hyperparameter selection, where these quantities may reduce training speed via their influence on the ability of each optimization step to reduce the loss on the current minibatch, rather than by influencing generalization between minibatches. 

While generalization between training inputs will play a key role in our analysis of generalization in {reinforcement learning}, the results from this chapter cannot immediately be applied to deep RL agents. This is because the estimators we use require a fixed optimization objective to provide a meaningful notion of training speed. In reinforcement learning, both the optimization objective and the input distribution are constantly changing. A deep RL agent that obtains a low training loss quickly may nonetheless achieve sub-optimal return in the environment. In order to understand the role of interference in deep reinforcement learning, it will first be necessary to understand how the dynamics of RL differ from those of the supervised learning setting. It is only with this understanding that we will be able to leverage the insights developed in this chapter and the previous one to predict and improve generalization in deep reinforcement learning.

\chapter{Dynamics of reinforcement learning}
\label{chp:rl-dynamics}
\minitoc
\section{Introduction}
\label{sec:rl-dynamics-introduction}
\begin{figure}[t!]
    \centering
    \includegraphics[keepaspectratio,width=.79\textwidth]{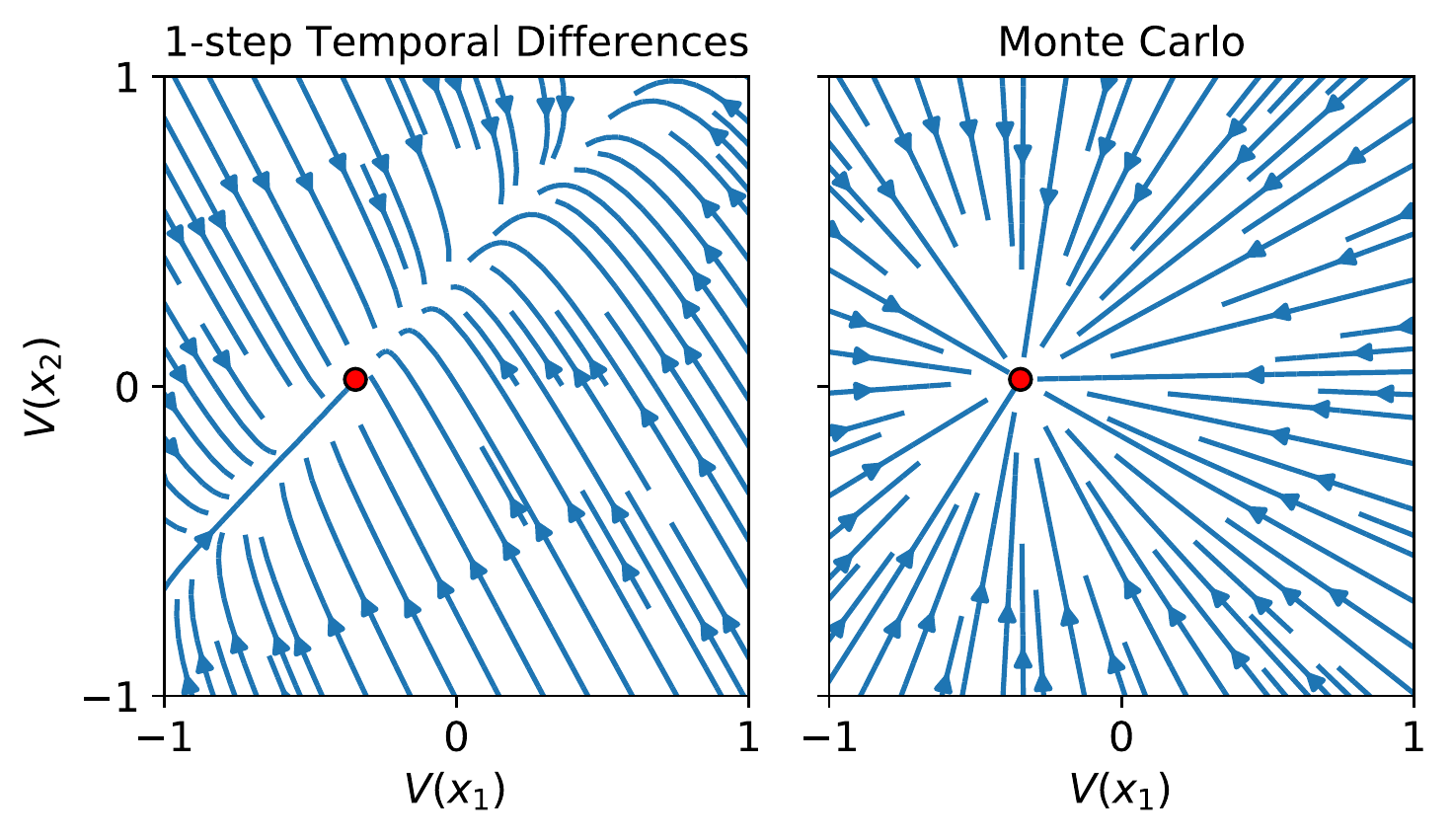}\hfill
    \caption{An example of qualitatively different value function dynamics for a two-state MDP for 1-step temporal difference learning and Monte Carlo learning, with fixed point $V^\pi$ in red.}
    \label{fig:two-state-example}
\end{figure}
The previous chapters have shown that properties of a network's training dynamics can provide useful insights into its generalization performance. The remainder of this thesis will apply similar analysis into the training of \textit{reinforcement learning} agents. As mentioned in the discussion of TD learning in Section~\ref{bkgd:rl}, the dynamics of reinforcement learning agents are less stable than those of supervised learning \citep{baird1993advantage}, leading to divergence even for simple linear function approximation schemes. This complexity necessitates a more detailed study of the learning dynamics of neural networks trained with value-based RL algorithms before it is possible to discuss the implications of these dynamics on generalization in full detail. 
To this end, the current chapter will provide a novel theoretical framework for our study of representation learning and illustrate its utility by analyzing the effect of various auxiliary tasks on these dynamics; Chapter~\ref{chp:rep-learning} will draw on insights from this framework to motivate an empirical study of the representations learned by deep RL agents, while Chapter~\ref{chp:gen-rl} will explore its implications on memorization and overfitting to training observations.

Our theoretical analysis of representation dynamics in RL is motivated in part by the robust improvements given by auxiliary tasks to deep reinforcement learning agents. While supervised learning often benefits from unsupervised pretraining \citep{radford2018improving, erhan2010does, lee2021pebble}, fitting auxiliary labels that are only tangentially related to the principal task has not been incorporated into standard learning pipelines. In contrast, training a network to predict not just the value function, but other properties of the environment such as the value of a pixel at the next timestep or expected state occupancies under the current policy, consistently helps RL agents improve their performance \citep{jaderberg2016reinforcement,mirowski2017learning, lin2019adaptive}. A commonly-held belief is that these benefits are mediated through improved representation learning.
This hypothesis naturally raises a number of questions that, broadly speaking, remain open. What makes a good auxiliary task? Can we predict how an auxiliary task will affect an agent's representation? When should one auxiliary task be used instead of another? More generally, how should this hypothesis about the mechanism of auxiliary tasks itself be tested? 

The complex interacting components of large-scale deep reinforcement learning agents make it difficult to extract general insights. 
In this chapter we aim to shed light on the answers to these questions by distilling the benefits of auxiliary tasks down to their effects on the dynamics of the learned representation.
We begin by considering a \emph{learning dynamics} framework for studying the evolution of an agent's predictions and learned representation; see Figure~\ref{fig:two-state-example} for a toy illustration, with full details given in Section~\ref{sec:learning-dynamics}. The central idea behind this framework is that it is not just \emph{what} an agent learns that dictates how its representation is shaped, but also \emph{how} it learns. 

This framework provides a model for representation learning in RL. Under this model, even in the case of model-free algorithms, agents can be shown to automatically incorporate the transition structure of the environment into their representations. We characterize the dynamics induced by a number of auxiliary tasks, with particular focus on ensemble predictions and random cumulant functions, and prove convergence of the induced representations to subspaces defined by certain decompositions of the environment's transition operator. These results will play a key role in our analysis in Chapter~\ref{chp:gen-rl}.
We then consider the effectiveness of auxiliary tasks in sparse-reward environments, and via the use of the learning dynamics framework, construct a hypothesis as to which auxiliary tasks should be particularly well suited to such environments. We test this conjecture in the Arcade Learning Environment \citep{bellemare2013arcade}, demonstrating strong performance with random cumulant auxiliary tasks. These findings motivate our study of representation learning in sparse-reward environments in Chapter~\ref{chp:rep-learning}.

\section{Mathematical framework}
\label{sec:learning-dynamics}
This section will present a mathematical framework from which to study the learning dynamics of RL agents, with a particular focus on the dynamics of an agent's representation. A description of the reinforcement learning setting along with the basic notation used in this chapter can be found in Chapter~\ref{chp:background}. While we will formulate the representation in terms of the values attained by some feature output, our analysis will focus on the \textit{evolution} of these features over time. In taking this perspective, we lay the groundwork for our later discussion of \textit{capacity} in Chapter~\ref{chp:rep-learning}, where we frame a representation in terms of not just the network's feature outputs under its current parameters, but also the optimization dynamics that these parameters induce.

\subsection{Features and representations}\label{sec:reps}

In many environments, it is impractical to store a value function as a table indexed by states. Even in settings where this is practical, it may not be desirable, as function approximation is necessary for generalization to new observations that the agent may encounter during or after training. Instead, it is typical to parameterize $V \in \mathbb{R}^{\statespace}$ through a \emph{feature map} $\phi : \mathcal{X} \rightarrow \mathbb{R}^\repdim$ and \emph{weight vector} $\mathbf{w} \in \mathbb{R}^\repdim$, leading to a factorization of the form
\begin{align*}
    V(x) = \langle \phi(x), \mathbf{w} \rangle \, .
\end{align*}
Such a parameterization may be amenable to more efficient learning, for example if $\phi$ abstracts away unimportant information, allowing for generalization between similar states.
Even more concisely, writing $\Phi \in \mathbb{R}^{\statespace\times\repdim}$ for the matrix with rows $\phi(x)$ yields
\begin{align}\label{eq:q-phi-w}
    V = \Phi \mathbf{w} \, .
\end{align}
The quantity $\Phi$ is often referred to as the agent's \emph{representation} of the environment \citep{boyan1999least, levine2017shallow, bertsekas2018feature,chung2018two,bellemare2019geometric,dabney2020value}.
In many small- and medium-scale applications, the representation is fixed ahead of time,
and only $\mathbf{w}$ is updated during learning; this is the linear function approximation regime. Many common choices of features relate to various decompositions of mathematical objects associated with the transition operator $P^\pi$. In deep reinforcement learning, however, $\Phi$ and $\mathbf{w}$ are learned simultaneously. In this chapter we will consider the representation to be the output of the penultimate layer of the network, meaning that the mapping from representation to value is linear. While it is possible to model $\Phi$ as the output of any network layer, the resulting dynamics will not necessarily be analytically tractable as the mapping from features to network outputs will be nonlinear.

\subsection{Subspace distances}
\begin{wrapfigure}{l}{0.35\linewidth}
    \centering
    \includegraphics[width=0.95\linewidth]{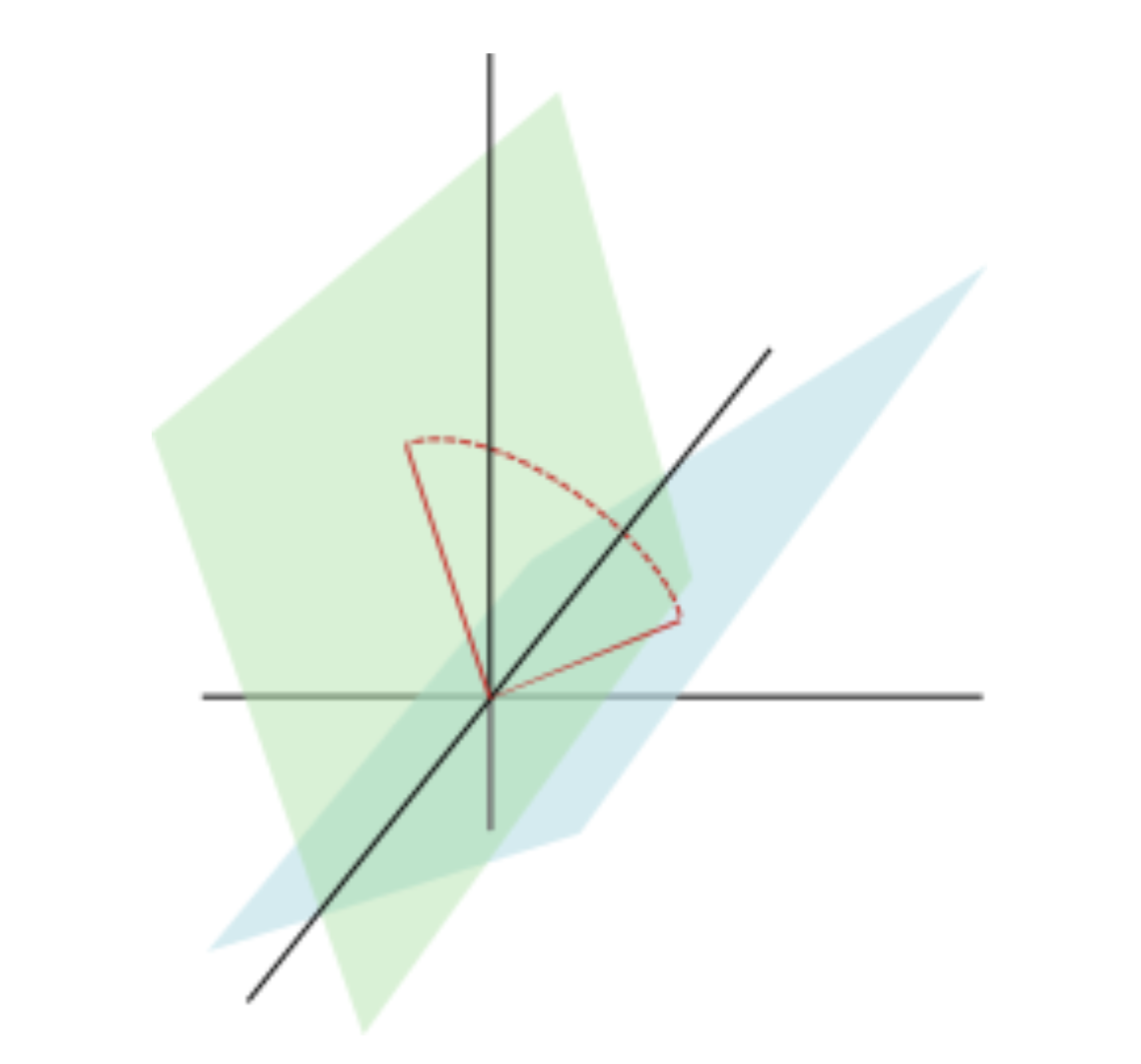}
    \caption{Visualizing the Grassmann distance between two 2-dimensional subspaces.}
    \label{fig:grassmann}
\end{wrapfigure}
The key results of this chapter will depend on characterizing the convergence of the subspace spanned by a set of feature vectors. To do so, we need a notion of distance between subspaces of $\mathbb{R}^{\mathcal{X}}$. 
The following definition follows \citet{ye2016schubert}. Intuitively, it can be thought of as generalizing the notion of an angle between vectors to one between subspaces.

\begin{definition}\label{def:grassmann-distance}
For two $K$-dimensional subspaces $Y_1, Y_2 \leq \mathbb{R}^\mathcal{X}$, the \emph{principal angles} $\theta_1, \ldots, \theta_K \in [0, \pi/2]$ between the subspaces are defined by taking orthonormal matrices $\mathbf{Y}_1 \in \mathbb{R}^{\mathcal{X} \times K}$ and $\mathbf{Y}_2 \in \mathbb{R}^{\mathcal{X} \times K}$ the columns of which span $Y_1$ and $Y_2$ respectively, and defining $\theta_k = \cos^{-1}(\sigma_k(\mathbf{Y}_1^\top \mathbf{Y}_2))$, where $\sigma_k(\mathbf{A})$ is the $k$\textsuperscript{th} singular value of the matrix $\mathbf{A}$. One can check that this definition is independent of the matrices $\mathbf{Y}_1$ and $\mathbf{Y}_2$, depending only on the subspaces $Y_1,Y_2$ themselves. The \emph{Grassmann distance} $d(Y_1, Y_2)$ between $Y_1$ and $Y_2$ is then defined as $\| \theta \|_2 = (\sum_{k=1}^K \theta_k^2)^{1/2}$.
\end{definition}

With these definitions in hand, we are now ready to give a precise version of the statement alluded to by the discussion and figure in Section~\ref{sec:rl-dynamics-introduction}. 

\section{Value function \& representation dynamics}
\label{sec:value-function}
We begin our study of learning dynamics in RL with a discussion of the dynamics followed by a \textit{value function} when Bellman updates are applied. We make some simplifying assumptions to avoid focusing on technicalities here, and give a discussion of the more general case in  Appendix~\ref{sec:more-general-value-function-results}.

\subsection{Value function dynamics}
\label{sec:vf-dynamics}
We consider a continuous-time analogue of one-step temporal difference (TD) learning dynamics. While in practice TD updates are performed discretely, the continuous-time regime lends itself more readily to analysis and has been leveraged frequently in the literature to provide theoretical guarantees on the corresponding discrete-time process \citep{borkar2000ode}. We begin with the following system:
\begin{align*}
    \partial_t V_t(x) = \mathbb{E}_\pi[R_0 + \gamma V_t(X_1)|X_0 = x]- V_t(x) \, ,
\end{align*}
for each $x \in \mathcal{X}$, which may also be written 
\begin{align*}
    \partial_t V_t(x) = R^\pi(x) + \gamma (P^\pi V_t)(x) - V_t(x) \, ,
\end{align*}
or in full matrix notation,
\begin{align}\label{eq:value-function-ode}
    \partial_t V_t = -(I - \gamma P^\pi) V_t + R^\pi \, .
\end{align}
The differential equation in \eqref{eq:value-function-ode} is an affine autonomous system, and is straightforwardly solvable.
\begin{restatable}{lemma}{lemODESoln}\label{lem:ode-soln}
If $(V_t)_{t \geq 0}$ satisfies Equation~\eqref{eq:value-function-ode} with initial condition $V_0$ at time $t=0$, then we have
\begin{align}\label{eq:value-function-ode-solution}
    V_t = \exp( -t (I - \gamma P^\pi ) )(V_0 - V^\pi) + V^\pi \, .
\end{align}
\end{restatable}
We recover as a straightforward corollary the well-known result that $V_t \rightarrow V^\pi$ as $t \rightarrow \infty$, since all eigenvalues of $(I - \gamma P^\pi)$ have strictly positive real part. 

However, the solution in Equation~\eqref{eq:value-function-ode-solution} also describes the \emph{trajectory} by which $V_t$ reaches this limiting value.  Figure~\ref{fig:two-state-example} provides an illustration of this in a small MDP; the value functions accumulate along a particular affine subspace of $\mathbb{R}^\mathcal{X}$ prior to convergence.
This phenomenon can be formalized via the Grassmann distance.
\begin{assumption}\label{assume:value-function-conditions}
    $P^\pi$ is real-diagonalizable, with strictly decreasing eigenvalue sequence $1=\lambda_1 > \lambda_2 > \cdots >  \lambda_{|\mathcal{X}|}$, and corresponding right-eigenvectors $U_1, \ldots, U_{|\mathcal{X}|}$. 
\end{assumption}

\begin{restatable}{proposition}{propOneValueFunction}\label{prop:one-value-function}
    Under Assumption~\ref{assume:value-function-conditions}, and $(V_t)_{t \geq 0}$ the solution to Equation~\eqref{eq:value-function-ode}, for almost every\footnote{In the measure-theoretic sense that the set of excluded initial  conditions $V_0$ has Lebesgue measure $0$.} initial condition $V_0$, we have
    \begin{align*}
        d(\langle V_t - V^\pi \rangle, \langle U_1 \rangle) \rightarrow 0 \, .
    \end{align*}
\end{restatable}
\begin{proof}[Proof sketch]
The full proof of this statement and those that follow can be found in Appendix~\ref{sec:proofs-rl-dynamics}. We provide a sketch as follows: note that we can write the value function trajectory $V_t$ with respect to the eigen-basis $U_1, \dots, U_{|\states|}$, where $V_0 = \sum_{i=1}^{|\states|} \alpha_0^i U_i$, as $V_t = \sum_{i=1}^{|\states|} \alpha_0^i \exp(-\lambda_i t) U_i = \sum_{i=1}^{|\states|} \alpha_t^i U_i$. This then gives the following form of the error $V_t - V^\pi$
    \begin{align}
        V_t - V^\pi &= \exp(-t(I - \gamma P^\pi)) (V_0 - V^\pi) = \sum_{i=1}^{|\mathcal{X}|} \alpha_i \exp(t(\gamma \lambda_i - 1)) U_i \\
        &= \alpha_1 \exp(t(\gamma - 1)) U_1 + o(\exp(t(\gamma \lambda_1 - 1))\label{eq:suff-grass}
    \end{align}
whenever $\lambda_2 < \lambda_1$. Via a lemma in Appendix~\ref{sec:aux-results}, we can show that this condition \eqref{eq:suff-grass} is sufficient to guarantee that $V_t - V^\pi$ converges to $U_1$ in Grassmann distance.
\end{proof}
We can observe this predicted behaviour in Figure~\ref{fig:two-state-example}, where the value function $V_t$ converges to the line $V^\pi + \alpha \mathbbm{1}$, i.e. the value function offset by a constant function under temporal difference learning dynamics. Note that $\mathbbm{1}$ is an eigenvector of any stochastic matrix, corresponding to eigenvalue $1$, the largest eigenvalue (i.e. $U_1$ for any ergodic $P^\pi$ will be the constant function).
A more general version of this statement can also be given with an ensemble of $\repdim$ value functions, which indicates that yet more information about the environment is contained in the learned collection. The proofs of these results relate to the {power method} in linear algebra.

\begin{restatable}{proposition}{propManyValueFunctions}\label{prop:many-value-functions}
    Under Assumption~\ref{assume:value-function-conditions}, and $(V^{(\repix)}_t)_{t \geq 0}$ the solution to Equation~\eqref{eq:value-function-ode} for each $\repix=1,\ldots,\repdim$, for almost every initial condition $(V_0^{(\repix)})_{\repix=1}^\repdim$, we have
    \begin{align*}
        d(\langle V^{(\repix)}_t - V^\pi \mid \repix \in [\repdim] \rangle, \langle U_{1:\repdim} \rangle) \rightarrow 0 \, .
    \end{align*}
\end{restatable}
The proof of this statement is analogous to that of Proposition~\ref{prop:one-value-function} and can be found in Appendix~\ref{sec:proofs-rl-dynamics}. It depends on an eigendecomposition of the transition operator, which allows us to characterize the convergence of the value function along each eigenspace. Because the value function converges more slowly along the principal eigenspaces of $P^\pi$, most of the mass of the resulting error term will reside in these spaces as training progresses. The intermediate values of $V_t$ are therefore biased towards the principal subspaces of $P^\pi$, and in this sense give us some information about the transition structure of the environment.
\keyinsight{
Even in an environment with no reward signal at all (in which case $V^\pi = 0$), an agent performing TD learning still picks up information about the transition structure of the environment within its value function.}

Due to the importance of the vectors $U_{1:K}$ in this analysis, we introduce the term \emph{eigen-basis functions} (EBFs) to describe them. Intuitively, the eigen-basis functions $U_{1:K}$ provide a basis for the $K$-dimensional subspace to which $V_t$ converges; we illustrate two demonstrative eigen-basis functions of a random walk on a gridworld in Figure~\ref{fig:ensemble_feature0}. For any MDP $\mdp$ and any policy $\pi$, we will have that $U_1 = \mathbbm{1}$, the constant function. In general, the eigenvectors corresponding to large positive eigenvalues of $P^\pi$ will be `smooth' with respect to the transition dynamics, in the sense that they will not vary much in expected value after an application  of $P^\pi$. This property will be discussed further in Chapter~\ref{chp:gen-rl}. 

We observe that a similar analysis, indicating similar behaviour, is possible for related learning algorithms such as $n$-step temporal difference learning and TD($\lambda$); see Appendix~\ref{sec:beyond-one-step} for further details. In contrast, Monte Carlo learning dynamics correspond to the differential equation
\begin{align*}
    \partial_t V_t = (I - \gamma P^\pi)^{-1}R^\pi - V_t \, ,
\end{align*}
which has the solution
\begin{align*}
    V_t = e^{-t}(V_0 - (I - \gamma P^\pi)^{-1} R^\pi) + (I - \gamma P^\pi)^{-1} R^\pi \, .
\end{align*}
The trajectory associated with this solution simply linearly interpolates between $V_0$ and $V^\pi$, as illustrated in Figure~\ref{fig:two-state-example}, and does not pick up any additional information about the environment in the value function as learning proceeds. See Appendix~\ref{sec:beyond-one-step} for further details.
This example serves to illustrate that it is not just \emph{what} an agent learns ($V^\pi$), but \emph{how} the agent learns (the trajectory $V_t$) that plays a key, measurable role in what environment information is incorporated in its value function. The notion of information used here is characterized by the subspace in which the learned value function lies; this property naturally relates to linear function approximation in determining the set of functions which can be well-approximated using a given set of features, leading us to consider whether similar subspace convergence properties to those shown for value functions arise in an agent's representation. 


\subsection{Representation dynamics}
\label{sec:rep-dynamics}
We now seek to apply similar dynamics analysis to our model of representation learning to better understand the learning dynamics followed by an agent's representation. Recall the parameterization of $V \in \mathbb{R}^{\mathcal{X}}$ from Section~\ref{sec:reps}, taking the form
\begin{align*}
    V = \Phi \mathbf{w} \, ,
\end{align*}
for $\Phi \in \mathbb{R}^{\mathcal{X} \times \repdim}$, $\mathbf{w} \in \mathbb{R}^{\repdim}$. 
Central to deep reinforcement learning is the idea that $\Phi$ and $\mathbf{w}$ are simultaneously learned from a single RL loss. 
As in the value function case, we will focus on the dynamics with single-step temporal difference learning; remarks on other learning algorithms are given in Appendix~\ref{sec:beyond-one-step}. 
The dynamics associated with single-step TD learning are given by
\begin{align}
    \partial_t \Phi_t & = -\alpha \frac{1}{2}\nabla_{\Phi_t} \| R^\pi + \square[\gamma P^\pi \Phi_t \mathbf{w}_t] - \Phi_t \mathbf{w}_t \|^2_2 \, , \label{eq:phi-ode} \\
    \partial_t \mathbf{w}_t & = -\beta\frac{1}{2}\nabla_{\mathbf{w}_t} \| R^\pi + \square[\gamma P^\pi \Phi_t \mathbf{w}_t] - \Phi_t \mathbf{w}_t \|^2_2 \label{eq:w-ode} \, ,
\end{align}
where $\alpha, \beta \in [0, \infty)$ are learning rates, implying that features and weights may be learned at different rates.
Further, $\square[\, \cdot\, ]$ denotes a \emph{stop-gradient} on its argument, indicating that we treat the instances of $\Phi_t$ and $\mathbf{w}_t$ within the expression as constants when computing derivatives; this reflects the fact that temporal difference learning is a \emph{semi-gradient} method \citep{sutton2018reinforcement}.

The use of a single loss to learn both the representation and weights corresponds to the approach taken in deep RL, and we will use these dynamics as an idealized model of the dynamics of deep neural networks. While this model ignores some practicalities of deep RL (such as visitation distributions, implicit bias from the function approximation architecture, and stochasticity introduced by minibatch training), it allows us to obtain valuable insights into representation dynamics which, as we will see in Section \ref{sec:experiments} and in the following chapters, accurately predict the behaviour of deep RL agents. 

\begin{restatable}{lemma}{lemCoupledDynamics}
Let $\Phi_t$ and $\mathbf{w}_t$ parameterize a value function approximator as defined above. Then
\begin{align}
    \partial_t \Phi_t & = \alpha (R^\pi + \gamma P^\pi \Phi_t \mathbf{w}_t - \Phi_t \mathbf{w}_t) \mathbf{w}_t^\top \, , \label{eq:phi-flow} \\
    \partial_t \mathbf{w}_t & = \beta \Phi_t^\top(R^\pi + \gamma P^\pi \Phi_t \mathbf{w}_t - \Phi_t \mathbf{w}_t) \label{eq:w-flow} \, .
\end{align}
\end{restatable}

This joint flow on $\Phi_t$ and $\mathbf{w}_t$ leads to much richer behaviour than the flow considered on value functions in the previous section. Without further assumptions, the evolution of the representation $\Phi_t$ may be complex, and will not necessarily incorporate environment information as described for the case of value functions in Proposition~\ref{prop:many-value-functions}.
In particular, in sparse-reward environments, the agent may learn to predict a near-zero value function by setting the weights $\mathbf{w}_t$ close to zero, which would effectively prevent any further updating of the features $\Phi_t$, ruling out the possibility of a result analogous to Proposition~\ref{prop:many-value-functions}. We will only be able to obtain the upcoming result of Theorem~\ref{thm:infinite-heads} by restricting the rate at which the weights evolve over the course of training.

\section{Auxiliary tasks}\label{sec:aux-dynamics} 
Having studied the temporal difference learning dynamics in Equations~\eqref{eq:phi-flow} \& \eqref{eq:w-flow}, we now examine how auxiliary value-prediction tasks influence the behaviour of the agent's representation during the learning process. 
As described previously, developing a granular description of the joint learning dynamics of the representation and weights of the learner is a complex task, and so we focus on the limiting case in which the number of auxiliary tasks is large relative to the dimensionality of the representation. We conclude that under certain conditions, representations learned in the many-task limit bear a close connection to the \emph{eigen-basis functions} described in Section~\ref{sec:value-function}, and also to the \emph{resolvent singular basis functions}, a new decomposition introduced in Section~\ref{sec:random-cumulants}. The reader may find it useful to refer to Appendix~\ref{sec:feature-selection} for a more detailed discussion of these decompositions.

\begin{table*}
\smaller
    \centering
    \begin{tabular}{c|c|c|c|c}
        \toprule
        Auxiliary task & Dynamics ($r=0$) & $\Phi_\infty$ ($r=0$)  & $\Phi_\infty$ ($r\neq 0$) & $\lim_{t \rightarrow \infty} \langle \Phi_t - \Phi_\infty \rangle$\\
        \midrule
        Ensemble & $-(I - \gamma P^\pi) \Phi_t$ & $0$ & $(I - \gamma P^\pi)^{-1} r\epsilon^\top$ & EBFs of $P^\pi$\\
        Random cumulants & $-(I - \gamma P^\pi) \Phi_t + Z_{\Sigma}$ & $\Psi Z_{\Sigma}$ &   $\Psi Z_{\Sigma}$ & EBFs of $P^\pi$\\
        Additional policies & $  -(I - \gamma P^{\overline{\pi}}) \Phi_t$ & $0$ & $ (I - \gamma P^{\overline{\pi}})^{-1} R^{\overline{\pi}} \epsilon^\top$ & EBFs of $P^{\bar{\pi}}$\\
        Multiple $\gamma$s & $-(I - \overline{\gamma} P^\pi)\Phi_t$ & $0$  & $(I - \overline{\gamma} P^{\pi})^{-1}R^{\pi} \epsilon^\top$ & EBFs of $P^{\pi}$\\
        \bottomrule
    \end{tabular}
    \caption[Summary of dynamics and limiting solutions under some common auxiliary tasks in the limit of infinitely-many prediction outputs.]{Summary of dynamics and limiting solutions under some common auxiliary tasks in the limit of infinitely-many prediction outputs. For additional policies, $\overline{\pi}$ denotes the average of the finite set of policies $\pi_1,\ldots,\pi_L$ under consideration ($L$ fixed and independent of $M$), and for multiple discount factors, $\overline{\gamma}$ denotes the average of the discount factors $\gamma_1,\ldots,\gamma_L$ under consideration. We let $\Psi = (I - \gamma P^\pi)^{-1}.$ See Section~\ref{sec:random-cumulants} and Appendix~\ref{appx:feature_theory} for additional details.
    }
    \label{tab:theory}
\end{table*}
\subsection{Ensemble value prediction}

We begin by considering the auxiliary task of \emph{ensemble value prediction} \citep{osband2016deep, anschel2017averaged, agarwal2019striving}. Rather than making a single prediction of the value function $Q^\pi$, the learner makes $M \in \mathbb{N}$ separate predictions as linear functions of a common representation $\Phi^{M}$, using $M$ independently initialized weights matrices $\mathbf{w}^{m} \in \mathbb{R}^{\repdim}$ ($m=1,\ldots,M$). We note that while at initialization $\Phi^M_0 \in \mathbb{R}^{\statespace \times d}$ is independent of $M$, its dynamics do depend on $M$ through the contribution of the weights. Simultaneous temporal difference learning on all predictions leads to the following dynamics:
\begin{align}
    \partial_t \Phi^{M}_t  \label{eq:ensemble-phi-flow}
    \!=  &   \alpha\! \sum_{m=1}^M (R^\pi\! +\! \gamma P^\pi \Phi^{M}_t \mathbf{w}_t^{m}\! -\! \Phi^{M}_t \mathbf{w}_t^{m})  (\mathbf{w}_t^{m} )^\top \, , \\
    \partial_t \mathbf{w}_t^{m} = & \beta (\Phi^{M}_t)^\top (R^\pi + \gamma P^\pi \Phi^M_t \mathbf{w}^{m}_t - \Phi_t \mathbf{w}^{m}_t ) \, . \label{eq:ensemble-w-flow}
\end{align}

The following result characterizes the representation learned by the agent in the many-tasks limit, again establishing a connection to EBFs; we follow the approach described by \citet{arora2019fine} in fixing the linear weights associated with the value function; this dramatically simplifies our analysis, while still describing practical settings in which the features and weights are trained separately as in \citet{chung2018two}. Analogous results follow when the \textit{learning rate} $\beta$ is scaled appropriately, resembling the results of \citet{jacot2018neural} and \citet{yang2021tensor}.

\begin{restatable}{theorem}{thmInfiniteHeads}\label{thm:infinite-heads}

For $M \in \mathbb{N}$, let $(\Phi^{M}_t)_{t \geq 0}$ be the solution to Equation~\eqref{eq:ensemble-phi-flow}, with each $\mathbf{w}^{m}_t$ for $m=1,\ldots,M$ initialized independently from $N(0, \sigma_M^2)$, and fixed throughout training ($\beta=0$). 
We consider two settings: first, where the learning rate $\alpha$ is scaled as $\frac{1}{M}$
and $\sigma_M^2 = 1$ for all $M$, and second where $\sigma_M^2 = \frac{1}{M}$ and the learning rate $\alpha$ is equal to $1$. These two settings yield the following dynamics, respectively:
\begin{align}
       \lim_{M \rightarrow \infty} \partial_t \Phi_t^{M} \overset{P}{=}& -(I - \gamma P^\pi )\Phi_t^{M}\quad  \text{, and } \\
       \lim_{M \rightarrow \infty} \partial_t \Phi_t^{M} \overset{D}{=}& -(I - \gamma P^\pi )\Phi_t^{M} + R^\pi \epsilon^\top \; \text{,  $\epsilon \sim \mathcal{N}(0, I)\,$.}
\end{align}
The corresponding limiting trajectories for a fixed initialization $\Phi_0 \in \mathbb{R}^{\statespace\times \repdim}$ are therefore given respectively by
\begin{align}
        \lim_{M \rightarrow \infty} \Phi_t^{M} \overset{P}{=}& \exp(-t(I - \gamma P^\pi))\Phi_0 \quad  \text{, and } \\
         \lim_{M \rightarrow \infty} \Phi_t^{M}  \overset{D}{=}& \exp(-t(I - \gamma P^\pi))(\Phi_0 - (I - \gamma P^\pi)^{-1} R^\pi \varepsilon^\top ) \nonumber \\
        & \qquad \ + (I - \gamma P^\pi)^{-1}R^\pi \varepsilon^\top  \, ,\,  \epsilon \sim \mathcal{N}(0, I)\,.
\end{align}
\end{restatable}
\begin{proof}[Proof sketch]
The key step in the proof of this result is to show that the limiting dynamics of $\phi^M_t$ depend on the matrix of outer products $\sum_{m=1}^M \bw_m \bw_m^\top$, which under the conditions provided in the theorem statement can be shown to converge to the identity matrix. In the case of scaled initializations, we obtain the following, though the derivation in the case of scaled learning rates is similar.
\begin{align}
    \partial_t \Phi^M_t &= (I - \gamma P^\pi)\Phi_t^M \sum_{m=1}^M \mathbf{w}^m (\mathbf{w}^m)^\top + \sum_{m=1}^M R^{\pi} (\mathbf{w}^m)^\top \\
    \lim_{M \rightarrow \infty} \partial_t \Phi^M_t &= (I - \gamma P^\pi)\Phi_t^M \lim_{M \rightarrow \infty}\sum_{m=1}^M \mathbf{w}^m (\mathbf{w}^m)^\top  + \lim_{M \rightarrow \infty} R^\pi(\sum_{m=1}^M \mathbf{w}^m)^\top \\
    &\overset{D}{=} (I - \gamma P^\pi )\Phi_t^M I + R^\pi \epsilon^\top, \; \epsilon \sim \mathcal{N}(0, I).
\end{align}
In the case of zero rewards, we then get analogous dynamics on $\Phi_t^M$ as we did on the value functions in Proposition~\ref{prop:many-value-functions}.
\end{proof}

In contrast to the case described in Section~\ref{sec:rep-dynamics}, this result indicates that the introduction of auxiliary tasks leads to useful environment information being incorporated into the representation.
Indeed, the dynamics described above imply the following convergence result, analogous to Proposition~\ref{prop:many-value-functions}.

\begin{restatable}{corollary}{propSubspaceConvergence}\label{prop:subspaceconvergence}
Under the feature flow \eqref{eq:ensemble-phi-flow} with $\mathbf{w}^m_t$ fixed at initialization for each $i = 1, \dots, M$, under either the scaled initialization or scaled learning rate assumption of Theorem~\ref{thm:infinite-heads}, and under Assumption~\ref{assume:value-function-conditions}, for almost all initializations $\Phi_0$, we have when $R^\pi = 0$
    \begin{align*}
        d(\langle \Phi_t \rangle, \langle U_{1:\repdim} \rangle) \rightarrow 0 \, ,\quad\text{as } t \rightarrow \infty.
    \end{align*}
\end{restatable}
The proof of this result follows straightforwardly from Theorem~\ref{thm:infinite-heads} and the proof of the convergence of value functions to the subspace spanned by $U_1, \dots, U_k$, and can be found in Appendix~\ref{sec:proofs-rl-dynamics}.
\keyinsight{Under the conditions of Theorem~\ref{thm:infinite-heads} and Corollary~\ref{prop:subspaceconvergence}, the ensemble auxiliary tasks cause the agent's representation $\Phi$ to align with EBFs.}

We show in Appendix \ref{sec:ensemble-dynamics} that this behaviour is observed in practice when $M \gg \repdim$ and the value of $\mathbf{w}^m_t$ is fixed at initialization for all $m$. We additionally compare the representations learned when $\mathbf{w}^m_t$ is allowed to vary over training. Here we find empirically that allowing the weights to vary during training induces dynamics that differ from those predicted by Theorem~\ref{thm:infinite-heads} for the fixed-weights setting.
To illustrate this, we follow the evolution of a single column of $\Phi_t$, i.e. a single feature vector $\phi_t = \Phi_t[ \;  : \,, 1 \; ]$, trained with the ensemble prediction dynamics of Equations~\eqref{eq:ensemble-phi-flow} \& \eqref{eq:ensemble-w-flow} on a simple four-rooms gridworld environment in
Figure~\ref{fig:feature-viz}.

\begin{figure}
    \centering
    \includegraphics[width=0.48\textwidth]{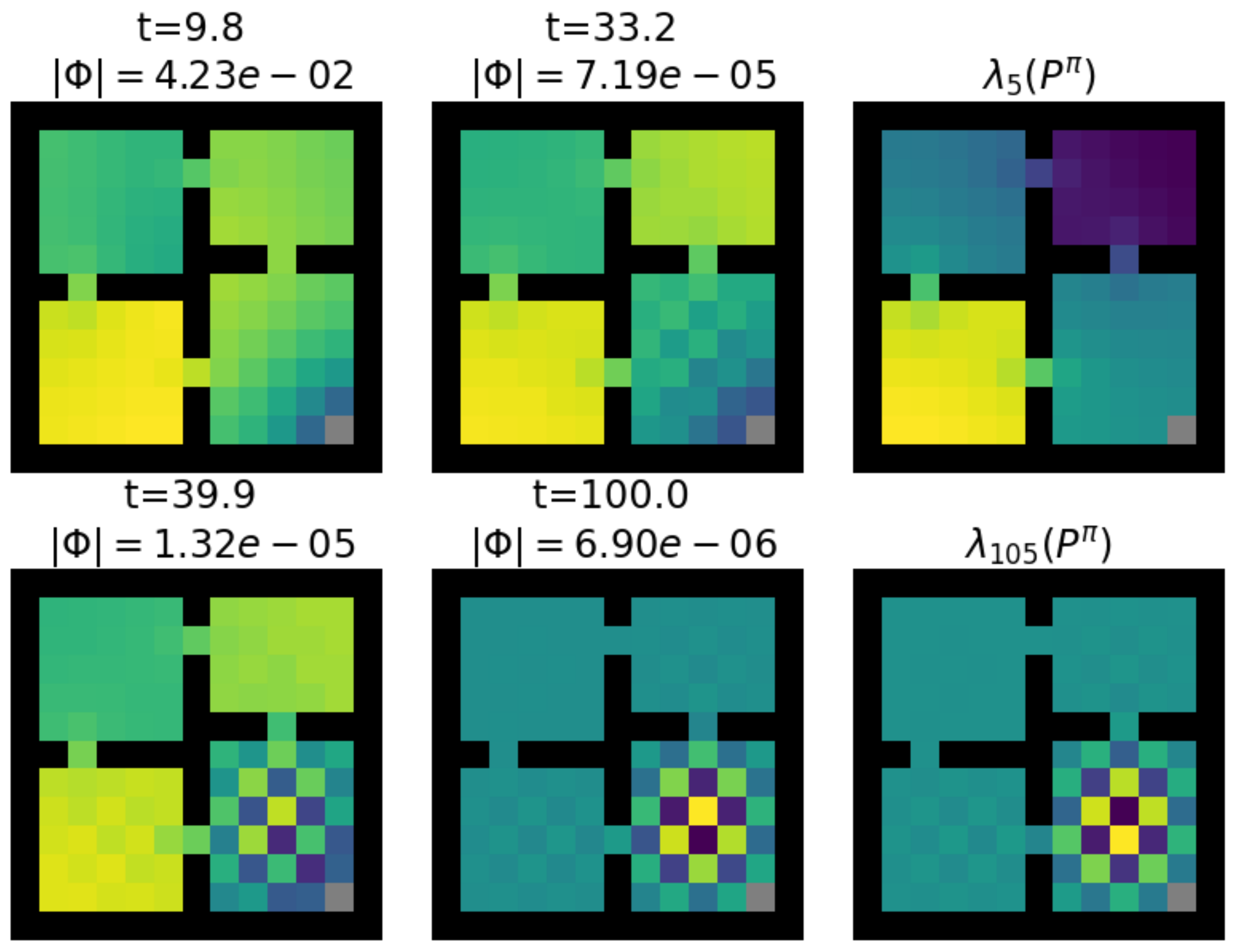}
    \caption[Study of feature dynamics.]{Visualization of a single column of $\Phi_t$ (i.e. feature vector)  after application of the ODE in Equation~\eqref{eq:phi-flow} for $t\in[0,100]$ in the four rooms environment, with $\repdim=10$ and $M = 20$. Early in its trajectory, $\phi_t$ exhibits similarity to smooth eigenfunctions (e.g. the eigenfunction corresponding to the 5$^{th}$ greatest eigenvalue $\lambda_5$ which we plot in the top right) of $P^\pi$, but later converges to non-smooth eigenfunctions (e.g. the eigenfunction corresponding to eigenvalue $\lambda_{105}$, the most negative eigenvalue, plotted in the bottom right).
    }
    \label{fig:feature-viz}
\end{figure}
 
We visualize $\phi_t$ along with two illustrative eigenvectors of the transition matrix $P^\pi$, corresponding to one positive and one negative eigenvalue. 
We observe that while the feature $\phi_t$ quickly evolves to resemble the smooth eigenvector corresponding to the positive eigenvalue for small values of $t$, it later converges to the non-smooth eigenvector corresponding to the most negative eigenvalue of the transition matrix $P^\pi$. While we leave further analysis to future work, this example hints at an intriguing relationship between the EBFs and the joint representation dynamics.

\subsection{Random cumulants}\label{sec:random-cumulants}

In the case of zero rewards, our previous results show that while from the perspective of subspaces the representation approaches the EBF subspace in Grassmann distance, in Euclidean distance the representation is approaching the zero matrix pointwise. This has important implications for the scenario of large-scale sparse-reward environments, in which the agent may not encounter rewards for long periods of time, and indicates that the agent's representation is at risk of collapsing in such cases.

Motivated by this analysis, we consider a means of alleviating this representation collapse, by learning value functions for \emph{randomly generated cumulants} \citep{osband2018randomized,dabney2020value}. Mathematically, the agent again makes many predictions from a common representation, with each prediction indexed by $m=1,\ldots,M$ attempting to learn the value function associated with a randomly drawn reward function $r^m \in \mathbb{R}^{\statespace}$ under the policy $\pi$. Thus, the agent's parameters are the representation $\Phi$ and a set of weights $\mathbf{w}^m$ for each prediction. The learning dynamics are then given by:
\begin{align}
    \partial_t \Phi^{M}_t  \label{eq:rc-phi-flow}
    \!=  &  \alpha \sum_{m=1}^M (r^m\!\! +\!\! \gamma P^\pi \Phi^{M}_t \mathbf{w}_t^{m}\!\! -\!\! \Phi^{M}_t \mathbf{w}_t^{m})  (\mathbf{w}_t^{m} )^\top  , \\
    \partial_t \mathbf{w}_t^{m} = & \beta (\Phi^{M}_t)^\top (r^m + \gamma P^\pi \Phi^M_t \mathbf{w}^{m}_t - \Phi_t \mathbf{w}^{m}_t ) \, . \label{eq:rc-w-flow}
\end{align}

The main result of this section is to show that, even in the absence of reward, the limiting distribution induced by random cumulant auxiliary tasks dynamics described in Equation~\eqref{eq:rc-phi-flow} is a non-zero subspace which, while dependent on the random reward function, is biased towards principal components of the transition matrix. 

\begin{restatable}{theorem}{ThmDistribution}\label{thm:distribution}
For fixed $M \in \mathbb{N}$, let the random rewards $(r^m)_{m=1}^M$ and weights $(\mathbf{w}^m)_{m=1}^M$ be as defined above, let $\alpha=1$, and consider the representation dynamics in Equation~\eqref{eq:rc-phi-flow}, with weights fixed throughout training ($\beta=0$). Let $\Sigma$ denote the covariance matrix of the random cumulant distribution. Then 
\begin{align}
    &\lim_{M \rightarrow \infty} \sum_{m=1}^M r^m   (\mathbf{w}^m)^\top \overset{D}{=} Z_\Sigma \sim \mathcal{N}(0, \Sigma), \text{ and } \nonumber \\
   & \lim_{M \rightarrow \infty} \Phi^M_t \overset{D}{=}  \exp(-t(I - \gamma P^\pi ))(\Phi_0 - (I - \gamma P^\pi)^{-1} Z_\Sigma) \nonumber \\
    & \qquad\qquad+ (I - \gamma P^\pi)^{-1} Z_\Sigma \nonumber \;.
\end{align}

As the columns of $Z_\Sigma$ are mean-zero, uncorrelated, with covariance matrix $\Sigma$, the limiting distribution of each column of $\Phi_\infty = \lim_{t \rightarrow \infty} \lim_{M \rightarrow \infty} \Phi^M_t$ has covariance $\Psi\Sigma \Psi^\top$, where $\Psi$ is the resolvent $(I - \gamma P^\pi)^{-1}$.
\end{restatable}

The proof of this result is obtained via a similar argument as in the case of Theorem~\ref{thm:infinite-heads}, though care must be taken to account for the non-zero and non-uniform rewards, and can be found in Appendix~\ref{sec:proofs-rl-dynamics}. Importantly, while $Z_\Sigma$ may have zero \textit{expected value} over its initialization distribution, with probability one it will not be zero. As a result, we obtain convergence to a limiting representation $\Phi_\infty$ whose value depends on the randomly initialized cumulants and the transition structure of the environment. This limiting value will thus be biased towards the lower-index eigen-basis functions, but will also have non-zero dot product with higher-index EBFs, in contrast to our results in the zero-reward setting. However, while this limiting value might result in non-zero mass assigned to lower EBFs, the \textit{error term} of $\Phi_\infty - \Phi^M_t$ will inherit similar convergence properties as in the ensemble prediction setting, becoming dominated by the EBFs over time.

\begin{restatable}{corollary}{propSubspaceConvergenceRC}\label{prop:subspaceconvergence-rc}
Under the feature flow \eqref{eq:rc-phi-flow} with $\mathbf{w}^m_t$ fixed at initialization for each $i = 1, \dots, M$ and Assumption~\ref{assume:value-function-conditions}, for almost all initializations $\Phi_0$, we have when $R^\pi = 0$
    \begin{align*}
        d(\lim_{M \rightarrow \infty} \langle \Phi^M_t - \Phi_\infty \rangle, \langle U_{1:\repdim} \rangle) \rightarrow 0 \, ,\quad\text{as } t \rightarrow \infty.
    \end{align*}
\end{restatable}

Theorem~\ref{thm:distribution} indicates that the left-singular vectors of $\Sigma^{1/2}\Psi$ (or equivalently, the right-eigenvectors of $\Psi \Sigma \Psi^\top$) are key to understanding the effects of random cumulants on representations; we introduce the term \emph{resolvent singular basis functions} (RSBFs) to refer to these vectors in the canonical case $\Sigma = I$. 
\vspace{2mm}

\keyinsight{
With random cumulant auxiliary tasks, under the assumptions of Theorem~\ref{thm:distribution} and Corollary~\ref{prop:subspaceconvergence-rc}, the distribution of the limiting representation does not collapse, and is characterized by the RSBFs of $P^\pi$, while the trajectory it follows to reach this subspace is determined by the EBFs of $P^\pi$. }

These decompositions of $P^\pi$ bear deep connections to prior work on feature learning. EBFs correspond to the eigendecomposition of the successor representation, which can be explicitly related to the proto-value functions described by \citet{mahadevan2009learning} when the transition matrix $P^\pi$ corresponds to that of a random walk policy \citep{machado2017eigenoption}. For symmetric $P^\pi$ we obtain an additional correspondence between EBFs and RSBFs, though we note that when $P^\pi$ is not symmetric the RSBFs may differ from both the EBFs and the singular value decomposition of the transition matrix $P^\pi$. We provide further discussion of RSBFs and comparisons against existing concepts in feature selection in Appendix~\ref{sec:feature-selection}.

The subspace given by the RSBFs has a further appealing connection to General Value Functions (GVFs). We provide additional details in Appendix~\ref{sec:bayes-opt}, where we show that RSBFs can be viewed as Bayes-optimal features in the sense that they minimize the expected value function approximation error given an isotropic Gaussian prior on an unknown reward function.

\subsection{Analysis of additional auxiliary tasks}

The infinite-task limit simplifies the analysis of a broad range of auxiliary tasks, and analogous results to Theorem \ref{thm:infinite-heads} can be easily derived for other tasks that involve value prediction for some discount factor and policy. We provide a summary of these results in Table \ref{tab:theory}, including their full statements and derivations in Appendix~\ref{apx:table-results}.
We consider two additional classes of auxiliary task: predicting the values of multiple policies \citep{dabney2020value}, and predicting the value function of the current policy under multiple discount factors \citep{fedus2019hyperbolic}.

Under the \textbf{multiple policies} auxiliary task, the agent's objective is to learn a set of value functions $V^1, \dots, V^M$ such that $V^i(x) = \mathbb{E}_{\pi_i}[R^{\pi_i}(x) + \gamma P^{\pi_i}V^i(x)]$. 
We consider an ensemble prediction variant of this objective, where given a fixed set of $k$ policies, we train an ensemble of $M$ predictors $V^{1,1}, \dots, V^{m, 1}, \dots, V^{1,k}, \dots, V^{m, k}$, where $m=\frac{M}{k}$ and the value function $V^{i, j}$ is trained on policy $\pi_j$. 

Under the \textbf{multiple discount factors} auxiliary task, the agent's objective is analogously to find $V^i(x) = \mathbb{E}_{\pi_i}[R^{\pi}(x) + \gamma_i P^{\pi}V^i(x)]$ for $\gamma_i \in \gamma_1, \dots, \gamma_k$. As with the multiple policies auxiliary task, we assign to each discount factor objective multiple prediction heads $V^{1,1}, \dots, V^{m, 1}, \dots, V^{1,k}, \dots, V^{m, k}$, where $m=\frac{M}{k}$ and the value function $V^{i, j}$ is trained on discount factor $\gamma_j$. 

In both cases, under the conditions of the previous theorems, the dynamics of the ensemble converge to the dynamics induced by the mean of the set of auxiliary tasks, implying the counter-intuitive result that training with multiple auxiliary tasks does not provide additional utility over the single task setting. This apparent shortcoming stems from the large number of auxiliary prediction objectives all operating on the same `feature subspace' of $\mathbb{R}^d$, essentially forcing any dimension of that subspace to fit the mean target value. It can be modified to more closely resemble practical settings by ensuring that the weights corresponding to each auxiliary task $\pi_i$ or $\gamma_i$ are initialized in \textit{orthogonal subspaces}, so that the vector space $V$ in which the representation evolves can be decomposed as the direct sum of the subspaces $V_i$ associated with each task $i$, $V = \oplus_{i\in[1, k]} V_i$. In this case, we obtain an analogous decomposition of the representation $\Phi$ and its corresponding dynamics, obtaining convergence to a direct sum of the limiting representation of each task. This suggests that the benefits of auxiliary tasks might be maximized by appropriate initialization schemes which encourage the representations learned for each task to be (linearly) independent.

\subsection{Dynamics in the infinite-width limit}

So far, the analysis of this chapter has studied an idealized feature-learning model where the features are represented by an extremely large matrix, and for which an update to the feature vector of one state will not influence the features of other states. This simplifies our analysis, but diverges from the function approximation regimes seen in practice, where a neural network is often used to construct the feature representation. We take a step towards the deep RL setting by considering a limiting case of DNN function approximation.

Recent results on neural networks in the limit of infinite width reveal that their learning dynamics are determined by a kernel, denoted the \emph{neural tangent kernel} (NTK) \citep{jacot2018neural}, described in more detail in Section~\ref{sec:bkgd-trajectory}. This allows us to apply the analysis of the previous sections to the infinite-width limit of neural networks, for which we obtain the following dynamics.
\begin{equation}
    \partial_t V_t = \Phi_{\Theta^{(L)}_\infty} (\gamma P^\pi - I)V_t \; .
\end{equation}
We now see that the kernel matrix $\Phi_{\Theta^{(L)}}$ influences the trajectory of the value function, rather than solely $P^\pi$ as previously. We can now directly quantify how the auxiliary losses influence the trajectory of the value function in the infinite-width regime, as we describe in the following theorem.
\begin{theorem}\label{thm:ntk-dynamics}
Let $f = (f^1, \dots, f^M): \statespace \rightarrow \mathbb{R}^M$ be computed by an $L$-layer neural network with layer widths $n_1, \dots, n_L=M$, parameters $\theta$, and Lipschitz nonlinearity $\sigma$. Let $f^1_t$ follow the Bellman-error minimizing flow of Eq \eqref{eq:value-function-ode}, and let the joint loss $\delta$ be of the form 
\begin{equation*}\delta(f) = \delta_1(f^1) + \delta_2(f^2, \dots, f^k)\;.
\end{equation*}
Under the conditions of Theorem 2 of \citet{jacot2018neural}, with the limiting NTK matrix $\Phi_{\Theta^{(L)}_\infty}$ defined therein, the dynamics $\partial_t f^1_{t}$ are independent of the values of $f^{2}_{t}, \dots, f^{k}_{t}$. Thus, for any set of auxiliary tasks $\delta_2$, number of heads $k > 1$, and auxiliary head values $f^2_t, \dots, f^m_t$,
\begin{equation}
  \lim_{n_1, \dots n_{L-1}\rightarrow \infty} \partial_t f^1_{t} = \Phi_{\Theta^{(L)}_\infty} \nabla_{f^1} \delta_1 (f_t^1) \;.
\end{equation}
\end{theorem}

This result highlights the importance of the shared feature dynamics between the network outputs in order for auxiliary tasks to influence value function learning. The initialization and dynamics of the NTK regime result in each head of the network evolving independently; essentially, each of the auxiliary tasks is operating on a mutually-orthogonal subset of the feature space, and as a result incorporating auxiliary tasks will not accelerate learning of the value function. This is in contrast to most DNN architectures used in deep RL, which tend to be relatively small and so exhibit significant interference (both positive and negative) between tasks. 
Two notable exceptions to this observation are categorical distributional reinforcement learning, where the softmax operator induces dependence between the different heads' loss functions, and random ensemble mixture (REM) \citep{agarwal2019striving}, where a random mixture is taken over an ensemble of value prediction heads before applying the TD loss. The analysis of the effect of this family of auxiliary tasks on representation dynamics in wide neural networks requires additional mathematical tools to account for their nonlinear output layer, however, and falls outside the scope of this chapter. Other parameterizations such as those studied by \citet{yang2021tensor} may further yield more interesting feature-learning dynamics even for the classes of auxiliary tasks studied in this chapter.

\section{Experiments}\label{sec:experiments}

In this section, we complement the previous theoretical results with empirical investigations in both tabular and deep reinforcement learning settings. These empirical results will motivate the directions taken in Chapter~\ref{chp:rep-learning}, where we will further explore both the sparse- and dense-reward setting.

\subsection{Feature generalization across the value-improvement path}\label{sec:feature-generalization}

Having established connections between the representations induced by auxiliary tasks and several decompositions of the environment transition operator, we now turn to the question of how useful these representations are to a reinforcement learning agent. In particular, we address how well representations learned under one policy \textit{generalize} under the policy improvement step to approximate future value functions, with particular attention paid to EBFs and RSBFs, the decompositions that feature in our earlier analysis.

\begin{figure}[!b]
    \centering
    \includegraphics[keepaspectratio,width=.648\textwidth]{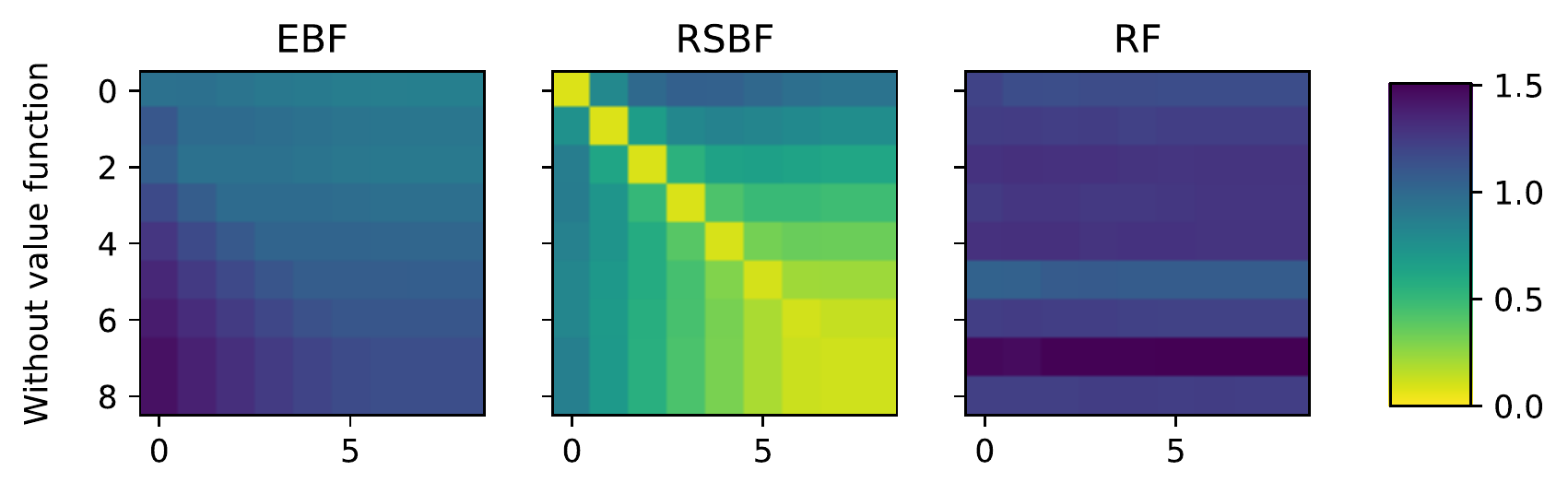}
    
    \includegraphics[keepaspectratio,width=.648\textwidth]{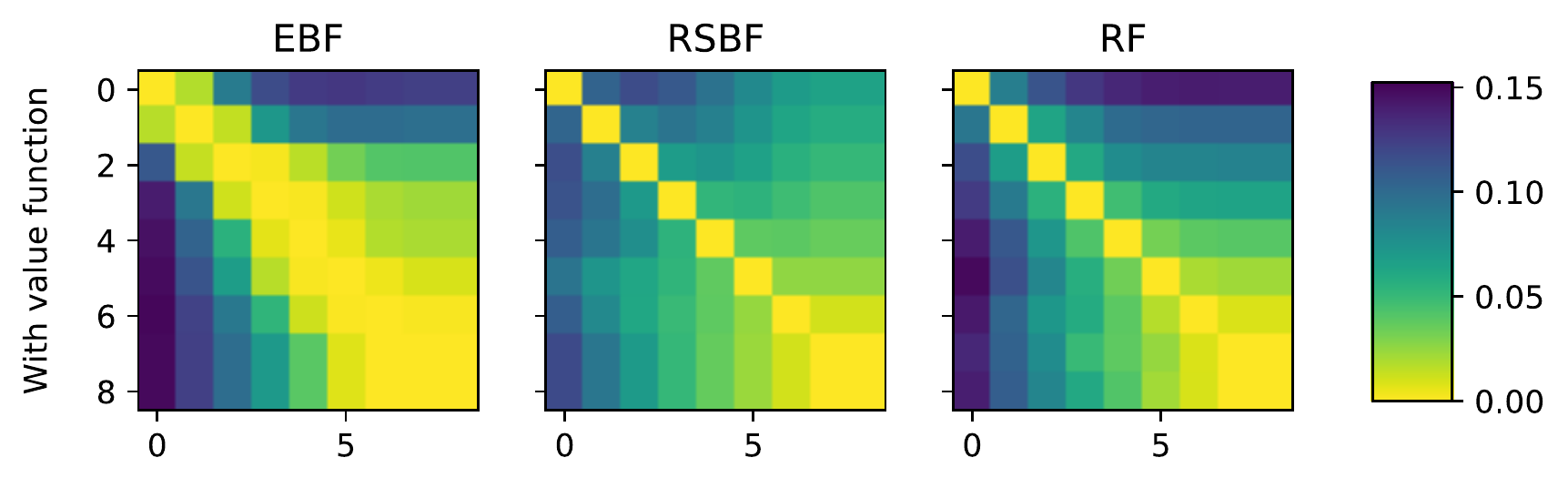}
    \caption{Transfer of EBFs, RSBFs, and RFs across the value-improvement path of a chain MDP, with and without the value function as an additional feature.
    }
    \label{fig:chain-transfer}
    \vspace{0.2cm}
\end{figure}

\begin{figure*}
    \includegraphics[width=\textwidth]{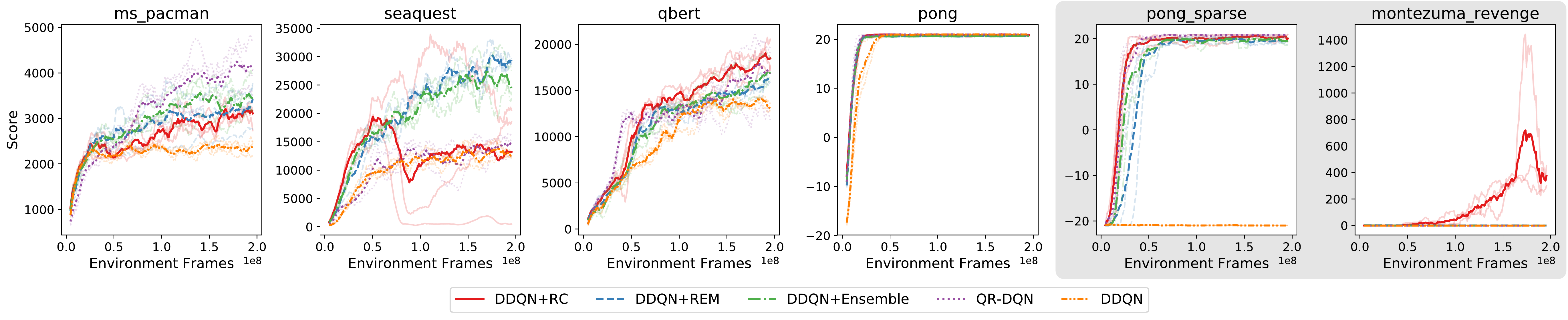}
    \caption[Learning curves for DDQN, DDQN+RC, DDQN+Ensemble, DDQN+REM, and QR-DQN agents on several dense reward ALE environments.]{Learning curves for DDQN, DDQN+RC, DDQN+Ensemble, DDQN+REM, and QR-DQN agents on several dense reward ALE environments. Two games with sparse rewards are shown in the shaded box.}
    \label{fig:deep-rl}
    \vspace{0.5cm}
\end{figure*}
\begin{figure}
    \centering
    \includegraphics[width=.75\textwidth]{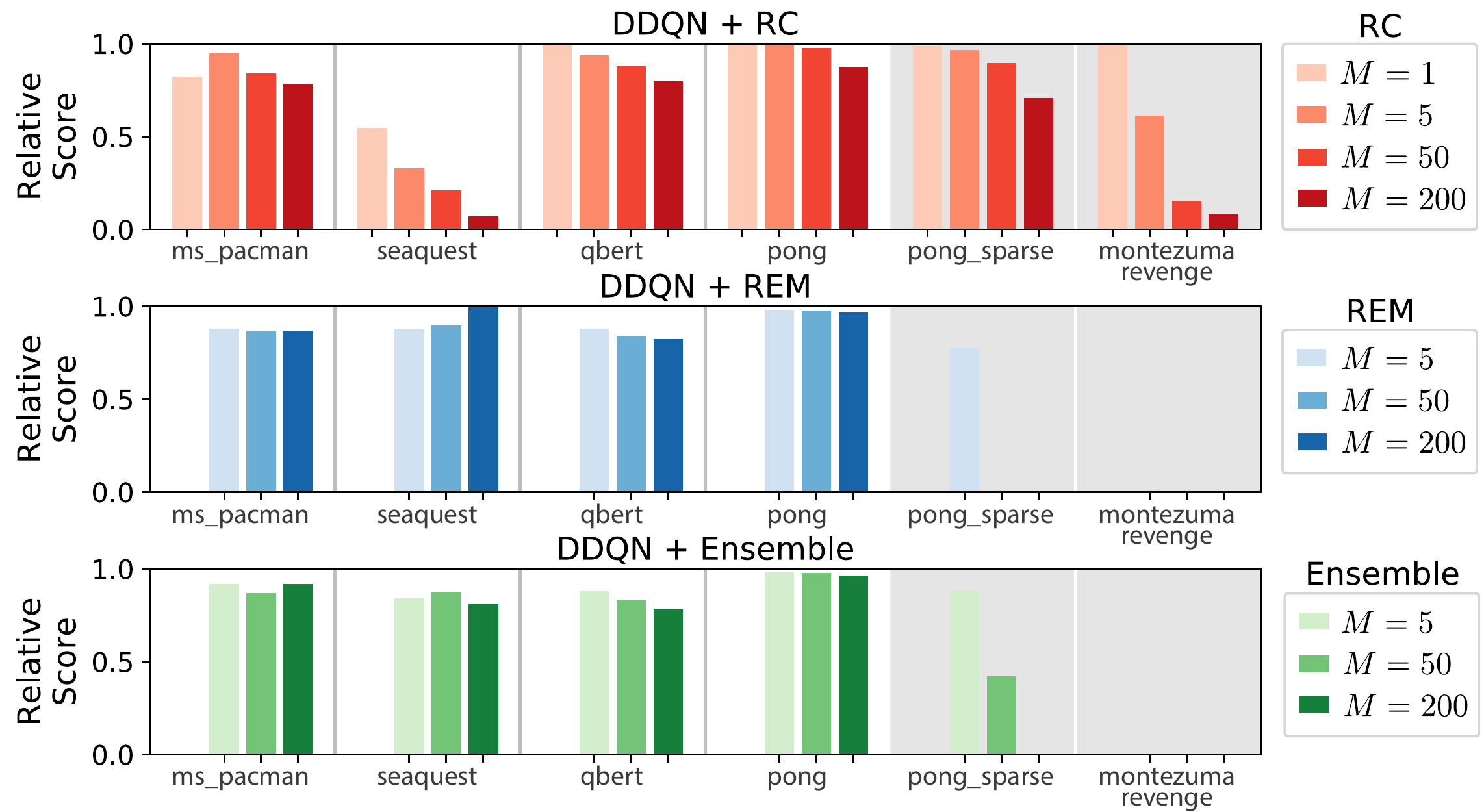}
    \caption[Sweep over number of auxiliary heads for RC, REM and Ensemble.]{Sweep over number of auxiliary heads for RC, REM and Ensemble. Each bar corresponds to a single game, presented in the same order as Figure~\ref{fig:deep-rl}. Relative score is the per-game score divided by the maximum over all algorithms and hyperparameters.}
    \label{fig:naux}
\end{figure}

To address this question empirically we run tabular policy iteration on a stochastic chain MDP, yielding a sequence of policies $(\pi_j)_{j=1}^J$ and associated value functions $(V_j)_{j=1}^J$. We then compute EBFs and RSBFs associated with $P^\pi$, and compute the acute angle between $V_j$ and the subspace spanned by these features, for each $j \in [J]$; this is in fact equal to a generalization of the Grassmann distance for subspaces of unequal dimension \citep{ye2016schubert}. We also compare against a baseline of isotropic randomly-generated features. Full experimental details are provided in Appendix~\ref{sec:experiment-details}. 

Results are given in the top row of Figure~\ref{fig:chain-transfer} for the case of four features; each individual heatmap plots Grassmann distances, with rows indexing the policy that generated the features, and columns indexing the policy yielding the target value function. In general, the RSBFs provide better transfer across policies in the improvement path relative to random features and EBFs. For times $j, j' \in [J]$, we observe that  the Grassmann distance between the RSBFs of $P^{\pi_j}$ and the value function of $j'$, $V^{\pi_{j'}}$, increases as $|j - j'|$ does. 

We also evaluate transfer when the vector $V^{\pi_j}$ is added to the set of features, in the bottom row of Figure~\ref{fig:chain-transfer}. This contains the subspace to which the value functions described in Proposition~\ref{prop:many-value-functions} converge, as the limiting solutions can be described as being of the form $V^{\pi_j} + u$ for $u \in \langle U_{1:\repdim} \rangle$. This results in a significant improvement to the ability of the EBFs to fit future value functions, suggesting that representation learning objectives which \textit{only} seek to incorporate information about the transition dynamics, without including any reward information, may be sub-optimal for representing value functions of interest. Surprisingly, we find that this addition results in the EBFs for $\pi_j$ outperforming RSBFs specifically in predicting $V^{\pi_{j+1}}$. This can be observed in the upper off-diagonal the EBF plot in Figure~\ref{fig:chain-transfer}.
We conclude that the dynamics induced by TD updates may be particularly beneficial to transfer between policies in the value-improvement path, and further study of this phenomenon is a promising avenue for future work. 

\subsection{Auxiliary tasks for large-scale environments with sparse rewards}\label{sec:deep-rl-aux}

We now turn our attention to deep reinforcement learning, particularly in the context of environments with sparse reward structure. Motivated by the theoretical results obtained in earlier sections, we study the effects of a variety of auxiliary tasks in this setting; our analysis indicates that random cumulants may be particularly effective in preventing representation collapse in such environments. Concretely, this section will seek to evaluate whether auxiliary tasks that incorporate non-zero reward signals can improve performance in sparse-reward tasks; Chapter~\ref{chp:rep-learning} will then investigate the mechanisms behind this in more depth.
\hypothesis{Representation collapse}{incorporating non-zero reward in auxiliary tasks will improve the learned representation, and thereby improve performance, by preventing representation collapse in deep RL agents.\label{hyp:collapse}}

To evaluate Hypothesis~\ref{hyp:collapse}, we empirically evaluate the performance of a variety of agents on a representative subset of Atari environments. We modify a Double DQN agent \citep{van2016deep} with a variety of auxiliary tasks, including random cumulants (RC) \citep{dabney2020value}, random ensemble mixtures (REM) \citep{agarwal2019striving}, an ensembling approach \citep{anschel2017averaged}, and also compare with QR-DQN, a distributional agent \citep{dabney2018distributional}. Full details of these agents, including specific implementation details for deep RL versions of these auxiliary tasks, are given in Appendix~\ref{sec:experiment-details}.

We evaluate these agents on a series of Atari games from the Arcade Learning Environment \citep{bellemare2013arcade,machado2018revisiting}, comprising Montezuma's Revenge, Pong, MsPacman, Seaquest, and Q$^*$bert.
In addition, we evaluate on a more challenging, sparse reward, version of Pong in which the agent does not observe negative rewards.\footnote{We attempted a similar modification of the other three dense reward games, but found no agent or configuration that was able to successfully learn on them. Full details, along with hyperparameters and results on these unsuccessful modifications, are given in Appendix~\ref{sec:experiment-details}.}

Figure~\ref{fig:deep-rl} shows the main results from these experiments.
Recall from Section~\ref{sec:aux-dynamics} that the random cumulant auxiliary task causes the agent's representation to converge to the RSBFs of $P^\pi$  in idealized settings. 
We hypothesize that this auxiliary task will therefore improve agent performance over ensemble-based auxiliary tasks in sparse-reward environments.
Our empirical results support our conjecture, with the random cumulant agent (DDQN+RC) generally performing well in the sparse-reward environments. Of particular note is the strong performance in Montezuma's Revenge. We expected reduced performance for DDQN+RC in the dense-reward games, but were surprised to observe improved performance here as well, with the exception of Seaquest.
Finally, Figure~\ref{fig:naux} shows the result of a hyperparameter sweep over the number of auxiliary task heads, revealing relevant differences in the three methods considered. Overall, we find that random cumulants are a promising auxiliary task specifically in sparse-reward environments. In the following chapter, we will investigate an approach which adapts the random cumulant task to avoid interference in dense-reward tasks like Seaquest while preserving its ability to prevent representation collapse in sparse-reward environments.

\section{Conclusions}

This chapter has introduced a framework based on learning dynamics to analyze representations in reinforcement learning. This led to a variety of theoretical results concerning learning with and without the presence of auxiliary tasks, as well as several straightforward models for studying representation learning empirically. In particular, we neatly characterized the \textit{trajectory} taken by value functions and, in some cases, representations under TD learning dynamics. With this, we were able to thoroughly test a new hypothesis on the effectiveness of particular auxiliary tasks in sparse-reward environments, which led to improved understanding of representation learning in RL, as well as practical modifications to deep RL algorithms.

There are many natural follow-up directions to this work, some of which we will explore in the following chapters. One direction is to further develop the theory associated with the learning dynamics perspective, in order to (i) understand how additional types of auxiliary tasks, in particular auxiliary tasks that do not correspond to value functions, affect the representations in the learning models developed in this chapter, (ii) extend the learning models themselves to incorporate further aspects of large-scale learning scenarios, such as sample-based learning and state-visitation distribution corrections, and (iii) investigate other common learning dynamics, such as gradient TD methods \citep{sutton2008convergent} or policy gradient updates. In some cases, as with non-uniform state distributions, such extensions are straightforward. However, other less trivial extensions are crucial to improve the applicability of this model to practical settings. The current formulation of the learning dynamics model does not capture several important factors of practical deep RL training paradigms, leaving a gap between the exact, continuous-time representation updates of Theorem~\ref{thm:infinite-heads} and those of finite-width deep neural networks whose architectures endow the representation dynamics with particular inductive biases. 

The following chapters will bridge this gap to extract useful methodological implications of this framework. First, Chapter~\ref{chp:rep-learning} will capitalize on our observation that sparse-reward environments encourage feature collapse to study the mechanisms by which networks progress along the value improvement path. It will propose an auxiliary task which avoids the pitfalls of the RCDQN objective, which can hinder progress in some dense-reward environments, and to develop a more fine-grained view on the impact of non-stationarity on the plasticity of the learned representation. Next, Chapter~\ref{chp:gen-rl} will study the implications of Theorem~\ref{thm:infinite-heads} and Proposition~\ref{prop:many-value-functions} as they relate to generalization and interference. This discussion will take a complementary view of the presentation in this chapter by considering not the subspace spanned by the predicted value functions, but rather the properties of the approximation error incurred by the agent.


\chapter{Capacity loss}
\label{chp:rep-learning}
\minitoc

\section{Introduction}

The reinforcement learning problem presents a number of difficulties not present in supervised learning. Chief among these is the \textit{non-stationarity} of the learning objective. RL agents must solve a sequence of similar prediction problems as they iteratively improve their prediction accuracy and their policy \citep{dabney2020value}, and sufficient improvement in each subproblem in this sequence is necessary to progress to the next subproblem. Ideally a solution to one subproblem should \textit{generalize} to the subsequent one, in the sense that the learned representation should enable rapid adaptation. At minimum, fitting one prediction problem should not hinder performance on the next. This challenge was alluded to in Chapter~\ref{chp:rl-dynamics}; we now confront it explicitly, focusing particularly on its effect in RL agents using deep neural networks as function approximators. Prior works have shown that the early training period of a network is critical for its ability to achieve optimal performance on a task \citep{frankle2020the, achille2018critical}, and that early exposure to diverse data is crucial for optimal generalization performance \citep{ash2020warm, berariu2021study}. This suggests that the non-stationary target functions prevalent in reinforcement learning may be particularly ill-suited to function approximation by DNNs trained with gradient-based optimization, particularly in sparse-reward environments where the network receives no learning signal in this critical phase. Indeed, several prior works studying the effect of re-initializing network parameters in reinforcement learning have found that this enables agents to break through plateaus \citep{fedus2020catastrophic} and improve generalization performance \citep{igl2021transient}. This chapter will study one mechanism driving these phenomena, and lay the groundwork for a more thorough empirical analysis of the effect of prediction targets on generalization in Chapter~\ref{chp:gen-rl}.

\subsection{Motivating example: the zero function}


We follow Chapter~\ref{chp:rl-dynamics} in framing representation-learning in terms of a learned feature layer $\phi_{\btheta}(\bx)$ which is composed with some simple (usually linear) output function to predict state-action values. Previously, we restricted ourselves to settings where the representation's dynamics did not depend on the particular value of the linear map $\bw$, as the ensuing dynamics became intractable to analyze theoretically when $\Phi$ and $\bw$ were coupled. We now take an empirical approach to study what happens when the dynamics of $\Phi$ take more complex forms. To do so, we will focus on the location of a given set of parameters in the optimization landscape. This can provide some information not available in the outputs of a feature layer, as there are many different sets of parameters which correspond to the same function outputs, but which may occupy locations in the loss landscape with very different properties, making some easier to use as a starting point for learning than others. 

To provide a concrete illustration of this, consider a single hidden layer ReLU network. We let $\bx$, $\bx \in \mathbb{R}^d$ denote the observation, and $\btheta \in \mathbb{R}^{k \times d}$ network's first layer weights. We then write $\phi_{\btheta}(\bx) = \sigma(\btheta^\top \bx)$ where $\sigma(y) = \max(0, y)$ is applied entry-wise to the vector $\btheta ^\top \bx$. The output of the network is then $f_{\btheta}(\bx) = \langle \bw, \phi_{\btheta}(\bx) \rangle$. Suppose we have encountered the (unlikely) scenario that all of the parameters $\btheta$ are initialized to have negative value, the input distribution is only supported on non-negative vectors, and $\bw$ is initialized to be exactly zero. Then, recalling the framework of Chapter~\ref{chp:rl-dynamics}, any regression objective with target $\by$ will face the following dynamics.

\begin{align}
    \nabla_{\btheta} (\by - f_{\btheta}(\bx))^2 &= (\by - f_{\btheta}(\bx) )  \cdot ( \bw \cdot \mathbf{0}) = \mathbf{0} \\
    \nabla_{\bw} (\by - f_{\btheta}(\bx))^2 &= (\by - f_{\btheta}(\bx) )  \cdot  \phi_{\btheta}(\bx) = \mathbf{0} 
\end{align}

As a result, the gradient descent trajectory is locked in at the current parameters and will not be able to escape them. In contrast, if $\btheta^\top \bx$ contains non-negative elements, then even if $\bw$ is initialized to equal zero and the network still outputs the same function value, the gradients for $\bw$ will be non-zero and so the gradient descent trajectory will be able to escape this initialization. Thus, under one parameterization of the zero function the network is never able to update its predictions, while under the other it can do so with relative ease. Relating this back to the analysis of the previous chapter, this condition is equivalent to saying that we would like at least one of $\partial_t \btheta_t$ and $\partial_t \bw_t$ to be non-zero. This failure mode can be considered a special case of \textit{mutually frozen weights} \citep{zilly2021on}, a more general failure mode whereby saturated units prevent networks from making learning progress. 
 
\subsection{Contributions}
The principal hypothesis of this chapter is that over the course of training, deep RL agents lose some of their capacity to quickly fit new prediction tasks, and in extreme cases this capacity loss prevents the agent entirely from making learning progress.
In other words, capacity loss results in representations which \textit{fail to generalize} to the value functions induced by new policies and new TD targets in the value improvement path. 
We will show that the ability of deep RL agents to fit new target functions declines over the course of training in several environments from the Atari suite \citep{bellemare2013arcade} and on a non-stationary supervised prediction tasks. We take a deeper look at the representation collapse phenomenon, where the feature outputs for every state in the environment inhabit a low-dimensional -- or possibly even zero -- subspace, extending the empirical analysis of Chapter~\ref{chp:rl-dynamics}. Finally, we provide evidence that representation collapse is a key factor in agents' failure to make performance improvements in sparse-reward environments. We saw in the previous chapter that predicting \textit{random cumulants} can prevent feature collapse in sparse-reward environments; however, it can also lead to interference in dense-reward environments, harming performance.

To address this limitation we propose a simple regularization technique, Initial Feature Regularization (\pyoi), to prevent capacity loss in both dense- and sparse-reward environments. In this approach, we regress a set of feature projection heads to their values at initialization. While the regression targets used in \pyoi can be viewed as an auxiliary task, the method does not incorporate any additional environment information into the learning objective. This allows us to isolate the effect of capacity loss on agent performance without introducing confounding from the interaction between the auxiliary environment information and the agent's representation. We find that this regularization scheme mitigates capacity loss in sequential supervised prediction tasks, and that RL agents trained with \pyoi avoid egregious cases of representation collapse in sparse reward environments. We further show that \pyoi works by regularizing the network's learning dynamics, and conduct ablation studies to better understand this mechanism.

One striking take-away from these results is that agents trained on so-called `hard exploration' games such as Montezuma's Revenge can attain significant improvements over existing competitive baselines \textit{without} using smart exploration algorithms. This suggests that the poor performance of deep RL agents in sparse-reward environments is not \textit{solely} due to inadequate exploration, but rather also in part due to poor representation learning as the network `overfits' to predicting the zero function. There are thus two levers one can pull to improve performance: increasing the amount of reward signal the agent encounters via improved exploration, and ensuring that it is able to effectively update its predictions when it does receive reward signal. These two levers produce complex feedback loops in practical environments: more accurate predictions produce more informative behaviours, and more informative behaviours provide information needed to improve the network's predictions. The sparse reward settings studied in this chapter highlight the necessity of network plasticity in enabling this virtuous cycle. Chapter~\ref{chp:gen-rl} will study a more nuanced form of overfitting to early targets that arises even in dense-reward environments, extending the insights from this chapter to a broader range of settings.

\section{Learning capacity in neural networks}

\label{sec:learning-capacity}
Each time an RL agent discovers a new source of reward in its environment or improves its ability to obtain this reward, the value function that it seeks to predict changes. Over the course of learning, a value-based RL agent attempts to solve a long sequence of target prediction problems, though in the case of value iteration-style algorithms such agents many only partially solve each problem before the next is constructed. Studies of neural networks in supervised learning suggest that this sequential fitting of new targets may be harmful to a neural network's ability to adapt to new objectives \citep{ash2020warm}. This presents a concern for deep RL, where networks are trained to fit a constantly-changing target and may need to quickly make significant changes to their predictions even late in the training process. In this section, we show that training on a sequence of prediction targets can indeed lead to a reduced ability to fit new objectives in deep neural networks, a phenomenon that we term \textit{capacity loss}, and confirm that capacity loss can arise in value-based deep RL agents. Further, we provide evidence that an agent's inability to quickly update its value function to distinguish states presents a barrier to performance improvement in deep RL agents trained on environments from the Atari suite.

\subsection{Target-fitting capacity}
The parameters of a neural network determine not just the network's current outputs, but also how these outputs will evolve over time via the magnitude and structure of its gradients and the curvature of the loss landscape. This evolution determines the \textit{capacity} of the network to learn to predict new targets by following an optimization trajectory. We therefore view the agent's representation in terms of the optimization dynamics that it induces. In particular, we are interested in identifying when an agent's current parameters are flexible enough to allow it to perform gradient updates that meaningfully change its predictions based on new reward information in the environment or evolving bootstrap targets, a notion formalized in the following definition.

\begin{definition}[Target-fitting capacity]\label{def:tf_cap}
Let $P_X \in \mathscr{P}(X)$ be some distribution over inputs $X$ and $P_\mathcal{F}$ a distribution over a family of functions $\mathcal{F}$ with domain $X$. Let $\mathcal{N} = (g_\theta, \mathbf{\theta}_0)$ represent the pairing of a neural network architecture with some initial parameters $\mathbf{\theta}_0$, and $\mathcal{O}$ correspond to an optimization algorithm. We measure the \textit{capacity} of $\mathcal{N}$ under the optimizer $\mathcal{O}$ to fit the data-generating distribution $\mathcal{D}=(P_X, P_\mathcal{F})$ as follows:

\begin{equation}
    \mathcal{C}(\mathcal{N}, \mathcal{O}, \mathcal{D}) = \mathbb{E}_{f \sim P_{\mathcal{F}}}[ \mathbb{E}_{x \sim P_X}[   (g_{\mathbf{\theta}'}(x) - f(x))^2 ]] \quad \;\text{where} \; \mathbf{\theta}' = \mathcal{O}(\theta_0, P_X, f) \, .
\end{equation}

\end{definition}

\begin{figure}
\centering
        \includegraphics[width=0.75\textwidth,keepaspectratio]{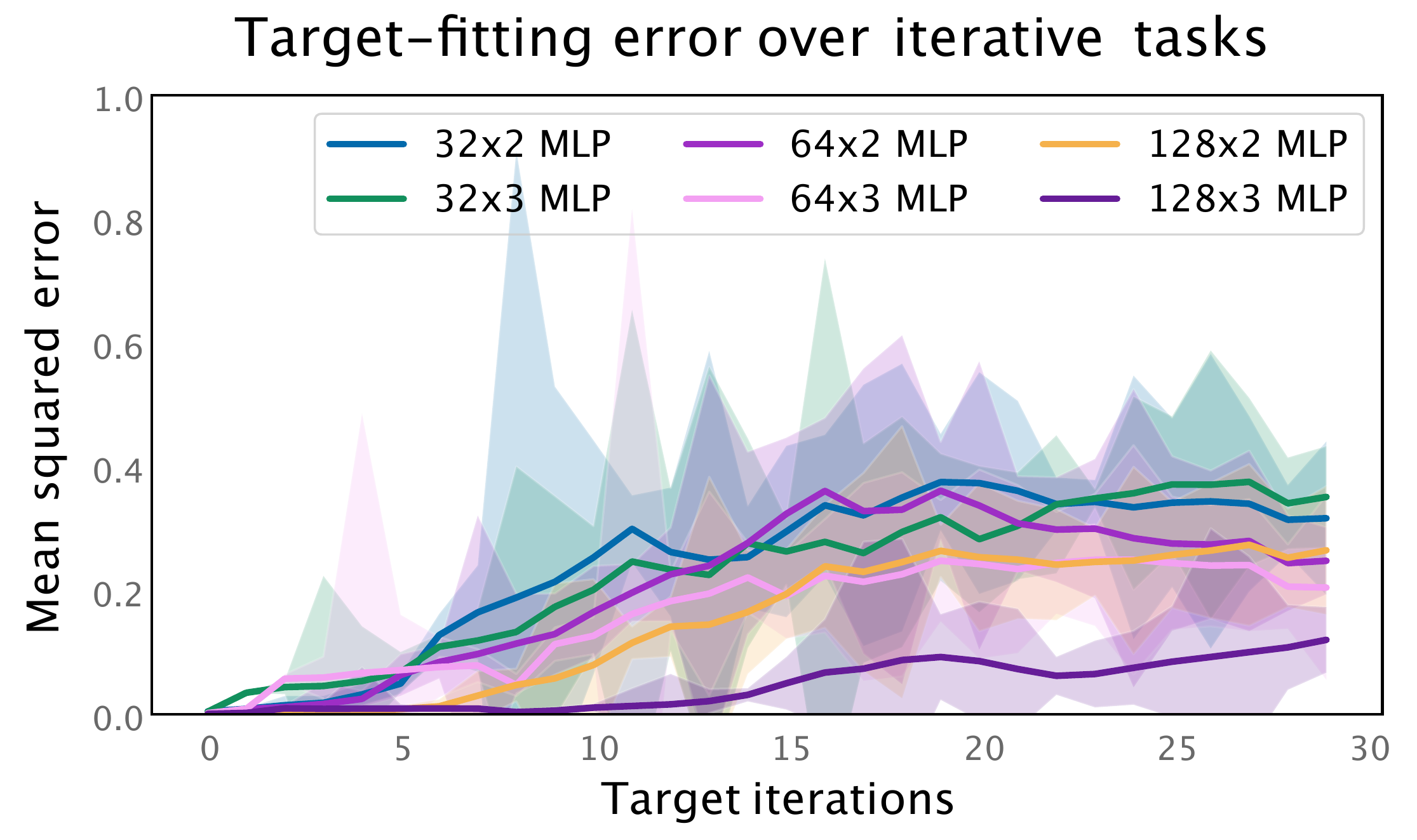}
         \caption{Deeper analysis on a sequential MNIST prediction setting showing that target-fitting capacity sees the greatest reduction in low capacity networks with ReLU units.}
        \label{fig:mnist-cap}
\end{figure}

This definition of capacity measures the ability of a network to reach a new set of targets within a limited optimization budget from its current parameters and optimizer state. The choice of optimization budget and target distribution are left as free variables, and different choices result in different notions of capacity. In reinforcement learning we ultimately care about the network's ability to fit its Bellman targets quickly, particularly to the extent that it allows the greedy policy to obtain high reward. However, only evaluating a network's ability to fit its Bellman targets within a small budget will not necessarily be a useful measure of capacity: for example, a network which can only output the zero function will attain low Bellman error immediately on a sparse-reward environment, but will fail to produce useful updates to improve its policy. Our evaluations of this measure will assign longer training budgets and use target functions that are independent of the current network parameters to avoid these pathologies; the effect of this choice is explored further in Appendix~\ref{appx:cap-loss-supervised}. 

The process of training a neural network to fit a set of labels must by necessity change some properties of the network. Works studying the information bottleneck principle \citep{tishby2015deep} for example, identify a compression effect of training on the latent representation, where inputs with similar labels are mapped to similar feature vectors. This compression can benefit generalization on the current task, but may make learning more difficult if in the future the network must distinguish between these inputs. In the face of the rapidly-changing nature of the targets used in value iteration algorithms, this compression has the potential to harm the learning process by impeding the network's ability to fit new targets rather than to help it. This motivates the first hypothesis of this section.

{\hypothesis{H1}{networks trained to iteratively fit a sequence of dissimilar targets will lose their capacity to fit new target functions within a fixed optimization budget compared to their capacity at initialization.}\label{hyp:capacity-loss}}

To evaluate {Hypothesis~\ref{hyp:capacity-loss}}, we construct a series of iterative prediction problems on the MNIST data set, a widely-used computer vision benchmark which consists of images of handwritten digits and corresponding labels. 
We first fit a series of labels computed by a randomly initialized target neural network $f_\theta$: we transform input-label pairs $(x,y)$ from the canonical MNIST dataset to $(x, f_\theta(x))$, where $f_\theta(x)$ is the network output. To generate a new task, we simply reinitialize the target network. Given a target function, we then train a student neural network for a fixed budget from the parameters obtained at the end of the previous iteration (using a random initialization for the first iteration), and repeat this procedure of target initialization and training thirty times. We use a subset of MNIST inputs of size 1000 to reduce computational cost. In these experiments we focus only on the ability of the network to fit its training data; a study of the effect of prediction targets on generalization is deferred to Chapter~\ref{chp:gen-rl}. We provide additional details in Appendix~\ref{appx:mnist-details}.

In Figure~\ref{fig:mnist-cap} we see that the networks trained on this task exhibit decreasing ability to fit later target functions under a fixed optimization budget. This effect is strongest in the smaller networks, matching the intuition that solving tasks which are more challenging for the network will result in greater capacity loss. We consider two other tasks in Appendix~\ref{appx:cap-loss-supervised}, as well as a wider range of architectures, obtaining similar results. 
We find in this broader set of evaluations that sufficiently over-parameterized networks (on the order of one million parameters for a task with one thousand data points) exhibit positive forward transfer; however, models which are not over-parameterized relative to the task difficulty consistently exhibit increasing error as the number of targets trained on grows. 

This raises a question: are the deep neural networks used by value-based RL agents on popular benchmarks in the over- or under-parameterized regime? 
The examples here differ in an important way from the value iteration methods common in deep RL: we run our optimizer on each task long enough for the loss to converge before turning to the next task. In RL problems, the network will not have converged to the value of the current policy before the changes to the value function induce a policy improvement step. We therefore turn our attention to reinforcement learning, to evaluate whether the phenomena observed in the idealized supervised learning setting also hold in value-based RL agents.

\hypothesis{H2}{the non-stationary prediction problems in value-based deep RL also result in capacity loss.}

\begin{figure}
\centering
        \includegraphics[width=0.849\linewidth]{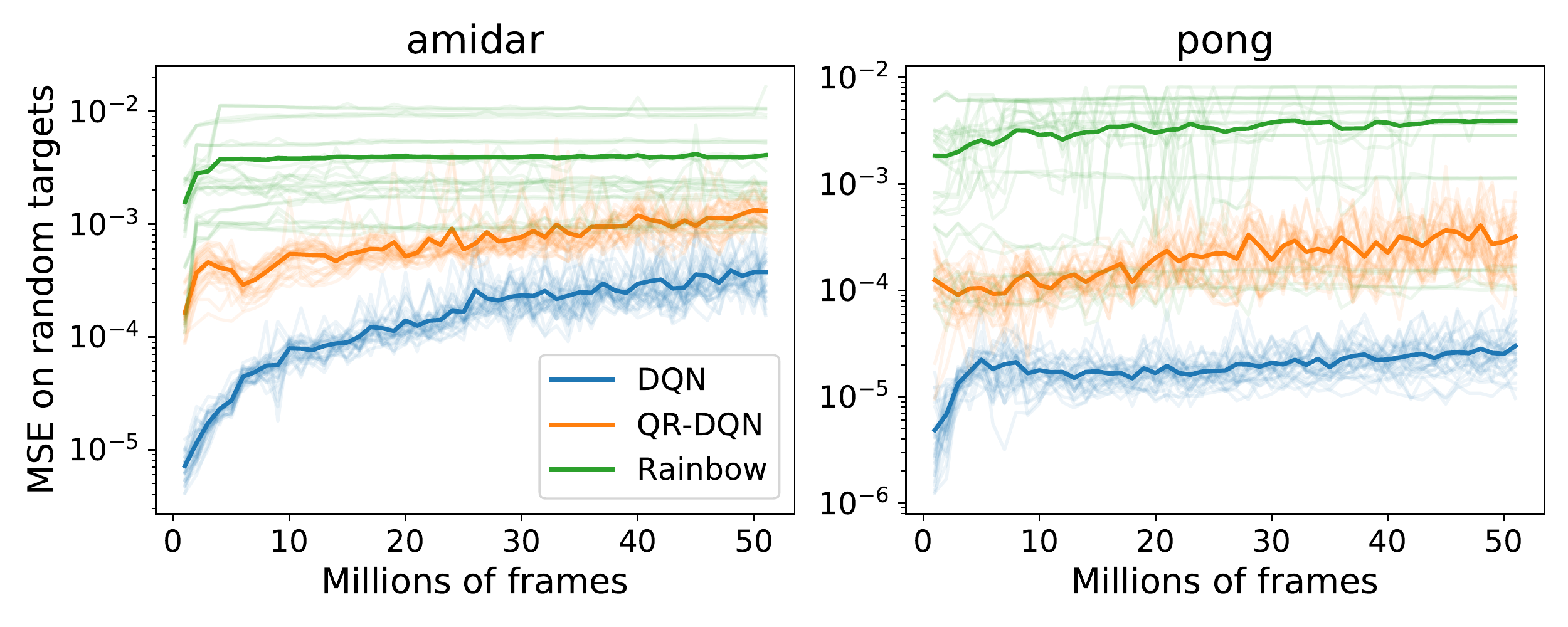}
       \caption{Neural networks exhibit a decline in their ability to fit random network outputs over the course of training in two demonstrative Atari environments.}
         \label{fig:atari-cap}
   
\end{figure}
To evaluate {Hypothesis~\ref{hyp:subspace}}, we train network checkpoints stored over the course of training to fit randomly generated target functions. We provide full details of this procedure in Appendix~\ref{appx:atari}. We generate target functions by randomly initializing a neural network with an identical architecture to the agent's, and use the outputs of this network as targets for regression.
We then load initial parameters from an agent checkpoint at some time $t$, sample inputs from the replay buffer, and regress on the random target function evaluated on these inputs. We then evaluate the mean squared error after training for fifty thousand steps. We consider a DQN \citep{mnih2015human}, a QR-DQN \citep{dabney2018distributional}, and a Rainbow agent \citep{hessel2018rainbow}. 
We observe in all three cases that as training progresses agents' checkpoints obtain increasing error after training to fit randomly generated targets; due to space limitations we only show two representative environments where this phenomenon occurs in Figure~\ref{fig:atari-cap}, and defer the full evaluation to Appendix~\ref{appx:tf-capacity-atari}. 

\keyinsight{When neural networks are trained to predict a sequence of challenging target functions, they get progressively worse at fitting these targets. At least some RL settings exhibit this property, resulting in \textit{capacity loss} over the course of training.} 

\subsection{Representation collapse and performance}
The notion of capacity in Definition~\ref{def:tf_cap} measures the ability of a network to \textit{eventually} represent a given target function. This definition reflects a number of intuitions about capacity: that networks which are good at fitting only a narrow range of targets should have lower capacity than more adaptable ones, and that capacity should decrease over the course of training as the network becomes more specialized. However, this quantity is computationally expensive to compute and is dependent on the choice of target function class.

We now introduce an alternative measure of capacity which captures a network's ability to quickly adapt to changes in the target function, while being significantly cheaper to compute than the more precise measure introduced in Definition~\ref{def:tf_cap}. 
We call this notion of capacity the \effdim, as it corresponds to an approximation of the rank of a feature embedding. Intuitively, the \effdim measures how easily states can be distinguished by updating only the final layer of the network.
This notion of state similarity is particularly relevant to sparse-reward environments, where policy improvement depends on the agent's ability to distinguish a handful of rewarding states from the vast non-rewarding majority. It also bears close resemblance to the implicit underparameterization phenomenon \citep{kumar2021implicit} studied previously, allowing us to evaluate how well these prior notions of capacity capture a network's ability to change its predictions and improve its performance.

\begin{definition}[\Effdim]
Let $\phi: \mathcal{X} \rightarrow \mathbb{R}^d$ be a feature mapping. Let $\mathbf{X}_n \in \mathcal{X}^n$ be a set of $n$ states in $\mathcal{X}$ sampled from some fixed distribution $P$. Fix $\varepsilon \ge 0$, and let $\phi(\mathbf{X}_n) \in \mathbb{R}^{n \times d}$ denote the matrix whose rows are the feature embeddings of states $x \in \mathbf{X}_n$. Let $\textnormal{SVD}(M)$ denote the multiset of singular values of a matrix $M$. Then the $\effdim$ of $\phi$ given input distribution $P$ is defined to be 
\begin{equation}\label{eq:eff-dim}
\rho(\phi, P, \epsilon ) = \lim_{n \rightarrow \infty} \mathbb{E}_{\mathbf{X}_n \sim P}[|\{\sigma \in \textnormal{SVD} \bigg (\frac{1}{\sqrt{n}}\phi(\mathbf{X}_n) \bigg ) | \sigma > \varepsilon \} | ] \, 
\end{equation}
for which a consistent estimator can be constructed as follows
\begin{equation}\label{eq:eff-dim-samples}
\hat{\rho}_n(\phi, \mathbf{X}, \epsilon) = |\{\sigma \in \textnormal{SVD} \bigg (\frac{1}{\sqrt{n}}\phi(\mathbf{X}) \bigg ) | \sigma > \varepsilon \} |  \, .
\end{equation}
\end{definition}
The $\effdim$ is equal to the dimension of the subspace spanned by $\phi(\mathcal{X})= \{\phi(x) \mid x \in \mathcal{X}\}$ when $\varepsilon=0$ and the state space $\cX$ is finite. For $\epsilon > 0$, it throws away small components of the feature matrix. We show that $\rho$ is well-defined and that $\hat{\rho}_n$ is a consistent estimator in Appendix~\ref{appx:consistency}. 
Our analysis of the \effdim resembles that of \citet{kumar2021implicit}, but differs in two important ways: first, our estimator does not normalize by the maximal singular value. This allows us to more cleanly capture \textit{representation collapse}, where the network features, and thus also their singular values, converge to zero. Second, we are interested in the capacity of agents with unlimited opportunity to interact with the environment, rather than in the data-limited regime. We compare our findings on feature rank against the \textit{srank} used in prior work in Appendix~\ref{appx:cap-loss-supervised}. 
\begin{figure}
    \centering
    \includegraphics[width=.995\textwidth,keepaspectratio]{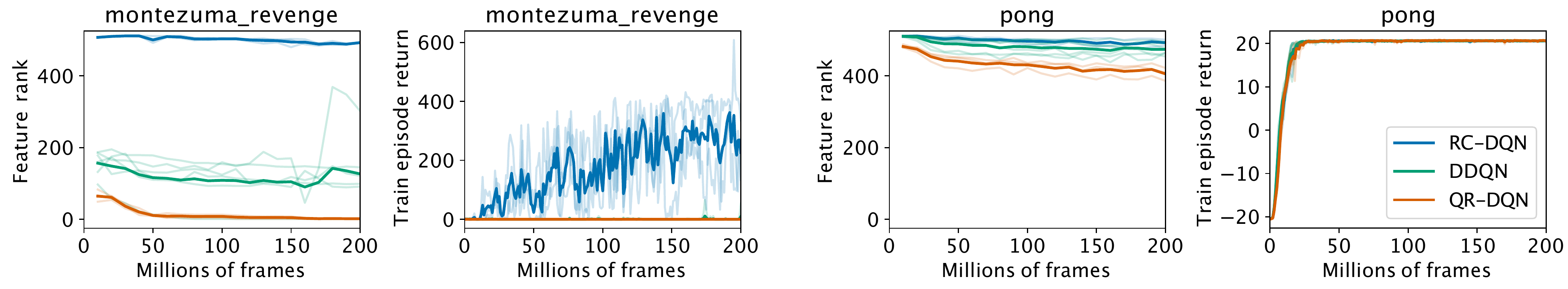}
    \caption[\Effdim and performance over the course of training.]{\Effdim and performance over the course of training. We observe that \effdim is higher for environments and auxiliary tasks which provide denser reward signals than for sparse reward problems. 
   }
   \label{fig:effdim_vanilla}
\end{figure}

We proceed to empirically evaluate the feature rank of various deep RL agents trained on games from the Atari suite to study its evolution over the course of training. 
We train a double DQN (DDQN) agent, a quantile regression (QRDQN) agent, and a double DQN agent with an auxiliary random cumulant prediction task (RC-DQN) \citep{dabney2020value}, on environments from the Atari suite, then evaluate $\hat{\rho}_n$ with $n=5000$ on agent checkpoints obtained during training. We revisit two of the environments studied in Chapter~\ref{chp:rl-dynamics}: Montezuma's Revenge (sparse reward), and Pong (dense reward), deferring two additional environments, a sparsified version of Pong in which the agent does not receive negative reward when its opponent scores, and Seaquest (dense reward, but more challenging than Pong), to Appendix~\ref{appx:feature-rank-atari}. We run 3 random seeds on each environment-agent combination.

We visualize agents' \effdim and performance in Figure~\ref{fig:effdim_vanilla}. Our findings confirm that the RC-DQN objective prevents representation collapse, as previously conjectured. More generally, non-trivial prediction tasks, either value prediction in the presence of environment rewards or auxiliary tasks, lead to higher \effdim. Unlike in the case of target-fitting capacity, \effdim does not decline monotonically, indicating that it measures a subtly different notion of capacity than simply an agent's ability to fit new target functions. This subtlety will be explored further in Chapter~\ref{chp:gen-rl}, where it will be interpreted as incorporating a particular inductive bias into the network structure. In Montezuma's Revenge, the higher \effdim  induced by RC-DQN corresponds to higher performance, but as we saw in Chapter~\ref{chp:rl-dynamics}, this auxiliary loss can have a detrimental effect on learning progress in complex, dense-reward games such as Seaquest, presumably due to interference between the random rewards and the true learning objective. Unlike in target-fitting capacity, we only see a consistent downward trend in sparse-reward environments, where a number of agents, most dramatically QRDQN, exhibit representation collapse. We discuss potential mechanisms behind this trend in Appendix~\ref{appx:feature_theory}. We conjecture that in practice, the \effdim is a better measure of the inductive bias of the network than it is a measure of capacity, but that it can accurately and cheaply detect extreme instances of feature collapse which may impede learning.

\begin{figure}
    \centering
    \includegraphics[width=0.95\linewidth]{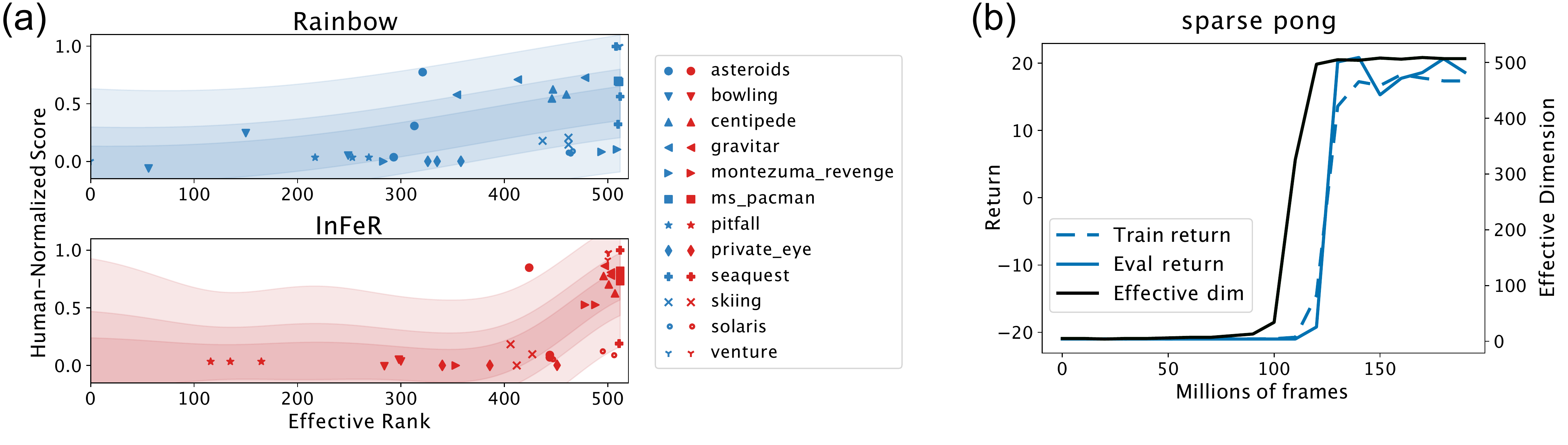}
    \caption[Evaluation of the correlation between \effdim and performance.]{\textbf{(a)}: Agent \effdim vs human-normalized score in games where Rainbow does not achieve superhuman performance. While $\effdim$ does not appear to solely determine agent performance, there is a positive correlation between $\effdim$ and human-normalized score. \textbf{(b)} An `unlucky' seed from our evaluations on the sparsified version of Pong, where learning progress occurs only after the agent recovers from representation collapse.
    }
    \label{fig:capacity-update}
\end{figure}

While the many moving parts in deep RL algorithms make it difficult to isolate the causal effect of a single representation property on performance, Figure~\ref{fig:capacity-update}a reveals a correlation between learning progress and \effdim on challenging games where agents fail to achieve human-level performance. We see this correlation in both a Rainbow \citep{hessel2018rainbow} agent, and an agent trained with the regularizer we will introduce in the coming section whose precise form is not relevant to this discussion. Further, in all of our evaluations we find that agents whose representations have collapsed do not make learning progress. This is best exemplified by the learning curves shown in Figure~\ref{fig:capacity-update}b, which highlights a particularly unlucky random seed that did not observe a point being scored during its initial random exploration period and experienced representation collapse. Eventually, after several million training frames, its $\effdim$ increases dramatically, and shortly \textit{after} this occurs the agent solves the task. We conclude that in its extreme form, representation collapse appears to completely prevent learning progress, but that the relationship between learning progress and \effdim in less extreme cases is complex, as capacity is one of many factors influencing performance in RL. In other words: capacity is a \textit{necessary} but not \textit{sufficient} condition for agents to make progress.

\keyinsight{Sparse-reward environments induce \textit{feature collapse}, whereby the representation converges to a low-dimensional subspace and `overfits' to the zero function. While it is possible to recover from feature collapse, agents fail to take successful policy improvement steps until this recovery occurs.}

\section{\pyoi: mitigating capacity loss with feature regularization}
\label{sec:pyoi}

The previous section showed that capacity loss can arise in online deep RL algorithms, and in some cases appears to be a bottleneck to performance. We now consider how it might be mitigated, and whether explicitly regularizing the network to preserve its initial capacity improves performance in environments where representation collapse occurs. Our approach involves a function-space perspective on regularization, encouraging networks to preserve their ability to output linear functions of their initial features. It further yields insight into the role of capacity loss in deep RL by isolating the effect of capacity on agent performance independent of the agent's exploration policy or auxiliary learning signals.

\subsection{Feature-space regularization}
Much like parameter regularization schemes seek to keep {parameters} close to their initial values, we wish to keep a network's ability to fit new targets close to its initial value. We motivate our approach with the intuition that a network which has preserved the ability to output functions it could represent at initialization should be better able to adapt to new targets. To this end, we will regress a set of network {outputs} towards the values they took at initialization. Our method, Initial Feature Regularization (\pyoi), applies an $\ell_2$ regularization penalty on the output-space level by regressing a set of auxiliary network output heads to match their values at initialization. Similar perspectives have been used to prevent catastrophic forgetting in continual learning \citep{benjamin2018measuring}, though in our case we care about the functional form of the network outputs, and not preserving the outputs associated with specific past inputs.

In our approach, illustrated in Figure~\ref{fig:effdim_pyoi}, we begin with a fixed deep Q-network with parameters $\theta$, and modify the network architecture by adding $k$ auxiliary linear prediction heads $g_i$ on top of the feature representation $\phi_\theta$. We take a snapshot of the agent's parameters at initialization $\theta_0$, and use the outputs of the $k$ auxiliary heads under these parameters as auxiliary prediction targets. We then compute the mean squared error between the outputs of the heads under the current parameters  $g_i(x; \theta_t)$ and their outputs at initialization $g_i(x; \theta_0)$. This approach has the interpretation of preserving subspaces of the features that were present at initialization. Because a randomly initialized network tends to produce low-magnitude outputs, a scaling factor $\beta$ ensures that the regularizer target value is on a comparable scale to that of the TD targets. This results in the following form of our regularization objective, where we let $\mathcal{B}$ denote the replay buffer sampling scheme used by the agent: 
\begin{equation}
    \mathcal{L}_{\pyoi}(\theta, \theta_0; \mathcal{B}, \beta) = \mathbb{E}_{x \sim \mathcal{B}} \bigg [\sum_{i=1}^k (g_i(x; \theta) - \beta g_i(x; \theta_0) )^2 \bigg ] \; .
\end{equation}


\begin{figure}[t]
    \centering
    \includegraphics[width=0.98\linewidth]{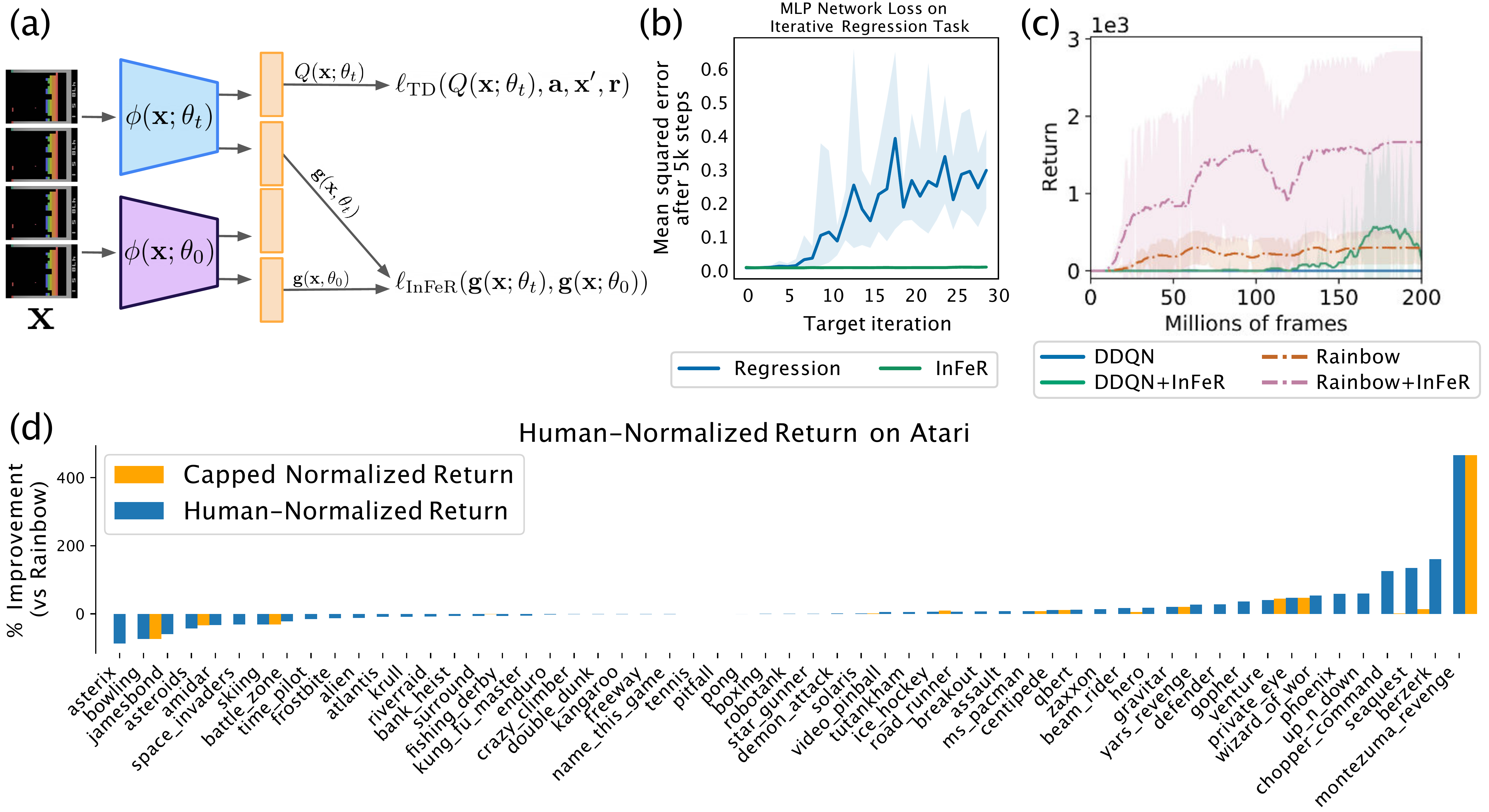}
    \caption[Visualization and evaluation of InFeR.]{\textbf{(a)} Visualization of \pyoi. \textbf{(b)}  Analysis of the effect of InFeR on capacity loss. \textbf{(c)} Effect of InFeR on performance in Montezuma's Revenge with respect to Rainbow and Double DQN baselines. \textbf{(d)} Performance of \pyoi relative to Rainbow on all 57 atari games.}
    \label{fig:effdim_pyoi}
\end{figure}

The loss can then be additively combined with any RL or supervised loss function. 

We first investigate whether this regularization scheme does indeed preserve capacity. To do so, we replicate the non-stationary MNIST prediction task studied in the previous section, but now incorporate \pyoi. We observe in Figure~\ref{fig:effdim_pyoi} that the \pyoi objective almost entirely eliminates capacity loss in a network whose performance would otherwise degrade to that of random guessing; while it is not a silver bullet, we show in Appendix~\ref{appx:infer-mnist} that this effect consistently occurs across a range of architectures and target classes. We show in Appendix~\ref{appx:feature-rank-atari} that similar phenomena occur in RL: \pyoi tends to increase the \effdim of agents trained on the Atari domain over the entire course of training; we study the early training period in Appendix~\ref{appx:tf-capacity-atari}. Our findings in both settings suggest that this form of regularization may prove fruitful in a variety of continual and reinforcement learning contexts beyond the ALE benchmark studied in the remainder of this chapter.

Our analysis of RL agents trained with \pyoi is made particularly insightful by the fact that \pyoi does not incorporate any additional information from the environment or induce any sophisticated exploration behaviour beyond that deployed by the agent it is applied to. As a result, it allows us to isolate the effect of capacity loss on performance. If adding the \pyoi objective improves performance, it can only be due to its effect on the agent's representation, or some knock-on effect thereof. We evaluate the effect of incorporating this loss in both DDQN \citep{van2016deep} and Rainbow \citep{hessel2018rainbow} agents, and include the relative performance improvement obtained by the \pyoi agents over Rainbow on 57 games from the Atari 2600 suite in Figure~\ref{fig:effdim_pyoi}, deferring the comparison to DDQN, where the regularizer improved performance slightly on average but only yielded significant improvements on sparse-reward games, to Appendix~\ref{appx:atari}. We observe a net improvement over the Rainbow baseline by incorporating the \pyoi objective, with significant improvements in games where agents struggle to obtain human performance. The evaluations in Figure~\ref{fig:effdim_pyoi} are for $k=10$ heads with $\beta=100$ and $\alpha=0.1$, and we show the method's robustness to these hyperparameters in Appendix~\ref{appx:hypers}.

The striking improvement obtained in the sparse-reward Montezuma's Revenge environment begs the question of whether such results can be replicated in other RL agents.
We follow the same experimental procedure as before, but now use the DDQN agent; see Figure~\ref{fig:effdim_pyoi}. We find that adding \infer to the DDQN objective produces a similar improvement as does adding it to Rainbow, leading the DDQN agent,
which only follows an extremely naive $\epsilon$-greedy exploration strategy and obtains zero reward at all points in
training, to exceed the performance of the noisy networks approach taken by Rainbow in the last 40 million training frames. 
This leads to two intriguing conclusions: first, that agents which are explicitly regularized to prevent representation collapse \textit{can} make progress in sparse reward problems without the help of good exploration strategies; and second, that this form of regularization yields significantly larger performance improvements in the presence of additional algorithm design choices that are designed to speed up learning progress.

\subsection{Understanding how \pyoi works}

Having observed that our regularizer both improves performance and mitigates capacity loss, we now take a closer look into the mechanisms by which it may enable learning progress. While \pyoi improves performance \textit{on average} across the Atari games, it does not do so uniformly: its improvements are concentrated principally on games where then Rainbow agent performs significantly below the human baseline. It further clearly slows down progress in a subset of environments such as Asteroids and Jamesbond. Without a deeper analysis, it is not immediately obvious what makes games like Jamesbond and Asterix, where \pyoi reduces performance, different from Seaquest and Berzerk. We now take a closer look at the mechanisms by which \pyoi is shaping the agent's representation in the hopes of explaining this differential effect. We consider two hypotheses.

\hypothesis{1}{\pyoi improves performance by preserving a random subspace of the representation that the final linear layer can use to better predict the value function. The effect of the regularizer on other aspects of the representation learning dynamics does not influence performance.\label{hyp:regularizer} }

Hypothesis~\ref{hyp:regularizer} implies that having access to a random feature vector can improve an agent's performance even in the absence of any regularization over deeper network layers. It assumes that having a wider variety of basis functions for the final linear layer to approximate the value function will benefit the network's adaptability. To evaluate Hypothesis~\ref{hyp:regularizer}, we concatenate a single dimension of the output of a randomly initialized network to the feature outputs of the network used to learn the Q-function, and train a linear layer on top of these joint learned and random features. We compare its performance with that of an identical network architecture with a single \pyoi auxiliary head. If Hypothesis~\ref{hyp:regularizer} were true, then we would expect this architecture to perform comparably to the \pyoi agent, as the final linear layer has access to a randomly initialized feature subspace of equal dimension to that regularized by \pyoi. If not, then we would expect the performance of the agents with access to the random features to be comparable to that of the vanilla Rainbow agents. Figure~\ref{fig:double} shows that the latter occurs, confirming that the effect of \pyoi on earlier layers is crucial to its success. 

But what exactly is the effect of \pyoi on earlier layers? In order to preserve its initial outputs, a network is limited in the extent to which it can modify the outputs of intermediate layers. We conjecture that \pyoi limits the degrees of freedom with which a network can overfit its representation in response to early target-fitting objectives, which may also reduce the flexibility of the network to make the changes necessary to fit the current value function. In cases where representation collapse is not a concern, this will have the effect of slowing down progress. In such cases, increasing the dimension of the layer to which we apply \pyoi should give the network more degrees of freedom to fit its targets, and so reduce the performance gap induced by the regularization.
\begin{figure}
        \includegraphics[height=.14\textheight,keepaspectratio]{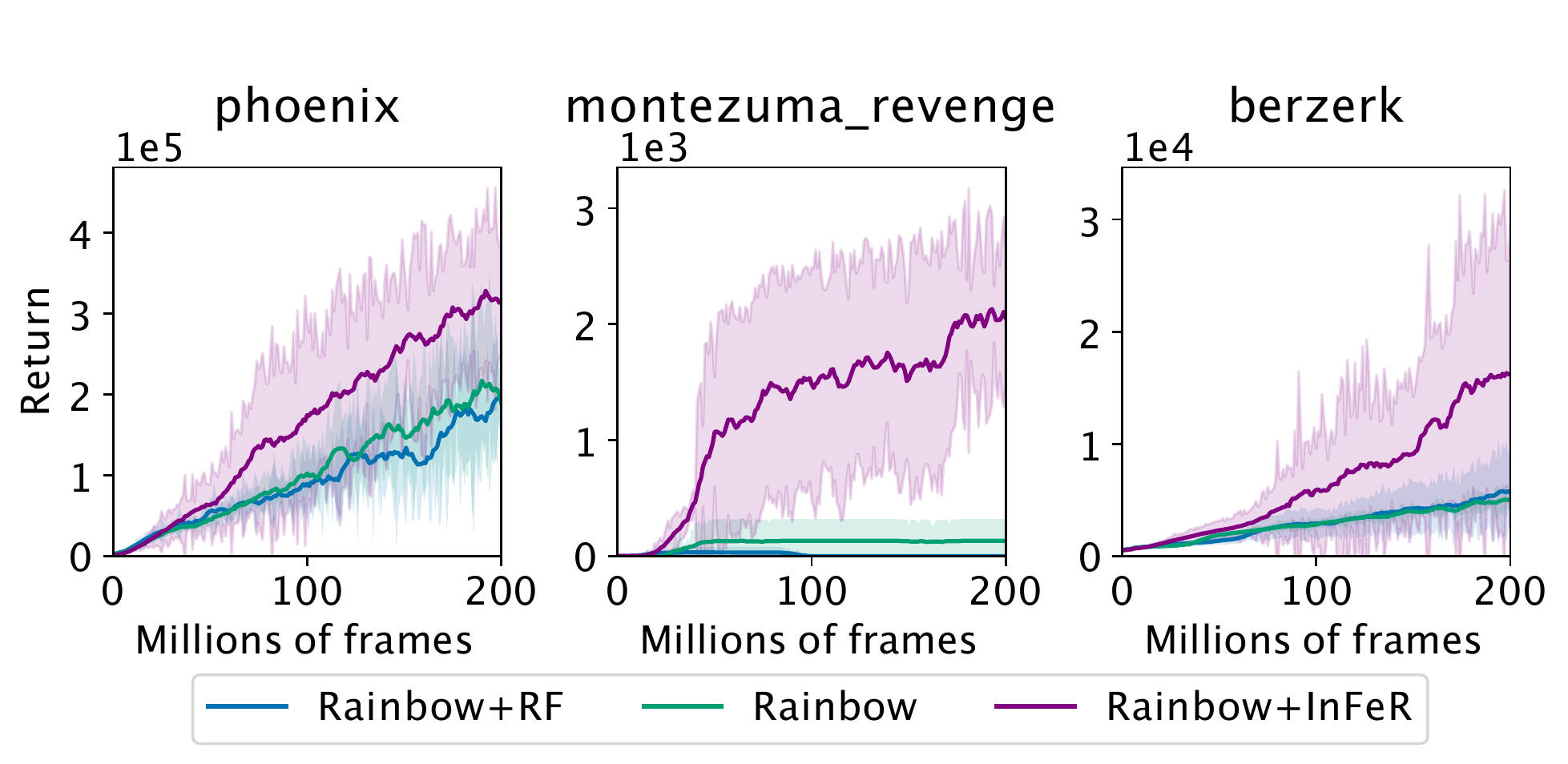}
        \includegraphics[height=.14\textheight,keepaspectratio]{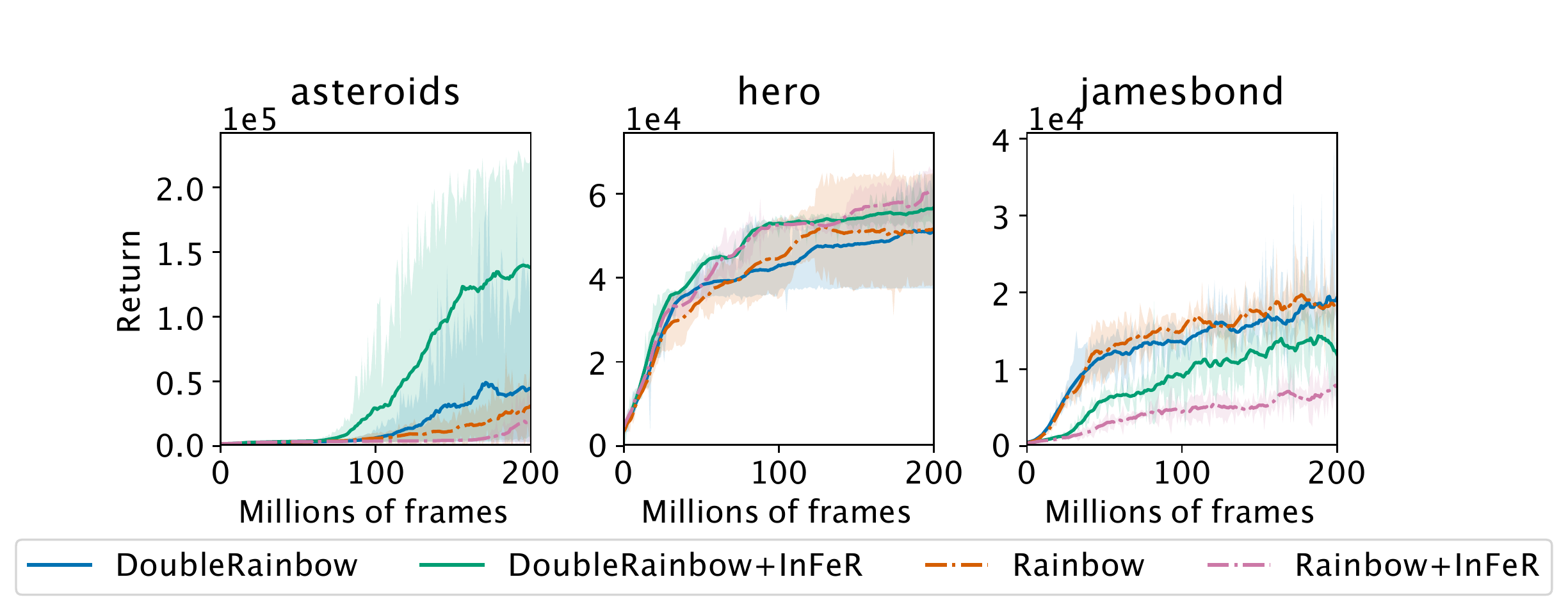}
        \caption[Deeper investigation into the effect of \pyoi.]{Left: agent performance does not improve over baseline when random features are added to the representation. Right: doubling the width of the neural network narrows the performance gap in games on which \pyoi under-performed relative to Rainbow.}
        \label{fig:double}
\end{figure}

\hypothesis{2}{\pyoi slows down the rate at which the learned features at every layer of the network can drift from their initialization in function space, regularizing the learning dynamics of the entire network. The precise subspace spanned by the auxiliary weights is not directly useful.\label{hyp:subspace}}

We test this hypothesis by doubling the width of the penultimate network layer and comparing the performance of \pyoi and Rainbow on games where \pyoi hurt performance in the original network. We see in Figure~\ref{fig:double} that increasing the network's size reduces, eliminates, or in some cases reverses the performance gap induced by \pyoi in the smaller architecture. We therefore conclude that the principal mechanism by which \pyoi affects performance is by regularizing the entire network's optimization dynamics. This finding has intriguing implications on the optimal network size in reinforcement learning problems, suggesting that given suitable regularization, performance in some environments may be improved simply by increasing the network size.

\section{Conclusions}

This chapter has identified a fundamental challenge facing deep RL agents: loss of the capacity to distinguish states and represent new target functions over the course of training. We have shown that this phenomenon is particularly salient in sparse-reward settings, in some cases leading to complete collapse of the representation and preventing the agent from ever making learning progress. To address this, we proposed a novel regularizer, \pyoi, with the goal of preserving capacity, yielding improved performance across a number of tasks in which deep RL agents have historically struggled to match human performance. Further investigation into this method suggests that it is performing a form of function-space regularization on the neural network, and that settings where it appears to reduce performance are actually instances of under-parameterization relative to the difficulty of the task. Particularly notable is the effect of incorporating \infer in the hard exploration game of Montezuma's Revenge: its success here suggests that effective representation learning can allow agents to learn good policies in sparse-reward environments even under naive exploration strategies. 

Our findings open up a number of exciting avenues for future work in reinforcement learning and beyond to better understand how to preserve plasticity in non-stationary prediction tasks. For example, a deeper study into the mechanisms by which capacity is lost similar to that of \citet{zilly2021on} could provide further practically useful insights for the design of optimization algorithms and regularizers in deep learning. Of particular interest is the development of methods which preserve network plasticity. Since the submission of the paper on which this chapter is based, exciting related work has emerged proposing optimization \citep{dohare2021continual} and resetting \citep{nikishin2022primacy, zaidi2022does} methods to improve network plasticity. 

One mysterious finding that we have not yet fully explained is the nuanced relationship between \effdim, capacity, and generalization. We observed that \effdim and target-fitting capacity are often correlated, but that this is not always the case: in some cases, an increase in the \effdim of a network was accompanied by a decreased ability to fit `structured' families of target functions, such as those given by the outputs of randomly initialized neural networks. Recent work has confirmed similarly nuanced relationship between \effdim and performance in offline RL \citep{gulcehre2022empirical}. This observation suggests that the notion of capacity loss we discussed previously is multi-faceted: over the course of training, networks can develop inductive biases which make them better suited to fitting certain function classes than others. But why should the neural network, which is after all trained via bootstrapping on its own outputs, see a \textit{decreased} ability to fit the outputs of other initializations of the same architecture as it improves its ability to linearly disentangle states? Chapter~\ref{chp:gen-rl} will answer this question, bringing to the fore the question of how the learning dynamics discussed thus far come to bear on generalization in deep RL.

\chapter{Interference and generalization}

\label{chp:gen-rl}
\minitoc
\section{Introduction}
Capacity loss can be viewed as an extreme form of overfitting, where a neural network reduces not just its current performance on a set of related targets to the current training task, but also its ability to fit these targets even after many optimization steps. This chapter studies a weaker notion of overfitting, whereby training a neural network on one set of targets results in an inductive bias that is ill-suited to the value function that the network will eventually be tasked with representing. Key to this notion of overfitting is the concept of \textit{interference}, the degree to which an update to a function approximator's output given one input influences its predictions for other input observations. We saw a form of interference presented in Chapter~\ref{chp:supervised}, where we discussed the degree to which a gradient update on one minibatch reduced the loss on other minibatches in a training dataset. Here we will take a more generic view and focus on the magnitude of the change in the predicted value of one state as the result of a semi-gradient update on another. Interference can be viewed as an inductive bias encoding the {similarity} of the value function at different inputs. Function approximation schemes with weaker interference, such as those induced by tabular value functions or tile coding schemes, have been shown empirically to produce more stable behaviour and faster convergence in value-based algorithms on a number of classic control domains \citep{ghassian2020improving}. However, such schemes by construction require treating the value functions for different states independently, limiting the potential for a model’s predictions to generalize to new observations and resulting in \textit{memorization}.  

An array of prior empirical works demonstrate that value-based RL algorithms frequently overfit to their training environment's observations and dynamics \citep{lewandowskigeneralization,farebrother2018generalization,cobbe2021phasic, zhang2018study}, essentially `memorizing' the value function in a way that fails to generalize even to minor perturbations of the input.
While a diverse set of training methodologies seek to mitigate overfitting \citep{igl2019generalization, raileanu2021automatic}, the {source} of this pathology remains under-explored. This can be attributed in part to the challenge of even defining what type of generalization is desirable in RL: unlike in supervised learning, we typically do not have access to the true optimal value function or policy in the environment, which means we cannot directly compare the performance of a policy or the accuracy of a value function to some ground truth objective \citep{liu2020towards}. 

Instead, we consider interference as a proxy for generalization.
While networks with a large degree of interference may not always extrapolate correctly, networks with zero interference between states will not extrapolate at all, making generalization impossible. 
The importance of interference to generalization has been discussed previously in Chapters~\ref{chp:supervised} and~\ref{chp:invariance}, and was hinted at in the analysis of representation dynamics in Chapter~\ref{chp:rep-learning}, where we saw that neural networks tended to evolve representations under which states were easier to disentangle, provided that the environment offered a sufficiently dense reward signal. We now extend this analysis to understand why this phenomenon occurs, and whether excessive disentanglement can harm generalization when the agent is exposed to new observations.

Our primary contributions in this chapter will be twofold: first, to provide a rigorous theoretical and empirical analysis of the relationship between generalization, interference, and the dynamics of temporal difference learning; second, to study the effect of distillation, which avoids the pitfalls of temporal difference learning, on generalization to novel environments. 
Towards this first contribution, we will extend the analysis of Chapters~\ref{chp:rl-dynamics} and~\ref{chp:rep-learning} to show that the dynamics of temporal difference learning accelerate convergence along non-smooth components of the value function first, resulting in implicit regularization towards learned representations that generalize weakly between states. Our findings present an explanation for the observation noted widely across the literature that TD-learning produces representations that are particularly vulnerable to overfitting \citep{raileanu2021decoupling,zhang2018study}.  

We then evaluate whether these findings hold empirically across a range of popular deep RL benchmarks. We measure interference by constructing a summary statistic which evaluates the extent to which optimization steps computed for one state influence predictions on other states, which we call the \textit{update rank}. This metric is similar to the notion of feature rank proposed in Chapter~\ref{chp:rep-learning}, but applies to the updates performed by a network rather than its outputs. We find that value-based agents trained with temporal difference (TD) methods learn representations with weak interference between states, performing updates similar to those of a lookup table, whereas networks trained with policy-gradient losses learn representations for which an update on one state has a larger effect on the policy at other states.
Finally, we show that post-training policy distillation is a cheap and simple approach to improve the generalization and robustness of learned policies. We find that distillation is particularly effective at increasing \textit{smoothness} in the student network's output compared to that of the teacher. This property benefits interpolation and robustness to perturbations, however it is not a panacea: we see that increasing smoothness via distillation does not provide significant performance improvements on generalization to new tasks that are dissimilar to the training environment. This limitation will be addressed in Chapter~\ref{chp:icp}, where we will present a new representation-learning objective that targets generalization to novel environments by leveraging a notion of invariance.

\section{Learning dynamics and smoothness}\label{sec:learning-smoothness}
This section will explore a tension between learning dynamics in neural networks, which tend to `generalize-then-memorize' \citep{kalimeris2019sgd}, and the dynamics of temporal difference learning with tabular value functions, discussed in Section~\ref{sec:vf_gen}, which tend to pick up information about the value function's global structure only late in training.
We go on to study how these learning dynamics may affect the structure of gradient updates in the function approximation setting in Section~\ref{sec:fa_gen}.

\textbf{Eigendecomposition of transition operators.}
An important concept in our theoretical analysis will be that of the eigendecomposition of the environment transition matrix. We will follow the precedent of prior work, and that of Chapter~\ref{chp:rl-dynamics}, in considering diagonalizable transition matrices  \citep{machado2017laplacian, stachenfeld2017hippocampus, mahadevan2005proto}.
The relationship between the smoothness of an eigenvector and its corresponding value has been noted previously \citep{mahadevan2007proto}. However, prior discussion of this connection has defaulted to an intuitive notion of smoothness without providing an explicit definition; this intuitive notion is illustrated by the eigenvectors of the MountainCar environment visualized in Figure~\ref{fig:mdp}. We provide a concrete definition of the smoothness of a function on the state space $\states$ of an MDP $\mdp$ in order to provide an unambiguous characterization to which we will refer throughout this chapter.

\begin{definition}
Given a function $V : \states \rightarrow \mathbb{R}$, MDP $\mdp$, and policy $\pi$, we define its expected variation $\nu(V)$ as follows.
\begin{equation}\nu(V) = \sum_{x \in \states} |V(x) - \mathbb{E}_{P^\pi(x'|x)}V(x')| \; \end{equation}
We say $V$ is \emph{smooth} if $\nu(V)$ is small.
\end{definition}
This expression reveals a straightforward relationship between the eigenvalue $\lambda_i$ associated with a normalized eigenvector $v_i$ and the smoothness of that eigenvector.
\begin{equation}
   \nu(v_i) = \sum_{x\in \states}|v_i(x) - \mathbb{E}_{P^\pi(x'|x)} v_i(x') | = \sum_{x \in \states} |(1-\lambda_i) v_i(x) |
\end{equation}
In other words, the eigenvalue of an eigenvector precisely determines its smoothness. If $\lambda = 1$, for example, then the eigenvector must be constant over the MDP, whereas if $\lambda = -1$, then we have $\mathbb{E}_{P^\pi(x'|x)}[V(x')] = - V(x)$ and the expected value fluctuates between extremes when stepping from one state to another. The \textit{variance} over next-state values can in principle be large even for functions of low variation by our definition, though in our empirical evaluations (see e.g. Figure~\ref{fig:mc-ff}) smooth eigenvectors tended to also exhibit little variance. For our analysis of the \textit{expected} updates performed by TD learning, we will find the smoothness of the expected updates to be a more useful quantity than the variance. 
\begin{figure*}
    \centering
    \includegraphics[width=0.95\linewidth]{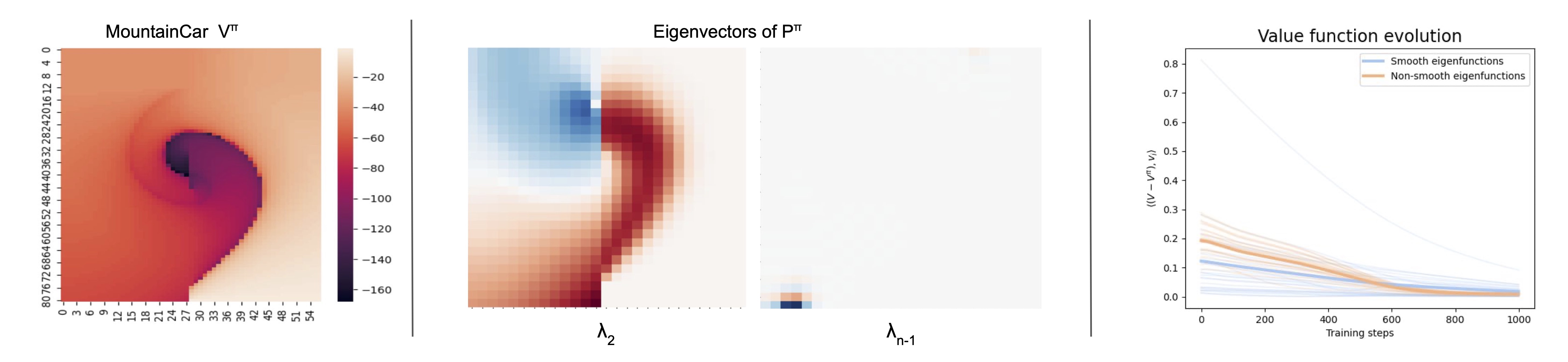}
    \caption[Value function evolution and eigendecomposition on the MountainCar environment.]{Left: value function of a near-optimal policy on MountainCar. States correspond to velocity (x-axis) and position (y-axis). Middle: eigenvectors associated with this policy computed for a discretization of the MountainCar state space. Right: value approximation error by eigen-basis coefficient along a trajectory generated by tabular TD updates with learning rate $\alpha = 0.1$ on the discretized MountainCar MDP. We compare 25 of the most-smooth eigenvectors with 25 eigenvectors corresponding to negative eigenvalues, and normalize the error by the magnitude of the projection of $V^\pi$ onto the basis spanned by each set of vectors. Each transparent line corresponds to the dot product with a different eigenvector, while the solid lines show the mean over each subspace. } 
    \label{fig:mdp}
\end{figure*}
\subsection{Tabular dynamics}
\label{sec:vf_gen}
In light of this more precise definition of smoothness, we now revisit our previous results on the convergence of value functions under temporal difference and Monte Carlo learning. We recall the analysis of Chapter~\ref{chp:rl-dynamics}, which expressed the dynamics of Monte Carlo (MC) updates as a continuous-time differential equation
\begin{equation*}
    \partial_t V_t = V^\pi - V_t
    \end{equation*}
where $V_t \in \mathbb{R}^\states$ is a function on the state space $\states$ of the MDP, resulting in the trajectory
    \begin{equation*}
    V_t = \exp ( -t)(V_0 - V^\pi) + V^\pi \, .
    \end{equation*}

Intuitively, this corresponds to a `straight line' trajectory where the estimated value function $V_t$ converges to $V^\pi$ along the shortest path in $\mathbb{R}^{\states}$. In practice, most deep RL algorithms more closely resemble temporal difference updates, which are expressed as 
\begin{align}
    \partial_t V_t &= -(I-\gamma P^\pi)V_t + R_t \\
    V_t &= \exp ( -t (I-\gamma P^\pi))(V_0 - V^\pi) + V^\pi \, . \label{eq:td_dynamics}
\end{align}

Whereas under Monte Carlo learning the value function converges equally quickly in all dimensions, its convergence under temporal difference learning depends on the environment transition matrix. As in Chapter~\ref{chp:rl-dynamics}, we will consider an MDP with a diagonalizable transition operator $P^\pi$ corresponding to some policy $\pi$. We write the decomposition of a function $V : \states \rightarrow \mathbb{R}$ as a sum of the eigen-basis vectors of $P^\pi$, written $\{v_1, \dots v_{|\states|} \}$, obtaining $V = \sum_{i=1}^{|\states|}\alpha_i v_i$ for some (unique) set of coefficients $(\alpha_i)_{i=1}^{|\states|}$.
Under this decomposition, we can show that a predicted value function trained via TD learning will converge more slowly along smooth eigenvectors of $P^\pi$.

\begin{restatable}{obs}{convergence}\label{obs:convergence}
    Let $P^\pi$ be diagonalizable, with eigenvectors $v_1, \dots, v_{|\states|}$ corresponding to eigenvalues $\lambda_1 > \dots > \lambda_{|\states|}$, and let $V_t$ be defined as in \eqref{eq:td_dynamics}. Write $V_t = \sum_{i=1}^{|\states|} \alpha^t_i v_i$ to express the value function at time $t$ with respect to the eigen-basis $\{v_i\}$. Then the convergence of $V_t$ to the value function $V^\pi = \sum_{i=1}^{|\states|} \alpha^\pi_i v_i$ can be expressed as follows.
    \begin{align*}
       \alpha^t_i - \alpha^\pi_i &=  \exp(-t(1-\gamma \lambda_i)) (\alpha_i^0 - \alpha_i^\pi)
    \end{align*}
\end{restatable}
\begin{proof}[Proof Sketch]
The full proof of this result can be found in Appendix~\ref{apx:proofs}. The key step lies in expressing the difference between $V_t$ and $V^\pi$ in terms of the eigen-basis $v_1, \dots, v_{|\states|}$, for which we can obtain a closed-form expression based on the eigenvalues $\lambda_1, \dots, \lambda_{|\states|}$. We consider the coefficient of $V_t$ corresponding to the basis vector $v_i$, which can be written out as follows.
\begin{align*}
    |V_t - V^\pi|[i] &= |\alpha^t_i - \alpha^\pi_i| \\
    &=   | \exp(-t(1-\gamma \lambda_i)) (\alpha_i^0 - \alpha_i^\pi) + \alpha_i^\pi  - \alpha^\pi_i| \\
    &=|\exp(-t(1-\gamma \lambda_i))  (\alpha_i^0 - \alpha_i^\pi) |\\
    &= \exp(-t(1-\gamma \lambda_i)) | (\alpha_i^0 - \alpha_i^\pi) | 
\end{align*}
We note that if the basis vectors are not orthogonal, this coefficient will not be equal to the projection of $V_t$ onto the basis vector $v_i$; however, understanding the evolution of coefficients still gives some insight into the convergence of smooth as opposed to non-smooth components of the value function space.
\end{proof}
The implications of Observation~\ref{obs:convergence} on the learned value function depend to some extent on the eigendecomposition of $V^\pi$. If $V^\pi$ is equal to the constant function, then we expect the high-frequency components of $V_t$ to quickly converge to zero. If $V^\pi$ puts weight on non-smooth eigenvectors, then early values of $V_t$ may assign disproportionately large weight to these components relative to their contribution to $V^\pi$. In practice, value functions tend to exhibit a mixture of smooth and discontinuous regions. The corresponding expression of $V^\pi$ with respect to the eigen-basis of $P^\pi$ consequently places non-zero coefficients on eigenvectors corresponding to negative eigenvalues in order to fit this discontinuity, though its spectrum is dominated by smooth eigenvectors. We include some illustrative examples of the spectra of value functions under different assumptions on the reward and transition matrix $P^\pi$ in Appendix~\ref{appx:numerical}, finding in most cases that the smooth eigenvectors of the MDP tend to be assigned greater mass by the value function. 
The following result highlights that non-smooth components of a predicted value function, while contributing relatively little to the Monte Carlo error, contribute disproportionately to the TD error, providing an incentive to fit these components early in training.

\begin{restatable}{thm}{tderror}
Let $P^\pi$ be real diagonalizable with eigenvalues $\lambda_1 > \dots \geq \lambda_n$ and $(v_k)_{k=1}^n$ the corresponding (normalized) eigenvectors. Then for any value function $V$, the TD error $\TD(V_t) = \|V_t - T^{\pi} V_t\|^2$ can be bounded as
\begin{align}
    \|\TD(V_t) \|^2 &= \| T^\pi V_t - V_t \|^2\\
    &= \| \sum (1-\gamma \lambda_i)(\alpha^\pi_i -\alpha^t_i)(v_i)\| ^2  \\
    &\leq \sum_{i=1}^n (\alpha^\pi_i - \alpha^t_i)^2 (1-\gamma \lambda_i)^2  \; 
\end{align}
with equality when $P^\pi$ has orthogonal eigenvectors.
\end{restatable}
The proof of this result can be found in Appendix~\ref{apx:proofs}. Monte Carlo updates, which simply regress on the value function, give equal weight to errors along any component of the basis. These incentives provide some intuition for the different trajectories followed by Monte Carlo and TD updates: in order to minimize the TD loss, the predicted value $V_t$ must quickly become accurate along non-smooth components of the value function $V^\pi$; however, its error due to smooth components such as the value function's bias term will have little effect on the loss and so converges more slowly. We provide an illustrative example of the relationship between the eigenvalue associated with a subspace and the convergence rate of the value function in that subspace in Figure~\ref{fig:mdp}.

\subsection{TD learning with function approximation}
\label{sec:fa_gen}
\begin{figure*}
    \centering
    \includegraphics[width=0.55\linewidth]{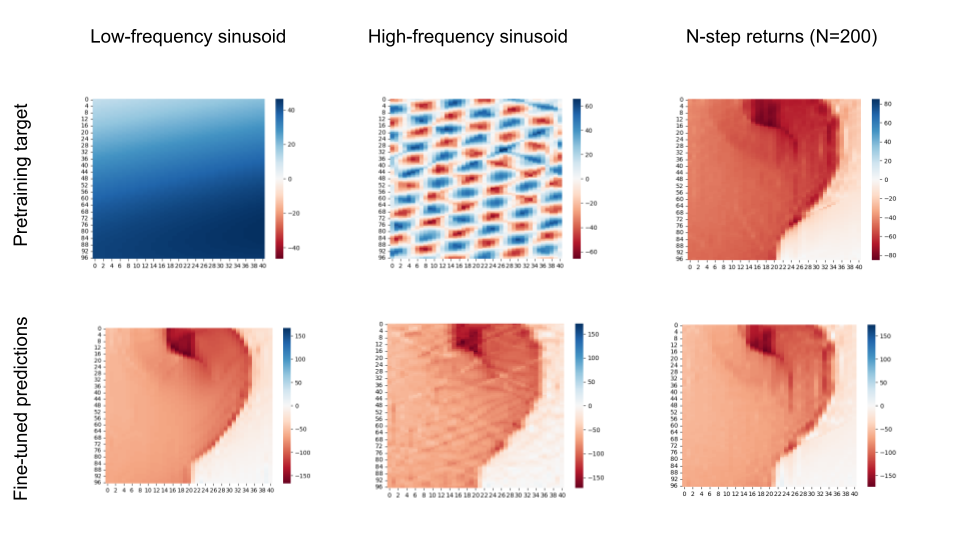}
    \includegraphics[width=0.4\linewidth]{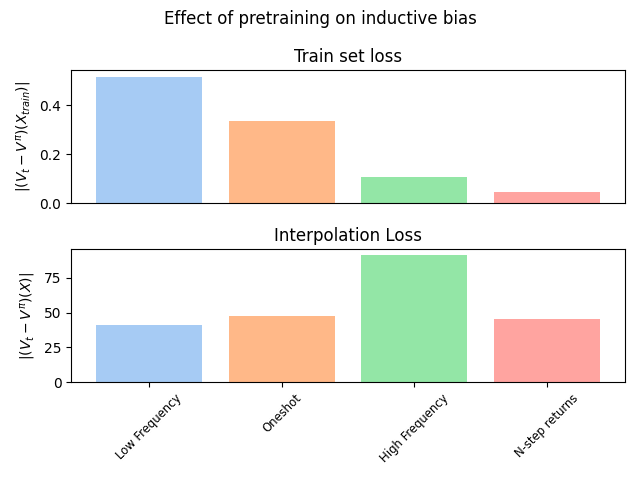}
    \caption[Networks trained to fit high-frequency target functions exhibit pathological interpolation properties when later fine-tuned on a value function.]{Networks trained to fit high-frequency target functions exhibit pathological interpolation properties when later fine-tuned on a value function. Left: visualization of pre-training targets (top) and final value estimate (bottom) after the pre-trained network is fine-tuned on the value function. Right: loss on the set of training states (top) and a finer-grained set of states which interpolate the training set (bottom) of each fine-tuned network.}
    \label{fig:mc-ff}
\end{figure*}
Most function approximation schemes leverage the assumption that states which are close together in observation space are likely to have similar values; i.e. they encode a preference towards smooth (with respect to the observations) functions. This pushes against the tendency of temporal difference updates to encourage the learned value function to fit the components of the value function with large variation first.
To investigate this tension, we consider the \textit{kernel gradient descent} regime.
\subsubsection{Kernel gradient descent}
Formally, a kernel is a positive definite symmetric function $K : \mathcal{X} \times \mathcal{X} \rightarrow \mathbb{R}$. In our case, we will define $\mathcal{X}$ to be the state space of an MDP. Letting $\mathbf{x} \subseteq \mathcal{X}$, we denote by $\tilde{K}$ the (symmetric) matrix $K(\mathbf{x},\mathbf{x})$ with entries $K(\mathbf{x},\mathbf{x})_{i,j} = K(\mathbf{x}_i, \mathbf{x}_j)$. Loosely speaking, a kernel can be thought of as measuring the similarity between two states, allowing us to encode certain forms of inductive bias into the learning dynamics of the agent. Importantly, the similarity of two states under $K$ does not inform us about how similar the states' initial values are, but rather how an update to the value function at one state influences the value of the other; in other words, it is a proxy for the \textit{interference} between two states.
Under kernel gradient descent, the trajectory of a function is defined in terms of a kernel $K$ and the function-space gradient of a cost function. We can translate TD semi-gradient updates into the kernel (semi-)gradient descent regime as follows, where we let $R^\pi$ denote the expected reward under policy $\pi$.
\begin{equation}
    \partial_t V_t = \tilde{K} ((\gamma P^\pi - I)V_t + R^\pi)
\end{equation}
It is straightforward then to obtain analogous results as before on the convergence of $V_t$ to $V^\pi$ based on the eigendecomposition of the matrix $\tilde{K} (\gamma P^\pi - I)$ in cases where this matrix is positive definite, though in general it may have negative eigenvalues. This decomposition will not in general have a closed form in terms of the eigendecompositions of $\tilde{K}$ and $P^\pi$, but special cases have been studied in the setting of linear regression by \citet{ghosh2020representations} and can be related to kernel gradient descent straightforwardly as discussed in Appendix~\ref{apx:proofs}. This setting also describes the dynamics of neural networks in the limit of infinite width \citep{jacot2018neural, fort2020deep, lee2020finite}, which follow kernel gradient descent with respect to the neural tangent kernel.

A more interesting case occurs when we assume some states in the environment are not updated during training. In this case, we can characterize their evolution (and potential influence on the training set via bootstrapping) using the direct sum of $K(\Xtrain, \Xtrain)$ and $K(\Xtest, \Xtrain)$. 
\begin{restatable}{thm}{ntk}\label{thm:ntk}
Let $K$ be a kernel and $\pi$ a fixed policy in an MDP with finite state space $X$. Let $\Xtrain \subset \states$ be the set of states in the support of $\pi$, $\Xtest = \states \setminus \Xtrain$, and let $V_t$ be a value trajectory obtained by applying kernel semi-gradient updates to some initial value function $V_0$ with kernel $K$. Let $K_{\mathrm{all}}$ be defined as
\begin{equation}
    K_{\mathrm{all}} = K(\Xtrain, \Xtrain) \oplus K(\Xtest, \Xtrain)
\end{equation}Then the trajectory of $V_t$ on the entire state space $X$ will be as follows,
    \begin{align}
        \partial_t V_t(X) &= (K_{\mathrm{all}})  [ (T^\pi V_t - V_t) (\Xtrain)]\;.
    \end{align}
\end{restatable}
A full derivation is provided in Appendix~\ref{apx:proofs}. These dynamics diverge notably from the standard kernel gradient descent regime in that changes to predictions on the test set have the potential to influence the dynamics of $V_t$ on the training set. This situation can arise when the agent is trained on a collection of sampled transitions which do not correspond to a contiguous trajectory, for example if transitions are removed from a replay buffer via some prioritization scheme, rather than in order of recency. In off-policy learning it can also arise when the policy used to construct the bootstrap target differs from that used to collect the replay buffer data. A large value of $K(\Xtest, \Xtrain)$ implies that updates to the training set hold influence over predictions on the test set, but at the cost of increasing asymmetry in $K_{\mathrm{all}}$ when viewed as an operator on $\mathbb{R}^\mathcal{X}$. In Appendix~\ref{appx:kernel-gd} we illustrate how this asymmetry can harm stability in the case of a simple radial basis function kernel. This can be viewed as a special case of off-policy temporal difference methods where the probability of sampling some states is set to zero \citep{ghosh2020representations}.

Combining insights from Theorem~\ref{thm:ntk} and Observation~\ref{obs:convergence}, we arrive at an intriguing conclusion: in the case of smooth kernels, the components of the value function most suitable to approximation via the kernel $K$ are precisely those which appear in the value estimate of the training set only later in the trajectory. As a result, the kernel does not receive the necessary information to generalize accurately to new observations until late in the training process. This observation runs contrary to the standard kernel regression regime, where one argument in support of early stopping is that kernel gradient descent methods converge along smooth components fastest \citep{jacot2018neural}. At the same time it is an obvious effect of bootstrapping, which requires that the agent update its predictions several times in order to propagate information about the value function through the entire input space. This effect is illustrated in Figure~\ref{fig:kernel-generalization} in Appendix~\ref{appx:kernel-gd}.
\subsubsection{Non-linear function approximation}
We saw in Chapter~\ref{chp:rep-learning} that the evolution of a network's representation is crucial to its ability to adapt to new learning signals. We now turn our attention to how the evolution of a network's features can influence its generalization properties. Our primary object of focus will be the gradient structure of a function approximator, whose analysis depends on the second-order effects of TD semi-gradient updates under finite step sizes. We consider the system
    \begin{equation} \label{eq:discrete_dynamics}
        \theta_{t+1} \gets \theta_t + \alpha \nabla_\theta V(\theta_t) \cdot [(\gamma P^\pi - I)V(\theta_t) + r ] 
    \end{equation} 
which can be viewed as an Euler discretization of the dynamics described in \eqref{eq:td_dynamics}. We will use the notation $f(\theta_t)$ to refer to the semi-gradient update on parameters $\theta_t$ inducing value function $V_{\theta_t}$, and write $\TD(\theta) = \frac{1}{2}\| V_\theta - \square T^\pi V_\theta \|^2$, where the $\square$ denotes a stop-gradient. Using the continuous time system in \eqref{eq:td_dynamics} to approximate these updates for time $t$ will gradually accumulate increasing errors, proportional to $(\alpha n)^2$, as it does not take into account the effect of the discrete step size on higher-order gradients of $V_\theta$. We apply a similar analysis to that of \citet{barrett2021implicit} and \citet{ smith2020origin} to understand the effect of the discrete learning dynamics on the gradient structure of $V_\theta$ itself. We let
\begin{equation} \small \label{eq:second-correction}
    f_1(\theta) = -\frac{1}{2} \nabla_\theta \| \nabla_\theta \TD(\theta) \|^2 + \gamma (\nabla_\theta ^\top V P^\pi \nabla_\theta V) f(\theta)
\end{equation}
to obtain a second-order correction describing the effect of gradient descent on the gradient structure of the learned representation.
\begin{restatable}[Second-order dynamics]{obs}{theoremsecond} \label{thm:second}
    Let $\theta_t$ be defined by the discrete-time system (\ref{eq:discrete_dynamics}) with step size $\alpha$. Let $f_1 (\theta)$ be defined as in (\ref{eq:second-correction}). Let $\ttheta_t$ denote the trajectory obtained by following the dynamics:
    \begin{align}
      \partial_t  \ttheta_t = f(\ttheta_t) + \frac{\alpha}{2} f_1(\ttheta_t)\;.
    \end{align}
    Then we have $ \theta_{n} \approx \ttheta_{n\alpha} + O( (n\alpha) ^3)$, where $\ttheta_{n\alpha}$ denotes the value of $\ttheta_t$ at time $t=n\alpha$.
\end{restatable}

The proof of this result follows a similar structure as that of \citet{smith2020origin} and can be found in Appendix~\ref{apx:proofs}; the key difference from the gradient descent regime is that we must now take into account the influence of gradient updates on the loss via the Bellman targets. This distinction presents itself in the form of $f_1$ constructed in \eqref{eq:second-correction}: the first term in the sum consists of a semi-gradient norm penalty term with respect to the instantaneous TD error, while the second term partially offsets this penalty along smooth components of the MDP, accounting for the target drift due to gradient updates. The first term is the standard gradient norm penalty observed in supervised regression. The second term has been studied in the setting of offline RL by \citet{kumar2021dr3}, where it plays a greater role than in the online RL problems in this chapter. 

The effect of this penalty term on the evolution of interference in deep neural networks is difficult to analyze in closed form due to the influence of the network Hessian, however our empirical observations from Chapter~\ref{chp:rep-learning} provide an illustrative example.  We recall that the previous section showed that early TD targets will typically be less smooth than the true value function. The gradient norm penalty will thus implicitly discourage interference between states in order to fit this relatively discontinuous function in a manner robust to noise in the optimization process. This observation is foreshadowed in the findings of Chapter~\ref{chp:rep-learning}, which found that deep RL agents learn to map nearby states to distinct feature vectors in dense-reward environments. We illustrate an analogous effect on interference in Figure~\ref{fig:qualitative}, where networks trained in dense-reward games (whose early TD targets will exhibit greater variation) exhibit weaker interference after training than they did at initialization. We can further observe the long-term impact of fitting non-smooth targets on the inductive bias of a network in Figure~\ref{fig:mc-ff}, where networks trained to fit high-frequency sinusoid functions exhibit pathological interpolation behaviour when later fine-tuned on the target value function.

In combination, the findings of this section suggest that the dynamics of temporal difference learning work to discourage interference between states in deep RL by fitting high-frequency components of the value function early in training while also encouraging robustness of the loss to noisy optimization steps. While this may result in more stable learning, as highlighted in Theorem~\ref{thm:ntk}, it has the double-edged effect of reducing the degree to which the network may generalize to novel observations. We now leverage these results to gain insight into deep RL agents trained in rich-observation environments. 

\section{Generalization and interference in deep RL}
\label{sec:rank-exps}
The inductive bias of a neural network is influenced heavily by the initial training period \citep{frankle2020the, achille2018critical, golatkar2019time}. Much like in human development, the data a neural network sees in the early stages of optimization sets the network's inductive bias in a way that can be difficult to unlearn later. The distribution of states and the targets that the agent is trained to fit during this period therefore have an outsize influence on what types of targets the network will be able to easily fit down the road. In Chapter~\ref{chp:rep-learning}, we saw that networks which did not receive any reward signal early in training had greater difficulty distinguishing states later in training. We now dive deeper into how dense rewards shape agents' representations: whereas sparse rewards discourage state disentanglement, we conjecture that dense reward environments may encourage excessive disentanglement -- in other words, \textit{discourage} interference -- at the expense of generalization in value-based deep RL algorithms.
We begin by presenting a quantitative approach to measure the degree to which interference is occurring between states in the agent's visitation distribution. Armed with this metric, we evaluate two concrete hypotheses. First, that deep neural networks trained with TD updates will exhibit weaker interference between states as training progresses compared to their value at initialization (Hypothesis~\ref{hyp:td}). Second, that networks trained with TD learning will exhibit weaker interference than those trained with policy gradient objectives (Hypothesis~\ref{hyp:distill}).

\subsection{Representation evolution in DQN agents}

We begin by presenting a method to track interference in deep RL agents. Given a set of transitions $\tau_1, \dots, \tau_n$ of the form $\tau_i=(\bx_i, a_i, r_i, \bx'_i)$ and a value function $V$ with parameters $\theta$, we let $\theta_i$ denote the network parameters after performing an optimization step with respect to the transition $\tau_i$. We then construct a matrix $\updatematrix$ entry-wise by computing the interference $\deltaint(\bx_i, \bx_j)$ between each pair of states $\bx_i, \bx_j$ in the sampled set. We refer to this object as the \textit{update matrix}, as it consists of the changes in the network outputs as a result of gradient updates. Recalling \eqref{eq:deltaint}, we let $Q_\theta(\bx)$ denote the vectorized action-value function output by a DQN network at state $\bx$ given parameters $\theta$, $g_\ell$ the TD semi-gradient update direction, and $\eta$ the optimizer state.
\begin{equation}
    [\updatematrix]_{i,j} = \deltaint (\bx_i, \bx_j) = \|Q_\theta(\bx_j, \cdot ) - Q_{\theta'}(\bx_j, \cdot)\| \text{ where } \theta' = \theta + \alpha g_\ell (\bx_i; \theta, \eta)
\end{equation}
\begin{figure}
    \centering
    \includegraphics[width=\linewidth]{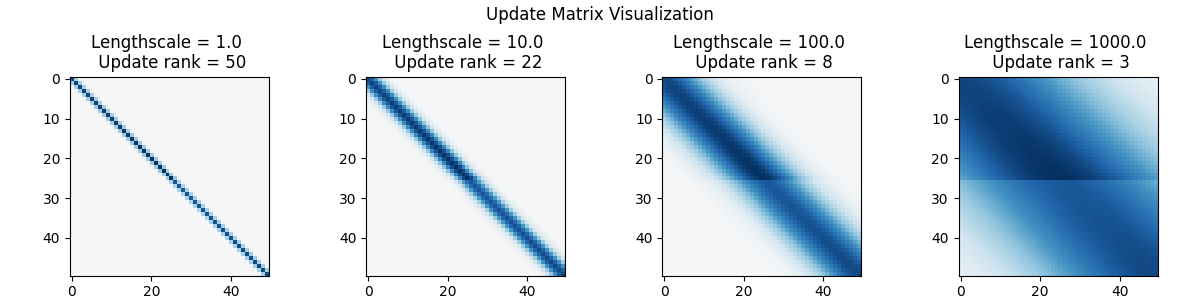}
    \caption[Update matrices for a toy kernel regression problem.]{Update matrices for a toy kernel regression problem using RBF kernels of different lengthscales. The smaller lengthscales discourage generalization, and induce update matrices with a more pronounced diagonal component.}
    \label{fig:update_heatmaps}
\end{figure}
Figure~\ref{fig:update_heatmaps} provides an illustration of the update matrix of a radial basis function kernel regression model where the lengthscale of the kernel is set to different values. In DNNs, the properties of this matrix will depend on the optimizer used to perform updates, leading to notable differences from the neural tangent kernel regime studied elsewhere \citep{yang2022overcoming} in the case of non-linear function approximators trained with adaptive optimizers.
At one extreme, the update matrix $\mathbf{I}_{\Delta}$ for a tabular value function (demonstrated by the short lengthscale kernel plot on the left hand size of Figure~\ref{fig:update_heatmaps}) will have non-zero entries only along the diagonal and the matrix will have full rank. At the other, if the value function is represented by a single parameter $\theta \in \mathbb{R}$, then every row will be identical up to a scalar multiple and the matrix will have rank one. More generally, the rank of this matrix can be interpreted as a proxy for whether an agent tends to \textit{generalize} updates between states (low rank), or whether it \textit{memorizes} the value of each state-action pair independently from other states (high rank). As in Chapters~\ref{chp:invariance} and \ref{chp:rep-learning}, we use an approximate version of the rank that discards negligible components of the matrix based on the singular value decomposition, analogous to the feature rank quantity computed in the previous chapter. We provide full computation details in Appendix~\ref{appx:update-details}
\begin{equation}{}
    \rho(\mathbf{I}_{\Delta}) = | \{ \sigma | \sigma \in \mathrm{SVD}(\mathbf{I}_{\Delta}) \text{ and } \sigma > \epsilon\} |
\end{equation}

We will refer to $\rho(\mathbf{I}_{\Delta})$ as the \textit{update rank}. Typically, $\mathbf{I}_{\Delta}$ will depend on the optimization state $\theta$ (which describes the network architecture, parameters, and optimizer state) of the agent and so we will sometimes write $\mathbf{I}_\Delta(\theta)$ to denote this dependence when it is not obvious. An alternative approach outlined by \citet{daneshmand2021batch} involves computing the Frobenius norm of the difference between the matrix $\mathbf{I}_\Delta$ and the identity, however this may overestimate interference in optimizers which use momentum due to constant terms in the update matrix. In our case, a change in the rank of $\mathbf{I}_\Delta(\theta_t)$ and $\mathbf{I}_\Delta(\theta_{t+n})$ after $n$ steps of optimization indicates that the inductive bias of the network has shifted to generalize more, in the case of a decrease in rank, or less, in the case of an increase in rank, between states. The update rank thus gives us a means of testing our first empirical hypothesis of this chapter. 

\hypothesis{H1}{deep neural networks trained with TD updates exhibit weaker interference between states as training progresses.\label{hyp:td}
}

We measure interference by evaluating $\rho(\mathbf{I}_{\Delta}(\theta))$ for states sampled from the replay buffer of a deep RL agent. Hypothesis~\ref{hyp:td} predicts that the rank $\rho(\mathbf{I}_{\Delta}(\theta_t))$ will increase as a function of the training step $t$. To test this we train a standard DQN architecture on environments from the Atari 2600 suite, and save a range of checkpoints throughout training. We illustrate the evolution of agents' update matrices $\mathbf{I}_{\Delta}(\theta_t)$ over the course of training in Figure~\ref{fig:qualitative}. We observe that RL agents trained in dense-reward environments tend to develop update matrices which resemble those of tabular value functions: they tend to have a pronounced diagonal or block-diagonal structure later in training, compared to a more linear structure early in training. Those trained in the absence of reward, i.e. those for which the target value function has no high-frequency components, maintain low-rank update matrices through training as our theory would predict. We find that similar results hold for a range of update rules, including distributional updates performed in the C51 algorithm \citep{bellemare2017distributional}. Quantitatively, we track the update rank of these checkpoints on the right hand side of Figure~\ref{fig:qualitative}. We include further evaluations in Appendix~\ref{apx:more-results}.

\begin{figure}[h]
\centering
\begin{minipage}{0.645\linewidth}
    \includegraphics[width=0.99\linewidth]{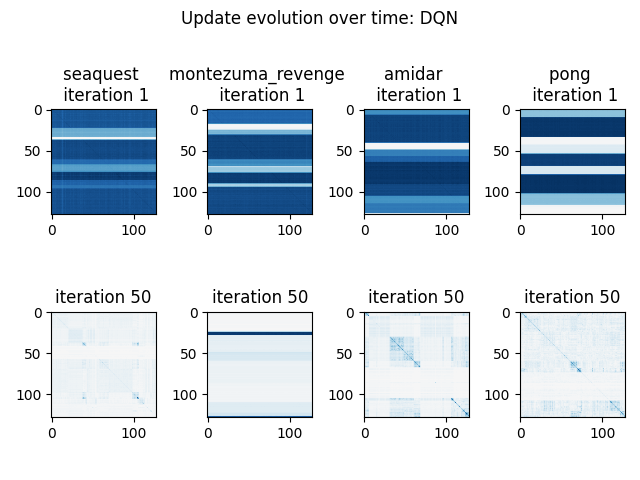}
    \end{minipage}
    \begin{minipage}{0.339 \linewidth}
    \includegraphics[width=0.993\linewidth]{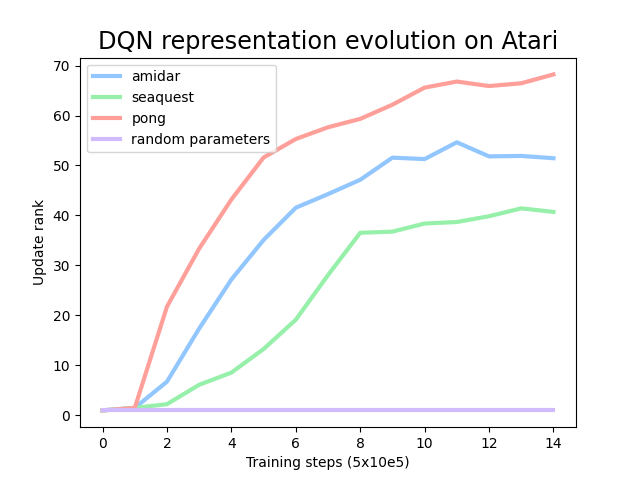}
    \includegraphics[width=0.993\linewidth]{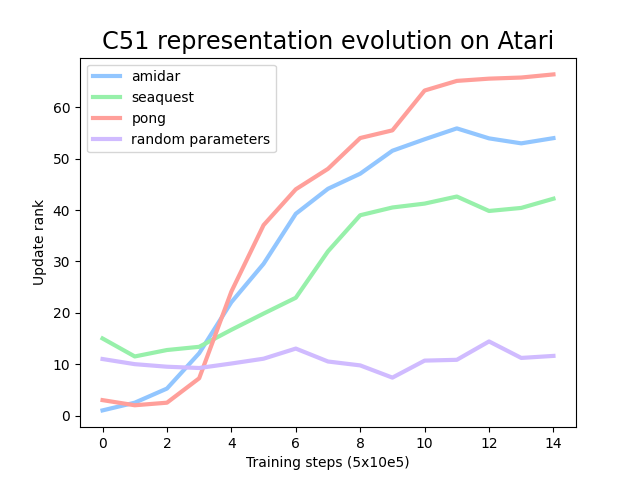}
    \end{minipage}
    \caption[Evolution of interference in DQN agents trained on Atari.]{Left: $\mathbf{I}_{\Delta}(\theta_t)$ for agents trained on games from Atari. The networks initially exhibit low update rank, but after 50 iterations (5M frames of experience), the updates rank increases significantly. This is tracked in the right hand side plots over the course of approximately 7M frames. Random parameters refers to the update rank obtained by a randomly initialized neural network.}\label{fig:qualitative} 
\end{figure}

\subsection{Actor-critic methods}

Hypothesis~\ref{hyp:td} was motivated by the observation that early temporal difference targets predominantly contain information about immediate rewards, implicitly rewarding networks which memorize rather than generalize the learned value function. However, policy gradient methods do not encounter this issue: these methods seek to maximize the probability of selecting an action with large advantage, which is independent of the smoothness of the TD targets. While policy gradient methods encounter their own challenges, in particular high gradient variance, these challenges are largely orthogonal to the particularly adversarial nonstationarity of TD learning. Actor-critic methods, which involve training a policy via policy gradient methods and a value function via temporal difference methods, therefore present an ideal test bed to study whether the decline in generalization observed during training in the previous section is unique to TD updates, or whether it is a more general property of RL agents. A priori, we have no reason to expect the policy gradient losses to exhibit significant increases in structure later in training. We therefore propose the following hypothesis.

\hypothesis{H2}{networks trained with TD learning will exhibit weaker interference than those trained with policy gradient objectives.
}
\begin{figure*}[h]
    \centering
    \includegraphics[width=\linewidth]{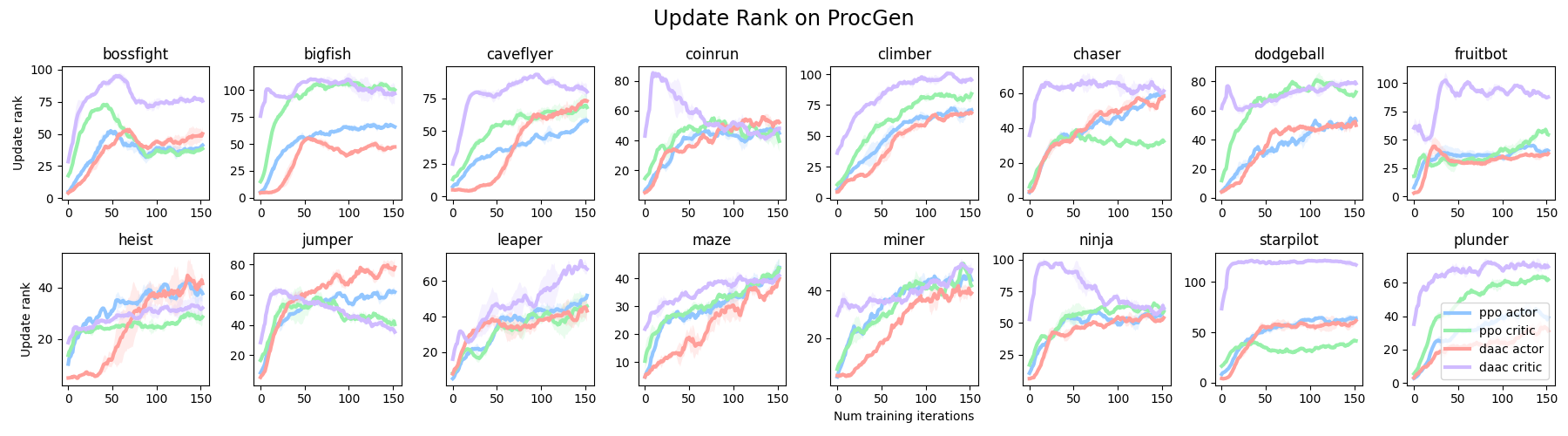}
    \caption[Update dimension of actor-critic methods in ProcGen.]{Update dimension of actor-critic methods in ProcGen. Shading indicates minimum and maximum values over 4 seeds. We observe that the update dimension of the separate critic architecture in DAAC (the lilac line) consistently has the highest update rank early in training, while the actors have the lowest rank in the early training stages and only surpass the DAAC critic later in training. }
    \label{fig:daac}
\end{figure*}
We repeat the analysis of the previous section with actor-critic methods to study whether this is the case, measuring the change in the actor's output policy at a state $\bx$ rather than Q-values. We run our evaluations in the ProcGen environment \citep{cobbe2019quantifying}, which consists of 16 games with procedurally generated levels. While the underlying mechanics of each game remain constant across the different levels, the layout of the environment may vary. The agent is given access to a limited subset of the levels during training, in this case 10, and evaluated on novel randomly generated levels. In this section we study interference only on the training environments. 
We investigate Hypothesis~\ref{hyp:distill} using two different algorithms on the ProcGen suite: PPO \citep{schulman2017proximal}, which uses a shared representation network for both the actor and critic, and DAAC \citep{raileanu2021decoupling}, where there are no shared parameters between the actor and the critic. We then evaluate the update dimension of the two different methods, and plot the results in Figure~\ref{fig:daac}. Additional details can be found in Appendix~\ref{appx:procgen-details}. Omitting the critic gradients from the actor's representation leads to significantly lower update dimensions early in training in a number of environments, including bigfish, heist, and miner. Further, the critic network in DAAC, which receives only TD gradients, exhibits markedly higher update rank in all environments in at least the early stages of training, and often throughout the entire trajectory, than the other networks which have access to the actor gradients.

\section{Post-training distillation and generalization}

The previous section has demonstrated that training value-based deep RL agents results in a bias towards weaker interference between inputs. Arguably, this weak interference may be beneficial to the stability of these learning algorithms, but it comes at the cost of generalization. This bias towards memorization arises when, during the network's crucial early development stage, it is trained to fit target functions that do not capture the global structure of the value function. 
This leaves us with an open question: how might we adapt the training procedures followed in deep RL to obtain agents that can adapt to inputs they have not seen before?

 One simple solution to this problem is to train a freshly initialized network on the final value function obtained by TD learning. If the teacher network was able to fit the low-frequency components of the value function -- even if this was achieved via memorization -- then the freshly initialized network will be able to benefit from incorporating this structure into its predictions from the start of its optimization procedure. In the case of policy distillation in actor-critic algorithms, it allows us to evade the influence of critic gradients without requiring a complete decoupling of the actor and critic during training time. Such approaches have seen success in prior work \citep{igl2019generalization, nikishin2022primacy}; this section presents a deeper study of a mechanism driving this success.

\subsection{Value distillation}
We begin by studying the effect of post-training distillation on robustness of the learned value function in environments from the Atari suite \citep{bellemare2013arcade}. We first consider value distillation as a means of eliminating the counterproductive bias towards memorization induced by early TD targets. We leverage a data collection policy from a pre-trained teacher network $q_t$, and perform distillation of a freshly initialized network $q_s$ on this data. We follow a similar procedure to that of \citet{ostrovski2021the} to perform distillation of the function $q_{\mathrm{s}}$ on data collected sampled from the teacher's replay buffer $\mathcal{B}_T$, leveraging their insight that distillation on \textit{all} action values, rather than only the value of the action taken by the teacher agent, yields significantly higher performance.  We additionally study the effect of behaviour cloning with entropy regularization, obtaining the objectives
\begin{align}\label{eq:value_distill}
    \ell_{\mathrm{VD}}(q_{\mathrm{S}}, q_{\mathrm{T}}) &= \mathbb{E}_{s \sim \mathcal{B}_{\mathrm{T}}} \bigg [\sum_{a \in \mathcal{A}} (q_{\mathrm{S}}(a) - q_{\mathrm{T}}(a) )^2 \bigg ]
    \intertext{and}
    \ell_{\mathrm{BC} }(\theta) &= \mathbb{E}_{s,a\sim \mathcal{B}_{\mathrm{T}}}[ \log \pi_\theta(s,a) + \lambda H(\pi_\theta(s)] \label{eq:policy_distill}
\end{align}
where $H(\cdot)$ denotes the entropy of the policy. We set $\lambda = 0.01$ in our evaluations.
We show results for value distillation \eqref{eq:value_distill}, which regresses on the outputs of the frozen Q-network, and behaviour cloning \eqref{eq:policy_distill}, which predicts the action taken by the frozen Q-network.
We track three quantities: the performance of the learned policy, the robustness of the learned policy to perturbations, and the consistency of the learned policy when interpolating between observations. The performance is measured by following an $\epsilon$-greedy policy in the training environment, with $\epsilon=0.01$. The robustness to perturbations is measured by tracking whether the network takes the same action under a Gaussian perturbation to its input as in the unperturbed observation. Finally, we investigate the network's interpolation behaviour by evaluating whether, given a convex combination of observations $o_1$ and $o_2$, the network takes the same action under the combination as it does in either of the original observations. Additional evaluation details can be found in Appendix~\ref{appx:atari-details}. We are interested in evaluating the folowing hypothesis.

\begin{figure}
    \centering
    \includegraphics[width=0.939\linewidth]{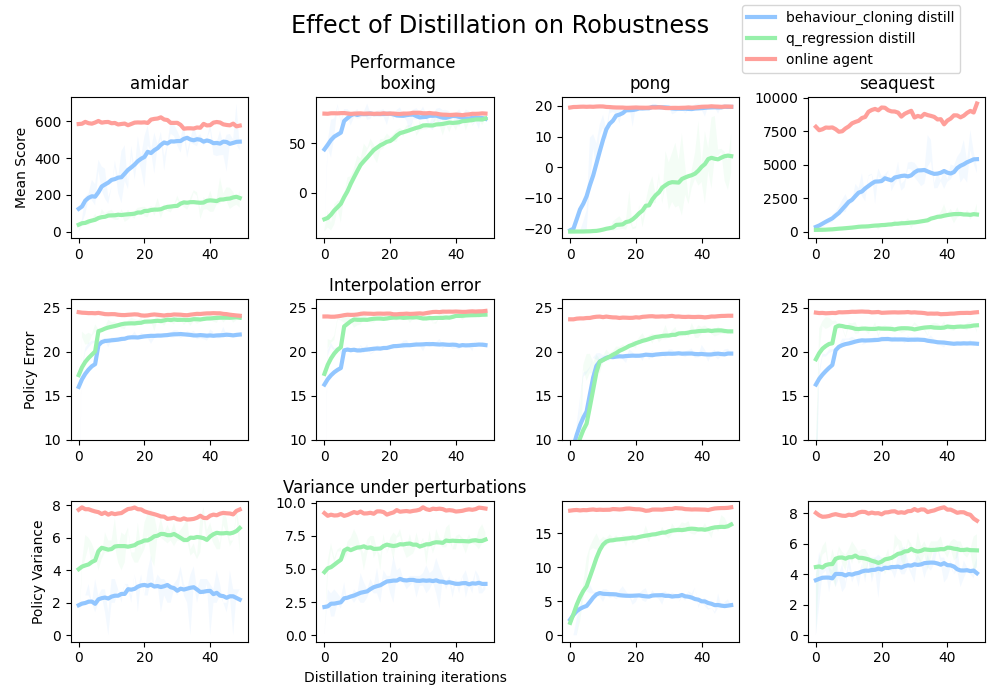}
    \caption[Performance and robustness to perturbations of different distillation approaches in games from the Atari suite.]{Performance and robustness to perturbations of different distillation approaches in games from the Atari suite. Post-training distillation results in policies that are more consistent under perturbations and under interpolation between observations. Axes indicate the $\ell_1$ norm between the policy on the original input batch and on the perturbed input batch.} \vspace{-1em}
    \label{fig:atari-gen}
\end{figure}

\hypothesis{H3}{post-training distillation will reduce overfitting and therefore improve robustness; this will hold more strongly for policy distillation than for value distillation due to the discontinuity of the value distillation targets. }
Figure~\ref{fig:atari-gen} shows that the distillation approaches yield policies that exhibit significant improvements in robustness to perturbations and are more consistent under interpolations between observations. We observe that the behaviour cloning method matches or nearly matches the performance of the pretrained agent in three of the four environments, while also obtaining the best robustness. Both behaviour cloning and value distillation improve upon the robustness of the teacher network that was trained online. We conclude that both value and in particular policy distillation increase smoothness in the form of improved robustness to perturbations and greater interpolation consistency. This finding motivates the next section, where we will dig deeper into policy distillation.

\subsection{Policy distillation}
We have previously shown that decoupling the policy and value function approximators in an actor-critic architecture resulted in weaker interference in the value approximator, and stronger interference in the policy network. However, it is not clear that this should necessarily lead to better final performance on the test environment, as the interaction between the policy and value function training is complex and difficult to predict -- a policy with a small generalization gap but weak training performance may nonetheless perform worse at evaluation time than a policy with a larger generalization gap. An alternate approach to remove the influence of critic gradients on the learned policy is to train a behaviour cloning agent on the final policy obtained after training. Because we only need to distill the policy, we avoid contaminating the actor with value approximation gradients while (hopefully) preserving the performance of the original agent on the training environments. We return to the ProcGen benchmark, with the hypothesis that post-training distillation of PPO agents should produce policies which improve on the ability of the final trained actor to generalize to new levels. 
\begin{figure*}
    \includegraphics[width=\linewidth]{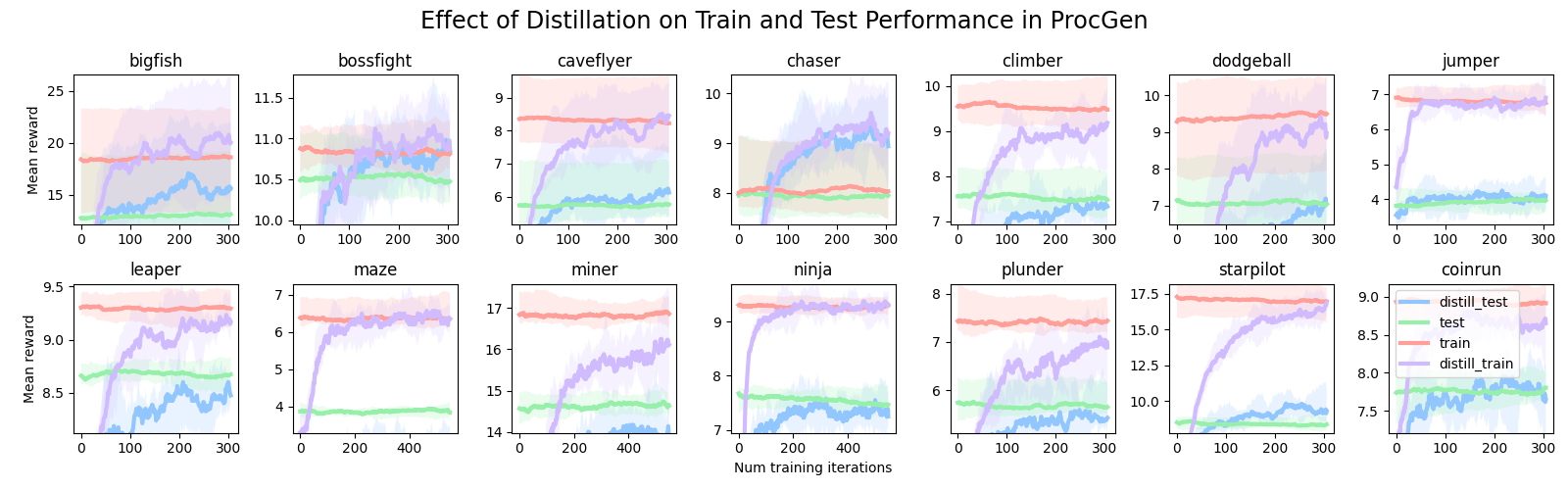}
    \caption[Effect of policy distillation on generalization in environments from the Procgen suite.]{Effect of policy distillation on generalization in environments from the Procgen suite. We plot the pretrained networks train environment and test environment performance, along with the performance of the distilled agent on the test environments. We see significant improvement on test environments in bigfish, caveflyer, chaser, climber, and bossfight. }
    \label{fig:procgen-distill}
\end{figure*}

\hypothesis{}{post-training policy distillation will reduce the generalization gap of actor-critic agents trained on procedurally generated environments.\label{hyp:distill}}

Hypothesis~\ref{hyp:distill} brings with it a subtle but crucial caveat: our discussion of generalization thus far has focused implicitly on \textit{within-environment} generalization, but procedurally generated benchmarks inherently evaluate agents on novel environments, inducing a distribution shift. Our primary tool in studying generalization in the previous section, the update rank, carries with it the assumption that a network's generalization behaviour on training observations will be representative of its behaviour on novel states that may be encountered later. However, the ProcGen benchmark requires generalization \textit{out-of-distribution} to novel environments whose observations may have probability zero under the training environments. A neural network will not necessarily exhibit similar generalization behaviour in a novel environment which may be highly visually dissimilar to its training environment. As such, methods such as policy distillation which increase generalization between observations in the training environment should only be expected to benefit test set performance if the test environments are sufficiently `within-distribution'.

With this caveat in mind, we proceed to evaluate Hypothesis~\ref{hyp:distill}. We train a set of PPO agents on the difficult setting of the ProcGen benchmark, using the same experiment configurations as that 
used by \citet{cobbe2020leveraging}. We then perform behaviour cloning on these agents, using data collected only from the training environments, by training a freshly initialized network (the distillation agent) to minimize an entropy-regularized KL divergence with the teacher's policy on trajectories collected by the teacher. Training was terminated at 50M sampled environment steps, except for Maze, Miner, and Ninja, which were run for 90M sampled steps as the training performance of the student had not yet matched that of the teacher after the initial budget. We then evaluate the distilled agent's performance on the test environments. Results are shown in Figure~\ref{fig:procgen-distill}.

We find a rich heterogeneity in the shape of the learning curves of the student on the train and test environments across the different ProcGen games. Post-training distillation in most settings slightly reduces the generalization gap obtained by the original network, with the notable exception of the `maze' environment. In settings where the gap between train and test performance is larger, we see a more mixed effect from distillation. The naivety of our approach suggests that it is likely that final performance and sample-efficiency of these agents could be improved using tools from the policy distillation literature such as allowing the student to generate training data \citep{czarnecki2019distilling, rusu2016policy, teh2017distral}.

We observe three broad categories of behaviour in the student agent's learning curves. First, in the two environments where PPO achieves a negligible generalization gap, distillation improves not only test performance but also performance on the training environments. This suggests that the memorization behaviour of PPO was limiting not only its test set performance but also its returns on training environments. We further do not observe any evidence of stereotypical overfitting in these environments: the train and test performance increase in lockstep with each other. In environments where there is a moderate gap between train and test environment teacher performance, the test set performance of the student tends to increase more slowly than its training set performance, resulting in slow but fairly consistent convergence of the learned policy to approximately the level of that of the teacher (or slightly exceeding that in the case of bigfish, caveflyer, jumper, and starpilot). In no environments do we see the classic sign of overfitting: a sequence of increasing, saturating, and decreasing training performance. Instead, the final test performance depends more on the relative positive slope of the test and the training performance curves -- in other words, it depends on the degree to which the agent is able to identify shared structure between the training environments that is more sophisticated than can be detected by measuring extrapolation between observations within a single training environment. This notion of shared structure will be explored in greater depth in Chapter~\ref{chp:icp}.

\section{Conclusions}

The analysis of this chapter presents a complementary perspective to that of Chapter~\ref{chp:rep-learning} in studying how learning dynamics can influence generalization, particularly in dense-reward environments. Our key take-away is that the nature of the targets that we ask a neural network to fit early in its training trajectory hold significant influence over the bias of the learning process towards certain types of functions later. In the case of sparse-reward environments, networks which overfit to the task of predicting the zero target struggled to fit any non-trivial target function later in training. In the case of dense rewards, we've shown in this chapter that networks evolve to be more prone to memorization as a result of fitting discontinuous and unstructured targets early in training. Broadly, our analysis suggests that the nature of the learning dynamics of value-based RL discourages generalization in deep neural networks. We have shown that temporal difference learning in the continuous-time dynamics setting of Chapter~\ref{chp:rl-dynamics} fits non-smooth components of the value function first, resulting in an implicit bias towards representations that encode near-tabular updates. In the context of prior work demonstrating that weaker generalization can improve the stability and convergence of RL algorithms, this phenomenon may be beneficial to an agent's stability at the cost of observational overfitting. We further show that post-training distillation improves generalization and robustness, mitigating some of the tendency of value-based RL objectives to encourage overfitting.

The findings presented here are a crucial stepping stone along the path to the principled development of robust and stable deep RL algorithms which are capable of strong generalization performance. Our insights may prove useful in a range of future directions, such as using different architectures during training and distillation, leveraging larger neural network function approximators to minimize harmful interference, and modifying the update rule used in TD learning to adaptively promote or inhibit interference between inputs. Further, the role of the optimizer is fundamental to the phenomena studied in this paper, and RL-specific optimization approaches may benefit from our findings. The notion of update rank that we consider here may also be adapted to other settings to shine a light onto where a network lies on the memorization-generalization spectrum at different points during training, potentially providing a mechanism to trigger early stopping in supervised learning settings, or as a meta-learning objective for hyperparameter tuning and architecture search. 

The experiments presented in this chapter have considered both single-environment (in the case of Atari and MountainCar) and multi-environment (ProcGen) problems. While our findings concerning interpolation and robustness to perturbations in single-environment RL are straightforward, analysis of the multi-environment setting is more subtle. The update rank we compute on agents measures the degree of generalization between observations from the training environments. Because the test environments may differ systematically from the training environments, the update rank may not capture properties of a neural network that enable effective generalization to these out-of-distribution test examples. Indeed, based on our observations from Chapter~\ref{chp:supervised} we have no reason to expect that invariance over training inputs  will generalize to sufficiently out-of-distribution test environment observations. The following chapter will study the multi-environment problem explicitly. It will outline sufficient assumptions on the structure of the data-generating process to ensure generalization is tractable, and present a novel representation-learning approach that is able to identify the shared structure between the training and test environments.

\chapter{Generalization across environments}
\label{chp:icp}
\minitoc
\section{Introduction}

A robot's LiDAR sensor is knocked askew. An autonomous vehicle encounters hail for the first time. In the distant future, a RoboChef discovers that the kitchen in which it usually works has been remodelled.
Reinforcement learning agents will frequently encounter situations at deployment that they did not see during training. Such situations will not necessarily correspond to novel observations from the agent's training environments, but they may share a broader notion of \textit{causal structure}. An agent which has learned to robustly exploit this structure will generalize with ease to the challenges it faces at deployment. 
However, in such cases good generalization between observations from the training environment will not necessarily be sufficient for an agent to pick up on this causal structure. We saw the limitations of such generalization in Chapter~\ref{chp:gen-rl} on the ProcGen benchmark. Reducing the degree of memorization performed by a neural network improved some types of robustness, but did not always result in better performance on test environments.
In the worst case, some training environments may contain spurious correlations that will not be present at test time. An agent which depends on these correlations may then generalize well to new observations in its training environments, but experience catastrophic failure on new environments where the correlation is not present  ~\citep{azhang2018natrl,Song2020Observational}. Two things are clearly necessary in order to obtain effective generalization outside of the training environment. First, we must look beyond interference to more sophisticated notions of invariance. Second, we must consider families of test environments for which the training environments provide sufficient information for generalization to be possible.

Generalization to new environments has been a topic of great interest to the RL community, but this interest has been concentrated on settings that make few explicit assumptions on the shared structure between training and test environments \citep{kirk2021survey}. Recent works~\citep{amit2018mlpacbayes,yin2019meta} have developed generalization bounds for the multi-task problem, but they depend on the number of tasks seen at training time, which can be prohibitively expensive given the sample complexity of RL even in the single task regime. To obtain stronger generalization results, we consider a multi-environment RL problem: like multi-task RL, the agent seeks to maximize return on a set of environments, only some of which are available to the agent during training. 
We make the assumption that there exists some latent {causal structure} in the form of a causal graph that is shared among all of the environments, and that the sources of variability between environments can be modelled by interventions on spurious variables in the causal graph. This family of environments, which we show to be equivalent to a \textit{Block MDP}~\citep{du2019pcid}, allows for observations to vary, but fixes the latent states, dynamics, and reward function.

Chapter \ref{chp:gen-rl} studied multi-environment generalization, but with relatively few assumptions on the shared structure between the different environments; we only required that the environments be drawn from the same data-generating process. We show in this chapter that the added assumption of shared structure allows for much stronger generalization results than have been obtained by prior work with significant sample complexity improvements. Our results are enabled by a key insight: we turn the problem of identifying features that will generalize well to novel environments from a learning problem, for which sample complexity may be prohibitive, to a causal identification problem, which will often be more tractable.
Indeed, where \citet{cobbe2019quantifying} and \citet{azhang2018genrl} find that agents trained using standard methods must see many thousands of training environments in order to successfully generalize to new environments, in our more restricted setting as few as two or three training environments can be sufficient to identify the correct causal structure. 

The main technical contribution of this chapter is the application of tools from {causal inference} to identify (and in some cases learn) representations that will enable an agent to generalize well to new environments. We propose a method motivated by a simple intuition: features which remain {invariant} across the different training environments are likely to be causally related to the reward and transition structure. Policies which depend on these features will therefore be robust to changes in the environment brought about by interventions on spurious state variables. In certain linear function approximation settings, we demonstrate that this method will, with high probability, learn an optimal state abstraction that generalizes correctly to novel environments, requiring many fewer training environments than would be necessary without the block MDP assumption. We then draw a connection between bisimulation and the minimal causal set of variables found by our algorithm, providing bounds on the model error and sample complexity of the method. We further show that using analogous invariant prediction methods for the nonlinear function approximation setting can yield improved generalization performance over multi-task and single-task baselines. 

\section{Background on causality}

A desirable property of a learned representation is that it should enable a policy trained on one environment to generalize well to new environments. A promising approach to learn these general representations, particularly in the partially observable setting, comes from training an agent to learn the {causal} structure of the environment \citep{zhang2019causal, de2019causal}. We refer to Section~\ref{sec:background:stateabstractions} for a discussion of state abstractions. This section will provide the relevant background on causality; the two concepts will be related in Section~\ref{sec:causal-abs}. Causal inference \citep{pearl2000causality}  is a powerful tool that machine learning researchers are beginning to leverage \citep{johansson2016learning, louizos2017causal}, developing explicit connections to generalization bounds and robustness \citep{shalit2017estimating}. While a great deal of work on causality focuses on learning an explicit causal graph \citep{de2019causal}, an alternative approach involving invariant prediction  \citep{Magliacane2018, peters2016causal, rojas2018invariant} is in fact closely aligned with the literature on generalization bounds discussed in Chapter~\ref{chp:background}. 
\subsection{Causal discovery}
\label{sec:causal_inf}
Causal inference concerns itself with identifying causal relationships from data \citep{pearl2000causality}. The central object of study is a Structural Causal Model (SCM), which characterizes the data generating distribution as a set of functions on observed and hidden variables. 
\begin{definition}[\citep{pearl2000causality}]
A structural causal model is a tuple  $(\mathbf{U}, \mathbf{V}, \mathcal{F}, P )$, where $\mathbf{U}$ is a set of exogenous variables (e.g. the unobserved source of stochasticity in the environment) drawn from the distribution $P$, $\mathbf{V}$ is a set of endogenous variables (e.g. the observed state $s$, the reward $r$, and the action $a$ in RL), and $\mathcal{F}$ is the set of functions $f_V: \mathrm{PA}_V \times \mathbf{U} \rightarrow \mathbf{V}$ with $\mathrm{PA}_V \subset \mathbf{V}$, which determine the value of endogenous variable $V$ for each $V \in \mathbf{V}$. 
\end{definition}
An SCM has a representation as a directed acyclic graph, called a causal Bayesian network or simply a causal graph, whose nodes are the variables in $\mathbf{U} \cup \mathbf{V}$ and whose edges are given by $\{(p, v) \mid p \in \mathrm{PA}_V, V \in \mathbf{V} \}$. An edge from variable $V_1$ to $V_2$ indicates that $V_1$ is a causal parent of $V_2$. This mapping is not one-to-one. In general, many SCMs may induce identical causal graphs.
Causal models enable reasoning about how changes to the data-generating process, called \textit{interventions}, affect the resulting distribution over variables. 

\begin{definition}[Do-Intervention]
A do-intervention on a variable $V$, denoted do($V=v$), in a causal model $S=(\mathbf{U}, \mathbf{V}, \mathcal{F}, P )$ is an operation that induces a new SCM $S'=(\mathbf{U}, \mathbf{V}, \mathcal{F}', P )$, where $\mathcal{F}' = \{ f_{W} \in \mathcal{F} \mid W \neq V \} \cup \{ f_{V=v} \}$ and $f_{V=v}(\mathbf{p}, \mathbf{u}) = v \; \forall \mathbf{p} \in \mathrm{PA}_V, \mathbf{u} \in U$.
\end{definition}
The randomized procedure used to assign patients to the treatment and control groups in medical trials is an example of an intervention. Predicting the effect of a do-intervention requires identifying the direction of the edges in the causal graph, a process known as \textit{causal discovery}. Many recent works have explored how to incorporate causal structure into deep RL by learning disentangled representations of the observation space \citep{bengio2017independently, suter2019robustly}.

\subsection{Invariant prediction}
A fundamental property of causal relationships is their invariance under interventions. \textit{Invariant prediction} seeks to identify causal structure from data by evaluating whether the relationships between variables in a prediction problem are invariant across different environments, where each environment corresponds to a different intervention on the data generating process. Relationships which are invariant over these interventions are assumed to be causal, and by applying a straightforward hypothesis-testing approach \citep{peters2016causal}, one can obtain high-probability guarantees on identifying the causal parents of a target variable of interest.

The approach of \citet{peters2016causal} is limited to problems in which the inputs consist of sets of variables whose relationship can be described by a causal graph. A more adaptable framework based on the same principle of invariance has emerged recently:  invariant risk minimization (IRM) \citep{arjovsky2019invariant}. This setting assumes that the learner has access to data which is partitioned into different environments, and that these environments share the same causal structure. Differences in the distribution of data in each environment are attributed to non-causal correlations, which the learner should avoid using for prediction. Unlike invariant causal prediction methods, there is no explicit assumption on the structure of the inputs as corresponding to distinct variables in the causal model. Instead, the learner seeks to identify a non-linear feature map which induces a linear classifier whose errors are invariant over the training environments. The mathematical formalism of this objective yields the minimization problem
\begin{equation}
\min_{\phi} \sum_{e \in \mathcal{E}}  \ell (w^\top \phi(X_e), Y_e) \text{ s.t. } \text{argmin}_{v} \ell(v^\top \phi(X_e), Y_e) = w , \; \forall e \in \mathcal{E}\;.
\end{equation}

The invariant risk minimization framework provides a compelling intuition, but the IRM objective proposed by \citet{arjovsky2019invariant} requires extensive hyperparameter tuning to obtain competitive results, and in many cases fails to capture the invariances motivating its proposal \citep{-kamath2021does}. In spite of these limitations, several works have extended the ideas behind IRM to a variety of settings \citep{ahuja2020invariant, alesiani2021continual}, and a notion of invariance across environments has recently been deployed by \citet{raileanu2021decoupling} to improve generalization to novel environments on the ProcGen suite.

\section{Problem setting}
\label{sec:problem_setup}
We consider a family of environments $\mathcal{M}_\mathcal{E} = \{(\mathcal{X}_e, \mathcal{A}, \mathcal{R}_e, \mathcal{P}_e, \gamma) \mid \; e \in \envs\}$, where $\mathcal{E}$ is some index set. For simplicity of notation, we drop the subscript $e$ when referring to the union over all environments $\mathcal{E}$, e.g. $\mathcal{X} = \cup_{e \in \mathcal{E}} \mathcal{X}_e$. Our goal is to use a subset $\envs_{\text{train}} \subset \envs$ of these environments to learn a policy $\pi$ which generalizes to {every} environment. Concretely, we seek a parameterized policy $\pi$ which maximizes the quantity

\begin{equation}
 \mathbb{E}_{\mathcal{E}_{\test}, \pi}[ \sum_{t=0}^\infty \gamma^t R(x_t, a_t)] \; .
\end{equation}

We denote the number of training environments by $N=|\envs_{\text{train}}|$. We assume that the environments share some structure, and consider different degrees to which this structure may be shared in this section. As we have alluded to previously, it is this presumed structure which will enable significant sample efficiency gains over more naive approaches.

\subsection{Block MDPs}
Block MDPs~\citep{du2019pcid} are described by a tuple $\langle \mathcal{S}, \mathcal{A}, \mathcal{X}, p, q, R \rangle$ with a finite, unobservable state space $\mathcal{S}$, a finite action space $\mathcal{A}$, and a possibly infinite, but observable space $\mathcal{X}$. Here $p$ denotes the latent transition distribution $p(s'|s,a)$ for $s,s'\in\mathcal{S}, a\in\mathcal{A}$, $q$ is the (possibly stochastic) emission function that generates observations from the latent state $q(x|s)$ for $x\in\mathcal{X}, s\in\mathcal{S}$, and $R$ denotes the reward function.
\begin{assumption}[Block structure~\citep{du2019pcid}]
\label{asmp:block}
Each observation $x$ uniquely determines its generating state $s$. That is, the observation space $\mathcal{X}$ can be partitioned into disjoint blocks $\mathcal{X}_s$, each containing the support of the conditional distribution $q(\cdot|s)$.
\end{assumption}
This assumption gives us the Markov property in $\mathcal{X}$, and is crucial for our empirical and theoretical results. \citet{zhang2020learning} discuss the partially observable setting. 

The definition of a block MDP given by \citet{du2019pcid} defines a single environment; our primary interest in this chapter, however is in a multi-environment setting. We will therefore concern ourselves with \textit{block MDP families}, which characterize a set of environments with shared dynamics. We define a block MDP family as a collection of environments $\mathcal{M}_{e}$, where each environment $e$ corresponds to an emission function $q_e$. Each environment $\mathcal{M}_e$ thus has the form $\langle \mathcal{S}, \mathcal{A}, \mathcal{X}, p, q_e, R \rangle$, where all terms except the emission function $q_e$ are shared between environments. We will move the potential randomness from $q_e$ into an auxiliary variable $\eta \in \Omega$, where $\Omega$ is some probability space, and write $q_e(s, \eta)$. Crucially, we enforce Assumption~\ref{asmp:block} on the union of the emission functions $q_e$. This entails that whenever $\text{range}(q_e (s, \cdot)) \cap \text{range}(q_{e'}(s', \cdot)) \neq \emptyset$, then $s = s'$. The decomposition of the state into a latent state $s$ and noise term $\eta$, and the corresponding causal structure, can be seen in Figure~\ref{fig:irm_model_irrelevant}.

Without additional assumptions on shared structure between the emission functions, generalization to novel environments can be an impossible task. We will therefore focus on settings where the $q_e$ overlap in structured ways -- for example, where $q_e$ outputs the concatenation of the noise and state variables, $q_e(s, \eta) = s \oplus f(\eta)$ -- such that it is possible to learn a feature map from observations to latent states that will generalize to new emission functions from this class. 

Our ultimate goal is to find a policy $\pi$ which maximizes cumulative reward over any novel test environment. This chapter will exclusively consider settings where the policy $\pi$ is a function of some state abstraction (also referred to as a feature map) $\phi(\mathcal{X})$. Under the block MDP assumption, if $\phi$ maps observations $\bx$ to the state $s$ that emitted them, then any policy which is optimal on the training environments will trivially be optimal on the test environments. As a result, most of our discussion in the coming sections will focus exclusively on the invariance of a learned state abstraction, with the ensuing optimality of the learned policy left implicit. Section~\ref{sec:relaxations} will characterize environments where learning such an invariant state abstraction is tractable.

\subsection{Relaxations}
\label{sec:relaxations}
\begin{figure}
    \centering
    \includegraphics[width=0.49\linewidth]{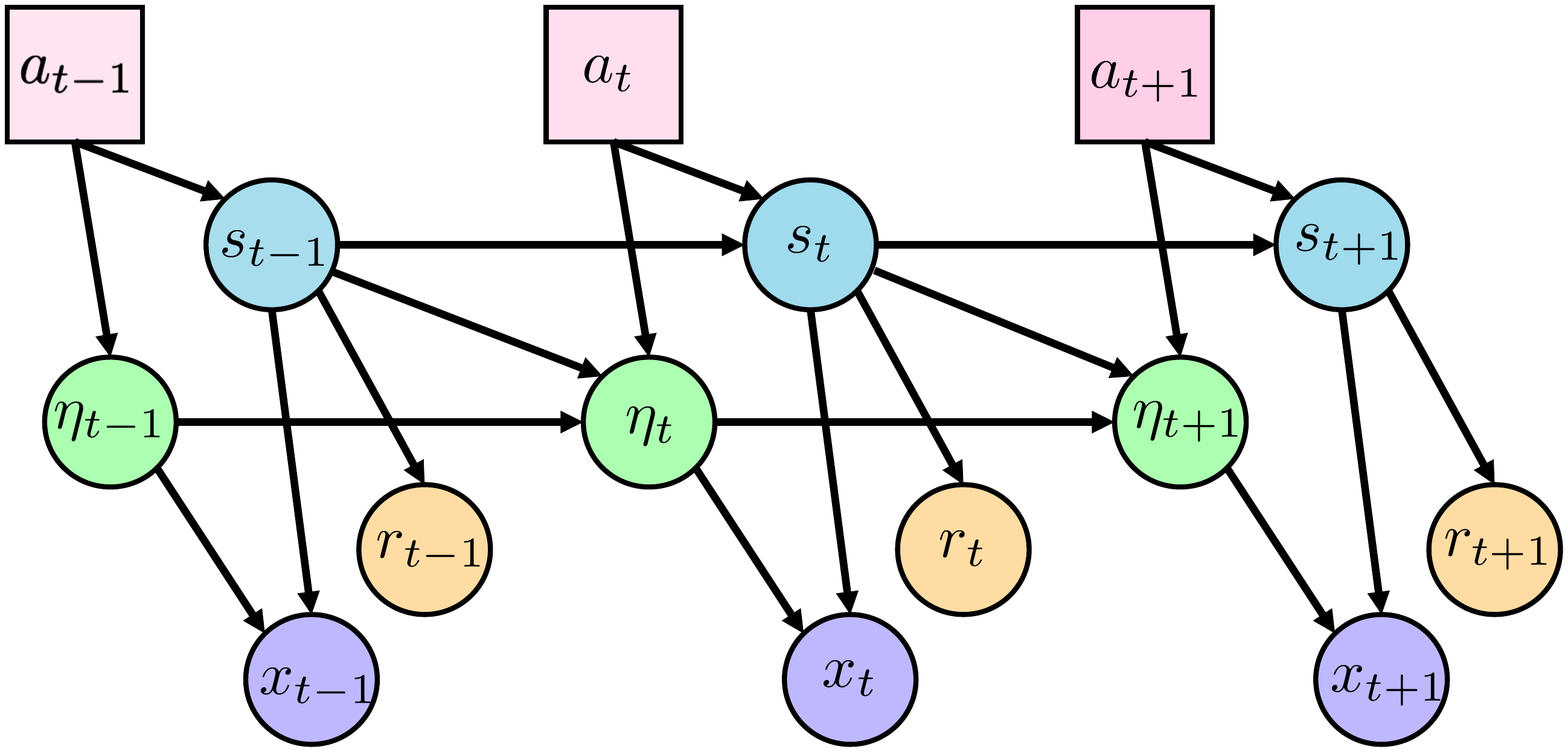}
    \includegraphics[width=0.49\linewidth]{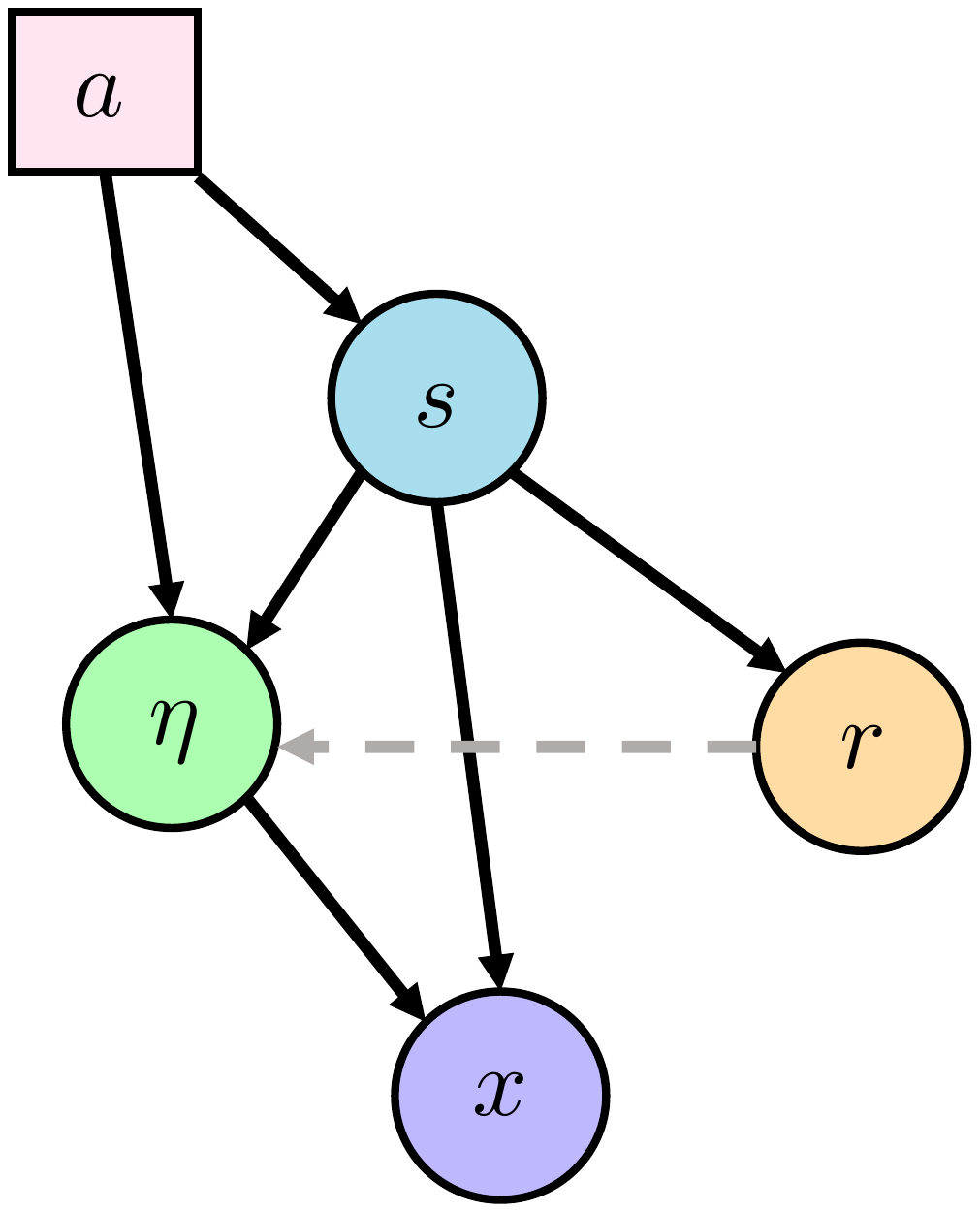}
    \caption[Graphical model of a block MDP with emissions $x$ from a state which can be decomposed into a causal component $s$ and a spurious component $\eta$.]{Graphical model of a block MDP with emissions $x$ from a state which can be decomposed into a causal component $s$ and a spurious component $\eta$. The left hand side illustrates the relationship between the variables at each time step, while the right hand side is a shorthand characterizing the same sequential distribution.}
    \label{fig:irm_model_irrelevant}
\end{figure}
Figure~\ref{fig:irm_model_irrelevant} illustrates the graphical model representing the transition dynamics of a Markov Decision Process under varying assumptions. This graphical model can be interpreted as a stationary causal graph, with arrows indicating causal dependence, as seen on the right hand side of the figure. Under this interpretation, different assumptions on the structure of the graph correspond to different classes of Markov Decision Process. We will be particularly interested in settings where different environments correspond to different correlation structure in the noise variables, as outlined below.

\textbf{Correlated noise variables.} The standard formulation of the block MDP assumes that the noise variable $\eta$ is sampled independently at random at every time step, which prevents multi-timestep correlations. We therefore also consider a more realistic \textit{relaxed block MDP}, where spurious variables may have different transition dynamics across the different environments so long as these dynamics are independent of the return. We introduce the notation $q_S$ to refer to the projection of the inverse mapping $q_e$ onto the invariant component of the latent state $s$, 
\begin{equation}\label{eq:inverse-emmision}
    q_S(x) = s \text{ where } \exists \eta, e \text{ s.t. } : q_e(s, \eta) = x
\end{equation}
and note that this will be invariant over environments. This formulation is equivalent to augmenting each MDP $\mathcal{M}_e \in \mdp$ with a noise variable $\eta_e$, such that for any observations $x = q_e(s, \eta_e)$, any observation $x' = q_e(s', \eta_e')$, and any action $a$, we have

\begin{equation*}p(x'|x, a) = p(s'|s, a) \int_{\eta_e':q(q_S(x'), \eta_e')=x'} p_e(\eta_e'|s, \eta_e) d \eta_e'\, . \end{equation*}

This dependency can be observed in the arrows between $\eta_{t-1}$ and $\eta_t$ in Figure~\ref{fig:irm_model_irrelevant}. We note that $p(s'|s,a)$ remains independent of $\eta$ (i.e. there is no edge between $\eta_t$ and $s_{t+1}$), while $\eta_{t+1}$ is permitted to depend on $s_t$ and $\eta_t$ (i.e. we permit edges between these nodes in the causal graph). The assumption that for each observation $x$ there is a unique generating latent state $s$ is crucial to ensure that the return will be independent of the variable $\eta$, and its absence results in a partially observable MDP. In principle, we might go further to relax the independence of $\eta_t$ on the previous timestep to also causally depend on the reward (the grey dashed line in Figure~\ref{fig:irm_model_irrelevant}); while the techniques shown in the subsequent sections can be applied in this setting, this chapter will focus on dependencies only between components of the state. 

\subsection{Assumptions on causal structure}

Importantly, under Assumption~\ref{asmp:block} or under the correlated noise variables assumption, the spurious variables $\eta_t$ are disjoint from the set of causal ancestors of the reward (i.e. the set of variables $X$ from which there is a directed path in $\mathcal{G}$ from $X$ to $r_t$ for some $t$). The suggestive notation of $s_t$ and $\eta_t$ to describe the environment-specific and environment-invariant components of the state is deliberate: we will show later that identifying the abstraction $(s_t, \eta_t) \rightarrow s_t$ is equivalent to a causal structure identification problem, and that this abstraction guarantees optimal generalization to new environments.

In order to obtain these results, we must enforce the Markov property on the MDP's causal graph. This goes beyond the requirement that the variables in the environment state at time $t$ can only affect the values of the state at time $(t+1)$, and can only affect the reward at time $t$, allowing us to consider the state and action at time $t$ as the only candidate for causal parents of the state at time $(t+1)$ and of the reward at time $t$.  We refer the reader to Figure~\ref{fig:statgm} to demonstrate how causal graphical models can be translated to this setting. 

\begin{assumption}[Independent Causal Mechanisms]\label{assmpt:causal_mechanisms}
Let $x^1$ and $x^2$ be components of the observation $\bx$. Let $X^i_{t}$ denote the value of $x^i$ at time $t$ for $i=1,2$. Then when no intervention is performed on the environment, we have the following independence,
\begin{equation}
     X^1_{t+1} \perp X^2_{t+1} \mid \bx_t \; .
\end{equation} 
\end{assumption}

The assumption of independent causal mechanisms is standard in the causal inference literature \citep{peters2017elements} but is stronger than the assumption of Markovian transition dynamics in the MDP. For example, an MDP whose state space consists of two variables $x_1$ and $x_2$ which are both set at each timestep by the same coin flip will have Markov transition dynamics, but its causal graph will not satisfy the Markov condition. Intuitively, Assumption~\ref{assmpt:causal_mechanisms} requires that the variables in the observation space reflect the environment's underlying data generating process. This assumption is crucial for our results in linear function approximation, though our results on the rich observation setting allow us to relax independence of causal mechanisms in the observation space.

While many of the most challenging technical problems in causal inference stem from more complex causal structures which do not satisfy the Markov property, these structures often arise in the first place because the data collector is not able to see the step-by-step evolution of the system. Put simply: an RL agent can observe that rain always occurs {before} a person opens their umbrella, while a supervised predictor will have access to only an instantaneous snapshot of the world with binary values of rain and umbrella. As a result, RL agents have access to much richer temporal information about the data-generating distribution, and this removes many, but not all, of the ambiguities that causal inference methods seek to resolve.

We have thus far required that the dynamics of the latent states in the environment be independent of the spurious variables; however, the spurious variables may still correlate with the return to an extent that a function approximator may use them to predict the value function or to select an action. For example, the spurious variables may be a copy of the latent state at the previous time step. In order to identify the relevant and irrelevant components of the state, the agent must receive suitable information from the training environments. In this chapter we will assume that the training environments are generated by interventions on the data-generating process shown on the right hand side of Figure~\ref{fig:irm_model_irrelevant}. 

\begin{assumption}[Environment Interventions]\label{assmpt:envs}
Let $\mathcal{X} = X_1 \times \dots \times X_n$, and $\mathcal{S} = X_{i_1} \times \dots X_{i_k}$. Each environment $e \in \mathcal{E}$ corresponds to a do-intervention or soft intervention \citep{eberhardt2007interventions} on a single variable $x_i$ in the observation space. 
\end{assumption}

We note that in order for the training environments to constitute a block MDP family, these interventions may only be applied to the spurious variables in the state if they are applied over multiple timesteps. An intervention on the causal variables -- for example, fixing the angle of a joint in a robotic simulator -- will change the dynamics on the shared latent states $\mathcal{S}$. While such interventions may provide valuable information about the structure of the world, they will also change the value function that the agent is trying to predict. Interventions which only influence the initial state distribution do not encounter this issue, but face the limitation of providing only a single transition where the target variable takes the desired value. As a result, we assume that the test environments will only contain interventions on spurious variables, but allow for potential interventions on causal ancestors of the reward in the training environments.

\section{State abstractions and causal feature sets}
\label{sec:causal-abs}
We begin our analysis with a simplified setting: we assume that the observation space $\states$ of the block MDP is a direct sum of variables $X_1, \dots,  X_n$, and that each of these variables can be represented as a node in a causal graph corresponding to the dynamics of the MDP which satisfies Assumption~\ref{assmpt:causal_mechanisms}. Let $I \subset [n]$ be an index set of the variables which correspond to the latent state $s$ in the block MDP. Then the state can be decomposed into $\bx = (\bx_I) \oplus (\bx_{I^C}) = s \oplus \eta$, where $s = \bx_I$ denotes the latent state and $\eta = \bx_{I^C}$ denotes the spurious variables. This setting admits a natural formulation of invariant causal prediction in order to identify the state abstraction $\phi(\bx) = \bx_I$. 
It is also straightforward to alternate between the causal graph formulation of the MDP and the block MDP formalism, noting that the emission function $q_e$ for an environment $e$ will simply be of the form $q_e(s, \eta) = s \oplus \eta = \bx_I \oplus \bx_{I^C}$. We can further show that the state abstraction $\bx \mapsto \bx_I$ is a model-irrelevance state abstraction (c.f. Definition~\ref{def:misa}) for all $\mdp_{e} \in \mdp$.

\begin{theorem}[Existence of model-irrelevance state abstractions]\label{thm:existence}
Let $\mdp_\mathcal{E}$ denote a block MDP family with joint observation space $\mathcal{X}_\mathcal{E} = \cup_{e \in \mathcal{E}} X_e$, where the observation space $X_e$ is the image of some emission function $q_e(s,\eta)$ s.t. $s \in \mathcal{S}, \eta \in \Omega$ as before. Let $f_e = q_S \mid_{\mathcal{X}_e}$ as defined in \eqref{eq:inverse-emmision}. Then $\phi = \cup_{e \in \mathcal{E}} f_e$ is a model-irrelevance state abstraction for each $\mdp_{e} \in \mdp_{\mathcal{E}}$.
\end{theorem}
\begin{proof}
First, note that $\cup_{e \in \mathcal{E}} f_e$ is well-defined by the block MDP assumption. By assumption we have that the reward function satisfies $R(x) = R(q_S(x))=R(\phi(x))$ for all observations $x$ and environments $e$. We further have that for any $x_1, x_2 $ s.t. $\phi(x_1) = \phi(x_2) = s$, and for any $s'$:
\begin{align}
    \int_{x' \in \phi^{-1}(s')}P(x'|x_1, a) &= P(s'|s,a) \int_{x' \in \phi^{-1}(s')} \int_{\eta': q_e(s', \eta) = x'} P(\eta'|s, \eta_1) \\
    &= P(s'|s,a)  \int_{x' \in \phi^{-1}(s')} \int_{\eta': q_e(s', \eta) = x'} P(\eta'|s, \eta_2) \\
    &=  \int_{x' \in \phi^{-1}(s')}P(x'|x_2, a) 
\end{align}
Noting that $\phi( q_e(s, \eta') ) = s$ and $\phi^{-1}(s) = \{ q_e(s, \eta) \mid \eta \in \Omega\}$, we obtain the desired result.
\end{proof}
 
We now consider whether, under Assumptions \ref{asmp:block}, \ref{assmpt:causal_mechanisms}, and \ref{assmpt:envs}, a model-irrelevance state abstraction can be obtained by causal inference methods. Intuitively, one would then expect that the {causal variables} (in this case, the causal \textit{ancestors} of the return, rather than the parents of the prediction target as in the case of regression problems) should have nice properties as a state abstraction: in particular, they should enable generalization to new environments. 
In what follows, we will generalize our definition of a model-irrelevance state abstractions to refer to a block MDP family for which the dynamics of the induced abstract MDP on each environment in the family are equivalent.

\begin{definition}\label{def:block-misa}
A model-irrelevance state abstraction over a block MDP family $\mdp_{\envs}$ is one for which the following two conditions hold. First, $\phi$ is consistent with the reward function $R$, i.e.
\begin{equation}
    R(\phi(\bx)) = R(\phi(\by)) \quad \forall \bx, \by 
    \in \cup_{e \in \envs} \mathcal{X}_e \; .
\end{equation}
Further, for all $e_1, e_2 \in \envs, \bx_1 \in e_1, \bx_2 \in e_2, \by \in \mathcal{X}_{e_1} \cup \mathcal{X}_{e_2}$ such that $\phi(\bx_1) = \phi(\bx_2) $ we have:
\begin{align}
\sum_{\bx' \in  \phi^{-1}(\by)} P_{e_1}(\bx'|\bx_1) = \sum_{\bx' \in  \phi^{-1}(\by)} P_{e_2}(\bx'|\bx_2)
\end{align}
\end{definition}
The following result highlights the connection between the causal ancestor set and model irrelevance: a state abstraction that selects the set of causal variables from the observation space of a block MDP will be a model-irrelevance abstraction for every environment $e \in \mathcal{E}$ in this stronger sense. 
\begin{figure}
    \centering
    \includegraphics[trim=180 80 180 100,clip,width=0.32 \textwidth ] {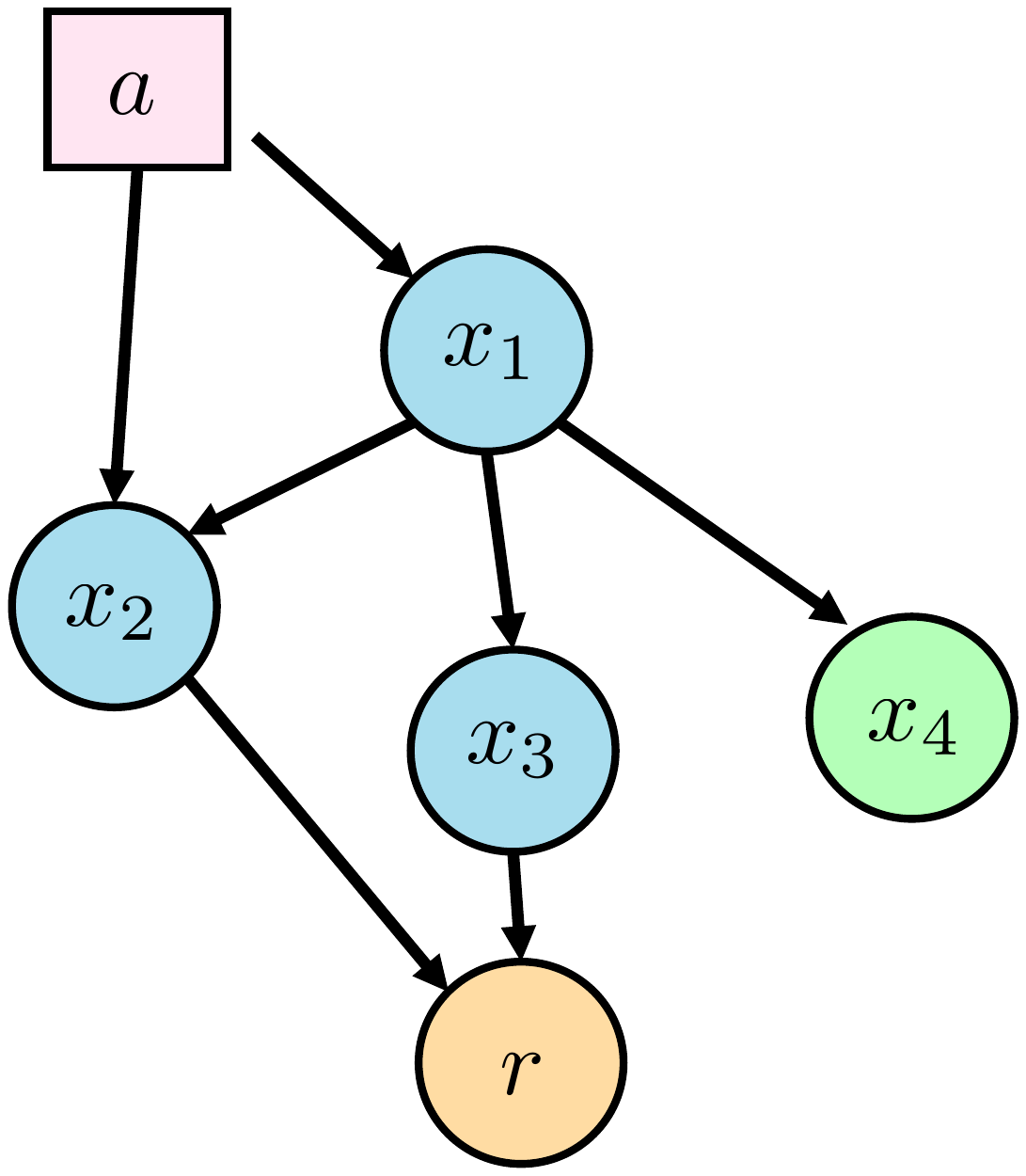}
    \includegraphics[trim=180 80 180 100,clip,width=0.32 \textwidth ] {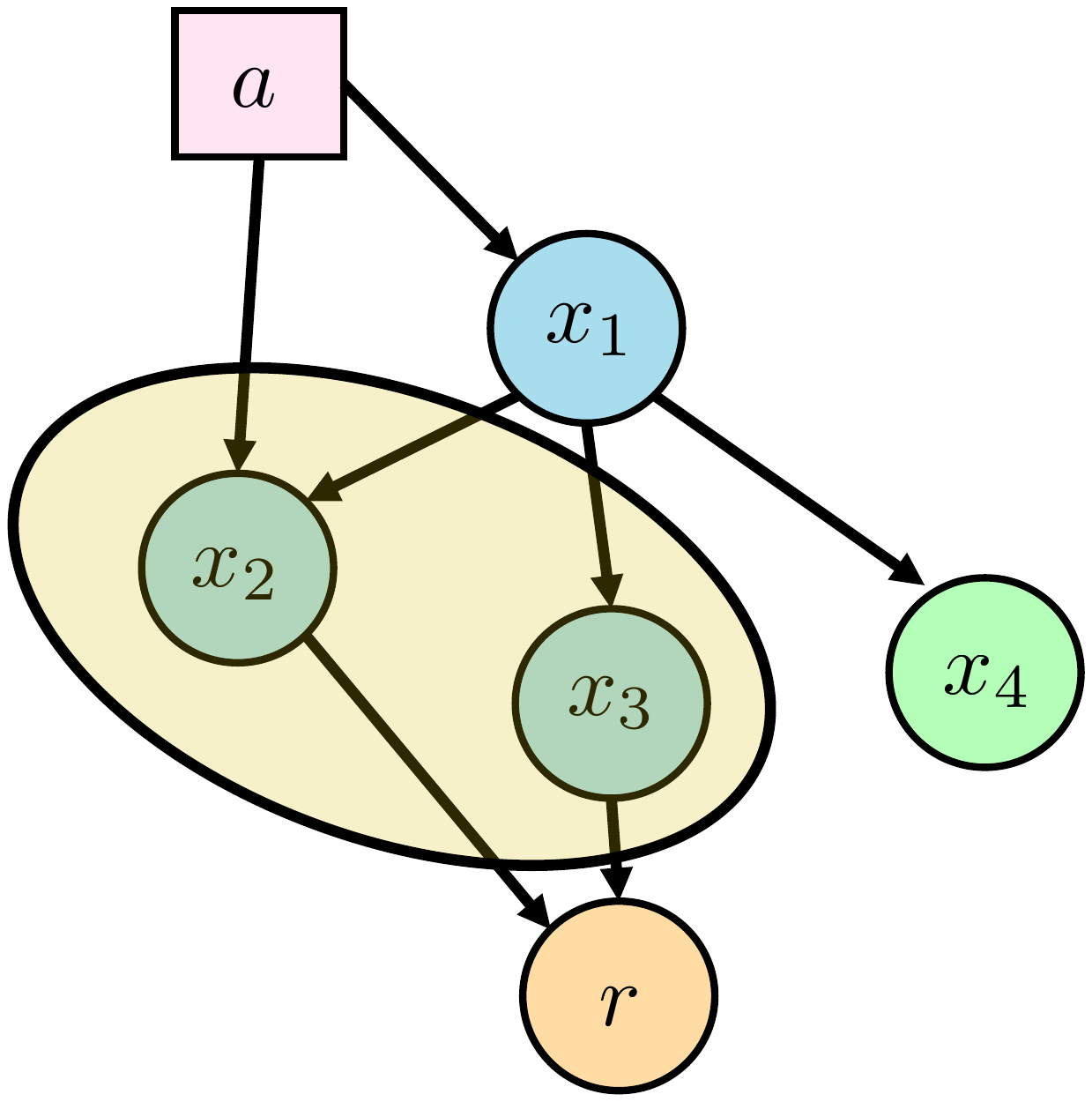}
    \includegraphics[trim=180 80 180 100,clip,width=0.32 \textwidth ] {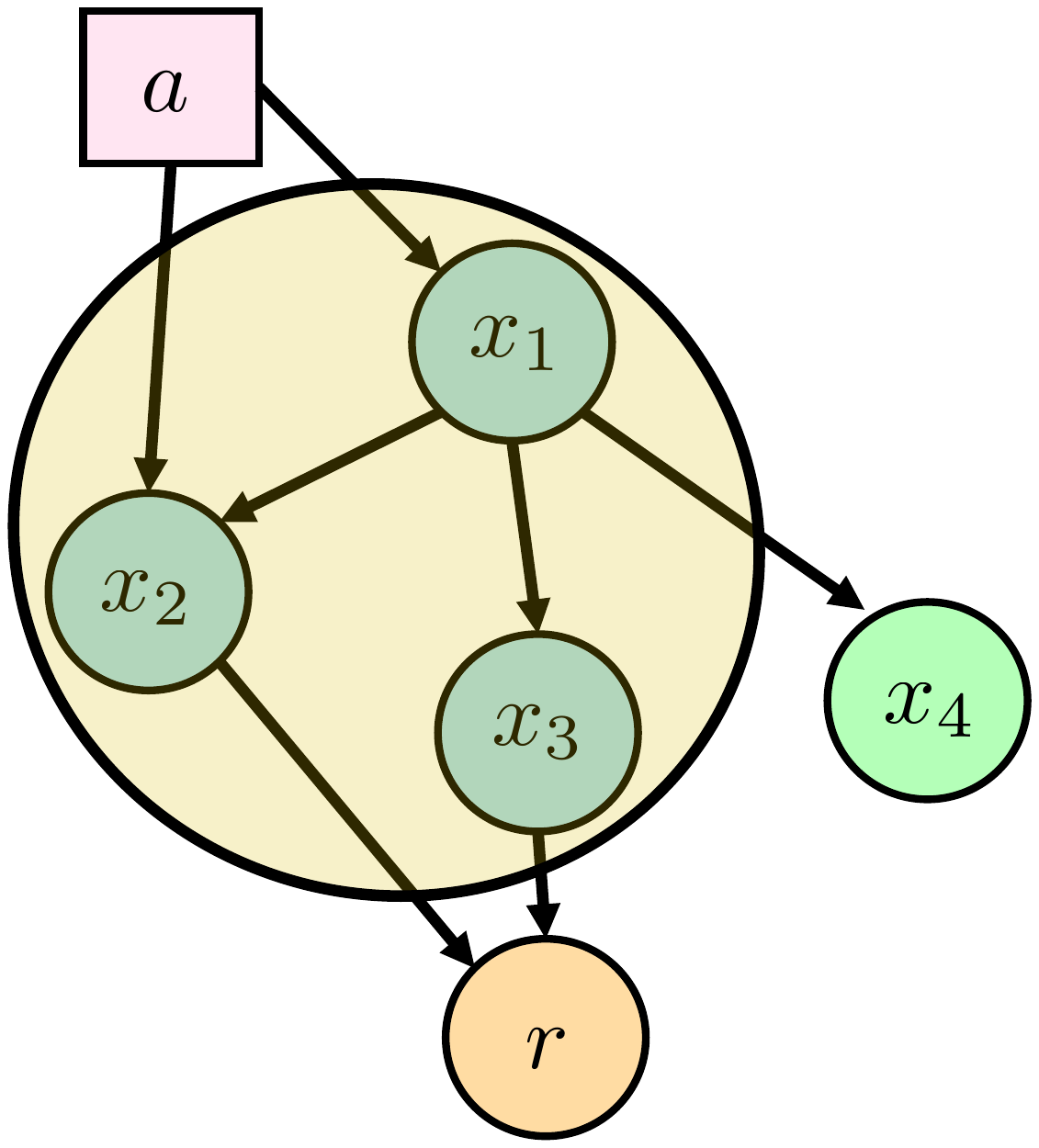}
    \caption[Illustration of a causal state abstraction where the observation is a concatenation of variables from a causal graph.]{Illustrations of a causal state abstraction where the observation is a concatenation of variables from a causal graph. The variables $x_{1:3}$ present the causal ancestors of the reward, and as per Theorem~\ref{thm:causalstate_modelirrelevance} can be interpreted as the latent state $s$ from Figure~\ref{fig:irm_model_irrelevant}. The state $x_4$ thus corresponds to the spurious variable $\eta$.}
    \label{fig:statgm}
\end{figure}

\begin{theorem} 
\label{thm:causalstate_modelirrelevance}
Consider a family of MDPs $\mdp_\mathcal{E} = \{(\mathcal{X}, A, \mathcal{R}, P_e, \gamma)|e \in \mathcal{E} \}$, with $\mathcal{X} = X_1 \times \dots \times X_k$ where each $X_i, i \in 1, \dots, k$ is some collection of sets, for which there exists a SCM $\mathcal{C} = ((x^1_t, \dots, x^k_t, A_t, R_t)_{t=1}^\infty, \{ \eta_1, \dots, \eta_k\}, \mathcal{F}, P)$ such that each $\mathcal{M}_e \in \mdp_{\mathcal{E}}$ corresponds to an intervention on a single variable $X_i$ in $\mathcal{C}$, where $X_i$ is not a causal ancestor of the reward. Let $M_\mathcal{E}$ satisfy Assumptions \ref{asmp:block}-\ref{assmpt:envs}. Let $S = \textbf{AN}(R)$ denote the ancestors of the reward $R$ in $\mathcal{C}$. Then the state abstraction $\phi_S(x) = [x]_S = (x_i \mid i \in S)$  is a \textit{model-irrelevance} abstraction over the block MDP family $\{\mdp_e \in \mathcal{E} \}$. 
\end{theorem}

\begin{proof}
This result follows as a special case of Theorem~\ref{thm:existence} after showing that $\mdp_{\mathcal{E}}$ is a block MDP family. To do so, we construct each component of such a block MDP family as follows: we set $\mathcal{S} = (X_i)_{i \in S}$, $\mathcal{A}=A$, $R = \mathcal{R}$ and $\mathcal{X}=\mathcal{X}$. We define the transition operator $p = P_{S}$, where $P_S$ denotes the marginal transition probability $P(X^{t+1}_S \mid X^{t}_S)$ which is well-defined as all causal parents of nodes in $X_S$ are contained in $X_S$. The emission function $q$ is defined on the joint space $X_S$ and $X_{S^C}$ via concatenation and re-ordering such that $q(s, \eta_e) \rightarrow x$ whenever $s = [x]_S$ and $\eta_e = [x]_{S^C}$. Crucially, the criterion that each environment correspond to an intervention requires that $\eta_e$ evolve over time according to the dynamics induced by $P_e$ for each environment $e$. However, because $S$ is closed under taking causal parents and because Assumptions 1 and 2 apply, this correlation of the noise variable $\eta_e^t$ over time will not interfere with the dynamics on $\mathcal{S}$. We thus obtain a (relaxed) block MDP family which precisely corresponds to $\mdp_{\mathcal{E}}$. This then implies that the map $\phi(x) = [x]_S$ is a model-irrelevance state abstraction.
\end{proof}

An important detail in the previous result is that the model-irrelevance state abstraction incorporates not just the parents of the reward, but also its ancestors. This is because in RL, we seek to model \textit{return} rather than solely rewards, and a variable which is a causal parent of the return may exert influence over several timesteps. We provide an illustration of such a state abstraction in the rightmost column of Figure~\ref{fig:statgm}. As a concrete example from the CartPole environment, only the position $x$ and pole angle $\theta$ are necessary to predict the reward. However, predicting the return requires knowledge of $\dot{\theta}$ and $\dot{x}$, their respective velocities.

\keyinsight{To obtain a state abstraction which will be a model-irrelevance abstraction on all environments $e$ in a block MDP, it suffices to identify the variables which are causal ancestors of the reward in the MDP's causal model. }

Strictly speaking, other subsets of variables which include the set of causal ancestors $S$ will be a model-irrelevance abstraction for each individual environment. What makes the set of causal ancestors $S$ special is that it applies to all environments in the block MDP. This property is crucial for generalization: under $\phi: \bx \rightarrow [\bx]_S$, we have that the union over environment transition operators acting on the latent space $\phi(\states)$ is well-defined. In other words, $\phi$ maps each environment to an abstract MDP whose dynamics and state space are identical across the block MDP family.

Under this definition a model-irrelevance state abstraction must remove all spurious variables from the input. Doing so is a non-trivial task, and requires observing the effects of their variation. It is not in general possible to identify spurious variables from only a single environment, particularly when spurious variables are highly correlated with their causal counterparts. The following proposition highlights the importance of Assumption~\ref{assmpt:envs} to the {identifiability} of the causal feature set. This result will require either that we have at least one environment corresponding to an intervention on each variable in the causal graph, or that we have a particular class of intervention applied to all variables simultaneously \citep{peters2016causal}.
\begin{proposition}[Identifiability of causal state abstractions]\label{prop:identifiability}
Consider a block MDP family whose corresponding SCM and set of training environments $\envs$ satisfy the conditions of Theorem 2 of \citet{peters2016causal}. Then the causal feature set $\phi_S$ is identifiable and corresponds to a model-irrelevance state abstraction for all test environments consisting of interventions on non-ancestors of the reward. Conversely, if these conditions are not satisfied then there may exist a state abstraction $\phi : \cup_{e \in \envs} \statespace_e \rightarrow \bar{\statespace}$ such that $\phi$ is a model-irrelevance abstraction over $\mathcal{E}_{\text{train}}$, but not over $\mathcal{E}$ globally.
\end{proposition}

\begin{proof}
The proof of the first statement follows immediately from the iterative application of the identifiability result of \citet{peters2016causal} to each variable in the causal variables set. These conditions require that we also show the agent interventions on the \textit{causal} variables; the resulting training environments will not necessarily satisfy the block MDP assumption, but may still be used for identification. Once a causal variable set has been identified, it is straightforward to use this variable set to train a value function on environments in the block MDP family. Alternatively, it is possible to show that for environments satisfying Assumption~\ref{assmpt:causal_mechanisms}, so long as there is a do-intervention on each variable in $\textbf{AN}(R)^C$ in the training set $\envs$, we preserve identifiability of the causal variable set. Constructing this set requires a slight modification of the aggregation procedure of \citet{peters2016causal} to return the largest set $S$ whose induced predictor is invariant.

For the converse, we consider a simple counterexample in which one non-ancestor of the return $\bx_m$ is constant in every training environment with value $v_m$, where $\bx = (\bx_i)_{i=1}^d \in \calX$. Then letting $S = \textbf{AN}(R)$, we observe that ${\phi = \bx \mapsto \bx_{S \cup \{m\}} }$ is also a model-irrelevance state abstraction for each training environment $\mdp_e$. However, it is straightforward to show that if $x_m$ is perturbed by a Bernoulli noise variable $Z \sim \mathrm{Ber}(0.5)$ in a new test environment, then these transition dynamics will not satisfy the conditions for $\phi$ to be a model-irrelevance state abstraction in the new environment. Indeed, we now have that the state abstraction $\phi$ violates the second condition of Definition~\ref{def:block-misa}. In particular, let $\bx$ be a state occurring in both a training environment $e_{\mathrm{train}}$ and the test environment $e_{\mathrm{test}}$. Let $\by \in \mathcal{X}_{e_{\mathrm{train}}}$ be such that $\by_m = v_m$ , $P_{e_{\mathrm{train}}}(\by \mid \bx) > 0$, and $P_{e_{\mathrm{test}}}(\by \mid \bx) > 0$. Then 
\begin{equation}
    \sum_{\bx' : \phi(\bx') = \phi(\by)}P_{e_{\mathrm{train}}}(\bx'|\bx) 
    = \sum_{\bx' : \phi(\bx') = \phi(\by)} 2 P_{e_{\mathrm{test}}}(\bx'|\bx) \; .
\end{equation} 
Because $P_{e_\mathrm{test}}$ spreads probability mass equally over states where $\bx_m = v_m$ and $\bx_m = v_m + 1$. This violates our requirement that the union over transition operators $P_{e_{\mathrm{train}}}$ and $P_{e_{\mathrm{test}}}$ yield a well-defined function.
\end{proof}

\subsection{Variable selection for linear predictors}
We have shown that identifying a model-irrelevance state abstraction that will generalize to any environment consistent with the block MDP is possible by identifying the causal feature set. The following \textit{Linear MISA} algorithm (Algorithm~\ref{alg:linear_misa}) provides a concrete approach to do this. We require the presence of a replay buffer $\mathcal{D}$, in which transitions are stored and tagged with the environment from which they came. The algorithm then applies ICP to find all causal ancestors of the reward iteratively. This approach has the benefit of inheriting many desirable properties from ICP -- under suitable identifiability conditions, it will return the exact causal variable set to a specified degree of confidence. 

It also inherits inconvenient properties: the ICP algorithm is exponential in the number of variables, and so this method is not efficient for high-dimensional observation spaces. We are further restricted to considering linear relationships of the observation to the reward and next state. Additionally, because we take the union over iterative applications of ICP, the confidence parameter $\alpha$ used in each call must be adjusted accordingly. Given $n$ observation variables, we use a conservative value of $\frac{\alpha}{n}$ to get an overall confidence of $\alpha$ from the procedure.

\begin{algorithm}[t]
\SetAlgoLined
\hspace*{\algorithmicindent} \textbf{Input:}  $\alpha$, a confidence parameter; $\mathcal{D}$, a replay buffer with observations $\mathcal{X}$. \\
\hspace*{\algorithmicindent} \textbf{Output:} {$S \subset \{1, \dots, k\}$, the causal state variables}. \\
 $S \gets \emptyset$\;
 stack $\gets$ r \;
 \While{stack is not empty}{
  $v$ = stack.pop() \;
  \If{$v \not \in S$}{
  $S' \gets$ \texttt{ICP}(v, $\mathcal{D}$, $\frac{\alpha}{\text{dim}(\mathcal{X})}$) \;
  $S \gets S \; \cup S'$\;
   stack.push($S'$)}
 
 }
\caption{Linear MISA: Model-irrelevance State Abstractions}
\label{alg:linear_misa}
\end{algorithm}

\subsubsection{Evaluation}
\label{sec:model_linear}
We illustrate the benefits of explicitly removing spurious features from the state input in Figure~\ref{fig:icp_result}. We consider a simple family of MDPs with state space $\mathcal{X} = \{ (x_1, x_2, x_3)\}$, with a transition dynamics structure such that $x_1^{t+1} = x_1^t + \epsilon_1^e$, $x_2^{t+1} = x_2^t + \epsilon_2^e$, and $x_3^{t+1} = x_2^t + \epsilon_3^e$. The reward $r(x)$ is set to be a linear function  of $x_1$ and $x_2$. We train on 3 environments with soft interventions on each noise variable. We run the linear MISA algorithm on batch data from these 3 environments to get a state abstraction $\phi(x) = \{x_1, x_2\}$, then train 2 linear predictors on $\phi(x)$ and $x$. We evaluate the {generalization error} on novel environments that correspond to different hard interventions on the value of the $x_3$ variable. We observe that the predictor trained on $\phi(x)$ attains zero generalization error because it zeros out $x_3$ automatically. However, any nonzero weight on $x_3$ in the least-squares predictor will lead to arbitrarily large generalization error, which is precisely what we observe in Figure \ref{fig:icp_result}.
\begin{figure}
    \centering
    \includegraphics[width=0.49\textwidth]{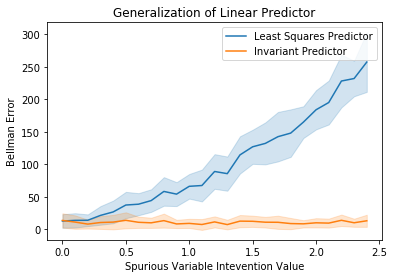}
    \caption[The presence of spurious {uncorrelated} variables in the state can still lead to poor generalization of linear function approximation methods.]{The presence of spurious {uncorrelated} variables in the state can still lead to poor generalization of linear function approximation methods. Figure shows Bellman error of linear function approximators in the presence of interventions. The linear model which is only given access to the invariant features obtained by Linear MISA is robust to these interventions, while the model trained on the entire input space exhibits increasing error as the intervention value grows. }
    \label{fig:icp_result}
\end{figure}

\section{State abstractions for rich observations}

The variable selection problem has been an enlightening regime in which to ground our intuition about the relationship between invariant prediction and generalization. However, most RL problems of interest involve more complex observation spaces where simply identifying an input mask is not sufficient to guarantee generalization to new environments. This section leverages the intuition discussed previously to present a method motivated by invariant prediction which learns approximate model-irrelevance state abstractions in rich observation settings.

\subsection{Block MDP generalization bounds}

Our goal in this section is to \textit{learn} representations that will generalize from the training environments to a novel test environment, as opposed to identifying input masks as we did in the previous section. 
However, normal PAC generalization bounds require a much larger number of environments than one could expect to obtain in the reinforcement learning setting. The appeal of an invariant representation is that it may allow for theoretical guarantees on learning the right state abstraction with many fewer training environments, as discussed by \citet{peters2016causal}. If the learned state abstraction is close to capturing the true latent dynamics of the block MDP family, then the model error, and by extension performance, in the test environment can be bounded by a function of the distance between the test environment's abstract state distribution and that of the training environments. In particular, if the state abstraction is {invariant} and always maps equivalent states in any pair of environments to the same latent representation, then the learned model error on the test environment will be equal to that on the training environments.

A more interesting case arises when the state abstraction $\phi$ is {approximately invariant}: such a state abstraction might map equivalent states from the training environment $M$ and test environment $M'$ to similar but non-identical feature vectors. For simplicity we will consider only deterministic environments, but similar reasoning can be applied to stochastic MDPs. In this case, a looser requirement on $\phi$ can still yield reasonable generalization error bounds for a learned transition dynamics model. Intuitively, the requirement we set is that the training environment dynamics function $T_M$ be Lipschitz with respect to $\phi$, and that the union of the training and test environment dynamics $T_M \cup T_{M'}$ be well-defined and Lipschitz. This is analogous to the requirement on the induced dynamics of a model-irrelevance state abstraction: if states from $M$ and $M'$ map to the same feature vector, then the successor states to each must also map to identical feature vectors. The Lipschitz requirement introduces a geometric notion of \textit{similarity}: if states from $M$ and $M'$ map to nearby feature vectors, then their successors must also be relatively close in feature space. With this requirement, it is possible to bound the generalization error of a learned latent dynamics model trained on $M$ in terms of the Lipschitz constant of $T_{M} \cup \T_{M'}$ and the Wasserstein distance between the state visitation distributions of the two environments in feature space. Intuitively, the Lipschitz constant quantifies the smoothness the environment dynamics with respect to $\phi$, and the Wasserstein distance between the state visitation distributions of the two environments in feature space, which measures how well $\phi$ has been able to discard irrelevant information. The following result assumes a fixed policy and a state-valued input to the transition function, but can be easily adapted to the control setting.

\begin{restatable}{theorem}{thmModelErr}[Model error bound]
\label{thm:model_error}
Let $M = \langle \states, A, R, P, \gamma \rangle$ and $M' = \langle \states', A, R', P', \gamma \rangle$ be two environments from a block MDP family, and let $\phi: \states \rightarrow \mathbb{R}^d$ denote a model-irrelevance state abstraction that maps states from both MDPs to feature vectors. Set $T_M: \states  \rightarrow \states$ as the (deterministic) transition function of $M$.
Suppose that the union of the dynamics of $M$ and $M'$ are $L$-Lipschitz with respect to the embedding $\phi$ and that $T:\mathbb{R}^d  \rightarrow \mathbb{R}^d$ is some approximate transition model satisfying ${\max_{s} \mathbb{E}\|T(\phi(x)) - \phi(T_M(x)) \| < \delta}$, for some $\delta > 0$. Let $W_1(\pi_1, \pi_2)$ denote the 1-Wasserstein distance between distributions $\pi_1, \pi_2$ on $\mathbb{R}^d$. Then
\begin{equation}
    \mathbb{E}_{x \sim M'}[\|T(\phi(x)) - \phi(T_{M'}(x)) \|] \leq \delta + 2LW_1(\pi_{\phi(M)}, \pi_{\phi(M')}).
\end{equation}
\end{restatable}
\begin{proof}
The proof follows from a straightforward decomposition of the error into three components: one which depends on the similarity between each input $x \in X_{M'}$ and the nearest neighbour $y$ in $X_{M}$, the model error on $y$, and the similarity between the representations of the successor states $T_M(x)$ and $T_M(y)$:
\begin{align*}
  & \mathbb{E}_{x \sim M'} [\|T(\phi(x)) - \phi(T_{M'}(x)) \|] \\&= \mathbb{E}_{x \sim M'} \bigg[\min_{y \in X_M} \|T(\phi(x)) - T(\phi(y)) + T(\phi(y)) -\phi( T_{M}(y)) + \phi(T_M(y)) - \phi(T_{M'}(x)) \| \bigg ] \\
    &\leq \mathbb{E}_{x \sim M'} \bigg [\min_{y \in X_M} \|T(\phi(x)) - T(\phi(y))\| + \|T(\phi(y)) - \phi(T_M(y))\|+\| \phi(T_M(y)) - \phi(T_{M'}(x))\| \bigg ]\\
    \intertext{Letting $\gamma$ be a coupling over the distributions of $\phi(\states')$ and $\phi(\states)$ which minimizes the objective of the Wasserstein distance, i.e. ${\mathbb{E}_{\gamma(\phi(x),\phi(y))} \|\phi(x)-\phi(y)\|= W_1(\pi, \pi')}$}
    &\leq \mathbb{E}_{\gamma(\phi(x), \phi(y))} \bigg [\|T(\phi(x)) - T(\phi(y))\| + \delta + L\|x-y\| \bigg ]\\
    &\leq \mathbb{E}_{\gamma(\phi(x), \phi(y))} \bigg [L\|\phi(x) - \phi(y)\|+ \delta + L\|\phi(x)-\phi(y)\| \bigg ] \\
    &=  \mathbb{E}_{\gamma(\phi(x),\phi(y))}\bigg[L\|\phi(x) - \phi(y)\|+ \delta + L\|\phi(x)-\phi(y)\| \bigg ] \\
    &= 2LW_1(\pi, \pi') + \delta
\end{align*}

\end{proof}

Analogous results can be obtained for value error \citep{zhang2020invariant}, which provide a bound on generalization performance that depends on the supremum of the dynamics and reward errors obtained by a learned model, along with the degree of invariance captured by the representation. The former corresponds to the empirical risk in a PAC bound, while the latter corresponds to the \textit{generalization gap}. A state abstraction which perfectly captures the equivalence classes of the block MDP family will attain error on new environments equal to that of its error on the training environments and exhibit no generalization gap. Meanwhile, a state abstraction equal to the identity map can induce arbitrarily large errors on test environments.

\subsection{Learning a model-irrelevance state abstraction}

\label{sec:model_irrelevant}

It is clearly desirable to learn an invariant representation across environments. How to do so is less clear. In this section, we propose one such approach based on a model-learning objective, but note that the principles guiding this method can easily be applied to other representation-learning methods. Our approach aims to learn a representation ${\phi : \mathcal{X} \rightarrow \mathbb{R}^d}$ which preserves the reward- and dynamics-relevant information of the state, while discarding environment-specific information in settings where such information is encoded non-linearly in rich observations. Ideally, a representation which discards environment-specific information in the training environments should also do so in new test environments, though guaranteeing this property is more challenging than in the variable selection setting. 

We now present an objective to learn a dynamics preserving state abstraction $\mathcal{Z}$. This requires disentangling the state space into a minimal representation that causally relates to the reward $s_t:=\phi(x_t)$ and all other features of the observation $\eta_t:=\varphi(x_t)$. 

We decompose our learning objective into two parts:
\begin{itemize}
    \item \textbf{Invariance}: we train a {\color{PineGreen}\textbf{task classifier}} on the shared latent representation ${C:\mathcal{Z}\mapsto [0,1]^N}$ with cross-entropy loss and employ an adversarial loss~\citep{tzeng2017ada} on $\phi$ to maximize the entropy of the classifier output to ensure task specific information is not passing through to $\mathcal{Z}$.
    \item \textbf{Dynamics preservation:} we train a reward model and a latent dynamics model on the state abstraction $\varphi$ to encourage the preservation this information. 
    \subitem The {\color{BrickRed}\textbf{reward model}} is an MLP $R$ trained to predict the sample reward $r$ given input transition $(\phi(\bx), a, \phi(\bx'))$, using the objective
    \begin{equation*}
    J_R(\phi,R)=\sum_i \mathbb{E}_{\pi_{b_i}}\big[(R(\phi(x_i), a,\phi(x_i')) - r_i')^2\big]\;.
\end{equation*}
    \subitem The {\color{Cerulean} \textbf{transition model}} is an MLP $f_s$  acting on the latent space $\mathcal{Z}$. The transition model is trained via a reconstruction loss, using a learned decoder $\phi^{-1}$ which takes as input a task-dependent state embedding $\psi(\bx)$ along with the latent-space model output $f_s(a, \phi(\bx))$, and outputs a predicted next observation. 
    \begin{align*}
    J_D(\phi,\psi, f_s,f_\eta)=\sum_i \mathbb{E}_{\pi_{b_i}}\big[&(\phi^{-1} (f_s(a,\phi(x_i)),
    f_\eta(a,\psi(x_i))) - x_i')^2\big],
\end{align*}
\end{itemize}


\noindent This gives us a final objective
\begin{equation}
    J_{\text{ALL}}(\phi, \psi, f_s, f_\eta, r) = {\color{Cerulean}J_D(\phi,\psi, f_s,f_\eta)} + {\color{BrickRed}\alpha_R J_R(\phi,r)} - {\color{PineGreen}\alpha_C H(C(\phi))},
\end{equation}
where $\alpha_R$ and $\alpha_C$ are hyperparameters and $H$ denotes entropy (Algorithm~\ref{alg:nonlinear_misa}). 

\begin{algorithm}[t]
\SetCustomAlgoRuledWidth{0.6\textwidth}
\SetAlgoLined
\hspace*{\algorithmicindent} \textbf{Input:} $\mathcal{E}$, a set of environments; $\pi_0$, an initial policy; $\phi_0$, $f_0$ an initial invariant feature map and transition model;  $\{\psi_0^e \mid e \in \mathcal{E} \}$ and $f^e_{\eta, 0}$, initial task-specific feature maps and transition models.\\
\hspace*{\algorithmicindent} \textbf{Output:} \text{$\phi$, an invariant state encoder.} \\
 $\pi \gets \pi_0$\;
 $\phi, f_s \gets \phi_0, f_{s,0}$ \;
 $\psi^e, f_\eta^e \gets \psi^e_0,f_{\eta,0}^e $ for $e \in \mathcal{E}$ \;
 $\mathcal{D}_e \gets \emptyset$ for $e \in \mathcal{E}$ \;
 \While{forever}
 {
    \For{$e \in \mathcal{E}$}
    {
        $a \gets \pi(x_e)$\;
        $x'_e, r \gets $ \texttt{step}$(x_e,a)$ \;
        \texttt{store}$(x_e,a,r,x'_e)$ \;
    }
    \For{$e \in \mathcal{E}$}
    {
    Sample batch $X_e$ from $\mathcal{D}_e$ \;
    $f_\eta^e,\psi^e \gets \nabla_{f_{\eta}^e,\psi^e} [J_D(X_e) ]$ \;
    }
    $f_s, \phi, r \gets \sum_{X_e}\nabla_{f_s,\phi} [J_{\text{ALL}}(X_e) ]$\;
    $C \gets \nabla_C$ \texttt{CE\_loss}$(C(\phi(\{x_e\}_{e\in\mathcal{E}}), \{e\}_{e\in\mathcal{E}})$ \; 
 }
 \caption{Nonlinear Model-irrelevance State Abstraction (MISA) Learning}
 \label{alg:nonlinear_misa}
\end{algorithm}

\subsection{Results}
We first evaluate MISA in a rich-observation model-learning task to evaluate whether the learned representation and abstract model do indeed generalize to complex dynamics structures where the factors of variation correspond to non-linear functions of the observation. We next look to imitation learning (Section~\ref{sec:imitation_learning}), where we show that the state abstractions learned by MISA allow the agent to zero-shot generalize to observations generated by new camera angles. Finally, we explore end-to-end reinforcement learning in the low-dimensional observation setting with correlated noise (Section~\ref{sec:reinforcement_learning}) and again show generalization capabilities where single task and multi-task methods fail.

\subsubsection{Model learning: rich observation setting}
\label{sec:model_nonlinear}
We evaluate the gradient-based MISA method (Algorithm~\ref{alg:nonlinear_misa}) in a setting with nonlinear  dynamics and rich observations. We set the agent the task of learning a model which can generalize to new background colours on pixel-valued control tasks. We randomly initialize the background colour of two training environments from Deepmind Control~\citep{deepmindcontrolsuite2018} from the range $[0, 255]^3$. We also randomly initialize another two backgrounds for evaluation. The orange line in Figure~\ref{fig:imitation_learning} shows performance on the evaluation environments in comparison to three baselines. In the first, we train on a single environment and test on another with our method, (\texttt{MISA - 1 env}). Without more than a single experiment to observe at training time, there is no way to identify whether or not the background can safely be ignored, and indeed we observe that the method appears to include information about the background in its representation. In the second baseline, we combine data from the two environments and train a model over all data (\texttt{Baseline - 1 decoder}), but without explicitly encouraging invariance across the two environments. The third is another invariance-based method which uses a gradient penalty, IRM~\citep{arjovsky2019invariant}. In the second case the error is tempered by seeing variance in the two environments at training time, but it is not as effective as MISA with two environments. In the case of IRM, the loss starts much higher but very slowly decreases, and we find it is very brittle to tune in practice. Implementation details are deferred to Appendix~\ref{app:model_nonlinear_implementation}.

\begin{figure}
    \centering
    \includegraphics[height=4.85cm]{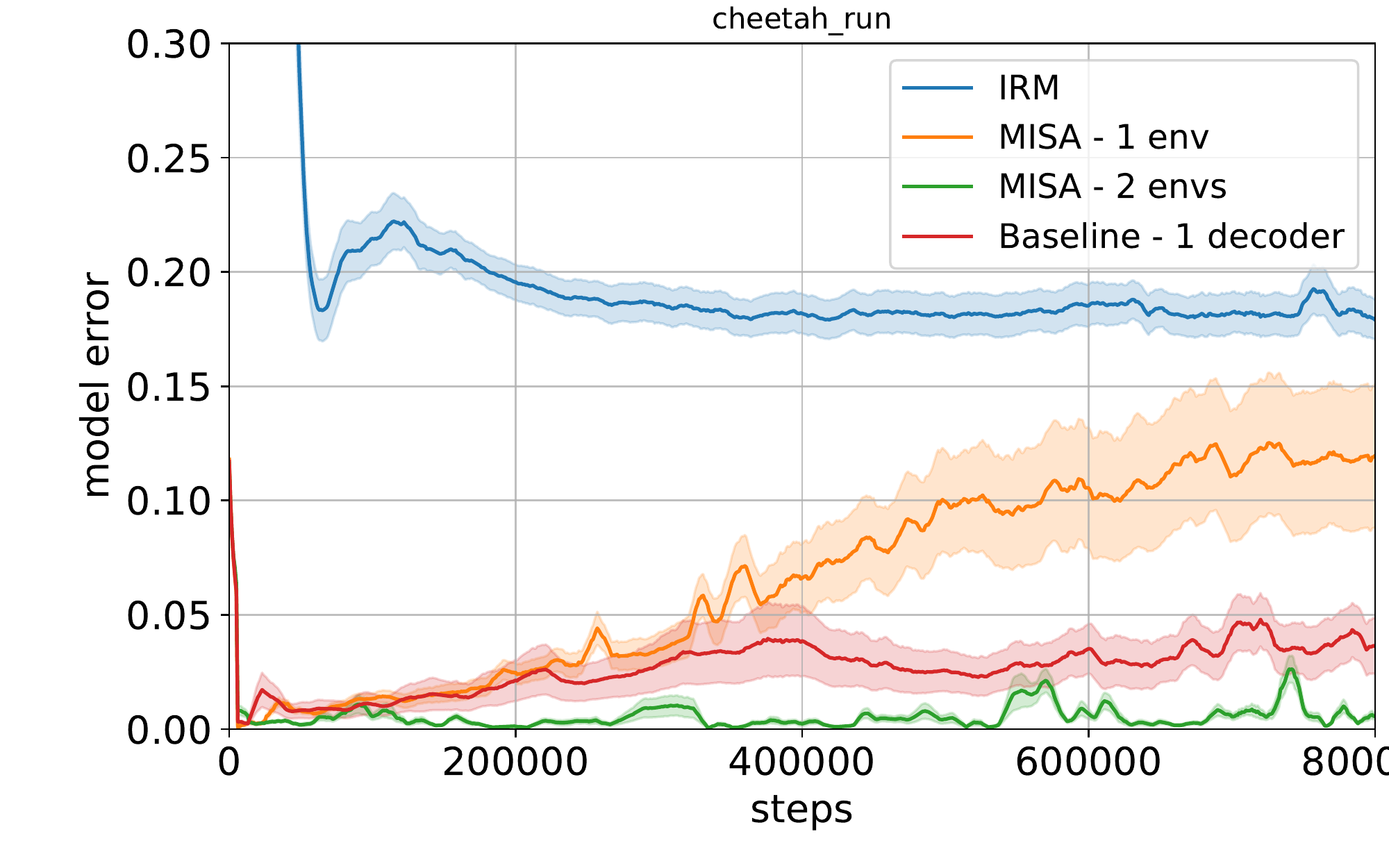}
    \hfill 
    \includegraphics[height=5.0cm]{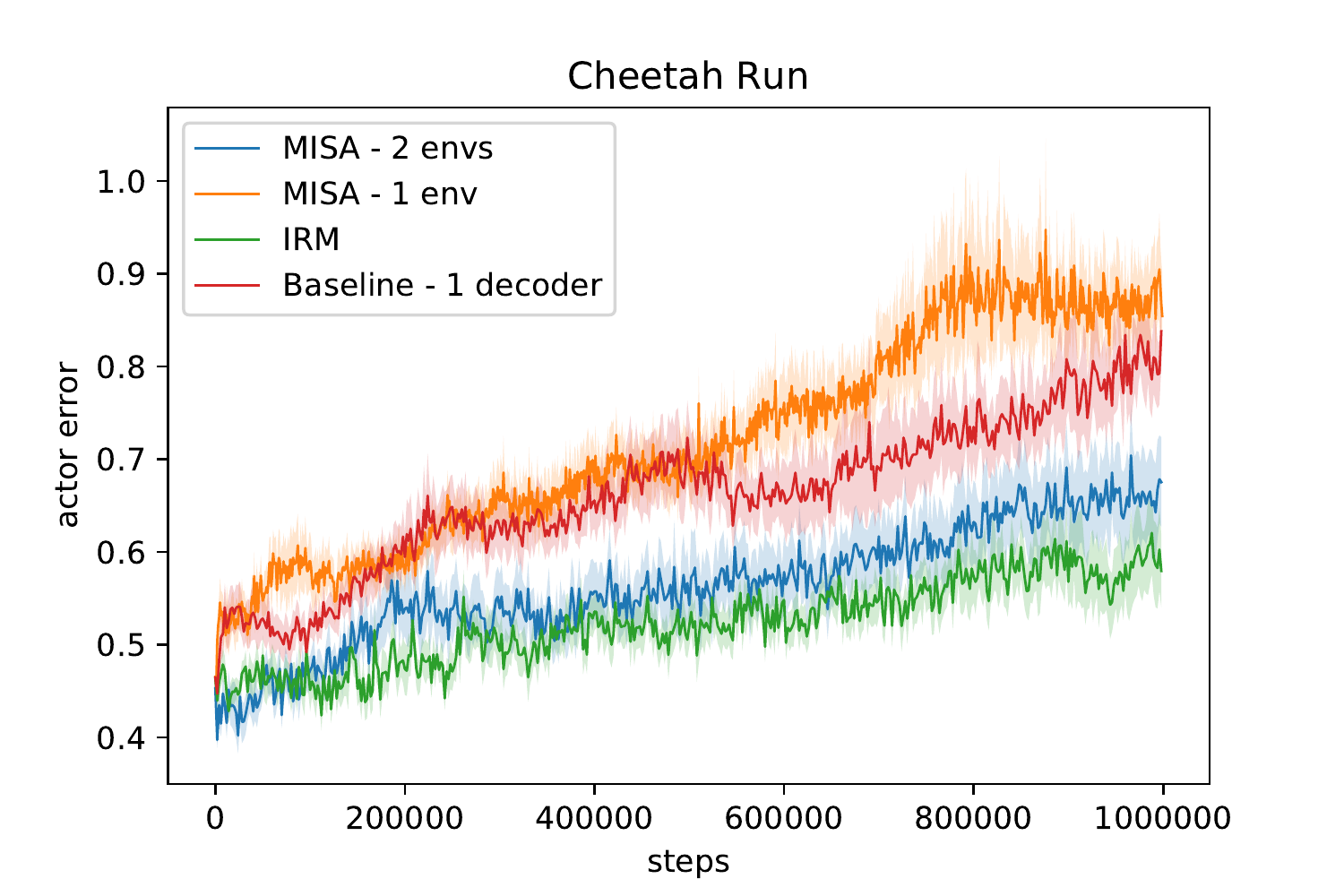}
    \caption[Evaluation of our nonlinear MISA algorithm.]{Evaluations over 10 seeds, with one standard error shaded. Left: Model error on evaluation environments on Cheetah Run from Deepmind Control.  Right: Actor error on evaluation environments on Cheetah Run from Deepmind Control.}
    \label{fig:imitation_learning}
\end{figure}

\subsection{Imitation learning}
\label{sec:imitation_learning}
Model-learning is often a useful auxiliary task in RL problems, but ultimately our the objective in RL is to learn an effective policy. We now evaluate how well a policy trained on top of the invariant representation found by MISA is able to generalize to new environments. We focus on imitation learning, as this setting involves fewer moving parts than an online learning setting, while preserving the policy-learning challenge of the RL problem. In this setup, we first train an expert policy using the proprioceptive state of Cheetah Run from the DeepMind Control suite \citep{deepmindcontrolsuite2018}. We then use this policy to collect a dataset for imitation learning in each of two training environments. When rendering these low dimensional images, we alter the camera angles in the different environments (Figure~\ref{fig:imitation_learning_envs}). We report the generalization performance as the test error when predicting actions in Figure~\ref{fig:imitation_learning}. Model error on the evaluation environment increases significantly over the course of training in all baselines; while we see that test error does increase with MISA as well, the error growth is significantly slower compared to single task and multi-task baselines. 

\begin{figure}

    \centering
    \includegraphics[height=2.2cm]{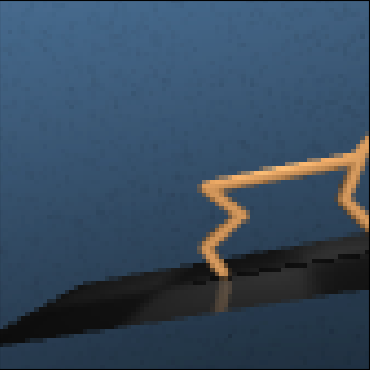} \hspace{20pt}
    \includegraphics[height=2.2cm]{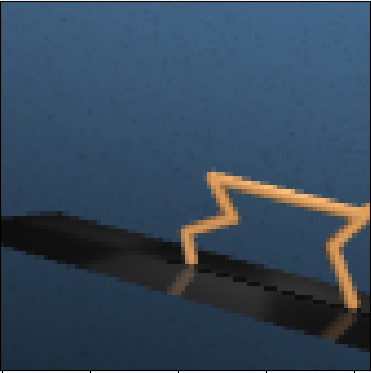} \hspace{20pt}
    \includegraphics[height=2.2cm]{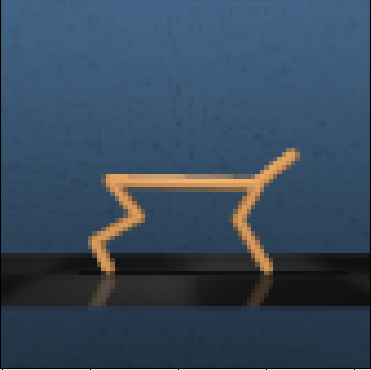}
    \caption{The Cheetah Run environment from Deepmind Control with different camera angles. The first two images are from the training environments and the last image is from evaluation environment.}
    \label{fig:imitation_learning_envs}

\end{figure}

\subsection{Reinforcement learning}
\label{sec:reinforcement_learning}
Finally, we evaluate Algorithm~\ref{alg:nonlinear_misa} on the full RL loop. This presents a greater challenge than was observed in the offline learning settings, as the invariant prediction component of the model must now face the possibility that the learned policy will visit different state distributions in different environments, particularly early in training. We therefore go back to the proprioceptive state in the \texttt{cartpole\_swingup} environment in DeepMind Control~\citep{deepmindcontrolsuite2018} to show that we can in some cases still learn a model-irrelevance state abstraction while training a policy. We use the Soft Actor Critic (SAC) algorithm~\citep{haarnoja2018sac} with an additional linear encoder, and add spurious correlated dimensions which are a multiplicative factor of the original state space. We also add an additional environment identifier to the observation. This factor varies across environments. We train on two environments with multiplicative factors $1\times$ and $2\times$, and test on $3\times$.  Like \citet{arjovsky2019invariant}, we also incorporate noise on the causal state to make the task harder, specifically Gaussian noise $\mathcal{N}(0, 0.01)$ to the true state dimension. This incentivizes the agent to attend to the spuriously correlated dimension instead, which has no noise. 
In Figure~\ref{fig:cartpole_swingup_rl} we see the generalization gap drastically improves with our method in comparison to training SAC with data over all environments in aggregate and with the IRM objective~\citep{arjovsky2019invariant} implemented on the critic loss. Intriguingly, we find that even without noise on the ground truth states, with only two environments, baseline SAC fails (Figure~\ref{fig:cartpole_swingup_rl}). Implementation details and more information about SAC can be found in Appendix~\ref{app:rl_implementation}. 

\begin{figure}
    \centering
    \includegraphics[height=5.5cm]{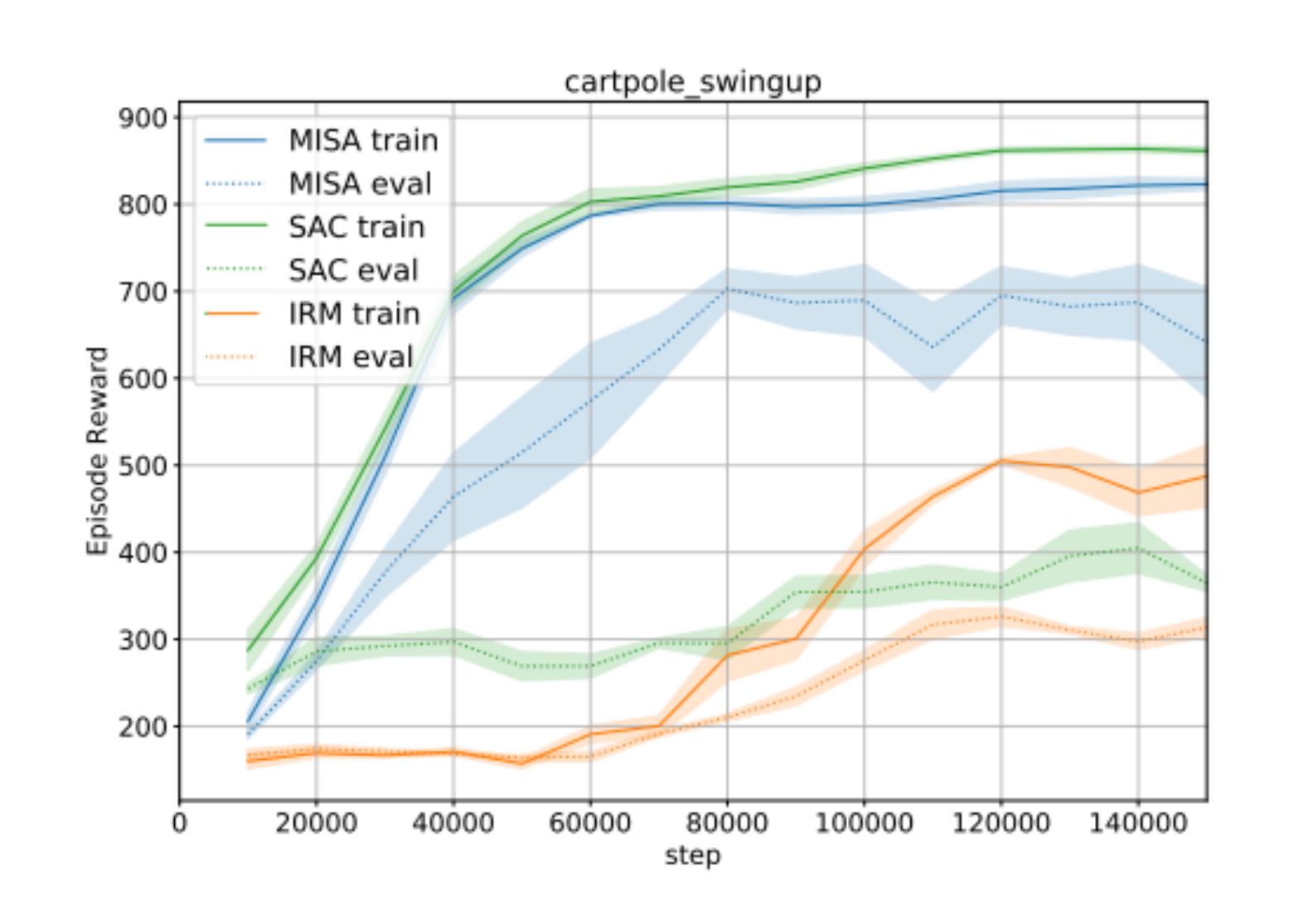} \hfill 
    \includegraphics[height=4.95cm]{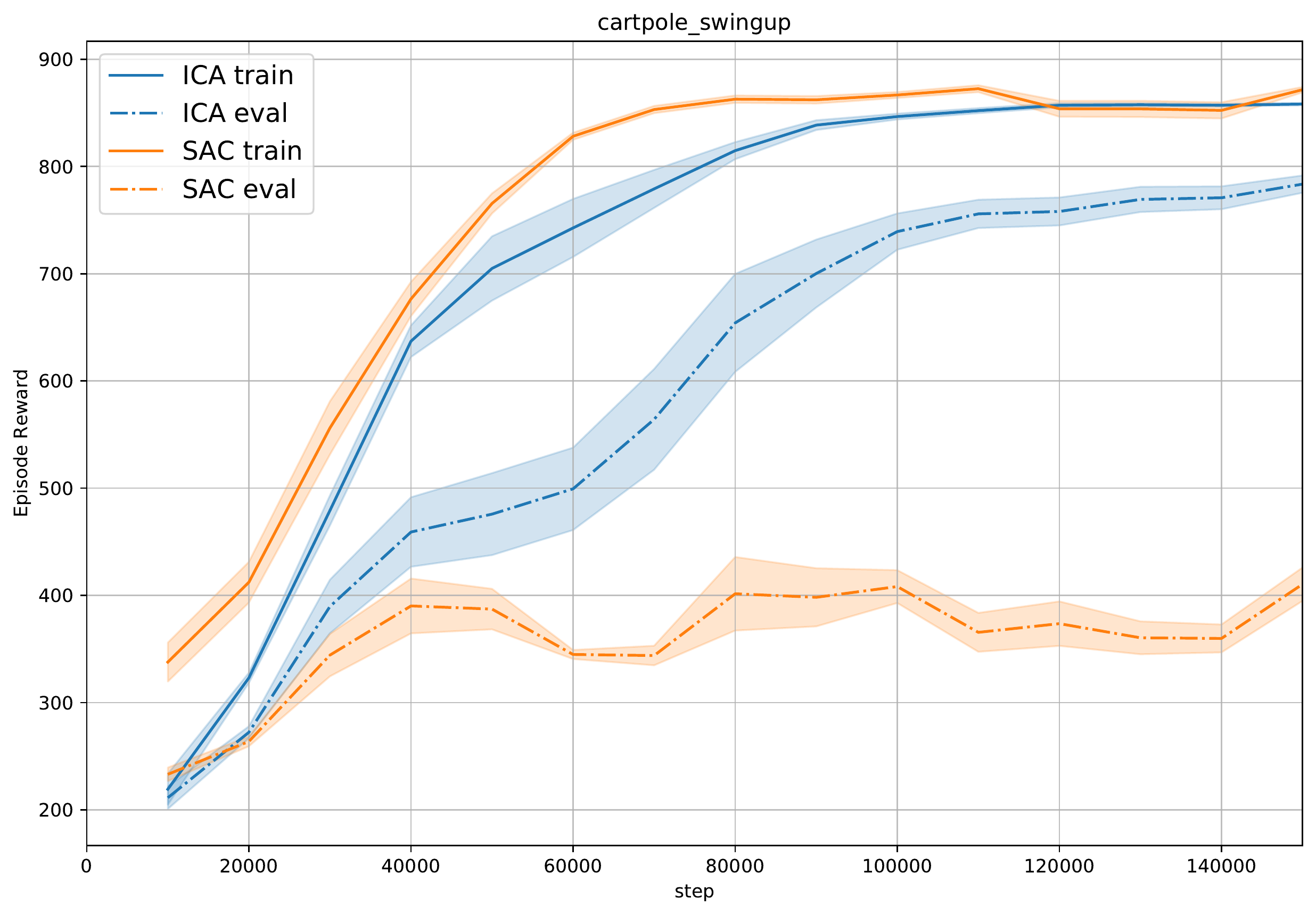}
    
    \caption[Performance of nonlinear MISA on RL tasks.]{Left: generalization gap in SAC performance with 2 training environments on \texttt{cartpole\_swingup} from DMC compared against two invariant prediction methods. Both experiments are evaluated with 10 seeds, standard error shaded. Right: generalization gap comparison between SAC and nonlinear MISA.}
    \label{fig:cartpole_swingup_rl}

\end{figure}

\section{Conclusions}
The findings presented in this chapter illustrate the importance of leveraging explicit assumptions on the causal structure of the environment for sample-efficient multi-environment generalization. We have shown that the application of invariant prediction methods in RL allows agents to learn policies that robustly generalize across environments with a shared causal structure after experiencing only a handful of training environments. We have further provided methods for both the low dimensional linear value function approximation setting and the deep RL setting which leverage invariant prediction to extract a causal representation of the state.

The block MDP family we defined in this chapter provides an explicit characterization of a rich class of MDPs in which generalization from a training task to a novel environment is tractable; however, promising future directions include extending these results to apply to an even broader class of problems. The proposed method in this chapter has not solved the generic problem of generalization to new environments in RL, but it has clearly demonstrated the benefits of applying the principle of invariant prediction to identify features that are useful for generalization in this expressive family of environments. Future work on this problem setting may find variations on invariant prediction which are more robust and which can handle the non-stationary rich-observation input distributions found in many deep RL problems. It may also extend these ideas to related problems including third-person imitation learning and sim-to-real transfer. 

If the reader is to come away with one insight, it is this: just as representations which enable generalization between training inputs enable better generalization in the single task setting (as we saw in Chapters~\ref{chp:supervised}, \ref{chp:invariance}, and \ref{chp:gen-rl}), a representation which enables generalization between different training \textit{environments} is likely to do the same in the multi-environment setting. Translating invariant prediction methods to RL is non-trivial; in large part  this is due to the non-stationarities described in Chapters~\ref{chp:rl-dynamics} and \ref{chp:rep-learning}. However, our findings here present a striking proof-of-concept of the utility of these ideas in reinforcement learning.

\chapter{Conclusion}

It is easy to ensure that a deep neural network learns what we tell it to learn: doing so simply requires ensuring sufficient expressivity for the size of the training set at hand. It is an unsolved problem to ensure that it learns what we \textit{want} it to learn, particularly in the context of deep reinforcement learning, where what we want a network to learn may evolve over the course of training. This thesis has argued that understanding how a network learns, by studying the properties of the learning trajectory taken over the course of training, is crucial to ensure that the learned function captures the underlying structure of the data it was trained on. We have applied this principle to a vast array of application areas, from Bayesian model selection to neural architecture search, from reinforcement learning to computer vision problems. Throughout these domains, a consistent trend has emerged: networks which are better able to use an update to their prediction on one data point to improve their accuracy on other data points train faster, generalize better, and have higher marginal likelihoods. 

\section*{Discussion}
Our primary quantity of interest has been interference, a notion of intra-dataset generalization. The intuition motivating this study is simple: a network which can leverage information from a gradient update on one data point to improve its performance on other data points should generalize better than one which treats the data as being independent or, worse, faces competing interests wherein improving a prediction on one data point reduces performance on others.  A model which experiences positive interference will converge more quickly and in principle generalize better than one which exhibits weak or conflicting interference. This intuition, inspired by a connection to the Bayesian marginal likelihood, motivated our study of training speed and invariance in Chapters~\ref{chp:invariance} and \ref{chp:supervised}. Many implications of this relationship have been leveraged in follow-up work to develop efficient performance estimators for neural architecture search and hyperparameter selection.

Whereas supervised learning involves fitting a single target function that remains fixed over the course of training, the targets that value-based RL objectives seek to fit are more mercurial, evolving alongside the network's predictions. This complicates the story relating training speed and generalization, and requires additional analytical and empirical tools. We developed a notion of subspace convergence to study the representations learned by an idealized model of deep RL agents in the theoretical results of Chapter~\ref{chp:rl-dynamics}. We went on to present a characterization of an agent's representation in terms of the types of functions it is biased towards fitting quickly in Chapter~\ref{chp:rep-learning}. These tools provided us with a number of fruitful insights into the learning dynamics of reinforcement learning agents, setting the stage for our study of generalization in Chapters~\ref{chp:gen-rl} and \ref{chp:icp}.

Our study of generalization in deep RL leveraged these tools to understand the nuanced trade-offs between generalization and stability in value-based RL algorithms. This analysis hinged on a key insight outlined in Chapter~\ref{chp:gen-rl}: that the targets which deep RL agents seek to fit initially contain relatively little information about environment structure, and only accumulate this structure later in training. By asking networks to fit unstructured, discontinuous targets in the critical early phases of training, we bias them towards regions of the parameter space where interference between observations is weaker, reducing their potential for generalization. This analysis presented some obvious work-arounds in the form of post-training distillation which we showed improve generalization and extrapolation, but the benefits of this approach were largely limited to generalization within a single environment. Chapter~\ref{chp:icp} concluded with discussion of generalization via invariant prediction, establishing connections to causality to address these limitations. We showed that by assuming a large degree of shared structure between environments, we can significantly improve upon the sample-efficiency of standard RL methods via an application of invariant prediction algorithms from causal inference.

The initial motivation of this thesis was to understand why neural networks generalize, and how this picture can become complicated in the setting of reinforcement learning. Looking back, it is clear that this was the wrong question: firstly, because as we have seen in this thesis neural networks often \textit{fail} to generalize, particularly in the case of out-of-distribution data and in RL problems; secondly, because there is not in general a single straightforward explanation as to why a given neural network generalizes well or poorly. Instead, a variety of properties of a network architecture and training procedure come into play, and to understand the relationship between these properties and generalization it is crucial to understand the learning dynamics of the training procedure. This thesis has demonstrated the utility of this perspective in both supervised and reinforcement learning settings, and provided a number of theoretical and empirical tools for further analysis.

\section*{Further work}

While there are many straightforward extensions of the results presented in the preceding chapters, a number of deeper questions also emerge concerning how to apply these ideas to modern training regimes. Chief among these is the recent success of pre-training gargantuan models via self-supervised learning on vast swathes of data. How should we evaluate interference in such cases? As opposed to the supervised learning regime of Chapters 3 and 4, state of the art training procedures involve showing the network a vast and diverse array of data and then fine-tuning its predictions on some downstream tasks. In these cases the line between a model's inductive bias and its learned outputs becomes blurred, making the application of the tools connecting training speed and model selection presented here nontrivial. Some interesting connections might arise between the analysis of Chapter~\ref{chp:rl-dynamics} and the self-supervised learning objectives used in pre-training large models: temporal difference learning can be viewed as implicitly training the network to be self-predictive before any 'fine-tuning' reward signal is given. It also raises an interesting question of whether the ideas behind invariance and interference might be used to shape pre-training objectives which build even better inductive biases than what is used by autoregressive models. 

A second cluster of intriguing ideas concern trade-offs between plasticity and generalization, and the precise relationship between loss of plasticity and declining generalization performance. Chapter~\ref{chp:rep-learning} suggests that in many cases the correlation between the rank of a feature embedding and the plasticity of the network which parameterizes it is weak. The degree to which the network disentangles inputs can measure not just its ability to linearly distinguish observations, but also the degree to which its predictions have needed to change in order to fit its previous targets. This presents a conflict between the intuitions of prior literature in reinforcement learning, that higher-dimensional embeddings should correspond to networks with greater flexibility, and that from deep learning, where significant changes to a network's outputs tend to reduce its plasticity. 
This tension raises an interesting question into the existence of trade-offs between plasticity and the network's ability to represent particular function classes. It may be the case that in order to solve complex prediction tasks, a network must necessarily reduce its ability to quickly fit unrelated targets. Analogously, it may be the case that excessive generalization too early in training could destabilize learning, making memorization a feature rather than a bug of the learning process. If such trade-offs exist, they will carry significant implications for the application of RL to increasingly complex tasks, where training trajectories are long and require a persistent ability to adapt to new information and generalize from old experiences. It may be the case that new architectures and training procedures are needed in order to avoid the pitfalls identified in this thesis, which take into account the continual nature of reinforcement learning problems. 

Finally, the role of generalization between data points can likely be leveraged more explicitly in the training process to encourage the network to develop the correct inductive biases for the data set it is trained on. To some degree, the implicit bias of gradient descent appears to already push large models towards solutions which generalize well after the training accuracy has saturated in many natural language tasks, a phenomenon commonly referred to as `grokking'. By more directly encouraging networks to pick up on features which generalize between many training inputs, it may be possible to accelerate this phenomenon to obtain solutions which generalize well earlier in training.

Beyond these concrete research directions, the findings presented in this thesis raise a deeper philosophical question.  We discussed in Section~\ref{sec:background-science} the distinction between mathematical truths and scientific theories, and have seen examples of each throughout the thesis. Our formal guarantees have consistently been forced to trade off between generality of the problem setting (e.g. the exact gradient  updates of Theorem~\ref{thm:infinite-heads}, the linear group action in Theorem~\ref{theorem:symmgd}, and the Lipschitz continuity of the transition operator in Theorem~\ref{thm:model_error}), and precision of the resulting guarantee. In contrast, many chapters have gone on to distill the intuition behind each theoretical result into a principle which yields empirically testable predictions, such as the training speed estimator of Chapter~\ref{chp:supervised} and the nonlinear invariant prediction methods of Chapter~\ref{chp:icp}. Arguably, it is the empirical test of the intuition which has provided the greatest practical insight into the phenomena of interest to this thesis. But is this empirical approach capable of truly giving an explanation of generalization?

On its face, the answer would appear to be negative. Our empirical approach to generalization is almost tautological. A network will generalize well to test data, we have shown, if it generalizes well between data points it encounters during learning.  However, the ultimate reason why a network generalizes on a given dataset is because its inductive bias is a good fit for a data. Developing a detailed understanding of the training trajectory is crucial to be able to obtain a measure of the fit between a network’s inductive bias and the data generating distribution. This detailed understanding of a network’s training trajectory has also shown to bring useful insights into more nuanced notions of generalization and capacity in reinforcement learning, where a network must not only generalize between observations drawn from some distribution, but also to new observations and new prediction targets. The learning dynamics framework has thus already born fruit in both providing insight into generalization under the classical i.i.d. assumption on the training and evaluation data, and broadening this notion of generalization to apply to a richer class of prediction problems arising in reinforcement learning. Whether or not the precise empirical methods and theoretical tools presented in this thesis are leveraged directly in future work, it is clear that the study of learning dynamics promises to provide deep insights into some of the most fundamental questions about generalization in the years to come.

\startappendices
\let\svaddcontentsline\addcontentsline
\renewcommand\addcontentsline[3]{%
  \ifthenelse{\equal{#1}{lof}}{}%
  {\ifthenelse{\equal{#1}{lot}}{}{\svaddcontentsline{#1}{#2}{#3}}}}
\chapter{The role of invariance in generalization}

\section{Proofs}
\label{apx:invar-proofs}   
\propEmpRisk* 

\begin{proof}[Proof of Proposition~\ref{prop:empirical:risk:order}]
  Let $\grp$ be a group with some probability measure $\haar$, and $\fclass$ a class of functions $f : \calX \to \bbR$. Let $\loss : \bbR \times \bbR \to \bbR_+$ be a loss function such that $\loss(f(\argdot),\argdot)\in L_2(\dgd)$ for every $f \in \fclass$. Then the augmented risk of any function $f \in \fclass$ is
  \begin{align*}
    \eRiskAug(f,\trdata) = \frac{1}{n}\sum_{i=1}^n  \bbE_{G\sim\haar}[\loss(f(G  X_i),Y_i)] \;.
  \end{align*}
  If $\loss$ is convex in the first argument, then by Jensen's inequality,
  \begin{align} \label{eq:jensens}
    \bbE_{G\sim\haar}[\loss(f(G X_i),Y_i)] \geq \loss(\bbE_{G\sim\haar}[f(G X_i)],Y_i) \;, \quad i = 1,2,\dotsc,n \;.
  \end{align}
  On the other hand, the $\grp$-symmetrization of $f(X)$ is $\invf{f}(X) = \bbE_{G\sim\haar}[f(G X)]$, with augmented risk
  \begin{align*}
    \eRiskAug(\invf{f},\trdata) & = 
      \frac{1}{n}\sum_{i=1}^n  \bbE_{G\sim\haar}[\loss(\bbE_{G\sim\haar}[f(G X_i)],Y_i)] \\
      & =\frac{1}{n}\sum_{i=1}^n  \loss(\bbE_{G\sim\haar}[f(G X_i)],Y_i) = \eRisk(\invf{f},\trdata) \;.
  \end{align*}
  Combined with \eqref{eq:jensens}, the reduction in empirical augmented risk 
  follows. The reduction in $\eRiskAug(Q,\trdata)$ follows trivially. 

  The variance-reduction 
  is established by extending the argument in the proof of \citet{chen2019invariance}. Specifically, by the conditional Jensen's inequality,
  \begin{align*}
    \Var_{\trdata\sim\dgd^n}\big[  \eRiskAug(f,\trdata)  \big]& = \Var[\bbE[\eRiskAug(f,\trdata) \mid \Orbit^n]] \geq \Var[ \bbE[ \eRisk(\invf{f},\trdata) \mid \Orbit^n]  ] \\
    &= \Var_{\trdata\sim\dgd^n}\big[ \eRisk(\invf{f},\trdata) \big] \;.
  \end{align*}
\end{proof}

\propSymmGD* 
\begin{proof}
Let $w \in V^*$, and suppose that $w$ is not invariant under the action of $\grp$. Let $\invf{w} = \bbE_{G\sim\haar}[\rho^*_G w]$, which is \ginv by construction. Because $\calX$ spans $V$ and , $w - \invf{w} \neq 0$ implies that $w^\top(\calX) \neq {\invf{w}}^\top(\calX)$. 

Consider the minimizer
\begin{align*}
  \hat{\bw} = \argmin_{\bw \in V^*} \eRiskAug(f_\bw,\trdata) = \argmin_{\bw \in V^*} \frac{1}{n}\sum_{i=1}^n \bbE_{G\sim\haar}[\loss(\bw^{\top}\rho_G X_i,Y_i)] \;,
\end{align*}
which is unique because $\loss$ is strictly convex by assumption. Assume that $\hat{\bw}$ is not \ginv. 
Applying Jensen's inequality, we have
\begin{align*}
  \eRiskAug(f_{\hat{\bw}},\trdata) & = \frac{1}{n}\sum_{i=1}^n \bbE_{G\sim\haar}[\loss(\hat{\bw}^{\top} \rho_G X_i,Y_i)] \\
  & > \frac{1}{n}\sum_{i=1}^n \loss(\bbE_{G\sim\haar}[\hat{\bw}^{\top}\rho_G X_i,Y_i)] \\
  & = \frac{1}{n}\sum_{i=1}^n \loss(\bbE_{G\sim\haar}[(\rho_{G^{-1}}^* \hat{\bw})]^{\top} X_i,Y_i)] \\
  & = \frac{1}{n}\sum_{i=1}^n \loss(\invf{\hat{\bw}} X_i, Y_i) = \eRiskAug(f_{\invf{\hat{\bw}}},\trdata) \;,
\end{align*}
which cannot be the case because $\hat{\bw}$ minimizes $\eRiskAug$. Therefore, $\hat{\bw}$ must be \ginv.

\end{proof}

\lempushforwardKL*    
The proof of Lemma~\ref{lem:pushforward:KL} relies on the chain rule of relative entropy. 
Let two probability measures, $\tmu \ll \tnu$ defined on the product space $(E_1 \times E_2, \calE_1 \otimes \calE_2)$, have marginal measures $\tmu_1\ll \tnu_1$ on $(E_1,\calE_1)$ (respectively, $\tmu_2 \ll \tnu_2$ on $(E_2,\calE_2)$) and regular conditional probability measures $\tmu_{2|1}\ll\tnu_{2|1}$ (resp.\ $\tmu_{1|2}\ll\tnu_{1|2}$). Recall the chain rule of relative entropy is
\begin{align} \label{eq:chain:rule}
  \KL{\tmu}{\tnu} = \KL{\tmu_1}{\tnu_1} + \bbE_{\tmu}\bigg[ \log \frac{d\tmu_{2|1}}{d\tnu_{2|1}}  \bigg] = \KL{\tmu_2}{\tnu_2} + \bbE_{\tmu}\bigg[ \log \frac{d\tmu_{1|2}}{d\tnu_{1|2}}  \bigg] \;.
\end{align}
Observe that each of the terms in the equalities is non-negative
In particular, when $\psi$ is non-injective, points of $(E_1,\calE_1)$ become equivalent; $(E_2,\calE_2)$ is a compressed version, and the probability measures $\mu$ and $\nu$ are similarly compressed.

\begin{proof}[Proof of Lemma~\ref{lem:pushforward:KL}]
  Given probability measures on $(E_1,\calE_1)$ $\mu \ll \nu$ (with density $m$ such that $\mu = m\cdot \nu$) and a measurable map $\psi : (E_1,\calE_1) \to (E_2,\calE_2)$, construct the probability measure $\tmu$ on $(E_1 \times E_2, \calE_1 \otimes \calE_2)$ as
  \begin{align*}
    \tmu(A \times B) = \mu(A \cap \psi^{-1}B) = \int_A \mu(dx_1) \int_{B} \delta_{\psi(x_1)}(dx_2) \;, \quad A \in \calE_1,\ B \in \calE_2 \;,
  \end{align*}
  and likewise for $\tnu$. Then in the notation of \eqref{eq:chain:rule}, $\tmu_1 = \mu\ll \nu = \tnu_1$, and $\tmu_{2|1} = \delta_{\psi(x_1)} = \tnu_{2|1}$. Therefore,
  \begin{align}
    \KL{\tmu}{\tnu} = \KL{\tmu_1}{\tnu_1} = \KL{\mu}{\nu} \;.
  \end{align}
  Alternatively, $\tmu_2 = \mu\circ\psi^{-1}$, $\tnu_2 = \nu\circ\psi^{-1}$, and it is straightforward to show that
  \begin{align}
    \bbE_{\tmu}\bigg[ \log \frac{d\tmu_{1|2}}{d\tnu_{1|2}}  \bigg] = \bbE_{\tmu}\bigg[ \log \frac{d\tmu_{1}}{d\tnu_{1}}  \bigg] - \bbE_{\tmu}\bigg[ \log \frac{d\tmu_{2}}{d\tnu_{2}}  \bigg] = \bbE_{\mu}\bigg[ \log \frac{m}{m\circ\psi}  \bigg] = \Delta_{\psi}(\mu\ ||\ \nu) \geq 0. \;.
  \end{align}
  Therefore,
  \begin{align}
    \KL{\tmu}{\tnu} = \KL{\mu}{\nu} = \KL{\mu\circ\psi^{-1}}{\nu\circ\psi^{-1}} + \Delta_{\psi}(\mu\ ||\ \nu) \;.
  \end{align}
\end{proof}
\subsection{Proof of \texorpdfstring{Theorem~\ref{thm:pac:bayes:da}}{Theorem 4}}

\thmpacbayesda*
The proof of our PAC-Bayes bound for data augmentation makes use of the following result due to \citet{leveretal2013tighterPACbayes}.

\begin{theorem}[\citet{leveretal2013tighterPACbayes}, Theorem 1] \label{lem:lever:bound}
  For any functions $A(f, \trdata)$, $B(f)$ over $\fclass$, either of which may be a statistic of the training data $\trdata$, any distribution $P$ over $\fclass$, any $\delta \in (0,1]$, any $t > 0$, and a convex function $\scD : \bbR \times \bbR \to \bbR$, with probability $\dgd^n$ at least $1 - \delta$, for all distributions $Q$ on $\fclass$,
  \begin{align} \label{eq:lever:bound}
    \scD\big( \bbE_{f\sim Q}[A(f, \trdata)],\bbE_{f\sim Q}[B(f)] \big) \leq \frac{1}{t} \bigg( \KL{Q}{P} + \log \frac{\calL_P}{\delta}  \bigg) \;,
  \end{align}
  where $\calL_P: = \bbE_{\trdata\sim \dgd, f\sim P}[e^{t\scD(A(f, \trdata,B(f))}]$ is the Laplace transform of $\scD(A(f),B(f))$.
\end{theorem}

As \citet{leveretal2013tighterPACbayes} discuss, many PAC-Bayes bounds in the literature can be obtained as special cases of Lemma~\ref{lem:lever:bound}, including Catoni's bound in Theorem~\ref{thm:catoni:bound}. In that case, which applies to 0-1 loss, $t=n$, $A(f) = \eRisk(f,\trdata)$, $B(f) = \risk(f)$, $C=n^{-1}$ and 
\begin{align}
  \scD_C(q,p) & := -\log (1-p(1-e^{-C})) - C q \;, \quad q,p \in (0,1), \ C > 0 \\
    & = -\log \bbE_{Z \sim \text{Bern}(p)}[e^{-CZ}] - Cq \;.
\end{align}
Basic calculations show that with these quantities, $\calL_P=1$.
\begin{align}
    \calL_P = \mathbb{E}_{\trdata, f}[e^{t \scD(A(f; \trdata), B(f))}] &= \mathbb{E}_{\trdata, f}[e^{t\scD(\eRisk(f; \trdata), \risk(f))}] \\
    &= \mathbb{E}_{f, \trdata}[e^{t(-\log \mathbb{E}[e^{CZ}] - C \eRisk(f; \trdata))} ]\\
    &= \mathbb{E}_{f}[e^{-t\log \mathbb{E}[e^{-CZ}]} \mathbb{E}_{\trdata}[e^{-tC \eRisk(f; \trdata)}]] \\
    &= \mathbb{E}_{f}[e^{-\log \mathbb{E}[e^{-CZ}]} \mathbb{E}_{Z_{1:n} \sim \text{Bern}(\risk(f))}[e^{-tC\frac{1}{n} \sum_{i=1}^n Z_i}]] \\
    &= \mathbb{E}_{f}[e^{-\log \mathbb{E}[e^{-CZ}]} \Pi_{i=1}^n \mathbb{E}_{Z_i \sim \text{Bern}(\risk(f))}[e^{-\frac{1}{n}Z_i}]] \\
    &= \mathbb{E}_{f}[e^{-t\log \mathbb{E}[e^{-CZ}]} (\mathbb{E}_{Z \sim \text{Bern}(\risk(f))}[e^{-CZ}])^n] \\
    &= \mathbb{E}_{f}[e^{-t\log \mathbb{E}_{Z \sim \text{Bern}(\risk(f))}[e^{-CZ}]} e^{t \log \mathbb{E}_{Z \sim \text{Bern}(\risk(f))}[e^{-CZ}]}] \\
    &= \mathbb{E}_{f} [1] = 1
\end{align}
Recall that
\begin{align} 
  \eRisk(f,\trdata) &:= \frac{1}{n} \sum_{i=1}^n \loss(f(X_i),Y_i) \label{eq:risk:1} \\
  \eRiskAug(f,\trdata) &:= \frac{1}{n} \sum_{i=1}^n \bbE_{G\sim\haar}[\loss(f(G X_i),Y_i)] \label{eq:risk:2} \\
  \eRiskAugMC(f,\trdata) &:= \frac{1}{nm} \sum_{i=1}^n \sum_{j=1}^m \loss(f(G_{ij} X_i),Y_i) \label{eq:risk:3} \;.
\end{align}

Let $(G_{ij})$ denote the collection of $m\cdot n$ random augmentation transformations sampled i.i.d.\ from $\haar$. 
\begin{lemma} \label{lem:div:bounds}
  Let $\loss$ be the binary loss, $P$ any distribution on $\fclass$, and assume that $\dgd$ is \ginv. Then
  \begin{align} \label{eq:div:bound:aug}
    \bbE_{f\sim P}\big[\bbE_{\trdata\sim\dgd}\big[ e^{n\scD_C(\eRiskAug(f,\trdata),\risk(f))}   \big] \big] & \leq \bbE_{f\sim P}\big[\bbE_{\trdata\sim\dgd}[e^{n\scD_C(\eRisk(f,\trdata),\risk(f))}] \big] = 1
  \end{align}
  and
  \begin{align} \label{eq:div:bound:aug:mc}
    \bbE_{f\sim P}\big[\bbE_{\trdata\sim\dgd}\big[ e^{n\scD_C(\eRiskAugMC(f,\trdata),\risk(f))}   \big] \big] & \leq \bbE_{f\sim P}\big[\bbE_{\trdata\sim\dgd}[e^{n\scD_C(\eRisk(f,\trdata),\risk(f))}] \big] = 1 \;.
  \end{align}
\end{lemma}
\begin{proof}
  Since the observations $(X_i,Y_i)$ are i.i.d., the expectation over $\trdata$ on the left-hand side of \eqref{eq:div:bound:aug} requires evaluating $\bbE_{\trdata\sim\dgd}\big[ e^{-C\bbE_{G\sim\haar}[ \loss(f(X_i),Y_i))]} \big]$. Using the convexity of $e^{-x}$, Jensen's inequality and Fubini's theorem yield
  \begin{align} \label{eq:augrisk:laplace}
    \bbE_{(X_i,Y_i)\sim\dgd}\big[ e^{-C\bbE_{G\sim\haar}[ \loss(f(G X_i),Y_i))]} \big] 
      & \leq \bbE_{(X_i,Y_i)\sim\dgd}\big[ \bbE_{G\sim\haar} \big[ e^{-C \loss(f(G X_i),Y_i))} \big] \big] \\
      & = \bbE_{G\sim\haar} \big[ \bbE_{(X_i,Y_i)\sim\dgd}\big[  e^{-C \loss(f(G X_i),Y_i))} \big] \big] \;. \nonumber
  \end{align}
  Now, $\grp$-invariance of $\dgd$ implies that $\bbE_{(X_i,Y_i)\sim\dgd}[h(gX_i,Y_i)] = \bbE_{(X_i,Y_i)\sim\dgd}[h(X_i,Y_i)]$ for all measurable functions $h : \calX \times \calY \to \bbR_+$ and all $g\in\grp$, which extends to {independent} random $G$ by Fubini's theorem. Therefore,
  \begin{align*}
    \bbE_{G\sim\haar} \big[ \bbE_{(X_i,Y_i)\sim\dgd}\big[  e^{-C \loss(f(G X_i),Y_i))} \big] \big]
      = \bbE_{(X_i,Y_i)\sim\dgd}\big[  e^{-C \loss(f(X_i),Y_i))} \big] = \bbE_{Z\sim\text{Bern}(\risk(f))}[e^{-CZ}] \;,
  \end{align*}
  which implies \eqref{eq:div:bound:aug}.

  For the second inequality \eqref{eq:div:bound:aug:mc}, observe that by Jensen's inequality,
  \begin{align*}
    \bbE_{\trdata\sim\dgd}\big[ e^{-nC\eRiskAugMC(f,\trdata)} \big]
      & = \prod_{i=1}^n \bbE_{(X_i,Y_i)\sim\dgd}\bigg[ \bbE_{(G_{ij})_{j=1}^m\sim\haar}\bigg[ \exp\bigg(-\frac{C}{m}\sum_{j=1}^m \loss(f(G_{ij}X_i),Y_i) \bigg) \bigg] \bigg] \\
      & \leq \prod_{i=1}^n \bbE_{(X_i,Y_i)\sim\dgd}\bigg[ \bbE_{(G_{ij})_{j=1}^m\sim\haar}\bigg[ \frac{1}{m} \sum_{j=1}^m e^{-C \loss(f(G_{ij}X_i),Y_i) } \bigg] \bigg] \\
      & = \prod_{i=1}^n \bbE_{(X_i,Y_i)\sim\dgd}\big[ \bbE_{G\sim\haar} \big[ e^{-C \loss(f(G X_i),Y_i) } \big] \big]
  \end{align*}
  Using the $\grp$-invariance of $\dgd$ once again, we have
  \begin{align*}
    \bbE_{\trdata\sim\dgd}\big[ e^{-nC\eRiskAugMC(f,\trdata)} \big] \leq \bbE_{\trdata\sim\dgd}\big[ e^{-nC\eRisk(f,\trdata)} \big] = \big(\bbE_{Z\sim\text{Bern}(\risk(f))}[e^{-CZ}] \big)^n \;,
  \end{align*}
  which implies \eqref{eq:div:bound:aug:mc}.
\end{proof}

Note that Lemma~\ref{lem:div:bounds} depends on the observation that for any distribution over functions, the symmetrized risk will be a lower-variance estimator than the empirical risk but will remain unbiased. This is only the case for \textit{invariant} data-generating distributions -- symmetrizing the risk of a non-symmetric $\dgd$ will not necessarily result in a valid PAC-Bayes bound.

\section{Examples, counterexamples, tighter bounds}

\subsection{Counterexamples} \label{appx:counterexamples}

\textbf{Feature averaging and non-convex losses.} We consider the binary classification setting with the zero-one loss and some function class $f$ bounded in $[0,1]$ -- that is $\ell(x, y) = \mathbbm{1}[|f(x) - y| > 1]$. Suppose that there exists some invariance $\grp$ in the data such that $y(x) = y(gx)$ for all $x, g$. Then consider a function which, for some small $\epsilon$, outputs $f(x) = \frac{1}{2} + y\epsilon$ on a $1 - 2\epsilon$ fraction of each equivalence class of the inputs, and $1 - y$ on $2\epsilon$ of the inputs in each equivalence class. Then $\mathbb{E}[f(gx)] = (1 - 2\epsilon) (\frac{1}{2} + y\epsilon) + 2\epsilon (1-y)$. When $y=0$, this expectation is $\frac{1}{2} + \epsilon$, and when $y=1$ it is $\frac{1}{2} [1 - \epsilon - 2\epsilon^2] < \frac{1}{2}$, so the feature-averaged model would have risk 1 whereas the original model had risk 0.

\textbf{Non-uniform data-generating distributions.}  When the data-generating distribution is not uniform over the set $\T$, then performing data augmentation with $\T$ will not necessarily lead to a more accurate estimate of the model's empirical risk. For example, consider the task of learning a function $g$ satisfying $g(x) = g(-x)$, bounded in magnitude by some constant $A$. Suppose, however, that positive numbers are much more likely under the data generating distribution, with $p(\mathbb{R}^+) = 1 - \epsilon$ for small $\epsilon$. Then the function $f(x) = \mathbbm{1}[x>0]g(x)$ will satisfy $\mathbb{E}[\|f(X_S) - g(X_s)\|] \neq \mathbb{E}[ \|f(X_{S^\text{aug}}) - g(X_{S^\text{aug}})\|]$. So the augmented risk is no longer an unbiased estimator of the empirical risk. Further, in this particular case its variance is also higher, as it will be equal to $\frac{1}{2} $Var$(g(x))$, in contrast to $\epsilon \text{Var}(g(x))$.

\subsection{Tighter PAC-Bayes bound for data augmentation} \label{appx:tighter:pacbayes:da}

Although Theorem~\ref{thm:pac:bayes:da} establishes that the i.i.d.\ PAC-Bayes bound \eqref{eq:catoni:bound} is valid for exact DA, the proof of Theorem~\ref{thm:pac:bayes:da} indicates that a tighter bound is possible. In particular, recall that when $\dgd$ is \ginv \citep{invariantdistributions,chen2019invariance},
\begin{align*}
  \bbE_{G\sim\haar}[\loss(f(GX),Y)] = \bbE_{(X,Y)\sim\dgd}[\loss(f(X),Y) \mid \Orbit] := \invf{\loss}_f(\Orbit) \;.
\end{align*}
$\invf{\loss}_f(\Orbit)$ is a random variable, the average loss on the random orbit with representative $\Orbit$, whose distribution is induced by $\dgd$. Therefore, we can write $\mathcal{L}_P$ in \eqref{eq:lever:bound} as
\begin{align*}
  \mathcal{L}_P = \bbE_{f\sim P}\bigg[ \bigg(\frac{ \bbE_{\Phi\sim\dgd} \big[ e^{-C\invf{\loss}_f(\Orbit)} \big]}{\bbE_{Z\sim\text{Bern}(\risk(f))}[e^{-CZ}]} \bigg)^n \bigg] \leq 1 \;.
\end{align*}
In general, this cannot be computed in closed form. However, it might be possible to estimate using the data (with appropriate modifications to the resulting bound) and samples $f \sim P$.


\section{Computation details for PAC-Bayes bounds}
\label{appx:invar-pb}
PAC-Bayes bounds for neural networks are computed via the following procedure: a deterministic neural network is trained to minimize the cross-entropy loss on the dataset. After it has reached a suitable training accuracy, we use these parameters as the initialization for the means and variances of the stochastic neural network weights used for the PAC-Bayes bounds. We directly optimize a surrogate of the PAC-Bayes bound (using the cross-entropy loss instead of the zero-one accuracy and using the reparameterization trick to get the derivatives of the variance parameters). The exact computation of the PAC-Bayes bound uses the union bound and discretization of the PAC-Bayes prior as described in \citep{dziugaite2017nonvacuous}. Reported values are at optimization convergence. 

The experiment code is provided with the paper submission, but we describe here at a high level the different models used in our empirical evaluations.

\textbf{FashionMNIST CNN:} the convolutional network used for FashionMNIST consists of two convolutional layers (with batch norm and max pooling) followed by a single fully connected layer. 

\textbf{LiDAR Permutation-Invariant Network:} we use a scaled-down version of the PointNet architecture \citep{qi2017pointnet}. We include two layers of 1D convolutions followed by a max-pooling layer that selects the maximum over input points for each channel. This layer is followed by two fully-connected layers leading into the final output. 

\textbf{Partially-Invariant Network:} we alter the previous architecture slightly so that it is only invariant to \textit{subgroups} of the permutation group on its inputs. Specifically, we partition the input into 8 disjoint subsets, and apply the previous model's permutation-invariant embedding layers to each partition. The result is a feature representation that is invariant to permutations within each partition of the input, but not between partitions. This representation is then fed through the same architecture. We note that we keep the number of convolutional filters per layer constant, which results in a larger feature embedding by a factor of 8 that is fed into the first fully connected layer. As a result, this model has significantly more parameters than the fully permutation-invariant model.

\textbf{Fully Connected Network:} the max-pooling operator of the previous two architectures is omitted. This network has many more parameters than either of the first two models, and is not invariant to any subgroup of the permutation group.

\chapter{Training speed and model selection}
\section{Linear model combination}

\label{sec:optimize-then-prune}

The estimator $\mathcal{L}(\data)$ reveals an intriguing connection between pruning in linear model combinations and Bayesian model selection. This connection arises when we consider the weight assigned to a model in a linear model combination when models are fit on data points iteratively, as in the marginal likelihood estimation algorithms discussed previously. 

We study a setting that presents a bridge between Bayesian updating and optimization by considering the limiting behaviour of a linear regressor trained on predictions output by Bayesian models fitted iteratively as seen in Algorithm~\ref{alg:estimate}. Concretely, we treat samples from the model posterior $P(Y_i | \data_{<i}, X_i, \model_j)$ as entries $\Phi[i, j]$ in a design matrix $\Phi$, and then consider properties of an optimal linear predictor trained on $\Phi$. While this does not perfectly replicate the setting of deep neural networks trained with gradient descent, it highlights an important property of gradient descent on non-stationary features: the feature which contributes the most to the output of a linear ensemble will often be the one which was most predictive of the target \textit{on average} over the course of training, rather than the one which is most correlated with the target at the end of training.

Our analysis requires a number of assumptions to be formalized. We assume a data set $\data = (X_i, Y_i)_{i=1}^n$ and a collection of $k$ models $\model_1, \dots, \model_k$. We train a linear regressor $w$ to fit the posterior predictive distributions of the models to the target $Y_i$; i.e. to regress on the dataset 
\begin{equation}
    (\Phi, Y) = \bigg (\phi_i=(\hat{Y}^i_1, \dots, \hat{Y}_n^i), Y_i\bigg)_{i=1}^n \text{ with } \hat{Y}_j^i \sim P(\hat{Y}|\data_{<i}, X_i, \model_j).
\end{equation}
The following result shows that the optimal linear regressor on this data generating distribution assigns the highest weight to the model with the highest $\mathcal{L}(\data)$ whenever the model errors are independent. This shows that magnitude pruning in a linear model combination is equivalent to approximate Bayesian model selection, under certain assumptions on the models.

\begin{restatable}{proposition}{PropBMS}\label{prop:modelselect}
Let $\model_1, \dots, \model_k$ be Bayesian linear regression models with fixed noise variance $\sigma_N^2$ and Gaussian likelihoods. Let $\Phi$ be a (random) matrix of posterior prediction samples, of the form $\Phi[i, j] = \hat{y}_i^j \sim P(y_j|\data_{<j}, x_j, \model_i)$. Suppose the following two conditions on the columns of $\Phi$ are satisfied: $\mathbb{E}\langle \Phi[:, i], y \rangle = \mathbb{E}\langle \Phi[:, j], y \rangle$ for all $i, j$, and $\mathbb{E}\langle \Pi_{y^\perp} \phi_i, \Pi_{y^\perp} \phi_j \rangle = 0$. Let $w^*$ denote the least-squares solution to the regression problem $\min_w \mathbb{E}_{\Phi}\|\Phi w - y\|^2$. Then the following holds
\begin{equation}  \argmax_i w^*_i = \argmax_i \mathcal{L}(\data | \model_i)  \qquad \forall w^* = \argmin_w \mathbb{E} \|\Phi w - y\|^2\;. \end{equation}
\end{restatable}

The assumption on the independence of model errors is crucial in the proof of this result: families of models with large and complementary systematic biases may not exhibit this behaviour. We observe in Section \ref{sec:BMS} that the conditions of Proposition 1 are approximately satisfied in a variety of model comparison problems, and running SGD on a linear combination of Bayesian models still leads to solutions that approximate Bayesian model selection.
We conjecture that analogous phenomena occur during training within a neural network. The proof of Proposition~\ref{prop:modelselect} depends on the observation that, given a collection of features, the best least-squares predictor will assign the greatest weight to the feature that best predicts the training data. While neural networks are not linear ensembles of fixed models, we conjecture that, especially for later layers of the network, a similar phenomenon will occur wherein weights from nodes that are more predictive of the target values over the course of training will be assigned higher magnitudes. We empirically investigate this hypothesis in Section \ref{sec:DNN_exp}. 

We consider three model selection problems in our empirical evaluations. In \textbf{prior variance selection} we evaluate a set of BLR models on a synthetic linear regression data set. Each model $\mathcal{M}_i$ has a prior distribution over the $d$ parameters of the form $w \sim \mathcal{N}(0, \sigma_i^2 I_d)$ for some $\sigma_i^2$, and the goal is to select the optimal prior variance (in other words, the optimal regularization coefficient). We additionally evaluate an analogous initialization variance selection method on an NTK network trained on a toy regression dataset. In \textbf{frequency (lengthscale) selection} we use as input a subset of the handwritten digits dataset MNIST given by all inputs labeled with a 0 or a 1. We compute random Fourier features (RFF) of the input to obtain the features for a Bayesian linear regression model, and perform model selection over the frequency of the features (full details on this in the appendix). This is equivalent to obtaining the lengthscale of an approximate radial basis function kernel. In \textbf{feature dimension selection}, we use a synthetic dataset \citep{wilson2020bayesian} of the form $(\textbf{X}, \textbf{y})$, where $x_i = (y_i + \epsilon_1,  y_i + \dots, y_i + \epsilon_{15}, \epsilon_{16}, \dots, \epsilon_{30})$ with $\epsilon_i$ i.i.d. noise variables. We then consider a set of models $\{\model_k\}$ with feature embeddings $\phi_k(x_i) = x_i[1, \dots, k]$. The optimal model in this setting is the one which uses exactly the set of `informative' features $x[1, \dots, 15]$. 

\begin{figure}
    \centering
    \includegraphics[width=0.325\linewidth]{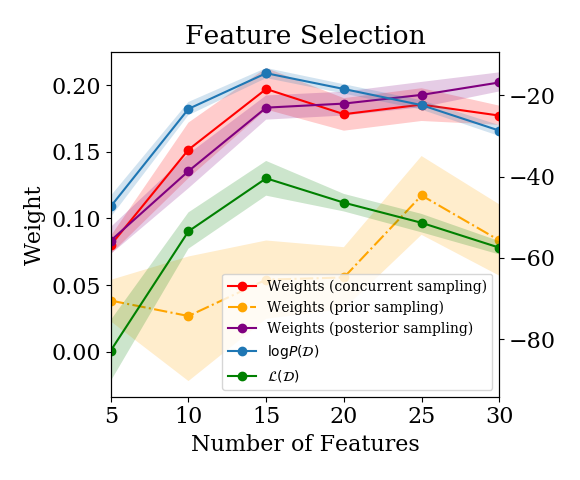}
    \includegraphics[width=0.325\linewidth]{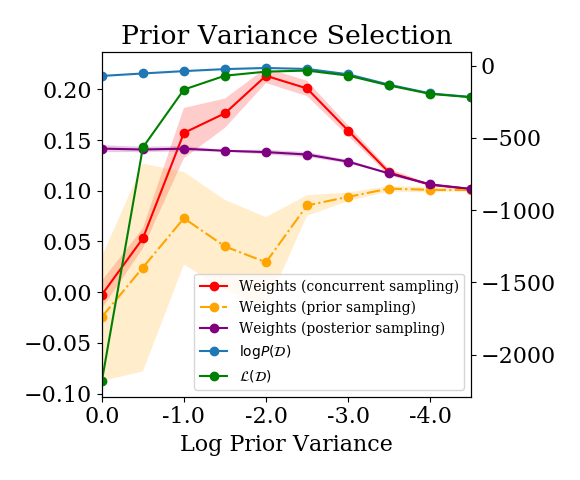}
    \includegraphics[width=0.325\linewidth]{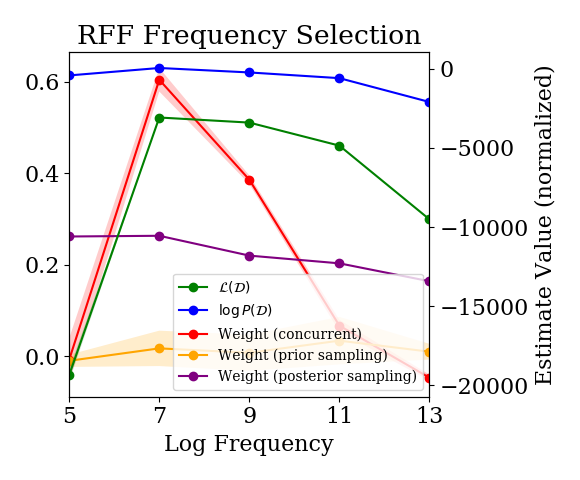}
    \hfill 
    \caption{Relative rankings given by optimize-then-prune, ML, and estimated $\mathcal{L}(\data)$. Left: feature selection. Middle: prior variance selection. Right: RFF frequency selection. Rankings are consistent with what our theoretical results predict. Results are averaged over $5$ runs.
    }
    \label{fig:app_ml_v_weight}
\end{figure}

We empirically evaluate the claims of Proposition~\ref{prop:modelselect} in settings where the assumptions only approximately hold. Concretely, we consider learning problems where the outputs of the models have roughly equal norm and approximately independent errors. We compare the ranking given by the true log marginal likelihood, the estimated $\mathcal{L}(\data)$, and the weight assigned to each model by the trained linear regressor. We consider three variations on how sampled predictions from each model are drawn to generate the features $\phi_i$: sampling the prediction for point $\hat{Y}_i$ from $P(\hat{Y}_i | \data_{<i})$ (`concurrent sampling' -- this is the setting of Proposition \ref{prop:modelselect}), as well as two baselines: the posterior $P(\hat{Y}_i |\data)$ (`posterior sampling'), and  the prior $P(\hat{Y}_i)$ (`prior sampling'). The baselines illustrate the importance of \textit{concurrently} fitting the models and the linear weights; fitting linear weights over a set of trained models simply identifies the one with the best training set performance, and fitting linear weights to the untrained models tends to favour those with lower-magnitude predictions. Their inclusion further highlights that in the model selection problems on which we evaluate the methods, it is not the case that either prior or posterior predictive likelihood is sufficient to correctly identify the best model.

Concretely, we find that the rankings of the marginal likelihood, its lower bound, and of the ranking given by concurrent optimization agree on the best model in all three of the model selection problems outlined previously. The prior and posterior sampling procedure baselines do not exhibit a consistent ranking with the log ML, and indeed exhibit opposite trends in the prior variance selection task. We visualize these results for the feature dimension selection problem in Figure \ref{fig:app_ml_v_weight}.

\subsubsection{Subnetwork selection in neural networks} \label{sec:sgd_submodel}
Finally, we evaluate whether our previous insights apply to submodels within a neural network, suggesting a potential mechanism which may bias SGD towards parameters with better generalization performance. Based on the previous experiments, we expect that nodes that have a lower sum over training errors (if evaluated as a classifier on their own) are favoured by gradient descent and therefore have a larger final weight than those which are less predictive of the data. If so, we can then view SGD followed by pruning (in the final linear layer of the network) as performing an approximation of a Bayesian model selection procedure. We replicate the model selection problem of the previous setting, but replace the individual models with the activations of the penultimate layer of a neural network, and replace the linear ensemble with the final linear layer of the network.  Full details on the experimental set-up can be found in Appendix \ref{sec:exp_details_sgd_submodels}. We find that our hypotheses hold here: SGD assigns larger weights to subnetworks that perform well, as can be seen in Figures~ \ref{fig:sgd_submodel_full} and \ref{fig:sgd_submodel_full_cifar}. This suggests that SGD is biased towards functions that generalize well, even within a single neural network. We find the same trend holds for CIFAR-10, which is shown in Appendix \ref{sec:exp_details_sgd_submodels}.

\begin{figure}
    \centering
    \includegraphics[ width=\linewidth]{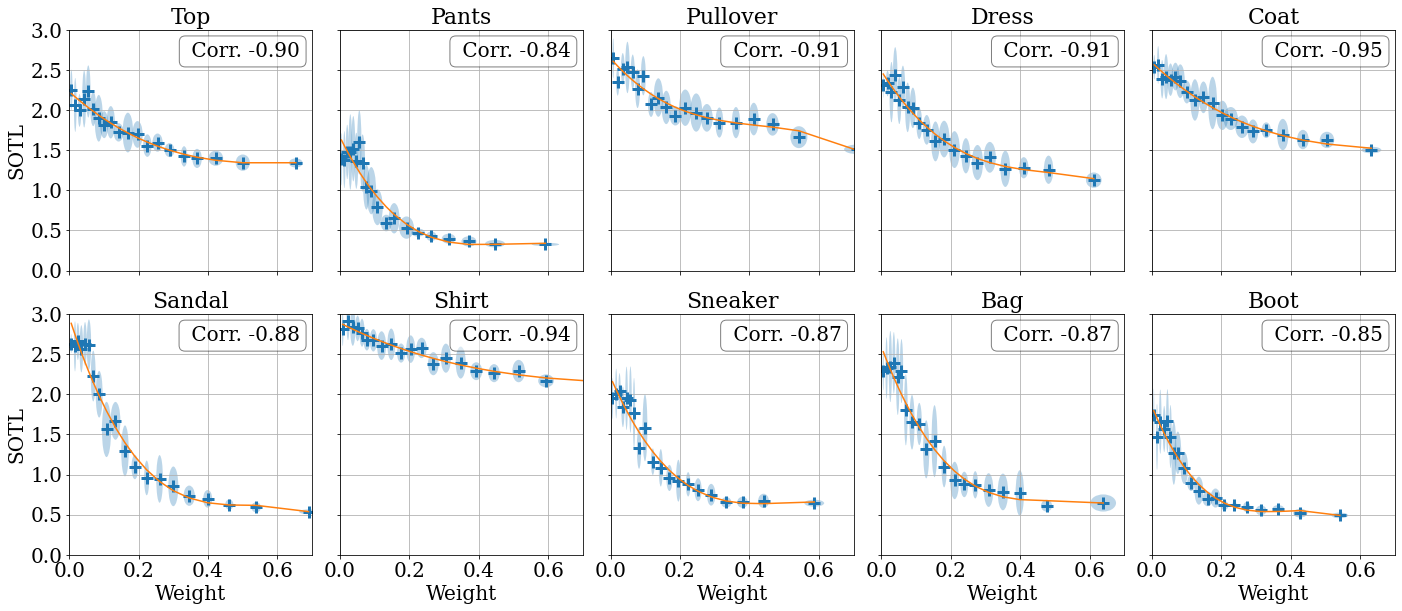}
    \caption{Weight assigned to subnetwork by SGD in a deep neural network (x-axis) versus the subnetwork performance (estimated by the sum of cross-entropy, on the y-axis) for different FashionMNIST classes. The light blue ovals denote depict $95\%$ confidence intervals, estimated over 10 seeds (i.e. 2$\sigma$ for both the weight and SOTL).  The orange line depicts the general trend.}
    \label{fig:sgd_submodel_full}
\end{figure}

\begin{figure}
    \centering
    \includegraphics[ width=\linewidth]{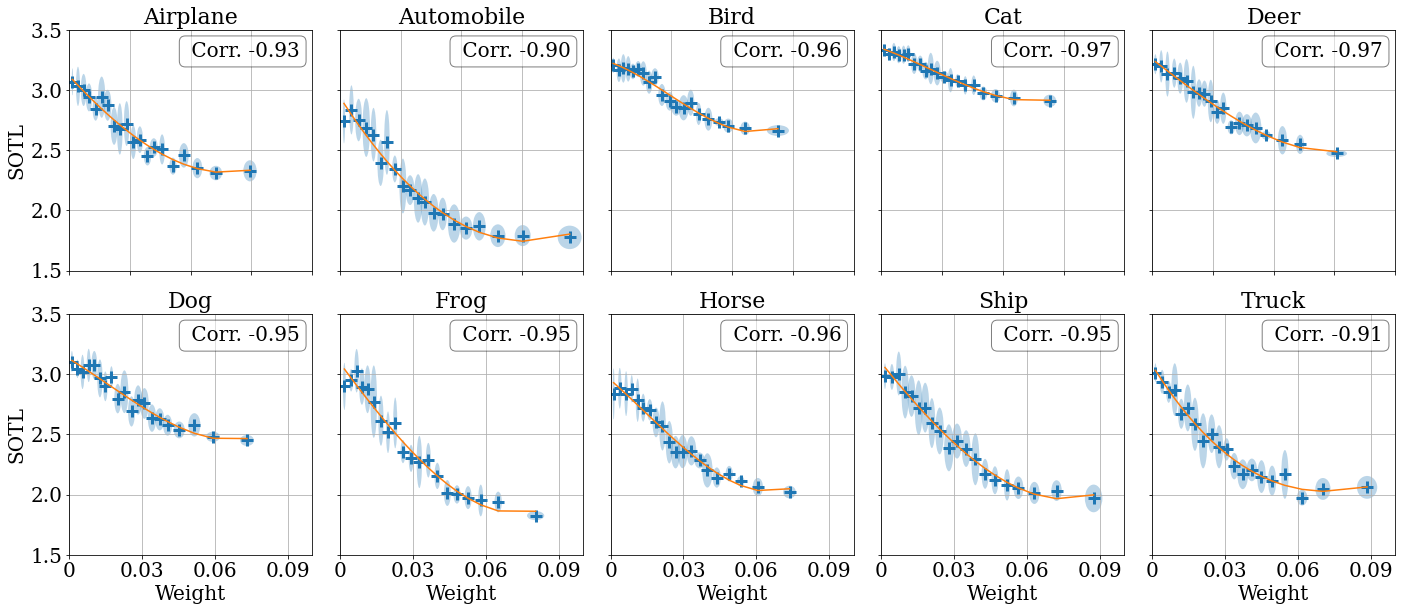}
    \caption{Weight assigned to subnetwork by SGD in a deep neural network (x-axis) versus the subnetwork performance (estimated by the sum of cross-entropy, on the y-axis) for different CIFAR-10 classes. The light blue ovals denote depict $95\%$ confidence intervals, estimated over 10 seeds (i.e. 2$\sigma$ for both the weight and SOTL).  The orange line depicts the general trend.}
    \label{fig:sgd_submodel_full_cifar}
\end{figure}

\section{Proofs of theoretical results} \label{sec:proofs-supervised}

\PropLk*
\begin{proof}
The result for $\mathcal{L}$ follows from a straightforward derivation:
\begin{align}
    \mathcal{L}(\data) &= \sum \int \log P(\data_i|\theta) dP(\theta|\data_{<i})  \\
    &= \sum \int \log [\frac{ P(\data_i|\theta) P(\theta | \data_{<i}) P(\data_i|\data_{<i})}{P(\theta | \data_{<i}) P(\data_i|\data_{<i})}] dP(\theta|\data_{<i})\\
    &= \sum \int \log\frac{ P(\theta|\data_{\leq i}))}{P(\theta|\data_{<i})} dP(\theta|\data_{<i}) + \sum \log P(\data_i| \data_{<i}) \\
    &= \sum  \bigg (  \log P(\data_i|\data_{<i})-\KL(P(\theta|\data_{<i})|| P(\theta|\data_{\leq i})) \bigg ) \\
    &= \log P(\data) - \sum_{i=1}^n \KL(P(\theta|\data_{<i})||P(\theta|\data_{\leq i})).
\end{align}
The result for $\hat{\mathcal{L}}_k$ follows immediately from Jensen's inequality, yielding
\begin{equation}
    \sum \mathbb{E}[\log \sum_{j=1}^k \frac{1}{k} p(\data_i|\theta_j)] \leq \sum \log \mathbb{E}[ \sum_{j=1}^k \frac{1}{k} p(\data_i|\theta_j)] =\sum \log \mathbb{E}[ p(\data_i|\theta_j)] =  \log P(\data) \; .
\end{equation}
Because $\mathcal{L}_k$ applies Jensen's inequality to a random variable with decreasing variance as a function of $k$, we expect the bias of $\mathcal{L}_k$ to decrease as $k$ grows, an observation characterized in Section \ref{sec:BMS}.
\end{proof}
\PropLS*
\begin{proof}
To show that the sum of the estimated log likelihoods is a lower bound on the log marginal likelihood, it suffices to show that each term in the sum of the estimates is a lower bound on the corresponding term in log marginal likelihood expression. Thus, without loss of generality we consider a single data point $\data_i = (x, y)$ and posterior distribution $p(y|x, \data_{<i})=\mathcal{N}(\mu, \sigma^2)$. 

Let $y \in \mathbb{R}$, $\hat{\mu}, \hat{\sigma}$ the standard estimators for sample mean and variance given sample $\hat{Y} \in \mathbb{R}^k$ sampled from $\mathcal{N}(\mu, \sigma^2)$. We want to show

\begin{equation}
\mathbb{E}_{\hat{Y} \sim \mathcal{N}(\mu, \sigma^2)}[\ln p(y|\hat{\mu}, \hat{\sigma}^2)] \leq \ln p(y|\mu, \sigma^2).
\end{equation}
We first note that $\hat{\mu}(\hat{Y}) \perp \hat{\sigma}(\hat{Y})$ for $\hat{Y}$ a collection of i.i.d. Gaussian random variables \citep{basu1955}. 
We also take advantage of the fact that the log likelihood of a Gaussian is concave with respect to its $\mu$ parameter and its $\sigma^2$ parameter. Notably, the log likelihood is \textit{not} concave w.r.t. the joint pair $(\mu, \sigma^2)$, but because the our estimators are independent, this will not be a problem for us. 
We proceed as follows by first decomposing the expectation over the samples $\hat{Y}$ into an expectation over $\hat{\mu}$ and $\widehat{\sigma^2}$
\begin{align}
\mathbb{E}_{X \sim \mathcal{N}(\mu, \sigma^2)}[\ln p(y| \hat{\mu}, \hat{\sigma}^2)] &= \mathbb{E}_{\hat{\mu}, Y_2, \dots, Y_N} \ln p(y|\hat{\mu}, \hat{\sigma}^2) \\
&= \mathbb{E}_{\hat{\mu}} \mathbb{E}_{\hat{\sigma}^2} \ln p(y|\hat{\mu}, \hat{\sigma}^2)
\intertext{We apply Jensen's inequality first to the inner expectation, then to the outer.}
&\leq \mathbb{E}_{\hat{\mu}} \ln p(y|\hat{\mu}, \mathbb{E}[\hat{\sigma}^2]) =  \mathbb{E}_{\hat{\mu} }\ln p(y|\hat{\mu}, \sigma^2)   \\
&\leq \ln p(y|\mu, \sigma^2)
\end{align}
So we obtain our lower bound.
\end{proof}

\ThmSTO*

\begin{proof}
The heavy lifting for this result has largely been achieved by Propositions \ref{prop:lk} and \ref{prop:ls}, which state that provided the samples $\theta^{i}_j$ are distributed according to the posterior, the inequalities will hold. It therefore remains only to show that the sample-then-optimize procedure yields samples from the posterior. The proof of this result can be found in Lemma 3.8 of \citet{osband2018randomized}, who show that the optimum for the gradient descent procedure described in Algorithm \ref{alg:estimate} does indeed correspond to the posterior distribution for each subset $\data_{<i}$. 

Finally, it is straightforward to express the lower bound estimator $\hat{\mathcal{L}}$ as the sum of regression losses. We obtain this result by showing that the inequality holds for each term $\log P(\data_i|\theta_i)$ in the summation. 

\begin{align}
 \log P(\data_i|\theta) &= \log[ \exp \bigg (-\frac{(\theta^\top x_i - y_i)^2 }{2 \sigma^2} \bigg )\frac{1}{\sqrt{2\pi}\sigma} ] \\
 &= -\frac{(\theta^\top x_i - y_i)^2 }{2 \sigma^2} -\frac{1}{2} \log (2 \pi \sigma^2) \\
 &= c_1 \ell_2(\data_i, \theta) + c_2
\end{align}

We note that in practice, the solutions found by gradient descent for finite step size and finite number of steps will not necessarily correspond to the exact local optimum. However, it is straightforward to bound the error obtained from this approximate sampling in terms of the distance of $\theta$ from the optimum $\theta^*$. Denoting the difference $|\theta - \theta^*|$ by $\delta$, we get
\begin{align}
   | \log P(\data_i|\theta^*) - \log P(\data_i|\theta)| &=  |
    \frac{((\theta^*)^\top x_i - y_i)^2 }{2 \sigma^2} -  \frac{((\theta)^\top x_i - y_i)^2 }{2 \sigma^2}| \\
    & \leq 
    \frac{1}{2 \sigma^2} | (\theta^*)^\top x_i - \theta^\top x_i|^2 \\
    &\leq |( (\theta^*)^\top x_i)^2 - (\theta^\top x_i)^2| + |2y||\theta^\top x - (\theta^*)^\top x| \\
    &\leq |(\theta^* - \theta)^\top x + 2((\theta^*)^\top x)((\theta^*-\theta)^\top x)| + |2y||\theta^\top x - (\theta^*)^\top x| \\
    &\leq |\theta^* - \theta||x| + 2|\theta^* x||\theta^* - \theta||x| + |2y||x||\theta - \theta^*|
\end{align}
and so the error in the estimate of $\log P(\data | \theta)$ will be proportional to the distance $|\theta - \theta^*|$ induced by the approximate optimization procedure.
\end{proof}

\CorNTK*
\begin{proof}
Follows immediately from the results of \citet{he2020bayesian} stating that the the limiting distribution of $f^k_\infty$ is precisely $P(f|\data^n_{\le k}, \model)$. We therefore obtain the same result as for Theorem \ref{thm:sto}, plugging in the kernel gradient descent procedure on $f$ for the parameter-space gradient descent procedure on $\theta$.
\end{proof}
The following Lemma will be useful in order to prove Proposition~\ref{prop:modelselect}. Intuitively, this result states that in a linear regression problem in which each feature $\phi_i$ is `normalized' (the dot product $\langle \phi_i, y \rangle = \langle \phi_j, y \rangle = \alpha$ for some $\alpha$ and all $i, j$) and `independent' (i.e. $\langle \Pi_{y^\perp} \phi_i, \Pi_{y^\perp} \phi_j \rangle = 0$), then the optimal linear regression solution assigns highest weight to the feature which obtains the least error in predicting $y$ on its own.
\begin{lemma}
Let $y \in \mathbb{R}^n$, and $\Phi \in \mathbb{R}^{d \times d}$ be a design matrix such that $\Phi[:, j] = \alpha y + \epsilon_j \forall j$ for some fixed $\alpha \geq 0$, with $\epsilon \in y^\perp$, and $\epsilon_i^\top \epsilon_j = 0$ for all $i \neq j$. Let $w^*$ be the solution to the least squares regression problem on $\Phi$ and $y$. Then 
\begin{equation}
    \argmax_i w_i = \argmin_i \|f_i(x) - y\|^2 = \argmax_i \mathcal{L}(\model_i)
\end{equation}
\end{lemma}
\begin{proof}
We express the minimization problem as follows. We let $\phi(x)$ = $( f_1(x), \dots, f_k(x))$, where $f_i(x) = \alpha y + \epsilon_i$, with $\epsilon_i \perp \epsilon_j $. We denote by $\mathbbm{1}$ the vector containing all ones (of length $k$). We observe that we can decompose the design matrix $\Phi$ into one component whose columns are parallel to $y$, denoted $\Phi_y$, and one component whose columns are orthogonal to $y$, denoted $\Phi_\perp$. Let $\sigma^2_i = \|\epsilon_i\|^2$.  By assumption, $\Phi_y = \alpha y \mathbbm{1}^\top$, and $\Phi_\perp^\top \Phi_\perp = \text{diag}(\sigma^2_1, \dots, \sigma^2_n) = \Sigma$. We then observe the following decomposition of the squared error loss of a weight vector $w$, denoted $\ell(w)$.
\begin{align*}
\ell(w) &= \| \Phi w - y\|^2 = (\Phi w - y)^\top (\Phi w - y) \\
&= ((\Phi_y + \Phi_\perp) w - y)^\top ((\Phi_y + \Phi_\perp)w - y)\\
&=(\Phi_y w - y)^\top (\Phi_y w - y) + w^\top \Phi_\perp^\top \Phi_\perp w \\
&= \|y\|^2 \|1 - \alpha \mathbbm{1}^\top w \|^2  + \sum \sigma_i^2 w_i^2 \\
\end{align*}
In particular, the loss decomposes into a term which depends on the sum of the $w_i$, and another term which will depend on the norm of the component of each model's predictions orthogonal to the targets $y$.


As this is a quadratic optimization problem, it is clear that an optimal $w$ exists, and so $w^\top \mathbbm{1}$ will take some finite value, say $\beta$. It is straightforward to conclude that for any fixed $\beta$, the solution to the minimization problem
\begin{equation}
    \min_w \sum w_i^2 \sigma_i^2 : w^\top \mathbbm{1} = \beta
\end{equation}
is such that the argmax over $i$ of $w_i$ is equal to the index of the predictor with the minimal error. 
\end{proof}
\PropBMS*
\begin{proof}
We begin by emphasizing the necessity of the two conditions: the first condition ensures that all of the predictions are of the same scale. This avoids situations where one feature is equal to a small multiple of the targets $y$ (i.e. $
Phi[:, i] = \beta y$ for small $\beta$) resulting in a disproportionately large weight on that model's predictions in order to normalize their scale.  The second condition ensures that the models do not have complementary systematic errors, such that the optimal predictor might assign one model's predictions a higher weight in order to cancel out errors in other models' predictions.
This is satisfied when we require that $\epsilon_i \perp \epsilon_j$ be orthogonal, and that $\zeta_i^j$ be sampled independently for all $i$ and $j$. This assumption is crucial: in the case where there errors are linearly independent but not orthogonal, we can arrive at situations where greater weight may be assigned to a model whose error is complementary to those of other models, despite this model not being the best fit for the data.

We note that our lower bound for each model in the linear regression setting is equal to $\mathbb{E} \sum_{i=1}^N \|f_k(x_i) + \zeta_i - y_i\|^2 + c$ where $c$ is a fixed normalizing constant. By the previous Lemma, we know that the linear regression solution $w^*$ based on the posterior means satisfies, $\max_i w^*_i = \max_i \mathcal{L}(\model_i)$. It is then straightforward to extend this result to the noisy setting.
\begin{align}
    \mathbb{E}[ \|\Phi w - y\|^2] &= \mathbb{E}[\|(\Phi_y + \Phi_\perp + \zeta)w - y\|^2] \\
    &= \mathbb{E}[((\Phi_y + \Phi_\perp + \zeta)w - y)^\top ((\Phi_y + \Phi_\perp + \zeta)w - y)] \\
    &= \|\Phi_y w - y\|^2 + w^\top \Phi_\perp ^\top  \Phi_\perp w + \mathbb{E}[w^\top \zeta^\top \zeta w] \\
    &=  (w^\top \mathbbm{1} - \alpha)^2\|y\|^2 + w^\top \Phi_\perp ^\top  \Phi_\perp w + \mathbb{E}[w^\top \zeta^\top \zeta w] \\
    &= (w^\top \mathbbm{1} - \alpha)^2\|y\|^2  + \sum w_i^2( \|\Phi_\perp[:, i] \|^2 + \|\zeta_i\|^2)
\end{align}
We again note via the same reasoning as in the previous Lemma that the model with the greatest lower bound will be the one which minimizes $\|\Phi_\perp[:, i]\|^2 + \|\zeta_i\|^2$, and that the weight given to index $i$ will be inversely proportional to this term.

It only remains to show that for each model $i$, the model which maximizes $\mathcal{L}(M_i)$ will also minimize $\|\Phi_\perp[:, i]\|^2 + \|\zeta_i\|^2$. This follows precisely from the Gaussian likelihood assumption. As we showed previously 
\begin{align}
    \mathcal{L}(\data | \model_i) = \mathbb{E}[\sum \log P(y_i | \data_{<i})] &\propto - 
\sum \mathbb{E}[\ell_2(y_i - \hat{y}_i] \\
    &= [\| y - \mu\|^2 + \mathbb{E}[\|\hat{y} - \mu \|^2] \\
    &= \alpha\|y\|^2 + \|\Phi_\perp[:, i]\|^2 + \mathbb{E}[\|\zeta_i\|^2]
\end{align}
and so finding the model $\model_i$ which maximizes $\mathcal{L}(\data, \model_i)$ is equivalent to picking the maximal index $i$ of $w^*$ which optimizes the expected loss of the least squares regression problem.
\end{proof}
\clearpage

\section{Experiments}

\subsection{Experimental details: model selection using trajectory statistics} \label{sec:ex_ms_blr_synthetic_data}

We consider 3 model selection settings in which to evaluate the practical performance of our estimators. In \textbf{prior variance selection} we evaluate a set of BLR models on a synthetic linear regression data set. Each model $\mathcal{M}_i$ has a prior distribution over the $d$ parameters of the form $w \sim \mathcal{N}(0, \sigma_i^2 I_d)$ for some $\sigma_i^2$, and the goal is to select the optimal prior variance (in other words, the optimal regularization coefficient). We additionally evaluate an analogous initialization variance selection method on an NTK network trained on a toy regression dataset. In \textbf{frequency (lengthscale) selection} we use as input a subset of the handwritten digits dataset MNIST given by all inputs labeled with a 0 or a 1. We compute random Fourier features (RFF) of the input to obtain the features for a Bayesian linear regression model, and perform model selection over the frequency of the features (full details on this in the appendix). This is equivalent to obtaining the lengthscale of an approximate radial basis function kernel. In \textbf{feature dimension selection}, we use a synthetic dataset \citep{wilson2020bayesian} of the form $(\textbf{X}, \textbf{y})$, where $x_i = (y_i + \epsilon_1,  y_i + \dots, y_i + \epsilon_{15}, \epsilon_{16}, \dots, \epsilon_{30})$. We then consider a set of models $\{\model_k\}$ with feature embeddings $\phi_k(x_i) = x_i[1, \dots, k]$. The optimal model in this setting is the one which uses exactly the set of `informative' features $x[1, \dots, 15]$. 

The synthetic data simulation used in this experiment is identical to that used in \citep{wilson2020bayesian}. Below, we provide the details. 

Let $k$ be the number of informative features and $d$ the total number of features. We generate a datapoint $\data_i  = \{x_i,y_i\}$ as follows:
\begin{enumerate}
    \item {Sample $y_i$}: $y_i \sim U([0, 1])$
    \item {Sample $k$ informative features}: $x_{i,j} \sim N(y_i, \sigma_0) \quad \forall j \in 1, \dots k$
    \item {Sample $\max(d-k,0)$ noise features}: $x_{i,k+j} \sim N(0, \sigma_1) \quad \forall j \in 1, \dots d-k$
    \item {Concatenate the features}: $X_i= [x_{i,1}, \dots x_{i,d}]$
\end{enumerate}

We set $\sigma_0= \sigma_1=1$, $k = 15$, $n = 30$, and let $d$ vary from $5$ to $n$. We then run our estimators on the Bayesian linear regression problem for each feature dimension, and find that all estimators agree on the optimal number of features, $k$.

To compute the random fourier features used for MNIST classification, we vectorize the MNIST input images and follow the procedure outlined by \citet{rahimi2008random} (Algorithm 1) to produce RFF features, which are then used for standard Bayesian linear regression against the binarized labels. The frequency parameter (which can also be interpreted as a transformation of the lengthscale of the RBF kernel approximated by the RFF model) is the parameter of interest for model selection.

\subsection{Experimental details: Bayesian model comparison} \label{sec:exp_details_sgd_dnn}

Here we provide further detail of the experiment in Section 4.2.1.
The goal of the experiment is to determine whether the connection between sum-over-training losses (SOTL) and model evidence observed in the linear regression setting extends to DNNs. In particular, the two sub-questions are: 
\begin{enumerate}
    \item Do models with a lower SOTL generalize better?
    \item Are these models favoured by SGD? 
\end{enumerate}

To answer these questions, we train a linear combination of NNs. We can answer subquestion [1] by plotting the correlation between SOTL and test performance of an individual model. Further, we address subquestion [2] by considering the correlation between test loss and linear weights assigned to each model.

Below we explain the set-up of the linear combination in more detail. We train a variety of deep neural networks along with a linear `ensemble' layer that performs a linear transformation of the concatenated logit outputs\footnote{These are pre-softmax outputs. To obtain the predicted probability of a class, they are fed through a softmax function.} of the classification models. Let $h_m(x_i)$ be logit output of model $m$ for input $x_i$, $\ell(y_i, h_i)$ be the loss for point $i$ (where $h_i$ is a logit) and $w_{m,t}$ be the weight corresponding to model $m$ at time step $t$. 

We consider two training strategies: we first train models individually using the cross-entropy loss between each model's prediction and the true label, only cross-entropy loss of the final ensemble prediction to train the linear weights. Mathematically, we update the models using the gradients
\begin{equation}
    \frac{\partial}{\partial \theta_m} \ell(y_i, h_m(x_i)),
\end{equation}
and the `ensemble' weights using
\begin{equation}
    \frac{\partial}{\partial w_m} \ell( y_i, \sum_m w_m h_m(x_i)).
\end{equation}
We refer to this training scheme as \textit{Parallel Training} as the models are trained in parallel.  We also consider the setting in which the models are trained using the cross entropy loss from the ensemble prediction backpropagated through the linear ensemble layer, i.e. the model parameters are now updated using: 
\begin{equation}
    \frac{\partial}{\partial \theta_m} \ell(y_i, \sum_m w_m h_m(x_i)).
\end{equation}
We refer to this scheme as the \textit{Concurrent Training}. 

We train a variety of different MLPs (with varying layers,and nodes) and convolutional neural networks (with varying layers, nodes and kernels) on FashionMNIST using SGD until convergence.

\subsection{Experimental details: SGD upweights submodels that perform well} \label{sec:exp_details_sgd_submodels}
Below we provide further details of the experiment in Section 4.2.2. The goal of the experiment is to determine whether SGD upweights sub-models that fit the data better.  

We train a MLP network (with units $200, 200, 10$) on FashionMMIST using SGD until convergence.  After training is completed, for every class of $y$, we rank all nodes in the penultimate layer by the norm of their absolute weight (in the final dense layer).  We group the points into submodels according to their ranking --  the $k$ nodes with the highest weights are grouped together, next the $k+1, \dots 2k$ ranked nodes are grouped, etc. We set $k=10$. 

We determine the performance of a submodels by training a simple logistic classifier to predict the class of an input, based on the output of the submodel. To measure the performance of the classifier, we use the cross-entropy loss. To capture the equivalent notion of the AUC, we estimate the performance of the sub-models throughout training, and sum over the estimated cross-entropy losses. 

Below, we show additional plots for the \textit{parallel} and \textit{concurrent} training schemes. The results are the same to those presented in the main text, and we observe  [1] a negative correlation between test performance and ensemble weights and [2] a strong correlation between SOTL and average test cross-entropy.

\begin{figure}[H]
    \begin{minipage}{.27\textwidth}
    \includegraphics[ width=\linewidth]{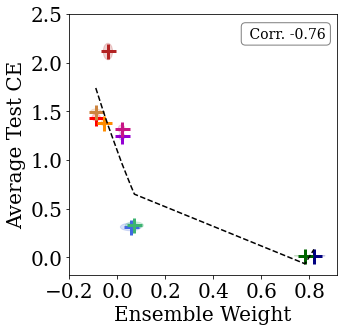}
    \end{minipage}
    \begin{minipage}{.27\textwidth}
    \includegraphics[ width=\linewidth]{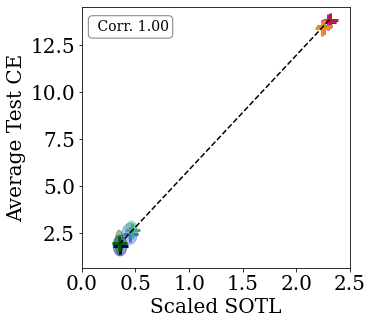}
    \end{minipage}
    \begin{minipage}{.27\textwidth}
    \includegraphics[ width=\linewidth]{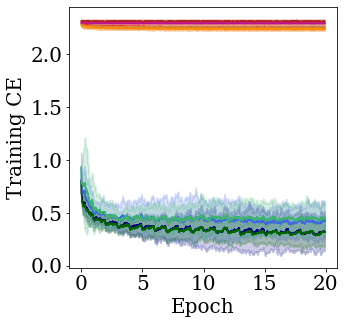}
    \end{minipage}
    \begin{minipage}{.17\textwidth}
    \includegraphics[ width=\linewidth]{figures/supervised/legend.png}
    \end{minipage}
    \caption{\textbf{Linear combinations of DNNs on FashionMNIST.}  Left: ensemble weights versus the test loss for parallel training; we observe a negative correlation. Middle: SOTL (standardized by the number of training samples) versus test loss for concurrent and concurrent training. We observe a strong correlation indicating that the SOTL generalizes well. Right: training curves for the different models in concurrent training schemes. All results are averaged over $10$ runs, and standard deviations are shown by the shaded regions around each observation. The model parameters, given in the parentheses, are the number of layers ($l$), nodes per layer ($n$) and kernel size ($k$), respectively. }
    \label{fig:mod_select_dnn_parallel}
\end{figure}

However, similarly to the linear setting, the difference in assigned weights is magnified in the concurrent training scheme. Here we find that in the concurrent training scheme, the ensemble focuses on training the CNNs (as can be seen from the training curve in Figure \ref{fig:mod_select_dnn} in the main text). This is likely because CNNs are able to learn more easily, leading to larger weights earlier on.

\chapter{Dynamics of reinforcement learning}

\section{Auxiliary results}
\label{sec:aux-results}
In this section, we state and prove some additional lemmas that are useful in proving the results stated in this chapter.

\begin{lemma}\label{lem:grassmann1}
    Let $x \in \mathbb{R}^d$, and let $(v_t)_{t \geq 0}$ be a sequence of vectors in $\mathbb{R}^d$ satisfying $v_t = f(t) x + o(f(t))$, for some function $f : [0, \infty) \rightarrow (0, \infty)$. Then $d(\langle v_t\rangle , \langle x\rangle ) \rightarrow 0$ as $t \rightarrow \infty$.
\end{lemma}
\begin{proof}
    The Grassmann distance $d(\langle v_t\rangle , \langle x\rangle )$ between two one-dimensional subspaces has a particular simple form, given by
    \begin{align*}
        d(\langle v_t \rangle , \langle x \rangle ) = \min\left( \arccos\left( \frac{\langle v_t, x \rangle}{\|v_t\| \|x\|} \right) , \arccos\left( \frac{\langle -v_t, x \rangle}{\|v_t\| \|x\|} \right) \right) \, .
    \end{align*}
    In our case, for sufficiently large $t$ this yields
    \begin{align*}
        d(\langle v_t\rangle , \langle x\rangle ) & = \arccos\left( \frac{\langle f(t) x + o(f(t)), x \rangle}{\| f(t) x + o(|f(t)|) \| \|x\|} \right) \\
        & = \arccos\left( \frac{\langle  x + o(1), x \rangle}{\| x + o(1) \| \|x\|} \right)\\
        & \rightarrow \arccos\left( \frac{\langle  x , x \rangle}{\| x \| \|x\|} \right)\\
        & = 0 \, .
    \end{align*}
\end{proof}

\begin{lemma}\label{lem:grassmann2}
    Let $U_1,\ldots,U_{|\mathcal{X}|}$ be a basis for $\mathbb{R}^{\statespace}$, let $K < |\mathcal{X}|$, and let $(a_{ij} |i \in [K], j \in [|\mathcal{X}|])$ be real coefficients.
    Let $0 < \beta_1 < \cdots < \beta_{|\mathcal{X}|}$
    , and consider time-dependent vectors $W_1(t),\ldots,W_d(t)$ defined by
    \begin{align*}
        W_i(t) = \sum_{j=1}^{|\mathcal{X}|} a_{ij} e^{-\beta_j t} U_j \, , \quad t \geq 0 \, .
    \end{align*}
    Then for almost all sets of coefficients $(a_{ij} |i \in [K], j \in [|\mathcal{X}|])$, we have
    \begin{align*}
        d(W_{1:K}(t) , U_{1:K}) \rightarrow 0 \, .
    \end{align*}
\end{lemma}

\begin{proof}
    Without loss of generality, we may take the vectors $U_1,\ldots,U_{|\mathcal{X}|}$ to be the canonical basis vectors. Under the assumptions of the theorem, we exclude initial conditions for which the matrix $A$ with $(i,j)$\textsuperscript{th} element $a_{ij}$ is not full rank. Note that under this condition, the matrix $A_t$ with $(k,i)$\textsuperscript{th} element $a_{ki}e^{\beta_i t}$ is also full rank for all but finitely many $t$. 
    By performing row reduction operations and scaling rows, for all such $t$ we may pass from $(W_{\repix}(t) \mid \repix \in [\repdim])$ to an alternative spanning set $(\widetilde{W}_{\repix}(t) \mid \repix \in [ \repdim])$ of the same subspace such that $\widetilde{W}_{\repix}(t) - U_\repix \in \langle U_{\repdim+1:|\mathcal{X}|}\rangle$, and $\|\widetilde{W}_\repix(t) - U_\repix\| = O(e^{-t(\beta_{\repdim+1} - \beta_\repix)}) = o(1)$. We therefore obtain an orthonormal basis for this subspace of the form $U_1 + o(1),\ldots, U_\repdim +o(1)$.
    
    We now use the singular value decomposition characterization of Grassmann distance in Definition~\ref{def:grassmann-distance}. Since we have obtained an orthonormal basis for the subspace $\langle W_\repix(t) \mid \repix \in [\repdim] \rangle$, the top-$K$ singular values of the matrix $(\sum_{\repix=1}^\repdim U_\repix U_\repix^\top)(\sum_{\repix=1}^\repdim (U_\repix + o(1))( U_\repix + o(1))^\top )$ determine the Grassmann distance. However, this matrix is equal to $\text{diag}(1,\ldots,1,0,\ldots,0) + o(1)$, with $K$ entries of $1$ in the diagonal matrix.
    But the top-$K$ singular values this matrix are $1+o(1)$, and so the principal angles between the subspaces are $o(1)$, and hence the Grassmann distance between the subspaces is $o(1)$, as required.
\end{proof}

\begin{restatable}{lemma}{lemmaRewardMatrix}\label{lem:reward-matrix}
For $M \in \mathbb{N}$, let $(r^m)_{m=1}^M$ be independent random variables drawn from some fixed mean-zero distribution in $\mathscr{P}(\mathbb{R}^{\statespace\times\actionspace})$ such that the covariance between coordinates $(x, a), (y, a)$ is $\Sigma_{xy}$, independent of $a \in \mathcal{A}$. Let $(\mathbf{w}^m)_{m=1}^M$ be independent random variables taking values in $\mathbb{R}^{\repdim \times \actionspace}$, with columns drawn independently from $\mathcal{N}(0, (1/M)I)$.
Then $\sum_{m=1}^M r^m (\mathbf{w}^m)^\top$ converges (in distribution) to a mean-zero Gaussian distribution over $\mathbb{R}^{\statespace \times \repdim}$, with independent columns, and individual columns having covariance matrix $\Sigma$.
\end{restatable}

\begin{proof}
    The proof simply follows by noting that $\sum_{m=1}^M r^m \mathbf{w}^m$ may be written $1/\sqrt{M} \sum_{m=1}^M r^m \varepsilon^m$, with $(\varepsilon^m)_{m=1}^\infty$ i.i.d.~$N(0,I)$ random variables. The individual terms have the desired mean and variance, and the resulting converge in distribution now follows from the central limit theorem.
\end{proof}

\begin{restatable}{lemma}{lemmaWLimit}\label{lem:w-limit}
For fixed $M$, let $(\mathbf{w}^m)_{m=1}^M$, $\mathbf{w}^m \in \mathbb{R}^d$, be sampled i.i.d. according to $\mathcal{N}(0, \frac{1}{M}I)$. Then the following hold.
\begin{equation}
    \lim_{M \rightarrow \infty} \sum_{m=1}^M \mathbf{w}^m (\mathbf{w}^m)^\top = I \text{ and } \lim_{M \rightarrow \infty} \sum_{m=1}^M \mathbf{w}^m \overset{D}{=} \epsilon \sim \mathcal{N}(0, I)
\end{equation}
\end{restatable}
\begin{proof}
We prove two results on the limit of $W = \sum_{m=1}^M \mathbf{w}^m (\mathbf{w}^m)^\top$ as $k \rightarrow \infty$. First
\begin{align*}
    \lim_{M \rightarrow \infty} \sum_{m=1}^M \mathbf{w}^m (\mathbf{w}^m)^\top &\overset{P}{=} I \, , \\ 
    \intertext{which we observe by evaluating an arbitrary diagonal and off-diagonal element of $\sum_{m=1}^M \mathbf{w}^m (\mathbf{w}^m)^\top$. For the diagonal terms, note that}
    \left(\sum_{m=1}^M \mathbf{w}^m (\mathbf{w}^m)^\top\right) [j, j] &= \sum_{m=1}^M (\mathbf{w}^m_{j})^2
\end{align*}
Now observe that
\begin{align*}
    \mathbb{E}\left\lbrack \sum_{m=1}^M (\mathbf{w}^m_j)^2 \right\rbrack &= M \frac{1}{M} = 1  \, , \text{ and } \quad 
    \text{Var} \left(\sum_{m=1}^M (\mathbf{w}^m_j)^2\right) = M \frac{1}{M^2} \rightarrow 0 \\
\end{align*}
Similarly, for the off-diagonal terms, let $j \not= \ell$. Then we have
\begin{align*}
    \left( \sum_{m=1}^M \mathbf{w}^m (\mathbf{w}^m)^\top\right) [j, \ell] &= \sum_{m=1}^M \mathbf{w}^m_j \mathbf{w}^m_\ell \, ,
\end{align*}
and further
\begin{align*}
    \mathbb{E}\left\lbrack \sum_{m=1}^M \mathbf{w}^m_j \mathbf{w}^m_\ell \right\rbrack  = 0 \, , \text{ and } \quad
    \text{Var}\left(\sum_{m=1}^M \mathbf{w}^m_\ell \mathbf{w}^m_j\right) &= M \frac{1}{M^2} \rightarrow 0; \quad
\end{align*}
The limit in probability is immediately implied by Chebyshev's inequality. The result on $\sum_{m=1}^M \mathbf{w}^m$ follows immediately from part 1 and the fact that a sum of Gaussian random variables is another Gaussian random variable whose mean and variance in this case will be a standard normal.
\end{proof}
\section{Proofs}
\label{sec:proofs-rl-dynamics}
\lemODESoln*

\begin{proof}
    Equation~\eqref{eq:value-function-ode-solution} can be verified as a solution to Equation~\eqref{eq:value-function-ode} by direct differentiation. Uniqueness of the solution follows since this is an autonomous initial value problem that satisfies the Lipschitz condition, and so the Picard-Lindelh\"of theorem applies.
\end{proof}

\propOneValueFunction*

\begin{proof}
    By Assumption~\ref{assume:value-function-conditions}, $P^\pi$ is diagonalisable, with eigenbasis $U_1,\ldots,U_{|\mathcal{X}|}$, with corresponding eigenvalues $\lambda_{1:|\mathcal{X}|}$ with strictly decreasing magnitudes $|\lambda_1| > \cdots > |\lambda_{|\mathcal{X}|}|$. We note then that $\exp-(t(I - \gamma P^\pi))$ is also diagonaisable under the same basis, with eigenvalues $\exp(t (\gamma \lambda_i - 1))$, for $i=1,\ldots,|\mathcal{X}|$. We may therefore expand $V_0$ with respect to this eigenbasis, and write
    \begin{align*}
        V_0 - V^\pi = \sum_{i=1}^{|\mathcal{X}|} \alpha_i U_i \, ,
    \end{align*}
    for some $\alpha_{1:|\mathcal{X}|} \in \mathbb{R}^{|\mathcal{X}|}$. Now note from the differential equation \eqref{eq:value-function-ode-solution}, we have
    \begin{align*}
        V_t - V^\pi = \exp(-t(I - \gamma P^\pi)) (V_0 - V^\pi) = \sum_{i=1}^{|\mathcal{X}|} \alpha_i \exp(t(\gamma \lambda_i - 1)) U_i \, .
    \end{align*}
    Note that as $P^\pi$ is a stochastic matrix, we have $|\lambda_i| \leq 1$ for all $i=1,\ldots,|\mathcal{X}|$, and hence $\exp(t(\gamma \lambda_i - 1)) \rightarrow 0$ for all $i=1,\ldots,|\mathcal{X}|$. Further, $\exp(t(\gamma \lambda_i - 1)) = o(\exp(t(\gamma \lambda_1 - 1)))$ for all $i=2,\ldots,|\mathcal{X}|$. 
    We make the additional assumption that $\alpha_1 \not= 0$, which makes the `almost every initial condition' assumption in the statement precise. Under this assumption, we therefore have
    \begin{align*}
        V_t - V^\pi &= \alpha_1 \exp(t(\gamma \lambda_1 - 1)) U_1 + \sum_{i=2}^{|\mathcal{X}|} \alpha_i \exp(t(\gamma \lambda_i - 1)) U_i\\
        &= \alpha_1 \exp(t(\gamma \lambda_1 - 1)) U_1 + o(\exp(t(\gamma \lambda_1 - 1))) \, .
    \end{align*}
    Then Lemma~\ref{lem:grassmann1} applies to give $d(\langle V_t - V^\pi \rangle, \langle U_1 \rangle) \rightarrow 0$, as required.
\end{proof}

\propManyValueFunctions*

\begin{proof}
    Expanding $V^{(\repix)}_0 - V^\pi$ with respect to $U_1,\ldots,U_{|\mathcal{X}|}$ for each $\repix=1,\ldots,|\mathcal{X}|$, we obtain expressions of the form
    \begin{align*}
        V^{(\repix)}_0 - V^\pi = \sum_{i=1}^{|\mathcal{X}|} a_{\repix i} U_i \, .
    \end{align*}
    By the ODE solution in Lemma~\ref{lem:ode-soln}, we then have
    \begin{align*}
        V^{(\repix)}_t - V^\pi = \sum_{i=1}^{|\mathcal{X}|} a_{\repix i}  e^{-t(1-\gamma\lambda_i)} U_i \, .
    \end{align*}
    We may now apply Lemma~\ref{lem:grassmann2} to obtain the desired result.
\end{proof}

\lemCoupledDynamics*

\begin{proof}
    This follows immediately by computing the derivatives in Equations~\eqref{eq:phi-ode} \& \eqref{eq:w-ode}, and so we omit the direct calculations.
\end{proof}

\thmInfiniteHeads*

\begin{proof}
We write the dynamics on $\Phi^M_t$ as follows and apply the results of Lemma~\ref{lem:w-limit}. We first consider the scaled initialization setting (implicitly setting the learning rate $\alpha=1$), where we find
\begin{align}
    \partial_t \Phi^M_t &= (I - \gamma P^\pi)\Phi_t^M \sum_{m=1}^M \mathbf{w}^m (\mathbf{w}^m)^\top + \sum_{m=1}^M R^{\pi} (\mathbf{w}^m)^\top \\
    \lim_{M \rightarrow \infty} \partial_t \Phi^M_t &= (I - \gamma P^\pi)\Phi_t^M \lim_{M \rightarrow \infty}\sum_{m=1}^M \mathbf{w}^m (\mathbf{w}^m)^\top  + \lim_{M \rightarrow \infty} R^\pi(\sum_{m=1}^M \mathbf{w}^m)^\top \\
    &\overset{D}{=} (I - \gamma P^\pi )\Phi_t^M I + R^\pi \epsilon^\top, \; \epsilon \sim \mathcal{N}(0, I).
\end{align}

We further observe that, for any finite interval, in the setting of zero reward we obtain \textit{uniform} convergence of the induced trajectory $\Phi_t^M$ to the trajectory of the limiting dynamics. We first observe that for a fixed initialization, we have that the induced dynamics are linear (in the zero-reward setting, affine otherwise) function of $\Phi^M_t$, and so 
\begin{align*}
    \partial_t \Phi^M_t &= (I - \gamma P^\pi) \Phi^M_t \sum_{m=1}^M w^m (w^m)^\top = \mathcal{L}^M \Phi^M_t  \\
    &\text{ where $\mathcal{L}^M(A) = (I - \gamma P^\pi) A \sum_{m=1}^M w^m (w^m)^\top$} \\
    \implies \Phi^M_t &= \exp(t \mathcal{L}^M)\Phi^M_0\; .
\intertext{Because the function $t \mapsto \exp (t A)$ is Lipschitz on a bounded interval for any $A$, this implies that for any finite interval $[0, T]$, the functions $t \mapsto \Phi^M_t$, as well as limiting solution, are $L$-Lipschitz for some $L$. Further, since the exponential is continuous, }
\lim_{M \rightarrow \infty} \Phi^M_t &= \lim_{M \rightarrow \infty} \exp(t \mathcal{L}^M) \Phi^M_0 = \exp(t \lim_{M \rightarrow \infty} \mathcal{L}^M) \Phi_0 \\
&= \exp (-t(I - \gamma P^\pi)) \Phi_0 = \Phi^\infty_t \;.
\intertext{Therefore, the functions $t \mapsto \Phi^M_t$ are $L$-Lipschitz and converge to the limit $\Phi^\infty_t$ on the interval $[0, T]$, which implies that they converge uniformly. }
\end{align*}

To evaluate the scaled learning rate setting, we observe that we now have
\begin{align}
    \partial_t \Phi^M_t &= \frac{1}{M} (I - \gamma P^\pi)\Phi_t^M \sum_{m=1}^M \mathbf{w}^m (\mathbf{w}^m)^\top + \sum_{m=1}^M R^{\pi} (\mathbf{w}^m)^\top \\
    \lim_{M \rightarrow \infty} \partial_t \Phi^M_t &= (I - \gamma P^\pi)\Phi_t^M \lim_{M \rightarrow \infty} \frac{1}{M}\sum_{m=1}^M \mathbf{w}^m (\mathbf{w}^m)^\top  + \lim_{M \rightarrow \infty} \frac{1}{M}R^\pi(\sum_{m=1}^M \mathbf{w}^m)^\top \\
    & = (I - \gamma P^\pi )\Phi_t^M I . \\
       \implies \lim_{M \rightarrow \infty} \Phi^M_t & = \exp(-t(I - \gamma P^\pi)) \Phi_0\, ,
\end{align}
almost surely. 
The principal difference between this and the scaled initialization setting is that here we divide the $R^\pi \mathbf{w}^\top$ term by $\frac{1}{M}$, whereas the scaled initialization is equivalent to scaling by $\frac{1}{\sqrt{M}}$. Therefore the scaled learning rate limit can be computed by the law of large numbers and converges in probability to its mean (zero), whereas under the scaled initialization it converges via the central limit theorem to a Gaussian distribution. 
\end{proof}

\propSubspaceConvergence*

\begin{proof}
    As described in the proof of Theorem~\ref{thm:infinite-heads}, we have $\Phi_t = \exp(-t(I - \gamma P^\pi))(\Phi_0 - \Phi_\infty) + \Phi_\infty$. Under Assumption~\ref{assume:value-function-conditions}, we may now apply an analogous argument as in Proposition~\ref{prop:many-value-functions} to the columns of $\Phi_t - \Phi_\infty$, and apply Lemma~\ref{lem:grassmann2} to obtain the desired result.
\end{proof}

\ThmDistribution*

\begin{proof}
We recall from Theorem~\ref{thm:infinite-heads} that the limiting dynamics follow the distribution
\begin{align}
    \lim_{t \rightarrow \infty} \lim_{M \rightarrow \infty} \Phi_t &\overset{D}{=} \lim_{t \rightarrow \infty} \exp(-t(I - \gamma P^\pi)) (\Phi_0 - (I -\gamma P^\pi)^{-1} Z_\Sigma) + (I - \gamma P^\pi)^{-1} Z_\Sigma\\
    & \overset{D}{=} (I - \gamma P^\pi) ^{-1} Z_\Sigma
\end{align}
for which we can straightforwardly apply known properties of Gaussian distributions: namely, that the distribution of a linear transformation $A$ of a Gaussian random variable with parameters $\mu, \Sigma$ is also Gaussian with mean $A\mu$ and covariance $A \Sigma A^\top$. Letting $A=(I - \gamma P^\pi)$ therefore gives the desired result.
\end{proof}

\propSubspaceConvergenceRC*

\begin{proof}
    As described in the proof of Theorem~\ref{thm:distribution}, we have $\Phi_t = \exp(-t(I - \gamma P^\pi))(\Phi_0 - (I - \gamma P^\pi)^{-1} Z_\Sigma ) + (I - \gamma P^\pi)^{-1} Z_\Sigma $. Under Assumption~\ref{assume:value-function-conditions}, we may now apply an analogous argument as in Proposition~\ref{prop:many-value-functions} to the columns of $\Phi_t - (I - \gamma P^\pi)^{-1} Z_\Sigma$, and apply Lemma~\ref{lem:grassmann2} to obtain the desired result.
\end{proof}

\section{Additional results from Table~\ref{tab:theory}}
\label{apx:table-results}
We begin this section by noting the following property of systems following linear dynamics.
\begin{lemma}\label{lem:dynamics-aux}
Let $\Phi_t \in \mathbb{R}^{\statespace \times M}$ follow the dynamics $\partial_t \Phi_t \overset{D}{=} A \Phi_t + B$, where $A$ is a linear operator for which all eigenvalues have negative real part, and $B$ is a vector. Then 
\begin{align}
    \lim_{t \rightarrow \infty} \Phi_t = -A^{-1}B \, .
\end{align}
Further, if $A$ is diagonalisable, with all eigenvalues of different magnitudes, 
\begin{align}
    \lim_{t \rightarrow \infty} d( \langle \Phi_t - \Phi_\infty \rangle, \langle U_{1:\repdim}(A) \rangle) = 0 \, ,
\end{align}
where $U_i(A)$ is the eigenvector of $A$ corresponding to the eigenvalue with $i$\textsuperscript{th} largest magnitude.
\end{lemma}
\begin{proof}
    We observe that the dynamics $\partial_t \Phi_t = A \Phi_t$ induce the trajectory 
    \begin{equation}
       \Phi_t = \exp(tA)\Phi_0 + (I - \exp(tA))(-A^{-1}B)\; ,
    \end{equation}
    with limit $\Phi_\infty = - A^{-1}B$. When $A$ is diagonalizable, we can therefore straightforwardly apply the results of Lemma~\ref{lem:grassmann2} to get that the limiting subspace will be characterized by the top $k$ eigenvectors of $A$. In the settings we are interested in, $A = - ( I - \gamma P^\pi)$ for some $\pi$ and some $\gamma$, and so the principal eigenvectors of $A$ will be the principal eigenvectors of $P^\pi$.
\end{proof}

The following two theorems characterize the learning dynamics under the past policies and multiple timescale auxiliary tasks listed in Table~\ref{tab:theory}. With these characterizations, it becomes straightforward to deduce $\Phi_\infty$ and the limiting subspace error as a direct consequence of the previous lemma.

\begin{theorem} \label{thm:pastpolicies-aux}
Let $\pi_1, \dots, \pi_L$ be a fixed set of policies. Given fixed $M$ and $L$, we define the indexing function $i_m = \lceil\frac{L}{m} \rceil$ for $m \in [1, M]$. Let $\Phi^M_t$ follow the dynamics
\begin{align}
    \partial_t \Phi_t^M &= \sum_{m=1}^M -( (I - \gamma P^{\pi_{i_m}})\Phi_t^M \mathbf{w}^m_t + R^{\pi_{i_m}})(\mathbf{w}^m_t)^\top
\end{align}
Then $\Phi_t^M$ satisfies the following dynamics and trajectory in the limit as $M \rightarrow \infty$, where $\bar{\pi} = \sum_{i=1}^L \pi_i$ and $\epsilon_i \in \mathbb{R}^d$ is an isotropic Gaussian with variance $\frac{1}{L}$. Note that we cannot naively average the rewards without changing the variance of the induced distribution unless $R^{\pi_i} = R^{\pi_j}$ for all $i,j$.
\begin{align}
    \lim_{M \rightarrow \infty} \partial_t \Phi_t^M &\overset{D}{=} -(I - \gamma P^{\bar{\pi}})\Phi_t + \sum_{i=1}^L R^{\pi_{i}} \epsilon_{i}\\
    \lim_{M \rightarrow \infty} \Phi_t^M &\overset{D}{=} \exp (-t(I - \gamma P^{\bar{\pi}} )) (\Phi_0 - \Phi_\infty ) + (I - \gamma P^{\bar{\pi}} )^{-1} \bigg ( \sum_{i=1}^L R^{\pi_{i}} \epsilon_{i}^\top \bigg )
\end{align}

\end{theorem}
\begin{proof}
The result on the trajectories follows immediately from the result on the dynamics, so it suffices to prove convergence of the dynamics. We approach this problem by decomposing the dynamics of $\Phi^M_t$ as follows.
\begin{equation}
    \partial_t \Phi_t^M = \sum_{m=1}^M - (I - \gamma P^{\pi_{i_m}})\Phi_t^M \mathbf{w}^m_t (\mathbf{w}^m_t)^\top - \sum_{m=1}^M R^{\pi_{i_m}}(\mathbf{w}^m_t)^\top \, .
\end{equation}

We first consider the random variables in the term which includes the rewards $R^\pi$. For this, we can directly apply the results from the previous theorems to the random variables $\epsilon_j = \sum_{m : i_m = j} w^m$, whose limiting variance is easily computed to be 
\begin{equation}
    \lim_{M \rightarrow \infty} \text{Var}\left(\sum_{m : i_m = j} \mathbf{w}^m\right) = \lim_{M \rightarrow \infty} \sum_{\lfloor \frac{j}{n}M \rfloor}^{\lfloor \frac{j+1}{n}M \rfloor} \frac{1}{M}I = \frac{1}{L}I \, . 
\end{equation}
For the term which depends on $\Phi_t$, we see
\begin{align}
    \sum_{m=1}^M - (I - \gamma P^{\pi_{i_m}} )\Phi^M_t \mathbf{w}^m_t (\mathbf{w}^m_t)^\top &= \sum_{i=1}^L \sum_{m : i_m = i}^M - (I - \gamma P^{\pi_{i_m}} )\Phi^M_t \mathbf{w}^m_t (\mathbf{w}^m_t)^\top \\
    &= \sum_{i=1}^L  - (I - \gamma P^{\pi_{i}} )\Phi^M_t \sum_{m : i_m = i}^M \mathbf{w}^m_t (\mathbf{w}^m_t)^\top  \, . \\
    \intertext{Since $L$ is finite and fixed, $\sum_{m:i_m=i}^M\mathbf{w}^m_t(\mathbf{w}^m_t)^\top$ converges to $\frac{1}{L}I$}
    & \underset{M \rightarrow \infty}{\longrightarrow} \sum_{i=1}^L - (I - \gamma P^{\pi_{i}} )\Phi^M_t \frac{1}{L}I \\
    &= -( I - \gamma \frac{1}{L}\sum_{i=1}^L P^{\pi_i})\Phi_t^M \\
    &= -(I - \gamma P^{\bar{\pi}} )\Phi_t^M \, .
\end{align}
And so the limiting distribution becomes
\begin{equation}
    \lim_{M \rightarrow \infty } \partial_t \Phi_t^M = -(I - \gamma P^{\bar{\pi}} )\Phi_t^M - \bigg ( \sum_{i=1}^L R^{\pi_{i}} \epsilon_{i} \bigg ) \, .
\end{equation}
\end{proof}
\begin{corollary}
The above result can be readily adapted to the setting in which each head predicts a randomly selected (deterministic) policy in MDPs with finite state and action spaces. Let $L = |\actionspace| ^{|\statespace|}$, $\{\pi_1, \dots, \pi_L\}$ be an enumeration of $\actionspace ^\statespace$, and $i_m$ denote the index of the policy randomly assigned to head $m$; then the above result still holds, and $\bar{\pi}$ is the uniform policy.
\end{corollary}
\begin{theorem} \label{thm:multiple-timescales}
We consider the task of predicting the value functions of a fixed policy under multiple discount rates $\gamma_1, \dots, \gamma_L$. For fixed $M$, $L$, let $i_m$ denote the indexing function defined in Theorem \ref{thm:pastpolicies-aux} Let $\Phi_t$ follow the dynamics
\begin{align}
    \partial_t \Phi_t^M &= \sum_{m=1}^M -( (I - \gamma_{i_m} P^{\pi})\Phi_t^M \mathbf{w}^m_t + R^{\pi})(\mathbf{w}^m_t)^\top \, .
\end{align}
Then the limiting dynamics as $M\rightarrow \infty$ of $\Phi_t^M$ are as follows, where $\bar{\gamma} = \sum \frac{1}{L} \gamma_i$
\begin{align}
    \lim_{M \rightarrow \infty} \partial_t \Phi_t^M &\overset{D}{=} -(I - \bar{\gamma} P^{\pi})\Phi_t + R^{\pi} \epsilon^\top\\
    \intertext{and}
    \lim_{M \rightarrow \infty} \Phi_t^M &\overset{D}{=} \exp (-t(I - \gamma P^{\bar{\pi}} )) (\Phi_0 - \Phi_\infty ) + (I - \gamma P^{\pi} )^{-1} R^{\pi} \epsilon^\top)\; .
\end{align}
\end{theorem}

\begin{proof}
We follow a similar derivation as for Theorem~\ref{thm:pastpolicies-aux} in deriving the component of the dynamics which depends on $\Phi^M_t$. The result of Theorem~\ref{thm:infinite-heads} immediately applies to the $\sum R^\pi (\mathbf{w}^m_t)^\top$ term:
\begin{align}
    \sum_{m=1}^M - (I - \gamma_{i_m} P^\pi )\Phi^M_t \mathbf{w}^m_t (\mathbf{w}^m_t)^\top &= \sum_{i=1}^L \sum_{m : i_m = i}^M - (I - \gamma_i P^\pi )\Phi^M_t \mathbf{w}^m_t (\mathbf{w}^m_t)^\top \\
    &= \sum_{i=1}^L  - (I - \gamma_i P^{\pi} )\Phi^M_t \sum_{m : i_m = i}^M \mathbf{w}^m_t (\mathbf{w}^m_t)^\top \, . \\
    \intertext{Since $L$ is finite and fixed, $\sum_{m:i_m=i}^M\mathbf{w}^m_t(\mathbf{w}^m_t)^\top$ converges to $\frac{1}{L}I$ as before:}
    & \underset{M \rightarrow \infty}{\longrightarrow} \sum_{i=1}^L - (I - \gamma_i P^{\pi} )\Phi^M_t \frac{1}{L}I \\
    &= -( I -\sum_{i=1}^L \frac{\gamma_i}{L}  P^{\pi})\Phi_t^M \\
    &= -(I - \bar{\gamma} P^{{\pi}} )\Phi_t^M \, .
\end{align}
\end{proof}

\section{Beyond diagonalisability assumptions}\label{sec:more-general-value-function-results}

In this section, we briefly describe extensions of the results of Chapter~\ref{chp:rl-dynamics} in scenarios where Assumption~\ref{assume:value-function-conditions} does not hold. There are two main cases we consider: (i) those in which $P^\pi$ is still diagonalisable, but does not have all eigenvalues with distinct magnitudes; and (ii) those in which $P^\pi$ is not diagonalisable.

In the former case, we do not have the different convergence rates of coefficients of different eigenvectors as in the proof of Proposition~\ref{prop:many-value-functions}. By similar arguments we can still deduce convergence of $V_t$ to the span of the eigenspaces with highest magnitude eigenvalues, but we can no longer deduce convergence to individual eigenspaces if there are several other eigenvalues with the same magnitude as the eigenvalue concerned. Note also that this includes the case where the matrix $P^\pi$ is complex- but not real-diagonalisable, since in such case non-real eigenvalues must come in conjugate pairs (which are necessarily of the same absolute value).

In the latter case, we no longer have an eigenbasis for $\mathbb{R}^{\statespace}$ based on $P^\pi$. However, we can consider the Jordan normal decomposition, and may still recover analogous results to those in Chapter~\ref{chp:rl-dynamics}, where convergence is now to the subspaces generated by \emph{Jordan blocks} with high absolute value eigenvalues. See \citet{parr2008analysis} for further commentary on Jordan normal decompositions in feature analysis.

\section{Extensions beyond one-step temporal difference learning}\label{sec:beyond-one-step}

Our analysis in Chapter~\ref{chp:rl-dynamics} has focused on the case of learning dynamics under one-step temporal difference learning. This choice is largely because one-step temporal difference learning is such a popular algorithm, not because the results do not hold more generally. In this section, we describe the elements of analogous results for $n$-step learning and TD($\lambda$) for interested readers. We focus on the case of value function dynamics, and believe extensions of the representation dynamics analysis in Chapter~\ref{chp:rl-dynamics} along these lines will be interesting directions for future work.

\subsection{Temporal difference learning with $n$-step returns}

In the case of $n$-step returns, the dynamics on the value function $(V_t)_{t \geq 0}$ are given by
\begin{align*}
    \partial_t V_t(x) = \mathbb{E}_\pi\left\lbrack \sum_{k=0}^{n-1} \gamma^k R_k + \gamma^n V_t(X_n)\middle| X_0 = x \right\rbrack - V_t(x) \, .
\end{align*}
In full vector notation, we have
\begin{align*}
    \partial_t V_t = -(I - \gamma^n (P^\pi)^n) V_t + \left\lbrack \sum_{k=0}^{n-1} (\gamma P^\pi)^k \right\rbrack R^\pi \, .
\end{align*}
The solution to this differential equation is
\begin{align*}
    V_t = \exp( -t (I - (\gamma P^\pi)^n ) )(V_0 - V^\pi) + V^\pi \, .
\end{align*}
This bears a close relationship with the result obtained for $1$-step temporal difference learning in Chapter~\ref{chp:rl-dynamics}. As expected, we obtain the same limit point. Further, under Assumption~\ref{assume:value-function-conditions}, $(P^\pi)^n$ has the same eigenvectors as $P^\pi$, and so results analogous to Propositions~\ref{prop:one-value-function} \& \ref{prop:many-value-functions} hold for $n$-step temporal difference learning too under these conditions.

\subsection{Temporal difference learning with $\lambda$-returns}

In the case of temporal difference learning with $\lambda$-returns (for $\lambda \in [0,1)$), the dynamics on the value function $(V_t)_{t \geq 0}$ are given by
\begin{align*}
    \partial_t V_t(x) = \mathbb{E}_\pi\left\lbrack \sum_{k=0}^{\infty} (\lambda\gamma)^k (P^\pi)^k ( R^\pi + \gamma P^\pi V_t(X_{k+1}) - V_t(X_k))\middle| X_0 = x \right\rbrack - V_t(x) \, .
\end{align*}
In full vector notation, we have
\begin{align*}
    \partial_t V_t = \sum_{k=0}^{\infty} (\lambda\gamma)^k (P^\pi)^k ( R^\pi + \gamma P^\pi V_t - V_t)
\end{align*}
The solution to this differential equation is
\begin{align*}
    V_t = \exp\left(t\left((1-\lambda) \sum_{k=1}^\infty \lambda^{k-1}\gamma^k (P^\pi)^k - I\right)\right) (V_0 - V^\pi) + V^\pi \, .
\end{align*}
As with $n$-step temporal difference learning, this bears a close relationship with the result obtained for $1$-step temporal difference learning in Chapter~\ref{chp:rl-dynamics}. As expected, we obtain the same limit point. Further, under Assumption~\ref{assume:value-function-conditions}, each $(P^\pi)^k$ has the same eigenvectors as $P^\pi$, and so results analogous to Propositions~\ref{prop:one-value-function} \& \ref{prop:many-value-functions} hold for $n$-step temporal difference learning too under these conditions.

\section{Bayes-optimality of RSBFs} \label{sec:bayes-opt}

We can develop the discussion of RSBFs beyond their properties as a matrix decomposition described in Section~\ref{sec:feature-selection} to observe that the RSBFs characterize the Bayes-optimal features for predicting an unknown value function given an isotropic Gaussian prior distribution on the reward, and further characterize a Bayesian posterior over value functions given by conditioning on the known dynamics of the MDP. We will denote by $V_K(\Psi)$ the top $K$ eigenvectors of the matrix $\Psi \Psi^\top$, i.e. the top $K$ left singular vectors of $\Psi$.
\begin{restatable}{corollary}{corrBayesOpt}
Under an isotropic Gaussian prior on reward function $r \in \mathbb{R}^{\statespace}$, the subspace $V_K(\Psi)$ corresponds to the optimal subspace with respect to the following regression problem.
\begin{equation}
    \min_{\Phi \in \mathbb{R}^{\statespace \times K}} \mathbb{E}_{r \sim \mathcal{N}(0, I)} \left\lbrack \| \Pi_{\Phi^\perp} (I - \gamma P^\pi)^{-1} r \|^2 \right\rbrack \, ,
\end{equation}
where $\Pi_{\Phi^\perp}$ denotes orthogonal projection onto the orthogonal complement of $\Phi$.
\end{restatable}

\begin{proof}
Let $S$ denote some subspace $S \subset V$.
    \begin{align}
        \mathbb{E}[\|\Pi_s \Psi r\|^2] &= \mathbb{E}[r^\top \Psi^\top \Pi_s^\top \Pi_s \Psi r]
        \intertext{We note that for any real symmetric matrix $A$ we can rewrite $A = \sum \alpha_i v_i v_i^\top$.}
        \mathbb{E}[r^\top \Psi^\top \Pi_s^\top \Pi_S \Psi r] &= \mathbb{E}[ r^\top (\sum \alpha_i v_i v_i^\top) r] = \mathbb{E}[\sum \alpha_i (r^\top v_i) (v_i^\top r)] \\
        &= \mathbb{E}[\sum \alpha_i v_i^\top r r^\top v_i] = \sum \alpha_i v_i^\top \mathbb{E}[r r^\top] v_i \\
        &= \sum \alpha_i v_i^\top v_i = \text{Tr}(\Psi^\top \Pi_S^\top \Pi_S \Psi) =\text{Tr}(\Psi^\top \Pi_S \Psi)
        \intertext{Finally, we can re-express the minimization problem as follows}
        \text{argmin}_{S: \text{Dim}(S) = k} \text{Tr}(\Psi^\top(\Pi_{S^\perp})\Psi) &= \text{argmax}_{S:\text{Dim}(S) = k} \text{Tr}(\Psi^\top \Pi_S \Psi)
        \intertext{Now, because the subspace spanned by the top $k$ left-singular vectors $\{u_1, \dots, u_k\}$ of $\Psi$ is known to be the maximizer of the above equation, we finally obtain}
        &= \langle u_1, \dots, u_k \rangle =  V_K(\Psi) \; .
    \end{align}
\end{proof}

\begin{corollary}
The limiting distribution of $\Phi^M_t$ under the random cumulant auxiliary task described in Theorem~\ref{thm:distribution} is equivalent to the Bayesian posterior over value functions obtained by conditioning on the dynamics $P^\pi$, and given a prior distribution on the reward function equal to $\mathcal{N}(0, \Sigma)$.
\end{corollary}
\begin{proof}
Each column of $Z_\Sigma$ is sampled from an isotropic Gaussian distribution, and therefore each feature $\phi_i \overset{D}{=} (I - \gamma P^\pi) \epsilon_i$. It therefore suffices to show that under a suitable prior distribution, the distribution of $\phi_i$ is equal to a Bayesian posterior. 
For this, it suffices to show that such a posterior can be obtained by conditioning on the transition dynamics $P^\pi$, and looking at the induced pushforward measure on the reward distribution. Noting that $(I - \gamma P^\pi)$ is invertible, we then obtain the following prior over $V^\pi$, assuming an isotropic Gaussian prior on $p_r(r)$ and any arbitrary distribution over potential transition dynamics $p_\pi(P^\pi)$ which covers $\mathbb{R}^{|S| \times d}$.
\begin{align}
    P(V^\pi) &= \int_{(r, P^\pi)} \mathbbm{1}[(I - \gamma P^\pi)^{-1}r = V^\pi] dp_r(r)dp_{\pi}( P^\pi) 
    \intertext{We observe that the random variable $V^\pi$ has conditional distribution $P(V^\pi|P^\pi) = P((I - \gamma P^\pi)^{-1}r)$, whose density is proportional to $p_r( (I - \gamma P^\pi)V)$ by the change of variables formula.}
    P(V^\pi | P^\pi) &= c p_r(r =  (I - \gamma P^\pi)V^\pi) \\
    \intertext{ Because our prior over $r$ is equal to the initialization distribution of $\epsilon_i$, we obtain}
    &= c p_{\text{init}}(\epsilon_i = (I - \gamma P^\pi) V^\pi) \\
    \intertext{ which is precisely the limiting distribution $p_\infty$ of $\phi_i$ (again applying the change of variables formula).}
    &= p_\infty (\phi_i = (I - \gamma P^\pi)^{-1} \epsilon_i = V^\pi)
\end{align}
So we see that the limiting distribution of $\phi_i$ is equal to the prior over value functions conditioned on the transition dynamics.
\end{proof}

\section{Further discussion of features and operator decompositions}\label{sec:feature-selection}

Proto-value functions (PVFs), were first defined by \citet{mahadevan2007proto} as the eigenvectors of the \textit{incidence matrix} induced by the environment transition matrix $P$. In the ensuing years, the term PVF has been used to refer to a number of related but not necessarily equivalent concepts. To clarify our use of the term and the relationship of our decompositions of the resolvent and transition matrices of an MDP, we provide a brief discussion here; a summary is provided in Table~\ref{table:features}.

We will use $A$ to refer to the adjacency matrix of the unweighted, undirected graph induced by the matrix $P$ (i.e. $A[i,j]$ is 1 if there exists some action with nonzero probability of taking the agent from state $i$ to state $j$ or from state $j$ to state $i$, and 0 otherwise). $L_G$ will refer to the graph Laplacian based on this matrix $A$. 

We can additionally consider the Laplacian of the weighted, directed graph defined by $P^\pi$; we will refer to this matrix as $L_{P^\pi}$, in reference to its dependence on the probability of transitioning. $T$ denotes the matrix defined by a collection of sampled transitions indexed by $t$, with entries $T_{it} = -1$ if the transition $t$ leaves $i$ and $+1$ if it enters state $i$.

Our first observation is that eigendecomposition and SVD are equivalent for symmetric matrices because any real symmetric matrix has an orthogonal eigenbasis; this means that performing either decomposition yields the same eigenvectors and easily related eigenvalues. Our second observation is that when $P^\pi$ is \textit{not} symmetric, its singular value decomposition and eigendecomposition may diverge; further, the relationship between the SVD of the resolvent matrix $\Psi = (I - \gamma P^\pi)^{-1}$ and of $P^\pi$ is no longer straightforward, despite the eigenspaces of the two matrices being analogous. This means that analysis of the singular value decomposition of $P^\pi$ does not immediately imply any results about the resolvent matrix. 

\begin{table}[!ht]
    \centering
    \begin{tabular}{c|c|c}
        Matrix & SVD & Eigendecomposition (ED)  \\
        \hline
        $L_G$ & PVFs \citep{mahadevan2007proto} & Equivalent to SVD  \\
        $T$ & sometimes $\equiv$ ED($L_G$) \citep{machado2017laplacian}  & not discussed \\
        $L_{P^\pi}$ & $\neq$ ED($L_{P^\pi}$)  & \citet{stachenfeld2014design} \\
        $(I - \gamma P^\pi)^{-1}$ &  RSBFs & $ \equiv L_{P^\pi}$ \\
        $P^\pi$ & \citet{behzadian2018feature} & $\equiv L_{P^\pi}$ \\
    \end{tabular}
    \caption{Summary of decompositions of various matrices associated with MDP transition operators, and associated features.}
    \vspace{0.5cm}
    \label{table:features}
\end{table}

Finally, we note that applying a uniform random walk policy may not be sufficient to guarantee that $P^\pi$ will be symmetric, and that in general it will not be possible to obtain a policy which will symmetrize the transition matrix. For example: when $G$ is a connected, non-regular graph (as is the case in many environments such as chains), there must be a node $v$ of degree $d$ adjacent to a node $v'$ of degree $d' \neq d$. A random walk policy will assign $p(v, v') = \frac{1}{d}$, while $p(v', v)$ will receive probability $\frac{1}{d'}$; thus, $P^\pi$ will not be symmetric. Fortunately, this is not a barrier to spectral analysis; the eigenvectors and eigenvalues of $P^\pi$ will still be real, as their transition matrix will be \textit{similar} to a symmetric matrix. We defer to \citet{machado2017laplacian} for a more detailed discussion of this relationship.

\section{Learning dynamics for ensemble prediction}
\label{sec:ensemble-dynamics}

We provide some visualizations of the induced behaviour on features as a result of training an ensemble with multiple heads and zero reward, replicating the analysis of Section \ref{sec:reps}, to highlight how the eigendecomposition of $P^\pi$ affects the learned representations. We run our evaluations on the Four-Rooms Gridworld by initializing $\Phi \in \mathbb{R}^{105 \times 10}$ (i.e. $|\statespace| = 105$ and the number of features $d=10$) and simulating the ODE defined in Equation~\ref{eq:ensemble-phi-flow} for time $t=100$ with transition matrix $P^\pi$ defined by the uniform random policy on this Gridworld. In some cases, the features converged to zero quickly and so we show a final $t < 100$ to highlight the behaviour of the representation before it reaches zero.

We consider three variables which we permit to vary: the initialization scheme of features, in one case sampled from an isotropic Gaussian \texttt{rand} or from a randomly initialized 2-layer MLP \texttt{nn}); whether the weight matrix is fixed at initialization \texttt{fix} or permitted to follow the flow defined by Equation~\ref{eq:w-ode} \texttt{train}; and finally the number of `heads', \texttt{M}=1, 20, and 200. 

In Figure \ref{fig:ensemble_predictions}, we plot the output of an arbitrary head $\mathbf{w}^m$ of the ensemble. In Figure \ref{fig:ensemble_feature0} we visualize the value of a single feature (i.e. a single column of $\Phi$).

We observe, as predicted, that for fixed heads in the overparameterized regime, the features (and the value functions they induce) converge to smooth eigenfunctions. We do not see meaningful convergence of the features trained in conjunction with a single weight vector. In contrast, the value functions and features trained in conjunction with ensembles with more heads than the feature dimension consistently resemble the eigenfunctions of $P^\pi$. When $(\mathbf{w}^m)$ are held fixed, we see convergence to smooth eigenfunctions as predicted by our theory; when $(\mathbf{w}^m)$ are permitted to vary according to the flow in Equation~\ref{eq:ensemble-w-flow}, we see convergence to the most eigenfunction corresponding to the most negative eigenfunction of $P^\pi$.

\begin{figure}[!ht]
    \centering
    \includegraphics[width=0.85\linewidth]{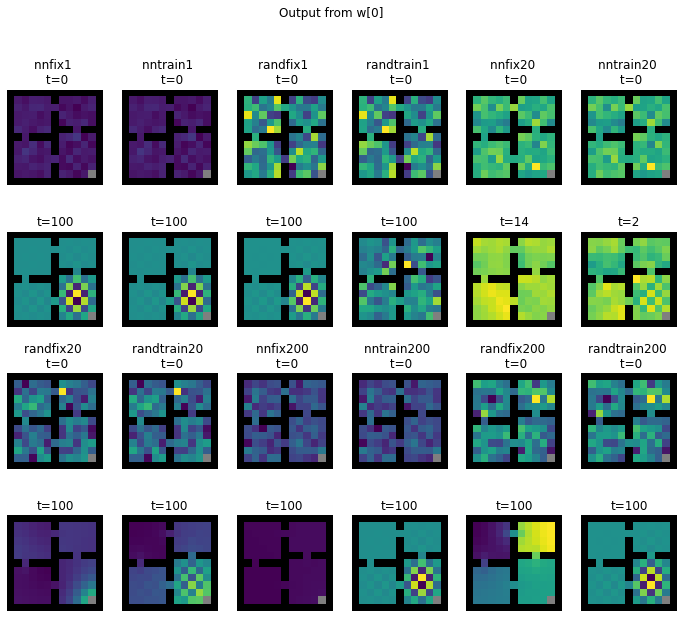}
    \caption{Value functions learned by the ensemble head at index 0 for different training regimes. Plot titles of form (feature initialization scheme, train/fix weight matrix, number of heads in ensemble). Observe that the representation learned with fixed weights tends to converge to smoother eigenfunctions than those learned with weights that are also allowed to train.  f}
    \label{fig:ensemble_predictions}
\end{figure}
\begin{figure}[!ht]
    \centering
    \includegraphics[width=0.85\linewidth]{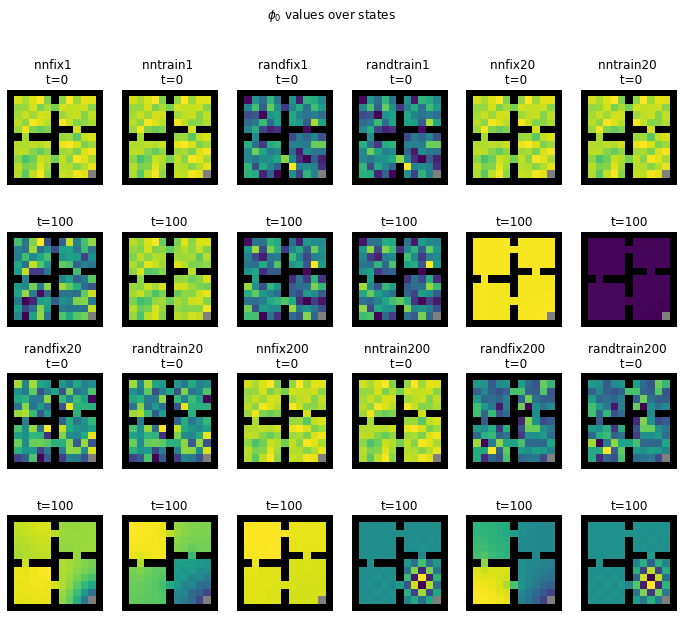}
    \caption{Values of ensemble feature at index 0 for different training regimes. Plot titles of form (feature initialization scheme, train/fix weight matrix, number of heads in ensemble). Observe that the representation learned with fixed weights tends to converge to smoother eigenfunctions than those learned with weights that are also allowed to train.}
    \label{fig:ensemble_feature0}
\end{figure}

\section{Experimental details}
\label{sec:experiment-details}

\subsection{Experimental details for Figure~\ref{fig:feature-viz}}

In our evaluations of the evolution of single feature vectors, we compute the continuous-time feature evolution defined in Equation~\eqref{eq:phi-ode}, using $P^\pi$ defined by a random walk on a simple Four-Rooms Gridworld with no reward. We use a randomly initialized representation $\Phi \in \mathbb{R}^{ |\statespace| \times 10}$, and use a single column of this matrix in our feature visualization (we observed similar behaviour in each feature). To compute trajectories, we use the SciPy ODE solver \texttt{solve\_ivp} \citep{2020SciPy}.

\subsection{Experimental details for Section~\ref{sec:feature-generalization}}

Here, we provide details of the environment used in producing Figure~\ref{sec:feature-generalization}. The environment is a 30-state chain, with two actions, \texttt{left} and \texttt{right}, which move the agent one state to the left or right, respectively. When the agent cannot move further left or right (due to being at an end state of the chain), the result of the corresponding action keeps the agent in the same state. There is additionally environment stochasticity of $0.01$, meaning that with this probability, a uniformly random action is executed instead. This stochasticity ensures that $P^\pi$ satisfies the conditions of Assumption~\ref{assume:value-function-conditions}. Taking the action \text{left} in the left-most state incurs a reward of $+2$, and taking the action \texttt{right} in the right-most state incurs a reward of $+1$; all other rewards are zero.

\subsection{Experimental details for Section~\ref{sec:deep-rl-aux}}

We modify a base Double DQN agent \citep{van2016deep} and evaluate on the ALE without sticky actions \citep{bellemare2013arcade}. Our agents are implemented in Jax \citep{jax2018github}, and are based on the DQN Zoo \citep{dqnzoo2020github}. Unless otherwise mentioned, all hyperparameters are as for the default Double DQN agent, with the exception of the epsilon parameter in the evaluation policy, which is set to 0.001 in all agents, and the optimizer, which for agents using auxiliary tasks CV, REM and Ensemble is Adam with epsilon $0.1/32^2$, and a lightly tuned learning rate; see below for further details.

Experimental results shown in bar plots, such as Figures~\ref{fig:naux}  and~\ref{fig:rc_sweep}, report a ``relative score'' which is the per-game score normalized by the maximum average score achieved by any agent or configuration. The same, per-game, normalization values are used for all such figures.

\textbf{Auxiliary task details.} In this section, we describe the implementations of all auxiliary tasks considered in the main text.
\begin{itemize}
    \item \emph{QR-DQN.} The implementation and hyperparameters match QR-DQN-1 in \citet{dabney2018distributional}.
    \item \emph{DDQN+RC.} We use a many-head DQN network which is identical to the standard neural network used for DQN, except that the output dimension is $(M+1) \times |\actionspace|$ instead of $|\actionspace|$, where $M$ is the number of auxiliary heads. Random cumulants are generated using a separate neural network with the same architecture as a standard DQN, but with output dimension equal to the number of auxiliary heads. The width of the Huber loss for each auxiliary head is equal to the number of auxiliary tasks. Let $\phi(x) \in \mathbb{R}^M$ be the output of the cumulant network given input observation $x$, with $M$ the number of auxiliary heads. Then the cumulant for auxiliary head $m$, at time step $t$, is given by $c_t = s \times (\phi(x_{t+1}) - \phi(x_t))$, where $s \in \mathbb{R}$ is a scaling factor. We performed a small hyperparameter sweep over scaling factors in $\{1, 10, 100, 500\}$, finding $s = 100$ to provide the best performance and use this value for all reported experiments. Note that this auxiliary task and the details are nearly identical to the \emph{CumulantValues} auxiliary task of \citet{dabney2020value}, except that we do not pass the values through a tanh non-linearity as this did not appear to have any impact in practice. We performed a hyperparameter sweep over learning rates and gradient norm clipping for this agent, considering learning rates $\{0.00025, 0.0001, 0.00005\}$ and gradient clipping in $\{10, 40\}$. We found that a learning rate of $0.00005$ and gradient norm clipping of $40$ to work best and use these values for all experiments.
    \item \emph{DDQN+REM.} We use a many-head variant of Double DQN, with heads trained according to the REM loss of \citet{agarwal2019striving}. For the agent's policy, an argmax over a uniform average of the heads is used. We swept over learning rates of $0.0001$ and $0.00005$, generally finding $0.00005$ to perform best. 
    \item \emph{DDQN+Ensemble.} As for the REM auxiliary task, we use a many-head variant of Double DQN. Each head is trained using its own double DQN loss, and the resulting losses are averaged. For the agent's policy, an argmax over a uniform average of the heads is used. We swept over learning rates of $0.0001$ and $0.00005$, generally finding $0.00005$ to perform best.
\end{itemize}

\textbf{Modified dense-reward games.} We modified four Atari games (Pong, MsPacman, Seaquest, and Q*bert) to obtain sparse, harder versions of these games to test the performance of random cumulants and other auxiliary tasks. The details of these games are given below. In each case a low-valued, commonly encountered reward is `censored', which means that during training the agent observes a reward of $0$ instead of the targeted reward. When evaluated, and thus for all empirical results reported, the standard uncensored rewards are reported.
\begin{itemize}
    \item \emph{Sparse Pong.} All negative rewards are censored (i.e. set to 0 before being fed to the agent), so the agent receives a reward of +1 for scoring against the opponent, but no reward when it concedes a point to the opponent. As $0$, $1$, and $-1$ are the only rewards in Pong, this modification makes Pong significantly harder. The agent can no longer learn to `avoid losing points`, but can only improve by learning to score points directly.
    \item \emph{Sparse MsPacman.} All rewards less than or equal to 10 are censored. This corresponds to rewards for the numerous small pellets that MsPacman eats, but not the larger pellets or ghosts. Each level ends when all of the small pellets are consumed, thus, by hiding these from the agent we may have significantly changed the primary incentive for the agent to advance the game.
    \item \emph{Sparse Seaquest.} All rewards less than or equal to 20 are censored. This corresponds to the rewards for shooting the sharks underwater, but not the rewards for picking up divers or surfacing. Additionally, even the rewards for sharks increase beyond this level, and thus become visible, once the agent has surfaced and collected enough divers.
    \item \emph{Sparse Q*bert.} All rewards less than or equal to 25 are censored. These are the rewards for flipping the colour of a tile, which is the primary source of reward and the mechanism for advancing to the next level of the game. Once all tiles are flipped, the agent will go to the next level. However, the agent can still observe rewards for going to the next level and for dispatching the enemies.
\end{itemize}

As described in the main text, we found that the sparse versions of MsPacman, Seaquest, and Q*bert were too difficult for any agent we tested to achieve a reasonable level of performance. In Figure~\ref{fig:sparse-learning-curves}, we display the performance of several auxiliary tasks on these games, noting that the performance achieved is extremely low in comparison to the agents trained on the standard versions of these games (see Section~\ref{sec:deep-rl-aux}).

\begin{figure}[!htb]
    \centering
    \null
    \hfill
    \includegraphics[keepaspectratio,width=\textwidth]{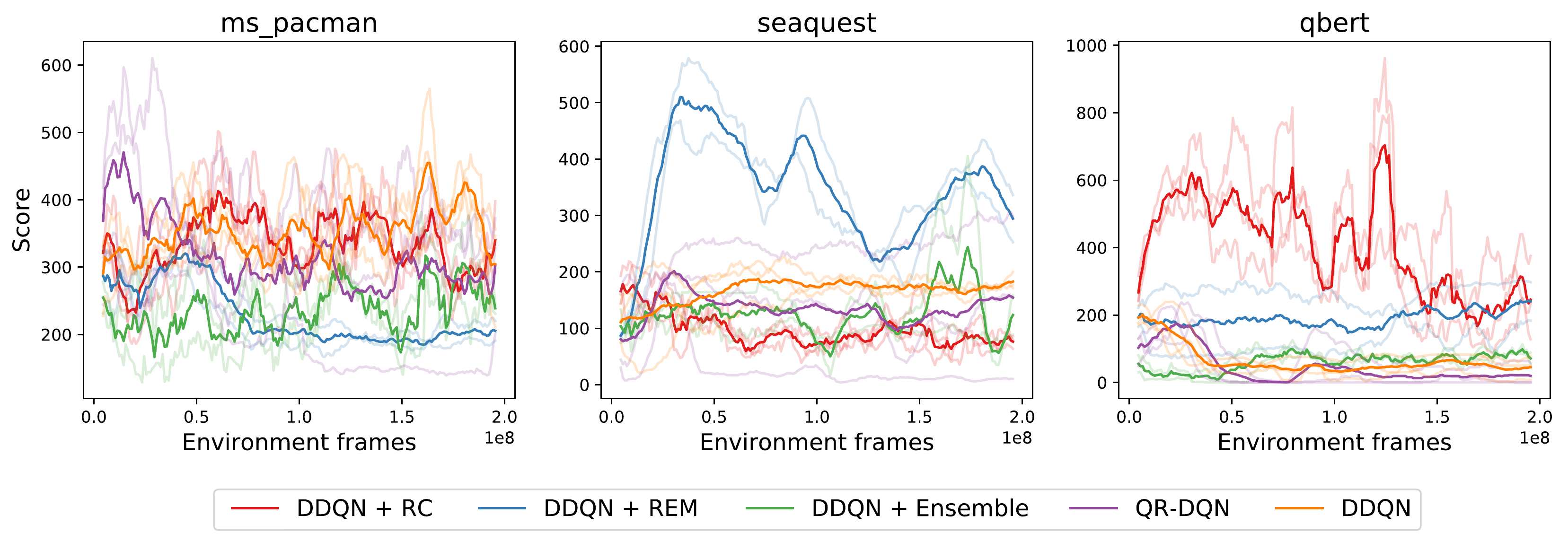}
    \caption{Learning curves on sparsified MsPacman (left), sparsified Seaquest (centre), and sparisifed Q*bert (right).}
    \label{fig:sparse-learning-curves}
\end{figure}

\textbf{Hyperparameter sweeps.} In Figure~\ref{fig:rc_sweep} we vary the weight of the auxiliary loss for the random cumulants agents, with the aim of understanding how this hyperparameters affect each method's performance. Next, in Figures~\ref{fig:ens_sweep} and~\ref{fig:rem_sweep} we present the results of a hyperparameter sweep for Ensemble and REM respectively. For these two, since there is no separate auxiliary loss as in RC, we vary number of heads and the learning rate. Results presented in the main text use the best settings for each algorithm found from these sweeps.

\begin{figure}[!htb]
    \centering
    \includegraphics[keepaspectratio,width=.75\textwidth]{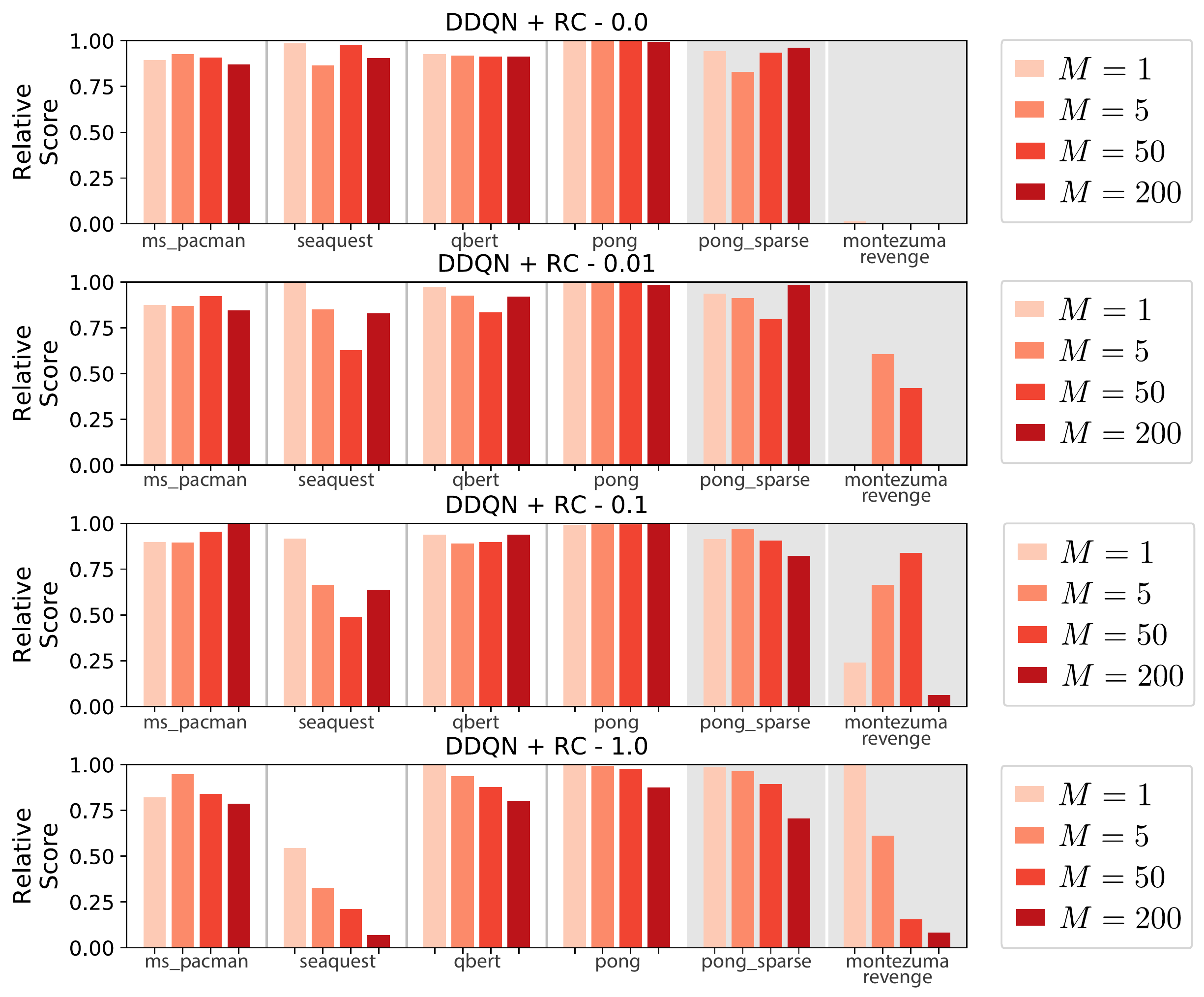}
    \caption{Results of hyper-parameter sweep for Random Cumulant (RC) method, where each row is for a different value of multiplicative scale applied to the auxiliary losses and each bar corresponds to the number of auxiliary heads ($M$). Note that the first row of results corresponds to initializing a network with the auxiliary heads, but setting the weight to zero, effectively disabling the auxiliary task.}
    \label{fig:rc_sweep}
\end{figure}

\begin{figure}[!htb]
    \centering
    \includegraphics[keepaspectratio,width=.75\textwidth]{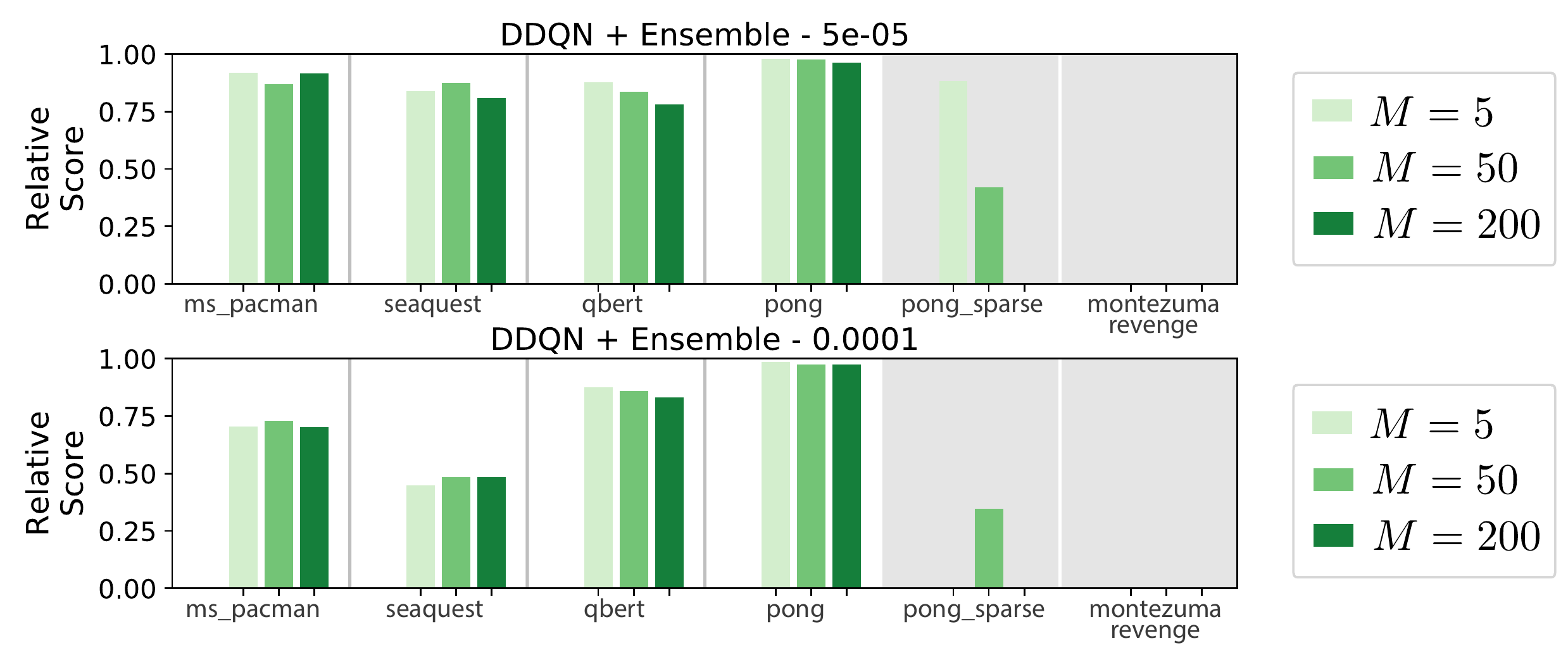}
    \caption{Results of hyper-parameter sweep for the Ensemble method, where each row is for a different learning rate and each bar corresponds to the number of auxiliary heads ($M$).}
    \label{fig:ens_sweep}
\end{figure}

\begin{figure}[!htb]
    \centering
    \includegraphics[keepaspectratio,width=.75\textwidth]{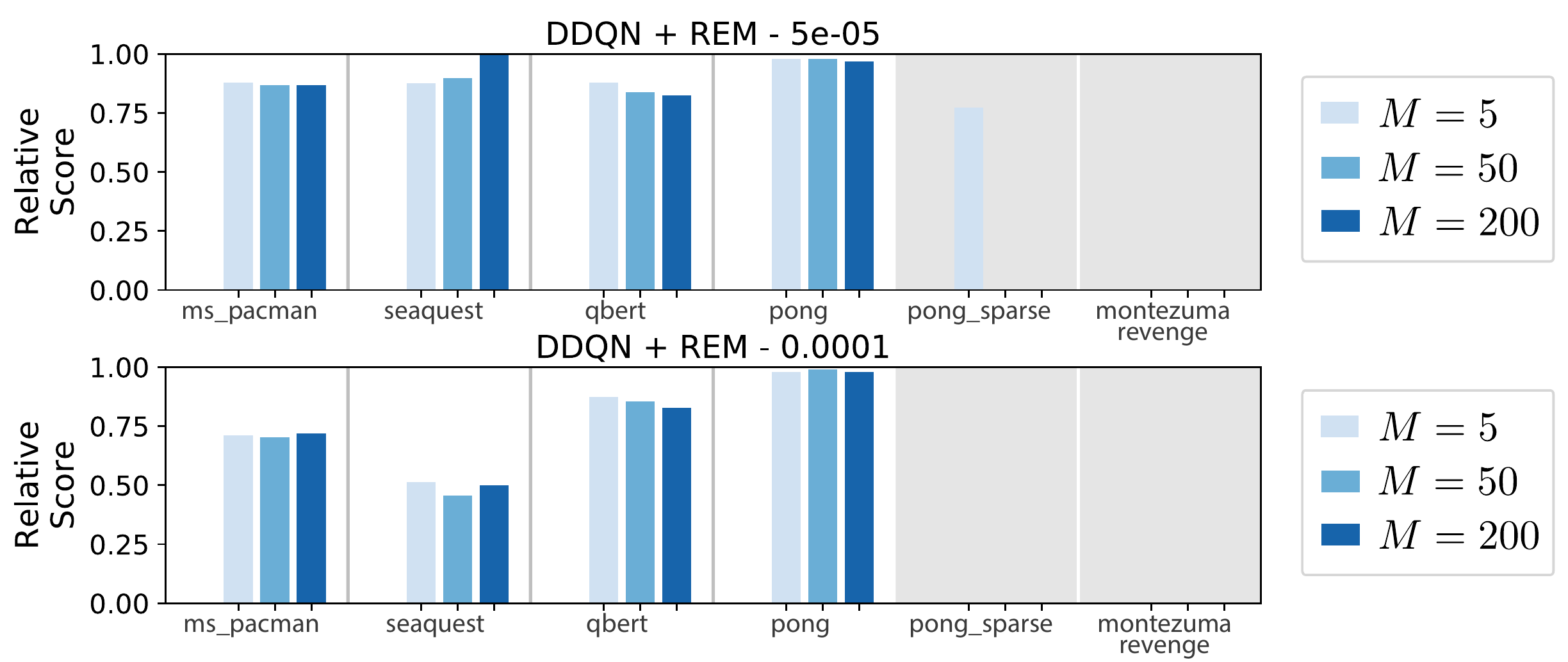}
    \caption{Results of hyper-parameter sweep for the REM method, where each row is for a different learning rate and each bar corresponds to the number of auxiliary heads ($M$).}
    \label{fig:rem_sweep}
\end{figure}

\chapter{Capacity loss}

\section{Proofs}
\subsection{Estimator consistency}
\label{appx:consistency}
We here show that our estimator of the agent's \effdim is consistent. First recall 

\begin{align}
\left(\frac{1}{\sqrt{n}} \Phi_n\right)^\top \left(\frac{1}{\sqrt{n}} \Phi_n\right) &= \frac{1}{n} \sum_{i=1}^n \phi(x_i) \phi(x_i)^\top \, . \\
\intertext{The following property of the expected value holds}
\mathbb{E}_{x \sim P} [ \phi(x)\phi(x)^\top ] &= \mathbb{E}\left\lbrack \frac{1}{n} \sum_{i=1}^n \phi(x_i)\phi(x_i)^\top   \right\rbrack \, . \\
\intertext{It is then straightforward to apply the strong law of large numbers. To be explicit, we consider an element of $M = \mathbb{E}[\phi \phi^\top] $, $M_{ij}$. }
\mathbb{E}[(\phi(x) \phi(x)^\top)_{ij}]  &= M_{ij} = \mathbb{E}[\phi_i(x) \phi_j(x)]
\implies \sum_{k=1}^n \frac{1}{n} \phi_i(x_k)\phi_j(x_k) \overset{a.s.}{\rightarrow} M_{ij}  \, .
\end{align}


Since we have convergence for any $M_{ij}$, we get convergence of the resulting matrix to $M$. Because the singular values of $\Phi$ are the eigenvalues of $M$ and the eigenvalues are continuous functions of that matrix, the eigenvalues of $M_n$ converge to those of $M$ almost surely. Then for almost all values of $\epsilon$, the threshold estimator $N(\lambda_1, \dots, \lambda_k; \epsilon) = | \{\lambda_i > \epsilon\} |$ will converge to $N(\text{spec}(M); \epsilon )$. Specifically, the estimator will be convergent for all values of $\epsilon$ which are not eigenvalues of $M$ itself. 

\subsection{Feature collapse case study: quantile regression}
\label{appx:feature_theory}
We apply similar analysis to that of Chapter~\ref{chp:rl-dynamics} to better understand the effect of sparse-reward environments on representation collapse. To do so, we return to the setting where $\Phi_t$ are non-parametric features and $w_t$ a linear function approximator which jointly parameterize a value function $V_t = \langle \Phi_t(x), w_t \rangle$. We recall the dynamics
\begin{equation}
    \partial_t \Phi_t = \alpha (\gamma P^\pi - I) \Phi_t (w_t w_t^\top) + R^\pi w_t^\top
\end{equation}
and
\begin{equation}
    \partial_t w_t = \beta \Phi_t^\top [(\gamma P^\pi - I)\Phi_t w_t + R^\pi] \, ,
\end{equation}
where $P^\pi \in \mathbb{R}^{\mathcal{X} \times \mathcal{X}}$ is the matrix of state-transition probabilities under $\pi$, and $R^\pi \in \mathbb{R}^{\mathcal{X}}$ is the vector of expected rewards.

In value-based deep RL, we model a Q-function which takes as input an observation $\bx$ and outputs a vector $Q(s, a_i)_{i=1}^{n_a} \in \mathbb{R}^{n_a}$. This resembles the setting of \textit{ensemble prediction}, whose dynamics we also recall here.
\begin{align}
    \partial_t \Phi^{M}_t 
    \!=  &   \alpha\! \sum_{m=1}^M (R^\pi\! +\! \gamma P^\pi \Phi^{M}_t w_t^{m}\! -\! \Phi^{M}_t w_t^{m})  (w_t^{m} )^\top \, , \\
    \partial_t w_t^{m} = & \beta (\Phi^{M}_t)^\top (R^\pi + \gamma P^\pi \Phi^M_t w^{m}_t - \Phi_t w^{m}_t ) \, . 
\end{align}

The ensemble learning regime also bears similarity to the quantile regression DQN (QR-DQN) objective \citep{dabney2018distributional}, which learns to fit a set of quantiles to the distribution taken by the return when treated as a random variable. Each quantile $\tau_i \in (0, 1)$ is trained using a quantile regression loss, of the form
\begin{equation*}
    \mathcal{L}^{\tau}_{\mathrm{QR}}(\theta) := \mathbb{E}_{\hat{Z} \sim Z} [ \rho_\tau(\hat{Z} - \theta)] \quad \text{ where } \rho_\tau(u) = u(\tau - \delta_{u < 0}), \; \forall u \in \mathbb{R}.
\end{equation*}
This loss is known to yield unbiased gradients, however it has the undesirable property that the gradients of $\mathcal{L}^{\tau}_{\mathrm{QR}}$ remain constant even as $\theta \rightarrow \tau$, which can result in pathological convergence properties. The QR-DQN objective addresses this by implementing a Huber loss \citep{huber64robust}, whose gradients tend to zero at its minimum. 
\begin{equation}
\mathcal{L}_{\kappa}(u) = \begin{cases}
\frac{1}{2}u^2 & \text{if }|u| \leq \kappa \\
\kappa(|u| - \frac{1}{2}\kappa) & \text{otherwise}.
\end{cases}
\end{equation}

For a sufficiently large number of quantiles and under similar conditions as in Theorem~\ref{thm:infinite-heads} the dynamics induced by this objective bear marked resemblance to those of the ensemble prediction setting. In particular, when no reward is observed in the environment, we again obtain that the features will converge to the zero vector. The setting of this result is distinct from that of deep neural network representation dynamics, as neural networks use discrete optimization steps, finite learning rates, and typically do not leverage linear ensembles. However, we emphasize two crucial observations that suggest the intuition developed in this setting may be relevant: first, in sparse reward environments the representation will be pushed to zero along dimensions spanned by the linear weights used to compute outputs. Once sufficiently many independent weight vectors are being used to make predictions, this effectively forces \textit{every} dimension of the representation to fit the zero vector output. We would therefore expect representation collapse to be particularly pronounced in the QR-DQN agents trained on sparse-reward environments, as in this setting we obtain many independently initialized heads all identically  trying to fit the zero target. 

\begin{restatable}{theorem}{thmQrdqn}\label{thm:qr-dqn}
Let let $\bm{\tau}=(\tau_i)_{i=1}^M$ be a set of quantile functions $\tau_i : \statespace \rightarrow \mathbb{R}$ with corresponding estimators $(\hat{\tau}^i)_{i=1}^M$ of the form $\hat{\bm{\tau}} = \hat{\tau}^i(\bx) =\langle \phi(\bx),  \bw^i \rangle$ for $\bw^i \in \mathbb{R}^d$, $\phi: \statespace \rightarrow \mathbb{R}^d$. Let $\pi$ be some policy in an MDP $\mdp$, with $R^\pi = \mathbf{0}$, and $P^\pi$ be the corresponding transition function. Let $\Phi^M_t$, $\bw^i_t$ follow the dynamics
\begin{align}
    \partial_t \Phi_t &= \beta^M \nabla_\Phi\qrdqnlosstau\\
    \partial_t \bw_t &= \alpha^M \nabla_{\bw}\qrdqnlosstau\;.
\end{align}
Then if $\alpha^M = 0$ or $\alpha^M = O(\frac{1}{M})$, letting $\Phi_\infty = \lim_{t \rightarrow \infty} \Phi_t$, we obtain
\begin{equation}
    \lim_{M \rightarrow \infty} \Phi_{\infty} \rightarrow \bm{0}.
\end{equation}
\end{restatable}

\begin{proof}
We break the derivation of this result into two cases. We first observe that, when $\alpha=0$ and $\beta=1$, the QR-DQN update reduces $\|V_t\|_\infty$ by either $\kappa$ if $\|(\gamma P^\pi - I) V_t\|_\infty > \kappa$ or $\gamma \|V_t\|_\infty$ if $\|(\gamma P^\pi - I ) V_t\|_\infty < \kappa$. In the latter case, the QR-DQN objective is identical to the TD updates studied in Chapter~\ref{chp:rl-dynamics}, and we obtain convergence of $\Phi$ to zero in the case of infinitely many heads, and in the case of $M=d$ heads with orthogonal initialization aligned with the features.

In the case where $M \rightarrow \infty$, when $w^i$ are initialized to scale with $\frac{1}{\sqrt{M}}$, we obtain straightforwardly that $\|\Phi_0 w^i\| \overset{P}{\rightarrow} 0$. Thus, the limiting probability that any error falls outside the $\kappa$ cutoff of the Huber loss will be zero, and the dynamics on $\Phi$ will be identical to the TD dynamics studied previously, allowing us to apply the result of Theorem~\ref{thm:infinite-heads}.

In the case where $M=d$ and the $w^i$ are initialized such that $\langle \phi^i, w^j \rangle = \delta_{i,j} \phi^i$, we need only show that the dynamics followed by quantile regression are sufficiently well-behaved to push $\|V_t\|_\infty < \kappa$ for some finite $t$. 

We note that in this case we can express the dynamics as follows, and for the rest of this discussion assume without loss of generality that $\phi^i_t(x) > 0 \forall x \in \statespace$.
\begin{align}
\partial_t \Phi_t &= - \sum_{i=1}^M \min ( (I - \gamma P^\pi)\Phi_t w^i, \kappa) (w^i)^\top \\
&= - \sum_{i=1}^M \min ( (I - \gamma P^\pi)\phi^i_t w^i, \kappa) (w^i)^\top \\
\implies \partial_t \phi^k_t &= -  -\min( (I - \gamma P^\pi) \phi^k_t w^k, \kappa) (w^k)^\top \\
&= -\min ((I - \gamma P^\pi) \phi^k_t, \kappa)
\end{align}

Let $\phi^i_t$ be a feature vector at some time $t$. Let $x_{\max}$ maximize $\{|\phi^i_t(x)| :x \in \statespace \}$. Then we know that the time-derivative on $\partial_t \phi^i_t(x_{\max})$ must have the opposite sign to $\phi^i_t(x_{\max})$, as it takes the following form.
\begin{align}
    \partial_t \phi^i_t(x_{\max}) &= - \min( (I - \gamma P^\pi)\phi^i_t(x_{\max}) , \kappa) \\
    \mathrm{sign}(\partial_t \phi^i_t(x_{\max}) &= - \mathrm{sign}(\phi^i_t(x_{\max})
\end{align}

Now, because this applies to the maximal value of $\phi^i_t$ over all states, we get that the function $\tilde{\phi}^{i}(t) = \max_{x \in \statespace} \phi^i_t(x)$ is decreasing. The rate of this decrease is either constant, in which case it is equal to $\kappa$, or we get the standard exponential convergence described in \ref{eq:td_dynamics}. 
\end{proof}

\keyinsight{In the presence of many prediction targets, such as those induced by a quantile regression objective, sparse-reward environments will induce low dimensional feature representations in deep RL agents.}


\section{Sequential supervised learning}
\label{appx:supervised}

\subsection{Details: target-fitting capacity on non-stationary MNIST}
\label{appx:mnist-details}

In addition to our evaluations in the Atari domain, we also consider a variant of the MNIST dataset in which the labels change over the course of training. 
\begin{itemize}[leftmargin=0.5cm]
    \item \textbf{Inputs and Labels:} We use 1000 randomly sampled input digits from the MNIST dataset and assign either binary or random targets.
    \item \textbf{Distribution Shift:} We divide training into $N=30$ or $N=10$ iterations depending on the structure of the target function. In each iteration, a target function is randomly sampled, and the network's parameters obtained at the end of the previous iteration are used the initial values for a new optimization run. We use the Adam \citep{kingma2014adam} optimizer with learning rate \texttt{1e-3}, and train to minimize the mean squared error between the network outputs and the targets for either 3000 or 5000 steps depending on the nature of the target function.
    \item \textbf{Architecture:} we use a standard fully-connected architecture with ReLU activations, and vary the width and depth of the network. The parameters at the start of the procedure are initialized following the defaults in the Jax Haiku library.
\end{itemize}

We note that the dataset sizes, training budgets, and network sizes in the following experiments are all relatively small. This was chosen to enable short training times and decrease the computational budget necessary too replicate the experiments. The particular experiment parameters were selected to be the fastest and cheapest settings in which we could observe the capacity loss phenomenon, while still being nontrivial tasks. In general, we found that capacity loss is easiest to measure in a 'sweet spot' where the task for a given architecture is simple enough for a freshly-initialized network to attain low loss, but complex enough that the network cannot trivially solve the task. In the findings of the following section, we see how some of the larger architectures do not exhibit capacity loss on `easier' target functions, but do on more challenging ones that exhibit less structure. This suggests that replicating these results in larger networks will be achievable, but will require re-tuning the task difficulty to the larger network's capacity.

\subsection{Additional evaluations}
\label{appx:cap-loss-supervised}
\begin{figure}
    \centering
    \includegraphics[width=\linewidth]{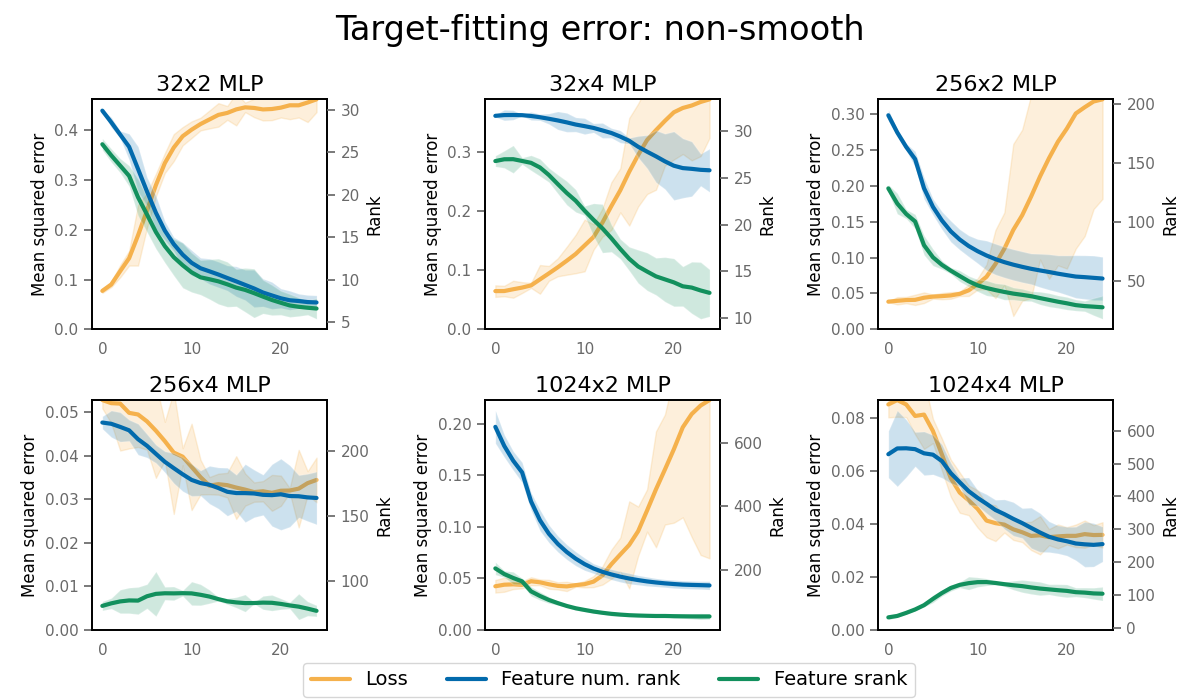}
    \caption{Mean squared error at the end of training on each iteration of the \textbf{hash-MNIST} task. Target-fitting error increases over time in smaller networks, but increasing the depth or width of the network slows down capacity loss, enabling positive transfer in the largest networks we studied.}
    \label{fig:hash-mnist}
\end{figure}
We expand on the MNIST target-fitting task shown in the main text by considering how network size and target function structure influences capacity loss. 

\begin{itemize}
    \item \textbf{Random-MNIST} (smooth) this task uses the images from the MNIST dataset as inputs. The goal is to perform regression on the outputs of a randomly initialized, fixed neural network. We use a small network for this task, consisting of two width-30 fully connected hidden layers with ReLU activations which feed into a final linear layer which outputs a scalar. Because the network outputs are small, we scale them by 10 so that it is not possible to get a low loss by simply predicting the network's bias term. This task, while randomly generated, has some structure: neural networks tend to map similar inputs to similar outputs, and so the inductive bias of the targets will match that of the function approximator we train on them. 
    \item \textbf{Hash-MNIST} (non-smooth) uses the same neural network architecture as the previous task to generate targets, however rather than using the scaled network output as the target, we multiply the output by 1e3 and feed it into a sine function. The resulting targets no longer have the structure induced by the neural network. This task amounts to memorizing a set of labels for the input points.
    \item \textbf{Threshold-MNIST} (sparse) replaces the label of an image with a binary indicator variable indicating whether the label is smaller than some threshold. To construct a sequence of tasks, we set the threshold at iteration $i$ to be equal to $i$. This means that at the first iteration, the labels are of the form $(x,0)$ for all inputs $x$. At the second iteration, they are of the form $(x, \delta(y<1) )$, where $y$ is the digit in the image $x$, and so on.
\end{itemize}
We consider MLP networks of varying widths and depths, noting that the network architecture used to generate the random targets is fixed and independent of the approximating architecture. We are interested in evaluating whether factors such as target function difficulty, network parameterization, and number of target functions previously fit influence the network's ability to fit future target functions. Our results are shown in Figure~\ref{fig:hash-mnist},~\ref{fig:random-mnist}, and \ref{fig:threshold-mnist}. We visualize srank and $\effdim$ of the features output at the network's penultimate layer, in addition to the loss obtained at the end of each iteration. 

\begin{figure}
    \centering
    \includegraphics[width=\linewidth]{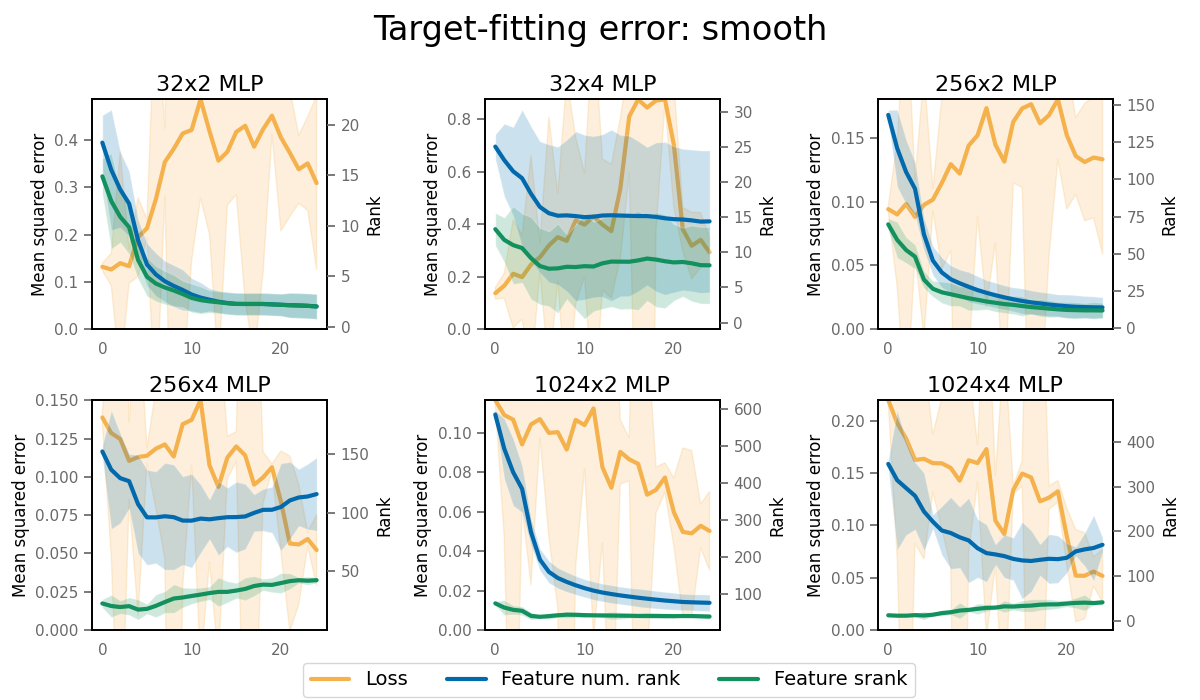}
    \caption{Mean squared error after 2e3 training steps on the \textbf{random-MNIST} task. Target-fitting error increases over time in under-parameterized networks, but increasing the depth or width of the network slows down capacity loss, enabling positive transfer in the largest network we studied.}
    \label{fig:random-mnist}
\end{figure}
\begin{figure}
    \centering
    \includegraphics[width=\linewidth]{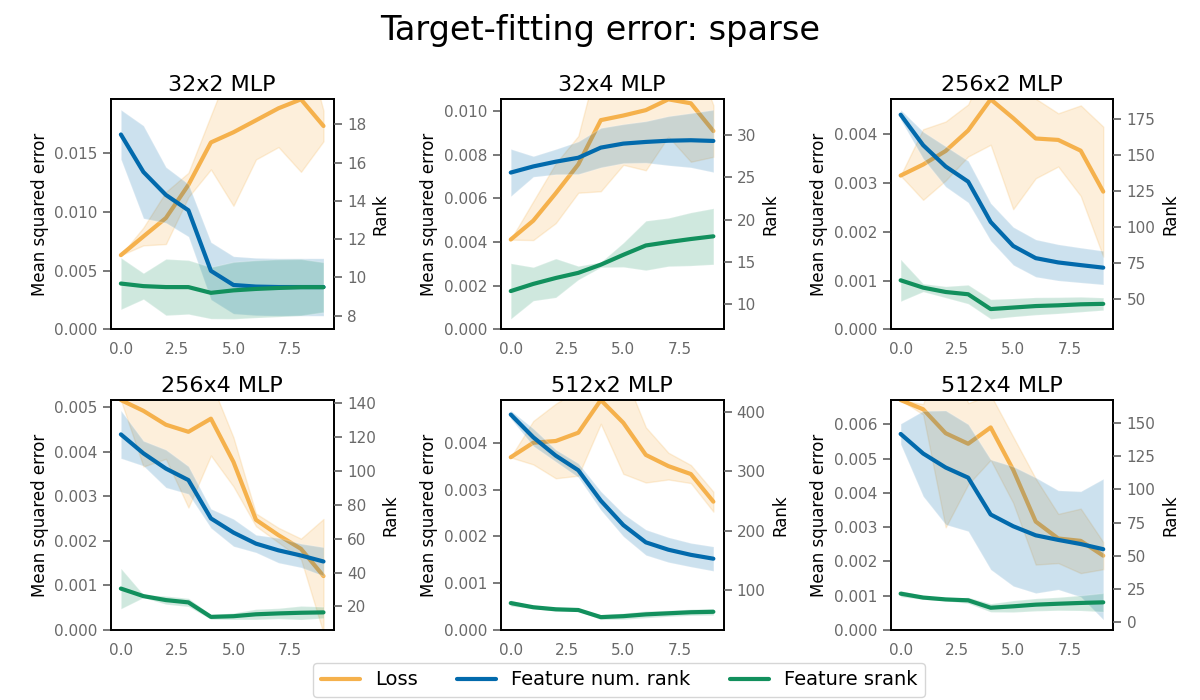}
    \caption{Mean squared error after 2e3 training steps on the \textbf{threshold-MNIST} task. Target-fitting error increases over time in under-parameterized networks, but increasing the depth or width of the network slows down capacity loss, enabling positive transfer in the largest network we studied.}
    \label{fig:threshold-mnist}
\end{figure}
\subsection{Effect of \pyoi on target-fitting capacity in MNIST}
\label{appx:infer-mnist}
In addition to our study of the Atari suite, we also study the effect of \pyoi on the non-stationary MNIST reward prediction task with a fully-connected architecture; see Figure~\ref{fig:pyoi_on_mnist}. We find that it significantly mitigates the decline in target-fitting capacity demonstrated in Figure~\ref{fig:mnist-cap}.

\begin{figure}
    \centering
    \includegraphics[width=\linewidth]{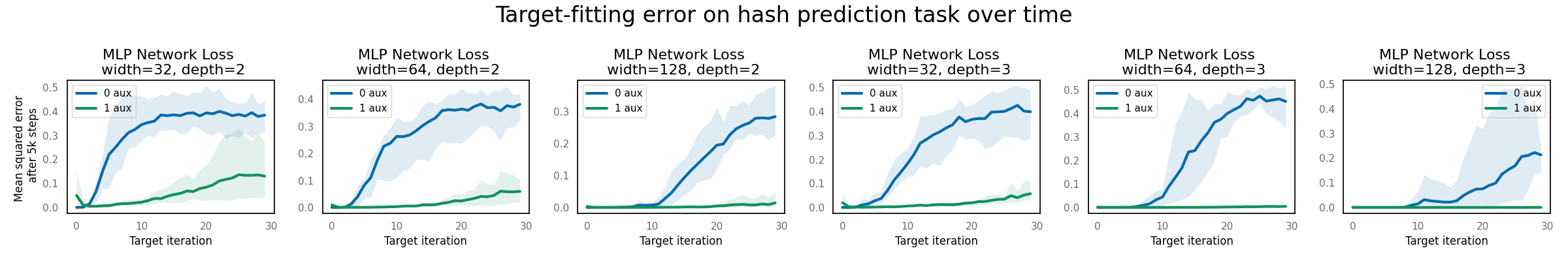}
    \includegraphics[width=\linewidth]{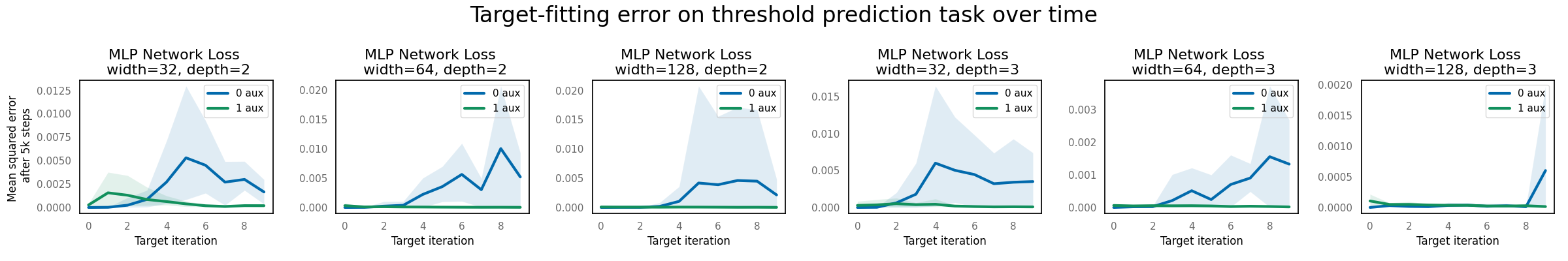}
    \includegraphics[width=\linewidth]{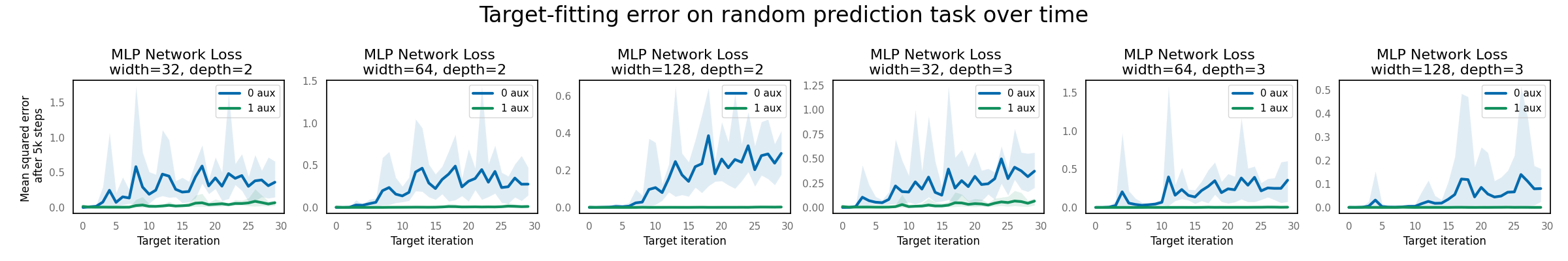}
    \caption{Effect of adding \pyoi to the regression objective in a random reward prediction problem on the non-stationary MNIST environment studied previously. We see that the \pyoi objective produces networks that can consistently outperform those trained with a standard regression objective, exhibiting minimal capacity loss in comparison to the same network architecture trained on the same sequence of targets. }
    \label{fig:pyoi_on_mnist}
\end{figure}
\section{Atari evaluations}
\label{appx:atari}
We now present full evaluations of many of the quantities described in the corresponding chapter, along with a study of the sensitivity of InFeR to its hyperparameters. We use the same training procedure for all of the figures in this section, loading agent parameters from checkpoints to compute the quantities shown.

\subsection{Hyperparameter sensitivity of InFeR in deep reinforcement learning agents}
\label{appx:hypers}
We report results of hyperparameter sweeps over the salient hyperparameters relating to InFeR, so as to assess the robustness of the method. For both the DDQN and Rainbow agents augmented with InFeR, we sweep over the number of auxiliary predictions (1, 5, 10, 20), the cumulant scale used in the predictions (10, 100, 200), and the scale of the auxiliary loss (0.01, 0.05, 0.1, 0.2). We consider the capped human-normalized return across four games (Montezuma's Revenge, Hero, James Bond, and MsPacman), and run each hyperparameter configuration with 3 seeds. Results are shown in Figure~\ref{fig:ddqn-sweep} for the DDQN agent; we compare performance as each pair of hyperparameters varies (averaging across the other hyperparameter, games, and seeds, and the last five evaluation runs of each agent). Corresponding results for Rainbow are given in Figure~\ref{fig:rainbow-sweep}.

\begin{figure}
    \centering
    \null
    \hfill
    \includegraphics[keepaspectratio,width=.32\textwidth]{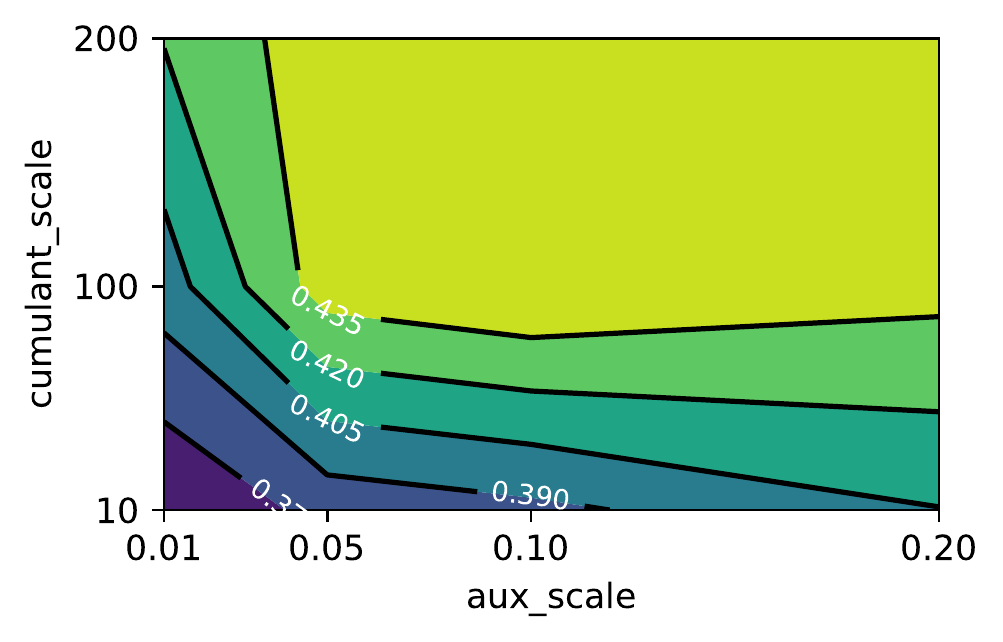}
    \includegraphics[keepaspectratio,width=.32\textwidth]{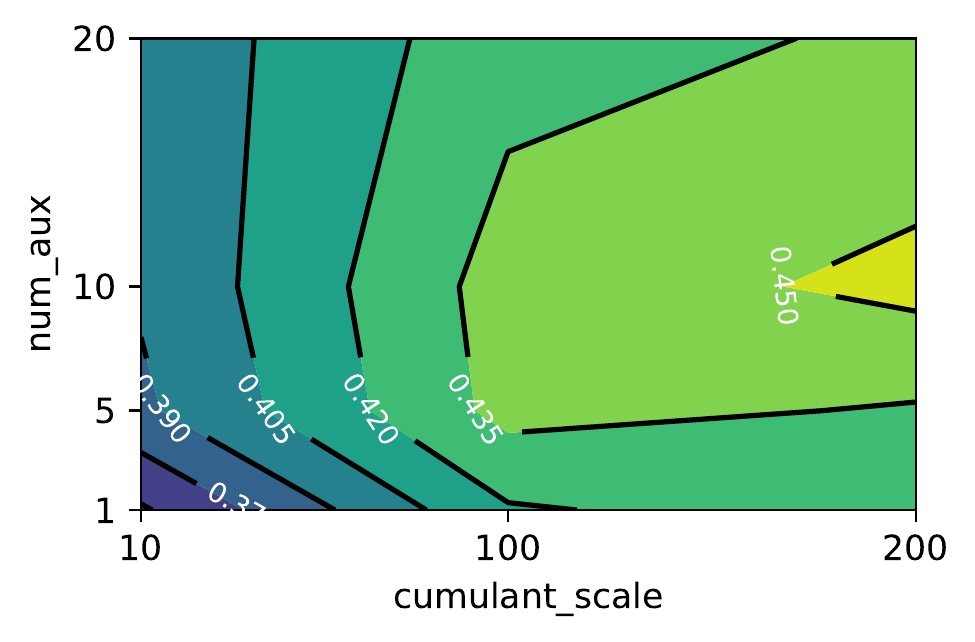}
    \includegraphics[keepaspectratio,width=.32\textwidth]{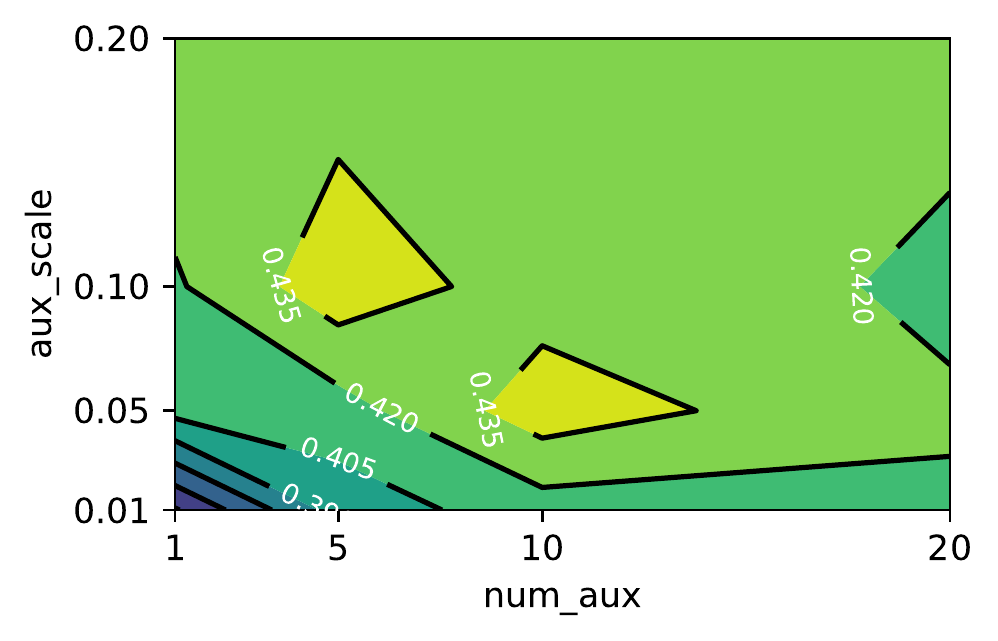}
    \hfill
    \null
    \caption{Hyperparameter sweeps for the DDQN+InFeR agent. Each contour plot shows average capped human-normalized score at the end of training marginalized over all hyperparameters not shown on its axes.}
    \label{fig:ddqn-sweep}
\end{figure}

\begin{figure}
    \centering
    \null
    \hfill
    \includegraphics[keepaspectratio,width=.32\textwidth]{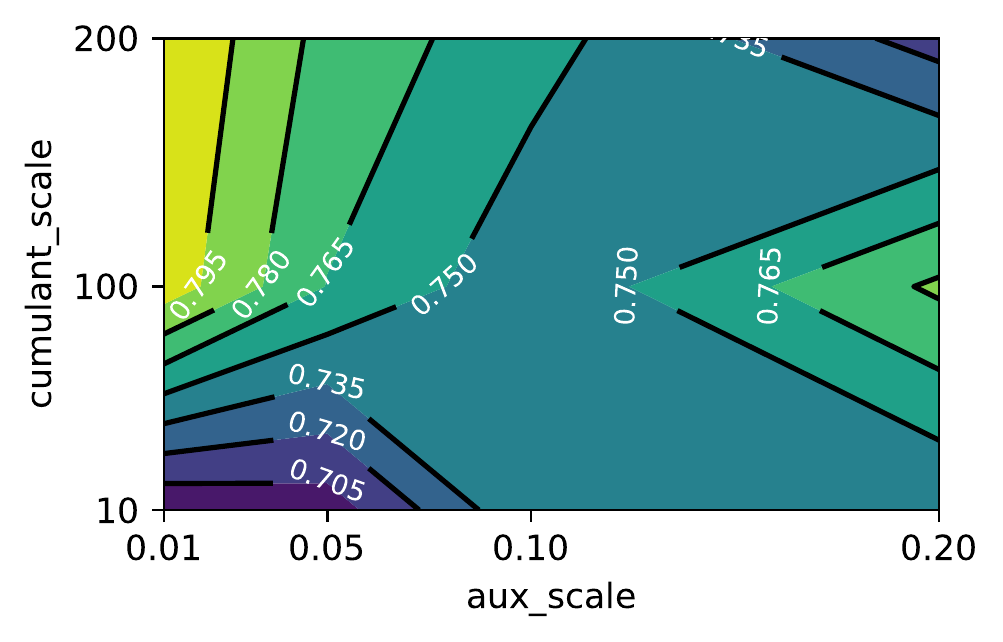}
    \includegraphics[keepaspectratio,width=.32\textwidth]{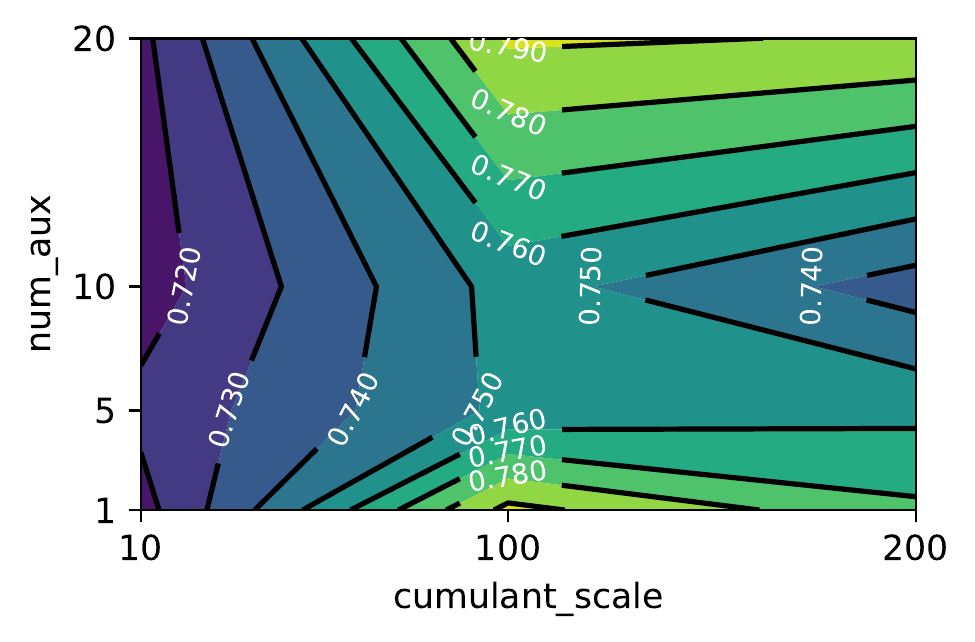}
    \includegraphics[keepaspectratio,width=.32\textwidth]{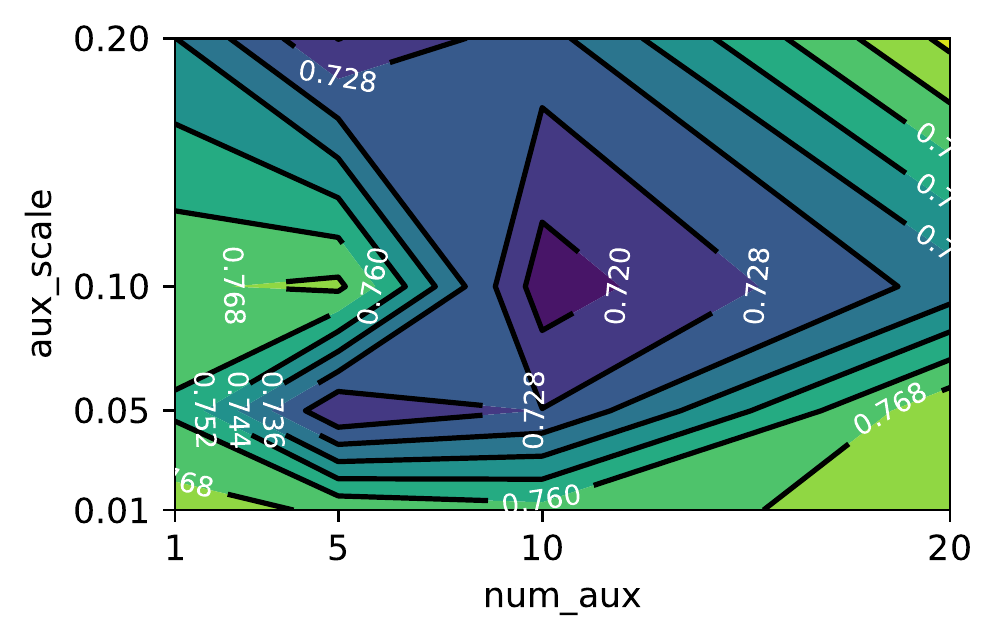}
    \hfill
    \null
    \caption{Hyperparameter sweeps for the Rainbow+InFeR agent. Each contour plot shows average capped human-normalized score at the end of training marginalized over all hyperparameters not shown on its axes.}
    \label{fig:rainbow-sweep}
\end{figure}
\label{appx:evals}

\begin{itemize}[leftmargin=0.5cm]
    \item \textbf{Agent:} We train a Rainbow agent \citep{hessel2018rainbow} with the same architecture and hyperparameters as are described in the open-source implementation made available by \citet{dqnzoo2020github}. We additionally add InFeR, as described in Section~\ref{sec:pyoi}, with 10 heads, gradient weight 0.1 and scale 100. 
    \item \textbf{Training:} We follow the training procedure found in the Rainbow implementation mentioned above. We train for 200 million frames, with 500K evaluation frames interspersed every 1M training frames. We save the agent parameters and replay buffer every 10M frames to estimate feature dimension and target-fitting capacity.
\end{itemize}

\subsection{Feature rank}
\label{appx:feature-rank-atari}

We first extend the results shown in Figure~\ref{fig:effdim_vanilla} to two additional games: Seaquest, and a sparsified version of Pong in which the agent does not receive negative rewards when the opponent scores. In these settings, we stored agent checkpoints once every 10M frames in each 200M frame trajectory, and used 5000 sampled inputs from the agent's replay buffer to estimate the feature rank, using the cutoff $\epsilon=0.01$. Results are shown in Figure~\ref{fig:full-effdim-perf-apx}.
\begin{figure}
    \centering
    \includegraphics[keepaspectratio,width=.9\textwidth]{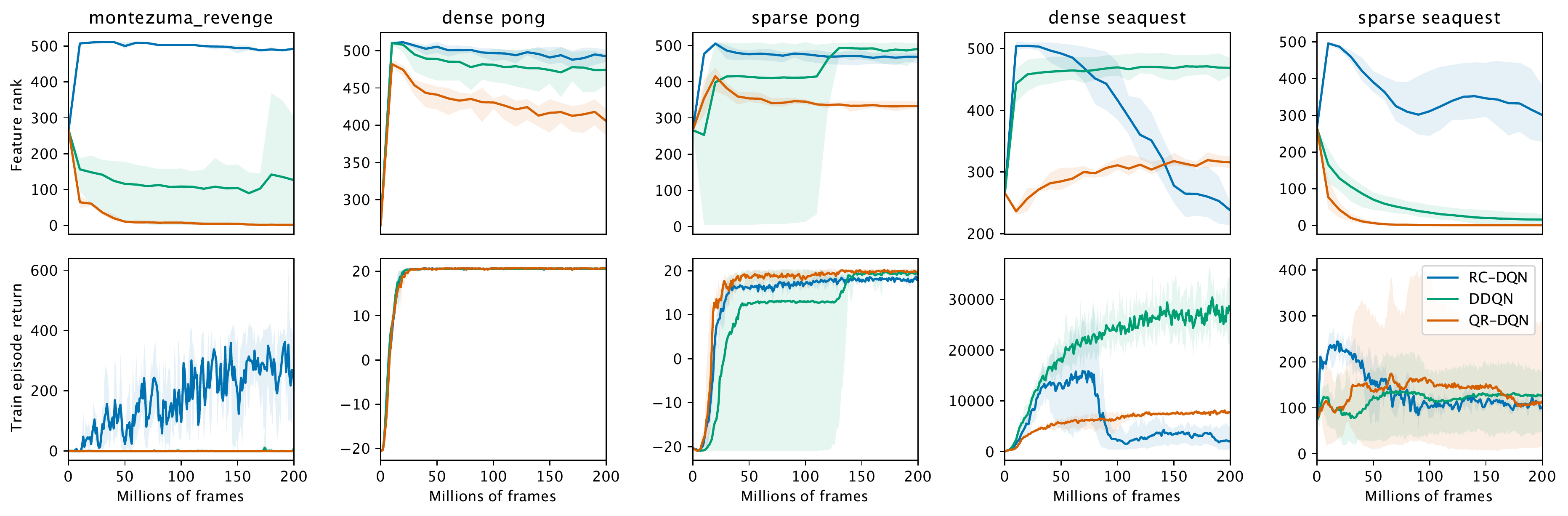}
    \caption{Feature rank and performance of RL agents on demonstrative Atari environments.}
    \label{fig:full-effdim-perf-apx}
\end{figure}

We further evaluate the evolution of feature rank in agents trained on all 57 games in the arcade learning environment in Figure~\ref{fig:effdim_all}. We find that the decline in dimension after the first checkpoint at 10M frames shown across the different agents in the selected games also occurs more generally in Rainbow agents across most environments in the Atari benchmark. We also show that in most cases adding InFeR mitigates this phenomenon. Our observations here do not show a uniform decrease in feature rank or a uniformly beneficial effect of InFeR. The waters become particularly muddied in settings where neither the Rainbow nor Rainbow+InFeR agent consistently make learning progress such as in tennis, solaris, and private eye. It is outside the scope of this work to identify precisely why the agents do not make learning progress in these settings, but it does not appear to be due to the type of representation collapse that can be effectively prevented by InFeR.

\textbf{Procedure.} We compute the feature rank by sampling $n=50000$ transitions from the replay buffer and take the set of origin states as the input set. We then compute a $n \times d$ matrix whose row $i$ is given by the output of the penultimate layer of the neural network given input $S_i$. We then take the singular value decomposition of this matrix and count the number of singular values greater than $0.01$ to get an estimate of the dimension of the network's representation layer.

In most games, we see a decline in feature rank after the first checkpoint at 10M frames. Strikingly, this decline in dimension holds even in the online RL setting where the agent's improving policy presumably leads it to observe a more diverse set of states over time, which under a fixed representation would tend to increase the numerical rank of the feature matrix. This indicates that even in the face of increasing state diversity, agents' representations face strong pressure towards degeneracy. It is worth noting, however, that the agents in dense-reward games do tend to see their feature rank increase significantly early in training; this is presumably due to the network initially learning to disentangle the visually similar states that yield different bootstrap targets.

\begin{figure}
    \centering
    \includegraphics[width=.85\linewidth]{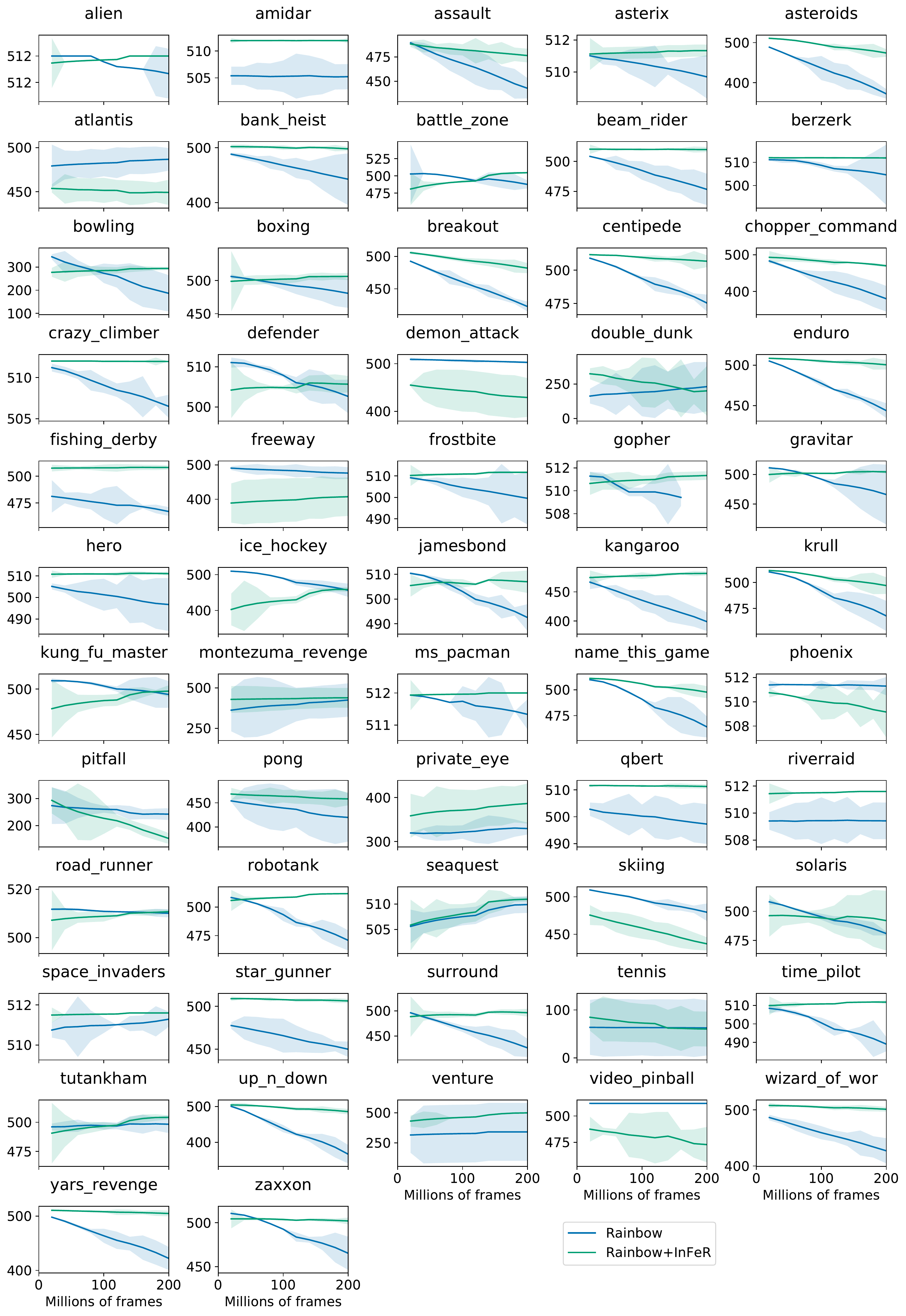}
    \caption{feature rank of agent representations over the course of training on all 57 games in the Atari benchmark. We compare Rainbow against Rainbow+InFeR. }
    \label{fig:effdim_all}
\end{figure}

\subsection{Target-fitting capacity early in training}
\label{appx:tf-capacity-atari}
In this section we examine the target-fitting capacity of neural networks trained with DQN, QR-DQN, and Rainbow over the course of $50$ million environment frames on five games in the Atari benchmark (amidar, montezuma's revenge, pong, bowling, and hero). Every $1$ million training frames we save a checkpoint of the neural network weights and replay buffer. For each checkpoint, we generate a random target network by initializing network weights with a new random seed. We then train the checkpoint network to predict the output of this random target network for $10000$ mini-batch updates (batch size of $32$) under a mean squared error loss, for states sampled from the first $100,000$ frames in the checkpoint's replay buffer. Furthermore, we repeat this for $10$ seeds used to initialize the random target network weights.

The results of this experiment are shown in Figure~\ref{fig:atari_target_fit_cap_comb} (in orange), where the solid lines show means and shaded regions indicate standard deviations over all seeds (both agent seeds ($5$) and target fitting seeds ($10$), for a total of $50$ trials). We also show srank and $\effdim$ of the features output at the network's penultimate layer for each of the checkpointed networks used for target fitting. These are computed using the network features generated from $1000$ states sampled randomly from that checkpoint's replay buffer. For feature rank, averages and standard deviations are only over the $5$ agent seeds.

\begin{figure}[ht]
    \centering
    \includegraphics[width=0.85\linewidth]{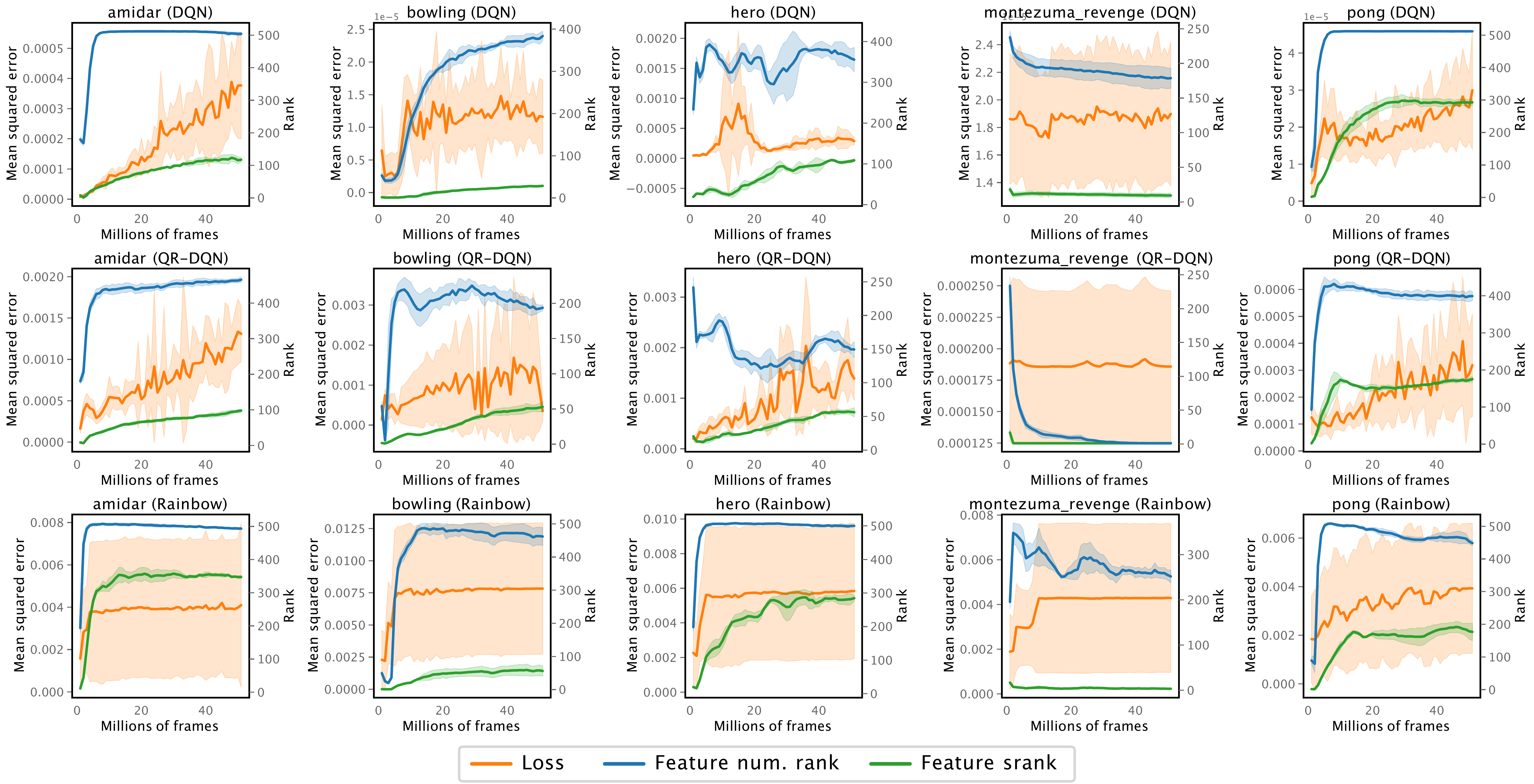}
    \caption{Mean squared error, after $10000$ training steps for the target-fitting on random network targets. We also show the corresponding feature rank of the pre-trained neural network (before target-fitting).}
    \label{fig:atari_target_fit_cap_comb}
\end{figure}

\subsection{Performance}

We provide full training curves for both Rainbow and Rainbow+InFeR on all games in Figures~\ref{fig:pyoirainbowhncap} \& \ref{fig:pyoirainbowhncap-double} (capped human-normalized performance), and \ref{fig:pyoirainboweval} \& \ref{fig:pyoirainboweval-double} (raw evaluation score). We also provide evaluation performance curves for DDQN and DDQN+InFeR agents in Figure~\ref{fig:ddqn-eval}.

\begin{figure}
    \centering
    \includegraphics[width=.85\linewidth]{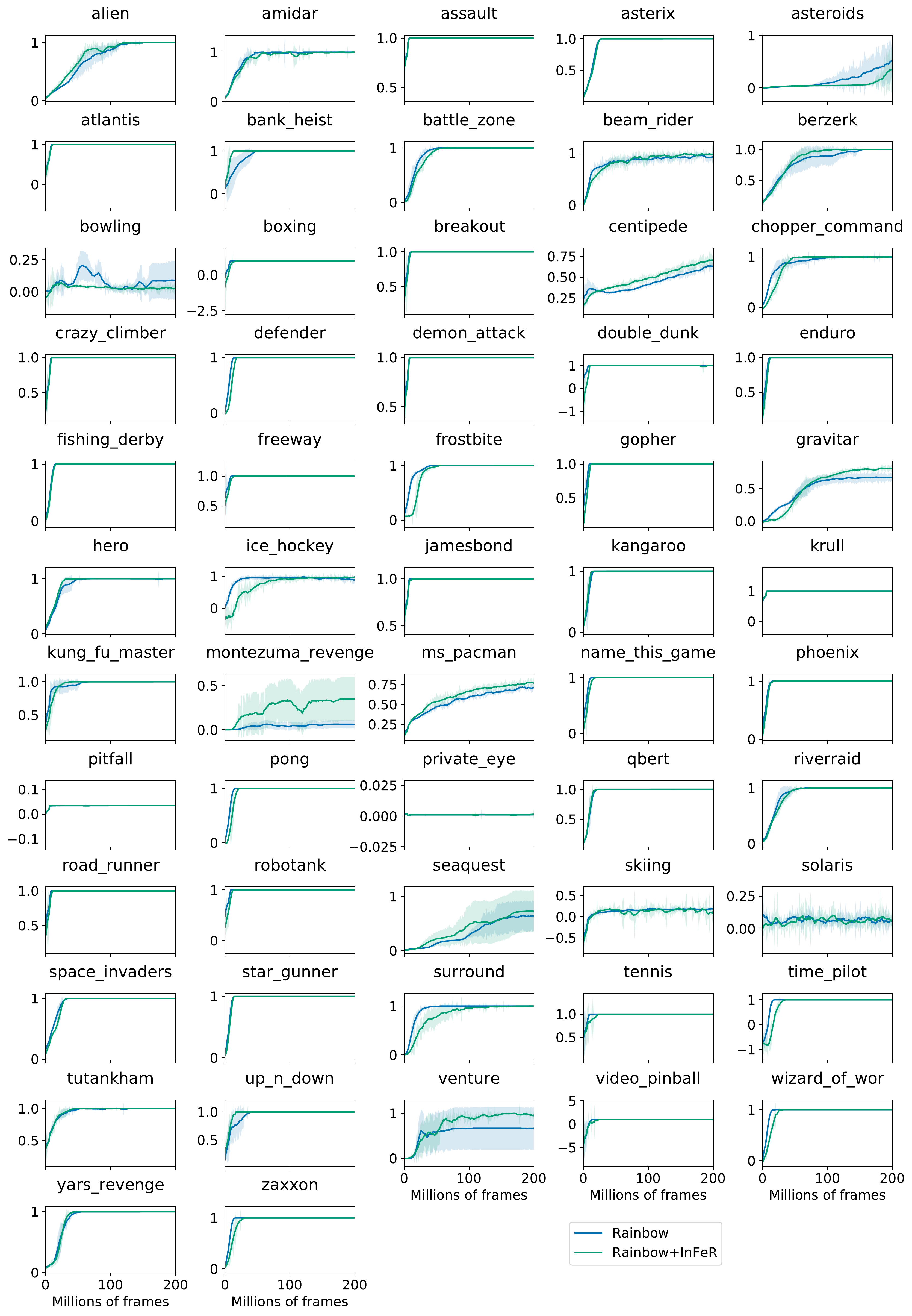}
    \caption{Full evaluation of capped human-normalized performance on Atari benchmarks for the default Rainbow architecture.}
    \label{fig:pyoirainbowhncap}
\end{figure}

\begin{figure}
    \centering
    \includegraphics[width=.85\linewidth]{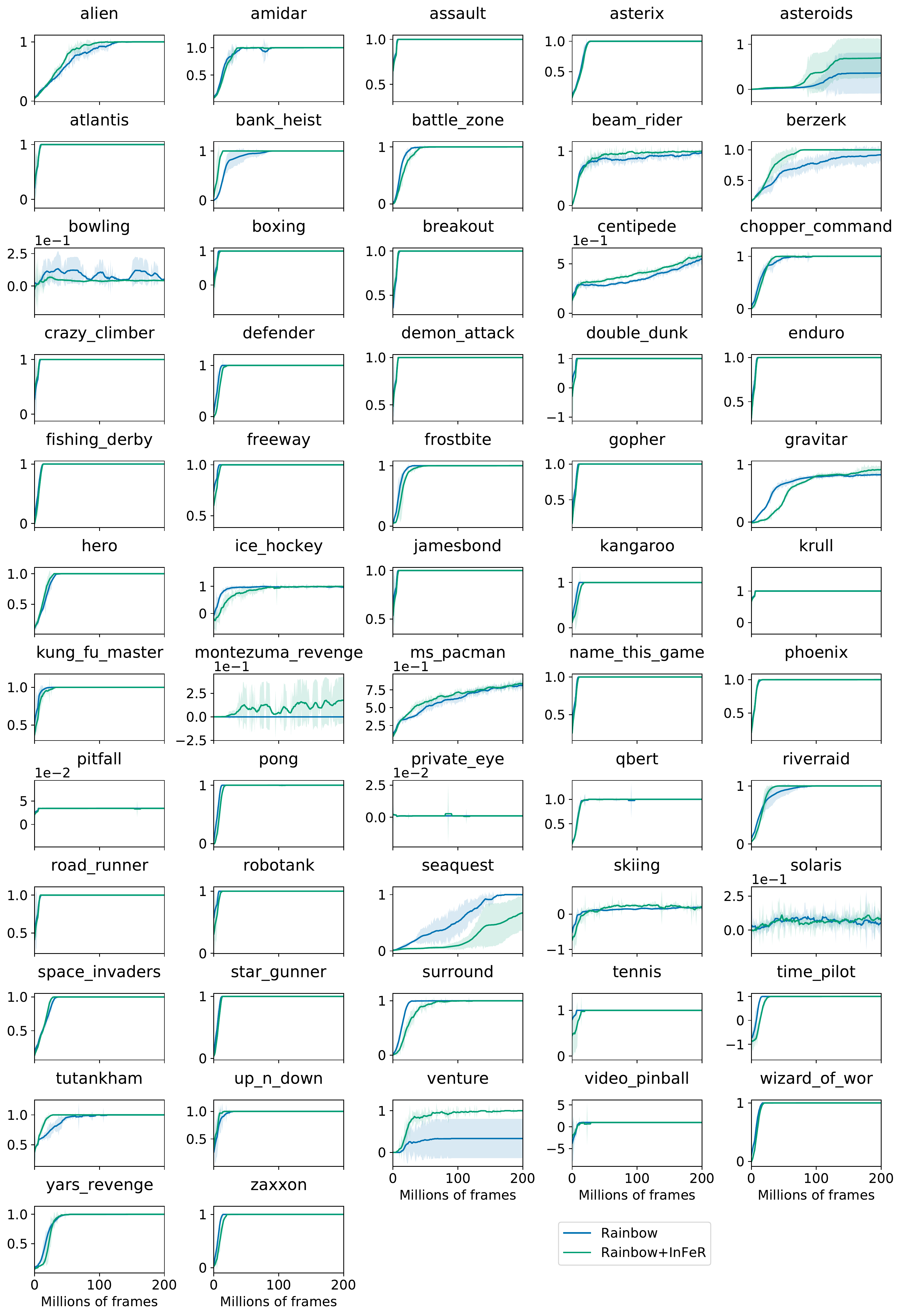}
    \caption{Full evaluation of capped human-normalized performance on Atari benchmarks in the double-width Rainbow architecture.}
    \label{fig:pyoirainbowhncap-double}
\end{figure}

\begin{figure}
    \centering
    \includegraphics[width=.85\linewidth]{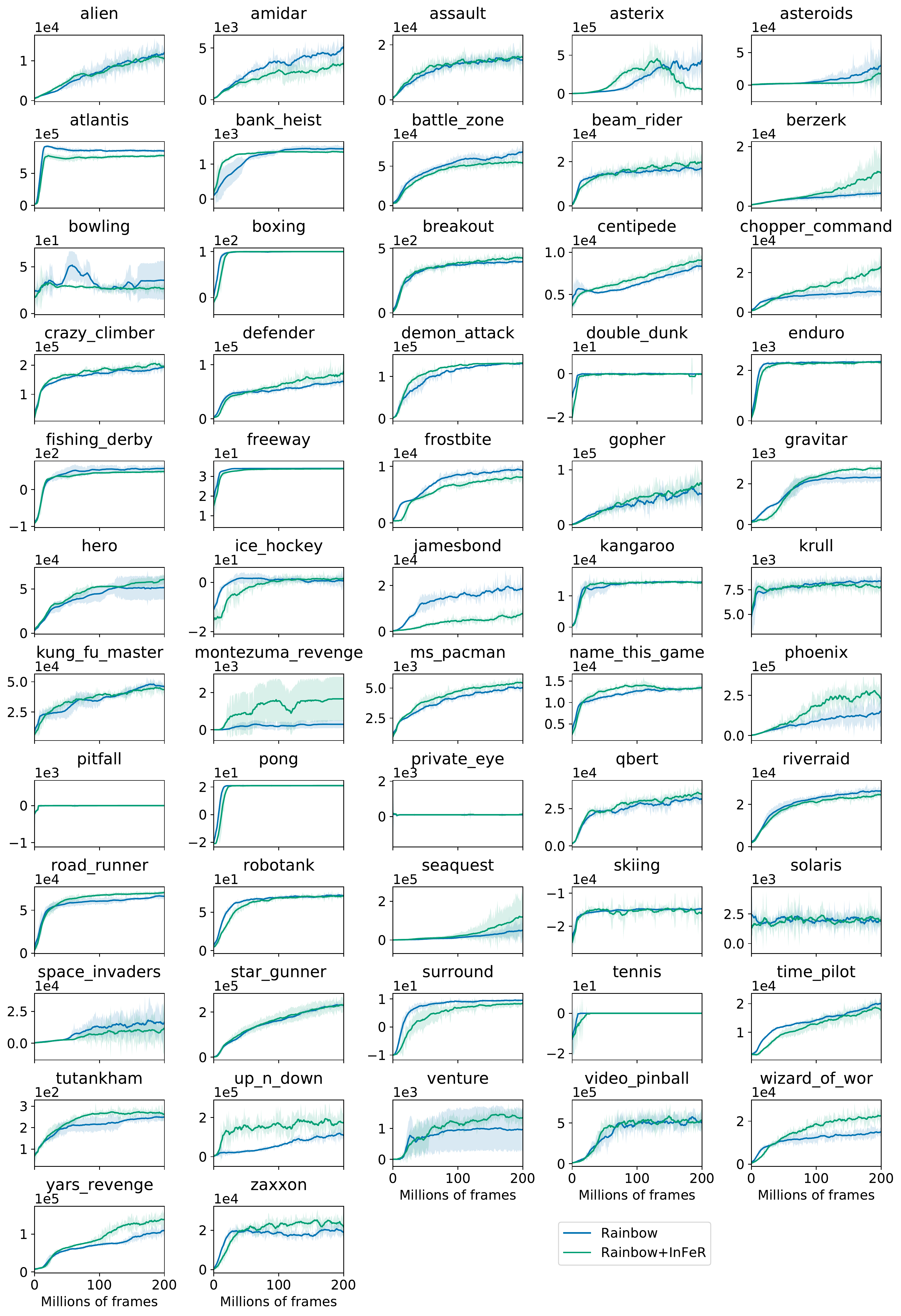}
    \caption{Full evaluation of raw scores on Atari benchmarks for the default Rainbow architecture.}
    \label{fig:pyoirainboweval}
\end{figure}

\begin{figure}
    \centering
    \includegraphics[width=.85\linewidth]{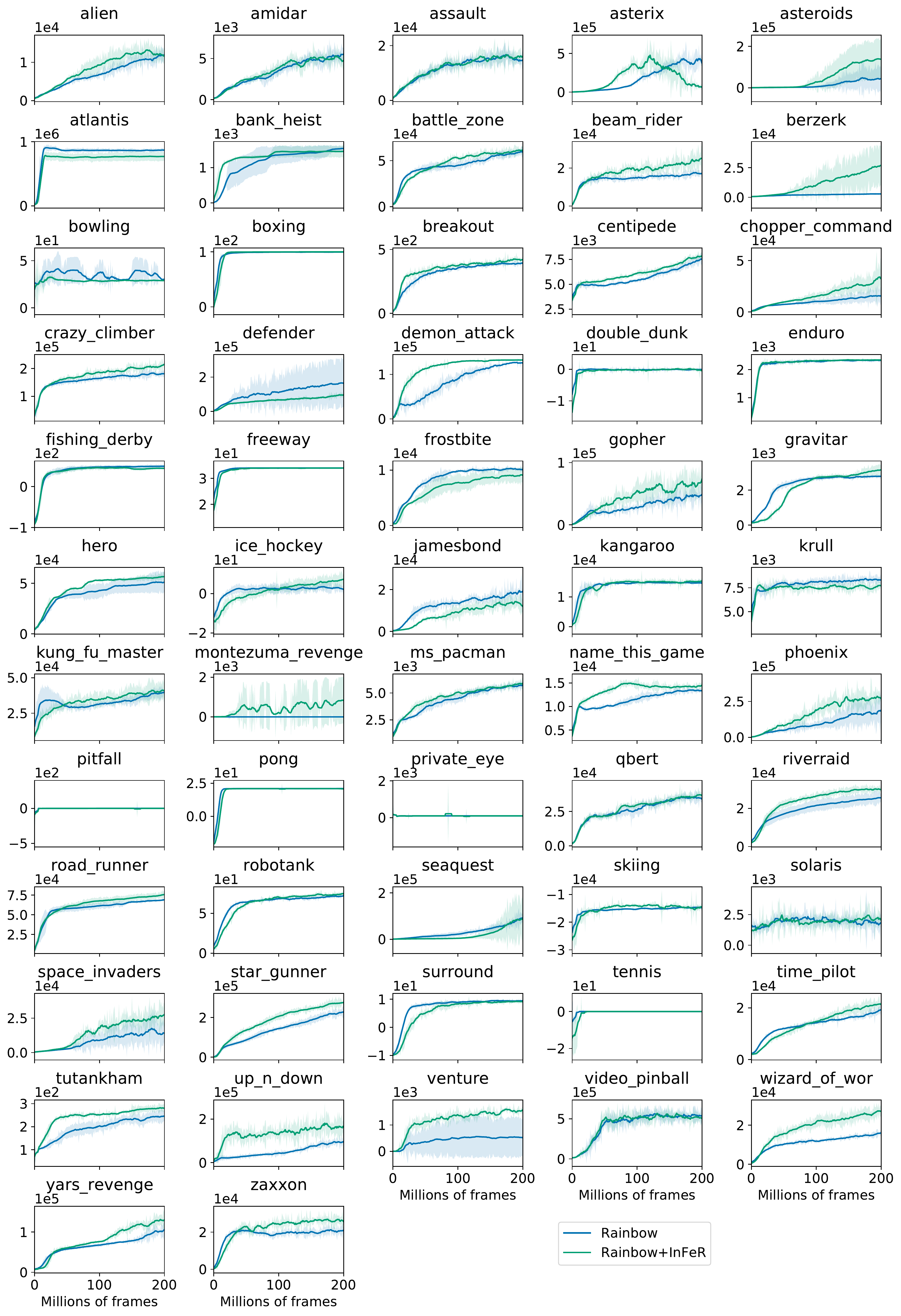}
    \caption{Full evaluation of raw scores on Atari benchmarks for the double-width Rainbow architecture.}
    \label{fig:pyoirainboweval-double}
\end{figure}

\begin{figure}
    \centering
    \includegraphics[width=.85\linewidth]{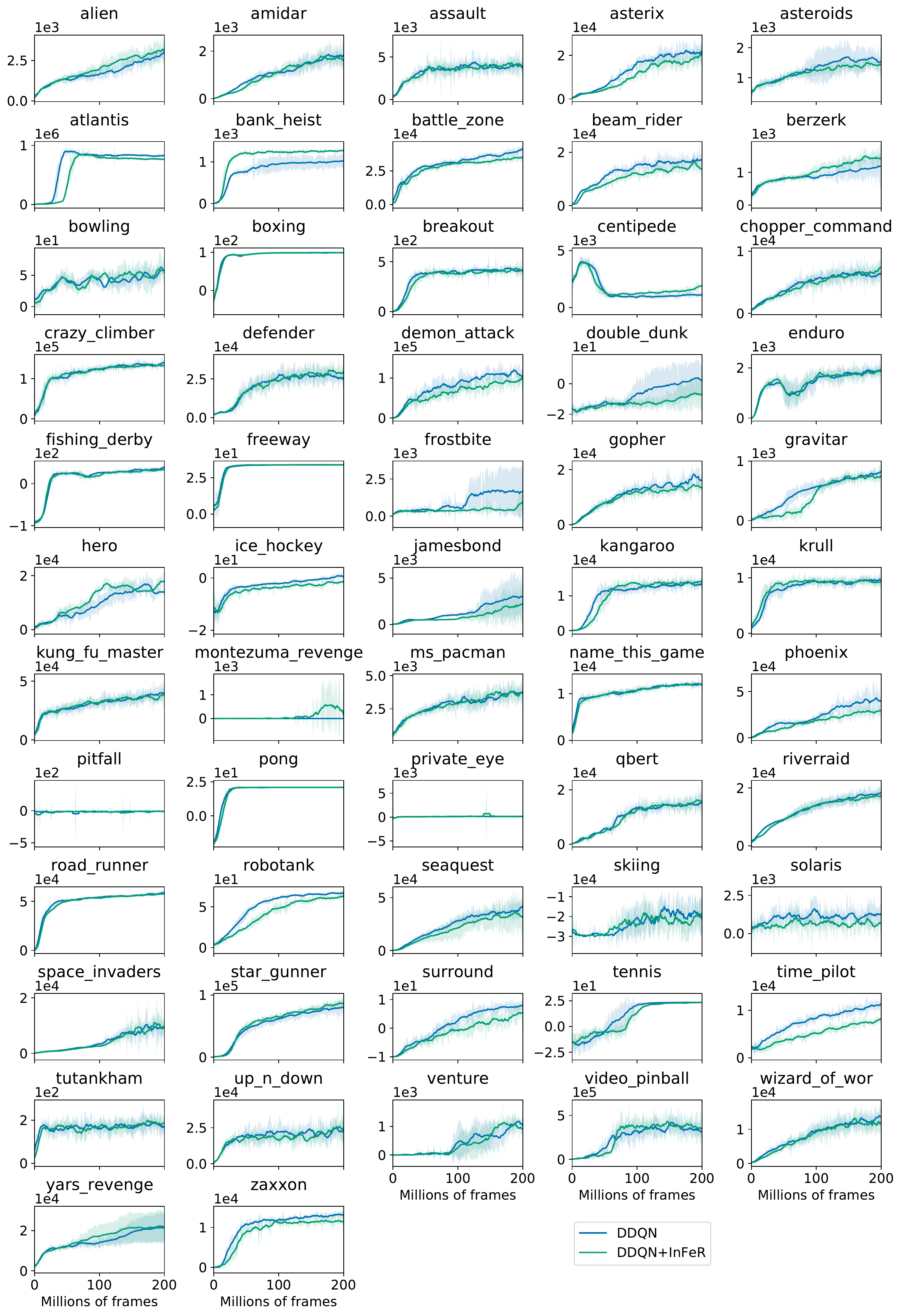}
    \caption{Evaluations of the effect of \pyoi on performance of a Double DQN agent. Overall we see a slight improvement in average performance over all games, alongside a significant improvement in Montezuma's Revenge.}
    \label{fig:ddqn-eval}
\end{figure}
\chapter{Interference and generalization}

\section{Proofs}

\subsection{Proofs of main results}
\label{apx:proofs}
\convergence*

\begin{proof}
Recall we assume the following dynamical system

\begin{align*}
    \partial_t V_t &= -(I - \gamma P^\pi)V_t + R 
    \intertext{Inducing the trajectory}
    V_t &= \exp( - t(I - \gamma P^\pi)) (V_0 - V^\pi) + V^\pi 
    \intertext{As we assume $P^\pi$ is diagonalizable, this implies that $(I - \gamma P^\pi)$ is also diagonalizable. Let $u_1, \dots, u_n$ denote the right eigenvectors of $P^\pi$ with corresponding eigenvalues $\lambda_1 \ge \dots \ge \lambda_n$. Let $V_0 = \sum \alpha^0_i u_i$. }
    V_t &= \sum \alpha_i^t u_i \\
    &= \exp (-t(I - \gamma P^\pi)) (\sum \alpha_i^0 - \alpha^\pi_i u_i) + \sum \alpha^\pi_i u_i \\
    &= \sum  \exp(-t(1-\gamma \lambda_i)) \bigg ( \sum  (\alpha_i^0 - \alpha_i^\pi  ) u_i + \sum \alpha_i^\pi u_i \bigg )
    \intertext{Now, we consider the value of $V_t - V^\pi$ along each coordinate. Note that we have not assumed an orthogonal eigenbasis, thus cannot speak directly to the norm of the projection of this difference onto the eigenspace corresponding to each eigenvector $\lambda_k$. However, treating the eigendecomposition as a basis, we can discuss how the coordinates $\alpha^t_i$ of the value function $V_t$ converge with respect to this basis.}
    |V_t - V^\pi|[i] = |\alpha^t_i - \alpha^\pi_i| &=   | \exp(-t(1-\gamma \lambda_i)) (\alpha_i^0 - \alpha_i^\pi) + \alpha_i^\pi  - \alpha^\pi_i| \\
    &=|\exp(-t(1-\gamma \lambda_i))  (\alpha_i^0 - \alpha_i^\pi) | = \exp(-t(1-\gamma \lambda_i)) | (\alpha_i^0 - \alpha_i^\pi) | 
\end{align*}
We conclude by noting that for large values of $\lambda_i$, the exponential term $\exp ( - t(1 - \gamma \lambda_i)) $ will decay more slowly as a function of $t$ than for smaller values of $\lambda_i$. Thus, these coordinates (which correspond to non-smooth functions over the state space) will converge fastest. When the eigenvectors form an orthogonal basis, as is the case for symmetric $P^\pi$, we can go further and observe that this convergence will apply to the norm of the projection of the value function into the corresponding eigenspace.   Thus for symmetric $P^\pi$, we obtain the following stronger convergence result, where $U_k$ denotes the eigenspace corresponding to the eigenvalue $\lambda_k$.
\begin{equation}
    \| \Pi_{U_k} ( V_t - V^\pi ) \| = \exp(-t (1 - \gamma \lambda_k)) \| \Pi_{U_k} (V_0 - V^\pi ) \|
\end{equation}
\end{proof}
\tderror*
Let $V_0 = \sum  \alpha_i v_i$. Then, letting $V_t$ be defined as in Equation~\ref{eq:td_dynamics}.
\begin{equation}
\TD(V_t) \leq \sum_{i=1}^n \exp(-2t(1-\gamma \lambda_i))( \alpha^\pi_i -  \alpha_i^0)^2 (1-\gamma \lambda_i)^2 \; .
\end{equation}
\begin{proof}

By our assumption on the diagonalizability of $P^\pi$, we can leverage the previous result on the coordinates of $V_t$.
\begin{align*}
    V_t - V^\pi &=  \sum  \exp(-t(1-\gamma \lambda_i)) \bigg ( \sum  (\alpha_i^0 - \alpha_i^\pi  ) u_i \bigg ) \\
    \intertext{We then bound the TD error as follows.}
    \|V_t - \gamma P^\pi V_t - R \|^2 &= \| V_t - \gamma P^\pi V_t + \gamma P^\pi V^\pi -\gamma P^\pi V^\pi -R \| \\
    &= \| V_t - \gamma P^\pi V^\pi - R -  \gamma P^\pi (V_t - V^\pi) \| \\
    \intertext{Since $V^\pi = R + \gamma P^\pi V^\pi$, we obtain the following.}
    &= \| (I -  \gamma P^\pi) (V_t - V^\pi )\|^2 \\
    &= \| \sum (1 - \gamma \lambda_k) (\alpha_i^t - \alpha_i^\pi) u_i \|^2 \\
    &\leq \sum (\alpha^\pi_i - \alpha^t_i)^2(1-\gamma \lambda_i)^2
\end{align*}
The remainder follows a straightforward substitution.
\end{proof}

\theoremsecond* 

\begin{proof}
While our prior analysis has considered the continuous time system $\ttheta_t$, this does not perfectly approximate the discrete system $\theta_t$. When a fixed step size is used, the first-order continuous-time approximation accrues error roughly proportional to $\alpha t$. We then follow the procedure of \citet{barrett2021implicit}, applying a Taylor expansion to the evolution of $\ttheta_t$ with respect to time. We will use the notation $\ttheta(t)$ to denote the explicit dependence of $\ttheta$ as a function of time.
\begin{align}
    \ttheta(\alpha t) &= \ttheta(0) + \sum \frac{(\alpha t)^n}{n!} \theta^{(n)}(0) \\
    &= \ttheta(0) + \alpha t f(\ttheta(0)) + \frac{(\alpha t)^2}{2} \nabla_\theta f \cdot f(\ttheta(0)) +O(\alpha^3)\\
    &= \ttheta(0) + \alpha t f(\ttheta(0)) + \frac{(\alpha t)^2}{2} f_1(\theta(0)) +O(\alpha^3) \\
    \intertext{Relating this back to the discrete system $\theta_t$}
    \theta_{1} &= \theta_0 + \alpha f(\theta_0) = \ttheta(0) + \alpha f(\ttheta(0)) \\
    \theta_1 &= \ttheta(1) - \frac{\alpha^2}{2}f_1(\ttheta(0)) + O(\alpha^3)
    \intertext{Thus, the system $\partial_t \check{\theta}_t = f(\check{\theta}_t) + \alpha^2/2 f_1( \check{\theta}_t)$ satisfies}
    \theta_1 &= \check{\theta}(1) + O(\alpha^3)
\end{align}

We begin by observing that $\nabla_\theta \| V_\theta - \square T V_\theta \|^2 = (V_\theta - T V_\theta) \cdot \nabla_\theta V_\theta = f(\theta)$. 
\begin{align}
   \theta_{} &= \theta_0 + \alpha n f(\theta_0) + (\alpha n) ^2/2 \nabla_\theta f (\theta_0) \cdot f(\theta_0) + O( (\alpha n)^3)  \\
   &= \theta_0  + \alpha n f(\theta_0) + \frac{(\alpha n)^2}{2} f_1(\theta_0) + O((n\alpha)^3)\\
   \intertext{We then express $f_1(\theta)$ as follows.}
   f_1(\theta_0) &= \nabla_\theta[ f (\theta_0)] \cdot [f(\theta_0)] \\
   &= [\nabla^2_w V_\theta \cdot ((\gamma P^\pi - I)V_\theta + r) + \nabla_\theta V_\theta \cdot ((\gamma P^\pi - I) \nabla_\theta V_\theta)][f(\theta)] \\
   &= [\nabla_\theta^2 V_\theta \cdot ( (\gamma P^\pi - I)V_\theta + r) + \nabla_\theta V_\theta \cdot \nabla_\theta V_\theta][f(\theta)] \\ & \quad + \gamma [\nabla_\theta V_\theta P^\pi \nabla_\theta V_\theta][f(\theta)] \\
   \intertext{Noting that the left hand side term is equal to the gradient of the gradient norm penalty for the stop-gradient version of the TD regression problem, we simplify as follows:}
   &= \frac{1}{2}\nabla_\theta \| \nabla_\theta \frac{1}{2}\|V_\theta - \square T^\pi V_\theta\|^2 \|^2 + \gamma [\nabla_\theta V_\theta \cdot P^\pi \cdot \nabla_\theta V_\theta][f(\theta)]
\end{align}
We note that, unlike in the stochastic gradient descent setting, $f_1$ does not correspond to a gradient of any function. Instead,  it corresponds to the second-order we would get for a frozen target, which corresponds to a gradient norm penalty, plus a term that measures the alignment of the gradients at each state and its expected successor. Intuitively, both of these terms minimize the `variance' in the loss induced by noisy, discrete gradient steps. The flatter loss surfaces induced by the gradient norm penalty will naturally lead to greater robustness to parameter perturbations. The gradient alignment term reflects the observation previously that non-smooth functions contribute the most to the TD error, and so encourages the first-order gradient effects on successive states to move in a similar direction. 

We note that this final observation seems to be at odds with the tendency for TD learning to encourage more tabular updates. Why would a second-order correction term which promotes flat minima and gradient alignment result in tabular updates? To answer this, we point to the tendency of TD targets to converge along. the non-smooth components of the value function first. We are therefore faced with finding a flat region of parameter space to fit a discontinuous function. A representation which succeeds at this will benefit from minimizing interference between states, as the gradients for one transition will be on average uncorrelated with even nearby other states. The gradient alignment penalty suggests that, while the implicit regularization will prefer flat minima, smooth interference patterns which move other states in a similar direction to the current state will be penalized less than non-smooth directions.
\end{proof}

\begin{restatable}{cor}{corr-second}
    The second-order dynamics push features towards precisely the worst direction w.r.t. stability. I.p. looking at the set of positive definite representations introduced by \citet{ghosh2020representations} we see
    \begin{equation}
        \{v : v^\top P^\pi v <  \gamma^{-1} \|v\|_{\Xi} \}
    \end{equation}
    whereas the optimal gradients for the second order term implicitly solve the following optimization problem
    \begin{equation}
        \min  \mathbb{E}_{x \sim \eta(x)}[g(x)^\top g(x) - \gamma g(x)^\top (P^\pi g)(x)]
    \end{equation}
    
\end{restatable}

\ntk* 
\begin{proof}
We leverage the dynamics $\partial_t V_t = K(X,X) \nabla_\theta V_\theta \cdot ((\gamma P^\pi - I) + r)$ and follow the derivation of Section 5 of \citet{jacot2018neural}. 
\end{proof}

We can develop intuitions for the kernel gradient descent setting by considering the special case of linear function approximation, where $K(x_1, x_2) = \langle \phi(x_1), \phi(x_2) \rangle$ for some feature map $\phi$. For the moment, we will define $\Phi$ to be a matrix consisting of features for every state in the state space $X$ (i.e. we update all states in the mdp at once). We then obtain
\begin{align}
    \partial_t \mathbf{w}_t & = \alpha \Phi^\top (R^\pi + \gamma P^\pi \Phi \mathbf{w}_t - \Phi \mathbf{w}_t)  \, .
\end{align}
We can express the evolution of the value function constructed by multiplication of $\Phi$ and $w$ as follows.
\begin{align}
    \partial_t V_t &= (\partial_w V_t)^\top \partial_t w_t = \Phi \partial_t w_t \\
    &= -\Phi (\Phi^\top (I - \gamma P^\pi) \Phi) w \\
    &= - \Phi \Phi^\top (I - \gamma P^\pi) V_t \\
    &= - K (I - \gamma P^\pi) V_t
    \intertext{We further consider the dynamics of the value function on inputs outside of the set of states on which the Bellman updates are computed as follows.}
    \partial_t V_t(\xtest) &= (\partial_w V_t(\xtest))^\top \partial_t w_t \\
    &= - \phi(\xtest)^\top  \Phi^\top (I - \gamma P^\pi) V_t  \\
    &= - K(\xtest, \Xtrain) K(\Xtrain, \Xtrain)^{-1} \partial_t V_t
\end{align}

We now lift the assumption that all states are updated. In this more general kernel gradient descent setting, we let $K$ be a kernel as before, with $\tilde{K} = K(\Xtrain, \Xtrain)$ and $\kappa_{\xtest} = K(\xtest, \Xtrain)$. We then obtain the following dynamics
\begin{align}
    \partial_t V_t(\xtest) &= \kappa_{\xtest} \tilde{K}^{-1} \partial_t V_t(\Xtrain) \\
    \intertext{In particular, this results in the following trajectory.}
    V_t(\xtest) &= V_0(\xtest) + \kappa_{\xtest} \tilde{K}^{-1} [ V_t(\Xtrain) - V_0(\Xtrain)]
\end{align}

An interesting case study occurs when we consider, e.g., off-policy evaluation where the bootstrap targets used in the TD updates may not have been visited by the agent during training. This will be the case in many offline RL problems, where the action that would be selected by the policy we seek to evaluate was not taken by the behaviour policy, and so the agent leverages bootstrap targets which are not updated as part of the training set of states. In such cases, we will decompose the state space as $X = \Xtrain \oplus \Xtest$. The dynamics we get in this case look quite different from standard kernel regression, as the dynamics of the training states will depend on the predictions on the `test' states. To condense notation, we will use $T^\pi V_t$ to refer to an application of the Bellman operator $V_t \mapsto \gamma P^\pi V_t + R^\pi$.  

\begin{align}
    \partial_t V_t(\Xtrain) &= \Phi_{\train} \Phi_\train^\top (  (T^\pi V_t)(\Xtrain) - V_t(\Xtrain)) \\
    \partial_t V_t(\Xtest) &= \Phi_{\test} \Phi_\train^\top  ((T^\pi V_t)(\Xtrain) - V_t(\Xtrain))
    \intertext{We note that $(T^\pi V_t)(\Xtrain)$ depends on both $V(\Xtrain)$ and $V(\Xtest)$ due to the application of the Bellman operator $T^\pi$. We thus end up with the following joint system.}
    \partial_t V_t(\Xtrain \oplus \Xtest) &= \Phi_{\test} \Phi_\train^\top  ((T^\pi V_t)(\Xtrain) - V_t(\Xtrain)) \oplus \Phi_{\train} \Phi_\train^\top (  (T^\pi V_t)(\Xtrain) - V_t(\Xtrain)) \\
    \partial_t V_t(\Xtrain \oplus \Xtest) &= (\Phi_{\test} \oplus \Phi_{\train} ) \Phi_\train^\top (  (T^\pi V_t)(\Xtrain) - V_t(\Xtrain)) 
    \intertext{Using a non-standard notation of $K_1 \oplus K_2:= X \mapsto K_1(X) \oplus K_2(X)$, we can then rewrite the above in terms of the dot product kernel $K(x,x')$ as follows.}
    \partial_t V_t(X_{\mathrm{all}}) &= (\tilde{K} \oplus \kappa_{\xtest})  [ (T^\pi V_t - V_t) (\Xtrain)]
\end{align}       
We emphasize that while this at first looks as though the dynamics are independent of the value $V_t(\Xtest)$, this is an artefact of the Bellman operator notation $(T^\pi V_t) (X_t)$, which hides the dependence of the Bellman targets $ T^\pi V_t$ on $\Xtest$. In particular, we can write $(T^\pi V_t)(X_t) = \Pi_{\Xtrain}[\gamma P^\pi V_t (\Xtrain \oplus \Xtest) + R^\pi]$, which makes this dependence more explicit but is less succinct.
\section{Experiment details}
\label{apx:details}
\subsection{Estimation of update rank}
\label{appx:update-details}
To estimate the update rank of an agent, we sample $k$ transitions from the agent's replay buffer and compute the matrix $A(\theta)$ as described in Section~\ref{sec:rank-exps}. We use the agent's current optimizer state and its current parameters in this computation. We then take the singular value decomposition of $A$ to obtain $k$ singular values $S = \{\sigma_1, \dots, \sigma_k\}$. We then threshold using the numerical approach taken in prior works \citep{maddox2020rethinking}, and compute the size of the set $S_{\epsilon} = \{ \sigma \in S : \sigma > \epsilon \max(S) \}$. This allows us to ignore directions of near-zero variation in the update matrix. In practice, we use $\epsilon = 0.1$. 

Because the Q-functions learned by value-based deep RL agents are vector- rather than scalar-valued functions of state, and our estimator depends on an 2-dimensional update matrix, we must make a choice on how to represent the change in the state-value function. We considered taking the maximum over actions, the mean over actions, selecting a fixed action index, and selecting the action taken in the transition on which the update was computed, and found that both choices produced similar trends. In all evaluations in this paper, Q-functions are reduced using the max operator. We apply the same approach for distributional agents by taking the expectation over the distribution associated with each state-action pair. 

To evaluate the policy-based agents, whose outputs correspond to distributions over actions, we compute the norm of the difference in the output probability distributions for each state in lieu of taking the difference of output values. I.e., the entry $A_{i,j} = \| p_\theta(x_j) - p_{\theta_i}(x_j) \|$, where the discrete probability distribution $p_\theta$ is taken as a vector. 
\subsection{ProcGen}\label{appx:procgen-details}
\begin{figure}
    \centering
    \includegraphics[width=0.35\linewidth]{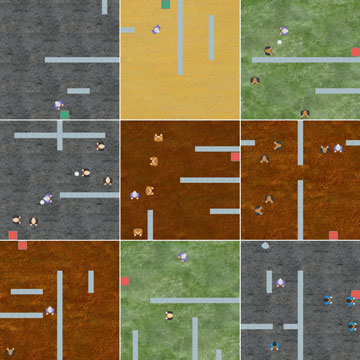}
    \caption{Example levels from the dodgeball environment.}
    \label{fig:procgen-viz}
\end{figure}
The ProcGen benchmark consists of sixteen procedurally generated environments. Each environment consists of a set of randomly generated levels, of which a fixed subset are used for training and a disjoint subset are used for evaluation. Levels differ superficially in their observations and initial sprite layouts but retain the same underlying structure, as can be seen in Figure~\ref{fig:procgen-viz}. The observation space is a box space with the RGB pixels the agent sees in a numpy array of shape (64, 64, 3).

Our PPO and DAAC agents use the same hyperparameters and implementation as is provided by \citet{raileanu2021decoupling}. Our behaviour cloning objective minimizes the KL divergence between the distillation agent the pretrained agent's policies, with an entropy bonus equal to that used to train the original PPO agent.  
\subsection{Atari}\label{appx:atari-details}

We additionally perform evaluations on environments from the Atari benchmarks. Due to computational constraints, we consider only a subset of the entire benchmark. We obtain a mixture of easy games, such as pong and boxing, and more challenging games like seaquest, where we measure difficulty by the time it takes for the agent to meet human performance. For some experiments, we used the sparse-reward environment Montezuma's Revenge.

In our distillation experiments, we train the original agent for 50M frames using $\epsilon$-greedy exploration with $\epsilon = 0.1$, and train the distillation agents for a number of updates equivalent to 10M frames of data collected online. We base our implementation off of the open-source implementations in \citet{ostrovski2021the}. 

For our behaviour cloning objective, we use the same architecture as is used for DQN, but feed the final layer of actions into a softmax to obtain a probability distribution over actions, which we denote as $P_\theta(a|x)$. Given a state-action pair taken by the target agent, we implement the following behaviour cloning loss for distillation
\begin{equation}
    \ell(\theta, x_i, a_i) = -\log P_\theta(a_i | x_i) -0.1 H(P_\theta(\cdot | x_i))
\end{equation}
where $H$ denotes the entropy of a distribution. We use a replay capacity of 1e6 and allow the pre-trained agent to collect additional data during distillation to further increase the training set size of the distilled agents.

\section{Additional numerical evaluations}
\label{appx:numerical}
We provide additional numerical evaluations to provide insight into the theoretical results of Section~\ref{sec:learning-smoothness}.
\subsection{Fourier analysis}
We begin by studying the Fourier decomposition of value and reward functions in popular Atari domains by treating the value function as a function of \textit{time} rather than as a function of \textit{observations}. In this sense, the Fourier decomposition is measuring the continuity of the value function with respect to time and so is a closer approximation of the notion of smoothness we focus on in Section~\ref{sec:vf_gen}. We show our evaluations in Figure~\ref{fig:atari_fourier}.
\begin{figure}
    \centering
    \includegraphics[width=\linewidth]{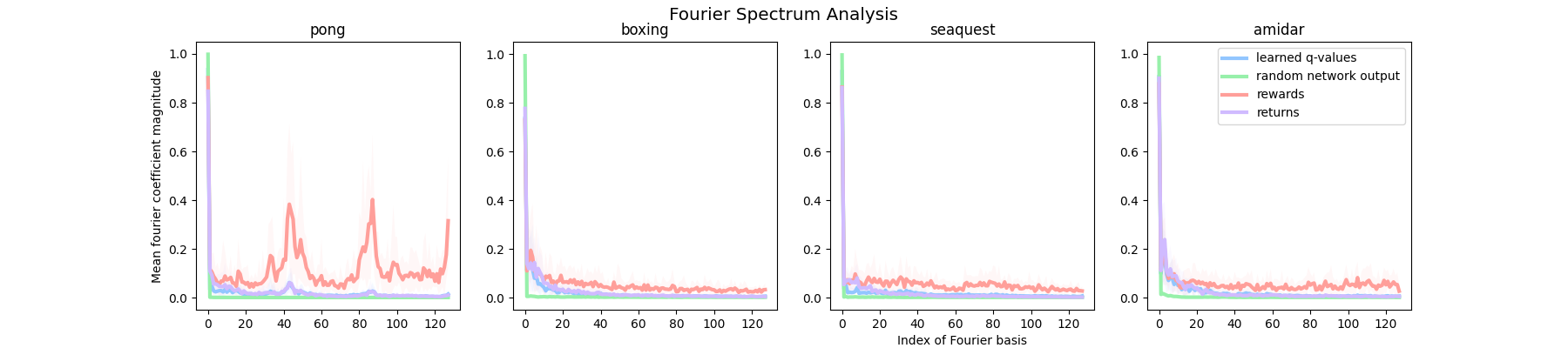}
    \includegraphics[width=\linewidth]{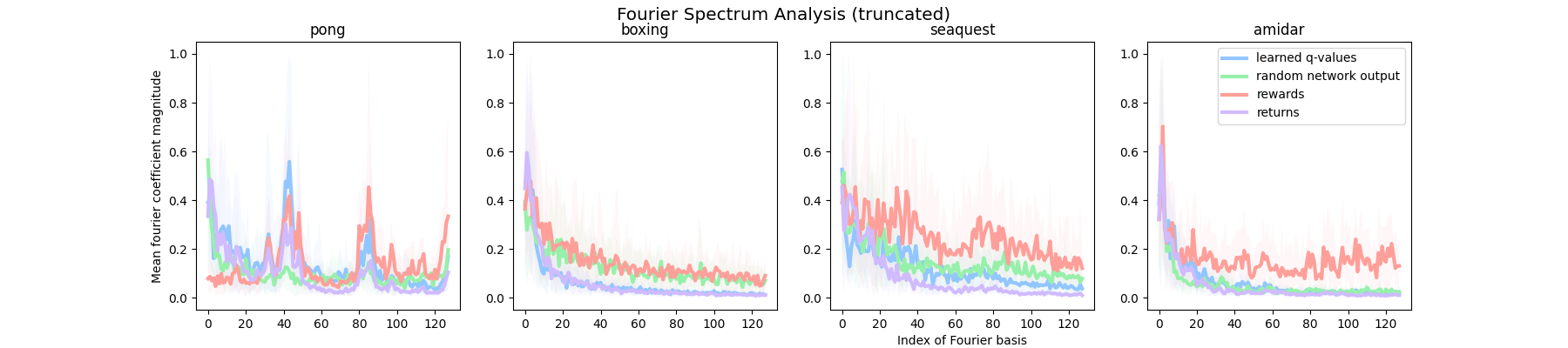}
    \caption{Fourier decomposition of Atari value functions when viewed as a function of time. We sample $k$ consecutive states from the replay buffer and compute the predicted value on each state (fixing an arbitrary action) to get a function $V: \{1, \dots, k \} \rightarrow \mathbb{R}$. We then compute the Fourier decomposition of this function. The top row shows indices $k=0 \dots 50$, while the bottom row omits the $k=0$ index (the constant function) to better illustrate the rate of decay of the spectrum of each function.}
    \label{fig:atari_fourier}
\end{figure}

\subsection{Kernel gradient descent}
\label{appx:kernel-gd}
We include an illustration of the kernel gradient descent dynamics described in Section~\ref{sec:fa_gen} in Figure~\ref{fig:kernel-dynamics}. We run our evaluations using a radial basis function (RBF) kernel of varying lengthscale, with shorter lengthscales corresponding to weaker generalization between states. While the shorter lengthscale corresponds to more stable learning dynamics and better fitting of the value function on the training set, it also induces greater value approximation error on the test states. In contrast, the longer lengthscales result in better generalization to novel test states under Monte Carlo dynamics, but result in divergence for large values of $\gamma$.
\begin{figure}
    \centering
    \includegraphics[width=0.49\linewidth]{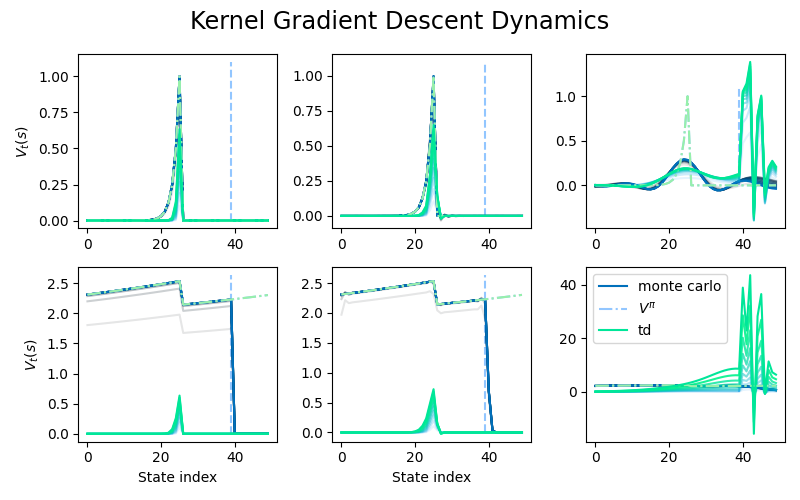} \vline 
    \includegraphics[width=0.49\linewidth]{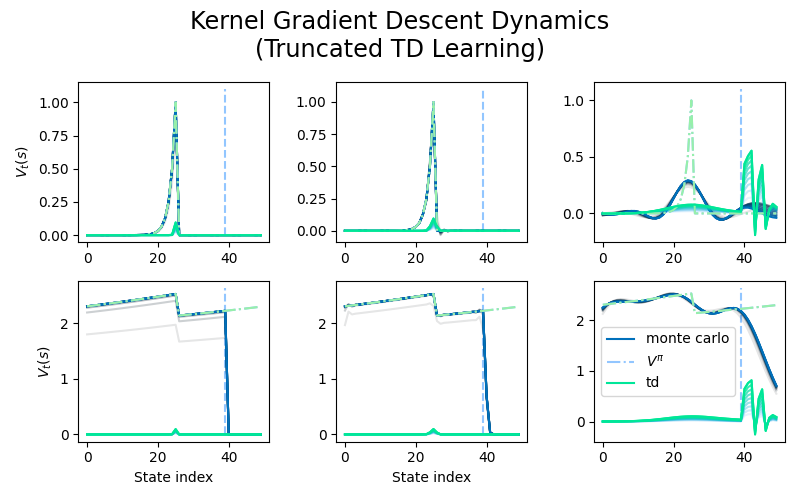}
    \caption{Numerical evaluations of kernel gradient descent with an RBF kernel. The MDP in question is a "circle MDP" whose states are integers $n \in \{1, \dots, 50\}$. We assume the agent is `trained' on states 1 to 40, and does not perform value function updates on the final ten states, use the policy which always takes the agent from state $n$ to $n+1 \mod 50$, and set a single reward at state 25. Each row corresponds to a different value of the discount factor $\gamma$: the top corresponds to $\gamma = 0.5$, and the bottom to $\gamma = 0.99$. Each column corresponds to the lengthscale which parameterizes the kernel, going left to right: 0.01, 1.0, and 100. The left hand side and right hand side are distinguished by the number of update steps which the TD dynamics are evaluated for. The LHS runs TD for  only 20 steps, while the RHS runs it for 100 steps. MC updates are run for 1500 steps on both figures. We see that for $\gamma = 0.99$, the larger-lengthscale kernel predictions diverge under TD dynamics, though not Monte Carlo. The Monte Carlo dynamics further nicely illustrate the trade-off between generalizing out of the training set and ability to fit the discontinuities of the value function on the training set. The larger lengthscale has lower MSE from the value function on the test set, but fails to fit the discontinuity  of the value function at the reward state. Meanwhile, the smaller lengthscales easily fit the value function on the training set but predict zero for all over states. }
    \label{fig:kernel-dynamics}
\end{figure}

Additionally, as promised in Section~\ref{sec:fa_gen}, we illustrate the role of smooth eigenfunctions in generalization in Figure~\ref{fig:kernel-generalization}. To produce this figure, we randomly generate an unweighted graph and then construct an MDP whose dynamics correspond to a random walk on this graph. We consider the generalization error of a kernel regression process where the kernel $K_S$ is of the form $ K_S(x,y) = \sum_{i \in S} v_{\lambda_i}(x) v_{\lambda_i}(y)$ for some $S \subseteq \mathrm{spec}(P^\pi)$. In the right-hand-side plot of Figure~\ref{fig:kernel-generalization}, we set $S=\{1, \dots, 20\}$, so that our analysis concentrates on smooth eigenfunctions. We then consider the generalization error of this smooth kernel when we only regress on a subset of the state space selected uniformly at random\footnote{Because the MDP-generating process is invariant to permutations of the state indices, we sample the indices $\{1, \dots, \lfloor |\states| \times \mathrm{training fraction} \rfloor \}$, and average over randomly generated MDPs. }. We study the effect of varying the size of this set, i.e. the fraction of states in the training set, in Figure~\ref{fig:kernel-generalization}, in order to quantify the degree to which additional information about the value function translates to improved generalization.
We consider three regression problems: regression on $V^\pi$, regression on the projection of $V^\pi$ onto the span of $T = \{v_1, \dots, v_{20} \}$, and $B = \{v_{n-19}, \dots, v_{n} \}$. Unsurprisingly, we see that the smooth kernel is able to improve its generalization performance as the size of the training set increases when it is set to regress $V^\pi$ or $\Pi_{T} V^\pi = V^\pi_T$. However, when the kernel regresses only on the projection of $V^\pi$ onto the non-smooth eigenvectors, we do not see a benefit of adding additional training points: because there is no information about the smooth components of the function in the targets, adding additional data points will not help to improve regression accuracy. The left hand side of the figure shows similarly that fitting local information in the form of $n$-step returns for small $n$ also does not provide the kernel with sufficient information for it to be able to extrapolate and improve its generalization error as the size of the training set increases.

\begin{figure}
    \centering
    \includegraphics[width=0.8\linewidth]{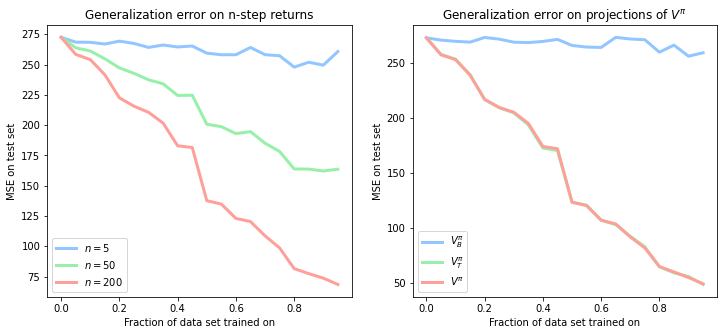}
    \caption{Generalization of predicted function under kernel regression using $n$-step return targets evaluated on a random subset of states (left), and projecting value function onto top or bottom eigenvectors of $P^\pi$ (right). We see a similar trend where for larger $n$ (corresponding to smoother targets), the kernel regression method generalizes better with increasing dataset sizes. For smaller $n$ and for the projection of $V^\pi$ onto non-smooth eigenvectors, adding additional data points does not improve generalization performance.}
    \label{fig:kernel-generalization}
\end{figure}
\section{Additional empirical results}
\label{apx:more-results}
\subsection{Additional value distillation results}
We consider three different types of regression to the outputs of the pre-trained network, along with two more traditional bootstrapping methods for offline RL. \texttt{Q-regression} regresses the outputs of the distilled network to those of the pre-trained network for every action. \texttt{qa-regression} only  does q-value regression on the action taken by the pre-trained agent. \texttt{adv-regression} regresses on the advantage function (computed as the q-value minus the mean over all actions) given  by the pre-trained agent; \texttt{qr} does quantile regression q-learning on the offline data; \texttt{double-q} performs a standard double q-learning update on the offline data.

We find that all of these methods obtain an initial update rank significantly below that of the pre-trained network when they begin training, which increases over time. Regression to the advantages obtains a significantly lower update rank than any other method, suggesting that the advantage function may be much smoother than the action-value function. With respect to performance on the original environment, we see that the methods which use all action values at every update obtain significantly higher performance than those which only update a single action at a time. This improvement in performance is not mediated by an auxiliary task effect or an  increase in the network's ability to distinguish states: the advantage regression network attains low update rank  but high performance, while the qr-regression task provides a great deal of information to the representation but is not competitive with the q-regression network. 
\begin{figure}
    \centering
    \includegraphics[width=0.957\linewidth]{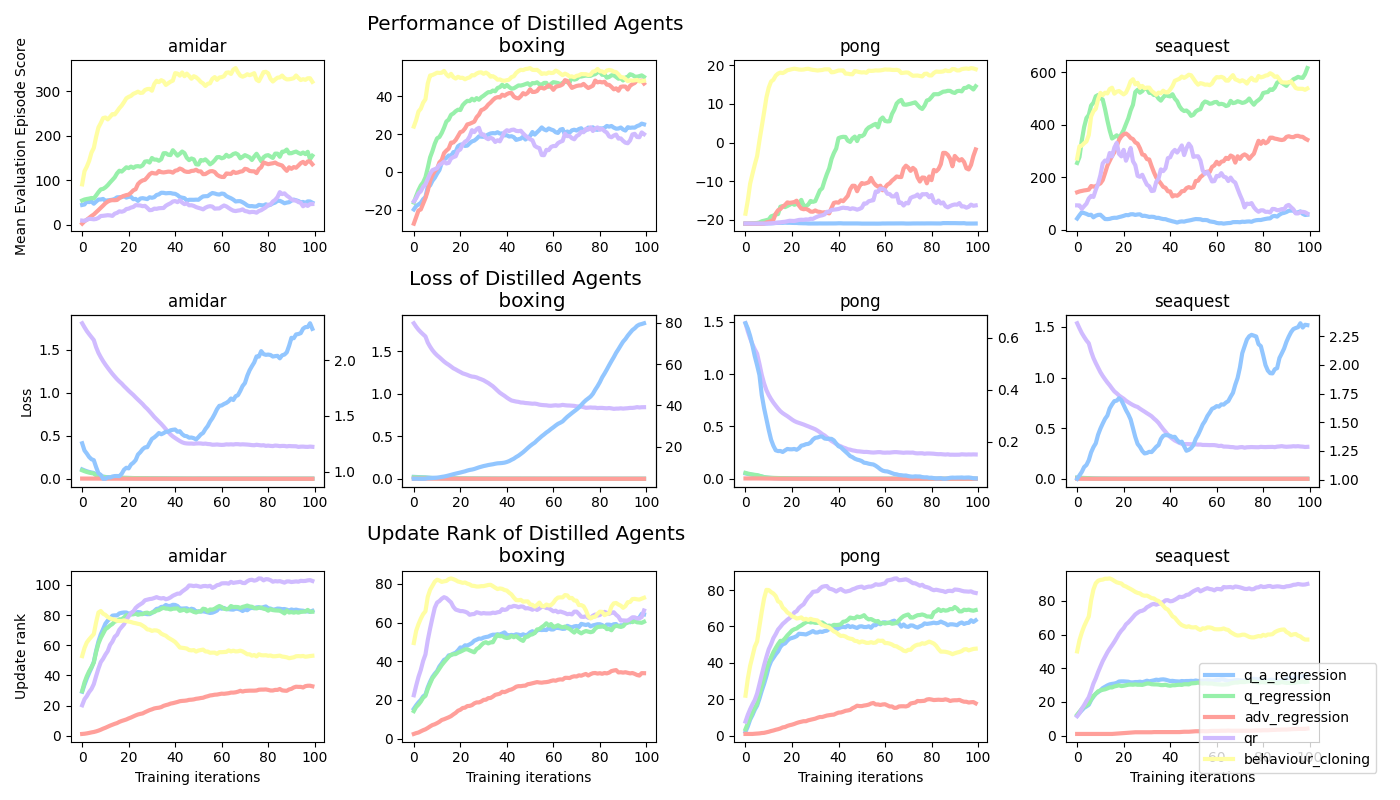}
    \caption{Results from post-training distillation on a variety of objectives. We note that advantage regression tends to exhibit the lowest update rank, with the qr agent tending to exhibit the highest update rank and the q-regression objectives falling somewhere in between. Because the behaviour cloning objective minimizes a cross-entropy loss rather than a regression loss, further investigation is required to understand how the trajectory of its update dimension differs from those of the regression objectives.}
    \label{fig:tandem-apx}
\end{figure}
\subsection{More detailed update trajectories}

We include a more detailed view of the update matrices obtained by DQN and C51 agents during the first 7 million frames of training, roughly 5\% of the training budget, in Figure~\ref{fig:updates-long}. We see that even early in training, the DQN and C51 agents both exhibit significant overfitting behaviour. Note that states are sampled uniformly at random from the replay buffer, and then assigned an index based on the output of a clustering algorithm to improve readability of the figures.

\begin{figure}
    \centering
    \includegraphics[width=0.48\linewidth]{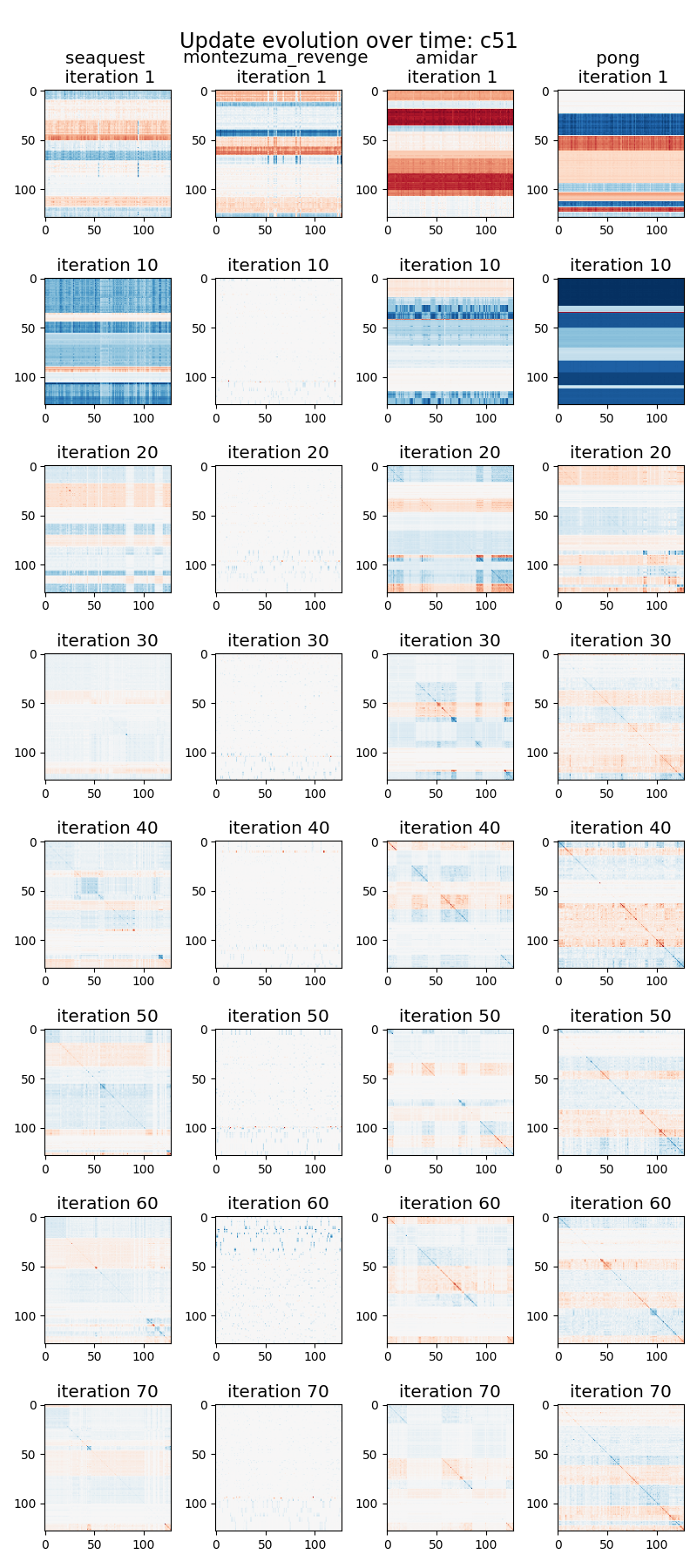}
    \vline
    \includegraphics[width=0.48\linewidth]{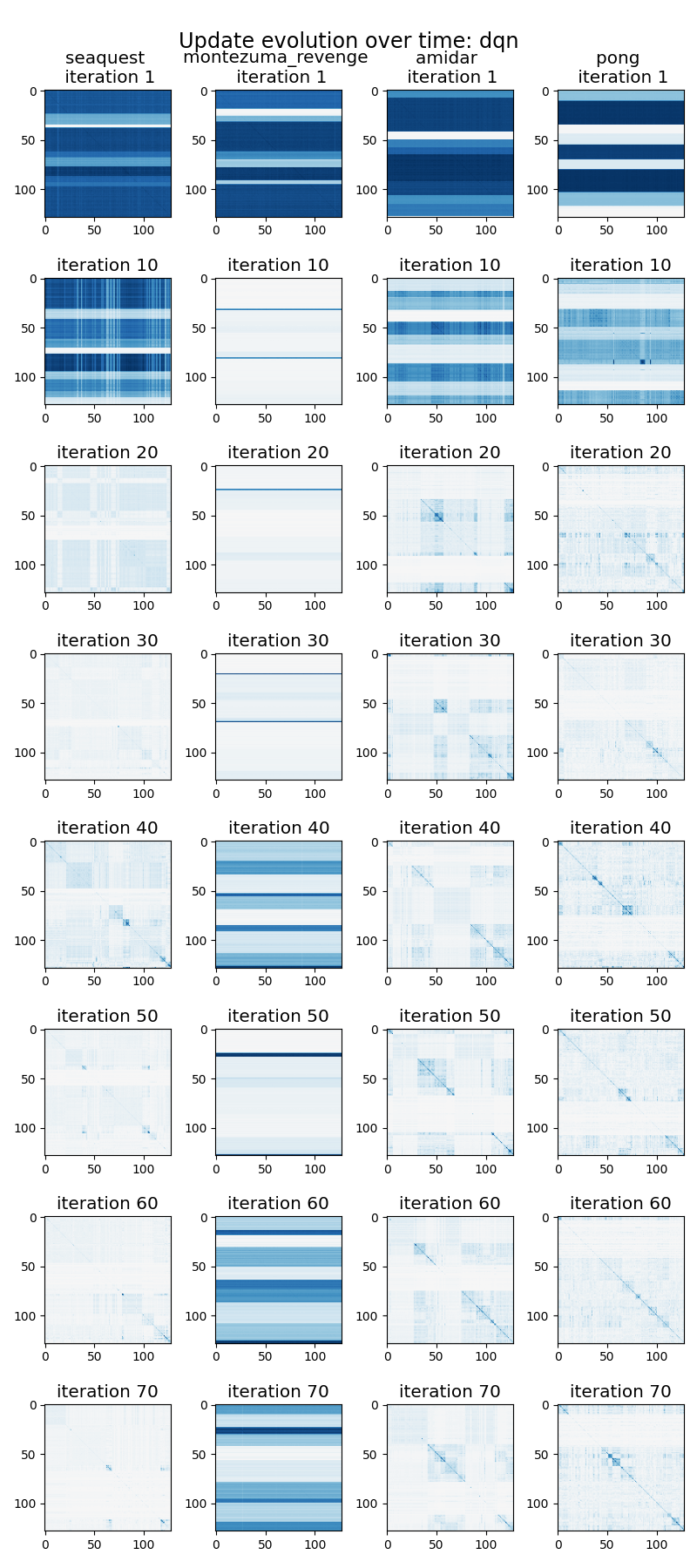}
    \caption{Update matrices for distributional and DQN agents on four games from the Atari suite, chosen to represent a range of reward densities and difficulties. Each iteration corresponds to 1e5 training frames.}
    \label{fig:updates-long}
\end{figure}

\chapter{Generalization across environments}

\section{Implementation details}
\subsection{Model learning: rich observations}
\label{app:model_nonlinear_implementation}
For the model learning experiments we use an almost identical encoder architecture as in~\citet{deepmindcontrolsuite2018}, with two more convolutional layers to the convnet trunk. Secondly, we use \texttt{ReLU} activations after each convolutional layer, instead of \texttt{ELU}. We use kernels of size $3 \times 3$ with $32$ channels  for all the convolutional layers and set stride to $1$ everywhere, except of the first convolutional layer, which has stride $2$. We then take the output of the convolutional net and feed it into a single fully-connected layer normalized by \texttt{LayerNorm}~\citep{ba2016layernorm}. Finally, we add \texttt{tanh} nonlinearity to the $50$ dimensional output of the fully-connected layer.

The decoder consists of one fully-connected layer that is then followed by four deconvolutional layers. We use \texttt{ReLU} activations after each layer, except the final deconvolutional layer that produces pixels representation. Each deconvolutional layer has kernels of size $3 \times 3$ with $32$ channels and stride $1$, except of the last layer, where stride is $2$.

The dynamics and reward models are all MLPs with two hidden layers with 200 neurons each and \texttt{ReLU} activations.

\subsection{Reinforcement learning}
\label{app:rl_implementation}
For the reinforcement learning experiments we modify the Soft Actor-Critic PyTorch implementation by \citet{pytorch_sac} and augment with a shared encoder between the actor and critic, the general model $f_s$ and task-specific models $f_{\eta}^e$. The forward models are multi-layer perceptions with ReLU non-linearities and two hidden layers of 200 neurons each. The encoder is a linear layer that maps to a 50-dim hidden representation. We also use L1 regularization on the $S$ latent representation. We add two additional dimensions to the state space, a spurious correlation dimension that is a multiplicative factor of the last dimension of the ground truth state, as well as an environment id. We add Gaussian noise $\mathcal{N}(0, 0.01)$ to the original state dimension, similar to how \citet{arjovsky2019invariant} incorporate noise in the label to make the task harder for the baseline.

Soft Actor Critic (SAC)~\cite{haarnoja2018sac} is an off-policy actor-critic method that uses the maximum entropy framework to derive soft policy iteration. At each iteration, SAC performs soft policy evaluation and improvement steps. The policy evaluation step fits a parametric soft Q-function $Q(x_t, a_t)$  using transitions sampled from the replay buffer $\mathcal{D}$ by minimizing the soft Bellman residual,
\begin{equation*}
    J(Q) = \mathbb{E}_{(x_t, x_t, r_t, x_{t+1}) \sim \mathcal{D}} \bigg[ \bigg(Q(x_t, a_t) - r_t - \gamma \Bar{V}(x_{t+1})\bigg)^2  \bigg].
\end{equation*}
The target value function $\Bar{V}$ is approximated via a Monte-Carlo estimate of the following expectation,
\begin{equation*}
    \Bar{V}(x_{t+1}) = \mathbb{E}_{a_{t+1} \sim \pi} \big[\Bar{Q}(x_{t+1}, a_{t+1}) - \alpha  \log \pi(a_{t+1}|x_{t+1}) \big],
\end{equation*}
where $\bar{Q}$ is the target soft Q-function parameterized by a weight vector obtained from an exponentially moving average of the Q-function weights to stabilize training. The  policy improvement step then attempts to project a parametric policy $\pi(a_t|x_t)$  by minimizing KL divergence between the  policy and a Boltzmann distribution induced by the Q-function, producing the following objective,
\begin{equation*}
    J(\pi)= \mathbb{E}_{x_t \sim \mathcal{D}} \bigg[ \mathbb{E}_{a_t \sim \pi} [\alpha \log (\pi(a_t | x_t)) - Q(x_t, a_t)] \bigg].
\end{equation*}

We provide the hyperparameters used for the RL experiments in \cref{table:rl_hyper_params}.

\begin{table}[hb!]
\centering
\begin{tabular}{|l|c|}
\hline
Parameter name        & Value \\
\hline
Replay buffer capacity & $1000000$ \\
Batch size & $1024$ \\
Discount $\gamma$ & $0.99$ \\
Optimizer & Adam \\
Critic learning rate & $10^{-5}$ \\
Critic target update frequency & $2$ \\
Critic Q-function soft-update rate $\tau_{\textrm{Q}}$ & 0.005 \\
Critic encoder soft-update rate $\tau_{\textrm{enc}}$ & 0.005 \\
Actor learning rate & $10^{-5}$ \\
Actor update frequency & $2$ \\
Actor log stddev bounds & $[-5, 2]$ \\
Encoder learning rate & $10^{-5}$ \\
Decoder learning rate & $10^{-5}$ \\
Decoder weight  decay & $10^{-7}$  \\
L1 regularization weight & $10^{-5}$ \\
Temperature learning rate & $10^{-4}$ \\
Temperature Adam's $\beta_1$ & $0.9$ \\
Init temperature & $0.1$ \\
\hline
\end{tabular}\\
\caption{\label{table:rl_hyper_params} A complete overview of used hyper parameters.}
\end{table}


\chapter*{Contributions to joint-authored work}
The research in this thesis is the product of a number of valuable collaborations. 

\textbf{Chapter 3} is based on work presented at the NeurIPS 2019 workshop on machine learning with guarantees \citep{lyle2020benefits}. All empirical results in this section were proposed and implemented by me. While I produced an initial analysis of PAC-Bayes bounds under symmetries and showed a preliminary version of Theorem~\ref{lemma:KL:gen} for finite groups independently, I worked closely with Benjamin Bloem-Reddy to generalize these initial results to the form presented in this thesis. 

\textbf{Chapter 4} is principally built on the NeurIPS 2020 paper \citep{lyle2020bayesian}, but also draws on some of the insights developed in a tandem work led by Binxin Ru and myself which was presented at NeurIPS 2021 \citep{ru2020revisiting}. I was responsible for all theoretical and empirical results concerning linear models and gradient alignment, while Lisa Schut extended and improved on my initial toy experiments on neural networks to generate Figures~\ref{fig:mod_select_dnn} and \ref{fig:mod_select_dnn_parallel}. Figure~\ref{fig:tse-variance} was produced independently by me, but was inspired by analysis initially performed by Binxin Ru. 

\textbf{Chapter 5} was written in close collaboration with Mark Rowland and Will Dabney, and the initial research question sparking this project was brainstormed during joint meetings involving all authors. I ran the evaluations for Figure~\ref{fig:feature-viz} and performed the evaluations on linear value function evolution in Appendix~\ref{sec:ensemble-dynamics}. Mark proposed the continuous-time formalism and proved the initial subspace convergence result for a single value function, while I proposed the `exact feature update' framework and proved the results concerning auxiliary tasks and ensemble prediction.  

\textbf{Chapter 6} is based on work I led during an internship at DeepMind \citep{lyle2021understanding}. I conducted additional analysis into the learned representations of the agents used for the empirical results of that work, and eventually formulated the main hypotheses of this chapter based on that analysis. I wrote and ran the experiment code for the supervised and Atari benchmarks. My co-authors ran code I had written during my internship after I returned to Oxford to generate final versions of several figures, and provided helpful discussion throughout the project. 

\textbf{Chapter 7} is based on \citep{lyle2022generalization}, for which I was the lead author. I was responsible for the implementations in all empirical results and for the proposal and initial proofs of the theoretical results. The other authors on the paper provided feedback on early drafts of the paper and recommendations for interesting experimental conditions to pursue. 

\textbf{Chapter 8} is based on a paper co-first-authored by myself and Amy Zhang. I stated and proved all theoretical results included in the thesis, proposed the algorithm for linear state abstractions, and implemented the linear model experiments. Both Amy and I worked together to propose the nonlinear MISA architecture and training objective.

\bibliographystyle{plainnat}
\bibliography{references.bib}

\begin{thebibliography}{319}
\providecommand{\natexlab}[1]{#1}
\providecommand{\url}[1]{\texttt{#1}}
\expandafter\ifx\csname urlstyle\endcsname\relax
  \providecommand{\doi}[1]{doi: #1}\else
  \providecommand{\doi}{doi: \begingroup \urlstyle{rm}\Url}\fi

\bibitem[Abdolshah et~al.(2021)Abdolshah, Le, George, Gupta, Rana, and
  Venkatesh]{abdolshah2021new}
Majid Abdolshah, Hung Le, Thommen~Karimpanal George, Sunil Gupta, Santu Rana,
  and Svetha Venkatesh.
\newblock A new representation of successor features for transfer across
  dissimilar environments.
\newblock In \emph{Proceedings of the 38th International Conference on Machine
  Learning (ICML)}, pages 1--9. PMLR, 2021.

\bibitem[Abel et~al.(2019)Abel, Arumugam, Asadi, Jinnai, Littman, and
  Wong]{abel2019state}
David Abel, Dilip Arumugam, Kavosh Asadi, Yuu Jinnai, Michael~L Littman, and
  Lawson~LS Wong.
\newblock State abstraction as compression in apprenticeship learning (icml).
\newblock In \emph{Proceedings of the AAAI Conference on Artificial
  Intelligence. AAAI Press.}, 2019.

\bibitem[Achille et~al.(2018)Achille, Rovere, and Soatto]{achille2018critical}
Alessandro Achille, Matteo Rovere, and Stefano Soatto.
\newblock Critical learning periods in deep networks.
\newblock In \emph{International Conference on Learning Representations
  (ICLR)}, 2018.

\bibitem[Advani et~al.(2020)Advani, Saxe, and Sompolinsky]{advani2020high}
Madhu~S Advani, Andrew~M Saxe, and Haim Sompolinsky.
\newblock High-dimensional dynamics of generalization error in neural networks.
\newblock \emph{Neural Networks}, 132:\penalty0 428--446, 2020.

\bibitem[Agarwal et~al.(2020)Agarwal, Schuurmans, and
  Norouzi]{agarwal2019striving}
Rishabh Agarwal, Dale Schuurmans, and Mohammad Norouzi.
\newblock An optimistic perspective on offline reinforcement learning.
\newblock In \emph{Proceedings of the 37th International Conference on Machine
  Learning (ICML)}, 2020.

\bibitem[Ahuja et~al.(2020)Ahuja, Shanmugam, Varshney, and
  Dhurandhar]{ahuja2020invariant}
Kartik Ahuja, Karthikeyan Shanmugam, Kush Varshney, and Amit Dhurandhar.
\newblock Invariant risk minimization games.
\newblock In \emph{Proceedings of the 37th International Conference on Machine
  Learning (ICML)}, pages 145--155. PMLR, 2020.

\bibitem[Akaike(1998)]{Akaike1998}
Hirotogu Akaike.
\newblock \emph{Information Theory and an Extension of the Maximum Likelihood
  Principle}, pages 199--213.
\newblock Springer, 1998.

\bibitem[Alesiani et~al.(2021)Alesiani, Yu, and Niepert]{alesiani2021continual}
Francesco Alesiani, Shujian Yu, and Mathias Niepert.
\newblock Continual invariant risk minimization, 2021.

\bibitem[Amit and Meir(2018)]{amit2018mlpacbayes}
Ron Amit and Ron Meir.
\newblock Meta-learning by adjusting priors based on extended {PAC}-{B}ayes
  theory.
\newblock In Jennifer Dy and Andreas Krause, editors, \emph{Proceedings of the
  35th International Conference on Machine Learning (ICML)}, volume~80 of
  \emph{Proceedings of Machine Learning Research}, pages 205--214,
  Stockholmsmässan, Stockholm Sweden, 10--15 Jul 2018. PMLR.

\bibitem[Anschel et~al.(2017)Anschel, Baram, and Shimkin]{anschel2017averaged}
Oron Anschel, Nir Baram, and Nahum Shimkin.
\newblock Averaged-{DQN}: Variance reduction and stabilization for deep
  reinforcement learning.
\newblock In \emph{Proceedings of the 34th International Conference on Machine
  Learning (ICML)}, 2017.

\bibitem[Arjovsky et~al.(2019)Arjovsky, Bottou, Gulrajani, and
  Lopez-Paz]{arjovsky2019invariant}
Martin Arjovsky, L{\'e}on Bottou, Ishaan Gulrajani, and David Lopez-Paz.
\newblock Invariant risk minimization.
\newblock \emph{arXiv preprint arXiv:1907.02893}, 2019.

\bibitem[Arora et~al.(2019)Arora, Du, Hu, Li, and Wang]{arora2019fine}
Sanjeev Arora, Simon~S Du, Wei Hu, Zhiyuan Li, and Ruosong Wang.
\newblock Fine-grained analysis of optimization and generalization for
  overparameterized two-layer neural networks.
\newblock In \emph{Proceedings of the 36th International Conference on Machine
  Learning (ICML)}, 2019.

\bibitem[Arpit et~al.(2017)Arpit, Jastrz{\k{e}}bski, Ballas, Krueger, Bengio,
  Kanwal, Maharaj, Fischer, Courville, Bengio, et~al.]{arpit2017closer}
Devansh Arpit, Stanis{\l}aw Jastrz{\k{e}}bski, Nicolas Ballas, David Krueger,
  Emmanuel Bengio, Maxinder~S Kanwal, Tegan Maharaj, Asja Fischer, Aaron
  Courville, Yoshua Bengio, et~al.
\newblock A closer look at memorization in deep networks.
\newblock In \emph{Proceedings of the 34th International Conference on Machine
  Learning (ICML)}, pages 233--242. PMLR, 2017.

\bibitem[Ash and Adams(2020)]{ash2020warm}
Jordan Ash and Ryan~P Adams.
\newblock On warm-starting neural network training.
\newblock In \emph{Advances in Neural Information Processing Systems}, 2020.

\bibitem[Ba et~al.(2016)Ba, Kiros, and Hinton]{ba2016layernorm}
Jimmy~Lei Ba, Jamie~Ryan Kiros, and Geoffrey~E. Hinton.
\newblock Layer normalization.
\newblock \emph{arXiv preprint}, 2016.

\bibitem[Baird(1993)]{baird1993advantage}
L~Baird.
\newblock Advantage updating.
\newblock \emph{Technical Report}, 1993.

\bibitem[Barreto et~al.(2017)Barreto, Dabney, Munos, Hunt, Schaul, van Hasselt,
  and Silver]{barreto2017successor}
Andr{\'e} Barreto, Will Dabney, R{\'e}mi Munos, Jonathan~J Hunt, Tom Schaul,
  Hado~P van Hasselt, and David Silver.
\newblock Successor features for transfer in reinforcement learning.
\newblock In \emph{Advances in neural information processing systems}, pages
  4055--4065, 2017.

\bibitem[Barrett and Dherin(2021)]{barrett2021implicit}
David Barrett and Benoit Dherin.
\newblock Implicit gradient regularization.
\newblock In \emph{International Conference on Learning Representations
  (ICLR)}, 2021.

\bibitem[Bartlett(1997)]{bartlett1997valid}
Peter~L Bartlett.
\newblock For valid generalization the size of the weights is more important
  than the size of the network.
\newblock In \emph{Advances in neural information processing systems}, pages
  134--140, 1997.

\bibitem[Bartlett et~al.(2017)Bartlett, Foster, and
  Telgarsky]{bartlett2017spectrally}
Peter~L Bartlett, Dylan~J Foster, and Matus~J Telgarsky.
\newblock Spectrally-normalized margin bounds for neural networks.
\newblock In \emph{Advances in Neural Information Processing Systems}, pages
  6240--6249, 2017.

\bibitem[Bartlett et~al.(2020)Bartlett, Long, Lugosi, and
  Tsigler]{bartlett2020benign}
Peter~L Bartlett, Philip~M Long, G{\'a}bor Lugosi, and Alexander Tsigler.
\newblock Benign overfitting in linear regression.
\newblock \emph{Proceedings of the National Academy of Sciences}, 117\penalty0
  (48):\penalty0 30063--30070, 2020.

\bibitem[Bartoldson et~al.(2020)Bartoldson, Morcos, Barbu, and
  Erlebacher]{bartoldson2020generalization}
Brian Bartoldson, Ari Morcos, Adrian Barbu, and Gordon Erlebacher.
\newblock The generalization-stability tradeoff in neural network pruning.
\newblock \emph{Advances in Neural Information Processing Systems},
  33:\penalty0 20852--20864, 2020.

\bibitem[Basu(1955)]{basu1955}
D.~Basu.
\newblock On statistics independent of a complete sufficient statistic.
\newblock \emph{Sankhyā: The Indian Journal of Statistics (1933-1960)},
  15\penalty0 (4):\penalty0 377--380, 1955.
\newblock ISSN 00364452.

\bibitem[Baum and Haussler(1989)]{baum1989size}
Eric~B Baum and David Haussler.
\newblock What size net gives valid generalization?
\newblock In \emph{Advances in neural information processing systems}, pages
  81--90, 1989.

\bibitem[Beck et~al.(2021)Beck, Sivasubramanian, Dani, Ramakrishnan, and
  Iyer]{beck2021effective}
Nathan Beck, Durga Sivasubramanian, Apurva Dani, Ganesh Ramakrishnan, and
  Rishabh Iyer.
\newblock Effective evaluation of deep active learning on image classification
  tasks.
\newblock \emph{arXiv preprint}, 2021.

\bibitem[Behzadian and Petrik(2018)]{behzadian2018feature}
Bahram Behzadian and Marek Petrik.
\newblock Feature selection by singular value decomposition for reinforcement
  learning.
\newblock In \emph{ICML Prediction and Generative Modeling Workshop}, 2018.

\bibitem[Belkin et~al.(2019)Belkin, Hsu, Ma, and Mandal]{belkin2018reconciling}
Mikhail Belkin, Daniel Hsu, Siyuan Ma, and Soumik Mandal.
\newblock Reconciling modern machine learning and the bias-variance trade-off.
\newblock \emph{Proceedings of the National Academy of Sciences}, 116\penalty0
  (32):\penalty0 15849--15854, 2019.

\bibitem[Bellemare et~al.(2013)Bellemare, Naddaf, Veness, and
  Bowling]{bellemare2013arcade}
Marc~G Bellemare, Yavar Naddaf, Joel Veness, and Michael Bowling.
\newblock The arcade learning environment: An evaluation platform for general
  agents.
\newblock \emph{Journal of Artificial Intelligence Research}, 47:\penalty0
  253--279, 2013.

\bibitem[Bellemare et~al.(2017)Bellemare, Dabney, and
  Munos]{bellemare2017distributional}
Marc~G Bellemare, Will Dabney, and R{\'e}mi Munos.
\newblock A distributional perspective on reinforcement learning.
\newblock In \emph{Proceedings of the 34th International Conference on Machine
  Learning (ICML)}, 2017.

\bibitem[Bellemare et~al.(2019)Bellemare, Dabney, Dadashi, Taiga, Castro, Roux,
  Schuurmans, Lattimore, and Lyle]{bellemare2019geometric}
Marc~G Bellemare, Will Dabney, Robert Dadashi, Adrien~Ali Taiga, Pablo~Samuel
  Castro, Nicolas~Le Roux, Dale Schuurmans, Tor Lattimore, and Clare Lyle.
\newblock A geometric perspective on optimal representations for reinforcement
  learning.
\newblock \emph{Advances in neural information processing systems}, 2019.

\bibitem[Bengio et~al.(2017)Bengio, Thomas, Pineau, Precup, and
  Bengio]{bengio2017independently}
Emmanuel Bengio, Valentin Thomas, Joelle Pineau, Doina Precup, and Yoshua
  Bengio.
\newblock Independently controllable features.
\newblock In \emph{Reinforcement Learning and Decision Making (RLDM)}, 2017.

\bibitem[Bengio et~al.(2020)Bengio, Pineau, and Precup]{bengio2020interference}
Emmanuel Bengio, Joelle Pineau, and Doina Precup.
\newblock Interference and generalization in temporal difference learning
  (icml).
\newblock In \emph{Proceedings of the 37th International Conference on Machine
  Learning (ICML)}, pages 767--777. PMLR, 2020.

\bibitem[Bengio et~al.(2014)Bengio, Mirza, Goodfellow, Courville, and
  Da]{bengio2013empirical}
Yoshua Bengio, Mehdi Mirza, Ian Goodfellow, Aaron Courville, and Xia Da.
\newblock An empirical investigation of catastrophic forgeting in
  gradient-based neural networks.
\newblock In \emph{International Conference on Learning Representations
  (ICLR)}, 2014.

\bibitem[Benjamin et~al.(2019)Benjamin, Rolnick, and
  Kording]{benjamin2018measuring}
Ari Benjamin, David Rolnick, and Konrad Kording.
\newblock Measuring and regularizing networks in function space.
\newblock In \emph{International Conference on Learning Representations
  (ICLR)}, 2019.

\bibitem[Benton et~al.(2021)Benton, Maddox, Lotfi, and Wilson]{benton2021loss}
Gregory~W Benton, Wesley~J Maddox, Sanae Lotfi, and Andrew~Gordon Wilson.
\newblock Loss surface simplexes for mode connecting volumes and fast
  ensembling.
\newblock \emph{Proceedings of the 38th International Conference on Machine
  Learning (ICML)}, 18--24 Jul 2021.

\bibitem[Berariu et~al.(2021)Berariu, Czarnecki, De, Bornschein, Smith,
  Pascanu, and Clopath]{berariu2021study}
Tudor Berariu, Wojciech Czarnecki, Soham De, Jorg Bornschein, Samuel Smith,
  Razvan Pascanu, and Claudia Clopath.
\newblock A study on the plasticity of neural networks.
\newblock \emph{arXiv preprint arXiv:2106.00042}, 2021.

\bibitem[Bertsekas(2018)]{bertsekas2018feature}
Dimitri~P Bertsekas.
\newblock Feature-based aggregation and deep reinforcement learning: A survey
  and some new implementations.
\newblock \emph{IEEE/CAA Journal of Automatica Sinica}, 6\penalty0
  (1):\penalty0 1--31, 2018.

\bibitem[Bertsekas and Tsitsiklis(1996)]{bertsekas1996neuro}
Dimitri~P Bertsekas and John~N Tsitsiklis.
\newblock \emph{Neuro-Dynamic Programming}.
\newblock Athena Scientific, 1st edition, 1996.

\bibitem[Blei et~al.(2017)Blei, Kucukelbir, and McAuliffe]{blei2017variational}
David~M Blei, Alp Kucukelbir, and Jon~D McAuliffe.
\newblock Variational inference: A review for statisticians.
\newblock \emph{Journal of the American statistical Association}, 112\penalty0
  (518):\penalty0 859--877, 2017.

\bibitem[{Bloem-Reddy} and {Teh}(2019)]{invariantdistributions}
B.~{Bloem-Reddy} and Y.\~W. {Teh}.
\newblock {Probabilistic symmetry and invariant neural networks}.
\newblock \emph{Journal of Machine Learning Research}, 2019.

\bibitem[Blundell et~al.(2015)Blundell, Cornebise, Kavukcuoglu, and
  Wierstra]{blundell2015weight}
Charles Blundell, Julien Cornebise, Koray Kavukcuoglu, and Daan Wierstra.
\newblock Weight uncertainty in neural network.
\newblock In \emph{Proceedings of the 32nd International Conference on Machine
  Learning (ICML)}, pages 1613--1622, 2015.

\bibitem[Borkar and Meyn(2000)]{borkar2000ode}
Vivek~S. Borkar and Sean~P. Meyn.
\newblock The o.d.e. method for convergence of stochastic approximation and
  reinforcement learning.
\newblock \emph{SIAM J. Control. Optim.}, 38:\penalty0 447--469, 2000.

\bibitem[Bousquet and Elisseeff(2002)]{bousquet2002stability}
Olivier Bousquet and Andr{\'e} Elisseeff.
\newblock Stability and generalization.
\newblock \emph{The Journal of Machine Learning Research}, 2:\penalty0
  499--526, 2002.

\bibitem[Boyan(1999)]{boyan1999least}
Justin~A Boyan.
\newblock Least-squares temporal difference learning (icml).
\newblock In \emph{ICML}, pages 49--56. Citeseer, 1999.

\bibitem[Bradbury et~al.(2018)Bradbury, Frostig, Hawkins, Johnson, Leary,
  Maclaurin, Necula, Paszke, Vander{P}las, Wanderman-{M}ilne, and
  Zhang]{jax2018github}
James Bradbury, Roy Frostig, Peter Hawkins, Matthew~James Johnson, Chris Leary,
  Dougal Maclaurin, George Necula, Adam Paszke, Jake Vander{P}las, Skye
  Wanderman-{M}ilne, and Qiao Zhang.
\newblock {JAX}: composable transformations of {P}ython+{N}um{P}y programs,
  2018.

\bibitem[Bubeck and Sellke(2021)]{bubeck2021universal}
S{\'e}bastien Bubeck and Mark Sellke.
\newblock A universal law of robustness via isoperimetry.
\newblock \emph{Advances in Neural Information Processing Systems}, 34, 2021.

\bibitem[Catoni(2007)]{catoni}
Olivier Catoni.
\newblock \emph{{PAC-B}ayesian supervised classification: the thermodynamics of
  statistical Learning (ICML)}.
\newblock Institute of Mathematical Statistics, 2007.

\bibitem[Chaudhari et~al.(2017)Chaudhari, Choromanska, Soatto, LeCun, Baldassi,
  Borgs, Chayes, Sagun, and Zecchina]{chaudhari2016entropy}
Pratik Chaudhari, Anna Choromanska, Stefano Soatto, Yann LeCun, Carlo Baldassi,
  Christian Borgs, Jennifer Chayes, Levent Sagun, and Riccardo Zecchina.
\newblock Entropy-sgd: Biasing gradient descent into wide valleys.
\newblock \emph{ICLR}, 2017.

\bibitem[Chen et~al.(2019)Chen, Dobriban, and Lee]{chen2019invariance}
Shuxiao Chen, Edgar Dobriban, and Jane~H Lee.
\newblock Invariance reduces variance: Understanding data augmentation in deep
  learning and beyond.
\newblock \emph{arXiv preprint arXiv:1907.10905}, 2019.

\bibitem[Chung et~al.(2018)Chung, Nath, Joseph, and White]{chung2018two}
Wesley Chung, Somjit Nath, Ajin Joseph, and Martha White.
\newblock Two-timescale networks for nonlinear value function approximation.
\newblock In \emph{International Conference on Learning Representations
  (ICLR)}, 2018.

\bibitem[Cobbe et~al.(2019)Cobbe, Klimov, Hesse, Kim, and
  Schulman]{cobbe2019quantifying}
Karl Cobbe, Oleg Klimov, Chris Hesse, Taehoon Kim, and John Schulman.
\newblock Quantifying generalization in reinforcement learning.
\newblock In Kamalika Chaudhuri and Ruslan Salakhutdinov, editors,
  \emph{Proceedings of the 36th International Conference on Machine Learning
  (ICML)}, volume~97 of \emph{Proceedings of Machine Learning Research}, pages
  1282--1289. PMLR, 09--15 Jun 2019.

\bibitem[Cobbe et~al.(2020)Cobbe, Hesse, Hilton, and
  Schulman]{cobbe2020leveraging}
Karl Cobbe, Chris Hesse, Jacob Hilton, and John Schulman.
\newblock Leveraging procedural generation to benchmark reinforcement learning.
\newblock In \emph{Proceedings of the 37th International Conference on Machine
  Learning (ICML)}, pages 2048--2056. PMLR, 2020.

\bibitem[Cobbe et~al.(2021)Cobbe, Hilton, Klimov, and
  Schulman]{cobbe2021phasic}
Karl~W Cobbe, Jacob Hilton, Oleg Klimov, and John Schulman.
\newblock Phasic policy gradient.
\newblock In \emph{Proceedings of the 37th International Conference on Machine
  Learning (ICML)}, pages 2020--2027. PMLR, 2021.

\bibitem[Cohen and Welling(2016)]{Cohen:Welling:2016}
Taco~S. Cohen and Max Welling.
\newblock Group equivariant convolutional networks.
\newblock In Maria~Florina Balcan and Kilian~Q. Weinberger, editors,
  \emph{Proceedings of The 33rd International Conference on Machine Learning
  (ICML)}, volume~48 of \emph{Proceedings of Machine Learning Research}, pages
  2990--2999. PMLR, 2016.

\bibitem[Cohen et~al.(2019)Cohen, Geiger, and
  Weiler]{cohenetal2019generaltheory}
Taco~S Cohen, Mario Geiger, and Maurice Weiler.
\newblock A general theory of equivariant {CNN}s on homogeneous spaces.
\newblock In \emph{Advances in Neural Information Processing Systems 32}, 2019.

\bibitem[Cubuk et~al.(2019)Cubuk, Zoph, Man{\'e}, Vasudevan, and Le]{Cubuk2018}
Ekin~Dogus Cubuk, Barret Zoph, Dandelion Man{\'e}, Vijay Vasudevan, and Quoc~V.
  Le.
\newblock Autoaugment: Learning augmentation policies from data.
\newblock \emph{Computer Vision and Pattern Recognition (CVPR)}, 2019.

\bibitem[Czarnecki et~al.(2019)Czarnecki, Pascanu, Osindero, Jayakumar,
  Swirszcz, and Jaderberg]{czarnecki2019distilling}
Wojciech~M Czarnecki, Razvan Pascanu, Simon Osindero, Siddhant Jayakumar,
  Grzegorz Swirszcz, and Max Jaderberg.
\newblock Distilling policy distillation.
\newblock In \emph{The 22nd International Conference on Artificial Intelligence
  and Statistics}, pages 1331--1340. PMLR, 2019.

\bibitem[Dabney et~al.(2018)Dabney, Rowland, Bellemare, and
  Munos]{dabney2018distributional}
Will Dabney, Mark Rowland, Marc~G Bellemare, and R{\'e}mi Munos.
\newblock Distributional reinforcement learning with quantile regression.
\newblock In \emph{AAAI Conference on Artificial Intelligence}, 2018.

\bibitem[Dabney et~al.(2021)Dabney, Barreto, Rowland, Dadashi, Quan, Bellemare,
  and Silver]{dabney2020value}
Will Dabney, Andr{\'e} Barreto, Mark Rowland, Robert Dadashi, John Quan, Marc~G
  Bellemare, and David Silver.
\newblock The value-improvement path: Towards better representations for
  reinforcement learning.
\newblock In \emph{Proceedings of the AAAI Conference on Artificial
  Intelligence}, volume~35, 2021.

\bibitem[Damianou and Lawrence(2013)]{damianou13a}
Andreas Damianou and Neil Lawrence.
\newblock Deep gaussian processes.
\newblock In \emph{Proceedings of Machine Learning Research}, volume~31, pages
  207--215. PMLR, 2013.

\bibitem[Daneshmand et~al.(2021)Daneshmand, Joudaki, and
  Bach]{daneshmand2021batch}
Hadi Daneshmand, Amir Joudaki, and Francis Bach.
\newblock Batch normalization orthogonalizes representations in deep random
  networks.
\newblock In \emph{Advances in Neural Information Processing Systems}, 2021.

\bibitem[Dao et~al.(2019)Dao, Gu, Ratner, Smith, De~Sa, and
  R{\'e}]{kerneltheory}
Tri Dao, Albert Gu, Alexander~J Ratner, Virginia Smith, Christopher De~Sa, and
  Christopher R{\'e}.
\newblock A kernel theory of modern data augmentation.
\newblock \emph{Proceedings of the 36th International Conference on Machine
  Learning, PMLR 97}, 2019.

\bibitem[Daxberger et~al.(2021)Daxberger, Kristiadi, Immer, Eschenhagen, Bauer,
  and Hennig]{daxberger2021laplace}
Erik Daxberger, Agustinus Kristiadi, Alexander Immer, Runa Eschenhagen,
  Matthias Bauer, and Philipp Hennig.
\newblock Laplace redux-effortless bayesian deep learning.
\newblock \emph{Advances in Neural Information Processing Systems}, 34, 2021.

\bibitem[Dayan(1993)]{dayan1993improving}
Peter Dayan.
\newblock Improving generalization for temporal difference learning: The
  successor representation.
\newblock \emph{Neural Computation}, 5\penalty0 (4):\penalty0 613--624, 1993.

\bibitem[de~Haan et~al.(2019)de~Haan, Jayaraman, and Levine]{de2019causal}
Pim de~Haan, Dinesh Jayaraman, and Sergey Levine.
\newblock Causal confusion in imitation learning.
\newblock \emph{Advances in Neural Information Processing Systems}, 2019.

\bibitem[De~Palma et~al.(2019)De~Palma, Kiani, and Lloyd]{de2019random}
Giacomo De~Palma, Bobak Kiani, and Seth Lloyd.
\newblock Random deep neural networks are biased towards simple functions.
\newblock \emph{Advances in Neural Information Processing Systems}, 32, 2019.

\bibitem[Degrave et~al.(2022)Degrave, Felici, Buchli, Neunert, Tracey,
  Carpanese, Ewalds, Hafner, Abdolmaleki, de~Las~Casas,
  et~al.]{degrave2022magnetic}
Jonas Degrave, Federico Felici, Jonas Buchli, Michael Neunert, Brendan Tracey,
  Francesco Carpanese, Timo Ewalds, Roland Hafner, Abbas Abdolmaleki, Diego
  de~Las~Casas, et~al.
\newblock Magnetic control of tokamak plasmas through deep reinforcement
  learning.
\newblock \emph{Nature}, 602\penalty0 (7897):\penalty0 414--419, 2022.

\bibitem[Dinh et~al.(2017)Dinh, Pascanu, Bengio, and Bengio]{dinh2017sharp}
Laurent Dinh, Razvan Pascanu, Samy Bengio, and Yoshua Bengio.
\newblock Sharp minima can generalize for deep nets.
\newblock In \emph{Proceedings of the 34th International Conference on Machine
  Learning-Volume 70}, pages 1019--1028. JMLR. org, 2017.

\bibitem[Dohare et~al.(2021)Dohare, Mahmood, and Sutton]{dohare2021continual}
Shibhansh Dohare, A~Rupam Mahmood, and Richard~S Sutton.
\newblock Continual backprop: Stochastic gradient descent with persistent
  randomness.
\newblock \emph{arXiv preprint arXiv:2108.06325}, 2021.

\bibitem[Dong et~al.(2020)Dong, Luo, Yu, Finn, and Ma]{dong2020expressivity}
Kefan Dong, Yuping Luo, Tianhe Yu, Chelsea Finn, and Tengyu Ma.
\newblock On the expressivity of neural networks for deep reinforcement
  learning.
\newblock In \emph{Proceedings of the 37th International Conference on Machine
  Learning (ICML)}, pages 2627--2637. PMLR, 2020.

\bibitem[Dong and Yang(2020)]{dong2020nasbench201}
Xuanyi Dong and Yi~Yang.
\newblock Nas-bench-201: Extending the scope of reproducible neural
  architecture search.
\newblock In \emph{International Conference on Learning Representations
  (ICLR)}, 2020.
\newblock URL \url{https://openreview.net/forum?id=HJxyZkBKDr}.

\bibitem[Donini et~al.(2018)Donini, Oneto, Ben-David, Shawe-Taylor, and
  Pontil]{donini2018empirical}
Michele Donini, Luca Oneto, Shai Ben-David, John~S Shawe-Taylor, and
  Massimiliano Pontil.
\newblock Empirical risk minimization under fairness constraints.
\newblock \emph{Advances in Neural Information Processing Systems}, 31, 2018.

\bibitem[Du et~al.(2019)Du, Krishnamurthy, Jiang, Agarwal, Dud{\'{\i}}k, and
  Langford]{du2019pcid}
Simon~S. Du, Akshay Krishnamurthy, Nan Jiang, Alekh Agarwal, Miroslav
  Dud{\'{\i}}k, and John Langford.
\newblock Provably efficient {RL} with rich observations via latent state
  decoding.
\newblock \emph{International Conference on Machine Learning}, pages
  1665--1674, 2019.

\bibitem[Duchi et~al.(2011)Duchi, Hazan, and Singer]{duchi2011adaptive}
John Duchi, Elad Hazan, and Yoram Singer.
\newblock Adaptive subgradient methods for online learning and stochastic
  optimization.
\newblock \emph{Journal of Machine Learning Research}, 12\penalty0
  (61):\penalty0 2121--2159, 2011.

\bibitem[Dutordoir et~al.(2020)Dutordoir, van~der Wilk, Artemev, and
  Hensman]{dutordoir20a}
Vincent Dutordoir, Mark van~der Wilk, Artem Artemev, and James Hensman.
\newblock Bayesian image classification with deep convolutional gaussian
  processes.
\newblock In \emph{Proceedings of Machine Learning Research}, volume 108, pages
  1529--1539. PMLR, 2020.

\bibitem[Duvenaud et~al.(2016)Duvenaud, Maclaurin, and
  Adams]{duvenaud2016early}
David Duvenaud, Dougal Maclaurin, and Ryan Adams.
\newblock Early stopping as nonparametric variational inference.
\newblock In \emph{Artificial Intelligence and Statistics}, pages 1070--1077,
  2016.

\bibitem[Dziugaite and Roy(2017)]{dziugaite2017nonvacuous}
Gintare~Karolina Dziugaite and Daniel~M Roy.
\newblock Computing nonvacuous generalization bounds for deep (stochastic)
  neural networks with many more parameters than training data.
\newblock In \emph{UAI}, 2017.

\bibitem[Dziugaite and Roy(2018{\natexlab{a}})]{dziugaite2017entropy}
Gintare~Karolina Dziugaite and Daniel~M Roy.
\newblock Entropy-sgd optimizes the prior of a pac-bayes bound: Generalization
  properties of entropy-sgd and data-dependent priors.
\newblock \emph{Proceedings of the 35th International Conference on Machine
  Learning, PMLR 80}, pages :1377--1386, 2018{\natexlab{a}}.

\bibitem[Dziugaite and Roy(2018{\natexlab{b}})]{dziugaite2018dependent}
Gintare~Karolina Dziugaite and Daniel~M Roy.
\newblock Data-dependent {PAC-Bayes} priors via differential privacy.
\newblock In \emph{Advances in Neural Information Processing Systems 31}, pages
  8430--8441, 2018{\natexlab{b}}.

\bibitem[Dziugaite et~al.(2020)Dziugaite, Drouin, Neal, Rajkumar, Caballero,
  Wang, Mitliagkas, and Roy]{dziugaite2020search}
Gintare~Karolina Dziugaite, Alexandre Drouin, Brady Neal, Nitarshan Rajkumar,
  Ethan Caballero, Linbo Wang, Ioannis Mitliagkas, and Daniel~M Roy.
\newblock In search of robust measures of generalization.
\newblock \emph{Advances in Neural Information Processing Systems}, 33, 2020.

\bibitem[Eberhardt and Scheines(2007)]{eberhardt2007interventions}
Frederick Eberhardt and Richard Scheines.
\newblock Interventions and causal inference.
\newblock \emph{Philosophy of Science}, 74\penalty0 (5):\penalty0 981--995,
  2007.
\newblock \doi{10.1086/525638}.

\bibitem[Erhan et~al.(2010)Erhan, Courville, Bengio, and
  Vincent]{erhan2010does}
Dumitru Erhan, Aaron Courville, Yoshua Bengio, and Pascal Vincent.
\newblock Why does unsupervised pre-training help deep learning?
\newblock In \emph{Proceedings of the thirteenth international conference on
  artificial intelligence and statistics}, pages 201--208. JMLR Workshop and
  Conference Proceedings, 2010.

\bibitem[Farahmand(2011)]{farahmand2011regularization}
Amir-massoud Farahmand.
\newblock \emph{Regularization in Reinforcement Learning}.
\newblock PhD Thesis, University of Alberta, 2011.

\bibitem[Farebrother et~al.(2018)Farebrother, Machado, and
  Bowling]{farebrother2018generalization}
Jesse Farebrother, Marlos~C Machado, and Michael Bowling.
\newblock Generalization and regularization in dqn.
\newblock \emph{arXiv preprint arXiv:1810.00123}, 2018.

\bibitem[Fawzi et~al.(2016)Fawzi, Samulowitz, Turaga, and
  Frossard]{fawzi2016adaptive}
Alhussein Fawzi, Horst Samulowitz, Deepak Turaga, and Pascal Frossard.
\newblock Adaptive data augmentation for image classification.
\newblock In \emph{2016 IEEE International Conference on Image Processing
  (ICIP)}, pages 3688--3692. Ieee, 2016.

\bibitem[Fedus et~al.(2019{\natexlab{a}})Fedus, Gelada, Bengio, Bellemare, and
  Larochelle]{fedus2019hyperbolic}
William Fedus, Carles Gelada, Yoshua Bengio, Marc~G Bellemare, and Hugo
  Larochelle.
\newblock Hyperbolic discounting and learning over multiple horizons.
\newblock In \emph{Reinforcement Learning and Decision Making (RLDM)},
  2019{\natexlab{a}}.

\bibitem[Fedus et~al.(2019{\natexlab{b}})Fedus, Ghosh, Martin, Bellemare,
  Bengio, and Larochelle]{fedus2020catastrophic}
William Fedus, Dibya Ghosh, John~D Martin, Marc~G Bellemare, Yoshua Bengio, and
  Hugo Larochelle.
\newblock On catastrophic interference in {A}tari 2600 games.
\newblock \emph{Reinforcement Learning and Decision Making},
  2019{\natexlab{b}}.

\bibitem[Filos et~al.(2021)Filos, Lyle, Gal, Levine, Jaques, and
  Farquhar]{filos2021psiphi}
Angelos Filos, Clare Lyle, Yarin Gal, Sergey Levine, Natasha Jaques, and Greg
  Farquhar.
\newblock Psiphi-learning: Reinforcement learning with demonstrations using
  successor features and inverse temporal difference learning.
\newblock \emph{Proceedings of the 38th International Conference on Machine
  Learning (ICML)}, 2021.

\bibitem[Finn et~al.(2017)Finn, Abbeel, and Levine]{finn2017model}
Chelsea Finn, Pieter Abbeel, and Sergey Levine.
\newblock Model-agnostic meta-learning for fast adaptation of deep networks.
\newblock In \emph{Proceedings of the 34th International Conference on Machine
  Learning (ICML)}, pages 1126--1135. PMLR, 2017.

\bibitem[Fort and Scherlis(2019)]{fort2019goldilocks}
Stanislav Fort and Adam Scherlis.
\newblock The goldilocks zone: Towards better understanding of neural network
  loss landscapes.
\newblock In \emph{Proceedings of the AAAI Conference on Artificial
  Intelligence}, volume~33, pages 3574--3581, 2019.

\bibitem[Fort et~al.(2019)Fort, Nowak, Jastrzebski, and
  Narayanan]{fort2019stiffness}
Stanislav Fort, Pawe{\l}~Krzysztof Nowak, Stanislaw Jastrzebski, and Srini
  Narayanan.
\newblock Stiffness: A new perspective on generalization in neural networks.
\newblock \emph{arXiv preprint arXiv:1901.09491}, 2019.

\bibitem[Fort et~al.(2020)Fort, Dziugaite, Paul, Kharaghani, Roy, and
  Ganguli]{fort2020deep}
Stanislav Fort, Gintare~Karolina Dziugaite, Mansheej Paul, Sepideh Kharaghani,
  Daniel~M Roy, and Surya Ganguli.
\newblock Deep learning versus kernel learning: an empirical study of loss
  landscape geometry and the time evolution of the neural tangent kernel.
\newblock \emph{Advances in Neural Information Processing Systems}, 33, 2020.

\bibitem[Fortunato et~al.(2018)Fortunato, Azar, Piot, Menick, Hessel, Osband,
  Graves, Mnih, Munos, Hassabis, et~al.]{fortunato2018noisy}
Meire Fortunato, Mohammad~Gheshlaghi Azar, Bilal Piot, Jacob Menick, Matteo
  Hessel, Ian Osband, Alex Graves, Volodymyr Mnih, Remi Munos, Demis Hassabis,
  et~al.
\newblock Noisy networks for exploration.
\newblock In \emph{International Conference on Learning Representations
  (ICLR)}, 2018.

\bibitem[Frankle and Carbin(2019)]{frankle2018the}
Jonathan Frankle and Michael Carbin.
\newblock The lottery ticket hypothesis: Finding sparse, trainable neural
  networks.
\newblock In \emph{International Conference on Learning Representations
  (ICLR)}, 2019.

\bibitem[Frankle et~al.(2020{\natexlab{a}})Frankle, Dziugaite, Roy, and
  Carbin]{frankle2020linear}
Jonathan Frankle, Gintare~Karolina Dziugaite, Daniel Roy, and Michael Carbin.
\newblock Linear mode connectivity and the lottery ticket hypothesis.
\newblock In Hal~Daumé III and Aarti Singh, editors, \emph{Proceedings of the
  37th International Conference on Machine Learning (ICML)}, volume 119 of
  \emph{Proceedings of Machine Learning Research}, pages 3259--3269. PMLR,
  13--18 Jul 2020{\natexlab{a}}.

\bibitem[Frankle et~al.(2020{\natexlab{b}})Frankle, Schwab, and
  Morcos]{frankle2020the}
Jonathan Frankle, David~J. Schwab, and Ari~S. Morcos.
\newblock The early phase of neural network training.
\newblock In \emph{International Conference on Learning Representations
  (ICLR)}, 2020{\natexlab{b}}.

\bibitem[Gal and Ghahramani(2016)]{gal2016dropout}
Yarin Gal and Zoubin Ghahramani.
\newblock Dropout as a bayesian approximation: Representing model uncertainty
  in deep learning.
\newblock \emph{Proceedings of the 33rd International Conference on Machine
  Learning (ICML)}, pages 1050--1059, 2016.

\bibitem[Gelada et~al.(2019)Gelada, Kumar, Buckman, Nachum, and
  Bellemare]{gelada2019deepmdp}
Carles Gelada, Saurabh Kumar, Jacob Buckman, Ofir Nachum, and Marc~G.
  Bellemare.
\newblock {D}eep{MDP}: Learning continuous latent space models for
  representation learning.
\newblock In Kamalika Chaudhuri and Ruslan Salakhutdinov, editors,
  \emph{Proceedings of the 36th International Conference on Machine Learning
  (ICML)}, volume~97 of \emph{Proceedings of Machine Learning Research}, pages
  2170--2179, Long Beach, California, USA, 09--15 Jun 2019. PMLR.

\bibitem[Germain et~al.(2016{\natexlab{a}})Germain, Bach, Lacoste, and
  Lacoste-Julien]{germain2016pac}
Pascal Germain, Francis Bach, Alexandre Lacoste, and Simon Lacoste-Julien.
\newblock Pac-bayesian theory meets bayesian inference.
\newblock In \emph{Advances in Neural Information Processing Systems}, pages
  1884--1892, 2016{\natexlab{a}}.

\bibitem[Germain et~al.(2016{\natexlab{b}})Germain, Bach, Lacoste, and
  Lacoste-Julien]{germain_pac-bayesian_2016}
Pascal Germain, Francis Bach, Alexandre Lacoste, and Simon Lacoste-Julien.
\newblock {PAC}-{Bayesian} theory meets {Bayesian} inference.
\newblock In \emph{Advances in {Neural} {Information} {Processing} {Systems}},
  pages 1884--1892, 2016{\natexlab{b}}.

\bibitem[Ghiassian et~al.(2020)Ghiassian, Rafiee, Lo, and
  White]{ghassian2020improving}
Sina Ghiassian, Banafsheh Rafiee, Yat~Long Lo, and Adam White.
\newblock Improving performance in reinforcement learning by breaking
  generalization in neural networks.
\newblock \emph{arXiv preprint arXiv:2003.07417}, 2020.

\bibitem[Ghosh and Bellemare(2020)]{ghosh2020representations}
Dibya Ghosh and Marc~G Bellemare.
\newblock Representations for stable off-policy reinforcement learning.
\newblock In \emph{Proceedings of the 37th International Conference on Machine
  Learning (ICML)}, 2020.

\bibitem[Glorot and Bengio(2010)]{glorot2010understanding}
Xavier Glorot and Yoshua Bengio.
\newblock Understanding the difficulty of training deep feedforward neural
  networks.
\newblock In \emph{Proceedings of the thirteenth international conference on
  artificial intelligence and statistics}, pages 249--256. JMLR Workshop and
  Conference Proceedings, 2010.

\bibitem[Gogianu et~al.(2021)Gogianu, Berariu, Rosca, Clopath, Busoniu, and
  Pascanu]{gogianu2021spectral}
Florin Gogianu, Tudor Berariu, Mihaela Rosca, Claudia Clopath, Lucian Busoniu,
  and Razvan Pascanu.
\newblock Spectral normalisation for deep reinforcement learning: an
  optimisation perspective.
\newblock \emph{Proceedings of the 38th International Conference on Machine
  Learning (ICML)}, pages 3734--3744, 2021.

\bibitem[Golatkar et~al.(2019)Golatkar, Achille, and Soatto]{golatkar2019time}
Aditya Golatkar, Alessandro Achille, and Stefano Soatto.
\newblock Time matters in regularizing deep networks: Weight decay and data
  augmentation affect early learning dynamics, matter little near convergence.
\newblock In \emph{Proceedings of the 33rd International Conference on Neural
  Information Processing Systems}, 2019.

\bibitem[Golub et~al.(1976)Golub, Klema, and Stewart]{golub1976rank}
Gene~H Golub, Virginia~C Klema, and Gilbert~W Stewart.
\newblock Rank degeneracy and least squares problems, 1976.

\bibitem[Goyal et~al.(2019)Goyal, Islam, Strouse, Ahmed, Larochelle, Botvinick,
  Levine, and Bengio]{goyal2018transfer}
Anirudh Goyal, Riashat Islam, DJ~Strouse, Zafarali Ahmed, Hugo Larochelle,
  Matthew Botvinick, Sergey Levine, and Yoshua Bengio.
\newblock Transfer and exploration via the information bottleneck.
\newblock In \emph{International Conference on Learning Representations
  (ICLR)}, 2019.

\bibitem[Graves(2011)]{graves2011practical}
Alex Graves.
\newblock Practical variational inference for neural networks.
\newblock In J.~Shawe-Taylor, R.~Zemel, P.~Bartlett, F.~Pereira, and K.Q.
  Weinberger, editors, \emph{Advances in Neural Information Processing
  Systems}, volume~24. Curran Associates, Inc., 2011.

\bibitem[Gulcehre et~al.(2022)Gulcehre, Srinivasan, Sygnowski, Ostrovski,
  Farajtabar, Hoffman, Pascanu, and Doucet]{gulcehre2022empirical}
Caglar Gulcehre, Srivatsan Srinivasan, Jakub Sygnowski, Georg Ostrovski,
  Mehrdad Farajtabar, Matt Hoffman, Razvan Pascanu, and Arnaud Doucet.
\newblock An empirical study of implicit regularization in deep offline rl.
\newblock \emph{arXiv preprint arXiv:2207.02099}, 2022.

\bibitem[Haar(1933)]{haar1933massbegriff}
Alfred Haar.
\newblock Der massbegriff in der theorie der kontinuierlichen gruppen.
\newblock \emph{Annals of mathematics}, pages 147--169, 1933.

\bibitem[Haarnoja et~al.(2018)Haarnoja, Zhou, Abbeel, and
  Levine]{haarnoja2018sac}
Tuomas Haarnoja, Aurick Zhou, Pieter Abbeel, and Sergey Levine.
\newblock Soft actor-critic: Off-policy maximum entropy deep reinforcement
  learning with a stochastic actor.
\newblock In Jennifer Dy and Andreas Krause, editors, \emph{Proceedings of the
  35th International Conference on Machine Learning (ICML)}, volume~80 of
  \emph{Proceedings of Machine Learning Research}, pages 1861--1870,
  Stockholmsmässan, Stockholm Sweden, 10--15 Jul 2018. PMLR.

\bibitem[Hanin and Rolnick(2019)]{hanin2019deep}
Boris Hanin and David Rolnick.
\newblock Deep relu networks have surprisingly few activation patterns.
\newblock In H.~Wallach, H.~Larochelle, A.~Beygelzimer, F.~d\textquotesingle
  Alch\'{e}-Buc, E.~Fox, and R.~Garnett, editors, \emph{Advances in Neural
  Information Processing Systems}, volume~32. Curran Associates, Inc., 2019.

\bibitem[Hansen and Wang(2021)]{hansen2021generalization}
Nicklas Hansen and Xiaolong Wang.
\newblock Generalization in reinforcement learning by soft data augmentation.
\newblock In \emph{2021 IEEE International Conference on Robotics and
  Automation (ICRA)}, pages 13611--13617. IEEE, 2021.

\bibitem[Hardt et~al.(2016)Hardt, Recht, and Singer]{hardt2015train}
Moritz Hardt, Ben Recht, and Yoram Singer.
\newblock Train faster, generalize better: Stability of stochastic gradient
  descent.
\newblock In \emph{Proceedings of the 33rd International Conference on Machine
  Learning (ICML)}. PMLR, 2016.

\bibitem[Harvey et~al.(2017)Harvey, Liaw, and Mehrabian]{harvey2019nearly}
Nick Harvey, Christopher Liaw, and Abbas Mehrabian.
\newblock Nearly-tight {VC}-dimension bounds for piecewise linear neural
  networks.
\newblock In Satyen Kale and Ohad Shamir, editors, \emph{Proceedings of the
  2017 Conference on Learning Theory}, volume~65 of \emph{Proceedings of
  Machine Learning Research}, pages 1064--1068. PMLR, 07--10 Jul 2017.

\bibitem[He et~al.(2020)He, Lakshminarayanan, and Teh]{he2020bayesian}
Bobby He, Balaji Lakshminarayanan, and Yee~Whye Teh.
\newblock Bayesian deep ensembles via the neural tangent kernel.
\newblock \emph{Advances in Neural Information Processing Systems}, 2020.

\bibitem[He and Su(2020)]{he2019local}
Hangfeng He and Weijie~J Su.
\newblock The local elasticity of neural networks.
\newblock \emph{International Conference on Learning Representations (ICLR)},
  2020.

\bibitem[He et~al.(2015)He, Zhang, Ren, and Sun]{he2015delving}
Kaiming He, Xiangyu Zhang, Shaoqing Ren, and Jian Sun.
\newblock Delving deep into rectifiers: Surpassing human-level performance on
  imagenet classification.
\newblock In \emph{Proceedings of the IEEE international conference on computer
  vision}, pages 1026--1034, 2015.

\bibitem[He et~al.(2016)He, Zhang, Ren, and Sun]{he2016deep}
Kaiming He, Xiangyu Zhang, Shaoqing Ren, and Jian Sun.
\newblock Deep residual learning for image recognition.
\newblock In \emph{Proceedings of the IEEE conference on computer vision and
  pattern recognition}, pages 770--778, 2016.

\bibitem[Hessel et~al.(2018)Hessel, Modayil, Van~Hasselt, Schaul, Ostrovski,
  Dabney, Horgan, Piot, Azar, and Silver]{hessel2018rainbow}
Matteo Hessel, Joseph Modayil, Hado Van~Hasselt, Tom Schaul, Georg Ostrovski,
  Will Dabney, Dan Horgan, Bilal Piot, Mohammad Azar, and David Silver.
\newblock Rainbow: Combining improvements in deep reinforcement learning.
\newblock In \emph{AAAI conference on artificial intelligence}, 2018.

\bibitem[Hinton and Van~Camp(1993)]{hinton1993keeping}
Geoffrey~E Hinton and Drew Van~Camp.
\newblock Keeping the neural networks simple by minimizing the description
  length of the weights.
\newblock In \emph{Proceedings of the sixth annual conference on Computational
  learning theory}, pages 5--13, 1993.

\bibitem[Hochreiter and Schmidhuber(1997{\natexlab{a}})]{hochreiter1997flat}
Sepp Hochreiter and J{\"u}rgen Schmidhuber.
\newblock Flat minima.
\newblock \emph{Neural Computation}, 9\penalty0 (1):\penalty0 1--42,
  1997{\natexlab{a}}.

\bibitem[Hochreiter and Schmidhuber(1997{\natexlab{b}})]{hochreiter1997long}
Sepp Hochreiter and J{\"u}rgen Schmidhuber.
\newblock Long short-term memory.
\newblock \emph{Neural computation}, 9\penalty0 (8):\penalty0 1735--1780,
  1997{\natexlab{b}}.

\bibitem[Hoffer et~al.(2017)Hoffer, Hubara, and Soudry]{hoffer2017train}
Elad Hoffer, Itay Hubara, and Daniel Soudry.
\newblock Train longer, generalize better: closing the generalization gap in
  large batch training of neural networks.
\newblock In \emph{Advances in Neural Information Processing Systems}, pages
  1731--1741, 2017.

\bibitem[Hornik et~al.(1989)Hornik, Stinchcombe, and
  White]{hornik1989multilayer}
Kurt Hornik, Maxwell Stinchcombe, and Halbert White.
\newblock Multilayer feedforward networks are universal approximators.
\newblock \emph{Neural networks}, 2\penalty0 (5):\penalty0 359--366, 1989.

\bibitem[Huber(1964)]{huber64robust}
Peter~J. Huber.
\newblock Robust estimation of a location parameter.
\newblock \emph{The Annals of Mathematical Statistics}, 35\penalty0
  (1):\penalty0 73--101, 1964.
\newblock ISSN 00034851.

\bibitem[Igl et~al.(2019)Igl, Ciosek, Li, Tschiatschek, Zhang, Devlin, and
  Hofmann]{igl2019generalization}
Maximilian Igl, Kamil Ciosek, Yingzhen Li, Sebastian Tschiatschek, Cheng Zhang,
  Sam Devlin, and Katja Hofmann.
\newblock Generalization in reinforcement learning with selective noise
  injection and information bottleneck.
\newblock In \emph{Advances in Neural Information Processing Systems}, 2019.

\bibitem[Igl et~al.(2021)Igl, Farquhar, Luketina, Boehmer, and
  Whiteson]{igl2021transient}
Maximilian Igl, Gregory Farquhar, Jelena Luketina, Wendelin Boehmer, and Shimon
  Whiteson.
\newblock Transient non-stationarity and generalisation in deep reinforcement
  learning.
\newblock In \emph{International Conference on Learning Representations
  (ICLR)}, 2021.

\bibitem[Ioffe and Szegedy(2015)]{ioffe2015batch}
Sergey Ioffe and Christian Szegedy.
\newblock Batch normalization: Accelerating deep network training by reducing
  internal covariate shift.
\newblock In \emph{Proceedings of the 32nd International Conference on Machine
  Learning (ICML)}, pages 448--456. PMLR, 2015.

\bibitem[Iyyer et~al.(2014)Iyyer, Boyd-Graber, Claudino, Socher, and
  Daum{\'e}~III]{iyyer-etal-2014-neural}
Mohit Iyyer, Jordan Boyd-Graber, Leonardo Claudino, Richard Socher, and Hal
  Daum{\'e}~III.
\newblock A neural network for factoid question answering over paragraphs.
\newblock In \emph{Proceedings of the 2014 Conference on Empirical Methods in
  Natural Language Processing ({EMNLP})}, pages 633--644, October 2014.

\bibitem[Izmailov et~al.(2018)Izmailov, Wilson, Podoprikhin, Vetrov, and
  Garipov]{izmailov2018averaging}
P~Izmailov, AG~Wilson, D~Podoprikhin, D~Vetrov, and T~Garipov.
\newblock Averaging weights leads to wider optima and better generalization.
\newblock In \emph{34th Conference on Uncertainty in Artificial Intelligence
  2018, UAI 2018}, pages 876--885, 2018.

\bibitem[Jaakkola et~al.(1994)Jaakkola, Jordan, and
  Singh]{jaakola1994convergence}
Tommi Jaakkola, Michael~I. Jordan, and Satinder~P Singh.
\newblock On the convergence of stochastic iterative dynamic programming
  algorithms.
\newblock \emph{Neural Computation}, 6\penalty0 (6), 1994.

\bibitem[Jacot et~al.(2018)Jacot, Gabriel, and Hongler]{jacot2018neural}
Arthur Jacot, Franck Gabriel, and Cl{\'e}ment Hongler.
\newblock Neural tangent kernel: Convergence and generalization in neural
  networks.
\newblock In \emph{Advances in neural information processing systems}, pages
  8571--8580, 2018.

\bibitem[Jaderberg et~al.(2015)Jaderberg, Simonyan, Zisserman,
  et~al.]{jaderberg2015spatial}
Max Jaderberg, Karen Simonyan, Andrew Zisserman, et~al.
\newblock Spatial transformer networks.
\newblock \emph{Advances in neural information processing systems}, 28, 2015.

\bibitem[Jaderberg et~al.(2017)Jaderberg, Mnih, Czarnecki, Schaul, Leibo,
  Silver, and Kavukcuoglu]{jaderberg2016reinforcement}
Max Jaderberg, Volodymyr Mnih, Wojciech~Marian Czarnecki, Tom Schaul, Joel~Z
  Leibo, David Silver, and Koray Kavukcuoglu.
\newblock Reinforcement learning with unsupervised auxiliary tasks.
\newblock In \emph{International Conference on Learning Representations
  (ICLR)}, 2017.

\bibitem[Jiang et~al.(2021{\natexlab{a}})Jiang, Grefenstette, and
  Rockt{\"a}schel]{jiang2021prioritized}
Minqi Jiang, Edward Grefenstette, and Tim Rockt{\"a}schel.
\newblock Prioritized level replay.
\newblock In \emph{International Conference on Machine Learning}, pages
  4940--4950. PMLR, 2021{\natexlab{a}}.

\bibitem[Jiang et~al.(2015)Jiang, Kulesza, and Singh]{jiang2015abstraction}
Nan Jiang, Alex Kulesza, and Satinder Singh.
\newblock Abstraction selection in model-based reinforcement learning.
\newblock In \emph{Proceedings of the 32nd International Conference on Machine
  Learning (ICML)}, 2015.

\bibitem[Jiang et~al.(2021{\natexlab{b}})Jiang, Zahavy, Xu, White, Hessel,
  Blundell, and Van~Hasselt]{jiang2021emphatic}
Ray Jiang, Tom Zahavy, Zhongwen Xu, Adam White, Matteo Hessel, Charles
  Blundell, and Hado Van~Hasselt.
\newblock Emphatic algorithms for deep reinforcement learning.
\newblock In Marina Meila and Tong Zhang, editors, \emph{Proceedings of the
  38th International Conference on Machine Learning (ICML)}, volume 139 of
  \emph{Proceedings of Machine Learning Research}, pages 5023--5033. PMLR,
  18--24 Jul 2021{\natexlab{b}}.

\bibitem[Jiang et~al.(2020)Jiang, Neyshabur, Krishnan, Mobahi, and
  Bengio]{jiang2020fantastic}
Yiding Jiang, Behnam Neyshabur, Dilip Krishnan, Hossein Mobahi, and Samy
  Bengio.
\newblock Fantastic generalization measures and where to find them.
\newblock In \emph{International Conference on Learning Representations
  (ICLR)}, 2020.

\bibitem[Jin et~al.(2020)Jin, Yang, Wang, and Jordan]{jin2020provably}
Chi Jin, Zhuoran Yang, Zhaoran Wang, and Michael~I Jordan.
\newblock Provably efficient reinforcement learning with linear function
  approximation.
\newblock In \emph{Conference on Learning Theory}, pages 2137--2143. PMLR,
  2020.

\bibitem[Johansson et~al.(2016)Johansson, Shalit, and
  Sontag]{johansson2016learning}
Fredrik Johansson, Uri Shalit, and David Sontag.
\newblock Learning representations for counterfactual inference.
\newblock In \emph{Proceedings of the 33rd International Conference on Machine
  Learning (ICML)}, pages 3020--3029, 2016.

\bibitem[Jong and Stone(2005)]{jong2005state}
Nicholas~K Jong and Peter Stone.
\newblock State abstraction discovery from irrelevant state variables.
\newblock In \emph{IJCAI}, volume~8, pages 752--757. Citeseer, 2005.

\bibitem[Kalimeris et~al.(2019)Kalimeris, Kaplun, Nakkiran, Edelman, Yang,
  Barak, and Zhang]{kalimeris2019sgd}
Dimitris Kalimeris, Gal Kaplun, Preetum Nakkiran, Benjamin Edelman, Tristan
  Yang, Boaz Barak, and Haofeng Zhang.
\newblock Sgd on neural networks learns functions of increasing complexity.
\newblock \emph{Advances in Neural Information Processing Systems},
  32:\penalty0 3496--3506, 2019.

\bibitem[Kamath et~al.(2021)Kamath, Tangella, Sutherland, and
  Srebro]{-kamath2021does}
Pritish Kamath, Akilesh Tangella, Danica Sutherland, and Nathan Srebro.
\newblock Does invariant risk minimization capture invariance?
\newblock In Arindam Banerjee and Kenji Fukumizu, editors, \emph{Proceedings of
  The 24th International Conference on Artificial Intelligence and Statistics},
  volume 130 of \emph{Proceedings of Machine Learning Research}, pages
  4069--4077. PMLR, 13--15 Apr 2021.

\bibitem[Keskar et~al.(2017)Keskar, Mudigere, Nocedal, Smelyanskiy, and
  Tang]{keskar2016large}
Nitish~Shirish Keskar, Dheevatsa Mudigere, Jorge Nocedal, Mikhail Smelyanskiy,
  and Ping Tak~Peter Tang.
\newblock On large-batch training for deep learning: Generalization gap and
  sharp minima.
\newblock \emph{ICLR}, 2017.

\bibitem[Khan et~al.(2019)Khan, Immer, Abedi, and Korzepa]{khan2019approximate}
Mohammad Emtiyaz~E Khan, Alexander Immer, Ehsan Abedi, and Maciej Korzepa.
\newblock Approximate inference turns deep networks into gaussian processes.
\newblock In H.~Wallach, H.~Larochelle, A.~Beygelzimer, F.~d'~Alch\'{e}-Buc,
  E.~Fox, and R.~Garnett, editors, \emph{Advances in Neural Information
  Processing Systems 32}, pages 3094--3104. Curran Associates, Inc., 2019.

\bibitem[Kingma and Ba(2015)]{kingma2014adam}
Diederik~P. Kingma and Jimmy Ba.
\newblock Adam: A method for stochastic optimization.
\newblock In \emph{International Conference on Learning Representations
  (ICLR)}, 2015.

\bibitem[Kirk et~al.(2021)Kirk, Zhang, Grefenstette, and
  Rockt{\"a}schel]{kirk2021survey}
Robert Kirk, Amy Zhang, Edward Grefenstette, and Tim Rockt{\"a}schel.
\newblock A survey of generalisation in deep reinforcement learning.
\newblock \emph{arXiv preprint arXiv:2111.09794}, 2021.

\bibitem[Kirkpatrick et~al.(2017)Kirkpatrick, Pascanu, Rabinowitz, Veness,
  Desjardins, Rusu, Milan, Quan, Ramalho, Grabska-Barwinska, Hassabis, Clopath,
  Kumaran, and Hadsell]{kirkpatrick2017overcoming}
James Kirkpatrick, Razvan Pascanu, Neil Rabinowitz, Joel Veness, Guillaume
  Desjardins, Andrei~A Rusu, Kieran Milan, John Quan, Tiago Ramalho, Agnieszka
  Grabska-Barwinska, Demis Hassabis, Claudia Clopath, Dharshan Kumaran, and
  Raia Hadsell.
\newblock Overcoming catastrophic forgetting in neural networks.
\newblock \emph{Proceedings of the National Academy of Sciences}, 114\penalty0
  (13):\penalty0 3521--3526, 2017.

\bibitem[Koehler et~al.(2021)Koehler, Zhou, Sutherland, and
  Srebro]{koehler2021uniform}
Frederic Koehler, Lijia Zhou, Danica Sutherland, and Nathan Srebro.
\newblock Uniform convergence of interpolators: Gaussian width, norm bounds and
  benign overfitting.
\newblock In M.~Ranzato, A.~Beygelzimer, Y.~Dauphin, P.S. Liang, and J.~Wortman
  Vaughan, editors, \emph{Advances in Neural Information Processing Systems},
  volume~34, pages 20657--20668. Curran Associates, Inc., 2021.

\bibitem[Konda and Tsitsiklis(2000)]{konda2000actor}
Vijay~R Konda and John~N Tsitsiklis.
\newblock Actor-critic algorithms.
\newblock In \emph{Advances in neural information processing systems}, pages
  1008--1014, 2000.

\bibitem[Kondor and Trivedi(2018)]{Kondor:Trivedi:2018}
Risi Kondor and Shubhendu Trivedi.
\newblock On the generalization of equivariance and convolution in neural
  networks to the action of compact groups.
\newblock In \emph{Proc.\ ICML 35}, volume~80 of \emph{PMLR}, pages 2747--2755,
  2018.

\bibitem[Kossen et~al.(2021)Kossen, Band, Lyle, Gomez, Rainforth, and
  Gal]{kossen2021self}
Jannik Kossen, Neil Band, Clare Lyle, Aidan~N Gomez, Tom Rainforth, and Yarin
  Gal.
\newblock Self-attention between datapoints: Going beyond individual
  input-output pairs in deep learning.
\newblock \emph{Advances in Neural Information Processing Systems}, 2021.

\bibitem[Krogh and Hertz(1991)]{krogh1991simple}
Anders Krogh and John Hertz.
\newblock A simple weight decay can improve generalization.
\newblock \emph{Advances in neural information processing systems}, 4, 1991.

\bibitem[Kullback and Leibler(1951)]{kullback1951}
S.~Kullback and R.~A. Leibler.
\newblock On information and sufficiency.
\newblock \emph{Ann. Math. Statist.}, 22\penalty0 (1):\penalty0 79--86, 03
  1951.
\newblock \doi{10.1214/aoms/1177729694}.

\bibitem[Kumar et~al.(2021)Kumar, Agarwal, Ghosh, and
  Levine]{kumar2021implicit}
Aviral Kumar, Rishabh Agarwal, Dibya Ghosh, and Sergey Levine.
\newblock Implicit under-parameterization inhibits data-efficient deep
  reinforcement learning.
\newblock In \emph{International Conference on Learning Representations
  (ICLR)}, 2021.

\bibitem[Kumar et~al.(2022)Kumar, Agarwal, Ma, Courville, Tucker, and
  Levine]{kumar2021dr3}
Aviral Kumar, Rishabh Agarwal, Tengyu Ma, Aaron Courville, George Tucker, and
  Sergey Levine.
\newblock Dr3: Value-based deep reinforcement learning requires explicit
  regularization.
\newblock \emph{Proceedings of the 39th International Conference on Machine
  Learning (ICML)}, 2022.

\bibitem[K{\"u}ttler et~al.(2020)K{\"u}ttler, Nardelli, Miller, Raileanu,
  Selvatici, Grefenstette, and Rockt{\"a}schel]{kuttler2020nethack}
Heinrich K{\"u}ttler, Nantas Nardelli, Alexander Miller, Roberta Raileanu,
  Marco Selvatici, Edward Grefenstette, and Tim Rockt{\"a}schel.
\newblock The nethack learning environment.
\newblock \emph{Advances in Neural Information Processing Systems},
  33:\penalty0 7671--7684, 2020.

\bibitem[Lakshminarayanan et~al.(2017)Lakshminarayanan, Pritzel, and
  Blundell]{lakshminarayanan2017simple}
Balaji Lakshminarayanan, Alexander Pritzel, and Charles Blundell.
\newblock Simple and scalable predictive uncertainty estimation using deep
  ensembles.
\newblock In \emph{Advances in neural information processing systems}, pages
  6402--6413, 2017.

\bibitem[Lan et~al.(2022)Lan, Tu, Oberman, Agarwal, and
  Bellemare]{lan2022generalization}
Charline~Le Lan, Stephen Tu, Adam Oberman, Rishabh Agarwal, and Marc~G
  Bellemare.
\newblock On the generalization of representations in reinforcement learning.
\newblock \emph{Artificial Intelligence and Statistics (AISTATS)}, 2022.

\bibitem[Langford and Shawe-Taylor(2003)]{langford2003pac}
John Langford and John Shawe-Taylor.
\newblock Pac-bayes \& margins.
\newblock \emph{Advances in neural information processing systems}, pages
  439--446, 2003.

\bibitem[Laskin et~al.(2020{\natexlab{a}})Laskin, Srinivas, and
  Abbeel]{laskin2020curl}
Michael Laskin, Aravind Srinivas, and Pieter Abbeel.
\newblock {CURL}: Contrastive unsupervised representations for reinforcement
  learning.
\newblock In \emph{Proceedings of the 37th International Conference on Machine
  Learning (ICML)}, 2020{\natexlab{a}}.

\bibitem[Laskin et~al.(2020{\natexlab{b}})Laskin, Lee, Stooke, Pinto, Abbeel,
  and Srinivas]{laskin2020reinforcement}
Misha Laskin, Kimin Lee, Adam Stooke, Lerrel Pinto, Pieter Abbeel, and Aravind
  Srinivas.
\newblock Reinforcement learning with augmented data.
\newblock \emph{Advances in Neural Information Processing Systems},
  33:\penalty0 19884--19895, 2020{\natexlab{b}}.

\bibitem[Lee et~al.(2018)Lee, Sohl-dickstein, Pennington, Novak, Schoenholz,
  and Bahri]{lee2018deep}
Jaehoon Lee, Jascha Sohl-dickstein, Jeffrey Pennington, Roman Novak, Sam
  Schoenholz, and Yasaman Bahri.
\newblock Deep neural networks as gaussian processes.
\newblock In \emph{International Conference on Learning Representations
  (ICLR)}, 2018.

\bibitem[Lee et~al.(2019)Lee, Xiao, Schoenholz, Bahri, Novak, Sohl-Dickstein,
  and Pennington]{lee2019wide}
Jaehoon Lee, Lechao Xiao, Samuel Schoenholz, Yasaman Bahri, Roman Novak, Jascha
  Sohl-Dickstein, and Jeffrey Pennington.
\newblock Wide neural networks of any depth evolve as linear models under
  gradient descent.
\newblock \emph{Advances in neural information processing systems}, 32, 2019.

\bibitem[Lee et~al.(2020)Lee, Schoenholz, Pennington, Adlam, Xiao, Novak, and
  Sohl-Dickstein]{lee2020finite}
Jaehoon Lee, Samuel Schoenholz, Jeffrey Pennington, Ben Adlam, Lechao Xiao,
  Roman Novak, and Jascha Sohl-Dickstein.
\newblock Finite versus infinite neural networks: an empirical study.
\newblock \emph{Advances in Neural Information Processing Systems}, 33, 2020.

\bibitem[Lee et~al.(2021)Lee, Smith, and Abbeel]{lee2021pebble}
Kimin Lee, Laura~M Smith, and Pieter Abbeel.
\newblock Pebble: Feedback-efficient interactive reinforcement learning via
  relabeling experience and unsupervised pre-training.
\newblock In \emph{Proceedings of the 38th International Conference on Machine
  Learning (ICML)}, pages 6152--6163. PMLR, 2021.

\bibitem[Lever et~al.(2013)Lever, Laviolette, and
  Shawe-Taylor]{leveretal2013tighterPACbayes}
Guy Lever, Fran{\c c}ois Laviolette, and John Shawe-Taylor.
\newblock Tighter {PAC-B}ayes bounds through distribution-dependent priors.
\newblock \emph{Theoretical Computer Science}, 473:\penalty0 4--28, 2013.

\bibitem[Levine et~al.(2017)Levine, Zahavy, Mankowitz, Tamar, and
  Mannor]{levine2017shallow}
Nir Levine, Tom Zahavy, Daniel~J Mankowitz, Aviv Tamar, and Shie Mannor.
\newblock Shallow updates for deep reinforcement learning.
\newblock In \emph{Advances in Neural Information Processing Systems}, 2017.

\bibitem[Lewandowski(2020)]{lewandowskigeneralization}
Alex Lewandowski.
\newblock Generalization across space and time in reinforcement learning.
\newblock \emph{NeurIPS Pre-Registration in Machine Learning Workshop}, 2020.

\bibitem[Li and Pathak(2021)]{li2021functional}
Alexander Li and Deepak Pathak.
\newblock Functional regularization for reinforcement learning via learned
  fourier features.
\newblock \emph{Advances in Neural Information Processing Systems}, 34, 2021.

\bibitem[Li et~al.(2018)Li, Xu, Taylor, Studer, and
  Goldstein]{li2018visualizing}
Hao Li, Zheng Xu, Gavin Taylor, Christoph Studer, and Tom Goldstein.
\newblock Visualizing the loss landscape of neural nets.
\newblock In \emph{Proceedings of the 32nd International Conference on Neural
  Information Processing Systems}, pages 6391--6401, 2018.

\bibitem[Li et~al.(2006)Li, Walsh, and Littman]{li2006towards}
Lihong Li, Thomas~J Walsh, and Michael~L Littman.
\newblock Towards a unified theory of state abstraction for mdps.
\newblock In \emph{ISAIM}, 2006.

\bibitem[Li and Hoiem(2017)]{li2017learning}
Zhizhong Li and Derek Hoiem.
\newblock Learning without forgetting.
\newblock \emph{IEEE transactions on pattern analysis and machine
  intelligence}, 40\penalty0 (12):\penalty0 2935--2947, 2017.

\bibitem[Lin et~al.(2019)Lin, Baweja, Kantor, and Held]{lin2019adaptive}
Xingyu Lin, Harjatin Baweja, George Kantor, and David Held.
\newblock Adaptive auxiliary task weighting for reinforcement learning.
\newblock In \emph{Advances in Neural Information Processing Systems}, 2019.

\bibitem[Liu et~al.(2020{\natexlab{a}})Liu, White, Yao, and
  White]{liu2020towards}
Vincent Liu, Adam White, Hengshuai Yao, and Martha White.
\newblock Towards a practical measure of interference for reinforcement
  learning.
\newblock \emph{arXiv preprint arXiv:2007.03807}, 2020{\natexlab{a}}.

\bibitem[Liu et~al.(2020{\natexlab{b}})Liu, White, Yao, and
  White]{liu2020measuring}
Vincent Liu, Adam~M White, Hengshuai Yao, and Martha White.
\newblock Measuring and mitigating interference in reinforcement learning.
\newblock \emph{arXiv preprint}, 2020{\natexlab{b}}.

\bibitem[Lo and Ghiassian(2019)]{lo2019overcoming}
Yat~Long Lo and Sina Ghiassian.
\newblock Overcoming catastrophic interference in online reinforcement learning
  with dynamic self-organizing maps.
\newblock \emph{arXiv preprint arXiv:1910.13213}, 2019.

\bibitem[Lopez-Paz and Ranzato(2017)]{lopez2017gradient}
David Lopez-Paz and Marc'Aurelio Ranzato.
\newblock Gradient episodic memory for continual learning (icml).
\newblock In \emph{Advances in Neural Information Processing Systems}, 2017.

\bibitem[Louizos et~al.(2017)Louizos, Shalit, Mooij, Sontag, Zemel, and
  Welling]{louizos2017causal}
Christos Louizos, Uri Shalit, Joris~M Mooij, David Sontag, Richard Zemel, and
  Max Welling.
\newblock Causal effect inference with deep latent-variable models.
\newblock In \emph{Advances in Neural Information Processing Systems}, pages
  6446--6456, 2017.

\bibitem[Lyle et~al.(2019{\natexlab{a}})Lyle, Bellemare, and
  Castro]{lyle2019comparative}
Clare Lyle, Marc~G Bellemare, and Pablo~Samuel Castro.
\newblock A comparative analysis of expected and distributional reinforcement
  learning.
\newblock In \emph{AAAI Conference on Artificial Intelligence},
  2019{\natexlab{a}}.

\bibitem[Lyle et~al.(2019{\natexlab{b}})Lyle, van~der Wilk, Kwiatkowska, Gal,
  and Bloem-Reddy]{lyle2020benefits}
Clare Lyle, Mark van~der Wilk, Marta Kwiatkowska, Yarin Gal, and Benjamin
  Bloem-Reddy.
\newblock On the benefits of invariance in neural networks.
\newblock \emph{NeurIPS Workshop on Machine Learning with Guarantees},
  2019{\natexlab{b}}.

\bibitem[Lyle et~al.(2020)Lyle, Schut, Ru, Gal, and van~der
  Wilk]{lyle2020bayesian}
Clare Lyle, Lisa Schut, Robin Ru, Yarin Gal, and Mark van~der Wilk.
\newblock A bayesian perspective on training speed and model selection.
\newblock \emph{Advances in Neural Information Processing Systems}, 33, 2020.

\bibitem[Lyle et~al.(2021{\natexlab{a}})Lyle, Rowland, Ostrovski, and
  Dabney]{lyle2021effect}
Clare Lyle, Mark Rowland, Georg Ostrovski, and Will Dabney.
\newblock On the effect of auxiliary tasks on representation dynamics.
\newblock In \emph{Artificial Intelligence and Statistics (AISTATS)},
  2021{\natexlab{a}}.

\bibitem[Lyle et~al.(2021{\natexlab{b}})Lyle, Zhang, Jiang, Pineau, and
  Gal]{lyle2021causal}
Clare Lyle, Amy Zhang, Minqi Jiang, Joelle Pineau, and Yarin Gal.
\newblock Resolving causal confusion in reinforcement learning via robust
  exploration.
\newblock In \emph{Self-Supervision for Reinforcement Learning Workshop - ICLR
  2021}, 2021{\natexlab{b}}.

\bibitem[Lyle et~al.(2022{\natexlab{a}})Lyle, Rowland, and
  Dabney]{lyle2021understanding}
Clare Lyle, Mark Rowland, and Will Dabney.
\newblock Understanding and preventing capacity loss in reinforcement learning.
\newblock In \emph{International Conference on Learning Representations
  (Spotlight)}, 2022{\natexlab{a}}.

\bibitem[Lyle et~al.(2022{\natexlab{b}})Lyle, Rowland, Dabney, Kwiatkowska, and
  Gal]{lyle2022generalization}
Clare Lyle, Mark Rowland, Will Dabney, Marta Kwiatkowska, and Yarin Gal.
\newblock Learning dynamics and generalization in reinforcement learning.
\newblock \emph{Proceedings of the 39th International Conference on Machine
  Learning (ICML)}, 2022{\natexlab{b}}.

\bibitem[Machado et~al.(2017)Machado, Bellemare, and
  Bowling]{machado2017laplacian}
Marlos~C. Machado, Marc~G. Bellemare, and Michael Bowling.
\newblock A {L}aplacian framework for option discovery in reinforcement
  learning.
\newblock In Doina Precup and Yee~Whye Teh, editors, \emph{Proceedings of the
  34th International Conference on Machine Learning (ICML)}, volume~70 of
  \emph{Proceedings of Machine Learning Research}, pages 2295--2304. PMLR,
  06--11 Aug 2017.

\bibitem[Machado et~al.(2018{\natexlab{a}})Machado, Bellemare, Talvitie,
  Veness, Hausknecht, and Bowling]{machado2018revisiting}
Marlos~C Machado, Marc~G Bellemare, Erik Talvitie, Joel Veness, Matthew
  Hausknecht, and Michael Bowling.
\newblock Revisiting the {A}rcade {L}earning {E}nvironment: Evaluation
  protocols and open problems for general agents.
\newblock \emph{Journal of Artificial Intelligence Research}, 61:\penalty0
  523--562, 2018{\natexlab{a}}.

\bibitem[Machado et~al.(2018{\natexlab{b}})Machado, Rosenbaum, Guo, Liu,
  Tesauro, and Campbell]{machado2017eigenoption}
Marlos~C Machado, Clemens Rosenbaum, Xiaoxiao Guo, Miao Liu, Gerald Tesauro,
  and Murray Campbell.
\newblock Eigenoption discovery through the deep successor representation.
\newblock In \emph{International Conference on Learning Representations
  (ICLR)}, 2018{\natexlab{b}}.

\bibitem[MacKay(1998)]{mackay1998choice}
David J.~C. MacKay.
\newblock Choice of basis for laplace approximation.
\newblock \emph{Journal of Machine Learning}, 33\penalty0 (1), 1998.

\bibitem[MacKay(1992)]{mackay1992bayesian}
David~JC MacKay.
\newblock \emph{Bayesian methods for adaptive models}.
\newblock PhD thesis, California Institute of Technology, 1992.

\bibitem[MacKay(2003)]{mackay2003information}
David~JC MacKay.
\newblock \emph{Information theory, inference and learning algorithms}.
\newblock Cambridge university press, 2003.

\bibitem[Maddox et~al.(2019)Maddox, Izmailov, Garipov, Vetrov, and
  Wilson]{maddox2019simple}
Wesley~J Maddox, Pavel Izmailov, Timur Garipov, Dmitry~P Vetrov, and
  Andrew~Gordon Wilson.
\newblock A simple baseline for bayesian uncertainty in deep learning.
\newblock In \emph{Advances in Neural Information Processing Systems}, pages
  13132--13143, 2019.

\bibitem[Maddox et~al.(2020)Maddox, Benton, and Wilson]{maddox2020rethinking}
Wesley~J Maddox, Gregory Benton, and Andrew~Gordon Wilson.
\newblock Rethinking parameter counting in deep models: Effective
  dimensionality revisited.
\newblock \emph{arXiv preprint arXiv:2003.02139}, 2020.

\bibitem[Magliacane et~al.(2018)Magliacane, van Ommen, Claassen, Bongers,
  Versteeg, and Mooij]{Magliacane2018}
Sara Magliacane, Thijs van Ommen, Tom Claassen, Stephan Bongers, Philip
  Versteeg, and Joris~M Mooij.
\newblock Domain adaptation by using causal inference to predict invariant
  conditional distributions.
\newblock In S.~Bengio, H.~Wallach, H.~Larochelle, K.~Grauman, N.~Cesa-Bianchi,
  and R.~Garnett, editors, \emph{Advances in Neural Information Processing
  Systems 31}, pages 10846--10856. Curran Associates, Inc., 2018.

\bibitem[Mahadevan(2005)]{mahadevan2005proto}
Sridhar Mahadevan.
\newblock Proto-value functions: Developmental reinforcement learning.
\newblock In \emph{Proceedings of the 22nd International Conference on Machine
  Learning (ICML)}, pages 553--560, 2005.

\bibitem[Mahadevan(2009)]{mahadevan2009learning}
Sridhar Mahadevan.
\newblock Learning representation and control in {M}arkov decision processes:
  {N}ew frontiers.
\newblock \emph{Foundations and Trends{\textregistered} in Machine Learning
  (ICML)}, 1\penalty0 (4):\penalty0 403--565, 2009.

\bibitem[Mahadevan and Maggioni(2007)]{mahadevan2007proto}
Sridhar Mahadevan and Mauro Maggioni.
\newblock Proto-value functions: A {L}aplacian framework for learning
  representation and control in {M}arkov decision processes.
\newblock \emph{Journal of Machine Learning Research}, 8\penalty0
  (Oct):\penalty0 2169--2231, 2007.

\bibitem[Mandt et~al.(2017)Mandt, Hoffman, and Blei]{mandt2017stochastic}
Stephan Mandt, Matthew~D Hoffman, and David~M Blei.
\newblock Stochastic gradient descent as approximate bayesian inference.
\newblock \emph{The Journal of Machine Learning Research}, 18\penalty0
  (1):\penalty0 4873--4907, 2017.

\bibitem[Matthews et~al.(2017)Matthews, Hron, Turner, and
  Ghahramani]{matthews2017}
{Alexander G de G} Matthews, Jiri Hron, Richard~E Turner, and Zoubin
  Ghahramani.
\newblock Sample-then-optimize posterior sampling for bayesian linear models.
\newblock \emph{Advances in Neural Information Processing Systems}, 2017.

\bibitem[{Matthews} et~al.(2018){Matthews}, Hron, Rowland, Turner, and
  Ghahramani]{matthews2018gaussian}
{Alexander G. de G.} {Matthews}, Jiri Hron, Mark Rowland, Richard~E. Turner,
  and Zoubin Ghahramani.
\newblock Gaussian process behaviour in wide deep neural networks.
\newblock In \emph{International Conference on Learning Representations
  (ICLR)}, 2018.

\bibitem[McAllester(1999)]{mcallester1999}
David~A. McAllester.
\newblock Some {PAC}-{Bayesian} {Theorems}.
\newblock \emph{Machine Learning (ICML)}, 37\penalty0 (3):\penalty0 355--363,
  1999.

\bibitem[Mirowski et~al.(2017)Mirowski, Pascanu, Viola, Soyer, Ballard, Banino,
  Denil, Goroshin, Sifre, Kavukcuoglu, Kumaran, and
  Hadsell]{mirowski2017learning}
Piotr Mirowski, Razvan Pascanu, Fabio Viola, Hubert Soyer, Andrew~J Ballard,
  Andrea Banino, Misha Denil, Ross Goroshin, Laurent Sifre, Koray Kavukcuoglu,
  Dharshan Kumaran, and Raia Hadsell.
\newblock Learning to navigate in complex environments.
\newblock In \emph{International Conference on Learning Representations
  (ICLR)}, 2017.

\bibitem[Misra et~al.(2020)Misra, Henaff, Krishnamurthy, and
  Langford]{misra2020kinematic}
Dipendra Misra, Mikael Henaff, Akshay Krishnamurthy, and John Langford.
\newblock Kinematic state abstraction and provably efficient rich-observation
  reinforcement learning.
\newblock In \emph{Proceedings of the 37th International Conference on Machine
  Learning (ICML)}, pages 6961--6971. PMLR, 2020.

\bibitem[Mnih et~al.(2015)Mnih, Kavukcuoglu, Silver, Rusu, Veness, Bellemare,
  Graves, Riedmiller, Fidjeland, Ostrovski, Petersen, Beattie, Sadik,
  Antonoglou, King, Kumaran, Wierstra, Legg, and Hassabis]{mnih2015human}
Volodymyr Mnih, Koray Kavukcuoglu, David Silver, Andrei~A Rusu, Joel Veness,
  Marc~G Bellemare, Alex Graves, Martin Riedmiller, Andreas~K Fidjeland, Georg
  Ostrovski, Stig Petersen, Charles Beattie, Amir Sadik, Ioannis Antonoglou,
  Helen King, Dharshan Kumaran, Daan Wierstra, Shane Legg, and Demis Hassabis.
\newblock Human-level control through deep reinforcement learning.
\newblock \emph{Nature}, 518\penalty0 (7540):\penalty0 529--533, 2015.

\bibitem[Morcos et~al.(2018)Morcos, Barrett, Rabinowitz, and
  Botvinick]{morcos2018importance}
Ari~S. Morcos, David~G.T. Barrett, Neil~C. Rabinowitz, and Matthew Botvinick.
\newblock On the importance of single directions for generalization.
\newblock In \emph{International Conference on Learning Representations
  (ICLR)}, 2018.

\bibitem[Nagarajan and Kolter(2019)]{nagarajan2019uniform}
Vaishnavh Nagarajan and J.~Zico Kolter.
\newblock Uniform convergence may be unable to explain generalization in deep
  learning.
\newblock In H.~Wallach, H.~Larochelle, A.~Beygelzimer, F.~d'~Alche-Buc,
  E.~Fox, and R.~Garnett, editors, \emph{Advances in Neural Information
  Processing Systems 32}, pages 11615--11626. Curran Associates, Inc., 2019.

\bibitem[Nakkiran et~al.(2021)Nakkiran, Kaplun, Bansal, Yang, Barak, and
  Sutskever]{nakkiran2019deep}
Preetum Nakkiran, Gal Kaplun, Yamini Bansal, Tristan Yang, Boaz Barak, and Ilya
  Sutskever.
\newblock Deep double descent: Where bigger models and more data hurt.
\newblock \emph{Journal of Statistical Mechanics: Theory and Experiment},
  2021\penalty0 (12):\penalty0 124003, 2021.

\bibitem[Neal(2012)]{neal2012bayesian}
Radford~M Neal.
\newblock \emph{Bayesian learning for neural networks}, volume 118.
\newblock Springer Science \& Business Media, 2012.

\bibitem[Negrea et~al.(2020)Negrea, Dziugaite, and Roy]{negrea2020defense}
Jeffrey Negrea, Gintare~Karolina Dziugaite, and Daniel Roy.
\newblock In defense of uniform convergence: Generalization via derandomization
  with an application to interpolating predictors.
\newblock In \emph{Proceedings of the 37th International Conference on Machine
  Learning (ICML)}, pages 7263--7272. PMLR, 2020.

\bibitem[Neyshabur et~al.(2015{\natexlab{a}})Neyshabur, Tomioka, and
  Srebro]{neyshabur2014search}
Behnam Neyshabur, Ryota Tomioka, and Nathan Srebro.
\newblock In search of the real inductive bias: {On} the role of implicit
  regularization in deep learning.
\newblock \emph{International Conference on Learning Representations (Workshop
  Track)}, 2015{\natexlab{a}}.

\bibitem[Neyshabur et~al.(2015{\natexlab{b}})Neyshabur, Tomioka, and
  Srebro]{neyshabur_norm-based_2015}
Behnam Neyshabur, Ryota Tomioka, and Nathan Srebro.
\newblock Norm-based capacity control in neural networks.
\newblock In \emph{Conference on {Learning (ICML)} {Theory}}, pages 1376--1401,
  2015{\natexlab{b}}.

\bibitem[Neyshabur et~al.(2017)Neyshabur, Bhojanapalli, McAllester, and
  Srebro]{neyshabur2017exploring}
Behnam Neyshabur, Srinadh Bhojanapalli, David McAllester, and Nati Srebro.
\newblock Exploring generalization in deep learning.
\newblock In \emph{Advances in Neural Information Processing Systems}, pages
  5947--5956, 2017.

\bibitem[Neyshabur et~al.(2019)Neyshabur, Li, Bhojanapalli, LeCun, and
  Srebro]{neyshabur2018the}
Behnam Neyshabur, Zhiyuan Li, Srinadh Bhojanapalli, Yann LeCun, and Nathan
  Srebro.
\newblock The role of over-parametrization in generalization of neural
  networks.
\newblock In \emph{International Conference on Learning Representations
  (ICLR)}, 2019.

\bibitem[Nikishin et~al.(2022)Nikishin, Schwarzer, D'Oro, Bacon, and
  Courville]{nikishin2022primacy}
Evgenii Nikishin, Max Schwarzer, Pierluca D'Oro, Pierre-Luc Bacon, and Aaron
  Courville.
\newblock The primacy bias in deep reinforcement learning.
\newblock In \emph{Proceedings of the 39th International Conference on
  International Conference on Machine Learning (ICML)}, 2022.

\bibitem[Osband et~al.(2016)Osband, Blundell, Pritzel, and
  Van~Roy]{osband2016deep}
Ian Osband, Charles Blundell, Alexander Pritzel, and Benjamin Van~Roy.
\newblock Deep exploration via bootstrapped {DQN}.
\newblock In \emph{Advances in Neural Information Processing Systems}, 2016.

\bibitem[Osband et~al.(2018)Osband, Aslanides, and
  Cassirer]{osband2018randomized}
Ian Osband, John Aslanides, and Albin Cassirer.
\newblock Randomized prior functions for deep reinforcement learning.
\newblock In \emph{Advances in Neural Information Processing Systems}, pages
  8617--8629, 2018.

\bibitem[Ostrovski et~al.(2021)Ostrovski, Castro, and Dabney]{ostrovski2021the}
Georg Ostrovski, Pablo~Samuel Castro, and Will Dabney.
\newblock The difficulty of passive learning in deep reinforcement learning.
\newblock In \emph{Advances in Neural Information Processing Systems}, 2021.

\bibitem[Packer et~al.(2018)Packer, Gao, Kos, Kr{\"{a}}henb{\"{u}}hl, Koltun,
  and Song]{packer2018genrl}
Charles Packer, Katelyn Gao, Jernej Kos, Philipp Kr{\"{a}}henb{\"{u}}hl,
  Vladlen Koltun, and Dawn Song.
\newblock Assessing generalization in deep reinforcement learning.
\newblock \emph{arXiv preprint}, 2018.

\bibitem[Parker-Holder et~al.(2022)Parker-Holder, Jiang, Dennis, Samvelyan,
  Foerster, Grefenstette, and Rockt{\"a}schel]{parker2022evolving}
Jack Parker-Holder, Minqi Jiang, Michael Dennis, Mikayel Samvelyan, Jakob
  Foerster, Edward Grefenstette, and Tim Rockt{\"a}schel.
\newblock Evolving curricula with regret-based environment design.
\newblock \emph{arXiv preprint arXiv:2203.01302}, 2022.

\bibitem[Parr et~al.(2008)Parr, Li, Taylor, Painter-Wakefield, and
  Littman]{parr2008analysis}
Ronald Parr, Lihong Li, Gavin Taylor, Christopher Painter-Wakefield, and
  Michael~L Littman.
\newblock An analysis of linear models, linear value-function approximation,
  and feature selection for reinforcement learning.
\newblock In \emph{Proceedings of the 25th International Conference on Machine
  Learning (ICML)}, 2008.

\bibitem[Pearl(2000)]{pearl2000causality}
Judea Pearl.
\newblock \emph{Causality: models, reasoning and inference}, volume~29.
\newblock Springer, 2000.

\bibitem[Peters et~al.(2016)Peters, B{\"u}hlmann, and
  Meinshausen]{peters2016causal}
Jonas Peters, Peter B{\"u}hlmann, and Nicolai Meinshausen.
\newblock Causal inference by using invariant prediction: identification and
  confidence intervals.
\newblock \emph{Journal of the Royal Statistical Society: Series B (Statistical
  Methodology)}, 78\penalty0 (5):\penalty0 947--1012, 2016.

\bibitem[Peters et~al.(2017)Peters, Janzing, and
  Sch{\"o}lkopf]{peters2017elements}
Jonas Peters, Dominik Janzing, and Bernhard Sch{\"o}lkopf.
\newblock \emph{Elements of causal inference: foundations and learning
  algorithms}.
\newblock MIT press, 2017.

\bibitem[Pohlen et~al.(2018)Pohlen, Piot, Hester, Azar, Horgan, Budden,
  Barth-Maron, Van~Hasselt, Quan, Ve{\v{c}}er{\'\i}k,
  et~al.]{pohlen2018observe}
Tobias Pohlen, Bilal Piot, Todd Hester, Mohammad~Gheshlaghi Azar, Dan Horgan,
  David Budden, Gabriel Barth-Maron, Hado Van~Hasselt, John Quan, Mel
  Ve{\v{c}}er{\'\i}k, et~al.
\newblock Observe and look further: Achieving consistent performance on atari.
\newblock \emph{arXiv preprint arXiv:1805.11593}, 2018.

\bibitem[Popper(1968)]{popper1968logic}
Karl~R Popper.
\newblock \emph{The logic of scientific discovery, Karl R. Popper}.
\newblock Hutchinson. London. GB, 1968.

\bibitem[Precup et~al.(2001)Precup, Sutton, and Dasgupta]{precup2001off}
Doina Precup, Richard~S Sutton, and Sanjoy Dasgupta.
\newblock Off-policy temporal-difference learning with function approximation.
\newblock In \emph{ICML}, pages 417--424, 2001.

\bibitem[Puterman(1994)]{puterman}
Martin~L. Puterman.
\newblock \emph{Markov Decision Processes: Discrete Stochastic Dynamic
  Programming}.
\newblock Wiley Series in Probability and Statistics. Wiley, 1994.
\newblock ISBN 978-0-47161977-2.
\newblock \doi{10.1002/9780470316887}.

\bibitem[Qi et~al.(2017)Qi, Su, Mo, and Guibas]{qi2017pointnet}
Charles~R Qi, Hao Su, Kaichun Mo, and Leonidas~J Guibas.
\newblock Pointnet: Deep learning on point sets for 3d classification and
  segmentation.
\newblock \emph{Proc.\ Computer Vision and Pattern Recognition (CVPR), IEEE},
  pages 652--660, 2017.

\bibitem[Quan and Ostrovski(2020)]{dqnzoo2020github}
John Quan and Georg Ostrovski.
\newblock {DQN} {Zoo}: Reference implementations of {DQN}-based agents, 2020.

\bibitem[Radford et~al.(2018)Radford, Narasimhan, Salimans, and
  Sutskever]{radford2018improving}
Alec Radford, Karthik Narasimhan, Tim Salimans, and Ilya Sutskever.
\newblock Improving language understanding by generative pre-training.
\newblock \emph{arXiv preprint}, 2018.

\bibitem[Raghu et~al.(2017)Raghu, Poole, Kleinberg, Ganguli, and
  Sohl-Dickstein]{raghu2017expressive}
Maithra Raghu, Ben Poole, Jon Kleinberg, Surya Ganguli, and Jascha
  Sohl-Dickstein.
\newblock On the expressive power of deep neural networks.
\newblock In \emph{Proceedings of the 34th International Conference on Machine
  Learning (ICML)}, pages 2847--2854. PMLR, 2017.

\bibitem[Rahaman et~al.(2019)Rahaman, Baratin, Arpit, Draxler, Lin, Hamprecht,
  Bengio, and Courville]{rahaman2019spectral}
Nasim Rahaman, Aristide Baratin, Devansh Arpit, Felix Draxler, Min Lin, Fred
  Hamprecht, Yoshua Bengio, and Aaron Courville.
\newblock On the spectral bias of neural networks.
\newblock In \emph{Proceedings of the 36th International Conference on Machine
  Learning (ICML)}, pages 5301--5310. PMLR, 2019.

\bibitem[Rahimi and Recht(2008)]{rahimi2008random}
Ali Rahimi and Benjamin Recht.
\newblock Random features for large-scale kernel machines.
\newblock In \emph{Advances in neural information processing systems}, pages
  1177--1184, 2008.

\bibitem[Raileanu and Fergus(2021)]{raileanu2021decoupling}
Roberta Raileanu and Rob Fergus.
\newblock Decoupling value and policy for generalization in reinforcement
  learning.
\newblock \emph{International Conference on Machine Learning (ICML)}, pages
  8787--8798, 2021.

\bibitem[Raileanu et~al.(2021)Raileanu, Goldstein, Yarats, Kostrikov, and
  Fergus]{raileanu2021automatic}
Roberta Raileanu, Maxwell Goldstein, Denis Yarats, Ilya Kostrikov, and Rob
  Fergus.
\newblock Automatic data augmentation for generalization in reinforcement
  learning.
\newblock In \emph{Advances in Neural Information Processing Systems}, 2021.

\bibitem[Raj et~al.(2017)Raj, Kumar, Mroueh, Fletcher, and
  Schoelkopf]{raj2017orbit_embeddings}
Anant Raj, Abhishek Kumar, Youssef Mroueh, Tom Fletcher, and Bernhard
  Schoelkopf.
\newblock {Local Group Invariant Representations via Orbit Embeddings}.
\newblock In Aarti Singh and Jerry Zhu, editors, \emph{Proceedings of the 20th
  International Conference on Artificial Intelligence and Statistics},
  volume~54 of \emph{Proceedings of Machine Learning Research}, pages
  1225--1235. PMLR, 2017.

\bibitem[Rakelly et~al.(2019)Rakelly, Zhou, Finn, Levine, and
  Quillen]{rakelly2019efficient}
Kate Rakelly, Aurick Zhou, Chelsea Finn, Sergey Levine, and Deirdre Quillen.
\newblock Efficient off-policy meta-reinforcement learning via probabilistic
  context variables.
\newblock In \emph{Proceedings of the 36th International Conference on Machine
  Learning (ICML)}, pages 5331--5340. PMLR, 2019.

\bibitem[Rasmussen(2003)]{rasmussen2003gaussian}
Carl~Edward Rasmussen.
\newblock Gaussian processes in machine learning (icml).
\newblock In \emph{Summer School on Machine Learning (ICML)}, pages 63--71.
  Springer, 2003.

\bibitem[Rasmussen and Ghahramani(2001)]{rasmussen2001occam}
Carl~Edward Rasmussen and Zoubin Ghahramani.
\newblock Occam's razor.
\newblock In \emph{Advances in neural information processing systems}, pages
  294--300, 2001.

\bibitem[Ravanbakhsh et~al.(2017)Ravanbakhsh, Schneider, and
  Poczos]{ravanbakhsh2017equivariance}
Siamak Ravanbakhsh, Jeff Schneider, and Barnabas Poczos.
\newblock Equivariance through parameter-sharing.
\newblock In \emph{Proceedings of the 34th International Conference on Machine
  Learning-Volume 70}, pages 2892--2901. JMLR. org, 2017.

\bibitem[Robbins and Monro(1951)]{robbins1951stochastic}
Herbert Robbins and Sutton Monro.
\newblock A stochastic approximation method.
\newblock \emph{The annals of mathematical statistics}, pages 400--407, 1951.

\bibitem[Rojas-Carulla et~al.(2018)Rojas-Carulla, Sch{\"o}lkopf, Turner, and
  Peters]{rojas2018invariant}
Mateo Rojas-Carulla, Bernhard Sch{\"o}lkopf, Richard Turner, and Jonas Peters.
\newblock Invariant models for causal transfer learning.
\newblock \emph{The Journal of Machine Learning Research}, 19\penalty0
  (1):\penalty0 1309--1342, 2018.

\bibitem[Ru et~al.(2021)Ru, Lyle, Schut, van~der Wilk, and
  Gal]{ru2020revisiting}
Binxin Ru, Clare Lyle, Lisa Schut, Mark van~der Wilk, and Yarin Gal.
\newblock Revisiting the train loss: an efficient performance estimator for
  neural architecture search.
\newblock \emph{Advances in Neural Information Processing Systems (Spotlight
  Presentation)}, 34, 2021.

\bibitem[Rusu et~al.(2016)Rusu, Colmenarejo, G{\"u}l{\c{c}}ehre, Desjardins,
  Kirkpatrick, Pascanu, Mnih, Kavukcuoglu, and Hadsell]{rusu2016policy}
Andrei~A Rusu, Sergio~Gomez Colmenarejo, {\c{C}}aglar G{\"u}l{\c{c}}ehre,
  Guillaume Desjardins, James Kirkpatrick, Razvan Pascanu, Volodymyr Mnih,
  Koray Kavukcuoglu, and Raia Hadsell.
\newblock Policy distillation.
\newblock In \emph{ICLR (Poster)}, 2016.

\bibitem[Salamon and Bello(2017)]{salamon2017deep}
Justin Salamon and Juan~Pablo Bello.
\newblock Deep convolutional neural networks and data augmentation for
  environmental sound classification.
\newblock \emph{IEEE Signal Processing Letters}, 24\penalty0 (3):\penalty0
  279--283, 2017.

\bibitem[Samvelyan et~al.(2021)Samvelyan, Kirk, Kurin, Parker-Holder, Jiang,
  Hambro, Petroni, Kuttler, Grefenstette, and
  Rockt{\"a}schel]{samvelyan2021minihack}
Mikayel Samvelyan, Robert Kirk, Vitaly Kurin, Jack Parker-Holder, Minqi Jiang,
  Eric Hambro, Fabio Petroni, Heinrich Kuttler, Edward Grefenstette, and Tim
  Rockt{\"a}schel.
\newblock Minihack the planet: A sandbox for open-ended reinforcement learning
  research.
\newblock In \emph{Thirty-fifth Conference on Neural Information Processing
  Systems Datasets and Benchmarks Track (Round 1)}, 2021.

\bibitem[Santurkar et~al.(2018)Santurkar, Tsipras, Ilyas, and
  Madry]{santurkar2018does}
Shibani Santurkar, Dimitris Tsipras, Andrew Ilyas, and Aleksander Madry.
\newblock How does batch normalization help optimization?
\newblock \emph{Advances in neural information processing systems}, 31, 2018.

\bibitem[Schaul et~al.(2019)Schaul, Borsa, Modayil, and Pascanu]{schaul2019ray}
Tom Schaul, Diana Borsa, Joseph Modayil, and Razvan Pascanu.
\newblock Ray interference: a source of plateaus in deep reinforcement
  learning.
\newblock \emph{Workshop on Reinforcement Learning and Decision Making}, 2019.

\bibitem[Schmitt et~al.(2018)Schmitt, Hudson, Zidek, Osindero, Doersch,
  Czarnecki, Leibo, Kuttler, Zisserman, Simonyan,
  et~al.]{schmitt2018kickstarting}
Simon Schmitt, Jonathan~J Hudson, Augustin Zidek, Simon Osindero, Carl Doersch,
  Wojciech~M Czarnecki, Joel~Z Leibo, Heinrich Kuttler, Andrew Zisserman, Karen
  Simonyan, et~al.
\newblock Kickstarting deep reinforcement learning.
\newblock \emph{arXiv preprint arXiv:1803.03835}, 2018.

\bibitem[Schrittwieser et~al.(2020)Schrittwieser, Antonoglou, Hubert, Simonyan,
  Sifre, Schmitt, Guez, Lockhart, Hassabis, Graepel,
  et~al.]{schrittwieser2020mastering}
Julian Schrittwieser, Ioannis Antonoglou, Thomas Hubert, Karen Simonyan,
  Laurent Sifre, Simon Schmitt, Arthur Guez, Edward Lockhart, Demis Hassabis,
  Thore Graepel, et~al.
\newblock Mastering atari, go, chess and shogi by planning with a learned
  model.
\newblock \emph{Nature}, 588\penalty0 (7839):\penalty0 604--609, 2020.

\bibitem[Schulman et~al.(2015)Schulman, Levine, Abbeel, Jordan, and
  Moritz]{schulman2015trust}
John Schulman, Sergey Levine, Pieter Abbeel, Michael Jordan, and Philipp
  Moritz.
\newblock Trust region policy optimization.
\newblock In \emph{Proceedings of the 32nd International Conference on Machine
  Learning (ICML)}, pages 1889--1897. PMLR, 2015.

\bibitem[Schulman et~al.(2017)Schulman, Wolski, Dhariwal, Radford, and
  Klimov]{schulman2017proximal}
John Schulman, Filip Wolski, Prafulla Dhariwal, Alec Radford, and Oleg Klimov.
\newblock Proximal policy optimization algorithms.
\newblock \emph{arXiv preprint arXiv:1707.06347}, 2017.

\bibitem[Schwarz et~al.(2018)Schwarz, Czarnecki, Luketina, Grabska{-}Barwinska,
  Teh, Pascanu, and Hadsell]{schwarz2018progress}
Jonathan Schwarz, Wojciech Czarnecki, Jelena Luketina, Agnieszka
  Grabska{-}Barwinska, Yee~Whye Teh, Razvan Pascanu, and Raia Hadsell.
\newblock Progress {\&} compress: {A} scalable framework for continual learning
  (icml).
\newblock In \emph{Proceedings of the 35th International Conference on Machine
  Learning (ICML)}, 2018.

\bibitem[Serre(1977)]{serre1977linear}
Jean-Pierre Serre.
\newblock \emph{Linear representations of finite groups}, volume~42.
\newblock Springer, 1977.

\bibitem[Shalit et~al.(2017)Shalit, Johansson, and
  Sontag]{shalit2017estimating}
Uri Shalit, Fredrik~D Johansson, and David Sontag.
\newblock Estimating individual treatment effect: generalization bounds and
  algorithms.
\newblock In \emph{Proceedings of the 34th International Conference on Machine
  Learning-Volume 70}, pages 3076--3085. JMLR. org, 2017.

\bibitem[Shao et~al.(2022)Shao, You, Yan, Sun, and Bohg]{shao2020self}
Lin Shao, Yifan You, Mengyuan Yan, Qingyun Sun, and Jeannette Bohg.
\newblock {GRAC:} self-guided and self-regularized actor-critic.
\newblock \emph{Conference on Robot Learning}, pages 267--276, 2022.

\bibitem[Sharkey and Sharkey(1995)]{sharkey1995analysis}
Noel~E Sharkey and Amanda~JC Sharkey.
\newblock An analysis of catastrophic interference.
\newblock \emph{Connection Science}, 1995.

\bibitem[Silver and Mercer(2002)]{silver2002task}
Daniel~L Silver and Robert~E Mercer.
\newblock The task rehearsal method of life-long learning: Overcoming
  impoverished data.
\newblock In \emph{Conference of the Canadian Society for Computational Studies
  of Intelligence}, pages 90--101, 2002.

\bibitem[Smith and Le(2018)]{smith2018}
Samuel~L. Smith and Quoc~V. Le.
\newblock A bayesian perspective on generalization and stochastic gradient
  descent.
\newblock In \emph{International Conference on Learning Representations
  (ICLR)}, 2018.

\bibitem[Smith et~al.(2020)Smith, Dherin, Barrett, and De]{smith2020origin}
Samuel~L Smith, Benoit Dherin, David Barrett, and Soham De.
\newblock On the origin of implicit regularization in stochastic gradient
  descent.
\newblock In \emph{International Conference on Learning Representations
  (ICLR)}, 2020.

\bibitem[Sodhani et~al.(2021)Sodhani, Zhang, and Pineau]{sodhani2021multi}
Shagun Sodhani, Amy Zhang, and Joelle Pineau.
\newblock Multi-task reinforcement learning with context-based tepresentations.
\newblock In \emph{International Conference on Machine Learning (ICML)}, pages
  9767--9779. PMLR, 2021.

\bibitem[Song et~al.(2020)Song, Jiang, Tu, Du, and
  Neyshabur]{Song2020Observational}
Xingyou Song, Yiding Jiang, Stephen Tu, Yilun Du, and Behnam Neyshabur.
\newblock Observational overfitting in reinforcement learning.
\newblock In \emph{International Conference on Learning Representations
  (ICLR)}, 2020.

\bibitem[Srivastava et~al.(2014)Srivastava, Hinton, Krizhevsky, Sutskever, and
  Salakhutdinov]{srivastava2014dropout}
Nitish Srivastava, Geoffrey Hinton, Alex Krizhevsky, Ilya Sutskever, and Ruslan
  Salakhutdinov.
\newblock Dropout: a simple way to prevent neural networks from overfitting.
\newblock \emph{The journal of machine learning research}, 15\penalty0
  (1):\penalty0 1929--1958, 2014.

\bibitem[Stachenfeld et~al.(2014)Stachenfeld, Botvinick, and
  Gershman]{stachenfeld2014design}
Kimberly~L Stachenfeld, Matthew Botvinick, and Samuel~J Gershman.
\newblock Design principles of the hippocampal cognitive map.
\newblock In \emph{Advances in Neural Information Processing Systems}, 2014.

\bibitem[Stachenfeld et~al.(2017)Stachenfeld, Botvinick, and
  Gershman]{stachenfeld2017hippocampus}
Kimberly~L Stachenfeld, Matthew~M Botvinick, and Samuel~J Gershman.
\newblock The hippocampus as a predictive map.
\newblock \emph{Nature neuroscience}, 20\penalty0 (11):\penalty0 1643--1653,
  2017.

\bibitem[Suter et~al.(2019)Suter, Miladinovic, Sch{\"o}lkopf, and
  Bauer]{suter2019robustly}
Raphael Suter, Djordje Miladinovic, Bernhard Sch{\"o}lkopf, and Stefan Bauer.
\newblock Robustly disentangled causal mechanisms: Validating deep
  representations for interventional robustness.
\newblock In \emph{Proceedings of the 36th International Conference on Machine
  Learning (ICML)}, pages 6056--6065. PMLR, 2019.

\bibitem[Sutton and Barto(2018)]{sutton2018reinforcement}
Richard~S Sutton and Andrew~G Barto.
\newblock \emph{Reinforcement learning: An introduction}.
\newblock MIT press, 2018.

\bibitem[Sutton et~al.(2000)Sutton, McAllester, Singh, and
  Mansour]{sutton2000policy}
Richard~S Sutton, David~A McAllester, Satinder~P Singh, and Yishay Mansour.
\newblock Policy gradient methods for reinforcement learning with function
  approximation.
\newblock In \emph{Advances in neural information processing systems}, pages
  1057--1063, 2000.

\bibitem[Sutton et~al.(2008)Sutton, Szepesv{\'a}ri, and
  Maei]{sutton2008convergent}
Richard~S Sutton, Csaba Szepesv{\'a}ri, and Hamid~Reza Maei.
\newblock A convergent {$O(n)$} algorithm for off-policy temporal-difference
  learning with linear function approximation.
\newblock \emph{Advances in Neural Information Processing Systems}, 2008.

\bibitem[Tassa et~al.(2018)Tassa, Doron, Muldal, Erez, Li, de~Las~Casas,
  Budden, Abdolmaleki, Merel, Lefrancq, Lillicrap, and
  Riedmiller]{deepmindcontrolsuite2018}
Yuval Tassa, Yotam Doron, Alistair Muldal, Tom Erez, Yazhe Li, Diego
  de~Las~Casas, David Budden, Abbas Abdolmaleki, Josh Merel, Andrew Lefrancq,
  Timothy Lillicrap, and Martin Riedmiller.
\newblock Deep{Mind} control suite.
\newblock Technical report, DeepMind, January 2018.

\bibitem[Taylor et~al.(2008)Taylor, Precup, and Panagaden]{taylor2008bounding}
Jonathan Taylor, Doina Precup, and Prakash Panagaden.
\newblock Bounding performance loss in approximate mdp homomorphisms.
\newblock \emph{Advances in Neural Information Processing Systems}, 21, 2008.

\bibitem[Teh et~al.(2017)Teh, Bapst, Czarnecki, Quan, Kirkpatrick, Hadsell,
  Heess, and Pascanu]{teh2017distral}
Yee Teh, Victor Bapst, Wojciech~M. Czarnecki, John Quan, James Kirkpatrick,
  Raia Hadsell, Nicolas Heess, and Razvan Pascanu.
\newblock Distral: Robust multitask reinforcement learning.
\newblock In I.~Guyon, U.~V. Luxburg, S.~Bengio, H.~Wallach, R.~Fergus,
  S.~Vishwanathan, and R.~Garnett, editors, \emph{Advances in Neural
  Information Processing Systems 30}, pages 4496--4506. Curran Associates,
  Inc., 2017.

\bibitem[Tishby and Zaslavsky(2015)]{tishby2015deep}
Naftali Tishby and Noga Zaslavsky.
\newblock Deep learning and the information bottleneck principle.
\newblock In \emph{2015 ieee information theory workshop (itw)}, pages 1--5.
  IEEE, 2015.

\bibitem[Tsitsiklis and Van~Roy(1996)]{tsitsiklis1996analysis}
John Tsitsiklis and Benjamin Van~Roy.
\newblock Analysis of temporal-diffference learning with function
  approximation.
\newblock \emph{Advances in neural information processing systems}, 9, 1996.

\bibitem[Tsitsiklis(1994)]{tsitsiklis1994asynchronous}
John~N Tsitsiklis.
\newblock Asynchronous stochastic approximation and {Q}-learning (icml).
\newblock \emph{Machine Learning (ICML)}, 16\penalty0 (3):\penalty0 185--202,
  1994.

\bibitem[Tzeng et~al.(2017)Tzeng, Hoffman, Saenko, and Darrell]{tzeng2017ada}
E.~Tzeng, J.~Hoffman, K.~Saenko, and T.~Darrell.
\newblock Adversarial discriminative domain adaptation.
\newblock In \emph{2017 IEEE Conference on Computer Vision and Pattern
  Recognition (CVPR)}, pages 2962--2971, Los Alamitos, CA, USA, jul 2017. IEEE
  Computer Society.
\newblock \doi{10.1109/CVPR.2017.316}.

\bibitem[Ullrich et~al.(2017)Ullrich, Meeds, and Welling]{ullrich2017soft}
Karen Ullrich, Edward Meeds, and Max Welling.
\newblock Soft weight-sharing for neural network compression.
\newblock In \emph{International Conference on Learning Representations
  (ICLR)}, 2017.

\bibitem[Valle-P{\'e}rez et~al.(2018)Valle-P{\'e}rez, Camargo, and
  Louis]{valle2018deep}
Guillermo Valle-P{\'e}rez, Chico~Q Camargo, and Ard~A Louis.
\newblock Deep learning generalizes because the parameter-function map is
  biased towards simple functions.
\newblock \emph{International Conference on Learning Representations}, 2018.

\bibitem[van~der Pol et~al.(2020{\natexlab{a}})van~der Pol, Kipf, Oliehoek, and
  Welling]{van2020plannable}
Elise van~der Pol, Thomas Kipf, Frans~A Oliehoek, and Max Welling.
\newblock Plannable approximations to mdp homomorphisms: Equivariance under
  actions.
\newblock \emph{Proceedings of the 19th International Conference on Autonomous
  Agents and MultiAgent Systems}, pages 1431--1439, 2020{\natexlab{a}}.

\bibitem[van~der Pol et~al.(2020{\natexlab{b}})van~der Pol, Worrall, van Hoof,
  Oliehoek, and Welling]{van2020mdp}
Elise van~der Pol, Daniel Worrall, Herke van Hoof, Frans Oliehoek, and Max
  Welling.
\newblock Mdp homomorphic networks: Group symmetries in reinforcement learning.
\newblock \emph{Advances in Neural Information Processing Systems},
  33:\penalty0 4199--4210, 2020{\natexlab{b}}.

\bibitem[van~der Wilk et~al.(2018)van~der Wilk, Bauer, John, and
  Hensman]{van2018learning}
Mark van~der Wilk, Matthias Bauer, ST~John, and James Hensman.
\newblock Learning invariances using the marginal likelihood.
\newblock In \emph{Advances in Neural Information Processing Systems}, pages
  9938--9948, 2018.

\bibitem[Van~Hasselt et~al.(2016)Van~Hasselt, Guez, and Silver]{van2016deep}
Hado Van~Hasselt, Arthur Guez, and David Silver.
\newblock Deep reinforcement learning with double {Q}-learning (icml).
\newblock In \emph{AAAI Conference on Artificial Intelligence}, 2016.

\bibitem[Van~Hasselt et~al.(2018)Van~Hasselt, Doron, Strub, Hessel, Sonnerat,
  and Modayil]{van2018deep}
Hado Van~Hasselt, Yotam Doron, Florian Strub, Matteo Hessel, Nicolas Sonnerat,
  and Joseph Modayil.
\newblock Deep reinforcement learning and the deadly triad.
\newblock \emph{arXiv preprint arXiv:1812.02648}, 2018.

\bibitem[Vapnik(1968)]{vapnik1968uniform}
Vladimir Vapnik.
\newblock On the uniform convergence of frequencies of occurrence of events to
  their probabilities.
\newblock \emph{Theory of Probability and its Applications}, 16, 1968.

\bibitem[Vapnik(1991)]{vapnik1991principles}
Vladimir Vapnik.
\newblock Principles of risk minimization for learning theory.
\newblock \emph{Advances in neural information processing systems}, 4, 1991.

\bibitem[Vapnik(1999)]{vapnik1999nature}
Vladimir Vapnik.
\newblock \emph{The nature of statistical learning theory}.
\newblock Springer science \& business media, 1999.

\bibitem[Vaswani et~al.(2017)Vaswani, Shazeer, Parmar, Uszkoreit, Jones, Gomez,
  Kaiser, and Polosukhin]{vaswani2017attention}
Ashish Vaswani, Noam Shazeer, Niki Parmar, Jakob Uszkoreit, Llion Jones,
  Aidan~N Gomez, \L~ukasz Kaiser, and Illia Polosukhin.
\newblock Attention is all you need.
\newblock \emph{Advances in Neural Information Processing Systems}, 2017.

\bibitem[Veeriah et~al.(2019)Veeriah, Hessel, Xu, Lewis, Rajendran, Oh, van
  Hasselt, Silver, and Singh]{veeriah2019discovery}
Vivek Veeriah, Matteo Hessel, Zhongwen Xu, Richard Lewis, Janarthanan
  Rajendran, Junhyuk Oh, Hado van Hasselt, David Silver, and Satinder Singh.
\newblock Discovery of useful questions as auxiliary tasks.
\newblock In \emph{Advances in Neural Information Processing Systems}, 2019.

\bibitem[Virtanen et~al.(2020)Virtanen, Gommers, Oliphant, Haberland, Reddy,
  Cournapeau, Burovski, Peterson, Weckesser, Bright, {van der Walt}, Brett,
  Wilson, Millman, Mayorov, Nelson, Jones, Kern, Larson, Carey, Polat, Feng,
  Moore, {VanderPlas}, Laxalde, Perktold, Cimrman, Henriksen, Quintero, Harris,
  Archibald, Ribeiro, Pedregosa, {van Mulbregt}, and {SciPy 1.0
  Contributors}]{2020SciPy}
Pauli Virtanen, Ralf Gommers, Travis~E. Oliphant, Matt Haberland, Tyler Reddy,
  David Cournapeau, Evgeni Burovski, Pearu Peterson, Warren Weckesser, Jonathan
  Bright, St{\'e}fan~J. {van der Walt}, Matthew Brett, Joshua Wilson, K.~Jarrod
  Millman, Nikolay Mayorov, Andrew R.~J. Nelson, Eric Jones, Robert Kern, Eric
  Larson, C~J Carey, {\.I}lhan Polat, Yu~Feng, Eric~W. Moore, Jake
  {VanderPlas}, Denis Laxalde, Josef Perktold, Robert Cimrman, Ian Henriksen,
  E.~A. Quintero, Charles~R. Harris, Anne~M. Archibald, Ant{\^o}nio~H. Ribeiro,
  Fabian Pedregosa, Paul {van Mulbregt}, and {SciPy 1.0 Contributors}.
\newblock {{SciPy} 1.0: Fundamental Algorithms for Scientific Computing in
  Python}.
\newblock \emph{Nature Methods}, 17:\penalty0 261--272, 2020.

\bibitem[Wang et~al.(2021)Wang, Lyle, and Kwiatkowska]{wang2021provable}
B~Wang, C~Lyle, and M~Kwiatkowska.
\newblock Provable guarantees on the robustness of decision rules to causal
  interventions.
\newblock In \emph{Proceedings of the International Joint Conference on
  Artificial Intelligence}, 2021.

\bibitem[Wang et~al.(2020)Wang, Wang, Du, and Krishnamurthy]{wang2020optimism}
Yining Wang, Ruosong Wang, Simon~Shaolei Du, and Akshay Krishnamurthy.
\newblock Optimism in reinforcement learning with generalized linear function
  approximation.
\newblock In \emph{International Conference on Learning Representations
  (ICLR)}, 2020.

\bibitem[Watkins and Dayan(1992)]{watkins1992q}
Christopher~JCH Watkins and Peter Dayan.
\newblock Q-learning.
\newblock \emph{Machine Learning (ICML)}, 8\penalty0 (3-4):\penalty0 279--292,
  1992.

\bibitem[Wei and Ma(2019)]{wei2019improved}
Colin Wei and Tengyu Ma.
\newblock Improved sample complexities for deep neural networks and robust
  classification via an all-layer margin.
\newblock In \emph{International Conference on Learning Representations
  (ICLR)}, 2019.

\bibitem[Welling and Teh(2011)]{welling2011bayesian}
Max Welling and Yee~W Teh.
\newblock Bayesian learning via stochastic gradient langevin dynamics.
\newblock In \emph{Proceedings of the 28th International Conference on Machine
  Learning (ICML-11)}, pages 681--688, 2011.

\bibitem[Wilson and Izmailov(2020)]{wilson2020bayesian}
Andrew~Gordon Wilson and Pavel Izmailov.
\newblock Bayesian deep learning and a probabilistic perspective of
  generalization.
\newblock \emph{Advances in neural information processing systems},
  33:\penalty0 4697--4708, 2020.

\bibitem[Wood and Shawe-Taylor(1996)]{Wood:ShaweTaylor:1996}
Jeffrey Wood and John Shawe-Taylor.
\newblock Representation theory and invariant neural networks.
\newblock \emph{Discrete Applied Mathematics}, 69\penalty0 (1):\penalty0
  33--60, 1996.

\bibitem[Wu et~al.(2020)Wu, Zhang, Valiant, and Re]{wu2020on}
Sen Wu, Hongyang Zhang, Gregory Valiant, and Christopher Re.
\newblock On the generalization effects of linear transformations in data
  augmentation.
\newblock In Hal~Daumé III and Aarti Singh, editors, \emph{Proceedings of the
  37th International Conference on Machine Learning (ICML)}, volume 119 of
  \emph{Proceedings of Machine Learning Research}, pages 10410--10420. PMLR,
  13--18 Jul 2020.

\bibitem[Yang and Schoenholz(2017)]{yang2017mean}
Ge~Yang and Samuel Schoenholz.
\newblock Mean field residual networks: On the edge of chaos.
\newblock In I.~Guyon, U.~Von Luxburg, S.~Bengio, H.~Wallach, R.~Fergus,
  S.~Vishwanathan, and R.~Garnett, editors, \emph{Advances in Neural
  Information Processing Systems}, volume~30. Curran Associates, Inc., 2017.

\bibitem[Yang et~al.(2022)Yang, Ajay, and Agrawal]{yang2022overcoming}
Ge~Yang, Anurag Ajay, and Pulkit Agrawal.
\newblock Overcoming the spectral bias of neural value approximation.
\newblock In \emph{International Conference on Learning Representations
  (ICLR)}, 2022.

\bibitem[Yang and Hu(2021)]{yang2021tensor}
Greg Yang and Edward~J Hu.
\newblock Tensor programs iv: Feature learning in infinite-width neural
  networks.
\newblock In \emph{Proceedings of the 38th International Conference on Machine
  Learning (ICML)}, pages 11727--11737. PMLR, 2021.

\bibitem[Yarats and Kostrikov(2020)]{pytorch_sac}
Denis Yarats and Ilya Kostrikov.
\newblock Soft actor-critic (sac) implementation in pytorch.
\newblock \url{https://github.com/denisyarats/pytorch_sac}, 2020.

\bibitem[Ye and Lim(2016)]{ye2016schubert}
Ke~Ye and Lek-Heng Lim.
\newblock Schubert varieties and distances between subspaces of different
  dimensions.
\newblock \emph{SIAM Journal on Matrix Analysis and Applications}, 37\penalty0
  (3):\penalty0 1176--1197, 2016.

\bibitem[Yin et~al.(2019)Yin, Tucker, Zhou, Levine, and Finn]{yin2019meta}
Mingzhang Yin, George Tucker, Mingyuan Zhou, Sergey Levine, and Chelsea Finn.
\newblock Meta-learning without memorization.
\newblock \emph{International Conference on Learning Representations (ICLR)},
  2019.

\bibitem[Zaidi et~al.(2022)Zaidi, Berariu, Kim, Bornschein, Clopath, Teh, and
  Pascanu]{zaidi2022does}
Sheheryar Zaidi, Tudor Berariu, Hyunjik Kim, J{\"o}rg Bornschein, Claudia
  Clopath, Yee~Whye Teh, and Razvan Pascanu.
\newblock When does re-initialization work?
\newblock \emph{arXiv preprint arXiv:2206.10011}, 2022.

\bibitem[Zhang et~al.(2018{\natexlab{a}})Zhang, Ballas, and
  Pineau]{azhang2018genrl}
Amy Zhang, Nicolas Ballas, and Joelle Pineau.
\newblock A dissection of overfitting and generalization in continuous
  reinforcement learning.
\newblock \emph{arXiv preprint}, 2018{\natexlab{a}}.

\bibitem[Zhang et~al.(2018{\natexlab{b}})Zhang, Wu, and
  Pineau]{azhang2018natrl}
Amy Zhang, Yuxin Wu, and Joelle Pineau.
\newblock Natural environment benchmarks for reinforcement learning.
\newblock \emph{arXiv preprint}, abs/1811.06032, 2018{\natexlab{b}}.

\bibitem[Zhang et~al.(2019)Zhang, Lipton, Pineda, Azizzadenesheli, Anandkumar,
  Itti, Pineau, and Furlanello]{zhang2019causal}
Amy Zhang, Zachary~C. Lipton, Luis Pineda, Kamyar Azizzadenesheli, Anima
  Anandkumar, Laurent Itti, Joelle Pineau, and Tommaso Furlanello.
\newblock Learning causal state representations of partially observable
  environments.
\newblock \emph{Reinforcement Learning and Decision Making}, 2019.

\bibitem[Zhang et~al.(2020{\natexlab{a}})Zhang, Lyle, Sodhani, Filos,
  Kwiatkowska, Pineau, Gal, and Precup]{zhang2020invariant}
Amy Zhang, Clare Lyle, Shagun Sodhani, Angelos Filos, Marta Kwiatkowska, Joelle
  Pineau, Yarin Gal, and Doina Precup.
\newblock Invariant causal prediction for block {MDP}s.
\newblock In \emph{Proceedings of the 37th International Conference on Machine
  Learning (ICML)}, 2020{\natexlab{a}}.

\bibitem[Zhang et~al.(2020{\natexlab{b}})Zhang, McAllister, Calandra, Gal, and
  Levine]{zhang2020learning}
Amy Zhang, Rowan~Thomas McAllister, Roberto Calandra, Yarin Gal, and Sergey
  Levine.
\newblock Learning invariant representations for reinforcement learning without
  reconstruction.
\newblock In \emph{International Conference on Learning Representations
  (ICLR)}, 2020{\natexlab{b}}.

\bibitem[Zhang et~al.(2017)Zhang, Bengio, Hardt, Recht, and
  Vinyals]{zhang2016understanding}
Chiyuan Zhang, Samy Bengio, Moritz Hardt, Benjamin Recht, and Oriol Vinyals.
\newblock Understanding deep learning requires rethinking generalization.
\newblock In \emph{International Conference on Learning Representations
  (ICLR)}, 2017.

\bibitem[Zhang et~al.(2018{\natexlab{c}})Zhang, Vinyals, Munos, and
  Bengio]{zhang2018study}
Chiyuan Zhang, Oriol Vinyals, Remi Munos, and Samy Bengio.
\newblock A study on overfitting in deep reinforcement learning.
\newblock \emph{arXiv preprint}, 2018{\natexlab{c}}.

\bibitem[{Zhao} et~al.(2018){Zhao}, {Wang}, {Yan}, {Mao}, {Shen}, and
  {Wang}]{machinehealth}
R.~{Zhao}, D.~{Wang}, R.~{Yan}, K.~{Mao}, F.~{Shen}, and J.~{Wang}.
\newblock Machine health monitoring using local feature-based gated recurrent
  unit networks.
\newblock \emph{IEEE Transactions on Industrial Electronics}, 65\penalty0
  (2):\penalty0 1539--1548, Feb 2018.

\bibitem[Zhou and Troyanskaya(2015)]{zhou2015predicting}
Jian Zhou and Olga~G Troyanskaya.
\newblock Predicting effects of noncoding variants with deep learning-based
  sequence model.
\newblock \emph{Nature Methods}, 12\penalty0 (10):\penalty0 931, 2015.

\bibitem[Zhou et~al.(2019)Zhou, Veitch, Austern, Adams, and
  Orbanz]{zhou2018nonvacuous}
Wenda Zhou, Victor Veitch, Morgane Austern, Ryan~P. Adams, and Peter Orbanz.
\newblock Non-vacuous generalization bounds at the imagenet scale: a
  {PAC-B}ayesian compression approach.
\newblock \emph{International Conference on Learning Representations (ICLR)},
  2019.

\bibitem[Zilly et~al.(2021)Zilly, Achille, Censi, and Frazzoli]{zilly2021on}
Julian~G. Zilly, Alessandro Achille, Andrea Censi, and Emilio Frazzoli.
\newblock On plasticity, invariance, and mutually frozen weights in sequential
  task learning (icml).
\newblock In \emph{Advances in Neural Information Processing Systems}, 2021.

\bibitem[Zintgraf et~al.(2019)Zintgraf, Shiarli, Kurin, Hofmann, and
  Whiteson]{zintgraf2019fast}
Luisa Zintgraf, Kyriacos Shiarli, Vitaly Kurin, Katja Hofmann, and Shimon
  Whiteson.
\newblock Fast context adaptation via meta-learning (icml).
\newblock In \emph{International Conference on Machine Learning (ICML)}, pages
  7693--7702. PMLR, 2019.

\bibitem[Zou et~al.(2012)Zou, Zhu, Yu, and Ng]{zou2012deep}
Will Zou, Shenghuo Zhu, Kai Yu, and Andrew~Y Ng.
\newblock Deep learning of invariant features via simulated fixations in video.
\newblock In \emph{Advances in Neural Information Processing Systems}, pages
  3203--3211, 2012.

\end{thebibliography}
\end{document}